\documentclass[12pt,twoside]{mitthesis}
\usepackage{lgrind}
\usepackage{cmap}
\usepackage[T1]{fontenc}
\usepackage{amsmath,amssymb,amsthm}
\usepackage[numbers]{natbib}
\usepackage{textcomp}
\usepackage{amsmath}
\usepackage{url}
\usepackage{amssymb}
\usepackage{amsbsy}
\usepackage{hyperref}
\usepackage{mathrsfs}
\usepackage{mathtools}
\usepackage{algorithm}
\usepackage{algorithmic}
\usepackage{tabularx}
\usepackage{enumitem}
\usepackage{mathtools}

\usepackage{xcolor}

\usepackage{microtype}
\usepackage{graphicx}
\usepackage{booktabs} 

\usepackage{mathtools}
\usepackage{url}
\usepackage{amsfonts,bm}

\usepackage{lastpage}
\usepackage{newfloat}
\usepackage{subcaption}
\usepackage{bbm}

\newcommand{\eqn}[1]{\begin{equation}\begin{split} #1 \end{split}\end{equation}}

\newcommand{\R}{\mathbb{R}}
\newcommand{\T}{\mathbfcal{T}|}
\newcommand{\Z}{\mathbb{Z}}

\newcommand{\Q}{\mathbb{Q}}
\newcommand{\E}{\mathbb{E}}
\newcommand{\h}{\hbar}
\newcommand{\p}[2]{\frac{\partial #1}{\partial #2}}

\def\ie{{\frenchspacing\it i.e.}}
\def\eg{{\frenchspacing\it e.g.}}
\def\etc{{\frenchspacing\it etc.}}

\def\expec#1{\langle#1\rangle}

\def\B{\textbf{B}}

\def\D{D}
\def\E{\textbf{E}}
\def\F{\textbf{F}}

\def\J{\textbf{J}}
\def\L{\textbf{L}}
\def\Ell{\mathcal{L}}

\def\T{\mathcal{T}}
\def\ro{{\rho}}
\def\vtheta{{\boldsymbol{\theta}}}
\def\vphi{{\boldsymbol{\phi}}}

\def\m{\boldsymbol{\mu}}

\def\b{\textbf{b}}
\def\c{{\textbf{c}}}

\def\f{{\textbf{f}}}
\def\g{\textbf{g}}
\def\h{\textbf{h}}
\def\l{\ell}
\def\p{\textbf{p}}

\def\u{\textbf{u}}

\DeclareMathAlphabet\mathbfcal{OMS}{cmsy}{b}{n}
\DeclarePairedDelimiter{\norm}{\lVert}{\rVert}

\def\x{\textbf{x}}

\def\y{\textbf{y}}

\DeclareMathOperator*{\argmax}{arg\,max}
\DeclareMathOperator*{\argmin}{arg\,min}

 \def\logplus{\log_+}
\def\DL{{\rm DL}}

\newtheorem{theorem}{Theorem}
\newtheorem{corollary}{Corollary}[theorem]
\newtheorem{lemma}{Lemma}[theorem]
\newtheorem{definition}{Def.}
\newtheorem{assumption}{Assumption}

\def\spose#1{\hbox to 0pt{#1\hss}}
\def\simlt{\mathrel{\spose{\lower 3pt\hbox{$\mathchar"218$}}
   \raise 2.0pt\hbox{$\mathchar"13C$}}}
\def\simgt{\mathrel{\spose{\lower 3pt\hbox{$\mathchar"218$}}
     \raise 2.0pt\hbox{$\mathchar"13E$}}}
 \def\simpropto{\mathrel{\spose{\lower 3pt\hbox{$\mathchar"218$}}
     \raise 2.0pt\hbox{$\propto$}}}

\def\eq#1{equation~(\ref{#1})}

\def\fig#1{Figure~\ref{#1}}
\def\Fig#1{Figure~\ref{#1}}

\def\Sec#1{Section~\ref{#1}}

\parskip=\medskipamount

\usepackage{booktabs}
\def\IB{\text{IB}}
\def\E{\mathbb{E}}
\def\log{\text{log}}
\def\D{\mathcal{D}}
\def\X{\mathcal{X}}
\def\KL{\operatorname{KL}}
\def\O{\mathcal{O}}
\usepackage{bbm}
\newcommand\numberthis{\addtocounter{equation}{1}\tag{\theequation}}

\usepackage[normalem]{ulem}
\usepackage{times}

\def\all{all}
\ifx\files\all  \else 

\usepackage{amsfonts}       
\usepackage{nicefrac}       
\usepackage{microtype}      

\def\logplus{\log_+}
\def\E{\mathbb{E}}
\def\R{\mathbb{R}}

\newcommand\independent{\protect\mathpalette{\protect\independenT}{\perp}}
\def\independenT#1#2{\mathrel{\rlap{$#1#2$}\mkern2mu{#1#2}}}

\let\emptyset\varnothing
\newcommand{\rulesep}{\unskip\ \vrule\ }

\DeclareFloatingEnvironment[name={Figure S}]{suppfigure}
\DeclareFloatingEnvironment[name={Table S}]{supptable}

\DeclarePairedDelimiter{\ceil}{\lceil}{\rceil}

\def\IB{\text{IB}}
\def\IBB{\emph{IB}}

\def\betath{\beta^{(\text{th})}}
\def\betanew{\beta^{(\text{new})}}
\def\pstar{p^*_\beta(z|x)}

\def\thetaa{{\boldsymbol{\theta}}}
\def\pthe{p_\thetaa(z|x)}
\def\Dthe{{\Delta\thetaa}}
\def\partialthe{{\frac{\partial}{\partial\thetaa}}}
\def\E{\mathbb{E}}
\def\log{\text{log}}
\def\logg{\emph{log}}
\def\e{\epsilon}
\def\D{\mathcal{D}}
\def\G{\mathcal{G}}
\def\KL{\operatorname{KL}}
\def\O{\mathcal{O}}
\def\Q{\mathcal{Q}}
\def\R{\mathbb{R}}
\def\S{\mathcal{S}}
\def\titlemath#1{\texorpdfstring{#1}{Lg}}
\def\X{\mathcal{X}}
\def\XX{\mathcal{X}}
\def\Y{\mathcal{Y}}
\def\Z{\mathcal{Z}}
\def\QQ{\mathcal{Q}_{\mathcal{Z}|\mathcal{X}}}
\def\QQA{\mathcal{Q}^{(0)}_{\mathcal{Z}|\mathcal{X}}}
\def\QQB{\mathcal{Q}^{(1)}_{\mathcal{Z}|\mathcal{X}}}
\def\QQC{\mathcal{Q}^{(2)}_{\mathcal{Z}|\mathcal{X}}}
\def\QQD{\mathcal{Q}^{(3)}_{\mathcal{Z}|\mathcal{X}}}
\def\QQf{\mathcal{Q}^{(f)}_{\mathcal{X}|\mathcal{Z}}}

\def\beq#1{\begin{equation}\label{#1}}
\def\eeq{\end{equation}}
\def\beqa{\begin{equation*}\begin{aligned}}
\def\eeqa{\end{aligned}\end{equation*}}

\def\D{{\bf D}}

\def\g{{\bf g}}
\def\h{{\bf h}}
\def\J{{\bf J}}
\def\m{{\bf m}}
\def\f{{\bf f}}
\def\x{{\bf x}}
\def\y{{\bf y}}
\def\yhat{\hat{\bf y}}
\def\z{{\bf z}}
\def\XX{\mathbf{X}}
\def\YY{\mathbf{Y}}

\def\gammavec{\boldsymbol\gamma}
\def\phivec{\boldsymbol\theta}
\def\muvec{\boldsymbol\mu}

\def\sout{s_{\rm out}}
\def\spool{s_{\rm pool}}
\def\slat{s_{\rm code}}

\usepackage{dblfloatfix} 

\newcommand{\estimates}{\overset{\scriptscriptstyle\wedge}{=}}
\newcommand{\indicator}[1]{\mathbbm{1}[[#1]]}

\makeatletter
\newcommand{\customlabel}[2]{%
  \protected@write \@auxout {}{\string \newlabel {#1}{{#2}{\thepage}{#2}{#1}{}} }%
  \hypertarget{#1}{#2}
}
\makeatother

\def\abovestrut#1{\rule[0in]{0in}{#1}\ignorespaces}
\def\belowstrut#1{\rule[-#1]{0in}{#1}\ignorespaces}

\usepackage{enumitem}
\newlist{legal}{enumerate}{10}
\setlist[legal]{label={(\arabic*)}}

\def\a{{\bf a}}
\def\b{{\bf b}}

\def\h{{\bf h}}

\def\dKL{d_{\rm{KL}}}

\def\Ell{{\mathcal L}}

\def\p{{\bf p}}
\def\pbar{\bar{p}}

\def\u{{\bf u}}

\def\what{{\hat w}}
\def\x{{\bf x}}
\def\y{{\bf y}}
\def\z{{\bf z}}

\def\B{{\bf B}}

\def\D{{\bf D}}
\def\E{\mathbb{E}}
\def\F{{\bf F}}

\def\P{{\bf P}}

\def\Q{{\bf Q}}
\def\R{{\bf R}}

\def\S{{\bf S}}

\def\T{{\bf T}}

\def\eqn#1{~(\ref{#1})}
\def\beqa#1{\begin{eqnarray}\label{#1}}
\def\eeqa{\end{eqnarray}}

\def\expec#1{\langle#1\rangle}

\renewcommand*{\citet}{\cite}
\renewcommand*{\citep}{\cite}

\begin{document}

\title{Intelligence, physics and information -- the tradeoff between accuracy and simplicity in machine learning}

\author{Tailin Wu}
\prevdegrees{Bachelor of Science in Physics, Peking University 2012}
\department{Department of Physics}

\degree{Doctor of Philosophy in Physics}

\degreemonth{February}
\degreeyear{2020}
\thesisdate{November 27, 2019}


\supervisor{Isaac L. Chuang}{Professor of Physics}
\supervisor{Max Tegmark}{Professor of Physics}

\chairman{Nergis Mavalvala}{Associate Department Head of Physics}

\maketitle



\cleardoublepage
\setcounter{savepage}{\thepage}
\begin{abstractpage}
%
%
%

How can we make machines more intelligent, so that they can make sense of the world, and become better at learning? To approach this goal, I believe that viewing intelligence in terms of many integral aspects, and also in terms of a universal two-term tradeoff between task performance and complexity, provides two feasible perspectives, and physics and information should play central underlying roles. In this thesis, I address several key questions in some aspects of intelligence, and study the phase transitions in the two-term tradeoff, using strategies and tools from physics and information.  

Firstly, how can we make the learning models more flexible and efficient, so that agents can learn quickly with fewer examples? Inspired by how physicists model the world, we approach this question by introducing a paradigm and an Artificial Intelligence Physicist (AI Physicist) agent for simultaneously learning many small specialized models (theories) and the domain they are accurate, which can then be simplified, unified and stored, facilitating few-shot learning in a continual way. We also introduce a Meta-Learning Autoencoder architecture, which utilizes learning good task representations to facilitate few-shot learning. 

Secondly, for representation learning, when can we learn a good representation, and how does learning depend on the structure of the dataset? We approach this question by studying phase transitions in the two-term tradeoff: the hyperparameter\footnote{The parameter whose value is set before the learning begins, e.g. relative strength between task performance and complexity in the learning objective.} setting where key quantities, e.g. prediction accuracy change in a discontinuous way. 
We introduce a technique for predicting when the second-order phase transitions will occur, and in the information bottleneck objective, we show that the formulas we derive reveal deep connections between the data, the model, the learned representation, and the loss landscape of the objective function. For example, each phase transition corresponds to learning a new component of nonlinear maximum correlation between the input and the target, and in classification, they correspond to the learning of new classes.

Thirdly, how can agents discover causality from observations? We address part of this question by introducing an algorithm that combines prediction and minimizing information from the input, for exploratory causal discovery from observational time series. 

Last but not least, how can we make classifiers more robust to noise? In the presence of label noise, we introduce Rank Pruning, a robust and general algorithm for classification with noisy labels, prove its consistency, and improve state-of-the-art of learning with noisy labels. 

I believe that building on the work of my thesis we will be one step closer to enable more intelligent machines that can make sense of the world.

\end{abstractpage}


\cleardoublepage

\section*{Acknowledgments}

Firstly, I would like to express my sincere gratitude to Professor Isaac Chuang, my thesis advisor. He encourages me to be a leader in the direction I work on, and his strict criteria help me form my own high standards to strive to realize my full potential. His keen insights in various aspects of physics and computer science have provided valuable guidance in my research. Moreover, he has provided valuable support and guidance in several key moments in my PhD career. Specifically, I would like to thank him for his support when I transformed my research direction from quantum computing experiment to machine learning in my mid-PhD; his introduction of me to a collaboration with Curtis Northcutt which helps me jump start my experience in machine learning; his introduction of me to a collaboration with Ian Fischer from Google, without which I would not have discovered the exciting intersection area between representation learning and phase transitions; and his insightful advice on graduating with momentum and job search. His support and guidance prove pivotal in hindsight, which I am deeply grateful.

Secondly, I would like to express my sincere thanks to Professor Max Tegmark, my thesis co-advisor. I'm grateful for having been working with Max, which greatly expanded my perspective and horizon in AI. As a physicist, Max's strive for simplicity and pursuit of  ``intelligible intelligence'' resonate with me, and has inspired my initial idea of AI Physicist and discovery of the learnability phase transition in the Information Bottleneck. Max has sharp intuitions, and his first-principles thinking stands out, which has influenced me to also think from first principles. I enjoy our close collaboration, as well as discussing universe, intelligence, life.

Thirdly, I would like to thank Ian Fischer. We begin collaborating since the end of 2018, and Ian was also my host for my internship at Google in the summer of 2019. Ian introduced me to the powerful perspective of viewing many machine learning problems in terms of information, which constitutes a major perspective that underlies my thesis. He has guided me with his rich experience in the Information Bottleneck and representation learning, and I also learn a great deal from his scalable numerical experimental skills. I enjoyed our extensive discussions everyday in the January of 2019, when we analyzed the Information Bottleneck together. I enjoyed a lot our close collaboration during my internship at Google, where we pushed the frontiers on various aspects of the Information Bottleneck.

I would like to thank my other thesis committee members: Professor Marin Solja\v{c}i\'c and Professor Riccardo Comin, who have provided valuable guidance and feedbacks in my thesis preparation and advice on career. I would also like to thank my academic advisor Professor Wolfgang Ketterle, who has provided insightful advice throughout my PhD.

I would like to thank my other collaborators: Curtis Northcutt, who I had a great time working with, with our complementary skills we really form a good team; Thomas Breuel and Jan Kautzin, who guided me during my internship at NVIDIA; Michael Skuhersky, with whom I am having fun with the \emph{C.elegans}.

I would also like to thank many other people who I have discussed research with: Professor Josh Tenenbaum, Alex Alemi, Kevin Murphy, who have provided valuable feedbacks on the research projects.

Last but not least, I would like to thank my parents, who have given me the chance to experience this wonderful Universe, and have given me unconditional support all the time. I love you!


\pagestyle{plain}
\tableofcontents
\newpage
\listoffigures
\newpage
\listoftables

\chapter{Introduction}
\label{chap1:introduction}

\section{What is intelligence?}
\label{sec:what_is_intelligence}

What is intelligence? This is a question that has been intriguing generations of scientists and artificial intelligence (AI) researchers. The understanding and building of intelligence not only is in itself an important question, but also has profound influence on the society, with the potential to help solve important problems. For example, it has helped scientists predict protein structure from genomic data with unprecedented precision \cite{alphafold}, simulate light scattering by multilayer nanoparticles and facilitate nanophotonic inverse design \cite{peurifoy2018nanophotonic}, predict molecule properties in quantum chemistry \cite{gilmer2017neural}, identify central nervous system tumours \cite{capper2018dna}, reconstruct neural circuit map \cite{januszewski2017high}, etc. But I believe we have just seen a tip of the iceberg of what it may achieve. Future intelligent machines may help solve important problems that benefit humanity as a whole, for example, design more efficient ways for space exploration, automatically discover new physics and new mathematics, identify mechanisms in biological systems and propose cure for disease, 
to name just a few. 

Physics has been extending its scope of study from space and time to matter and energy in all its forms, from the subatomic to the cosmological, from the elementary to the complex, and from the inanimate to living organisms. The study of intelligence is another frontier physics should and can illuminate. The wisdom and strategies developed by generations of physicists may help better understand and build intelligence, just as many of the techniques and perspectives in AI domain are inspired by physics, e.g. energy models \cite{lecun2006tutorial}, simulated annealing \cite{kirkpatrick1983optimization,vcerny1985thermodynamical}, Hamiltonian Monte Carlo (HMC) \cite{DUANE1987216}, critical behaviors in random Boolean expressions \cite{Kirkpatrick1297} and  data representations \cite{cubero2019statistical}, neural ordinary differential equations \cite{chen2018neural}, fluctuation-dissipation relations for stochastic gradient descent \cite{yaida2018fluctuationdissipation}, to name just a few. In turn, the better intelligent algorithms can help solve important physics problems, e.g. quantum state reconstruction \cite{carrasquilla2019reconstructing}, phase transitions \cite{carrasquilla2017machine,wang2016discovering,van2017learning}, planetary dynamics \cite{lam2018machine} and particle physics \cite{baldi2014searching}.

Before diving into studying it, it is important that we have a notion of what ``intelligence'' is.  As R. J. Sternberg has put it \cite{gregory1987oxford}, ``Viewed narrowly, there seem to be almost as many definitions of intelligence as there were experts asked to define it''. Still, there exist similarities among many of the definitions. Below I quote some of the prominent definitions:

\begin{legal}
\item ``Intelligence measures an agent's ability to achieve goals in a wide range of environments'' \citet{legg2007collection}.

\item ``The intelligence of a system is a measure of its skill-acquisition efficiency over a scope of tasks, with respect to priors, experience, and generalization difficulty'' \citet{chollet2019measure}. In this paper, the author argues that measuring skill by testing on the same kind of training task (``local generalization'') does not measure intelligence, since the skills can be bought by training with arbitrary number of examples or injecting prior knowledge by the developer, which should be on the orthogonal axis of intelligence. He then proposes a new measure of intelligence, which roughly translates to

$$\text{intelligence}:=\E\left[\frac{\text{skill} \times \text{(generalization difficulty)}}{\text{prior} + \text{experience}}\right]$$

where the generalization difficulty is defined as the ratio of the length of shortest program that solves the testing task given the shortest
program that achieves optimal training-time performance over the situations in the curriculum, over the length of the shortest program that solves the testing task. 

\item ``... the ability of a system to act appropriately in an uncertain environment, where appropriate action is that which increases the probability of success, and success is the achievement of behavioral subgoals that support the system's ultimate goal'' \citet{albus1991outline}.

\item ``Intelligence is the ability to use optimally limited resources - including time - to achieve goals'' \citet{kurzweil2000age}.

\item ``Humans (machines) are intelligent to the extent that our (their) actions can be expected to achieve our (their) objectives.'' \citet{russell2019humancompatible}.

\item ``Intelligence is the power to rapidly find an adequate solution in what appears \textit{a priori} (to observers) to be an immense search space'' \citet{lenat1992thresholds}.

\item ``Intelligence is the ability to process information properly in a complex environment. The criteria of properness are not predefined and hence not available beforehand. They are acquired as a result of the information processing.'' \citet{nakashima1999ai}.

\item ``Intelligence means getting better over time'' \citet{schank1991s}.

\item ``Intellignence = Ability to accomplish complex goals''
 \citet{tegmark2017life}.
 
\item ``Intelligence is the ability for an information processing system to adapt to its environment with insufficient knowledge and resources'' \citet{wang1995working}.

\item ``Intelligence is not a single, unitary ability, but rather a composite of several functions. The term denotes that combination of abilities required for survival and advancement within a particular culture'' \citet{anastasi1992counselors}.

\item ``...in its lowest terms intelligence is present where the individual animal, or human being, is aware, however dimly, of the relevance of his behaviour to an objective. Many definitions of what is indefinable have been attempted by psychologists, of which the least unsatisfactory are 1. the capacity to meet novel situations, or to learn to do so, by new adaptive responses and 2. the ability to perform tests or tasks, involving the grasping of relationships, the degree of intelligence being proportional to the complexity, or the abstractness, or both, of the relationship'' \citet{colman2015dictionary}.

\item ``Intelligence is assimilation to the extent that it incorporates all the given data of experience within its framework . . . There can be no doubt either, that mental life is also accommodation to the environment. Assimilation can never be pure because by incorporating new elements into its earlier schemata the intelligence constantly modifies the latter in order to adjust them to new elements.'' \citet{piaget2005psychology}.

\item `` ... certain set of cognitive capacities that enable an individual to adapt and thrive in any given environment they find themselves in, and those cognitive capacities include things like memory and retrieval, and problem solving and so forth. There's a cluster of cognitive abilities that lead to successful adaptation to a wide range of environments'' \citet{simonton2003interview}.
\end{legal}

We see that although the definitions vary, there exist similar aspects of the definitions. For example, (1) intelligence involves an agent's interaction with the \emph{environment}; (2) it is a property of the agent's \emph{information processing}; (3) it involves the agents ability to \emph{achieve goals} or \emph{solve tasks} in diverse environments, under environmental constraints (e.g. incomplete information, limited resources of computation, space or time).

To better understand intelligence, I believe the following two perspectives provide feasible routes. For the first perspective, instead of trying to directly understand and build intelligence  using a single formula or definition, I believe a better approach is to understand its different aspects. Just as in the understanding of life, although once doubting whether biological mechanisms could ever explain the property of being alive, biologists gradually uncovered aspects of life, for example metabolism, homeostasis and reproduction. Physicist Erwin Schr\"odinger also introduced the idea of an ``aperiodic crystal'' containing genetic information \cite{schrodinger1992life} that inspired the discovery of DNA. I think intelligence is a \emph{system}. Just as a human is a system, with its body, hands, sensory organisms, different parts of the brain (e.g. visual cortex, hippocampus, prefrontal cortex, etc.), intelligence should also have  many aspects that work integrally to form intelligent behavior.  By identifying and understanding the various aspects of intelligence (which we expand on in section \ref{sec:different_aspects}) with humans as a blueprint, we may have a feasible route to fully understand intelligence.

The second perspective is that, inspired by the above common aspect in definitions of intelligence that agent can solve tasks in diverse environments under  environmental constraints, I view intelligence as a result of the ability to solve a universal two-term trade-off in diverse settings, which may be a manifestation and approximation of the AIXI formalism \cite{legg2007universal} under different scenarios, and resonant with the notion of ``intelligence as compression'' \cite{legg2007universal}. In the two-term trade-off, we have one term that measures the agent's performance on the task, and another term that constrains resources of the agent, usually in the form of limiting the complexity of the learned model or representation. The agent has to find the best solution to a task under such constraints. In machine learning, this usually comes in the form of regularization.

In the above two perspectives, physics and information have the potential to play a central role. Since intelligence involves an agent's interaction with the environment, in order to achieve a diverse set of goals in the environment, the agent has to first understand the \emph{physics} of the environment, so as to utilize it and direct the environment to its need, just as humans cannot build rockets without understanding Newton's law of universal gravitation. To build intelligent agents, we can get inspirations from how generations of physicists study and model the different aspects of the world, from space and time to matter and energy, from the subatomic to the cosmological, from the elementary to the complex, and from the inanimate to living organisms. Therefore, the different aspects of intelligence, viewed from the first perspective, should benefit a lot from borrowing strategies and techniques from physics.

The second perspective, viewing intelligence as the ability to solve a universal two-term tradeoff in diverse settings, is also resonant with the central goal of physics. We want to find physics theories that not only can predict parts of Nature precisely, but are also \emph{simple}, where the simplicity may be governed by Occam's razor principle, or measured by Kolmogorov complexity \cite{kolmogorov1963tables} or description length \cite{rissanen1983universal}, or their approximations.

Information, on the other hand, should also lie at the core of intelligence. As we have stated before, intelligence can be viewed as a specific form of information processing. Moreover, I believe that this specific form of information processing should be independent of the substrate on which it performs, be it a human brain, a computer, a coordinated group of people, some emergent plasma behavior on the sun, or even an imaginary scenario of billions of ants doing the exact same information processing as a brain does. It is \emph{structure} of the information processing that matters, not the substrate that performs it.\footnote{This understanding is also resonant with the notion that consciousness is independent of the substrate, and only depend of the information processing \cite{tegmark2017life}.} To understand the information processing that gives rise to intelligence, information theory\footnote{Information theory studies the quantification, storage, and communication of information. It was originally introduced by Claude Shannon \cite{shannon1948mathematical} and further developed by many others.}, and the variational techniques developed thereupon, provide an ideal language and technique\footnote{There may be other techniques that address other aspects of the information processing, besides information theory.} to measure, constrain and optimize the information flow, regardless of the architecture of the model, and the task the agent is solving. Therefore, information theory may provide a valuable tool in developing different aspects of intelligence, and may arise as a natural objective in the universal two-term tradeoff, in either measuring the task performance, or constraining the complexity of the model, as we shall see in the following two sections.

As a sidenote, physics and information are not unrelated. Instead, information plays important roles in many branches of physics. For example, thermodynamics involves the study of \emph{entropy}\footnote{which is also called ``self information''.} and its interplay with other thermodynamic quantities; a generalized version of second law of thermodynamics \cite{bera2017generalized} involves  mutual information; quantum information studies the information processing and channels in quantum scenarios; to name just a few. Moreover, simplicity as measured by Kolmogorov complexity has a close connection with entropy \cite{grunwald2004shannon}.

The above two perspectives of intelligence, and the central role physics and information play,  lay the foundation of my thesis. In the following two sections I will expand each perspective in more detail, while putting my thesis work into the picture.

\section{First perspective: aspects of intelligence}
\label{sec:different_aspects}

Regarding intelligence as a system with many integral aspects working together, just as a computer or a human can be thought of as a system, I believe the following aspects are integral to intelligence:

Basic level:
\begin{legal}
\setlength\itemsep{0em}
\item \textbf{Sensory inputs}: it should be able to receive sensory inputs from the outside world.
\item \textbf{Intrinsic reward}: it should be able to generate intrinsic reward from interactions with the environment, directing its action.
\item \textbf{Memory}: it should have some form of memory.
\item \textbf{Actuators}: it should be able to influence the environment by its actuator.
\end{legal}

High level:
\begin{legal}
\setlength\itemsep{0em}
\item \textbf{Working robustly}: it should be able to work robustly in noisy and complex real environments.
\item \textbf{Working intelligibly\footnote{Although human intelligence is often not inteligible, intelligibility is arguably a desirable trait for AI.}}: we should be able to understand its goal and how it makes its decisions.
\item \textbf{Learning good representations}: it should be able to learn good internal representations for the tasks it is assigned to.
\item \textbf{Communication}: it should be able to communicate with human or other agents with succinct language.
\item \textbf{Predicting the future}: it should be able to learn to predict the future from the past, allowing it to understand how the world works, and better plan its action and achieving its goal.
\item \textbf{Causal inference}: it should be able to think in terms of cause and effect, and infer causal relations from observations or by performing experiments.
\item \textbf{Planning}: it should be able to plan its action to achieve its goal. 
\item \textbf{Few-shot learning}: it should be able to learn to learn across tasks, allowing it to learn with few examples for novel tasks.
\item \textbf{Lifelong learning}: it should be able to continuously learn new things without forgetting the past learning. This includes lifelong learning of prediction models, and lifelong skill acquisition.
\item \textbf{Emotional intelligence}: it should have a theory of mind, and be able to perceive other agents' feeling within their frame of reference.
\item (Something we don't know yet)
\end{legal}

\begin{figure}[ht!]
\centerline{\includegraphics[width=1\columnwidth]{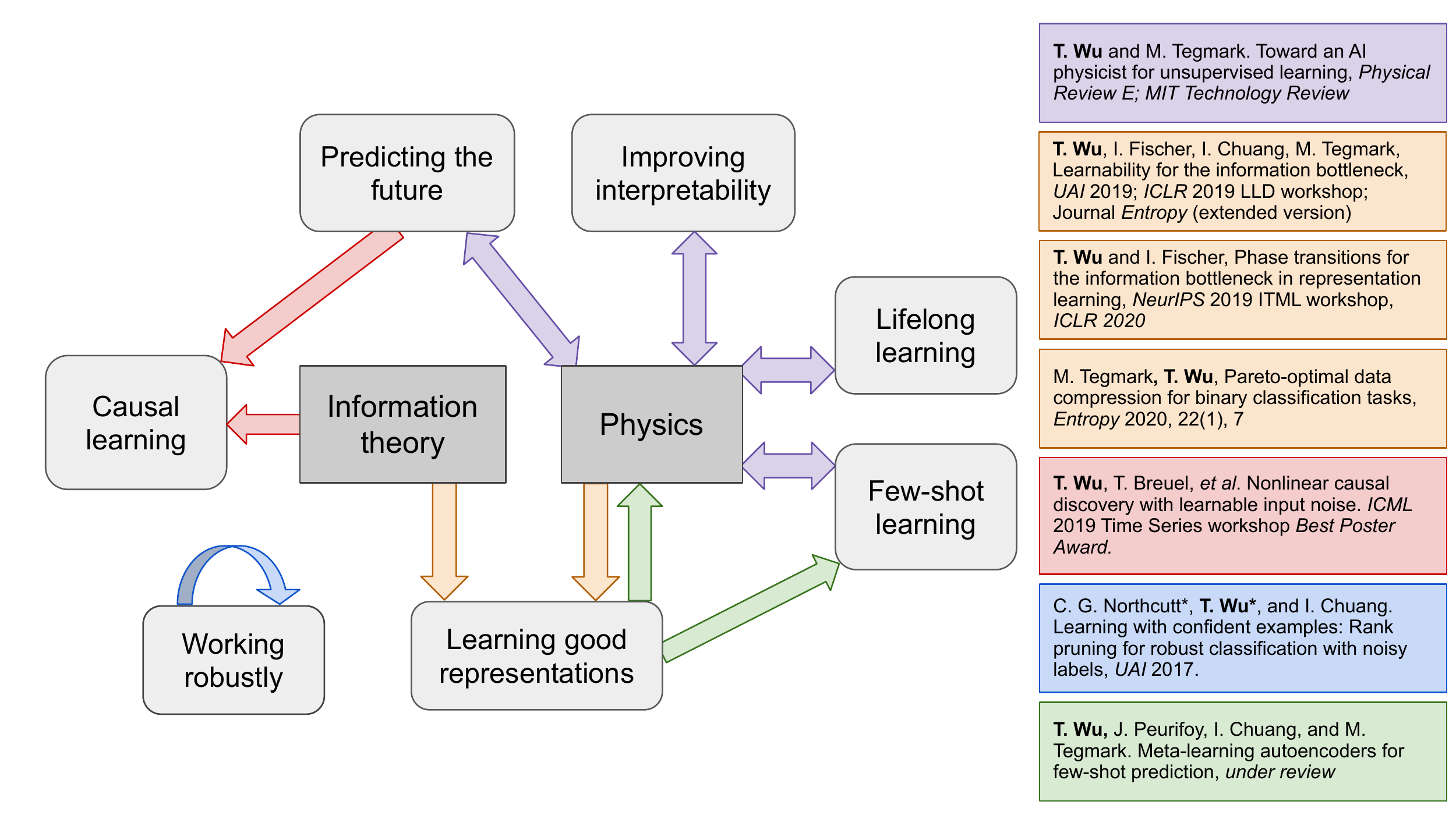}}
\caption{
How different parts of my thesis address the aspects of intelligence.
\label{fig:thesis_for_aspects}
}
\end{figure}

I believe a seamless integration of the above aspects in a single architecture will bring us closer to understanding intelligence. Moreover, the above aspects are not independent. Instead, doing well on one aspect will probably help some others. 

Based on the above understanding, I try in my thesis to combine aspects together if possible, and use one aspect to help another. My different thesis projects and their addressed aspects are shown in Fig. \ref{fig:thesis_for_aspects}. As we see, there are many connections between different aspects, and one aspect can help another. For example, the causal learning project (red) uses predicting the future for  causal learning; the meta-learning autoencoder project (green) applies learning good representations to achieve few-shot learning. 

Moreover, from Fig. \ref{fig:thesis_for_aspects}, we see that physics and information play a central role in connecting these different aspects. For example, the AI Physics project (purple) uses techniques inspired by physics, for predicting the future, improving interpretability, few-shot learning and lifelong learning aspects of intelligence, which in turn help the learning of physics theories. The information bottleneck projects (yellow) applies the analogy of phase transition in physics for understanding the phase transitions in representation learning, and whose objective uses information theory for representation learning.

\section{Second perspective: universal two-term trade-off}

As mentioned in Section \ref{sec:what_is_intelligence}, we may view intelligence as the ability to solve universal two-term trade-off in diverse settings. In fact, many machine learning objectives also have this format of two-term tradeoff, as follows:

\begin{equation}
\label{eq:universal_tradeoff}
L=\text{(Prediction loss)}+\beta\cdot\text{Complexity}
\end{equation}

For example, for unsupervised learning, the variational autoencoder \cite{kingma2013auto} has one term that encourages a small reconstruction loss, and another term that controls the KL-divergence between the posterior and the prior. For classification and regression, the Information Bottleneck \cite{tishby2000information} has one term that encourages the latent representation $Z$ to contain as much information about the target $Y$ as possible, and another term that controls the information contained in $Z$ about the input $X$. Similar formalisms can also be seen in InfoDropout \cite{achille2018emergence}, L1 regularizations, Kolmogorov structure functions  \cite{vereshchagin2004kolmogorov}, and Least Angle Regression \cite{efron2004least}, to name just a few.

In my thesis projects, this two-term tradeoff is also a recurring theme. Specifically,
\begin{legal}
\item AI Physicist (Chapter \ref{chap2:AI_physicist}): $$L=\DL(\text{data}|\text{model}) + \DL(\text{model})$$
\item Information Bottleneck (Chapter \ref{chap3:IB} and \ref{chap4:IB_phase_transition}): $$L=I(X;Z)-\beta\, I(Y;Z)$$
\item Pareto optimal data compression for binary classification (Chapter \ref{chap5:distillation}), equivalent to the following form: $$L=H(Z)-\beta\, I(Y;Z)$$
\item Causal learning with minimum mutual information regularization (Chapter \ref{chap6:causal}): $$L=\text{MSE}(f(\tilde{X});Y) + \lambda\, I(\tilde{X};X)$$
\end{legal}

From the expression of the two-term tradeoff (Eq. \ref{eq:universal_tradeoff}), we see that there is a hyperparameter $\beta$ that tunes the importance of the complexity constraint relative to the prediction loss. In one extreme, when the incentive to minimize complexity is significantly greater than that to minimize the prediction loss, we would expect that the optimum under the objective $L$ is a trivial solution with zero complexity. In the other extreme, we would expect that the optimum of the objective minimizes the prediction loss, but may use a very complicated model or representation, which may harm generalization and robustness \cite{shamir2010learning,alemi2016deep}. Between these two extremes, interesting phase transitions have been observed, where key quantities, e.g. prediction accuracy, change in a discontinuous way \cite{tishbyinfo,strouse2017information,chechik2005information,achille2018emergence,vereshchagin2004kolmogorov,rezende2018taming}. Usually, learning is at a specific $\beta$. Here, we are instead interested in understanding the full range of learning between the above two extremes, and specifically how the phase transitions depend on the structure of the task/dataset, the objective and the model. Moreover, there has been a long history of studying phase transition in physics. The intuition and strategies in physics may help the understanding of phase transitions in learning.

\section{Roadmap}
\subsection{Scope and structure of the thesis}
With the above motivation and perspectives in mind, this thesis tries to address some key questions in the above perspectives and aspects of intelligence. We do not attempt to solve intelligence in this thesis, since it is an extremely difficult and long-term endeavor that requires the efforts of researchers around the world, still needs many breakthroughs, and may be decades away \cite{muller2016future,grace2018will}. Instead, by laying out the perspectives and aspects, and addressing some key questions therein, we hope to make small steps towards a better understanding of the questions and improvements in answering these questions. Moreover, the lessons learned in answering those questions may help future efforts that build upon them.

As we have identified the different aspects of intelligence in Section \ref{sec:different_aspects}, each chapter of the thesis addresses a key question related to one or more aspects, or addresses the second perspective of the two-term tradeoff. We start with a simple question. Suppose that an agent, without knowing anything, experiences many environments each of which may have multiple physics laws, how can it make sense of the environments, and become better at learning over time? This is a simplified version of how a human or an agent may try to make sense of different environments, each of which may have parts of underlying models in common. Although it is a simplified version, it still poses a difficult challenge. Inspired by how physicists model the world, in Chapter \ref{chap2:AI_physicist} we introduce a learning paradigm  that learns and manipulates \emph{theories} as ``atoms'' of learning. We introduce four algorithms, differentiable-divide-and-conquer (DDAC), Occam's razor with minimum description length (MDL), unification and lifelong learning of theories, and integrate them into a simple AI Physicist agent, and demonstrate its capability in few-shot learning, lifelong learning, improving interpretability in a diverse set of prototypical physics environments.

Next, we study a key question in the universal two-term tradeoff, i.e. when we vary the hyperparameter $\beta$ that tunes the relative strength of controlling complexity and minimizing prediction loss, how do the phase transitions depend on the dataset, the objective and the model? We focus our attention on the Information Bottleneck (IB) objective, which provides a principled method to balance compression and prediction in representation learning via information. In Chapter \ref{chap3:IB}, we derive formulas that reveal how the learnability phase transition in IB depends on the structure of the dataset, which we analyze is determined by the \emph{conspicuous subset}, i.e. the most confident, typical and large, and imbalanced subset of the examples. The formulas also provide a tool to measure model capacity in a task-specific manner. We demonstrate that our theory and algorithm match closely with the experimentally observed onset of learning in mixture of Gaussian, MNIST and CIFAR10 datasets. In Chapter \ref{chap4:IB_phase_transition}, we generalize our approach to study the phase transitions for the full IB trade-off curve. We provide the first formula that gives the condition for the phase transitions in the most general setting, reveal that each phase transition is finding a (nonlinear) maximum correlation component between the input and target, orthogonal to the learned representation. We present an algorithm for discovering the phase transition points. We verify that our theory and algorithm accurately predict phase transitions in categorical datasets, predict the onset of learning new classes and class difficulty in MNIST, and predict prominent phase transitions.

In Chapter \ref{chap5:distillation}, we take a slightly different approach in studying the two-term tradeoff. In the scenario of the tradeoff between the $I(Y;Z)$ vs. $H(Z)$ in binary classification ($X$ is the input, $Y$ is the target, and $Z$ is a representation of $X$), we prove that we can use binning a uniformized and sorted histogram of $P(Y|X)$ to achieve the Pareto frontier of the $I(Y;Z)$ vs. $H(Z)$ tradeoff curve. We apply our technique to MNIST, FashionMNIST and CIFAR10 datasets, and  illustrate how it can be interpreted as an information-theoretically optimal image clustering algorithm.

An important aspect of intelligence is its ability to learn causal relations from observations. Different from mere associations, causal relations provide a succinct way to describe the underlying mechanism of the data-generating process. The learning of causal models is important especially under shifts in environments, where environmental factors may vary but the causal relations persist \cite{de2019causal}. Learning causal relations can also help the agent answer \emph{interventional} and \emph{counterfactual} questions, an important aspect of human reasoning. Furthermore, the understanding of causal relations between components of a system also constitutes an important aspect of scientific endeavor, including physics. In Chapter \ref{chap6:causal}, we ask the question: given multiple time series, without intervention, how can an agent discover the underlying causal relations? To address this question, we introduce an algorithm that combines prediction and minimizing information from the input, for exploratory causal discovery from observational time series, and demonstrate its effectiveness in synthetic, video game, breath rate vs. heart rate and \emph{C.elegans} datasets.

Suppose that an agent has experienced and  solved many similar tasks. How can it utilize its past learning to solve new tasks with few examples? In Chapter \ref{chap7:mela} we introduce the Meta-Learning Autoencoders (MeLA) architecture, which uses learning task representations for few-shot learning. When given a new task, its meta-recognition model maps the task into a task representation, and its meta-generative model maps the task representation into the weights and biases of a task-specific model. We demonstrate MeLA's effectiveness in physics prediction tasks, and show that it compares favorably with the state-of-the-art meta-learning algorithms.

Working robustly is another important aspect of intelligence. In the scenario of binary classification, where the labels are corrupted by an unknown noise process, how can a classifier still classify accurately as if the labels are not corrupted? In Chapter \ref{chap8:rankpruning}, we introduce Rank Pruning, a robust, time-efficient, general algorithm for both binary classification with noisy labels, and estimation of the fraction of mislabeling in the training sets. We prove that under certain assumptions, Rank Pruning achieves perfect noise estimation and equivalent expected risk as learning with correct labels, and provide closed-form solutions when those assumptions are relaxed. It improves the state-of-the-art of learning with noisy labels across F1 score, AUC-PR, and Error.

Finally, in Chapter \ref{chap9:conclusion}, I conclude the thesis, and provide prospects for future works, particularly detailing the specific directions for future work building on my thesis. This thesis is just a start, and countless opportunities lie ahead. I am excited to embark on the exciting journey of understanding intelligence and applying it for solving problems in society. 

\subsection{My contributions}

Chapters 2-8 correspond to separate published or submitted papers, presented as unchanged as possible. Since they report work done together with various co-authors, I will now detail my specific contributions.

For the AI Physicist project (Chapter \ref{chap2:AI_physicist}), I proposed the main architecture of the AI physicist agent, developed the differentiable divide-and-conquer, unification and lifelong learning aspects of the agent, and performed extensive experiments for improving and validating the agent. Max Tegmark developed the Occam's-razor-with-MDL aspect of the agent and created datasets and Mathematica scripts for evaluating the agent's performance. 
The project would not have  succeeded without close collaboration on 
proofs, algorithm development and writing, occasionally late into the night.

For the Learnability for the Information Bottleneck work (Chapter \ref{chap3:IB}), inspired by Max's suggestion of starting from simple scenarios, I discovered the Information Bottleneck learnability phase transition. Through extensive discussion with Ian Fischer, I developed the main theorems that relate the learnability phase transition to the structure of the dataset. I performed the experiments on synthetic and MNIST datasets, and Ian performed the CIFAR10 experiments.  For the initial draft, I wrote the method, proof and experiment sections, and Ian wrote the introduction and related work sections. Ian, Ike and Max all provided valuable feedback for significantly improving the drafts.

For the phase transition for the Information Bottleneck project (Chapter \ref{chap4:IB_phase_transition}), which was done during my internship at Google AI, I developed the main theories through extensive discussions with Ian and performing experiments together. I wrote the majority of the draft, and both Ian and I contributed significantly to the numerical experiments.

For the Pareto-optimal data compression project (Chapter \ref{chap5:distillation}), I mainly contributed to the experimental part, by running experiments to compute the Pareto frontier of MNIST, FashionMNIST and CIFAR10 datasets in binary classification scenarios. Max contributed to the motivation, theorems and writing of the paper, but we extensively discussed all aspects together.

For the causal learning work (Chapter \ref{chap6:causal}), which was mainly done during my internship at NVIDIA Research, I contributed to the initial idea, algorithm development, and extensive experiments for evaluating the methods, with Thomas Breuel providing lots of guidance during the process. I wrote the initial draft and 
subsequent re-submissions, with Thomas Breuel and Jan Kautz providing valuable feedback for improving the draft. Michael  Skuhersky contributed to the C.~elegans experiment corresponding text in the  re-submissions.

For the meta-learning autoencoder (MeLA) work (Chapter \ref{chap7:mela}), I contributed to the initial idea, main experiment and writing of the draft. John Peurifoy contributed to the comparison experiments of the paper. Ike provided valuable ideas for influence identification and interactive learning aspects of the MeLA architecture, and both Ike and Max provided valuable guidance, feedback during the whole process, as well as editing.

For the Rank Pruning work (Chapter \ref{chap8:rankpruning}), I contributed to the proposal of the robust noise estimator and  theory developments of the work, and Curtis Northcutt contributed to the initial ideas and initial experiments of the work. Both Curtis and I contributed significantly to the main experiments and writing of the paper, and Ike provided valuable guidance and feedback during the process.

\chapter{AI Physicist for few-shot, lifelong learning of physics}
\label{chap2:AI_physicist}

Imagine that you are an agent. As a start, you don't know anything about how the world works, but has been endowed with a perfect mechanism to learn. Suppose that you will experience many environments each of which may have multiple physics laws, how can you make sense of the environments, and become better at learning over time?

To address this question\footnote{Published in \emph{Physical Review E}, 100 (3), 033311, ``\href{https://journals.aps.org/pre/abstract/10.1103/PhysRevE.100.033311}{Toward an artificial intelligence physicist for unsupervised learning}'', Wu, Tailin and Max Tegmark \cite{wu2018toward}.}\footnote{The code is open-sourced at \href{https://github.com/tailintalent/AI_physicist}{github.com/tailintalent/AI\_physicist}.}, we look at how physicists model the world. Inspired by four common strategies with a long history in physics: divide-and-conquer, Occam's razor, unification and lifelong learning, we propose a novel paradigm 
centered around the learning and manipulation of \emph{theories}, which parsimoniously predict both aspects of the future (from past observations) and the domain in which these predictions are accurate. This is in sharp contrast with the standard approach of using a single model to learn everything. Specifically, we propose a novel generalized-mean-loss to encourage each theory to specialize in its comparatively advantageous domain, and a differentiable description length objective to downweight bad data and ``snap" learned theories into simple symbolic formulas. Theories are stored in a ``theory hub", which continuously unifies learned theories and can propose theories when encountering new environments. We test our implementation, the toy ``AI Physicist" learning agent, on a suite of increasingly complex physics environments. From unsupervised observation of trajectories through worlds involving random combinations of gravity, electromagnetism, harmonic motion and elastic bounces, our agent typically learns faster and produces mean-squared prediction errors about a billion times smaller than a standard feedforward neural net of comparable complexity, typically recovering integer and rational theory parameters exactly. 
Our agent successfully identifies domains with different laws of motion also for a nonlinear chaotic double pendulum in a piecewise constant force field.

\section{Introduction}
\label{IntroSec}

\subsection{Motivation}

The ability to predict, analyze and parsimoniously model observations is not only central to 
physics,
but also a goal of unsupervised machine learning, which is a key frontier in artificial intelligence (AI) research 
\cite{lecun2015deep}. Despite impressive recent progress with artificial neural nets,
they still get frequently outmatched by human researchers at such modeling, suffering from two drawbacks: 
\begin{enumerate}
\item Different parts of the data are often generated by different mechanisms in different contexts.
A big model that tries to fit all the data in one environment may therefore underperform in a new environment where some mechanisms are replaced by new ones, being inflexible and inefficient at combinatorial generalization \cite{battaglia2018relational}.  
\item Big models are generally hard to interpret, and may not reveal succinct and universal knowledge such as Newton's law of gravitation that explains only some aspects of the data. The pursuit of ``intelligible intelligence" in place of inscrutable black-box neural nets is important and timely, given the growing interest in AI interpretability from AI users and policymakers, especially for AI components involved in decisions and infrastructure where trust is important \cite{russell2015research, amodei2016concrete, boden2017principles, krakovna2016increasing}.
\end{enumerate}

\def\mytab{\hglue5mm}
\begin{table}[t]
\begin{center}
\begin{tabular}{|>{\raggedright}p{2.5cm}|p{7.4cm}|}
\hline
Strategy			&Definition\\
\hline
Divide-and-	     	&Learn multiple theories each of which  \\ 
\mytab conquer		& specializes to fit {\it part} of the data very well\\
\hline
Occam's		&Avoid overfitting by minimizing description\\
\mytab Razor		& length, which can include 
replacing fitted constants by simple integers or fractions.\\
\hline
Unification 		&Try unifying learned theories by introducing parameters\\
\hline
Lifelong 			&Remember learned solutions and try them\\
\mytab Learning		&on future problems\\
\hline
\end{tabular}
\end{center}
\caption{AI Physicist strategies tested.
\label{table1}
}
\end{table}
To address these challenges, we will borrow from physics the core idea of a  \emph{theory}, which parsimoniously predicts both aspects of the future (from past observations) and also the domain in which these predictions are accurate. This suggests an alternative to the standard machine-learning paradigm of fitting a single big model to all the data: instead, learning small theories one by one, and gradually accumulating and organizing them. This paradigm suggests the four specific approaches summarized in Table \ref{table1}, which we combine into a toy ``AI Physicist" learning agent: To find individual theories from complex observations, we use the divide-and-conquer strategy with multiple theories and a novel generalized-mean loss that encourages each theory to specialize in its own domain by giving larger gradients for better-performing theories. To find simple theories that avoid overfitting and generalize well, we use the strategy known as Occam's razor, favoring simple theories that explain a lot, using a computationally efficient approximation of the minimum-description-length (MDL) formalism. To unify similar theories found in different environments, we use the description length for clustering and then learn a ``master theory" for each class of theories. To accelerate future learning, we use a lifelong learning strategy where learned theories are stored in a theory hub for future use.

\subsection{Goals \& relation to prior work}

The goal
of the AI Physicist learning agent presented in this paper is quite limited, and does not even remotely approach the ambition level of problem solving by human physicists. The latter is likely to be almost as challenging as artificial general intelligence, which most AI researchers guess remains decades away \cite{muller2016future,grace2018will}.
Rather, the goal of this paper is to take a very modest but useful step in that direction, combining the four physics-inspired strategies above.  

Our approach complements other work on automatic program learning, such as neural program synthesis/induction
\cite{graves2014neural,sukhbaatar2015end,reed2015neural,parisotto2016neuro,devlin2017robustfill,bramley2018learning} 
and symbolic program induction \cite{Muggleton1991,lavrac1994inductive,liang2010learning,ellis2015unsupervised,dechter2013bootstrap} and builds on prior machine-learning work on divide-and-conquer \cite{cormen2009introduction,furnkranz1999separate,ghosh2017divide}, 
network simplification \cite{rissanen1978modeling,hassibi1993second,suzuki2001simple,grunwald2005advances,han2015deep,han2015learning} 
and continuous learning \cite{kirkpatrick2017overcoming,li2017learning,lopez2017gradient,nguyen2017variational}. 
It is often said that babies are born scientists, and there is arguably evidence for use of all of these four strategies during childhood development as well \cite{bramley2018learning}.

There has been significant recent progress on AI-approaches specifically linked to physics, including
physical scene understanding \cite{yildirim2018neurocomputational},
latent physical properties \cite{zheng2018unsupervised,battaglia2016interaction,chang2016compositional},
learning physics simulators \cite{watters2017visual},
physical concept discovery \cite{iten2018discovering},
an intuitive physics engine  \cite{lake2017building},
and
the ``Sir Isaac" automated adaptive inference agent \cite{daniels2015automated}.
Our AI Physicist is different and complementary in two fundamental ways that loosely correspond to the two motivations on the first page: 

\begin{enumerate}
\item All of these papers learn one big model to fit all the data. In contrast, the AI Physicist learns many small models applicable in different domains, using the divide-and-conquer strategy. 
\item Our primary focus is not 
on making approximate predictions or discovering latent variables, but on 
near-exact predictions and complete intelligibility. From the former perspective, it is typically irrelevant if a model parameter changes by a tiny amount, but from a physics perspective, one is quite interested to learn that gravity weakens like distance to the power $2$ rather than $1.99999314$.
\end{enumerate}

We share this focus on intelligibility with a long tradition of research on computational scientific discovery \cite{dvzeroski2007computational}, including the Bacon system \cite{langley1981data} and its successors \cite{langley1989data}, which induced physical laws from observations and which also used a divide-and-conquer strategy. Other work has extended this paradigm to support discovery of differential equation models from multivariate time series \cite{dzeroski1995discovering,bradley2001reasoning,langley2003robust,langley2015heuristic}.

The rest of this paper is organized as follows. In Section \ref{sec:method}, we introduce the architecture of our AI Physicist learning agent, and the algorithms implementing the four strategies. We present the results of our numerical experiments using a suite of physics environment benchmarks in Section \ref{sec:numerical_experiments}, and discuss our conclusions in Section IV, delegating supplementary technical details to a series of appendices.

\section{Methods}
\label{sec:method}

Unsupervised learning of regularities in time series can be viewed as a supervised learning problem of predicting the future from the past.
This paper focuses on the task of predicting the next state vector $\y_t\in \R^d$ in a sequence from the
concatenation $\x_t =(\y_{t-T},...,\y_{t-1})$  of the last $T$ vectors. 
 However, our AI Physicist formalism applies more generally to learning any function $\R^M\mapsto\R^N$ from examples.
In the following we first define \emph{theory}, then introduce a unified AI Physicist architecture implementing the four aforementioned strategies.

\subsection{Definition of Theory}

A theory $\mathcal{T}$ is a 2-tuple $(\f, c)$, where $\f$ is a prediction function that predicts $\y_t$ when $\x_t$ is within the theory's domain, and $c$ is a domain sub-classifier which takes $\x_t$  as input and outputs a logit of whether $\x_t$ is inside this domain. When multiple theories are present, the sub-classifier $c$'s outputs are concatenated and fed into a softmax function, producing probabilities for which theory is 
applicable. 
Both $\f$ and $c$ can be implemented by a neural net or symbolic formula,
and can be set to learnable during training and fixed during prediction/validation.

This definition draws inspirations from physics theories (conditional statements), such as ``a ball not touching anything (\emph{condition}) with vertical velocity and height $(v_0,h_0)$ will a time $t$ later have $\y\equiv(v,h)=(v_0-gt,h_0+v_0 t -gt^2/2)$ (\emph{prediction function})".
For our AI Physicist, theories constitute its ``atoms" of learning, as well as the building blocks for higher-level manipulations.

\begin{figure}[t]
\centerline{\includegraphics[width=100mm]{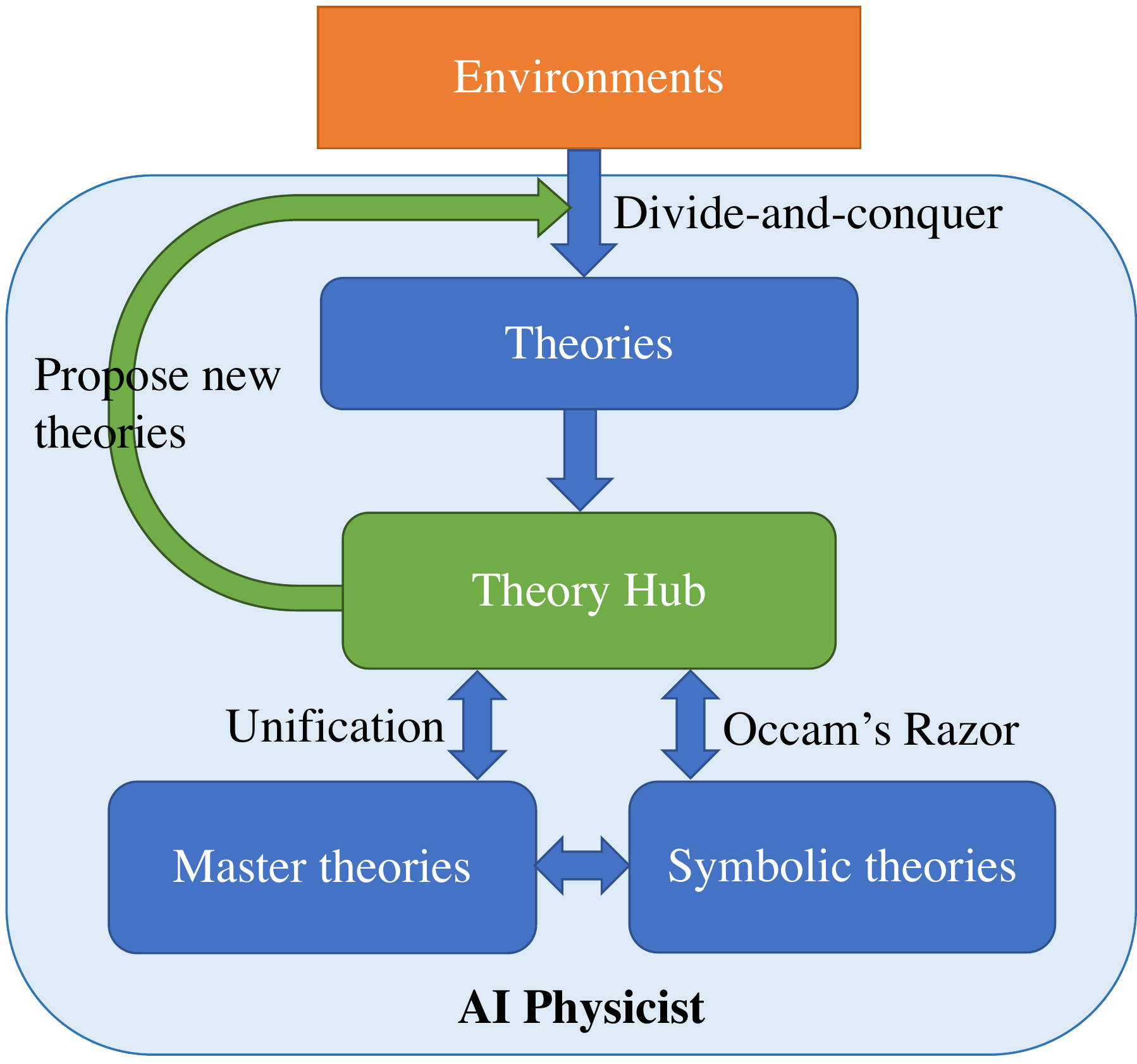}}
\caption{
AI Physicist Architecture
\label{fig:AI_physicist_architecture}
}
\end{figure}
\subsection{Divide-and-Conquer}
\label{sec:divide_and_conquer}

\subsection{AI Physicist Architecture Overview}

Figure \ref{fig:AI_physicist_architecture} illustrates the architecture of the AI Physicist learning agent. At the center is a theory hub which stores the learned and organized theories. When encountering a new environment, the agent first inspects the hub and proposes old theories that help account for parts of the data as well as randomly initialized new theories for the rest of the data. All these theories are trained via our divide-and-conquer strategy, first jointly with our generalized-mean loss then separately to fine-tune each theory in its domain (section \ref{sec:divide_and_conquer}). Successful theories along with the corresponding data are added to the theory hub.

The theory hub has two organizing strategies: (1) Applying Occam's razor, it snaps the learned theories, in the form of neural nets, into simpler symbolic formulas (section \ref{sec:Occams_Razor}). (2) Applying unification, it clusters and unifies the symbolic theories into master theories (section \ref{sec:unification}). The symbolic and master theories can be added back into the theory hub, improving theory proposals for new environments. The detailed AI Physicist algorithm is presented in a series of appendices. It has polynomial time complexity, as detailed in Appendix~\ref{ComplexitySec}.

Conventionally, a function $\f$ mapping $\x_t\mapsto\y_t$  is learned by parameterizing
$\f$ by some parameter vector $\vtheta$ that is adjusted to minimize a loss (empirical risk)
\beq{expected_risk}
\Ell\equiv\sum_t \ell[\f(\x_t), \y_t],
\eeq
where $\ell$ is some non-negative distance function quantifying how far each prediction is from the target, typically satisfying $\ell(\y,\y)=0$.
In contrast, a physicist observing an unfamiliar environment  does typically {\it not} try to predict everything with one model, instead starting with an easier question: is there any part or aspect of the world that can be described? 
For example, when Galileo famously tried to model the motion of swinging lamps in the Pisa cathedral, he completely ignored everything else, and made no attempts to simultaneously predict the behavior of sound waves, light rays, weather, or subatomic particles.
In this spirit, we allow multiple competing theories $\mathbfcal{T}=\{\T_i\}=\{(\f_i, c_i)\}$, $i=1,2,...M$, to specialize in different domains, with a novel generalized-mean loss
\beq{generalized_mean_risk_empirical}
\Ell_\gamma\equiv\sum_t\left(\frac{1}{M}\sum_{i=1}^M \ell[\f_i(\x_t), \y_t]^\gamma\right)^{1/\gamma}
\eeq

When $\gamma<0$, the loss $\Ell_\gamma$
will be dominated by whichever prediction function $\f_i$ fits each data point best. This dominance is controlled by $\gamma$, with 
$\Ell_\gamma\to \min_i \ell[\f_i(\x_t),\y_t]$ in the limit where $\gamma\to-\infty$.
This means that the best way to minimize $\Ell_\gamma$ is for each $\f_i$ to specialize by further improving its accuracy for the data points where it already outperforms the other theories.
The following Theorem \ref{thm:theorem_gradient} formalizes the above intuition, stating that under mild conditions for the loss function $\ell(\cdot,\cdot)$, the generalized-mean loss gives larger gradient w.r.t.~the error $|\hat{\y}_t-\y_t|$ for theories that perform better, so that a gradient-descent loss minimization encourages specialization. 

\begin{theorem}
\label{thm:theorem_gradient}
Let $\hat{\y}^{(i)}_t\equiv \f_i(\x_t)$ denote the prediction of the target $\y_t$ by the function $\f_i$, $i=1,2,...M$.
Suppose that $\gamma<0$ and $\ell(\hat{\y}_t, \y_t) = \ell(|\hat{\y}_t - \y_t|)$ for a monotonically increasing function 
$\ell(u)$ that vanishes on $[0,u_0]$ for some $u_0\geq 0$, with $\ell(u)^\gamma$ differentiable and strictly convex for $u>u_0$. \\
Then if $0<\ell(\hat{\y}_t^{(i)}, \y_t) < \ell(\hat{\y}_t^{(j)}, \y_t)$, we have 
\beq{gradient_greater}
\left|\frac{\partial\Ell_\gamma}{\partial u^{(i)}_t}\right| > \left|\frac{\partial \Ell_\gamma}{\partial u^{(j)}_t}\right|,
\eeq
 where $u^{(i)}_t\equiv|\hat{\y}_t^{(i)} - \y_t|$.
\end{theorem}
Appendix \ref{proof_theorem_gradient} gives the proof, and also shows
that this theorem applies to mean-squared-error (MSE) loss $\ell(u)=u^2$,
mean-absolute-error loss $\ell(u)=|u|$, Huber loss and our description-length loss from the next section.

We find empirically that the simple choice $\gamma=-1$ works quite well, striking a good balance between encouraging specialization for the best theory and also giving some gradient for theories that currently perform slightly worse. We term  this choice $\Ell_{-1}$ the ``harmonic loss", because it corresponds to the harmonic mean of the losses for the different theories. Based on the harmonic loss, we propose an unsupervised differentiable divide-and-conquer (DDAC) algorithm (Alg.~\ref{alg:divide_and_conquer} in Appendix~\ref{DivideAndConquerAlgo}) that simultaneously learns prediction functions $\{\f_i\}$ and corresponding domain classifiers $\{c_i\}$ from observations.

Our DDAC method's combination of multiple prediction modules into a single prediction is reminiscent of AdaBoost \cite{freund1997decision}.
While AdaBoost gradually upweights those modules that best predict {\it all} the data, DDAC instead 
identifies complementary modules that each predict some {\it part} of the data best, and encourages these modules to simplify and improve by specializing on these respective parts.

\subsection{Occam's Razor}
\label{sec:Occams_Razor}

The principle of  Occam's razor, that simpler explanations are better, is quite popular among physicists.
This preference for parsimony helped dispense with phlogiston, aether and other superfluous concepts.

Our method therefore incorporates the minimum-description-length (MDL) formalism \cite{rissanen1978modeling,grunwald2005advances}, which provides an elegant mathematical implementation of Occam's razor. It is rooted in Solomonoff's theory of inference \cite{solomonoff1964formal} and is linked to Hutter's AIXI approach to artificial general intelligence \cite{hutter2000theory}. 
The description length (DL) of a dataset $\D$ is defined as the number of bits required to describe it.
For example, if regularities are discovered that enable data compression, then the corresponding description length is defined as 
the number of bits of the program that produces $\D$ as its output (including both the code bits and the compressed data bits).
In our context of predicting a time series, this means that the description length is the number of bits required to describe the theories used plus the number of bits required to store all prediction errors.
Finding the optimal data compression and hence computing the MDL is a famous hard problem that involves searching an exponentially large space, but any discovery reducing the description length is a step in the right direction, and provably avoids the overfitting problem that plagues many alternative machine-learning strategies \cite{rissanen1978modeling,grunwald2005advances}.

The end-goal of the AI Physicist is to discover theories $\mathbfcal{T}$ minimizing the total description length, given by
\beq{description_length}
\DL(\mathbfcal{T},\D)=\DL(\mathbfcal{T})+\sum_t\DL(\u_t),
\eeq
where $\u_t=\hat{\y}_t-\y_t$ is the prediction error at time step $t$. By discovering simple theories that can each account for parts of the data very well, the AI Physicist strives to make both $\DL(\mathbfcal{T})$ and $\sum_t\DL(\u_t)$ small.

Physics has enjoyed great success in its pursuit of simpler theories using rather vague definitions of simplicity. In the this spirit, we choose to compute the description length DL not exactly, but using an approximate heuristic that is numerically efficient, and significantly simpler than more precise versions such as \cite{rissanen1983universal},
paying special attention to rational numbers since they are appear in many physics theories.
We compute the DL of both theories $\T$ and prediction errors $\u_t$ as the sum of the DL of all numbers that specify them, using the following conventions for the DL of integers, rational numbers and real numbers. Our MDL implementation differs from popular machine-learning approaches whose goal is efficiency and generalizability  \cite{Hinton:1993:KNN:168304.168306,han2015deep,blierdescription} rather than intelligibility.

\begin{figure}[pbt]
\centerline{\includegraphics[width=88mm]{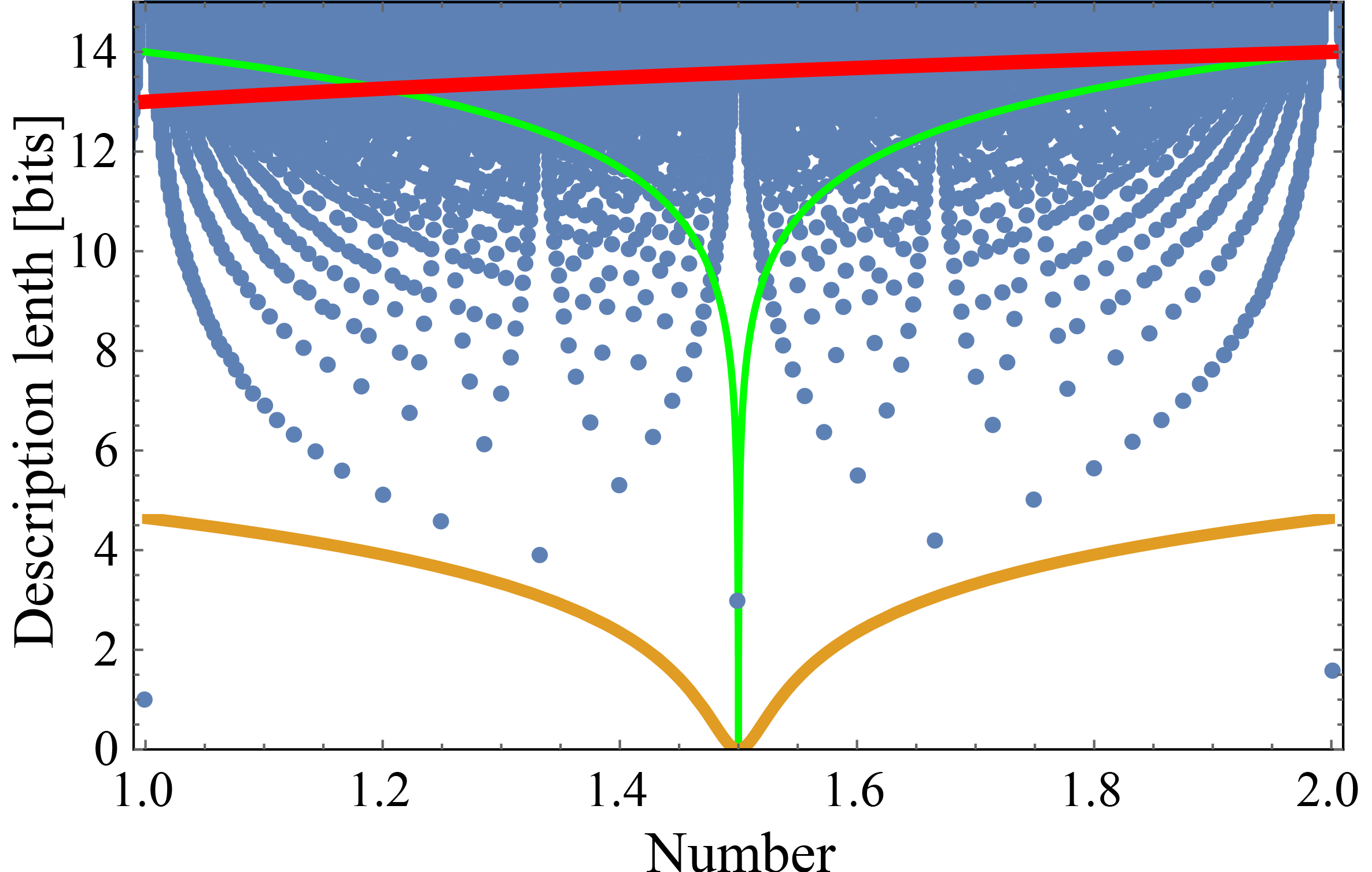}}
\caption{The description length $\DL$ is shown for real numbers $p$ with $\epsilon=2^{-14}$ (rising curve) and for 
rational numbers (dots).
Occam's Razor favors lower DL, and our MDL rational approximation of a real parameter $p$ is the lowest point after taking these ``model bits" specifying the approximate parameter and adding the ``data bits" $\Ell$ required to specify the prediction error made.
The two symmetric curves illustrate the simple example where $\Ell=\logplus\left({x-x_0\over\epsilon}\right)$
for $x_0=1.4995$, $\epsilon=2^{-14}$ and $0.02$, respectively.}
\label{RationalComplexityFig}
\end{figure}
The number of binary digits required to specify a natural number $n=1,2,3,...$ is approximately $\log_2 n$, so we define 
$\DL(n)\equiv\log_2 n$ for natural numbers. 
For an integer $m$, we define
\beq{IntegerDLeq}
\DL(m)\equiv 
\log_2 (1+|m|).
\eeq
For a rational number $q=m/n$, the description length is the sum of that for its integer numerator and (natural number) denominator, as illustrated in \fig{RationalComplexityFig}:
\beq{RationalDLeq}
\DL\left({m\over n}\right)=\log_2[(1+|m|)n].
\eeq
For a real number $r$ and a numerical precision floor $\epsilon$, we define
\beq{RealDLeq}
\DL(r)=\logplus\left({r\over\epsilon}\right),
\eeq
where the function
\beq{LogplusDefEq}
\logplus(x)\equiv{1\over 2}\log_2\left(1+x^2\right)
\eeq
is plotted in \fig{RationalComplexityFig}.
Since  $\logplus(x)\approx\log_2 x$ for $x\gg1$, 
$\DL(r)$ is approximately the description length of  the integer closest to $r/\epsilon$.
Since  $\logplus(x)\simpropto x^2$ for $x\ll 1$, $\DL(r)$ simplifies to a quadratic (mean-squared-error) loss function below the numerical precision, which will prove useful 
below.\footnote{Natural alternative definitions of $\logplus(x)$ include $\log_2\left(1+|x|\right)$, $\log_2\max(1,|x|)$, $(\ln 2)^{-1}\sinh^{-1}|x|$
and $(2\ln 2)^{-1}\sinh^{-1}(x^2)$. Unless otherwise specified, we choose $\epsilon=2^{-32}$ in our experiments.}

Note that as long as all prediction absolute errors $|u_i|\gg\epsilon$ 
for some dataset,

minimizing the total description length $\sum_i \DL(u_i)$ instead of the MSE $\sum_i u_i^2$
corresponds to minimizing the geometric mean instead of the arithmetic mean of the squared errors, which encourages focusing more on improving already well-fit points. 
$\sum_i \DL(u_i)$ drops by 1 bit whenever one prediction error is halved, which is can typically be achieved by fine-tuning the fit for many valid data points that are already well predicted while increasing DL for bad or extraneous points at most marginally.

For numerical efficiency, our AI Physicist minimizes the description length of 
\eq{description_length} in two steps:
1) All model parameters are set to trainable real numbers, and the DDAC algorithm is applied to minimize the harmonic loss $\Ell_{-1}$ with 
$\ell(\u)\equiv\sum_i \DL(u_i)$ using \eq{RealDLeq} and the annealing procedure for the precision floor described in Appendix~\ref{DivideAndConquerAlgo}.
2) Some model parameters are replaced by rational numbers as described below, followed by re-optimization of the other parameters. 
The idea behind the second step is that if a physics experiment or neural net training produces a parameter $p=1.4999917$, it would be natural to interpret this as a hint, and to check if $p=3/2$ gives an equally acceptable fit to the data, reducing total DL.
We implement step 2 using continued fraction expansion as described in Appendix \ref{appendix:OccamsRazorAlgo} and illustrated in \fig{NumberMDLfig}.

\begin{figure}[pbt]
\centerline{\includegraphics[width=88mm]{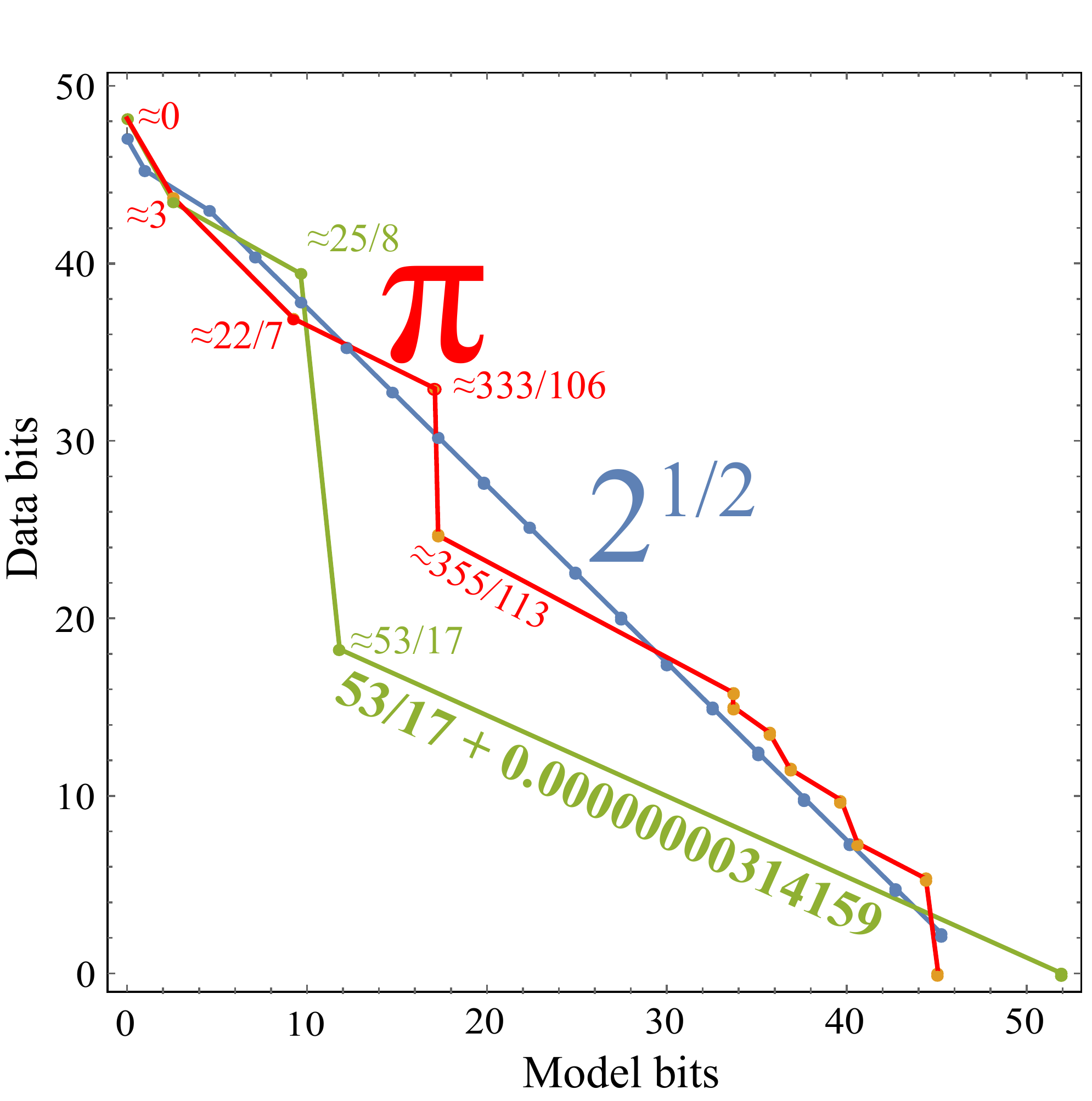}}
\caption{Illustration of our minimum-description-length (MDL) analysis 
of the parameter vector $\p=\{\pi,\sqrt{2},3.43180632382353\}$.
We approximate each real number $r$ as a fraction $a_k/b_k$ using the first $k$ terms of its  continued fraction expansion, and for each integer $k=1,...$, we plot the number of ``data bits" required to encode the prediction error $r-a_k/b_k$ to 14 decimal places versus the number of ``model bits" required to encode the rational approximation $a_k/b_k$, as described in the text. We then select the point with smallest bit sum (furthest down/left from the diagonal) as our first approximation candidate to test. 
Generic irrational numbers are incompressible; the total description length (model bits+data bits) is roughly independent of $k$ as is seen for $\pi$ and $\sqrt{2}$, corresponding to a line of slope $-1$ around which there are small random fluctuations. In contrast, the green/light grey curve (bottom) is for a parameter that is anomalously close to a rational number, and the curve reveals this by the approximation $53/17$ reducing the total description length (model$+$data bits) by about 16 bits.
\label{NumberMDLfig}
}
\end{figure}

\subsection{Unification}
\label{sec:unification}

Physicists aspire not only to find simple theories that explain aspects of the world accurately, but also to discover underlying similarities between theories and {\it unify} them. 
For example, when James Clerk Maxwell corrected and unified four key formulas describing electricity and magnetism into his eponymous equations ($\mathrm{d}\F=0$, 
$\mathrm{d}\star \F=\J$ in differential form notation), he revealed the nature of light and enabled the era of wireless communication.

Here we make a humble attempt to automate part of this process. The goal of the unification is to output a master theory $\mathscr{T}=\{(\f_\p,\cdot)\}$,  such that varying the parameter vector $\p\in\R^n$ can generate a continuum of theories $(\f_\p,\cdot)$ including previously discovered ones.
For example, Newton's law of gravitation can be viewed as a master theory unifying the gravitational force formulas around different planets by introducing a parameter $p$ corresponding to planet mass. Einstein's special relativity can be viewed as a master theory unifying the approximate formulas for $v\ll c$ and $v\approx c$ motion. 

We perform unification by first computing the description length $\text{dl}^{(i)}$ of the prediction function $\f_i$ (in  symbolic form) for each theory $i$ and performing clustering on $\{\text{dl}^{(i)}\}$. Unification is then achieved by discovering similarities and variations between the symbolic formulas in each cluster, retaining the similar patterns, and introducing parameters in place of the parameters that vary as detailed in Appendix \ref{UnificationAlgo}.

\subsection{Lifelong Learning}

Isaac Newton once said ``If I have seen further it is by standing on the shoulders of giants", emphasizing the utility of  building on past discoveries. At a more basic level, our past experiences enable us humans to model new environments much faster than if we had to re-acquire all our knowledge from scratch. We therefore embed a 
lifelong learning strategy
into the architecture of the AI Physicist. As shown in Fig. \ref{fig:AI_physicist_architecture} and Alg.~\ref{alg:overall_algorithm}, the theory hub stores successfully learned theories, organizes them with our Occam's razor and unification algorithms (reminiscent of what humans do while dreaming and reflecting), and when encountering new environments, uses its accumulated knowledge to propose new theories that can explain parts of the data. This both
ensures that past experiences are not forgotten and enables faster learning in novel environments. The detailed algorithms for proposing and adding theories are in Appendix \ref{appendix:theory_proposal_adding}.

\section{Results of  Numerical Experiments}
\label{sec:numerical_experiments}

\subsection{Physics Environments}

We test our algorithms on two suites of benchmarks, each with increasing complexity. In all cases, the goal is to predict the two-dimensional motion as accurately as possible. One suite involves chaotic and highly nonlinear motion of a charged double pendulum in two adjacent electric fields. The other suite involves balls affected by gravity, electromagnetic fields, springs and bounce-boundaries, as exemplified in \fig{WorldExampleFig}.
Within each spatial region, the force corresponds to a potential energy function $V\propto (ax+by+c)^n$ for some constants $a$, $b$, $c$, where 
$n=0$ (no force), 
$n=1$ (uniform electric or gravitational field), 
$n=2$ (spring obeying Hooke's law)
or $n=\infty$ (ideal elastic bounce), and optionally involves also a uniform magnetic field.
The environments are summarized in Table~\ref{DetailedResultsTable2}.

\begin{figure}[phbt]
\centerline{\includegraphics[width=86mm]{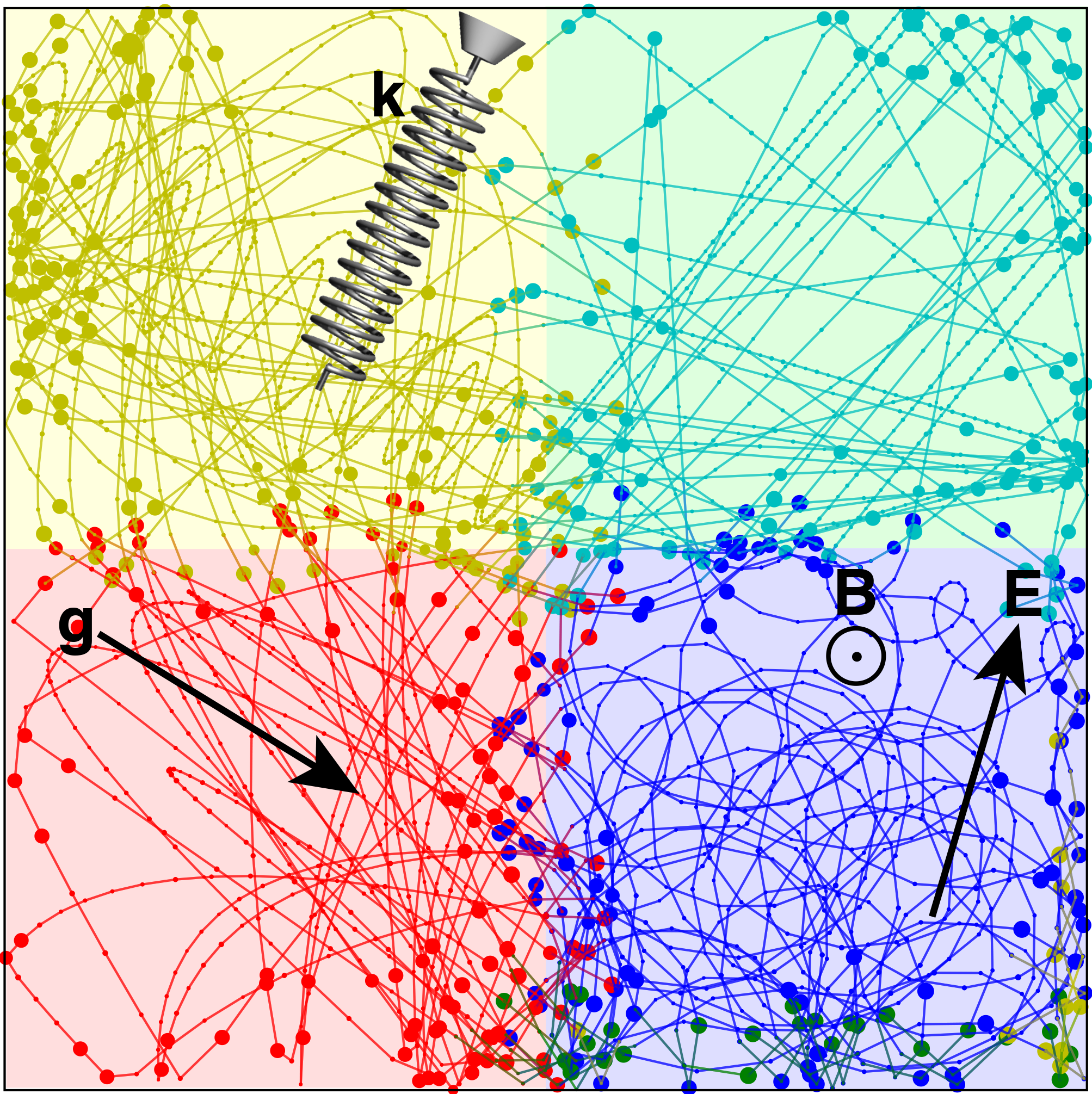}}
\caption{In this sample mystery world, a ball moves through a harmonic potential (upper left quadrant), a gravitational field (lower left) and an electromagnetic field (lower right quadrant) and bounces elastically from four walls. 
The only input to the AI Physicist is the sequence of dots (ball positions); the challenge is to learn all boundaries and laws of motion (predicting each position from the previous two). 
The color of each dot represents the domain into which it is classified by 
$\c$,
and its area represents the description length of the error with which its position is predicted ($\epsilon=10^{-6}$) after the DDAC
(differentiable divide-and-conquer) algorithm; the AI Physicist tries to minimize the total area of all dots.
\label{WorldExampleFig}
}
\end{figure}

\subsection{Numerical Results}

In the mystery world example of \fig{WorldExampleFig}, after the DDAC algorithm~\ref{alg:divide_and_conquer} taking the sequence of coordinates as the only input, we see that the AI Physicist has learned to simultaneously predict the future position of the ball from the previous two, and classify 
without external supervision the observed inputs into four big physics domains. The predictions are seen to be more accurate deep inside the domains (tiny dots) than near boundaries (larger dots) where transitions and bounces create small domains with laws of motion that are harder to infer because of complexity and limited data. Because these small domains can be automatically inferred and eliminated once the large ones are known as described in Appendix~\ref{DomainElimination}, all accuracy benchmarks quoted below refer to points in the large domains only.

After DDAC, the AI Physicist performs Occam's-razor-with-MDL (Alg.~\ref{alg:MDL_simplification}) 
on the learned theories. As an example, it discovers that the motion deep inside the lower-left quadrant obeys the difference equation parameterized by a learned 3-layer neural net, which after the first collapseLayer transformation simplifies to
\begin{equation}
\label{RawDifferenceEq}
\begin{aligned}
\hat{\y}_t&=&\left(
\begin{tabular}{rrrr}
-0.99999994&0.00000006&1.99999990&-0.00000012\\ 
-0.00000004&-1.0000000&0.00000004&2.00000000 
\end{tabular}
\right)\x_t\nonumber\\
&+& 
\left(
\begin{tabular}{r}
 0.01088213\\ 
-0.00776199
\end{tabular}
\right),
\end{aligned}
\end{equation}

with $\DL(\f)=212.7$ and $\sum_t\DL(\u_t)=2524.1$.
The snapping stage thereafter simplifies this to
\beq{RawDifferenceEqSnapped}
\hat{\y}_t=\left(
\begin{tabular}{rrrr}
-1&0&2&0\\ 
0&-1&0&2 
\end{tabular}
\right)\x_t
+ 
\left(
\begin{tabular}{r}
 0.010882\\ 
-0.007762 
\end{tabular}
\right).
\eeq
which has lower description length in both model bits ($\DL(\f)=55.6$) and data bits ($\sum_t\DL(\u_t)=2519.6$) and gets transformed to the symbolic expressions
\begin{eqnarray}
\hat{x}_{t+2}&=&2 x_{t+1}-x_t+0.010882,\nonumber\\
\hat{y}_{t+2}&=&2 y_{t+1}-y_t-0.007762,
\end{eqnarray}
where we have writen the 2D position vector $\y=(x,y)$ for brevity.
During unification (Alg.~\ref{UnificationAlgo}), the AI Physicist discovers multiple clusters of theories based on the DL of each theory, where one cluster has DL ranging between 48.86 and 55.63, which it unifies into a master theory $\f_\p$ with
\begin{equation}
\label{GravityUnifiedEq}
\begin{aligned}
\hat{x}_{t+2}&=&2 x_{t+1}-x_t+p_1,\nonumber\\
\hat{y}_{t+2}&=&2 y_{t+1}-y_t+p_2,
\end{aligned}
\end{equation}
effectively discovering a ``gravity" master theory out of the different types of environments it encounters.
If so desired, the difference equations~(\ref{GravityUnifiedEq}) can be automatically generalized to the more familiar-looking differential equations 
\begin{eqnarray}
\ddot x&=&g_x,\nonumber\\
\ddot y&=&g_y,\nonumber
\end{eqnarray}
where $g_i\equiv p_i(\Delta t)^2$,
and both the Harmonic Oscillator Equation and Lorentz Force Law of electromagnetism can be analogously auto-inferred from other master theories learned.

Many mystery domains in our test suite involve laws of motion whose parameters include both rational and irrational numbers. To count a domain as ``solved" below, we use the very stringent requirement that any rational numbers (including integers) must be discovered {\it exactly}, while irrational numbers must be recovered with accuracy $10^{-4}$.

We apply our AI Physicist to 40 mystery worlds in sequence (Appendix~\ref{DetailedResultsSec}). After this training, we apply it to a suite of 40 additional worlds to test how it learns different numbers of examples.
The results 
are shown in tables~\ref{DetailedResultsTable}
and~\ref{DetailedResultsTable2}, and  
Table~\ref{ResultsSummaryTable} summarizes these results using the median over worlds.
For comparison, we also show results for two simpler agents with similar parameter count: a ``baseline" agent consisting of a three-layer feedforward MSE-minimizing leakyReLU network and a ``newborn" AI Physicist that has not seen any past examples and therefore cannot benefit from the lifelong learning strategy.

We see that the newborn agent outperforms baseline on all the tabulated measures, and that the AI Physicist does still better. 
Using all data, the Newborn agent and AI Physicist are able to predict
with mean-squared prediction error below $10^{-13}$, more than nine orders of magnitude below baseline. Moreover, the Newborn and AI Physicist agents are able to simultaneously learn the domain classifiers with essentially perfect accuracy, without external supervision. Both agents are able to solve above 90\% of all the 40 mystery worlds according to our stringent criteria.

The main advantage of the AI Physicist over the Newborn agent is seen to be its learning speed, attaining given accuracy levels faster, especially during the early stage of learning. Remarkably, for the subsequent 40 worlds, the AI Physicist reaches 0.01 MSE within 35 epochs using as little as 1\% of the data, performing almost as well as with 50\% of the data much better than the Newborn agent. This illustrates that the lifelong learning strategy enables the AI Physicist to learn much faster in novel environments with less data.
This is much like an experienced scientist can solve new problems way faster than a beginner by building on prior knowledge about similar problems.
  
\begin{table}
\begin{center}
\begin{tabular}{|l|r|r|r|}
\hline
Benchmark			&Baseline		&Newborn		&AI Physicist\\
\hline
$\log_{10}$ mean-squared error 	&-3.89	&-13.95		&-13.88\\
Classification accuracy	&67.56\%		&100.00\%	&100.00\%\\
Fraction of worlds solved	&0.00\%		&90.00\%		&92.50\%\\
Description length for $\f$		&11,338.7		&198.9 		&198.9\\
Epochs until $10^{-2}$ MSE		&95			&83 			&15\\
Epochs until $10^{-4}$ MSE	&6925		&330		&45\\
Epochs until $10^{-6}$ MSE&$\infty$		&5403		&3895\\
Epochs until $10^{-8}$ MSE&$\infty$		&6590		&5100\\
\hline
$\log_{10}$ MSE		&			&			&\\
$\>\>$ using 100\% of data 	&-3.78		&-13.89		&-13.89\\
$\>\>$ using 50\% of data 	&-3.84		&-13.76		&-13.81	\\
$\>\>$ using 10\% of data 	&-3.16		&-7.38		&-10.54\\
$\>\>$ using 5\% of data 	&-3.06		&-6.06 		&-6.20\\
$\>\>$ using 1\% of data 	&-2.46		&-3.69 		&-3.95\\
\hline
Epochs until $10^{-2}$ MSE 	&			& 			&	\\
$\>\>$ using 100\% of data 	&95		&80	&15\\
$\>\>$ using 50\% of data 	&190		&152.5	&30\\
$\>\>$ using 10\% of data 	&195		&162.5	&30\\
$\>\>$ using 5\% of data 	&205		&165	&30\\
$\>\>$ using 1\% of data 	&397.5		&235	&35\\
\hline
\end{tabular}
\end{center}
\caption{Summary of numerical results, taking the median over 40 mystery environments
from Table~\ref{DetailedResultsTable} (top part) and on 40 novel environments with varying fraction of random examples (bottom parts), where each world is run with 10 random initializations and taking the best performance.
Accuracies refer to big regions only.
\label{ResultsSummaryTable}
}
\end{table}

\begin{figure}[phbt]
\centerline{\includegraphics[width=86mm]{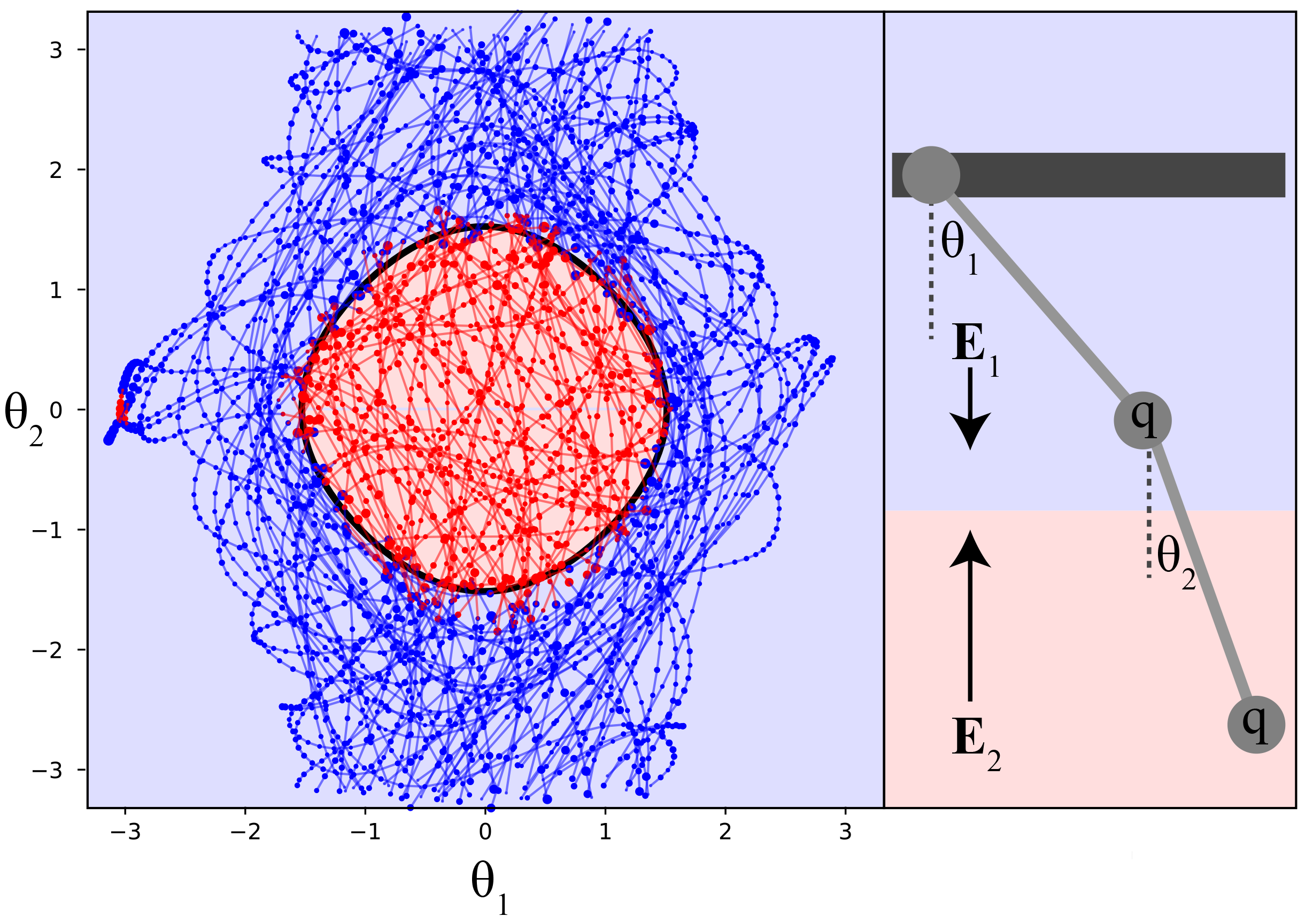}}
\vskip-5mm
\caption{In this mystery, a charged double pendulum moves through two different electric fields $\E_1$ and $\E_2$, with a domain boundary corresponding to $\cos\theta_1+\cos\theta_2=1.05$ (the black curve above left, where the lower charge crosses the $\E$-field boundary).
The color of each dot represents the domain into which it is classified by a Newborn agent, and its area represents the description length of the error with which its position is predicted, for a precision floor $\epsilon\approx 0.006$. In this world, the Newborn agent has a domain prediction accuracy of 96.5\%.
\label{PendulumFig}
}
\end{figure}

Our double-pendulum mysteries (Appendix~\ref{appendix:double_pendulum}) are more challenging for all the agents, because the motion is more nonlinear and indeed chaotic.
Although none of our double-pendulum mysteries get exactly solved according to our very stringent above-mentioned criterion, \fig{PendulumFig}
 illustrates that the Newborn agent does a good job: it discovers the two domains and classifies points into them with an accuracy of 96.5\%. Overall, the Newborn agent has a median best accuracy of 91.0\% compared with the baseline of  76.9\%.
 The MSE prediction error is comparable to the baseline performance ($\sim4\times 10^{-4})$ in the median, since both architectures have similar large capacity.
 We analyze this challenge and opportunities for improvement below.

\section{Conclusions}

We have presented a toy ``AI Physicist" unsupervised learning agent 
centered around the learning and manipulation of theories, which in polynomial time learns to parsimoniously predict both aspects of the future (from past observations) and the domain in which these predictions are accurate.

\subsection{Key findings}

Testing it on a suite of mystery worlds involving random combinations of gravity, electromagnetism, harmonic motion and elastic bounces, we found that its divide-and-conquer and Occam's razor strategies effectively identified domains with different laws of motion and reduced the mean-squared prediction error billionfold, typically recovering integer and rational theory parameters exactly. 
These two strategies both encouraged prediction functions to specialize: 
the former on the domains they handled best, and the latter on the data points within their domain that they handled best.
Adding the lifelong learning strategy 
greatly accelerated learning in novel environments.

\subsection{What has been learned?}

Returning to the broader context of unsupervised learning from \Sec{IntroSec} raises two important questions: what is the {\it difficulty} of the problems that our AI physicist solved, and what is the {\it generality} of our paradigm?

In terms of {\it difficulty}, our solved physics problems are clearly on the easier part of the spectrum, so if we were to have faced the {\it supervised} learning problem where the different domains were pre-labeled, the domain learning would have been a straightforward classification task and the forecasting task could have been easily solved by a standard feedforward neural network.
Because the real world is generally unlabeled, we instead tackled the more difficult problem where boundaries of multiple domains had to be learned concurrently with the dynamical evolution rules in a fully 
unsupervised fashion. 
The dramatic performance improvement over a traditional neural network seen in Table~\ref{DetailedResultsTable} reflects the power of the divide-and-conquer and Occam's razor strategies, and their robustness is indicated by the the fact that unsupervised domain discovery worked well even for the two-field non-linear double-pendulum system whose dynamic is notoriously chaotic and whose domain boundary is the curved rhomboid $\cos\theta_1 + \cos\theta_2 = 1.05$.

In terms of {\it generality}, our core contribution lies in the AI physicist paradigm we propose (combining divide-and-conquer, Occam’s razor, unification and lifelong learning), not in the specific implementation details.
Here we draw inspiration from the history of the Turing Machine: Turing's initial implementation of a universal computer was very inefficient for all but toy problems, but his framework laid out the essential architectural components that subsequent researchers developed into today's powerful computers.
What has been learned is that our AI physicist paradigm outperforms traditional deep learning
on a test suite of problems even though it is a fully general paradigm that is was not designed specifically for these problems. For example, it is defined to work for an arbitrary number of input spatial dimensions, spatial domains, past time steps used, boundaries of arbitrary shapes, and evolution laws of arbitrary complexity.

From the above-mentioned successes and failures of our paradigm, we have also learned about promising opportunities for improvement of the implementation which we will now discuss.
First of all, the more modest success in the double-pendulum experiments illustrated the value of  learned theories being {\it simple}: if they are highly complex, they are less likely to unify or generalize to future environments, and the correspondingly complex baseline model will have less incentive to specialize because it has enough expressive power to approximate the motion in all domains at once. 
It will therefore be valuable to improve techniques for simplifying complex learned neural nets.
The specific implementation details for the Occam's Razor toolkit would then change, but the principle and numerical objective would remain the same: reducing their total description length from \eq{description_length}.
There are many promising opportunities for this using techniques from the Monte-Carlo-Markov-Chain-based and genetic techniques \cite{real2017large}, reinforcement learning \cite{zoph2016neural,baker2016designing} and symbolic regression \cite{schmidt2009distilling,udrescu2019ai} 
literature to simplify and shrink the model architecture.
Also, it will be valuable and straightforward to generalize our implementation to simplify not only the prediction functions, but also the classifiers, for example to find sharp domain boundaries composed of hyperplanes or other simple surfaces.

Analogously, there are many ways in which the unification and life-long learning toolkits can be improved while staying within our AI physicist paradigm. For example, unification can undoubtedly be improved by using more sophisticated clustering techniques for grouping the learned theories with similar ones. Life-long learning can probably be made more efficient by using better methods for determining which previous theories to try when faced by new data, for example by training a separate neural network to perform this prediction task.

\subsection{Outlook}
In summary, these and other improvements to the algorithms that implement our AI Physicist paradigm could enable future unsupervised learning agents to learn simpler and more accurate models faster from fewer examples, and also to discover accurate symbolic expressions for more complicated physical systems. More broadly, AI has been used with great success to tackle problems in diverse areas of physics,
ranging from quantum state reconstruction \cite{carrasquilla2019reconstructing} to phase transitions \cite{carrasquilla2017machine,wang2016discovering,van2017learning}, planetary dynamics \cite{lam2018machine} and particle physics \cite{baldi2014searching}. 
We hope that building on the ideas of this paper may
one day enable AI to help us discover entirely novel physical theories from data.

  


{\bf Acknowledgements:} 
This work was supported by the The Casey and Family Foundation, the Ethics and Governance of AI Fund, the Foundational Questions Institute and the Rothberg Family Fund for Cognitive Science. We thank Isaac Chuang, John Peurifoy and Marin Solja\v{c}i\'c for helpful discussions and suggestions,  and the Center for Brains, Minds, and Machines (CBMM) for hospitality.

\chapter{Learnability phase transition: onset of learning}
\label{chap3:IB}

As I have stated in the Introduction (Chapter \ref{chap1:introduction}), the ability to solve a universal two-term tradeoff: a term  on task performance and a term on controlling complexity, provides an important perspective in intelligence. When the relative strength between the two terms vary, how do the learning bahave? In this chapter, and Chapter \ref{chap4:IB_phase_transition} and Chapter \ref{chap5:distillation}, we set out to address this question. Specifically, we study  the two-term tradeoff in the Information Bottleneck paradigm \cite{tishby2000information}, whose objective is based on information, and provides an insightful and principled approach for balancing compression and prediction for representation learning. This chapter and Chapter \ref{chap4:IB_phase_transition} will focus on the phase transitions, which like the phase transitions in physics, have key quantities for the system changing in a discontinuous way. Chapter \ref{chap5:distillation} will directly study the two-term tradeoff curve in the binary classification cases, where we introduce a method for directly compute the Pareto frontier.

In this chapter\footnote{This extended version is published in \emph{Entropy} 2019, \emph{21}(10), 924, \href{https://www.mdpi.com/1099-4300/21/10/924}{``Learnability for the information bottleneck''}, Wu, Tailin,
Ian Fischer, Isaac Chuang, Max Tegmark \cite{wu2019learnabilityEntropy}. Also \href{Association for Uncertainty in Artificial Intelligence
}{published} at Conference on Uncertainty in Artificial Intelligence (\emph{UAI} 2019) and \href{https://openreview.net/forum?id=SJePKo5HdV}{presented} at \emph{ICLR} 2019 LLD workshop as spotlight. \href{https://arxiv.org/abs/1907.07331}{arXiv: 1907.07331}.
}, we focus on  understanding the learnability transition in IB.
The IB objective $I(X;Z)-\beta I(Y;Z)$ employs a Lagrange multiplier $\beta$ to tune the trade-off between compression and prediction. 
However, in practice, not only is $\beta$ chosen empirically without theoretical guidance, there is also a lack of theoretical understanding between $\beta$, learnability, the intrinsic nature of the dataset and model capacity. 
In this chapter, we show that if $\beta$ is improperly chosen, learning cannot happen -- the trivial representation $P(Z|X)=P(Z)$ becomes the global minimum of the IB objective. 
We show how this can be avoided, by identifying a sharp phase transition between the unlearnable and the learnable which arises as $\beta$ is varied.
This phase transition defines the concept of IB-Learnability.
We prove several sufficient conditions for IB-Learnability, which provides theoretical guidance for choosing a good $\beta$. 
We further show that IB-learnability is determined by the largest \emph{confident}, \emph{typical}, and \emph{imbalanced subset} of the examples (the \emph{conspicuous subset}),
and discuss its relation with model capacity. 
We give practical algorithms to estimate the minimum $\beta$ for a given dataset. 
We also empirically demonstrate our theoretical conditions with analyses of synthetic datasets, MNIST, and CIFAR10.

\section{Introduction}
\label{sec:introduction}

\citet{tishby2000information} introduced the \textit{Information Bottleneck} (IB) objective function which learns a representation $Z$ of observed variables $(X,Y)$ that retains as little information about $X$ as possible, but simultaneously captures as much information about $Y$ as possible:
\begin{equation}
\label{eq:IB_beta}
\min \IB_\beta(X,Y;Z) = \min [I(X;Z) - \beta I(Y;Z)]
\end{equation}
$I(\cdot)$ is the mutual information.
The hyperparameter $\beta$ controls the trade-off between compression and prediction, in the same spirit as Rate-Distortion Theory~\citep{shannon}, but with a learned representation function $P(Z|X)$ that automatically captures some part of the ``semantically meaningful'' information, where the semantics are determined by the observed relationship between $X$ and $Y$. The IB framework has been extended to and extensively studied in a variety of scenarios, including Gaussian variables \citep{chechik2005information}, meta-Gaussians \citep{rey2012meta}, continuous variables via variational methods \citep{alemi2016deep,chalk2016relevant,fischer2018the}, deterministic scenarios \citep{strouse2017deterministic,kolchinsky2018caveats}, geometric clustering \citep{strouse2017information}, and is used for learning invariant and disentangled representations in deep neural nets \citep{achille2018emergence,achille2018information}. 

From the IB objective (Eq. \ref{eq:IB_beta}) we see that when $\beta\to0$ it will encourage $I(X;Z)=0$ which leads to a trivial representation $Z$ that is independent of $X$, while when $\beta\to+\infty$, it reduces to a maximum likelihood objective (e.g. in classification, it reduces to cross-entropy loss).
Therefore, as we vary $\beta$ from $0$ to $+\infty$, there must exist a point $\beta_0$ at which IB starts to learn a nontrivial representation where $Z$ contains information about $X$.

As an example, we train multiple variational information bottleneck  (VIB) models on binary classification of MNIST \cite{mnist} digits 0 and 1 with 20\% label noise at different $\beta$.
The accuracy vs. $\beta$ is shown in Fig. \ref{fig:mnist_0.2}. 
We see that when $\beta < 3.25$, no learning happens, and the accuracy is the same as random guessing.
Beginning with $\beta > 3.25$, there is a clear phase transition where the accuracy sharply increases, indicating the objective is able to learn a nontrivial representation.
In general, we observe that different datasets and model capacity will result in different $\beta_0$ at which IB starts to learn a nontrivial representation.
How does $\beta_0$ depend on the aspects of the dataset and model capacity, and how can we estimate it?
What does an IB model learn at the onset of learning?
Answering these questions may provide a deeper understanding of IB in particular, and learning on two observed variables in general.

In this work, we begin to answer the above questions. Specifically:

\begin{itemize}
\item We introduce the concept of \textit{IB-Learnability}, and show that when we vary $\beta$, the IB objective will undergo a phase transition from the inability to learn to the ability to learn (Section~\ref{sec:learnability}).
\item Using the second-order variation, we derive sufficient conditions for IB-Learnability, which provide upper bounds for the learnability threshold $\beta_0$  (Section~\ref{sec:suff_conditions}).
\item We show that IB-Learnability is determined by the largest \emph{confident}, \emph{typical}, and \emph{imbalanced subset} of the examples (the \textit{conspicuous subset}), reveal its relationship with the slope of the Pareto frontier at the origin on the information plane $I(X;Z)$ vs. $I(Y;Z)$, and discuss its relation to model capacity (Section~\ref{sec:discussion}).
\item We prove a deep relationship between IB-Learnability, our upper bounds on $\beta_0$, the hypercontractivity coefficient, the contraction coefficient, and the maximum correlation (Section~\ref{sec:discussion}).
\end{itemize}

\begin{figure}[t]
\begin{center}
\includegraphics[width=0.5\columnwidth]{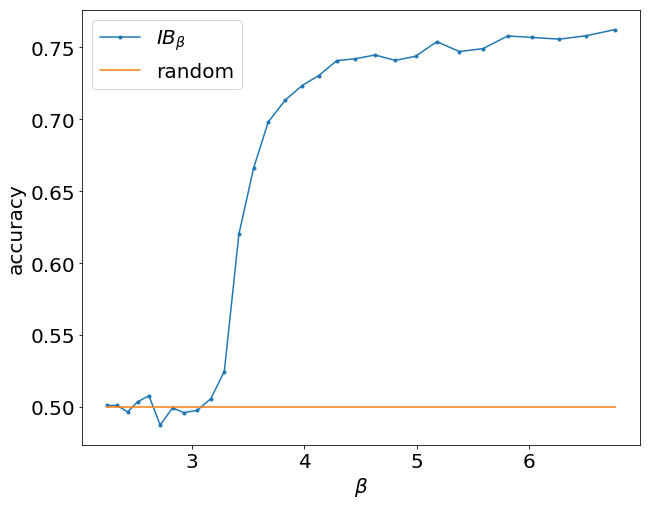}
\end{center}
\caption{Accuracy for binary classification of MNIST digits 0 and 1 with 20\% label noise and varying $\beta$. No learning happens for models trained at $\beta < 3.25$.} 
\label{fig:mnist_0.2}
\end{figure}

We also present an algorithm for estimating the onset of IB-Learnability and the conspicuous subset, which provide us with a tool for understanding a key aspect of the learning problem $(X,Y)$ (Section~\ref{sec:estimate}).

Finally, we use our main results to demonstrate on synthetic datasets, MNIST~\citep{mnist} and CIFAR10 \citep{cifar} that the theoretical prediction for IB-Learnability closely matches experiment, and show the conspicuous subset our algorithm discovers (Section~\ref{sec:experiments}).

\section{Related Work}
\label{sec:related_work}

The seminal IB work~\citep{tishby2000information} provides a tabular method for exactly computing the optimal encoder distribution $P(Z|X)$ for a given $\beta$ and cardinality of the discrete representation, $|Z|$. They did not consider the IB learnability problem as addressed in this work.

\citet{chechik2005information} presents the Gaussian Information Bottleneck (GIB) for learning a multivariate Gaussian representation $Z$ of $(X,Y)$, assuming that both $X$ and $Y$ are also multivariate Gaussians. Under GIB, they derive analytic formula for the optimal representation as a noisy linear projection to eigenvectors of the normalized regression matrix $\Sigma_{x|y}\Sigma_x^{-1}$, and the learnability threshold $\beta_0$ is then given by $\beta_0=\frac{1}{1-\lambda_1}$ where $\lambda_1$ is the largest eigenvalue of the matrix $\Sigma_{x|y}\Sigma_x^{-1}$. This work provides deep insights about relations between the dataset, $\beta_0$ and optimal representations in the Gaussian scenario, but the restriction to multivariate Gaussian datasets limits the generality of the analysis

Another analytic treatment of IB is given in~\cite{rey2012meta}, which reformulates the objective in terms of the copula functions.
As with the GIB approach, this formulation restricts the form of the data distributions -- the copula functions for the joint distribution $(X,Y)$ are assumed to be known, which is unlikely in practice.

\citet{strouse2017deterministic} presents the Deterministic Information Bottleneck (DIB), which minimizes the coding cost of the representation, $H(Z)$, rather than the transmission cost, $I(X;Z)$ as in IB.
This approach learns hard clusterings with different code entropies that vary with $\beta$.
In this case, it is clear that a hard clustering with minimal $H(Z)$ will result in a single cluster for all of the data, which is the DIB trivial solution.
No analysis is given beyond this fact to predict the actual onset of learnability, however.

The first amortized IB objective is in the Variational Information Bottleneck (VIB) of~\citet{alemi2016deep}.
VIB replaces the exact, tabular approach of IB with variational approximations of the classifier distribution ($P(Y|Z)$) and marginal distribution ($P(Z)$).
This approach cleanly permits learning a stochastic encoder, $P(Z|X)$, that is applicable to any $x \in \mathcal{X}$, rather than just the particular $X$ seen at training time.
The cost of this flexibility is the use of variational approximations that may be less expressive than the tabular method.
Nevertheless, in practice, VIB learns easily and is simple to implement, so we rely on VIB models for our experimental confirmation.

Closely related to IB is the recently proposed Conditional Entropy Bottleneck (CEB)~\citep{fischer2018the}.
CEB attempts to explicitly learn the Minimum Necessary Information (MNI), defined as the point in the information plane where $I(X;Y) = I(X;Z) = I(Y;Z)$.
The MNI point may not be achievable even in principle for a particular dataset.
However, the CEB objective provides an explicit estimate of how closely the model is approaching the MNI point by observing that a necessary condition for reaching the MNI point occurs when $I(X;Z|Y) = 0$.
The CEB objective $I(X;Z|Y) - \gamma I(Y;Z)$ is equivalent to IB at $\gamma = \beta + 1$, so our analysis of IB-Learnability applies equally to CEB.

\citet{kolchinsky2018caveats} 
shows that when $Y$ is a deterministic function of X, the ``corner point'' of the IB curve (where $I(X;Y)=I(X;Z)=I(Y;Z)$) is the unique optimizer of the IB objective for all $0<\beta'<1$ (with the parameterization of \citet{kolchinsky2018caveats}, $\beta'=1/\beta$), which they consider to be a ``trivial solution''.

However, their use of the term ``trivial solution'' is distinct from ours.
They are referring to the observation that all points on the IB curve contain uninteresting interpolations between two different but valid solutions on the optimal frontier, rather than demonstrating a non-trivial trade-off between compression and prediction as expected when varying the IB Lagrangian.
Our use of ``trivial'' refers to whether IB is capable of learning at all given a certain dataset and value of $\beta$.

\citet{achille2018information} apply the IB Lagrangian to the weights of a neural network, yielding InfoDropout.
In~\citet{achille2018emergence}, the authors give a deep and compelling analysis of how the IB Lagrangian can yield invariant and disentangled representations.
They do not, however, consider the question of the onset of learning, although they are aware that not all models will learn a non-trivial representation.
More recently, \citet{achille2018dynamics} repurpose the InfoDropout IB Lagrangian as a Kolmogorov Structure Function to analyze the ease with which a previously-trained network can be fine-tuned for a new task.
While that work is tangentially related to learnability, the question it addresses is substantially different from our investigation of the onset of learning.

Our work is also closely related to the hypercontractivity coefficient \citep{anantharam2013maximal,polyanskiy2017strong}, defined as $\sup_{Z-X-Y}\frac{I(Y;Z)}{I(X;Z)}$, which by definition equals the inverse of $\beta_0$, our IB-learnability threshold.
In \cite{anantharam2013maximal}, the authors prove that the hypercontractivity cofficient equals the contraction coefficient $\eta_{\KL}(P_{Y|X}, P_X)$, and \citet{kim2017discovering} propose a practical algorithm to estimate $\eta_{\KL}(P_{Y|X},P_X)$, which provides a measure for potential influence in the data.
Although our goal is different, the sufficient conditions we provide for IB-Learnability are also lower bounds for the hypercontractivity coefficient.

\section{IB-Learnability}
\label{sec:learnability}

We are given instances of $(x,y)$ drawn from a distribution with probability (density) $P(X,Y)$ with support of $\mathcal{X}\times\mathcal{Y}$, where unless otherwise stated, both $X$ and $Y$ can be discrete or continuous variables.
$(X,Y)$ is our \textit{training data}, and may be characterized by different types of noise.
The nature of this training data and the choice of $\beta$ will be sufficient to predict the transition from unlearnable to learnable.

We can learn a representation $Z$ of $X$ with conditional probability\footnote{%
  We use capital letters $X,Y,Z$ for random variables and lowercase $x,y,z$ to denote the instance of variables, with $P(\cdot)$ and $p(\cdot)$ denoting their probability or probability density, respectively.
}
$p(z|x)$, such that $X,Y,Z$ obey the Markov chain $Z\gets X\leftrightarrow Y$.
Eq.~\ref{eq:IB_beta} above gives the IB objective with Lagrange multiplier $\beta$, $\IB_\beta(X,Y;Z)$, which is a functional of $p(z|x)$: $\IB_\beta(X,Y;Z)=\IB_\beta[p(z|x)]$.
The IB learning task is to find a conditional probability $p(z|x)$ that minimizes $\IB_\beta(X,Y;Z)$.
The larger $\beta$, the more the objective favors making a good prediction for $Y$.
Conversely, the smaller $\beta$, the more the objective favors learning a concise representation.

How can we select $\beta$ such that the IB objective learns a useful representation?
In practice, the selection of $\beta$ is done empirically.
Indeed,~\citet{tishby2000information} recommends ``sweeping $\beta$''.
In this paper, we provide theoretical guidance for choosing $\beta$ by introducing the concept of $\IB$-Learnability and providing a series of $\IB$-learnable conditions.

\begin{definition}
\label{def:learnable}
$(X,Y)$ is $\IB_\beta$-learnable if there exists a $Z$ given by some $p_1(z|x)$, such that $\IB_\beta(X,Y;Z)\rvert_{p_1(z|x)} < \IB_\beta(X,Y;Z)\rvert_{p(z|x)=p(z)}$, where $p(z|x)=p(z)$ characterizes the trivial representation where 
$Z=Z_\text{trivial}$ is independent of $X$.
\end{definition}

If $(X;Y)$ is $\IB_\beta$-learnable, then when $\IB_\beta(X,Y;Z)$ is globally minimized, it will \emph{not} learn a trivial representation.
On the other hand, if $(X;Y)$ is not $\IB_\beta$-learnable, then when $\IB_\beta(X,Y;Z)$ is globally minimized, it may learn a trivial representation.

\paragraph{Trivial solutions.}
Definition~\ref{def:learnable} defines trivial solutions in terms of representations where $I(X;Z) = I(Y;Z) = 0$.
Another type of trivial solution occurs when $I(X;Z) > 0$ but $I(Y;Z) = 0$.
This type of trivial solution is not directly achievable by the IB objective, as $I(X;Z)$ is minimized, but it can be achieved by construction or by chance.
It is possible that starting learning from $I(X;Z) > 0, I(Y;Z) = 0$ could result in access to non-trivial solutions not available from $I(X;Z) = 0$.
We do not attempt to investigate this type of trivial solution in this work.

\paragraph{Necessary condition for IB-Learnability.}
From Definition~\ref{def:learnable}, we can see that $\IB_\beta$-Learnability for any dataset $(X;Y)$ requires $\beta > 1$.
In fact, from the Markov chain $Z\gets X\leftrightarrow Y$, we have $I(Y;Z) \le I(X;Z)$ via the data-processing inequality.
If $\beta \le 1$, then since $I(X;Z) \ge 0$ and $I(Y;Z) \ge 0$, we have that $\min(I(X;Z) - \beta I(Y;Z))=0 = \IB_\beta(X,Y;Z_{trivial})$.
Hence $(X,Y)$ is not $\IB_\beta$-learnable for $\beta \le 1$.

Due to the reparameterization invariance of mutual information, we have the following theorem for $\IB_\beta$-Learnability:

\begin{lemma}
\label{thm:homo_learnability}
Let $X'=g(X)$ be an invertible map (if $X$ is a continuous variable, $g$ is additionally required to be continuous). Then $(X,Y)$ and $(X',Y)$ have the same $\IB_\beta$-Learnability.
\end{lemma}

The proof for Lemma \ref{thm:homo_learnability} is in Appendix \ref{app:homo_learnability}.
Lemma \ref{thm:homo_learnability} implies a favorable property for any condition for $\IB_\beta$-Learnability: the condition should be invariant to invertible mappings of $X$.
We will inspect this invariance in the conditions we derive in the following sections.

\section{Sufficient conditions for IB-Learnability}
\label{sec:suff_conditions}

Given $(X,Y)$, how can we determine whether it is $\IB_\beta$-learnable?
To answer this question, we derive a series of sufficient conditions for $\IB_\beta$-Learnability, starting from its definition.
The conditions are in increasing order of practicality, while sacrificing as little generality as possible.

Firstly, Theorem \ref{thm:beta_monotonic} characterizes the $\IB_\beta$-Learnability range for $\beta$, with proof in Appendix \ref{app:beta_monotonic}:

\begin{theorem}
\label{thm:beta_monotonic}
If $(X,Y)$ is $\IB_{\beta_1}$-learnable, then for any $\beta_2>\beta_1$, it is $\IB_{\beta_2}$-learnable.
\end{theorem}

Based on Theorem \ref{thm:beta_monotonic}, the range of $\beta$ such that $(X,Y)$ is $\IB_\beta$-learnable has the form $\beta \in (\beta_0, +\infty)$.
Thus, $\beta_0$ is the \textit{threshold} of IB-Learnability.

\begin{lemma}
\label{lemma:stationary}
$p(z|x)=p(z)$ is a stationary solution for $\IB_\beta(X,Y;Z)$.
\end{lemma}
The proof in Appendix \ref{app:stationary} shows that both first-order variations
$\delta I(X;Z)=0$ and $\delta I(Y;Z)=0$ vanish at the trivial representation $p(z|x)=p(z)$, so
$\delta \IB_\beta[p(z|x)] = 0$ at the trivial representation.

Lemma \ref{lemma:stationary} yields our strategy for finding sufficient conditions for learnability: find conditions such that $p(z|x)=p(z)$ is not a local minimum for the functional $\IB_\beta[p(z|x)]$. 
Based on the necessary condition for the minimum (Appendix \ref{app:suff_1}), we have the following theorem \footnote{The theorems in this paper deal with learnability w.r.t. true mutual information. If parameterized models are used to approximate the mutual information, the limitation of the model capacity will translate into more uncertainty of $Y$ given $X$, viewed through the lens of the model.}:

\begin{theorem}[\textbf{Suff. Cond. 1}]
\label{thm:suff_1}
A sufficient condition for $(X, Y)$ to be $\IB_\beta$-learnable is that there exists a perturbation function\footnote{so that the perturbed probability (density) is $p'(z|x)=p(z|x)+\epsilon\cdot h(z|x)$. Also, for integrals, whenever a variable $W$ is discrete, we can simply replace the integral $(\int \cdot dw)$ by summation $(\sum_w\cdot)$.} $h(z|x)$ with
$\int h(z|x)dz = 0$, such that the second-order variation $\delta^2 \IB_\beta[p(z|x)] < 0$ at the trivial representation $p(z|x)=p(z)$.
\end{theorem}

The proof for Theorem  \ref{thm:suff_1} is given in Appendix \ref{app:suff_1}.
Intuitively, if $\delta^2 \IB_\beta[p(z|x)]\big\rvert_{p(z|x)=p(z)} < 0$, we can always find a $p'(z|x) = p(z|x) + \epsilon \cdot h(z|x)$ in the neighborhood of the trivial representation $p(z|x)=p(z)$, such that $\IB_\beta[p'(z|x)] < \IB_\beta[p(z|x)]$, thus satisfying the definition for $\IB_\beta$-Learnability.

To make Theorem \ref{thm:suff_1} more practical, we perturb $p(z|x)$ around the trivial solution $p'(z|x) = p(z|x) + \epsilon \cdot h(z|x)$, and expand $\IB_\beta[p(z|x) + \epsilon\cdot h(z|x)] - \IB_\beta[p(z|x)]$ to the second order of $\epsilon$.
We can then prove Theorem \ref{thm:suff_2}:

\begin{theorem}[\textbf{Suff. Cond. 2}]
\label{thm:suff_2}
A sufficient condition for $(X,Y)$ to be $\IB_\beta$-learnable is $X$ and $Y$ are not independent, and
\begin{equation}
\begin{aligned}
\label{eq:suff_2}
\beta > \inf_{h(x)}
\beta_0[h(x)]
\end{aligned}
\end{equation}

where the functional $\beta_0[h(x)]$ is given by
$$\beta_0[h(x)]=\frac{\E_{x \sim p(x)} [h(x)^2] - \left(\E_{x\sim p(x)} [h(x)]\right)^2}{\E_{y \sim p(y)}\left[\left(\E_{x \sim p(x|y)} [h(x)]\right)^2\right] - \left(\E_{x\sim p(x)} [h(x)]\right)^2}$$

Moreover, we have that $\left(\inf_{h(x)}\beta[h(x)]\right)^{-1}$ is a lower bound of the slope of the Pareto frontier in the information plane $I(Y;Z)$ vs. $I(X;Z)$ at the origin.
\end{theorem}

The proof is given in Appendix \ref{app:suff_2}, which also shows that if $\beta>\inf_{h(x)}\beta_0[h(x)]$ in Theorem \ref{thm:suff_2} is satisfied, we can construct a perturbation function $h(z|x)=h^*(x)h_2(z)$ with $h^*(x)=\argmin_{h(x)}\beta_0[h(x)]$, $\int h_2(z)dz=0, \int \frac{h_2^2(z)}{p(z)}dz>0$ for some $h_2(z)$, such that $h(z|x)$ satisfies Theorem \ref{thm:suff_1}. 
It also shows that the converse is true: if there exists $h(z|x)$ such that the condition in Theorem \ref{thm:suff_1} is true, then Theorem \ref{thm:suff_2} is satisfied\footnote{%
    We do not claim that any $h(z|x)$ satisfying Theorem~\ref{thm:suff_1} can be decomposed to $h^*(x)h_2(z)$ at the onset of learning.
    But from the equivalence of Theorems \ref{thm:suff_1} and \ref{thm:suff_2} as explained above, when there exists an $h(z|x)$ such that Theorem \ref{thm:suff_1} is satisfied, we can always construct an $h'(z|x)=h^*(x)h_2(z)$ that also satisfies Theorem \ref{thm:suff_1}.
}, i.e. $\beta>\inf_{h(x)}\beta_0[h(x)]$.
Moreover, letting the perturbation function $h(z|x)=h^*(x)h_2(z)$ at the trivial solution, we have

\begin{equation}
\begin{aligned}
\label{eq:what_first_learns}
p_\beta(y|x) = p(y) + \epsilon^2 C_z (h^*(x)-\overline{h}^*_x) \int p(x,y)(h^*(x)-\overline{h}^*_x)dx
\end{aligned}
\end{equation}

where $p_\beta(y|x)$ is the estimated $p(y|x)$ by IB for a certain $\beta$, $\overline{h}^*_x=\int h^*(x)p(x)dx$, and $C_z=\int\frac{h_2^2(z)}{p(z)}dz>0$ is a constant.
This shows how the $p_\beta(y|x)$ by IB explicitly depends on $h^*(x)$ at the onset of learning.
The proof is provided in Appendix~\ref{app:what_first_learns}.

Theorem \ref{thm:suff_2} suggests a method to estimate $\beta_0$: we can parameterize $h(x)$ e.g. by a neural network, with the objective of minimizing $\beta_0[h(x)]$.
At its minimization, $\beta_0[h(x)]$ provides an upper bound for $\beta_0$,  and $h(x)$ provides a \emph{soft clustering} of the examples corresponding to a nontrivial perturbation of $p(z|x)$ at $p(z|x)=p(z)$ that minimizes $\IB_\beta[p(z|x)]$. 

Alternatively, based on the property of $\beta_0[h(x)]$, we can also use a specific functional form for $h(x)$ in Eq.~(\ref{eq:suff_2}), and obtain a stronger sufficient condition for $\IB_\beta$-Learnability.
But we want to choose $h(x)$ as near to the infimum as possible.
To do this, we note the following characteristics for the R.H.S of Eq.~(\ref{eq:suff_2}):

\begin{itemize}
\item We can set $h(x)$ to be nonzero if $x \in \Omega_x$ for some region $\Omega_x\subset \X$ and 0 otherwise.
Then we obtain the following sufficient condition:
\begin{equation}
\begin{aligned}
\label{eq:suff_2_omega}
\beta>\inf_{h(x),\Omega_x\subset\X}\frac{\frac{\E_{x \sim p(x), x \in \Omega_x} [h(x)^2]}{\left(\E_{x \sim p(x), x \in \Omega_x} [h(x)]\right)^2} - 1}{\int\frac{dy}{p(y)}\left(\frac{\E_{x \sim p(x), x \in \Omega_x} [p(y|x) h(x)]}{\E_{x \sim p(x), x \in \Omega_x} [h(x)]}\right)^2 - 1}
\end{aligned}
\end{equation}

\item The numerator of the R.H.S. of Eq.~(\ref{eq:suff_2_omega}) attains its minimum when $h(x)$ is a constant within $\Omega_x$.
This can be proved using the Cauchy-Schwarz inequality: $\langle u,u \rangle \langle v,v \rangle \geq \langle u,v \rangle^2$, setting $u(x) = h(x) \sqrt{p(x)}$, $v(x) = \sqrt{p(x)}$, and defining the inner product as $\langle u,v \rangle = \int u(x) v(x) dx$.
Therefore, the numerator of the R.H.S. of Eq.~(\ref{eq:suff_2_omega}) $\ge \frac{1}{\int_{x \in \Omega_x} p(x)} - 1$, and attains equality when $\frac{u(x)}{v(x)} = h(x)$ is constant.
\end{itemize}

Based on these observations, we can let $h(x)$ be a nonzero constant inside some region $\Omega_x\subset\X$ and 0 otherwise, and the infimum over an arbitrary function $h(x)$ is simplified to infimum over $\Omega_x\subset\X$, and we obtain a sufficient condition for $\IB_\beta$-Learnability, which is a key result of this paper:

\begin{theorem}[\textbf{Conspicuous Subset Suff. Cond.}]
\label{thm:suff_3}
A sufficient condition for $(X,Y)$ to be $\IB_\beta$-learnable is $X$ and $Y$ are not independent, and
\begin{equation}
\begin{aligned}
\label{eq:suff_3}
\beta > \inf_{\Omega_x\subset \mathcal{X}}\beta_0(\Omega_x)
\end{aligned}
\end{equation}
where 
$$\beta_0(\Omega_x)=\frac{\frac{1}{p(\Omega_x)} - 1}{\E_{y \sim p(y|\Omega_x)} \left[ \frac{p(y|\Omega_x)}{p(y)} - 1 \right]}$$
$\Omega_x$ denotes the event that $x \in \Omega_x$, with probability $p(\Omega_x)$. 

$\left(\inf_{\Omega_x\subset\X}\beta_0(\Omega_x)\right)^{-1}$ gives a lower bound of the slope of the Pareto frontier in the information plane $I(Y;Z)$ vs. $I(X;Z)$ at the origin.
\end{theorem}

The proof is given in Appendix \ref{app:suff_3}. In the proof we also show that this condition is invariant to invertible mappings of $X$.

\section{Discussion}
\label{sec:discussion}

\subsection{The conspicuous subset determines \texorpdfstring{$\beta_0$}{Lg}.}

From Eq. (\ref{eq:suff_3}), we see that three characteristics of the subset $\Omega_x\subset\X$ lead to low $\beta_0$:
\textbf{(1) confidence:} $p(y|\Omega_x)$ is large;
\textbf{(2) typicality and size:} the number of elements in $\Omega_x$ is large, or the elements in $\Omega_x$ are typical, leading to a large probability of $p(\Omega_x)$;
\textbf{(3) imbalance:} $p(y)$ is small for the subset $\Omega_x$, but large for its complement.
In summary, $\beta_0$ will be determined by the largest \emph{confident}, \emph{typical} and \emph{imbalanced subset} of examples, or an equilibrium of those characteristics. We term $\Omega_x$ at the minimization of $\beta_0(\Omega_x)$ the \emph{conspicuous subset}.

\subsection{Multiple phase transitions.}
Based on this characterization of $\Omega_x$, we can hypothesize datasets with multiple learnability phase transitions.
Specifically, consider a region $\Omega_{x0}$ that is small but ``typical'', consists of all elements confidently predicted as $y_0$ by $p(y|x)$, and where $y_0$ is the least common class.
By construction, this $\Omega_{x0}$ will dominate the infimum in Eq.~(\ref{eq:suff_3}), resulting in a small value of $\beta_0$.
However, the remaining $\mathcal{X} - \Omega_{x0}$ effectively form a new dataset, $\mathcal{X}_1$.
At exactly $\beta_0$, we may have that the current encoder, $p_0(z|x)$, has no mutual information with the remaining classes in $\mathcal{X}_1$; i.e., $I(Y_1;Z_0) = 0$.
In this case, Definition~\ref{def:learnable} applies to $p_0(z|x)$ with respect to $I(X_1;Z_1)$.
We might expect to see that, at $\beta_0$, learning will plateau until we get to some $\beta_1 > \beta_0$ that defines the phase transition for $\mathcal{X}_1$.
Clearly this process could repeat many times, with each new dataset $\mathcal{X}_i$ being distinctly more difficult to learn than $\mathcal{X}_{i-1}$.

\subsection{Similarity to information measures.}
The denominator of $\beta_0(\Omega_x)$ in Eq.~(\ref{eq:suff_3}) is closely related to mutual information.
Using the inequality $x - 1 \ge \log(x)$ for $x > 0$, it becomes:
\begin{align*}
\E_{y \sim p(y|\Omega_x)} \bigg[\frac{p(y|\Omega_x)}{p(y)} - 1 \bigg] &\ge \E_{y \sim p(y|\Omega_x)} \bigg[ \log \frac{p(y|\Omega_x)}{p(y)} \bigg] = \tilde{I}(\Omega_x;Y)
\end{align*}

where $\tilde{I}(\Omega_x;Y)$ is the mutual information ``density'' at $\Omega_x\subset\X$.
Of course, this quantity is also $\mathbb{D}_{\KL}[p(y|\Omega_x)||p(y)]$, so we know that the denominator of Eq.~(\ref{eq:suff_3}) is non-negative. Incidentally, $\E_{y \sim p(y|\Omega_x)} \big[\frac{p(y|\Omega_x)}{p(y)} - 1 \big]$ is the density of ``rational mutual information'' (\cite{lin2016criticality}) at $\Omega_x$.

Similarly, the numerator of $\beta_0(\Omega_x)$ is related to the self-information of $\Omega_x$:
$$\frac{1}{p(\Omega_x)} - 1 \ge \log \frac{1}{p(\Omega_x)} = -\log\ p(\Omega_x) = h(\Omega_x)$$
so we can estimate $\beta_0$ as:
\begin{equation}
\label{eq:info_approx}
\beta_0 \simeq \inf_{\Omega_x\subset \X} \frac{h(\Omega_x)}{\tilde{I}(\Omega_x;Y)}
\end{equation}
Since Eq.~(\ref{eq:info_approx}) uses upper bounds on both the numerator and the denominator, it does not give us a bound on $\beta_0$, only an estimate.

\subsection{Estimating model capacity.}
The observation that a model cannot distinguish between cluster overlap in the data and its own lack of capacity gives an interesting way to use IB-Learnability to measure the capacity of a set of models relative to the task they are being used to solve.
For example, for a classification task, we can use different model classes to estimate $p(y|x)$.
For each such trained model, we can estimate the corresponding IB-learnability threshold $\beta_0$.
A model with smaller capacity than the task needs will translate to more uncertainty in $p(y|\Omega_x)$, resulting in a larger $\beta_0$.
On the other hand, models that give the same $\beta_0$ as each other all have the same capacity relative to the task, even if we would otherwise expect them to have very different capacities.
For example, if two deep models have the same core architecture, but one has twice the number of parameters at each layer, and they both yield the same $\beta_0$, their capacities are equivalent with respect to the task.
Thus, $\beta_0$ provides a way to measure model capacity in a task-specific manner.

\subsection{Learnability and the Information Plane.}

\begin{figure}[t]
\begin{center}
\includegraphics[width=0.5\columnwidth]{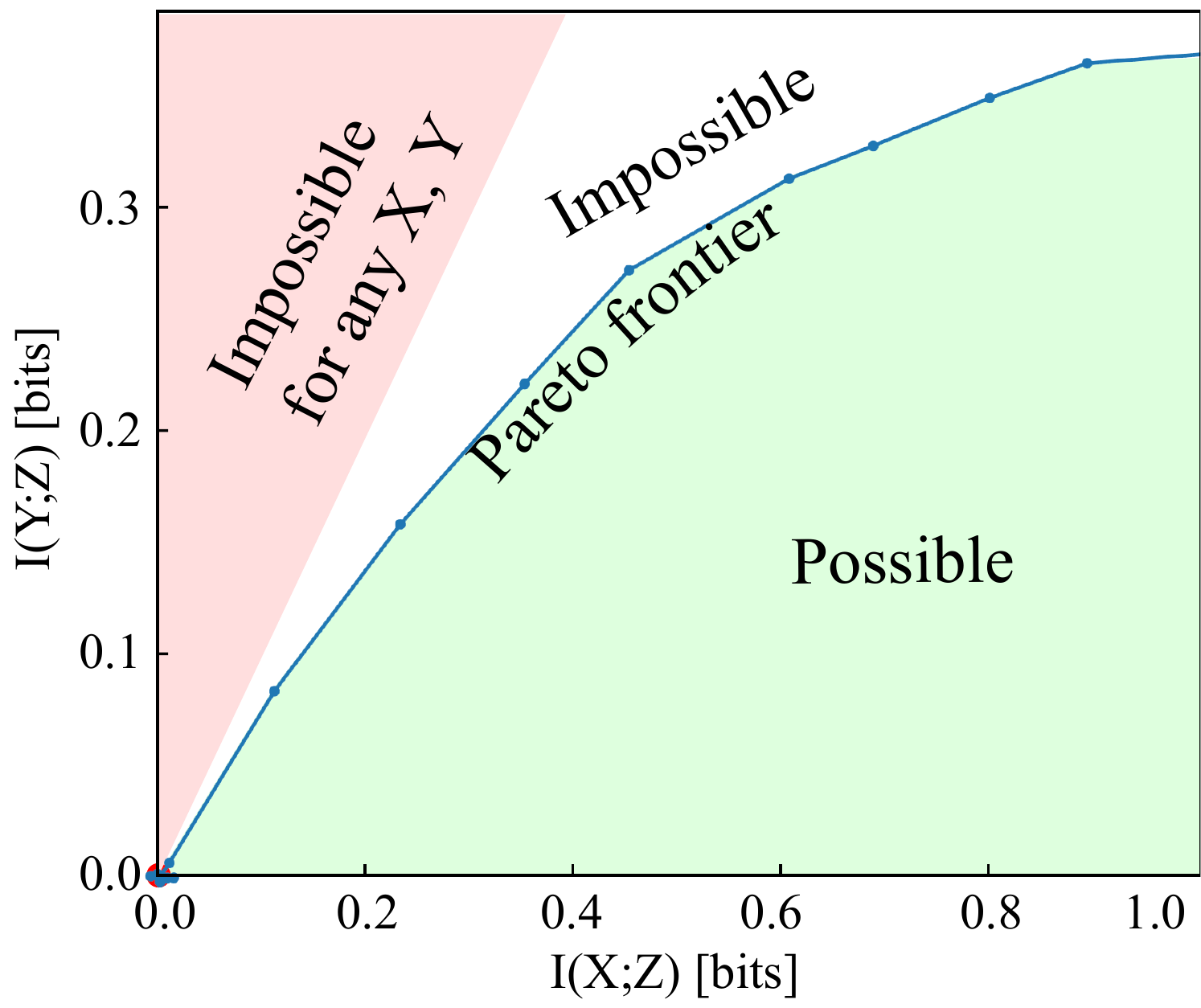}
\end{center}
\caption{The Pareto frontier of the information plane, $I(X;Z)$ vs $I(Y;Z)$, for the binary classification of MNIST digits 0 and 1 with 20\% label noise described in Sec.~\ref{sec:introduction} and Fig.~\ref{fig:mnist_0.2}.
For this problem, learning happens for models trained at $\beta > 3.25$.
$H(Y)=1$ bit since only two of ten digits are used, and $I(Y;Z) \le I(X;Y) \approx 0.5$ bits $<H(Y)$ because of the 20\% label noise.
The true frontier is differentiable; the figure shows a variational approximation that places an upper bound on both informations, horizontally offset to pass through the origin.
}
\label{fig:frontier}
\end{figure}

Many of our results can be interpreted in terms of the geometry of the Pareto frontier illustrated in Fig.~\ref{fig:frontier}, which describes the trade-off between increasing $I(Y;Z)$ and decreasing $I(X;Z)$.
At any point on this frontier that minimizes $\IB_\beta^{\min} \equiv \min I(X;Z) - \beta I(Y;Z)$, the frontier will have slope $\beta^{-1}$ if it is differentiable. If the frontier is also concave (has negative second derivative), then this slope $\beta^{-1}$ will take its maximum $\beta_0^{-1}$ at the origin, which implies
$\IB_\beta$-Learnability for $\beta > \beta_0$, so that the threshold for $\IB_\beta$-Learnability is simply the inverse slope of the frontier at the origin.
More generally, as long as the Pareto frontier is differentiable, the threshold for IB$_\beta$-learnability is the inverse of its maximum slope. Indeed, Theorem \ref{thm:suff_2} and Theorem \ref{thm:suff_3} give lower bounds of the slope of the Pareto frontier at the origin.

\subsection{IB-Learnability, hypercontractivity, and maximum correlation.}

IB-Learnability and its sufficient conditions we provide harbor a deep connection with hypercontractivity and maximum correlation:

\begin{align}
\frac{1}{\beta_0} &= \xi(X;Y)=\eta_\text{KL} \ge \sup_{h(x)}\frac{1}{\beta_0[h(x)]} = \rho_m^2(X;Y) \label{eq:relations_all}
\end{align}

which we prove in Appendix~\ref{app:maximum_corr}.
Here $\rho_m(X;Y)\equiv\max_{f,g} \mathbb{E}[f(X)g(Y)]$ s.t. $\mathbb{E}[f(X)]=\mathbb{E}[g(Y)]=0$ and $\mathbb{E}[f^2(X)]=\mathbb{E}[g^2(Y)]=1$ is the \textit{maximum correlation}~\citep{hirschfeld1935connection,gebelein1941statistische}, $\xi(X;Y)\equiv\sup_{Z-X-Y}\frac{I(Y;Z)}{I(X;Z)}$ is the \textit{hypercontractivity coefficient}, and $\eta_\text{KL}(p(y|x),p(x))\equiv\sup_{r(x)\neq p(x)}\frac{\mathbb{D}_\text{KL}(r(y)||p(y))}{\mathbb{D}_\text{KL}(r(x)||p(x))}$ is the \textit{contraction coefficient}.
Our proof relies on \citet{anantharam2013maximal}'s proof $\xi(X;Y)=\eta_\text{KL}$.
Our work reveals the deep relationship between IB-Learnability and these earlier concepts and provides additional insights about what aspects of a dataset give rise to high maximum correlation and hypercontractivity: the most confident, typical, imbalanced subset of $(X,Y)$.

\section{Estimating the IB-Learnability Condition}
\label{sec:estimate}

Theorem \ref{thm:suff_3} not only reveals the relationship between the learnability threshold for $\beta$ and the least noisy region of $P(Y|X)$, but also provides a way to practically estimate $\beta_0$, both in the general classification case, and in more structured settings.

\subsection{Estimation Algorithm}

Based on Theorem \ref{thm:suff_3}, for general classification tasks we suggest Algorithm \ref{alg:estimating_beta} to empirically estimate an upper-bound $\tilde{\beta}_{0} \ge \beta_0$, as well as discovering the conspicuous subset that determines $\beta_0$.

We approximate the probability of each example $p(x_i)$ by its empirical probability, $\hat{p}(x_i)$.
E.g., for MNIST, $p(x_i) = \frac{1}{N}$, where $N$ is the number of examples in the dataset.
The algorithm starts by first learning a maximum likelihood model of $p_\theta(y|x)$, using e.g. feed-forward neural networks.
It then constructs a matrix $P_{y|x}$ and a vector $p_y$ to store the estimated $p(y|x)$ and $p(y)$ for all the examples in the dataset.
To find the subset $\Omega$ such that the $\tilde{\beta}_0$ is as small as possible, by previous analysis we want to find a \emph{conspicuous} subset such that its $p(y|x)$ is large for a certain class $j$ (to make the denominator of Eq.~(\ref{eq:suff_3}) large), and containing as many elements as possible (to make the numerator small).

We suggest the following heuristics to discover such a conspicuous subset. For each class $j$,
we sort the rows of $(P_{y|x})$ according to its probability for the pivot class $j$ by decreasing order, and then perform a search over $i_\text{left}, i_\text{right}$ for $\Omega=\{i_\text{left},i_\text{left}+1,...,i_\text{right}\}$.
Since $\tilde{\beta}_0$ is large when $\Omega$ contains too few or too many elements, the minimum of $\tilde{\beta}_0^{(j)}$ for class $j$ will typically be reached with some intermediate-sized subset, and we can use binary search or other discrete search algorithm for the optimization.
The algorithm stops when $\tilde{\beta}_0^{(j)}$ does not improve by tolerance $\varepsilon$. The algorithm then returns the $\tilde{\beta}_0$ as the minimum over all the classes $\tilde{\beta}_0^{(1)},...\tilde{\beta}_0^{(N)}$, as well as the conspicuous subset that determines this $\tilde{\beta}_0$.

\begin{algorithm}[h]
\setstretch{1.5}
 \caption{\textbf{Estimating the upper bound for $\beta_0$ and identifying the conspicuous subset.}}
\label{alg:estimating_beta}
\begin{algorithmic}
\STATE {\bfseries Require}: Dataset $\D=\{(x_i,y_i)\},i=1,2,...N$. The number of classes is $C$.
\STATE {\bfseries Require} $\varepsilon$: tolerance for estimating $\beta_0$
\STATE 1: Learn a maximum likelihood model $p_\theta(y|x)$ using the dataset $\D$.
\STATE 2: Construct matrix $(P_{y|x})$ such that $(P_{y|x})_{ij}=p_\theta(y=y_j|x=x_i)$.
\vspace{0.5ex}
\STATE 3: Construct vector $p_y=(p_{y1},..,p_{yC})$ such that $p_{yj}=\frac{1}{N}\sum_{i=1}^N(P_{y|x})_{ij}$.
\STATE 4: \textbf{for} $j$ \textbf{in} $\{1,2,...C\}$:
\vspace{0.5ex}
\STATE 5: \ \ \ \ $P_{y|x}^{(\text{sort} j)}\gets$Sort the rows of $P_{y|x}$ in decreasing values of $(P_{y|x})_{ij}$.
\vspace{1.5ex}
\STATE 6: \ \ \ \ $\tilde{\beta}_0^{(j)},\Omega^{(j)}\gets$Search $i_\text{left}$, $i_\text{right}$ until $\tilde{\beta}_0^{(j)}=\textbf{Get}\boldsymbol{\beta}(P_{y|x},p_y,\Omega)$ is minimal with tolerance $\varepsilon$, 
\STATE \ \ \ \ \ \ \ \ where $\Omega=\{i_\text{left},i_\text{left}+1,...i_\text{right}\}$.
\STATE 7: \textbf{end for}
\STATE 8: $j^*\gets\argmin_j\{\tilde{\beta}_0^{(j)}\}, j=1,2,...N$.
\vspace{1.5ex}
\STATE 9: $\tilde{\beta}_0\gets\tilde{\beta}_0^{(j^*)}$.
\vspace{1.5ex}
\STATE 10: $P_{y|x}^{(\tilde{\beta}_0)}\gets$ the rows of $P_{y|x}^{(\text{sort} j^*)}$ indexed by $\Omega^{(j^*)}$.
\vspace{1.5ex}
\STATE 11: \textbf{return} $\tilde{\beta}_0, P_{y|x}^{(\tilde{\beta}_0)}$
\STATE
\STATE \textbf{subroutine Get$\boldsymbol{\beta}$}($P_{y|x}, p_y, \Omega$):
\STATE s1: $N\gets $ number of rows of $P_{y|x}$.
\STATE s2: $C\gets $ number of columns of $P_{y|x}$.
\STATE s3: $n\gets$ number of elements of $\Omega$.
\STATE s4: $(p_{y|\Omega})_j\gets\frac{1}{n}\sum_{i\in\Omega}(P_{y|x})_{ij}$, $j=1,2,...,C$.
\vspace{1.5ex}
\STATE s5: $\tilde{\beta}_0\gets \frac{\frac{N}{n} - 1}{\sum_{j} \big[ \frac{(p_{y|\Omega_x})_j^2}{p_{yj}} - 1 \big]}$
\vspace{1.5ex}
\STATE s6: \textbf{return} $\tilde{\beta}_0$
\end{algorithmic}
\end{algorithm}

After estimating $\tilde{\beta}_0$, we can then use it for learning with IB, either directly, or as an anchor for a region where we can perform a much smaller sweep than we otherwise would have.
This may be particularly important for very noisy datasets, where $\beta_0$ can be very large.

\subsection{Special Cases for Estimating \texorpdfstring{$\beta_0$}{Lg}}

Theorem \ref{thm:suff_3} may still be challenging to estimate, due to the difficulty of making accurate estimates of $p(\Omega_x)$ and searching over $\Omega_x\subset\X$.
However, if the learning problem is more structured, we may be able to obtain a simpler formula for the sufficient condition.

\subsubsection{Class-conditional label noise.}
Classification with noisy labels is a common practical scenario.
An important noise model is that the labels are randomly flipped with some hidden class-conditional probabilities and we only observe the corrupted labels.
This problem has been studied extensively \citep{angluin1988learning,natarajan2013learning,liu2016classification,xiao2015learning,northcutt2017learning}.
If IB is applied to this scenario, how large $\beta$ do we need?
The following corollary provides a simple formula.

\begin{corollary}
\label{corollary:suff_3_class_conditional}
Suppose that the true class labels are $y^*$, and the input space belonging to each $y^*$ has no overlap.
We only observe the corrupted labels $y$ with class-conditional noise $p(y|x,y^*) = p(y|y^*)$, and $Y$ is not independent of $X$.
We have that a sufficient condition for $\IB_\beta$-Learnability is:
\begin{equation}
\begin{aligned}
\label{eq:suff_3_class_conditional}
\beta > \inf_{y^*} \frac{\frac{1}{p(y^*)} - 1}{\sum_y \frac{p(y|y^*)^2}{p(y)} - 1}
\end{aligned}
\end{equation}
\end{corollary}

We see that under class-conditional noise, the sufficient condition reduces to a discrete formula which only depends on the noise rates $p(y|y^*)$ and the true class probability $p(y^*)$, which can be accurately estimated via e.g. \citet{northcutt2017learning}.
Additionally, if we know that the noise is class-conditional, but the observed $\beta_0$ is greater than the R.H.S. of Eq.~(\ref{eq:suff_3_class_conditional}), we can deduce that there is overlap between the true classes.
The proof of Corollary \ref{corollary:suff_3_class_conditional} is provided in Appendix \ref{app:corollaries}.

\subsubsection{Deterministic relationships.}
Theorem \ref{thm:suff_3} also reveals that $\beta_0$ relates closely to whether $Y$ is a deterministic function of $X$, as shown by Corollary \ref{corollary:suff_3_2}:
\begin{corollary}
\label{corollary:suff_3_2}

Assume that $Y$ contains at least one value $y$ such that its probability $p(y)>0$. If $Y$ is a deterministic function of $X$ and not independent of $X$, then a sufficient condition for $\IB_\beta$-Learnability is $\beta > 1$.
\end{corollary}

The assumption in the corollary \ref{corollary:suff_3_2} is satisfied by classification, and certain regression problems.\footnote{%
    The following scenario does not satisfy this assumption: for certain regression problems where $Y$ is a continuous random variable and the probability density function $p_Y(y)$ is bounded, then for any $y$, the \emph{probability} $P(Y=y)$ has measure 0.
}
This corollary generalizes the result in \cite{kolchinsky2018caveats} which only proves it for classification problems.
Combined with the necessary condition $\beta>1$ for any dataset $(X,Y)$ to be $\IB_\beta$-learnable (Section \ref{sec:learnability}), we have that under the assumption, if $Y$ is a deterministic function of $X$, then a necessary and sufficient condition for $\IB_\beta$-learnability is $\beta>1$; i.e., its $\beta_0$ is 1.
The proof of Corollary \ref{corollary:suff_3_2} is provided in Appendix \ref{app:corollaries}.

Therefore, in practice, if we find that $\beta_0 > 1$, we may infer that $Y$ is not a deterministic function of $X$.
For a classification task, we may infer that either some classes have overlap, or the labels are noisy.
However, recall that finite models may add effective class overlap if they have insufficient capacity for the learning task, as mentioned in Section~\ref{sec:suff_conditions}.
This may translate into a higher observed $\beta_0$, even when learning deterministic functions.

\section{Experiments}
\label{sec:experiments}

To test how the theoretical conditions for $\IB_\beta$-learnability match with experiment, we apply them to synthetic data with varying noise rates and class overlap, MNIST binary classification with varying noise rates, and CIFAR10 classification, comparing with the $\beta_0$ found experimentally.
We also compare with the algorithm in \citet{kim2017discovering} for estimating the hypercontractivity coefficient (=$1/\beta_0$) via the contraction coefficient $\eta_{\text{KL}}$.
Experiment details are in Section~\ref{app:experiment}.

\subsection{Synthetic Dataset Experiments}
\label{sec:synthetic_exp}

We construct a set of datasets from 2D mixtures of 2 Gaussians as $X$ and the identity of the mixture component as $Y$.
We simulate two practical scenarios with these datasets: \textbf{(1)} noisy labels with class-conditional noise, and \textbf{(2)} class overlap.
For (1), we vary the class-conditional noise rates.
For (2), we vary class overlap by tuning the distance between the Gaussians.
For each experiment, we sweep $\beta$ with exponential steps, and observe $I(X;Z)$ and $I(Y;Z)$.
We then compare the empirical $\beta_0$ indicated by the onset of above-zero $I(X;Z)$ with predicted values for $\beta_0$.

\begin{figure}[t]
\begin{center}
\includegraphics[width=0.78\columnwidth]{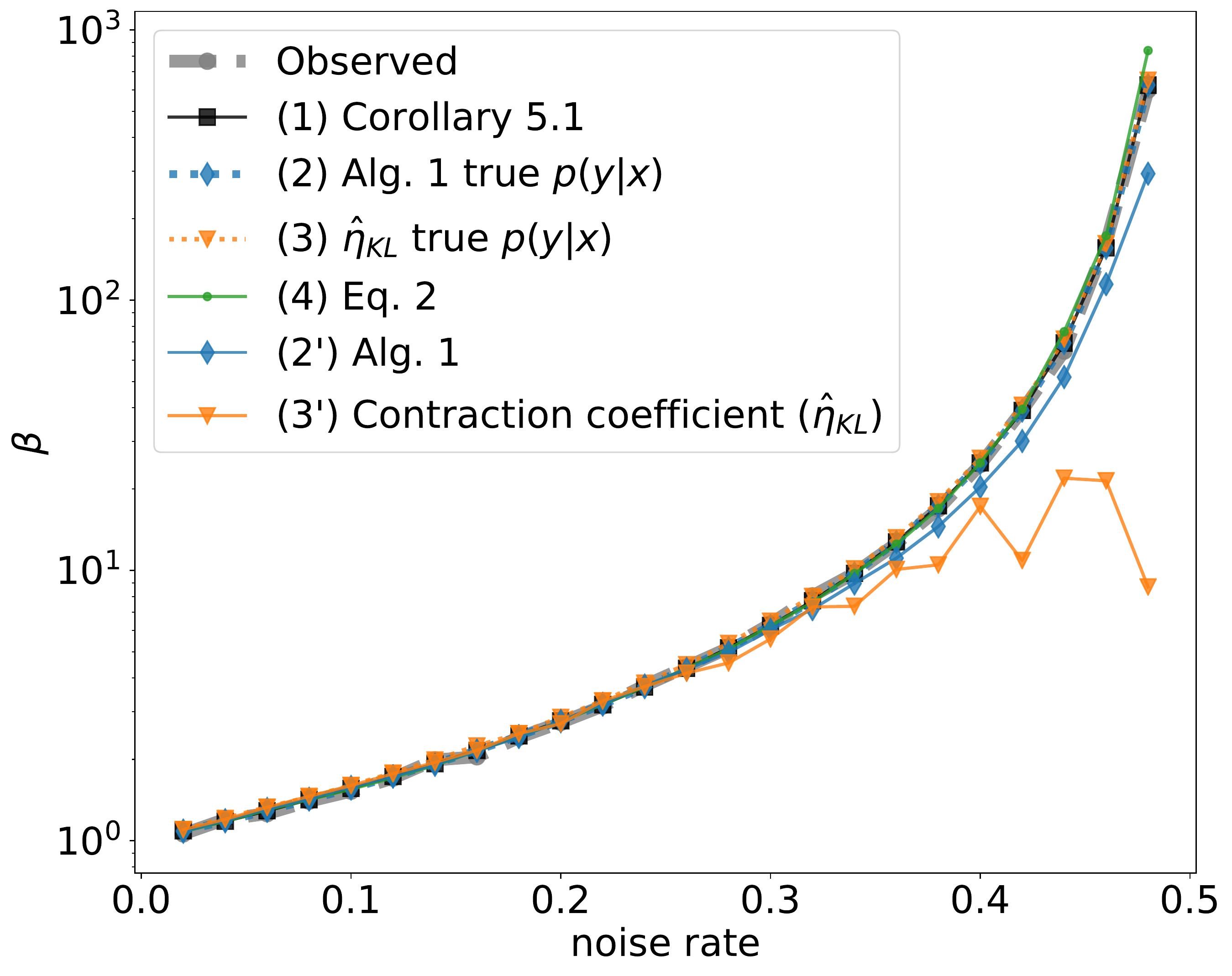}
\end{center}
\caption{Predicted vs. experimentally identified $\beta_0$, for mixture of Gaussians with varying class-conditional noise rates.}
\label{fig:gauss_noise_beta}
\end{figure}

\begin{table}[ht!]
\begin{center}
\caption{
Full table of values used to generate Fig.~\ref{fig:gauss_noise_beta}.}
\label{table:class_cond_noise}
\vskip 0.1in
\setlength{\tabcolsep}{4pt}  
\begin{tabular}{r | c c c c c c c }

 & & & (2) Alg. \ref{alg:estimating_beta} & (3) $\hat{\eta}_{\KL}$ & & & \\
Noise rate    & Observed & (1) Corollary~\ref{corollary:suff_3_class_conditional} & true $p(y|x)$ & true $p(y|x)$ & (4) Eq.~\ref{eq:suff_2} & (2$'$) Alg.~\ref{alg:estimating_beta} & (3$'$) $\hat{\eta}_{\KL}$ \\
\hline
\hline\noalign{\smallskip}
0.02 &    1.06 &    1.09 &    1.09 &    1.10 &    1.08 &    1.08 &   1.10 \\
0.04 &    1.20 &    1.18 &    1.18 &    1.21 &    1.18 &    1.19 &   1.20 \\
0.06 &    1.26 &    1.29 &    1.29 &    1.33 &    1.30 &    1.31 &   1.33 \\
0.08 &    1.40 &    1.42 &    1.42 &    1.45 &    1.42 &    1.43 &   1.46 \\
0.10 &    1.52 &    1.56 &    1.56 &    1.60 &    1.55 &    1.58 &   1.60 \\
0.12 &    1.70 &    1.73 &    1.73 &    1.78 &    1.71 &    1.73 &   1.77 \\
0.14 &    1.99 &    1.93 &    1.93 &    1.99 &    1.90 &    1.91 &   1.95 \\
0.16 &    2.04 &    2.16 &    2.16 &    2.24 &    2.15 &    2.15 &   2.16 \\
0.18 &    2.41 &    2.44 &    2.44 &    2.49 &    2.43 &    2.42 &   2.49 \\
0.20 &    2.74 &    2.78 &    2.78 &    2.86 &    2.76 &    2.77 &   2.71 \\
0.22 &    3.15 &    3.19 &    3.19 &    3.29 &    3.19 &    3.21 &   3.29 \\
0.24 &    3.75 &    3.70 &    3.70 &    3.83 &    3.71 &    3.75 &   3.72 \\
0.26 &    4.40 &    4.34 &    4.34 &    4.48 &    4.35 &    4.31 &   4.17 \\
0.28 &    5.16 &    5.17 &    5.17 &    5.37 &    5.12 &    4.98 &   4.55 \\
0.30 &    6.34 &    6.25 &    6.25 &    6.49 &    6.24 &    6.03 &   5.58 \\
0.32 &    8.06 &    7.72 &    7.72 &    8.02 &    7.63 &    7.19 &   7.33 \\
0.34 &    9.77 &    9.77 &    9.77 &   10.13 &    9.74 &    8.95 &   7.37 \\
0.36 &   12.58 &   12.76 &   12.76 &   13.21 &   12.51 &   11.11 &  10.09 \\
0.38 &   16.91 &   17.36 &   17.36 &   17.96 &   16.97 &   14.55 &  10.49 \\
0.40 &   24.66 &   25.00 &   25.00 &   25.99 &   25.01 &   20.36 &  17.27 \\
0.42 &   39.08 &   39.06 &   39.06 &   40.85 &   39.48 &   30.12 &  10.89 \\
0.44 &   64.82 &   69.44 &   69.44 &   71.80 &   76.48 &   51.95 &  21.95 \\
0.46 &  163.07 &  156.25 &  156.26 &  161.88 &  173.15 &  114.57 &  21.47 \\
0.48 &  599.45 &  625.00 &  625.00 &  651.47 &  838.90 &  293.90 &   8.69 \\
\hline
\end{tabular}
\end{center}
\end{table}

\paragraph{Classification with class-conditional noise.}
In this experiment, we have a mixture of Gaussian distribution with 2 components, each of which is a 2D Gaussian with diagonal covariance matrix $\Sigma=\text{diag}(0.25, 0.25)$.
The two components have distance 16 (hence virtually no overlap) and equal mixture weight.
For each $x$, the label $y \in \{0,1\}$ is the identity of which component it belongs to.
We create multiple datasets by randomly flipping the labels $y$ with a certain noise rate $\rho = P(y=0|y^*=1) = P(y=1|y^*=0)$.
For each dataset, we train VIB models across a range of $\beta$, and observe the onset of learning via random $I(X;Z)$ (Observed).
To test how different methods perform in estimating $\beta_0$, we apply the following methods:
\textbf{(1)} Corollary 5.1, since this is classification with class-conditional noise, and the two true classes have virtually no overlap;
\textbf{(2)} Alg.~\ref{alg:estimating_beta} with true $p(y|x)$;
\textbf{(3)} The algorithm in \citet{kim2017discovering} that estimates $\hat{\eta}_{\KL}$, provided with true $p(y|x)$;
\textbf{(4)} $\beta_0[h(x)]$ in Eq.~(\ref{eq:suff_2});
\textbf{(2$'$)} Alg.~\ref{alg:estimating_beta} with $p(y|x)$ estimated by a neural net;
\textbf{(3$'$)} $\hat{\eta}_{\KL}$ with the same $p(y|x)$ as in (2$'$).
The results are shown in Fig.~\ref{fig:gauss_noise_beta} and in Table \ref{table:class_cond_noise}.

From Fig.~\ref{fig:gauss_noise_beta} and Table \ref{table:class_cond_noise} we see the following. 
\textbf{(A)} When using the true $p(y|x)$, both Alg.~\ref{alg:estimating_beta} and $\hat{\eta}_\text{KL}$ generally upper bound the empirical $\beta_0$, and Alg.~\ref{alg:estimating_beta} is generally tighter.
\textbf{(B)} When using the true $p(y|x)$, Alg.~\ref{alg:estimating_beta} and Corollary \ref{corollary:suff_3_class_conditional} give the same result.
\textbf{(C)} Comparing Alg.~\ref{alg:estimating_beta} and $\hat{\eta}_{\KL}$ both of which use the same empirically estimated $p(y|x)$, both approaches provide good estimation in the low-noise region; however, in the high-noise region, Alg.~\ref{alg:estimating_beta} gives more precise values than $\hat{\eta}_{\KL}$, indicating that Alg.~\ref{alg:estimating_beta} is more robust to the estimation error of $p(y|x)$. 
\textbf{(D)} Eq.~(\ref{eq:suff_2}) empirically upper bounds the experimentally observed $\beta_0$, and gives almost the same result as theoretical estimation in Corollary \ref{corollary:suff_3_class_conditional} and Alg.~\ref{alg:estimating_beta} with the true $p(y|x)$.
In the classification setting, this approach doesn't require any learned estimate of $p(y|x)$, as we can directly use the empirical $p(y)$ and $p(x|y)$ from SGD mini-batches.

This experiment also shows that for dataset where the signal-to-noise is small, $\beta_0$ can be very high.
Instead of blindly sweeping $\beta$, our result can provide guidance for setting $\beta$ so learning can happen.

\begin{figure}[t!]
\begin{center}
\includegraphics[width=0.7\columnwidth]{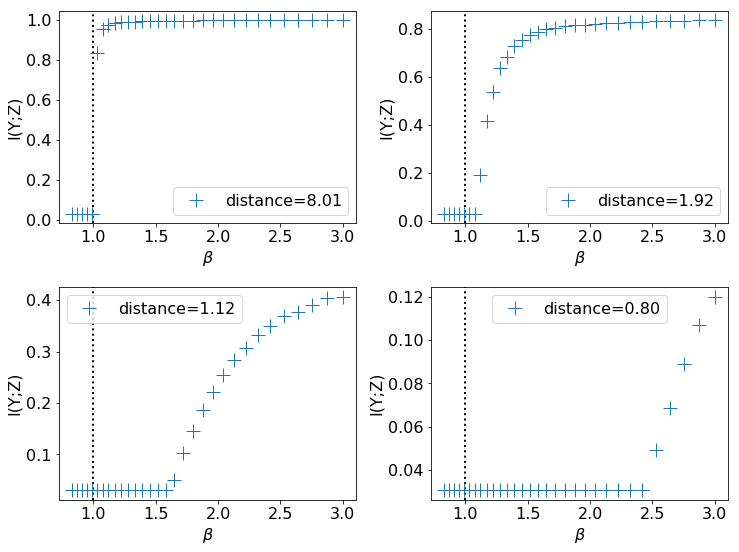}
\end{center}
\caption{
$I(Y;Z)$ vs. $\beta$, for mixture of Gaussian datasets with different distances between the two mixture components.
The vertical lines are $\beta_{0,\text{predicted}}$ computed by the R.H.S. of Eq.~(\ref{eq:suff_3_class_conditional}).
As Eq.~(\ref{eq:suff_3_class_conditional}) does not make predictions w.r.t. class overlap, the vertical lines are always just above $\beta_{0,\text{predicted}} = 1$.
However, as expected, decreasing the distance between the classes in $X$ space also increases the true $\beta_0$.
}
\label{fig:gauss_overlap_beta_indi}
\end{figure}

\begin{figure}[t!]
\begin{center}
\includegraphics[width=0.8\columnwidth]{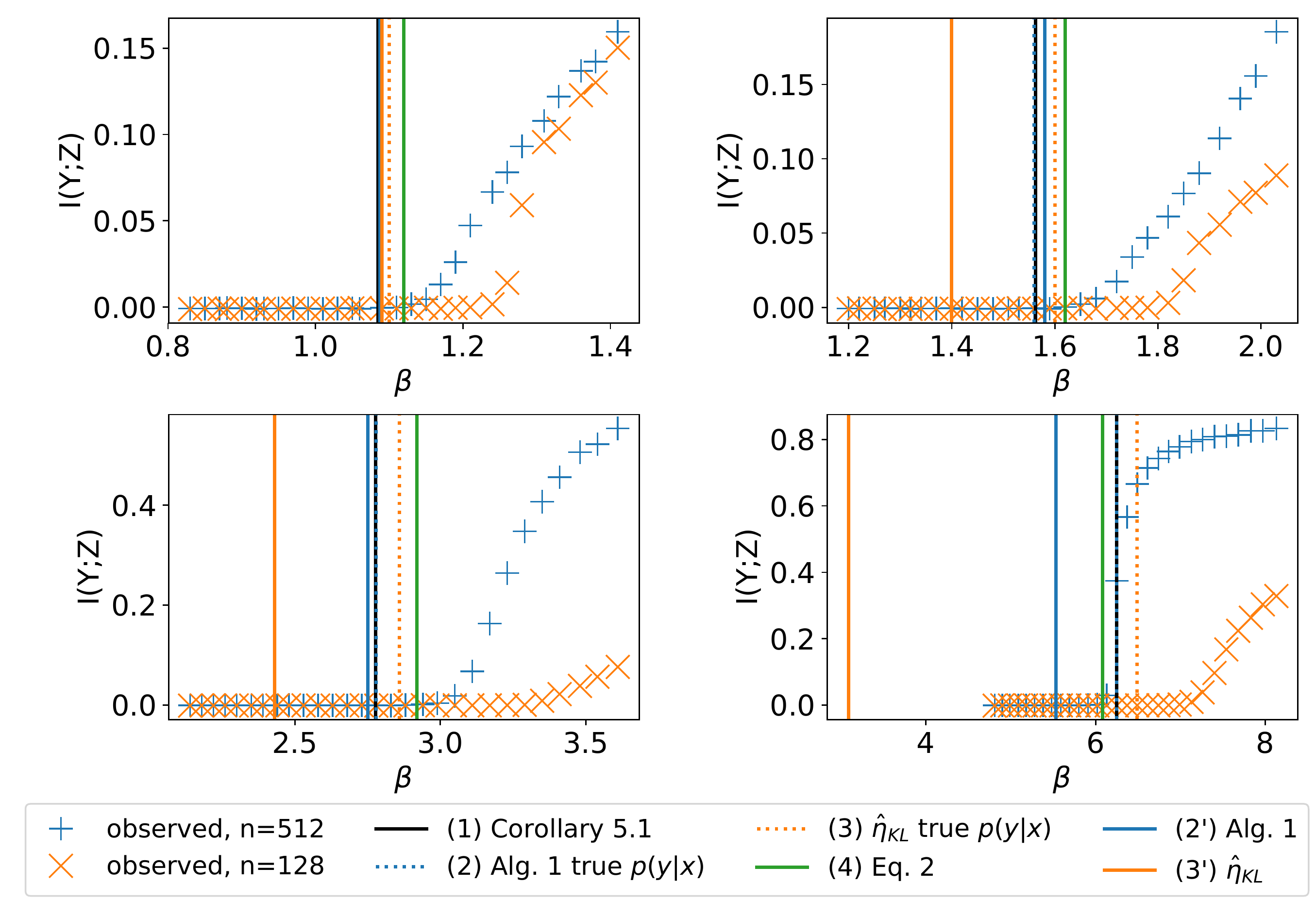}
\end{center}
\caption{
$I(Y;Z)$ vs. $\beta$ for the MNIST binary classification with different hidden units per layer $n$ and noise rates $\rho$: (upper left) $\rho=0.02$, (upper right) $\rho=0.1$, (lower left) $\rho=0.2$, (lower right) $\rho=0.3$.
The vertical lines are $\beta_{0}$ estimated by different methods.
$n=128$ has insufficient capacity for the problem, so its observed learnability onset is pushed higher, similar to the class overlap case.
}
\label{fig:mnist_noise_beta_indi}
\end{figure}

\paragraph{Classification with class overlap.}
\label{sec:exp_overlap}

In this experiment, we test how different amounts of overlap among classes influence $\beta_0$.
We use the mixture of Gaussians with two components, each of which is a 2D Gaussian with diagonal covariance matrix $\Sigma=\text{diag}(0.25,0.25)$.
The two components have weights 0.6 and 0.4.
We vary the distance between the Gaussians from 8.0 down to 0.8 and observe the $\beta_{0,exp}$.
Since we don't add noise to the labels, if there were no overlap and a deterministic map from $X$ to $Y$, we would have $\beta_0=1$ by Corollary \ref{corollary:suff_3_2}.
The more overlap between the two classes, the more uncertain $Y$ is given $X$.
By Eq.~\ref{eq:suff_3} we expect $\beta_0$ to be larger, which is corroborated in Fig.~\ref{fig:gauss_overlap_beta_indi}.

\subsection{MNIST Experiments}

We perform binary classification with digits 0 and 1, and as before, add class-conditional noise to the labels with varying noise rates $\rho$.
To explore how the model capacity influences the onset of learning, for each dataset we train two sets of VIB models differing only by the number of neurons in their hidden layers of the encoder: one with $n=512$ neurons, the other with $n=128$ neurons.
As we describe in Section \ref{sec:suff_conditions}, insufficient capacity will result in more uncertainty of $Y$ given $X$ from the point of view of the model, so we expect the observed $\beta_{0}$ for the $n=128$ model to be larger.
This result is confirmed by the experiment (Fig. \ref{fig:mnist_noise_beta_indi}). Also, in Fig. \ref{fig:mnist_noise_beta_indi} 
we plot $\beta_0$ given by different estimation methods. We see that the observations (A), (B), (C) and (D) in Section \ref{sec:synthetic_exp} still hold.

\subsection{MNIST Experiments using Equation \ref{eq:suff_2}}
\label{sec:mnist_eq2}

To see what IB learns at its onset of learning for the full MNIST dataset, we optimize Eq. (\ref{eq:suff_2}) w.r.t. the full MNIST dataset, and visualize the clustering of digits by $h(x)$.
Eq.~(\ref{eq:suff_2}) can be optimized using SGD using any differentiable parameterized mapping $h(x) : \mathcal{X} \rightarrow \mathbbm{R}$.
In this case, we chose to parameterize $h(x)$ with a PixelCNN++ architecture~\citep{pixelcnn,pxpp}, as PixelCNN++ is a powerful autoregressive model for images that gives a scalar output (normally interpreted as $\log\,p(x)$).
Eq.~(\ref{eq:suff_2}) should generally give two clusters in the output space, as discussed in Section~\ref{sec:suff_conditions}.
In this setup, smaller values of $h(x)$ correspond to the subset of the data that is easiest to learn.
Fig.~\ref{fig:mnist_hx_hist} shows two strongly separated clusters, as well as the threshold we choose to divide them.
Fig.~\ref{fig:mnist_sort_by_hx} shows the first 5,776 MNIST training examples as sorted by our learned $h(x)$, with the examples above the threshold highlighted in red.
We can clearly see that our learned $h(x)$ has separated the ``easy'' one (1) digits from the rest of the MNIST training set.

\begin{figure}[t!]
\begin{center}
\includegraphics[width=\columnwidth]{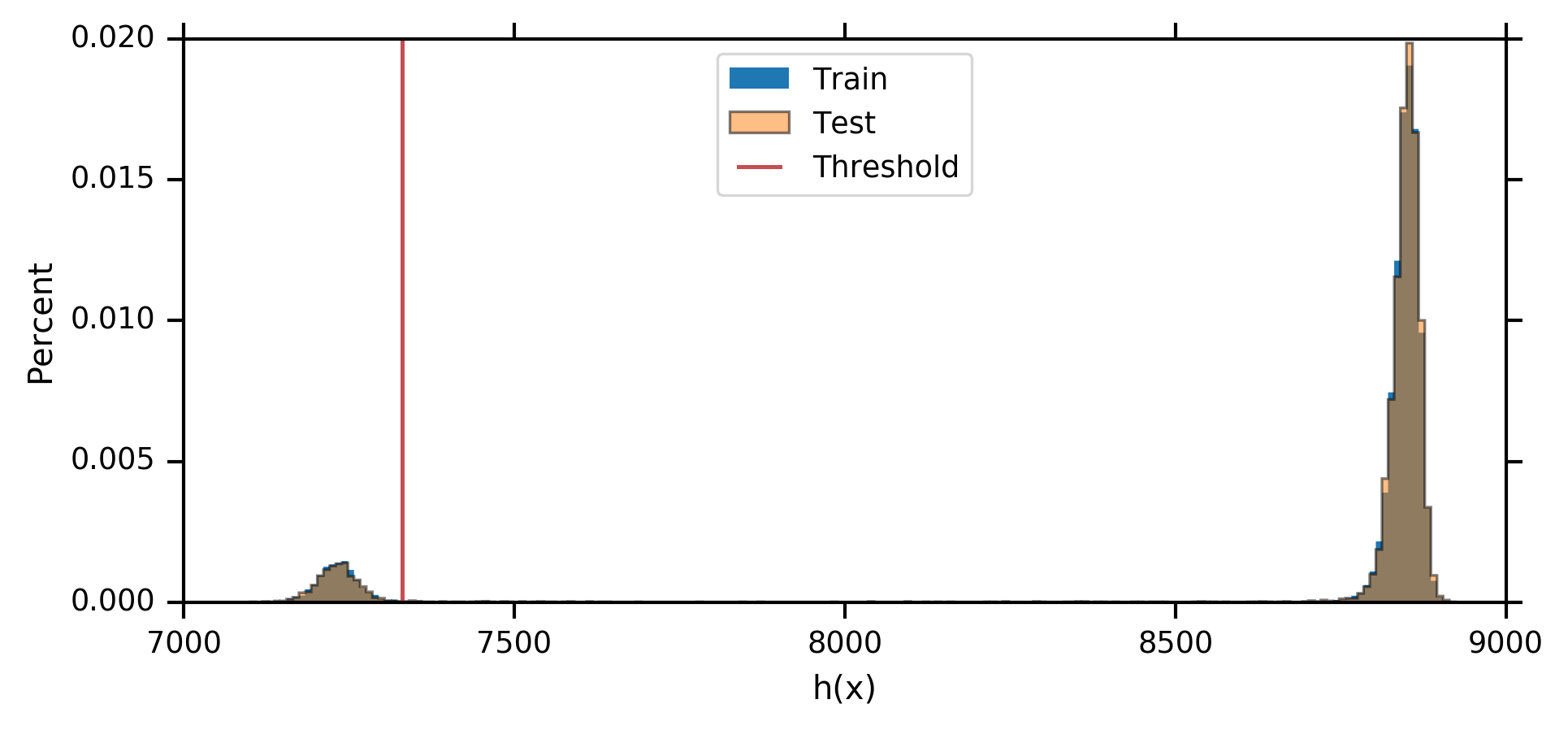}
\end{center}
\caption{
Histograms of the full MNIST training and validation sets according to $h(X)$.
Note that both are bimodal, and the histograms are indistinguishable.
In both cases, $h(x)$ has learned to separate most of the ones into the smaller mode, but difficult ones are in the wide valley between the two modes.
See Figure~\ref{fig:mnist_sort_by_hx} for all of the training images to the left of the red threshold line, as well as the first few images to the right of the threshold.
}
\label{fig:mnist_hx_hist}
\end{figure}

\begin{figure}[t!]
\begin{center}
\includegraphics[width=0.6\columnwidth]{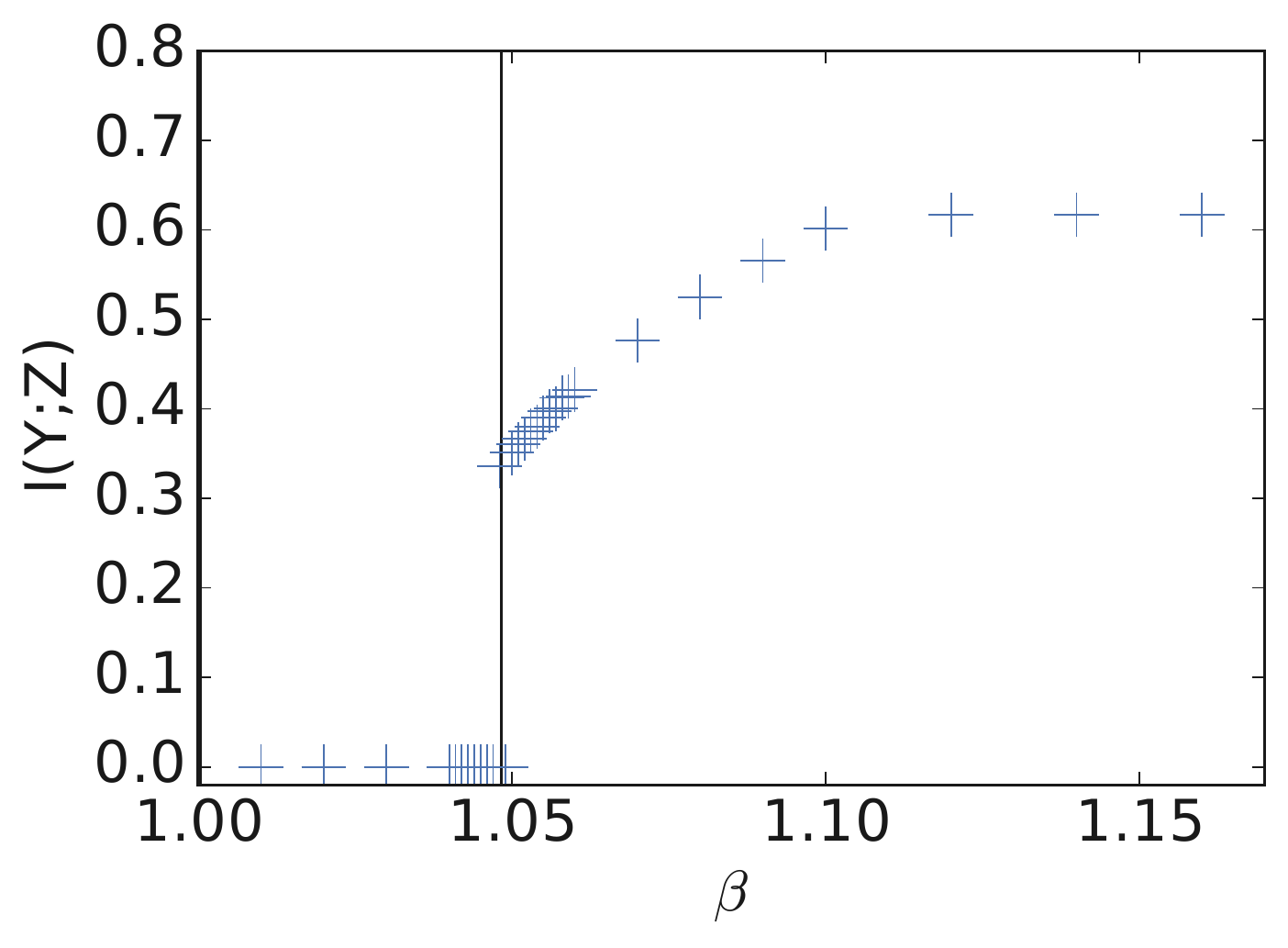}
\end{center}
\vskip -0.12in
\caption{
Plot of $I(Y;Z)$ vs $\beta$ for CIFAR10 training set with 20\% label noise.
Each blue cross corresponds to a fully-converged model starting with independent initialization.
The vertical black line corresponds to the predicted $\beta_0=1.0483$ using Alg.~\ref{alg:estimating_beta}.
The empirical $\beta_0=1.048$.
}
\label{fig:cifar10_experiments}
\end{figure}

\begin{figure*}[p]
\begin{center}
\includegraphics[width=1\columnwidth]{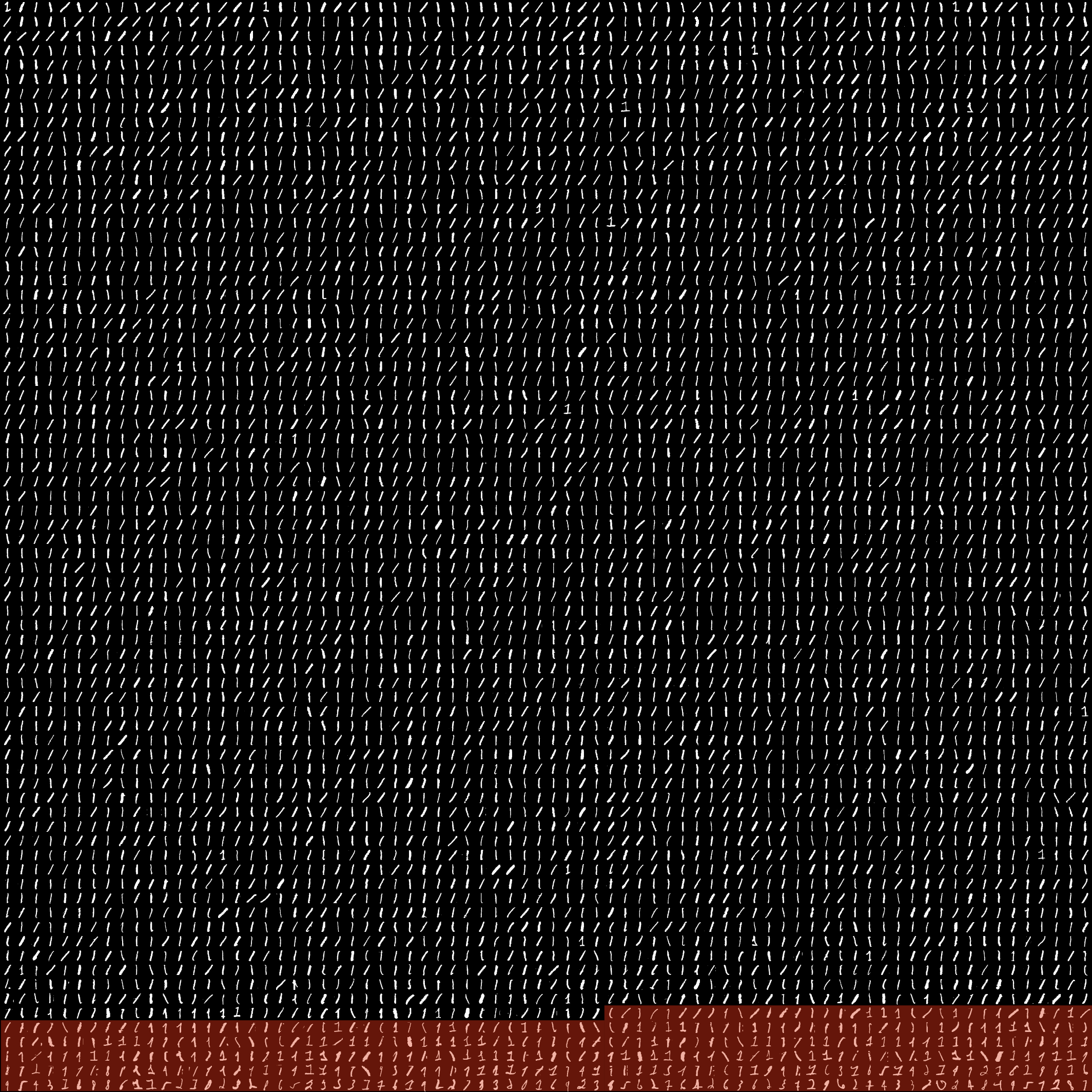}
\end{center}
\caption{
The first 5776 MNIST training set digits when sorted by $h(x)$.
The digits highlighted in red are above the threshold drawn in Figure~\ref{fig:mnist_hx_hist}.
}
\label{fig:mnist_sort_by_hx}
\end{figure*}

\subsection{CIFAR10 Forgetting Experiments}

For CIFAR10~\citep{cifar}, we study how \textit{forgetting} varies with $\beta$.
In other words, given a VIB model trained at some high $\beta_2$, if we anneal it down to some much lower $\beta_1$, what $I(Y;Z)$ does the model converge to?
Using Alg.~\ref{alg:estimating_beta}, we estimated $\beta_0 = 1.0483$ on a version of CIFAR10 with 20\% label noise, where the $P_{y|x}$ is estimated by maximum likelihood training with the same encoder and classifier architectures as used for VIB. 
For the VIB models, the lowest $\beta$ with performance above chance was $\beta = 1.048$, a very tight match with the estimate from Alg.~\ref{alg:estimating_beta}.
See Appendix~\ref{app:cifar_details} for details.

\section{Conclusion}

In this paper, we have presented theoretical results for predicting the onset of learning, and have shown that it is determined by the conspicuous subset of the training examples.
We gave a practical algorithm for predicting the transition as well as discovering this subset, and showed that those predictions are accurate, even in cases of extreme label noise.
We proved a deep connection between IB-learnability, our upper bounds on $\beta_0$, the hypercontractivity coefficient, the contraction coefficient, and the maximum correlation.
We believe that these results provide a deeper understanding of IB, as well as a tool for analyzing a dataset by discovering its conspicuous subset, and a tool for measuring model capacity in a task-specific manner.

Our work also raises other questions, such as whether there are other phase transitions in learnability that might be identified.
We hope to address some of those questions in future work.

\chapter{Intermediate phase transitions}
\label{chap4:IB_phase_transition}

In Chapter \ref{chap3:IB}, we have studied the first phase transition in Information Bottleneck (IB), the learnability transition, in detail. In general, when tuning the relative strength between compression and prediction terms in IB, how do the two terms behave, and what's their relationship with the dataset and the learned representation?
In this chapter\footnote{The paper \href{https://openreview.net/forum?id=HJloElBYvB}{``Phase transitions for the information bottleneck in representation learning''} is published as a conference paper at \emph{ICLR} 2020 \cite{Wu2020Phase}. A short version is presented at \href{https://sites.google.com/view/itml19/home}{\emph{NeurIPS} 2019 Workshop on Information Theory and Machine Learning}. Authors: Wu, Tailin and Ian Fischer. \href{https://arxiv.org/abs/2001.01878}{arXiv:2001.01878}.}, we set out to answer this question by studying multiple phase transitions in the IB objective: $\IB_\beta[p(z|x)]=I(X;Z)-\beta I(Y;Z)$, where sudden jumps of $\frac{dI(Y;Z)}{d\beta}$ and prediction accuracy are observed with increasing $\beta$.
We introduce a definition for IB phase transitions as a qualitative change of the IB loss landscape, and show that the transitions correspond to the onset of learning new classes. Using second-order calculus of variations, we derive a formula that provides a practical condition for IB phase transitions, and draw its connection with the Fisher information matrix for parameterized models.
We provide two perspectives to understand the formula, revealing that each IB phase transition is finding a component of maximum (nonlinear) correlation between $X$ and $Y$ orthogonal to the learned representation, in close analogy with canonical-correlation analysis (CCA) in linear settings. Based on the theory, we present an algorithm for discovering phase transition points.
Finally, we verify that our theory and algorithm accurately predict phase transitions in categorical datasets, predict the onset of learning new classes and class difficulty in MNIST, and predict prominent phase transitions in CIFAR10 experiments.

\section{Introduction}
\label{sec:phase_introduction}

The Information Bottleneck ($\IB$) objective~\citep{tishby2000information}:
\begin{equation}
\label{eq:phase_IB_beta}
\IB_\beta[p(z|x)]:= I(X;Z) - \beta I(Y;Z)
\end{equation}
explicitly trades off model compression ($I(X;Z)$, $I(\cdot;\cdot)$ denoting mutual information) with predictive performance ($I(Y;Z)$) using the Lagrange multiplier $\beta$, where $X,Y$ are observed random variables, and $Z$ is a learned representation of $X$.
The IB method has proved effective in a variety of scenarios, including improving the robustness against adversarial attacks \citep{alemi2016deep,fischer2018the}, learning invariant and disentangled representations~ \citep{achille2018emergence,achille2018information}, underlying information-based geometric clustering \citep{strouse2017information}, improving the training and performance in adversarial learning ~\citep{peng2018variational}, and facilitating skill discovery~\citep{sharma2019dynamics} and learning goal-conditioned policy~\citep{goyal2019infobot} in reinforcement learning.

From Eq. (\ref{eq:phase_IB_beta}) we see that when $\beta\to0$ it will encourage $I(X;Z)=0$ which leads to a trivial representation $Z$ that is independent of $X$, while when $\beta\to+\infty$, it reduces to a maximum likelihood objective\footnote{%
  For example, in classification, it reduces to cross-entropy loss.
} that does not constrain the information flow.
Between these two extremes, how will the IB objective behave? 
Will prediction and compression performance change smoothly, or do there exist interesting transitions in between? 
In \citet{wu2019learnabilityEntropy}, the authors observe and study the learnability transition, i.e. the $\beta$ value such that the IB objective transitions from a trivial global minimum to learning a nontrivial representation.
They also show how this first phase transition relates to the structure of the dataset.
However, to answer the full question, we need to consider the full range of $\beta$. 

\paragraph{Motivation.}
To get a sense of how $I(Y;Z)$ and $I(X;Z)$ vary with $\beta$, we train Variational Information Bottleneck (VIB) models~\citep{alemi2016deep} on the CIFAR10 dataset~\citep{cifar}, where each experiment is at a different $\beta$ and random initialization of the model.
Fig.~\ref{fig:cifar_multiple_phase_transition_noise} shows the $I(X;Z)$, $I(Y;Z)$ and accuracy vs. $\beta$, as well as $I(Y;Z)$ vs. $I(X;Z)$ for CIFAR10 with 20\% label noise (see Appendix \ref{app:cifar_details} for details).

\begin{figure}[t]
\begin{center}
\includegraphics[width=1\columnwidth]{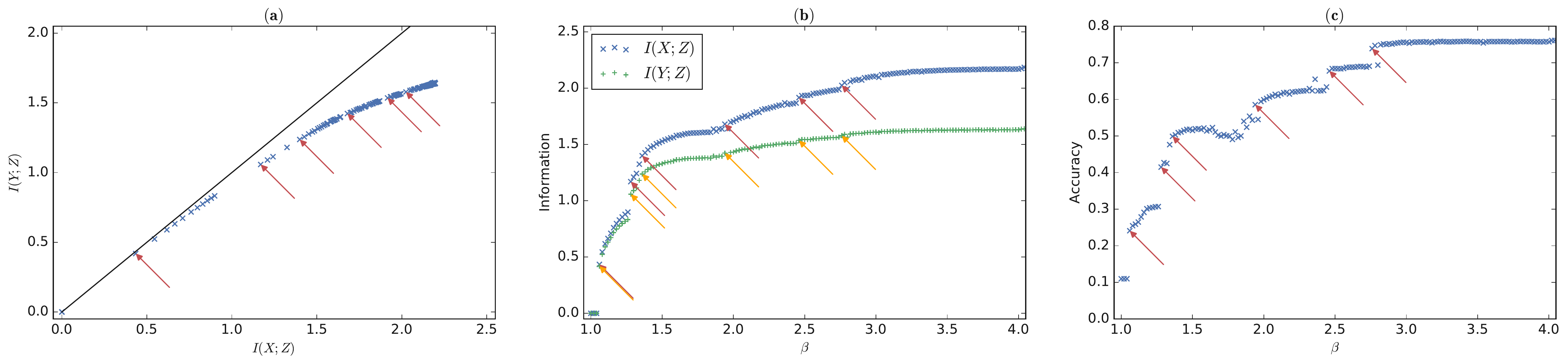}
\end{center}
\caption{%
CIFAR10 plots \textbf{(a)} showing the information plane, as well as $\beta$ vs \textbf{(b)} $I(X;Z)$ and $I(Y;Z)$, and \textbf{(c)} accuracy, all on the training set with 20\% label noise.
The arrows point to empirically-observed phase transitions.
The vertical lines correspond to phase transitions found with Alg.~\ref{alg:coordinate_descent}.
}
\label{fig:cifar_multiple_phase_transition_noise}
\end{figure}

From Fig.~\ref{fig:cifar_multiple_phase_transition_noise}\textbf{(b)(c)}, we see that as we increase $\beta$, instead of going up smoothly, both $I(X;Z)$ and $I(Y;Z)$ show multiple phase transitions, where the slopes $\frac{dI(X;Z)}{d\beta}$ and $\frac{dI(Y;Z)}{d\beta}$ are discontinuous and the accuracy has discrete jumps.
The observation lets us refine our question: When do the phase transitions occur, and how do they depend on the structure of the dataset?
These questions are important, since answering them will help us gain a better understanding of the IB objective and its close interplay with the dataset and the learned representation. 

Moreover, the IB objective belongs to a general form of two-term trade-offs in many machine learning objectives: $L=\text{Prediction-loss}+\beta\cdot\text{Complexity}$,
where the complexity term generally takes the form of regularization.
Usually, learning is set at a specific $\beta$.
Many more insights can be gained if we understand the behavior of the prediction loss and model complexity with varying $\beta$, and how they depend on the dataset.
The techniques developed to address the question in the IB setting may also help us understand the two-term tradeoff in other learning objectives.

\paragraph{Contributions.}
In this work, we begin to address the above question in  IB settings. Specifically:

\begin{itemize}

\item We identify a \emph{qualitative} change of the IB loss landscape w.r.t. $p(z|x)$ for varying $\beta$ as IB phase transitions (Section \ref{sec:phase_definition}). 

\item Based on the definition, we introduce a quantity $G[p(z|x)]$ and use it to prove a theorem giving a practical condition for IB phase transitions.
We further reveal the connection between $G[p(z|x)]$ and the Fisher information matrix when $p(z|x)$ is parameterized by $\thetaa$ (Section \ref{sec:phase_definition}).

\item We reveal the close interplay between the IB objective, the dataset and the learned representation, by showing that in IB, each phase transition corresponds to learning a new nonlinear component of maximum correlation between $X$ and $Y$, orthogonal to the previously-learned $Z$, and each with decreasing strength (Section \ref{sec:phase_understanding}).

\end{itemize}

To the best of our knowledge, our work provides the first theoretical formula to address IB phase transitions in the most general setting. In addition, we present an algorithm for iteratively finding the IB phase transition points (Section \ref{sec:alg_for_discovering_transition}).
We show that our theory and algorithm give tight matches with the observed phase transitions in categorical datasets, predict the onset of learning new classes and class difficulty in MNIST, and predict prominent transitions in CIFAR10 experiments (Section \ref{sec:phase_experiments}).

\section{Related Work}

The Information Bottleneck Method~\citep{tishby2000information} provides a tabular method based on the Blahut-Arimoto (BA) Algorithm~\citep{blahut1972computation} to numerically solve the IB functional for the optimal encoder distribution $P(Z|X)$, given the trade-off parameter $\beta$ and the cardinality of the representation variable $Z$.
This work has been extended in a variety of directions, including to the case where all three variables $X,Y,Z$ are multivariate Gaussians~\citep{chechik2005information}, cases of variational bounds on the IB and related functionals for amortized learning~\citep{alemi2016deep,achille2018emergence,fischer2018the}, and a more generalized interpretation of the constraint on model complexity as a Kolmogorov Structure Function~\citep{achille2018dynamics}.
Previous theoretical analyses of IB include~\citet{rey2012meta}, which looks at IB through the lens of copula functions, and \citet{shamir2010learning}, which starts to tackle the question of how to bound generalization with IB.
We will make practical use of the original IB algorithm, as well as the amortized bounds of the Variational Informormation Bottleneck~\citep{alemi2016deep} and the Conditional Entropy Bottleneck~\citep{fischer2018the}.

Phase transitions, where key quantities change discontinuously with varying relative strength in the two-term trade-off, have been observed in many different learning domains, for multiple learning objectives.
In \cite{rezende2018taming}, the authors observe phase transitions in the latent representation of $\beta$-VAE for varying $\beta$.
\cite{strouse2017information} utilize the kink angle of the phase transitions in the Deterministic Information Bottleneck (DIB)~\citep{strouse2017deterministic} to determine the optimal number of clusters for geometric clustering. 
\citet{tegmark2019pareto} explicitly considers critical points in binary classification tasks using a discrete information bottleneck with a non-convex Pareto-optimal frontier.
In \cite{achille2018emergence}, the authors observe a transition on the tradeoff of $I(\theta;X,Y)$ vs. $H(Y|X,\theta)$ in InfoDropout.
Under IB settings, \cite{chechik2005information} study the Gaussian Information Bottleneck, and analytically solve the critical values $\beta_i^c=\frac{1}{1-\lambda_i}$, where $\lambda_i$ are eigenvalues of the matrix $\Sigma_{x|y}\Sigma_x^{-1}$, and $\Sigma_x$ is the covariance matrix.
This work provides valuable insights for IB, but is limited to the special case that $X$, $Y$ and $Z$ are jointly Gaussian. 
Phase transitions in the general IB setting have also been observed, which \citet{tishbyinfo} describes as ``information bifurcation''.
In \cite{wu2019learnabilityEntropy}, the authors study the first phase transition, i.e. the learnability phase transition, and provide insights on how the learnability depends on the dataset.
Our work is the first work that addresses all the IB phase transitions in the most general setting, and provides theoretical insights on the interplay between the IB objective, its phase transitions, the dataset, and the learned representation.

\section{Formula for IB phase transitions}
\label{sec:phase_definition}

\subsection{Definitions}
Let $X\in\mathcal{X}, Y\in\mathcal{Y}, Z\in\mathcal{Z}$ be random variables denoting the input, target and representation, respectively, having a joint probability distribution $p(X, Y, Z)$, with $\X\times\mathcal{Y}\times\mathcal{Z}$ its support. $X$, $Y$ and $Z$ satisfy the Markov chain $Z-X-Y$, i.e. $Y$ and $Z$ are conditionally independent given $X$. We assume that the integral (or summing if $X$, $Y$ or $Z$ are discrete random variables) is on $\X\times\mathcal{Y}\times\mathcal{Z}$. We use $x$, $y$ and $z$ to denote the instances of the respective random variables. The above settings are used throughout the paper. We can view the IB objective $\IB_\beta[p(z|x)]$ (Eq. \ref{eq:phase_IB_beta}) as a functional of the encoding distribution $p(z|x)$. To prepare for the introduction of IB phase transitions, we first define \emph{relative perturbation function} and  \emph{second variation}, as follows.

\begin{definition}
\textbf{Relative perturbation function:} For $p(z|x)$, its relative perturbation function $r(z|x)$ is a bounded function that maps $\mathcal{X}\times \mathcal{Z}$ to $\mathbb{R}$ and satisfies $\E_{z\sim p(z|x)}[r(z|x)]=0$. Formally, define $\QQ:=\{r(z|x): \mathcal{X}\times\mathcal{Z}\to \mathbb{R}\ \big|\ \E_{z\sim p(z|x)}[r(z|x)]=0, \text{and } \exists M>0 \text{ s.t. } \forall X\in\mathcal{X},Z\in\mathcal{Z}, |r(z|x)|\le M\}$. We have that $r(z|x)\in \QQ$ iff $r(z|x)$ is a relative perturbation function of $p(z|x)$.
The perturbed probability (density) is $p'(z|x)=p(z|x)\left(1+\epsilon\cdot r(z|x)\right)$ for some $\epsilon>0$.
\end{definition}

\begin{definition}
\textbf{Second variation:} Let functional $F[f(x)]$ be defined on some normed linear space $\mathscr{R}$. Let us add a perturbative function $\epsilon\cdot h(x)$ to $f(x)$, and now the functional 
$F[f(x)+\epsilon\cdot h(x)]$ can be expanded as

\begin{equation*}
\begin{aligned}
\Delta F[f(x)]&=F[f(x)+\epsilon \cdot h(x)]-F[f(x)]\\
&=\varphi_1[\e\cdot h(x)]+\varphi_2[\e\cdot h(x)]+ \varphi_r[\e\cdot h(x)]\norm{\e\cdot h(x)}^2
\end{aligned}
\end{equation*}

such that $\lim\limits_{\e\to0}\varphi_r[\e\cdot h(x)]=0$, where $\norm{\cdot}$ denotes the norm, $\varphi_1[\e\cdot h(x)]=\epsilon \frac{d F[f(x)]}{d\epsilon}$ is a linear functional of $\epsilon \cdot h(x)$, and is called the \emph{first variation}, denoted as $\delta F[f(x)]$. $\varphi_2[\e\cdot h(x)]=\frac{1}{2}\epsilon^2 \frac{d^2 F[f(x)]}{d\epsilon^2}$ is a quadratic functional of $\epsilon \cdot h(x)$, and is called the \emph{second variation}, denoted as $\delta^2 F[f(x)]$.
\end{definition}

We can think of the perturbation function $\e\cdot h(x)$ as an infinite-dimensional ``vector'' ($x$ being the indices), with $\e$ being its amplitude and $h(x)$ its direction. 
With the above preparations, we define the IB phase transition as a change in the local curvature on the global minimum of $\IB_\beta[p(z|x)]$.

\begin{definition}
\label{definition:phase_transition_0}
\textbf{IB phase transitions:} Let $r(z|x)\in \QQ$ be a perturbation function of $p(z|x)$, $\pstar$ denote the optimal solution of $\IB_\beta[p(z|x)]$ at $\beta$, where the IB functional $\IB[\cdot]$ is defined in Eq. (\ref{eq:phase_IB_beta}). The IB phase transitions $\beta_i^c$ are the $\beta$ values satisfying the following two conditions:

(1) $\forall r(z|x)\in \QQ$, $\delta^2\IB_\beta[p(z|x)]\big|_{p^*_{\beta}(z|x)}\ge0$;

(2) 
$\lim\limits_{\beta'\to\beta^+}\inf\limits_{r(z|x)\in \QQ}\delta^2\IB_{\beta'}[p(z|x)]\big|_{p^*_{\beta}(z|x)}=0^-$.

Here $\beta^+$ and $0^-$ denote one-sided limits.

\end{definition}

We can understand the $\delta^2\IB_\beta[p(z|x)]$ as a local ``curvature'' of the IB objective $\IB_\beta$ (Eq. \ref{eq:phase_IB_beta}) w.r.t. $p(z|x)$, along some relative perturbation $r(z|x)$.
A phase transition occurs when the convexity of $\IB_\beta[p(z|x)]$ w.r.t. $p(z|x)$ changes from a minimum to a saddle point in the neighborhood of its optimal solution $\pstar$ as $\beta$ increases from $\beta_c$ to $\beta_c+0^+$.
This means that there exists a perturbation to go downhill and find a better minimum.
We validate this definition empirically below.

\subsection{Condition for IB phase transitions}

The definition for IB phase transition (Definition \ref{definition:phase_transition_0}) indicates the important role  $\delta^2\IB_\beta[p(z|x)]$ plays on the optimal solution in providing the condition for phase transitions.
To concretize it and prepare for a more practical condition for IB phase transitions, we expand $\IB_\beta[p(z|x)(1+\epsilon\cdot r(z|x))]$ to the second order of $\epsilon$, giving:
\begin{lemma}
\label{lemma:second_order_variation}
For $\IBB_\beta[p(z|x)]$, the condition of $\ \forall r(z|x)\in \QQ$, $\delta^2\IBB_\beta[p(z|x)]\ge0$ is equivalent to $\beta\le G[p(z|x)]$.
The threshold function $G[p(z|x)]$ is given by:
\begingroup\makeatletter\def\f@size{9.8}\check@mathfonts
\begin{equation}
\begin{aligned}
\label{eq:G}
G[p(z|x)] &:= \inf_{r(z|x)\in\QQ} \G[r(z|x);p(z|x)] \\
\G[r(z|x);p(z|x)]&:= 
\frac {\E_{x,z\sim p(x,z)}[r^2(z|x)]-\E_{z\sim p(z)}\left[\left(\E_{x\sim p(x|z)}[r(z|x)]\right)^2\right]}
      {\E_{y,z\sim p(y,z)}\left[\left(\E_{x\sim p(x|y,z)}[r(z|x)]\right)^2\right]-\E_{z\sim p(z)}\left[\left(\E_{x\sim p(x|z)}[r(z|x)]\right)^2\right]}
\end{aligned}
\end{equation}
\endgroup
\end{lemma}

The proof is given in Appendix \ref{app:lemma_G}, in which we also give Eq. (\ref{eq_app:G_empirical1}) for empirical estimation.
Note that Lemma \ref{lemma:second_order_variation} is very general and can be applied to any $p(z|x)$, not only at the optimal solution $\pstar$.

\paragraph{The Fisher Information matrix.}
In practice, the encoder $p_\thetaa(z|x)$ is usually parameterized by some parameter vector $\thetaa=(\theta_1,\theta_2,...\theta_k)^T\in\Theta$, e.g. weights and biases in a neural net, where $\Theta$ is the parameter field. An infinitesimal change of $\thetaa'\gets \thetaa +\Delta\thetaa$ induces a relative perturbation $\e\cdot r(z|x)\simeq\Delta\thetaa^T\frac{\partial \log p_\thetaa (z|x)}{\partial \thetaa}$ on $p_\thetaa(z|x)$, from which we can compute the threshold function $G_\Theta[p_\thetaa(z|x)]$:

\begin{lemma}
\label{lemma:G_theta}
For $\IBB_\beta[p_\thetaa(z|x)]$ objective, the condition of $\forall \Dthe\in\Theta$, $\delta^2\IBB_\beta[p_\thetaa(z|x)]\ge0$ is equivalent to $\beta\le G_\Theta[p_\thetaa(z|x)]$, where

\beq{eq_app:G_with_theta}
G_\Theta[p_{\thetaa}(z|x)]:=\inf_{\Delta\thetaa\in\Theta}\frac{\Delta\thetaa^T\left(\mathcal{I}_{Z|X}(\thetaa)-\mathcal{I}_{Z}(\thetaa)\right)\Delta\thetaa}{\Delta\thetaa^T\left(\mathcal{I}_{Z|Y}(\thetaa)-\mathcal{I}_{Z}(\thetaa)\right)\Delta\thetaa}=\lambda_\emph{max}^{-1}
\eeq

where $\mathcal{I}_{Z}(\thetaa):=\int dz p_\thetaa (z)\left(\frac{\partial \logg p_\thetaa(z)}{\partial \thetaa}\right)\left(\frac{\partial \logg p_\thetaa(z)}{\partial \thetaa}\right)^T$ is the Fisher information matrix of $\thetaa$ for $Z$, $\mathcal{I}_{Z|X}(\thetaa):= \int dxdz p(x) p_\thetaa(z|x)\left(\frac{\partial \logg p_\thetaa(z|x)}{\partial \thetaa}\right)\left(\frac{\partial \logg p_\thetaa(z|x)}{\partial \thetaa}\right)^T$,\\ $\mathcal{I}_{Z|Y}(\thetaa):= \int dydz p(y) p_\thetaa(z|y)\left(\frac{\partial \logg p_\thetaa(z|y)}{\partial \thetaa}\right)\left(\frac{\partial \logg p_\thetaa(z|y)}{\partial \thetaa}\right)^T$  are the conditional Fisher information matrix \citep{fisherproperty} of $\thetaa$ for $Z$ conditioned on $X$ and $Y$, respectively.
$\lambda_\emph{max}$ is the largest eigenvalue of $C^{-1}\left(\mathcal{I}_{Z|Y}(\thetaa)-\mathcal{I}_{Z}(\thetaa)\right)(C^T)^{-1}$ with $v_\emph{max}$ the corresponding eigenvector, where $CC^T$ is the Cholesky decomposition of the matrix $\mathcal{I}_{Z|X}(\thetaa)-\mathcal{I}_{Z}(\thetaa)$, and $v_\emph{max}$ is the eigenvector for $\lambda_\emph{max}$. The infimum is attained at $\Dthe=(C^T)^{-1}v_\emph{max}$.
\end{lemma}

The proof is in appendix \ref{app:G_theta}. We see that for parameterized encoders $p_\thetaa(z|x)$, each term of $G[p(z|x)]$ in Eq. (\ref{eq:G}) can be replaced by a bilinear form with the Fisher information matrix of the respective variables. 
Although this lemma is not required to understand the more general setting of Lemma~\ref{lemma:second_order_variation}, where the model is described in a functional space, Lemma~\ref{lemma:G_theta} helps understand $G[p(z|x)]$ for parameterized models, which permits directly linking the phase transitions to the model's parameters.

\paragraph{Phase Transitions.}
Now we introduce Theorem \ref{thm:phase_transition} that gives a concrete and practical condition for IB phase transitions, which is the core result of the paper:

\begin{theorem}
\label{thm:phase_transition}
The IB phase transition points $\{\beta_i^c\}$ as defined in Definition \ref{definition:phase_transition_0}  are given by the roots of the following equation:
\begin{equation}
\label{eq:phase_transition_root}
G[\pstar]=\beta
\end{equation}
where $G[p(z|x)]$ is given by Eq. (\ref{eq:G}) and $\pstar$ is the optimal solution of $\IB_\beta[p(z|x)]$ at $\beta$.

\end{theorem}

We can understand Eq. (\ref{eq:phase_transition_root}) as the condition when $\delta^2\IB_\beta[p(z|x)]$ is \emph{about} to be able to be negative at the optimal solution $\pstar$ for a given $\beta$.  The proof for Theorem \ref{thm:phase_transition} is given in Appendix \ref{app:phase_transition}. In Section \ref{sec:phase_understanding}, we will analyze Theorem \ref{thm:phase_transition} in detail.

\section{Understanding the formula for IB phase transitions}
\label{sec:phase_understanding}

In this section we set out to understand $G[p(z|x)]$ as given by Eq. (\ref{eq:G}) and the phase transition condition as given by Theorem \ref{thm:phase_transition}, from the perspectives of Jensen's inequality and representational maximum correlation.

\subsection{Jensen's Inequality}
\label{sec:Jensen}
The condition for IB phase transitions given by Theorem \ref{thm:phase_transition} involves $G[p(z|x)]=\inf\limits_{r(z|x)\in\QQ}\G[r(z|x);p(z|x)]$ which is in itself an optimization problem.
We can understand $G[p(z|x)]=\inf\limits_{r(z|x)\in\QQ}\frac{A-C}{B-C}$ in Eq. (\ref{eq:G}) using Jensen's inequality:

\beq{eq:Jensens}
\underbrace{\E_{x,z\sim p(x,z)}[r^2(z|x)]}_A\ge \underbrace{\E_{y,z\sim p(y,z)}\left[\left(\E_{x\sim p(x|y,z)}[r(z|x)]\right)^2\right]}_B\ge\underbrace{\E_{z\sim p(z)}\left[\left(\E_{x\sim p(x|z)}[r(z|x)]\right)^2\right]}_C
\eeq

The equality between $A$ and $B$ holds when the perturbation $r(z|x)$ is constant w.r.t. $x$ for any $z$; the equality between $B$ and $C$ holds when $\E_{x\sim p(x|y,z)}[r(z|x)]$ is constant w.r.t. $y$ for any $z$. 
Therefore, the minimization of $\frac{A-C}{B-C}$ encourages the relative perturbation function $r(z|x)$ to be as constant w.r.t. $x$ as possible (minimizing intra-class difference), but as different w.r.t. different $y$ as possible (maximizing inter-class difference), resulting in a \emph{clustering} of the values of $r(z|x)$ for different examples $x$ according to their class $y$. Because of this clustering property in classification problems, we conjecture that there are at most $|\Y|-1$ phase transitions, where $|\Y|$ is the number of classes, with each phase transition differentiating one or more classes. 

\subsection{Representational Maximum Correlation}
\label{sec:representational_maximum_correlation}

Under certain conditions we can further simplify $G[p(z|x)]$ and gain a deeper understanding of it.
Firstly, inspired by maximum correlation \citep{anantharam2013maximal}, we introduce two new concepts, \emph{representational maximum correlation} and \emph{conditional maximum correlation}, as follows.

\begin{definition}
\label{def:rho_r}
Given a joint distribution $p(X,Y)$, and a representation $Z$ satisfying the Markov chain $Z-X-Y$, the representational maximum correlation $\rho_r(X,Y;Z)$ is defined as
\begin{align}
\rho_r(X,Y;Z)&:=\sup_{\left(f(x,z),g(y,z)\right)\in\S_1}\E_{x,y,z\sim p(x,y,z)}[f(x,z)g(y,z)]
\end{align}

where $\S_1=\{(f:\X\times\Z\to\R, g:\Y\times\Z\to\R)\,\big|\,f, g\text{ bounded}, \text{and } \E_{x\sim p(x|z)}[f(x,z)]=\E_{y\sim p(y|z)}[g(y,z)]=0,\ \E_{x,z\sim p(x,z)}[f^2(x,z)]=\E_{y,z\sim p(y,z)}[g^2(y,z)]=1\}$.

The conditional maximum correlation $\rho_m(X,Y|Z)$ is defined as:
\begin{equation}
\rho_m(X,Y|Z):= \sup_{\left(f(x),g(y)\right)\in\S_2}\E_{x,y\sim p(x,y|z)}[f(x)g(y)]
\end{equation}
where $\S_2=\{(f:\X\to\R, g:\Y\to\R)\,\big|\,f, g\text{ bounded},\text{ and } \forall z\in\Z: \E_{x\sim p(x|z)}[f(x)]=\E_{y\sim p(y|z)}[g(y)]=0$, $\E_{x\sim p(x|z)}[f^2(x)]=\E_{y\sim p(y|z)}[g^2(y)]=1\}$.
\end{definition}

We prove the following Theorem \ref{thm:rho_r_property}, which expresses $G[p(z|x)]$ in terms of representational maximum correlation and related quantities, with proof given in Appendix \ref{app:representational_maximum_correlation}.

\begin{theorem}
\label{thm:rho_r_property}

Define $\QQA:=\{r(z|x): \mathcal{X}\times\mathcal{Z}\to \mathbb{R}\,\big|\,r \text{ bounded}\}$. If $\QQA$ and $\QQ$ satisfy: $\forall r(z|x)\in\QQA$, there exists\footnote{%
  For discrete $X$, $Z$ such that the cardinality $|\Z|\ge|\X|$, this is generally true since in this scenario, $h(x,z)$ and $s(z)$ have $|\X||\Z|+|\Z|$ unknown variables, but the condition has only $|\X||\Z|+|\X|$ linear equations. 
  The difference between $\QQ$ and $\QQA$ is that $\QQA$ does not have the requirement of $\E_{p(z|x)}[r(z|x)]=0$. Combined with Lemma \ref{lemma:G_invariance}, this condition allows us to replace $r(z|x)\in\QQ$ by $r(z|x)\in\QQA$ in Eq. (\ref{eq:G}).
} $r_1(z|x)\in\QQ$, $s(z)\in\{s(z):\Z\to\R\,|\,s \text{ bounded}\}$ s.t. $r(z|x)=r_1(z|x)+s(z)$, then we have:

\begin{enumerate}[label=(\roman*)]
\item The representation maximum correlation and $G$:
\begin{align}
\label{eq:G_simplified}
G[p(z|x)]=\frac{1}{\rho_r^2(X,Y;Z)}
\end{align}

\item The representational maximum correlation and conditional maximum correlation:
\begin{align}
\rho_r(X,Y;Z)&=\sup_{Z\in\mathcal{Z}}[\rho_m(X,Y|Z)]
\end{align}

\item When $Z$ is continuous, an optimal relative perturbation function $r(z|x)$ for $G[p(z|x)]$ is given by

\begin{equation}
r^*(z|x)=h^*(x)\sqrt{\frac{\delta(z-z^*)}{p(z)}}
\end{equation}
where $z^*=\argmax\limits_{z\in\mathcal{Z}}\rho_m(X,Y|Z=z)$, and $h^*(x)$ is the optimal solution for the learnability threshold function $h^*(x)=\argmin\limits_{h(x)\in\{h:\X\to\R\,|\,h\text{ bounded}\}} \beta_0[h(x)]$ with $p(X,Y|Z=z^*)$ ($\beta_0[h(x)]$ is given in Theorem 4 of \citet{wu2019learnabilityEntropy}).

\item For discrete $X$, $Y$ and $Z$, we have

\beq{eq:rho_r_svd}
\rho_r(X,Y;Z)=\max_{Z\in\Z}\sigma_2(Z)
\eeq

where $\sigma_2(Z)$ is the second largest singular value of the matrix $Q_{X,Y|Z}:=\left(\frac{p(x,y|z)}{\sqrt{p(x|z)p(y|z)}}\right)_{x,y}=\left(\frac{p(x,y)}{\sqrt{p(x)p(y)}}\sqrt{\frac{p(z|x)}{p(z|y)}}\right)_{x,y}$.

\end{enumerate}
\end{theorem}

Theorem \ref{thm:rho_r_property} furthers our understanding of $G[p(z|x)]$ and the phase transition condition (Theorem \ref{thm:phase_transition}), which we elaborate as follows.

\paragraph{Discovering maximum correlation in the orthogonal space of a learned representation:} 
Intuitively, the representational maximum correlation measures the maximum linear correlation between $f(X,Z)$ and $g(Y,Z)$ among all real-valued functions $f,g$, under the constraint that $f(X,Z)$ is ``orthogonal'' to $p(X|Z)$ and $g(Y,Z)$ is ``orthogonal'' to $p(Y|Z)$.
Theorem \ref{thm:rho_r_property} (i) reveals that $G[p(z|x)]$ is the inverse square of this representational maximum correlation.
Theorem \ref{thm:rho_r_property} (ii) further shows that $G[p(z|x)]$ is finding a specific $z^*$ on which maximum (nonlinear) correlation between $X$ and $Y$ conditioned on $Z$ can be found.
Combined with Theorem \ref{thm:phase_transition}, we have that when we continuously increase $\beta$, for the optimal representation $Z^*_\beta$ given by $\pstar$ at $\beta$, $\rho_r(X,Y;Z^*_\beta)$ shall monotonically decrease due to that $X$ and $Y$ has to find their maximum correlation on the orthogonal space of an increasingly better representation $Z^*_\beta$ that captures more information about $X$.
A phase transition occurs when $\rho_r(X,Y;Z^*_\beta)$ reduces to $\frac{1}{\sqrt{\beta}}$, after which as $\beta$ continues to increase, $\rho_r(X,Y;Z^*_\beta)$ will try to find maximum correlation between $X$ and $Y$ orthogonal to the full previously learned representation.
This is reminiscent of canonical-correlation analysis (CCA) \citep{hotelling1992relations} in linear settings, where components with decreasing linear maximum correlation that are orthogonal to previous components are found one by one.
In comparison, we show that in IB, each phase transition corresponds to learning a new \emph{nonlinear} component of maximum correlation between $X$ and $Y$ in $Z$, orthogonal to the previously-learned $Z$. In the case of classification where different classes may have different difficulty (e.g. due to label noise or support overlap), we should expect that classes that are less difficult as measured by a larger maximum correlation between $X$ and $Y$ are  learned earlier.

\paragraph{Conspicuous subset conditioned on a single $z$:}
Furthermore, we show in (iii) that an optimal relative perturbation function $r(z|x)$ can be decomposed into a product of two factors, a $\sqrt{\frac{\delta(z-z^*)}{p(z)}}$ factor that only focus on perturbing a specific point $z^*$ in the representation space, and an $h^*(x)$ factor that is finding the ``conspicuous subset'' \citep{wu2019learnabilityEntropy}, i.e. the most confident, large, typical, and imbalanced subset in the $X$ space for the distribution $(X,Y)\sim p(X,Y|z^*)$.

\paragraph{Singular values} 
In categorical settings, (iv) reveals a connection between $G[p(z|x)]$ and the singular value of the $Q_{X,Y|Z}$ matrix. Due to the property of SVD, we know that the square of the singular values of $Q_{X,Y|Z}$ equals the non-negative eigenvalue of the matrix $Q^T_{X,Y|Z}Q_{X,Y|Z}$. Then the phase transition condition in Theorem \ref{thm:phase_transition} is equivalent to a (nonlinear) eigenvalue problem. This is resonant with previous analogy with CCA in linear settings, and is also reminiscent of the linear eigenvalue problem in Gaussian IB \citep{chechik2005information}.

\section{Algorithm for phase transitions discovery in classification}
\label{sec:alg_for_discovering_transition}

As a consequence of the theoretical analysis above, we are able to derive an algorithm to efficiently estimate the phase transitions for a given model architecture and dataset.
This algorithm also permits us to empirically confirm some of our theoretical results in Section~\ref{sec:phase_experiments}.

Typically, classification involves high-dimensional inputs $X$.
Without sweeping the full range of $\beta$ where at each $\beta$ it is a full learning problem, it is in general a difficult task to estimate the phase transitions.
In Algorithm~\ref{alg:coordinate_descent}, we present a two-stage approach.

In the first stage, we train a single maximum likelihood neural network $f_\thetaa$ with the same encoder architecture as in the (variational) IB to estimate $p(y|x)$, and obtain an $N\times C$ matrix $p(y|x)$, where $N$ is the number of examples in the dataset and $C$ is the number of classes.
In the second stage, we perform an iterative algorithm w.r.t. $G$ and $\beta$, alternatively, to converge to a phase transition point.

Specifically, for a given $\beta$, we use a Blahut-Arimoto type IB algorithm \citep{tishby2000information} to efficiently reach IB optimal $\pstar$ at $\beta$, then use SVD (with the formula given in Theorem \ref{thm:rho_r_property} (iv)) to efficiently estimate $G[\pstar]$ at $\beta$ (step 8).
We then use the $G[\pstar]$ value as the new $\beta$ and do it again (step 7 in the next iteration).
At convergence, we will reach the phase transition point given by $G[\pstar]=\beta$ (Theorem \ref{thm:phase_transition}).
After convergence as measured by patience parameter $K$, we slightly increase $\beta$ by $\delta$ (step 13), so that the algorithm can discover the subsequent phase transitions.

\begin{algorithm}[t]
   \caption{\textbf{Phase transitions discovery for IB}}
\label{alg:coordinate_descent}
\begin{algorithmic}
\STATE {\bfseries Require} $(X, Y)$: the dataset 
\STATE {\bfseries Require} $f_\thetaa$: a neural net with the same encoder architecture as the (variational) IB
\STATE {\bfseries Require} $K$: patience
\STATE {\bfseries Require} $\delta$: precision floor
\STATE {\bfseries Require} $R$: maximum ratio between $\betath$ and $\beta$.
\STATE // \textit{First stage: fit $p(y|x)$ using neural net $f_
\thetaa$:}
\STATE 1: $p(y|x)\gets$ fitting $(X,Y)$ using $f_\thetaa$ via maximum likelihood.
\STATE 2: $p(x)\gets \frac{1}{N}$
\STATE
\STATE // \textit{Second stage: coordinate descent using $G[p(z|x)]$ and IB algorithm:}
\STATE 3: $\beta_0^c\gets\betath(1)$
\STATE 4: $\mathbb{B}\gets\{\beta_0^c\}$   \ \ \  //\textit{$\mathbb{B}$ is a set collecting the phase transition points}
\STATE 5: $(\betanew,\beta, \text{count})\gets(\beta_0^c,1,0)$
\STATE 6: \textbf{while} $\frac{\betanew}{\beta}<R$ \textbf{do}:
\STATE 7: \ \ \ \ \ \ \ \ $\beta\gets\betanew$
\STATE 8: \ \ \ \ \ \ \ \ $\betanew\gets\betath(\beta)$
\STATE 9: \ \ \ \ \ \ \ \ \textbf{if} $\betanew-\beta<\delta$ \textbf{do}:
\STATE 10: \ \ \ \ \ \ \ \ \ \ \ \ \ \ $\text{count}\gets\text{count}+1$
\STATE 11: \ \ \ \ \ \ \ \ \ \ \ \ \ \ \textbf{if} $\text{count}>K$ \textbf{do}:
\STATE 12: \ \ \ \ \ \ \ \ \ \ \ \ \ \ \ \ \ \ \ \  $\mathbb{B}\gets \mathbb{B}\cup\{\betanew\}$
\STATE 13: \ \ \ \ \ \ \ \ \ \ \ \ \ \ \ \ \ \ \ \  $\betanew\gets\betanew+\delta$
\STATE 14: \ \ \ \ \ \ \ \ \ \ \ \ \ \ \textbf{end if} 
\STATE 15: \ \ \ \ \ \ \textbf{else}: $\text{count}\gets 0$ 
\STATE 16: \ \ \ \ \ \ \textbf{end if}
\STATE 17: \textbf{end while}
\STATE 18: \textbf{return} $\mathbb{B}$

\STATE
\STATE \textbf{subroutine} $\betath(\beta)$:
\STATE s1: Compute $\pstar$ using the IB algorithm \citep{tishby2000information}.
\STATE s2: $\betanew\gets G[\pstar]$  using SVD (Eq. \ref{eq:G_simplified} and \ref{eq:rho_r_svd}).
\STATE s3: \textbf{return} $\betanew$
\end{algorithmic}
\end{algorithm}

\section{Empirical study}
\label{sec:phase_experiments}
We quantitatively and qualitatively test the ability of our theory and Algorithm \ref{alg:coordinate_descent} to provide good predictions for IB phase transitions.
We first verify them in fully categorical settings, where $X,Y,Z$ are all discrete, and we show that the phase transitions can correspond to learning new classes as we increase $\beta$.
We then test our algorithm on versions of the MNIST and CIFAR10 datasets with added label noise.

\subsection{Categorical dataset}
For categorical datasets, $X$ and $Y$ are discrete, and $p(X)$ and $p(Y|X)$ are given.
To test Theorem \ref{thm:phase_transition}, we use the Blahut-Arimoto IB algorithm to compute the optimal $\pstar$ for each $\beta$.
$I(Y;Z^*)$ vs. $\beta$ is plotted in Fig. \ref{fig:beta_th_vs_beta} (a).
There are two phase transitions at $\beta_0^c$ and $\beta_1^c$.
For each $\beta$ and the corresponding $\pstar$, we use the SVD formula (Theorem \ref{thm:rho_r_property}) to compute $G[\pstar]$, shown in Fig. \ref{fig:beta_th_vs_beta} (b).
We see that $G[\pstar]=\beta$ at exactly the observed phase transition points $\beta_0^c$ and $\beta_1^c$. Moreover, starting at $\beta=1$, Alg. 1 converges to each phase transition points within few iterations.
Our other experiments with random categorical datasets show similarly tight matches.

Furthermore, in Appendix \ref{app:subset_separation} we show that the phase transitions correspond to the onset of separation of $p(z|x)$ for subsets of $X$ that correspond to different classes.
This supports our conjecture from Section \ref{sec:Jensen} that there are at most $|\Y|-1$ phase transitions in classification problems.

\begin{figure}[t]
\begin{center}
\begin{subfigure}[b]{0.34\textwidth}
\centering
\includegraphics[width=\textwidth]{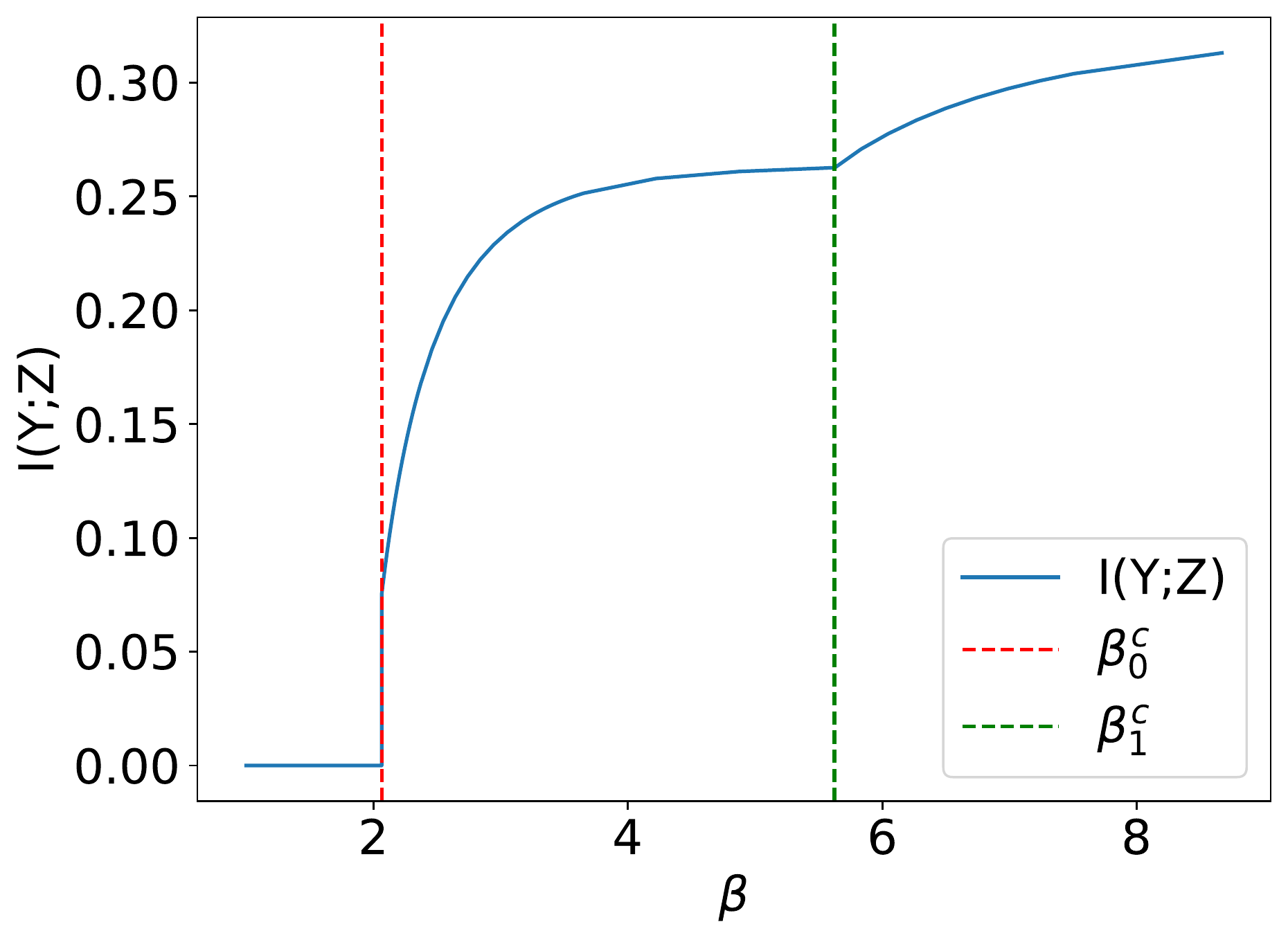}
\caption{}
\end{subfigure}
\begin{subfigure}[b]{0.37\textwidth}
\centering
\includegraphics[width=\textwidth]{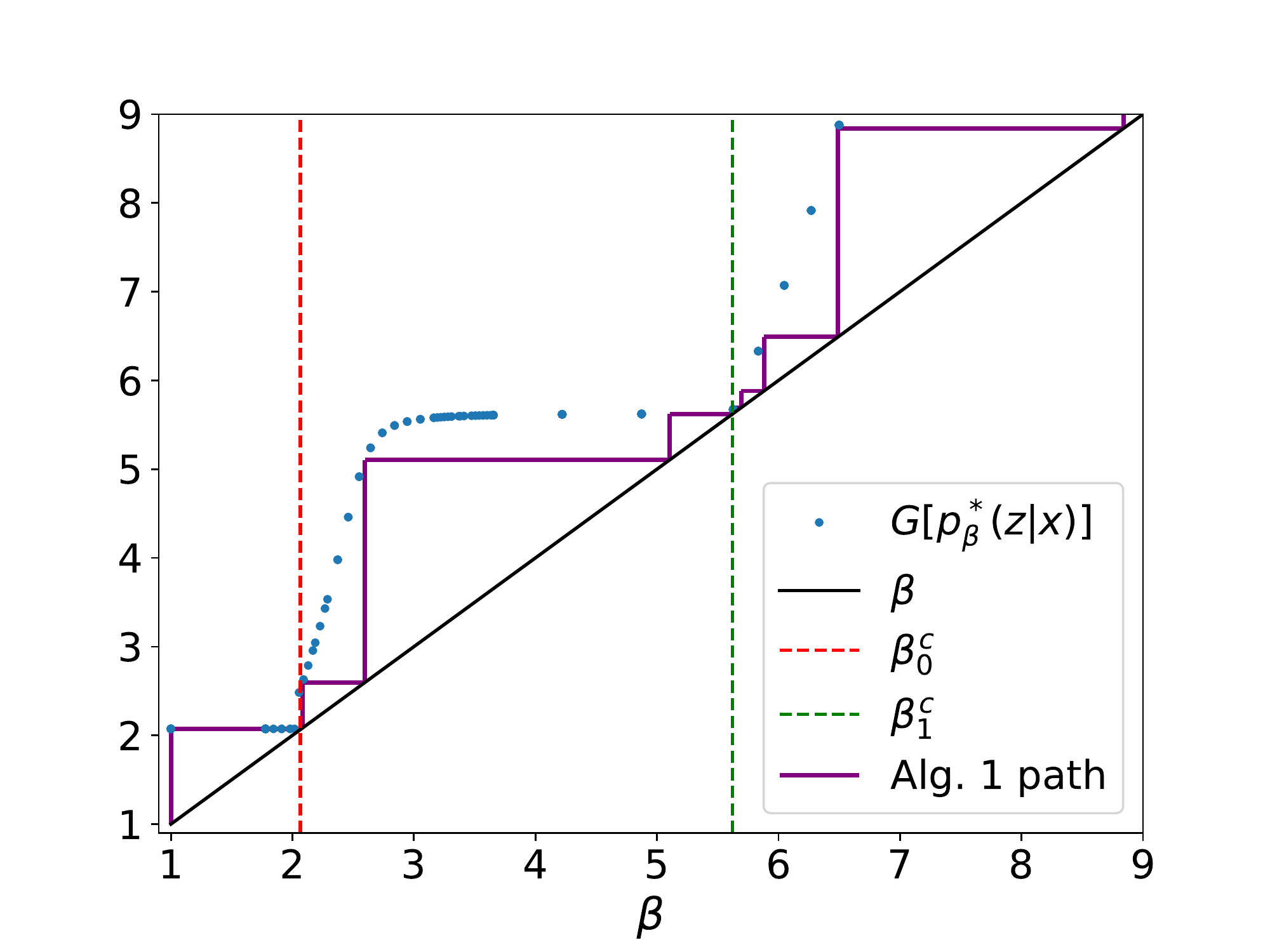}
\caption{}
\end{subfigure}
\end{center}
\caption{(a) $I(Y;Z^*)$ vs. $\beta$ for a categorical dataset with $|X|=|Y|=|Z|=3$, where $Z^*$ is given by $\pstar$, and the vertical lines are the experimentally discovered phase transition points $\beta_0^c$ and $\beta_1^c$. (b) $G[\pstar]$ vs. $\beta$ for the same dataset, and the path for Alg. \ref{alg:coordinate_descent}, with $\beta_0^c$ and $\beta_1^c$ in (a) also plotted. The dataset is given in Fig. \ref{fig:py_x}.}
\label{fig:beta_th_vs_beta}
\end{figure}

\subsection{MNIST dataset}
For continuous $X$, how does our algorithm perform, and will it reveal aspects of the dataset?
We first test our algorithm in a 4-class MNIST with noisy labels\footnote{%
  We use 4 classes since it is simpler than the full 10 classes, but still potentially possesses phase transitions. We use noisy label to mimic realistic settings where the data may be noisy and also to have controllable difficulty for different classes.
}, whose confusion matrix and experimental settings are given in Appendix \ref{app:MNIST_exp}.
Fig. \ref{fig:phase_mnist} (a) shows the path Alg. \ref{alg:coordinate_descent} takes.
We see again that in each phase Alg. \ref{alg:coordinate_descent} converges to the phase transition points within a few iterations, and it discovers in total 3 phase transition points.
Similar to the categorical case, we expect that each phase transition corresponds to the onset of learning a new class, and that the last class is much harder to learn due to a larger separation of $\beta$. Therefore, this class should have a much larger label noise so that it is hard to capture this component of maximum correlation between $X$ and $Y$, as analyzed in representational maximum correlation (Section \ref{sec:representational_maximum_correlation}).
Fig. \ref{fig:phase_mnist} (b) plots the per-class accuracy with increasing $\beta$ for running the Conditional Entropy Bottleneck \citep{fischer2018the} (another variational bound on IB).
We see that the first two predicted phase transition points $\beta_0^c$, $\beta_1^c$ closely match the observed onset of learning class 3 and class 0.
Class 1 is observed to learn earlier than expected, possibly due to the gap between the variational IB objective and the true IB objective in continuous settings. 
By looking at the confusion matrix for the label noise (Fig. \ref{fig:py_x_4}),
we see that the ordering of onset of learning: class 2, 3, 0, 1, corresponds exactly to the decreasing diagonal element $p(\tilde{y}=1|y=1)$ (increasing noise) of the classes, and as predicted, class 1 has a much smaller diagonal element $p(\tilde{y}=1|y=1)$ than the other three classes, which makes it much more difficult to learn.
This ordering of classes by difficulty is what our representational maximum correlation predicts.

\begin{figure}
    \centering
    \begin{subfigure}{.36\columnwidth}
    \includegraphics[scale=0.22]{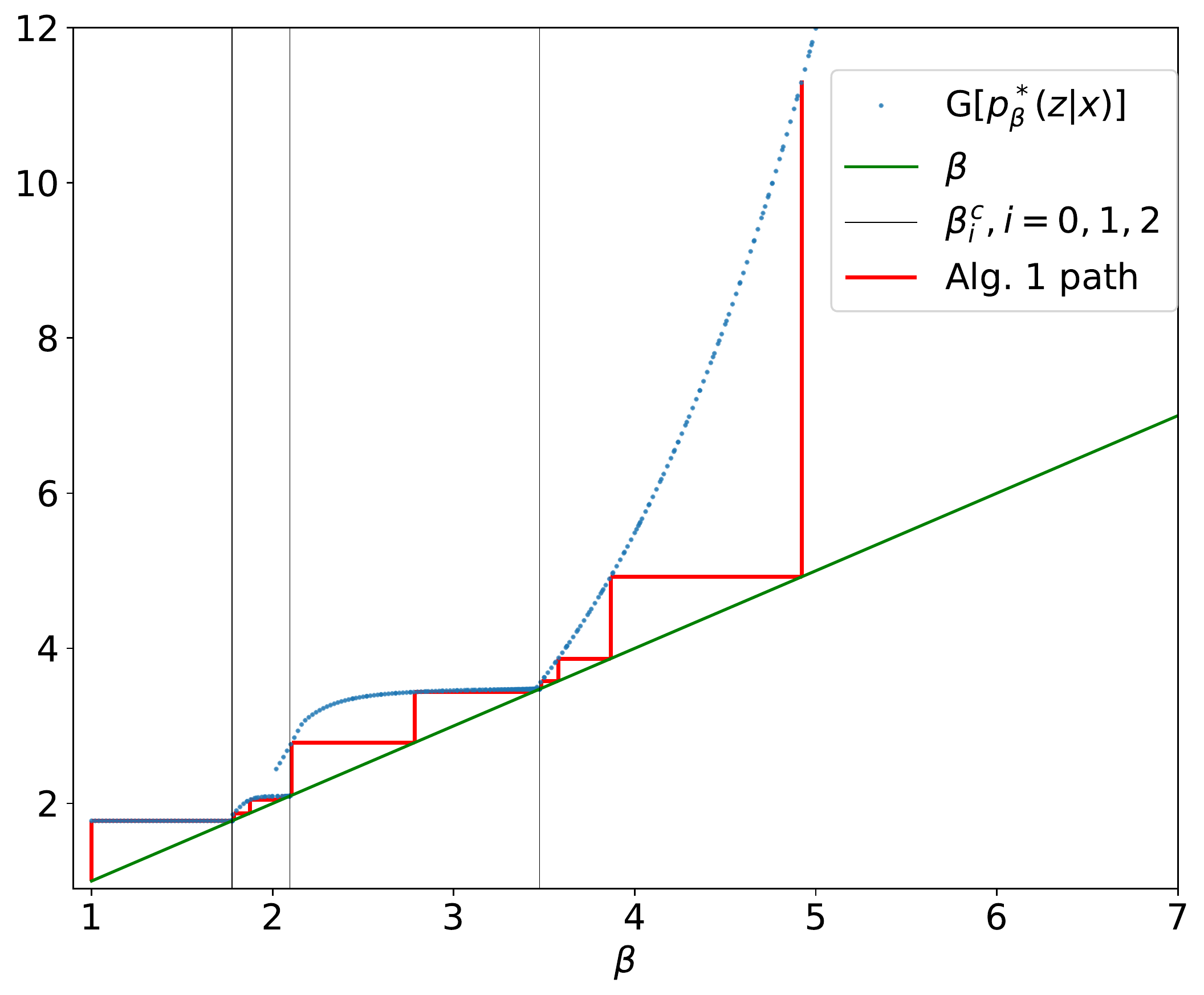}
    \caption{}
    \end{subfigure}
    \begin{subfigure}{.42\columnwidth}
    \includegraphics[scale=0.22]{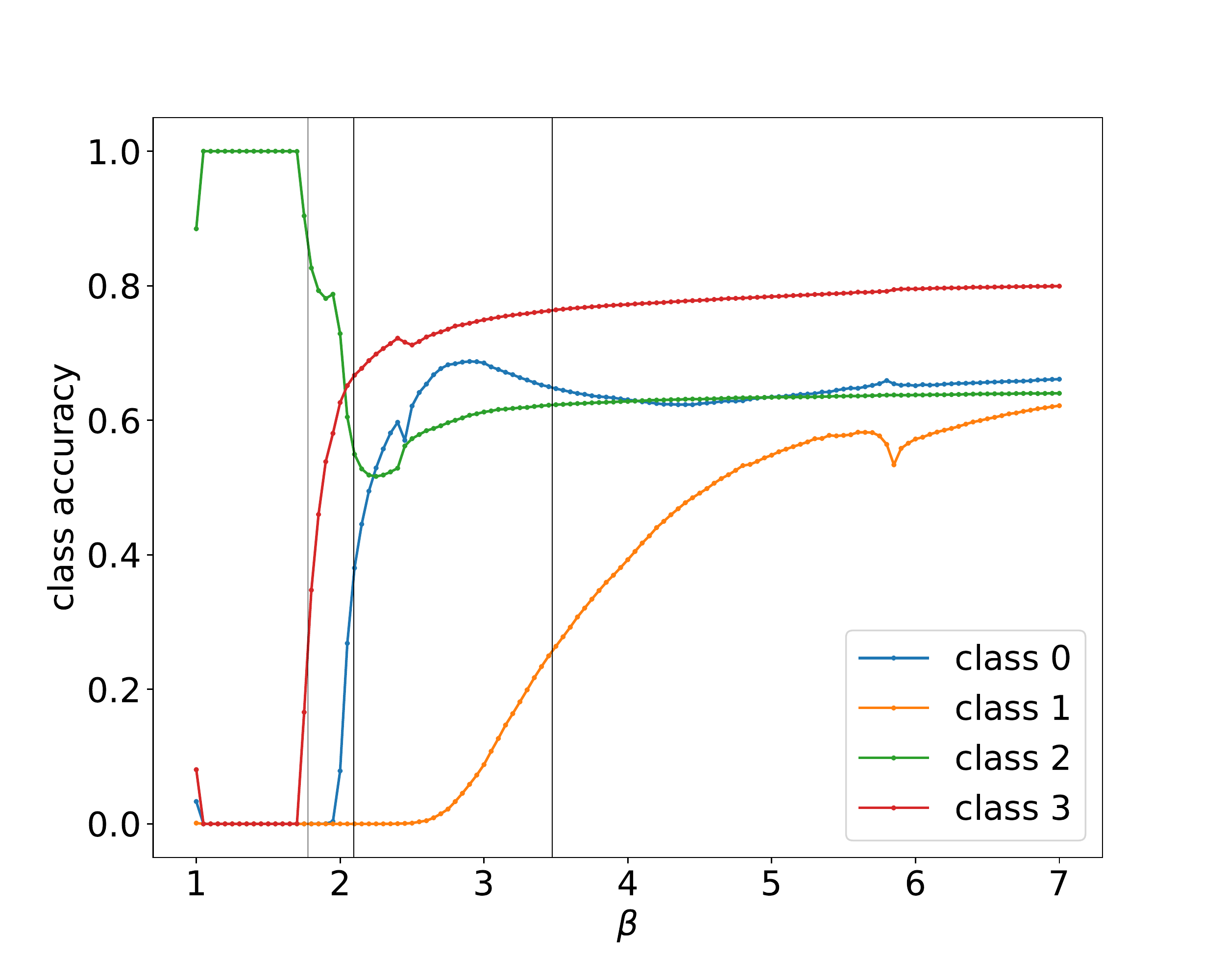}
    \caption{}
    \end{subfigure}
    \caption{%
    (a) Path of Alg. \ref{alg:coordinate_descent} starting with $\beta=1$, where the maximum likelihood model $f_\thetaa$ is using the same encoder architecture as in the CEB model. 
    This stairstep path shows that Alg. \ref{alg:coordinate_descent} is able to ignore very large regions of $\beta$, while quickly and precisely finding the phase transition points.
    Also plotted is an accumulation of $G[\pstar]$ vs. $\beta$ by running Alg. 1 with varying starting $\beta$ (blue dots).
    (b) Per-class accuracy vs. $\beta$, where the accuracy at each $\beta$ is from training an independent CEB model on the dataset.
    The per-class accuracy denotes the fraction of correctly predicted labels by the CEB model for the observed label $\tilde{y}$.
    }%
    \label{fig:phase_mnist}%
\end{figure}

\subsection{CIFAR10 dataset}

Finally, we investigate the CIFAR10 experiment from Section \ref{sec:phase_introduction}.
The details of the experimental setup are described in Appendix \ref{app:phase_cifar_details}.
This experiment stretches the current limits of our discrete approximation to the underlying continuous representation being learned by the models.
Nevertheless, we can see in Fig. \ref{fig:phase_cifar_exp} that many of the visible empirical phase transitions are tightly identified by Alg. \ref{alg:coordinate_descent}. Particularly, the onset of learning is predicted quite accurately; the large interval between the predicted $\beta_3=1.21$ and $\beta_4=1.61$ corresponds well to the continuous increase of $I(X;Z)$ and $I(Y;Z)$ at the same interval. And Alg. 1 is able to identify many dense transitions not obviously seen by just looking at $I(Y;Z)$ vs. $\beta$ curve alone.
Alg. 1 predicts 9 phase transitions, exactly equal to $|\Y|-1$ for CIFAR10.

\begin{figure}
\centering
\begin{subfigure}{.32\columnwidth}
\begin{center}
\includegraphics[scale=0.19]{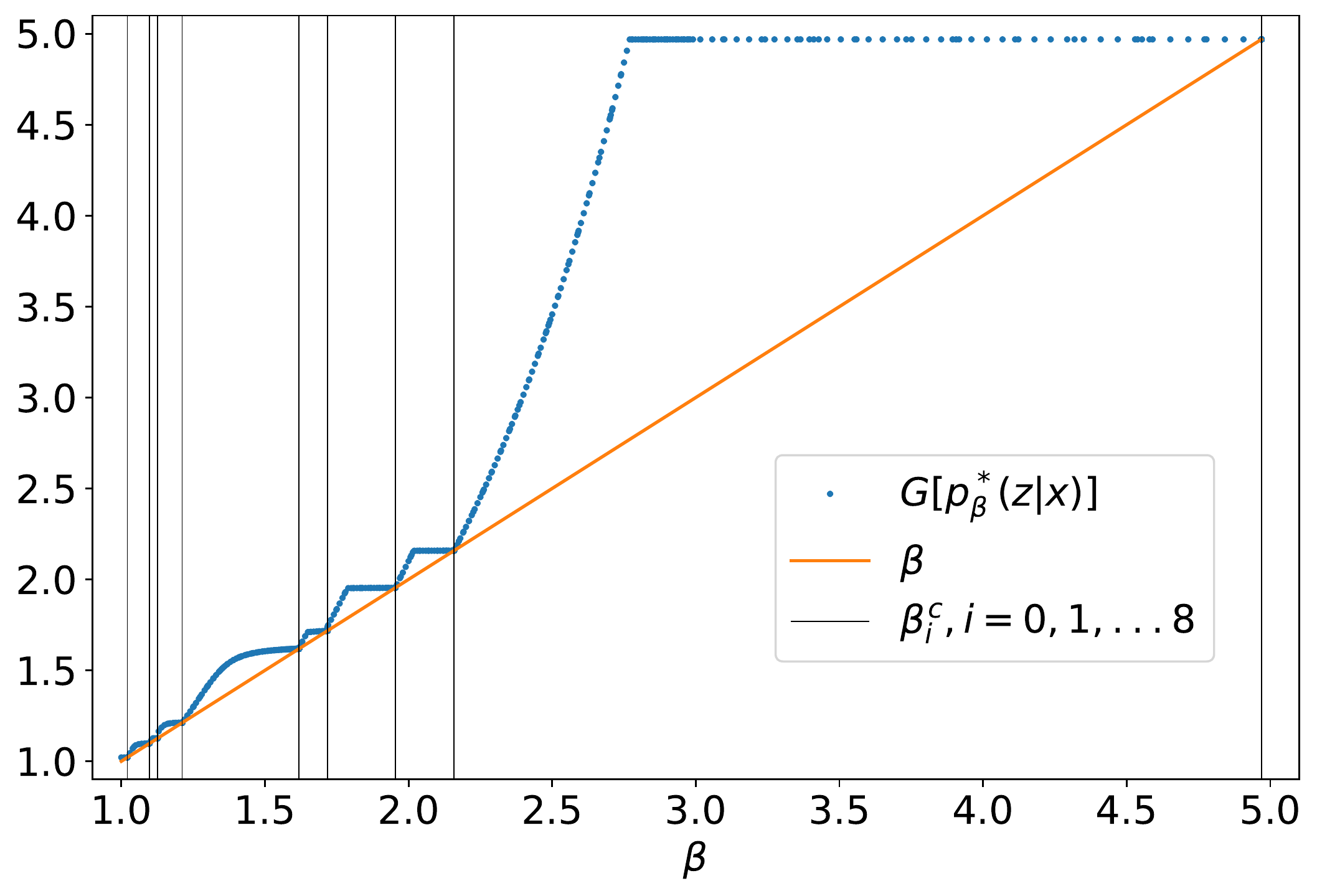}
\end{center}
\caption{}
\end{subfigure}
\begin{subfigure}{.32\columnwidth}
\includegraphics[scale=0.25]{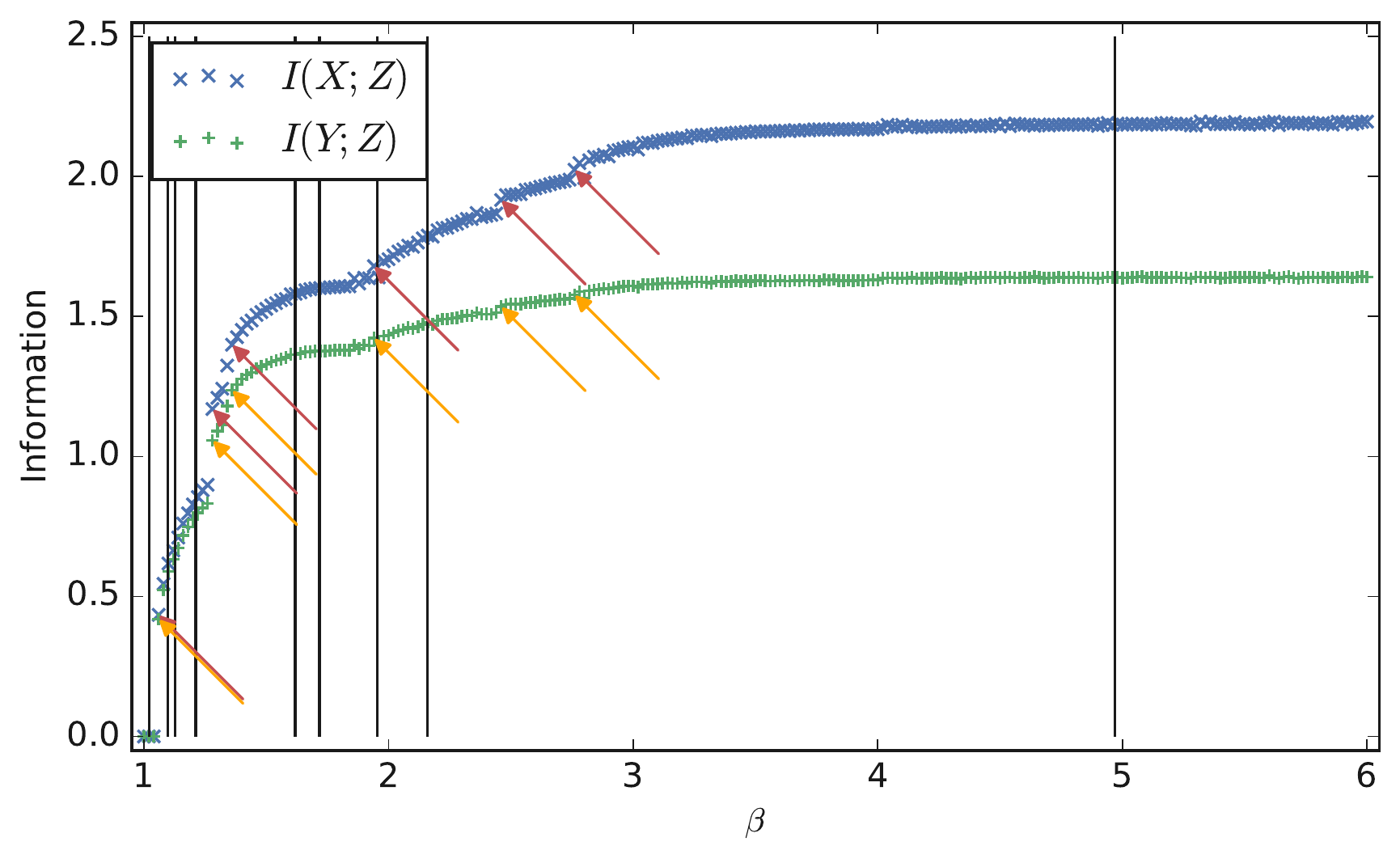}
\caption{}
\end{subfigure}
\hfill
\begin{subfigure}{.32\columnwidth}
\includegraphics[scale=0.26]{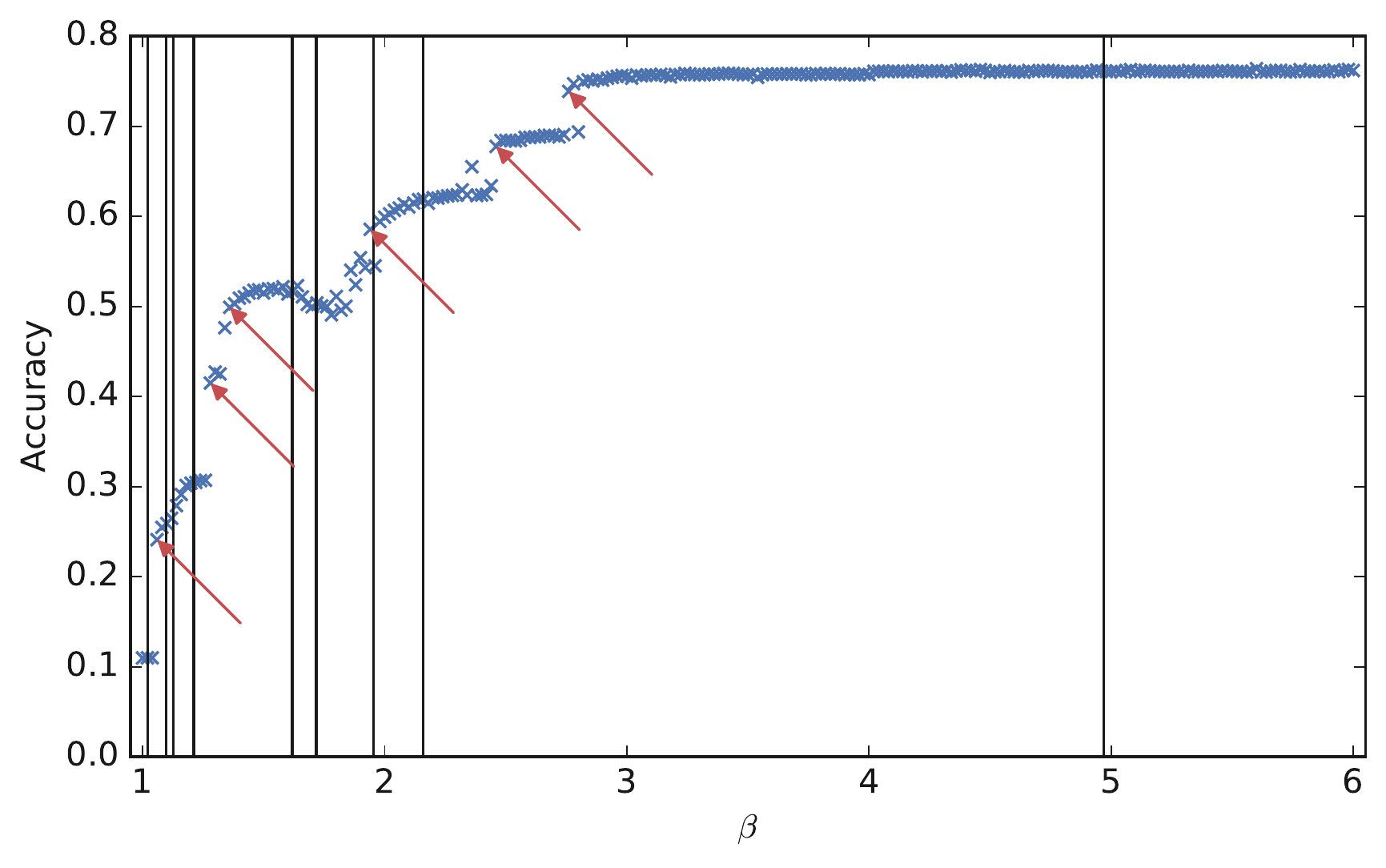}
\caption{}
\end{subfigure}
\caption{%
  \textbf{(a)} Accumulated $G[\pstar]$ vs. $\beta$ by running Alg. 1 with varying starting $\beta$ (blue dots).
  Also plotted are predicted phase transition points.
  \textbf{(b)} $I(X;Z)$ and $I(Y;Z)$ vs. $\beta$.
  The manually-identified phase transition points are labelled with arrows.
  The vertical black lines are the phase transitions identified by Alg. \ref{alg:coordinate_descent}, denoted as $\beta_0^c$, $\beta_1^c$,... $\beta_8^c$, from left to right.
  \textbf{(c)} Accuracy vs. $\beta$ with the same sets of points identified.
  The most interesting region is right before $\beta=2$, where accuracy \emph{decreases} with $\beta$.
  Alg. \ref{alg:coordinate_descent} identifies both sides of that region, as well as points at or near all of the early obvious phase transitions.
  It also seems to miss later transitions, possibly due to the gap between the variational IB objective and the true IB objective in continuous settings.
}%
\label{fig:phase_cifar_exp}%
\end{figure}

\section{Conclusion}

In this work, we observe and study the phase transitions in IB as we vary $\beta$.
We introduce the definition for $\IB$ phase transitions, and based on it derive a formula that gives a practical condition for $\IB$ phase transitions.
We further understand the formula via Jensen's inequality and representational maximum correlation. We reveal the close interplay between the IB objective, the dataset and the learned representation, as each phase transition is learning a nonlinear maximum correlation component in the orthogonal space of the learned representation. 
We present an algorithm for finding the phase transitions, and show that it gives tight matches with observed phase transitions in categorical datasets, predicts onset of learning new classes and class difficulty in MNIST, and predicts prominent transitions in CIFAR10 experiments.
This work is a first theoretical step towards a deeper understanding of the phenomenon of phase transitions in the Information Bottleneck. We believe our approach will be applicable to other ``trade-off'' objectives, like $\beta$-VAE~\citep{betavae} and InfoDropout~\citep{achille2018emergence}, where the model's ability to predict is balanced against a measure of complexity.

\chapter{Pareto-optimal data compression for binary classification tasks}
\label{chap5:distillation}

The\footnote{The paper ``Pareto-optimal data compression for binary classification tasks'' is Published at \emph{Entropy} 2020, 22(1), 7. Authors: Tegmark, Max and Tailin Wu. \cite{tegmark2019pareto}}\footnote{The code is open-sourced at \href{https://github.com/tailintalent/distillation}{github.com/tailintalent/distillation}.} goal of lossy data compression is to reduce the storage cost of a data set $X$ while retaining as much information as possible about something ($Y$) that you care about. For example, what aspects of an image $X$ contain the most information about whether it depicts a cat?
Mathematically, this corresponds to finding a mapping $X\to Z\equiv f(X)$ that 
maximizes the mutual information $I(Z,Y)$ while the entropy $H(Z)$ is kept below some fixed threshold. 
We present a new method for mapping out the Pareto frontier for classification tasks, reflecting the tradeoff between retained entropy and class information. We first show how a random variable $X$ (an image, say) drawn from a class $Y\in\{1,...,n\}$ can be distilled into a vector $W=f(X)\in \mathbb{R}^{n-1}$ losslessly, so that $I(W,Y)=I(X,Y)$; for example, for a binary classification task of cats and dogs, each image $X$ is mapped into a single real number $W$ retaining all information that helps distinguish cats from dogs. For the $n=2$ case of binary classification, we then show how $W$ can be further compressed into a discrete variable $Z=g_\beta(W)\in\{1,...,m_\beta\}$ by binning $W$ into $m_\beta$ bins, in such a way that varying the parameter $\beta$ sweeps out the full Pareto frontier, solving a generalization of the Discrete Information Bottleneck (DIB) problem.
We argue that the most interesting points on this frontier are ``corners" maximizing $I(Z,Y)$ for a fixed number of bins $m=2,3...$ which can be conveniently be found without multiobjective optimization. We apply this method to the CIFAR-10, MNIST and Fashion-MNIST datasets, illustrating how it can be interpreted as an information-theoretically optimal image clustering algorithm. We find that these Pareto frontiers are not concave, and that recently reported DIB phase transitions correspond to transitions between these corners, changing the number of clusters.

\section{Introduction}

A core challenge in science, and in life quite generally, is data distillation: 
keeping only a manageably small fraction of our available data $X$ while retaining as much information as possible about something ($Y$) that we care about. For example, what aspects of an image contain the most information about whether it depicts a cat ($Y=1$) rather than a dog ($Y=2$)?
Mathematically, this motivates finding a mapping $X\to Z\equiv g(X)$ that 
maximizes the mutual information $I(Z,Y)$ while the entropy $H(Z)$ is kept below some fixed threshold. 
The tradeoff between $H_*=H(Z)$ (bits stored) and $I_*=I(Z,Y)$ (useful bits) is described by a Pareto frontier, defined as  
\beq{ParetoDefEq}
I_*(H_*) \equiv\sup_{\{g: H[g(X)]\le H_*\}} I[g(X),Y],
\eeq
and illustrated in \fig{paretoAnalyticFig} (this is for a toy example described below; we compute the Pareto frontier for our cat/dog example in \Sec{ResultsSec}).
The shaded region is impossible because 
$I(Z,Y)\le I(X,Y)$ 
and  $I(Z,Y)\le H(Z)$.
The colored dots correspond to random likelihood binnings into various numbers of bins, as described in the next section, and the upper envelope of all attainable points define the Pareto frontier. Its ``corners'', which are marked by black dots and maximize $I(Z,Y)$ for $M$ bins ($M=1,2,...$), are seen to lie close to the vertical dashed lines 
$H(Z)=\log M$, corresponding to all bins having equal size. We plot the $H$-axis flipped to conform with the tradition that up and to the right are more desirable.
The core goal of this paper is to present a method for computing such Pareto frontiers.

\begin{figure}[h!]
\begin{center}
\vskip-5.4mm
\includegraphics[width=0.8\columnwidth]{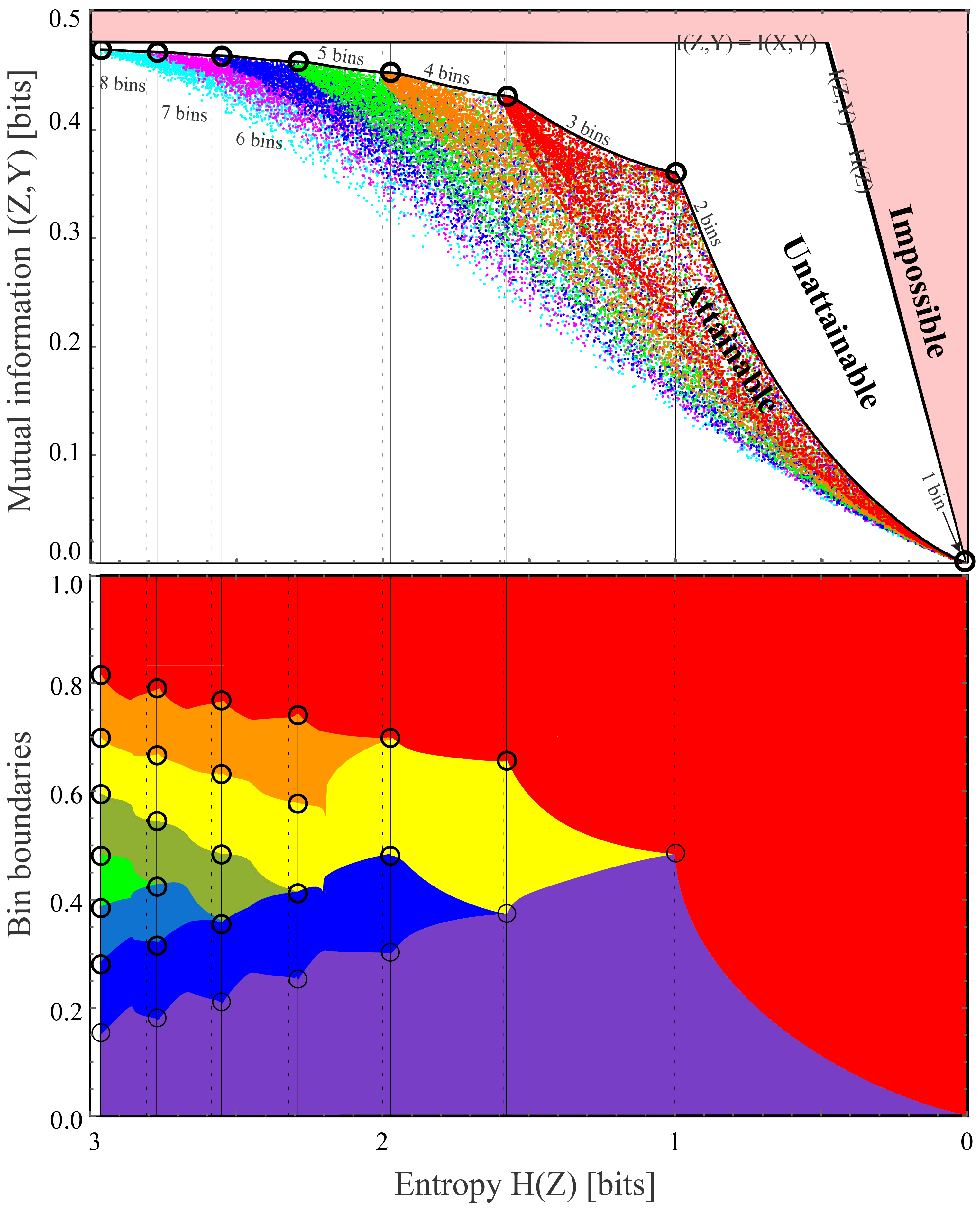}
\end{center}
\vskip-8mm
\caption{The Pareto frontier (top panel) for compressed versions $Z=g(X)$ of our warmup dataset $X)\in[0,1]^2$ with classes $Y\in\{1,2\}$, showing the maximum attainable class information $I(Z,Y)$ for a given entropy $H(Z)$, mapped using the method described in this paper using the likelihood binning in the bottom panel.
}
\label{paretoAnalyticFig}
\end{figure}

\subsection{Objectives \& relation to prior work}

In other words, the goal of this paper is to analyze soft rather than hard classifiers:
not to make the most accurate classifier, but rather to compute the Pareto frontier that reveals the 
most accurate (in an information-theoretic sense) classifier $Z$ given a constraint on its bit content $H(Z)$.
This Pareto frontier challenge is part of the broader quest for data distillation:
lossy data compression that retains as much as possible of the information that is useful to us. 
Ideally, the information can be partitioned into a set of independent chunks and sorted from most to least useful, enabling us to select the number of chunks to retain so as to optimize our tradeoff between utility and data size.
Consider two random variables $X$ and $Y$ which may each be vectors or scalars.
For simplicity, consider them to be discrete with finite entropy\footnote{The discreteness restriction loses us no generality in practice, since  since we can always discretize real numbers by rounding them to some very large number of significant digits.}.
For prediction tasks, we might interpret $Y$ as the future state of a dynamical system that we wish to predict from the present state $X$. For classification tasks, we might interpret $Y$ as a class label that we wish to predict from an image, sound, video or text string $X$.  Let us now consider various forms of ideal data distillation, as summarized in Table~\ref{ComparisonTable}.

\begin{table}[h!]
\begin{tabular}{|c|l|l|l|}
\hline
Random		&What is				&\multicolumn{2}{c|}{Probability distribution}\\
\cline{3-4}
vectors		&distilled?			&Gaussian			&Non-Gaussian\\
\hline
1			&Entropy				&PCA				&Autoencoder\\
			&$H(X)=\sum_i H(Z_i)$	&$\z=\F\x$			&$Z=f(X)$\\
\hline
2			&Mutual information		&CCA				&Latent reps\\
			&$I(X,Y)=\sum_i I(Z_i,Z'_i)$	&$\z=\F\x$			&$Z=f(X)$\\
			&					&$\z'=\G\y$			&$Z'=g(Y)$\\
\hline
\end{tabular}
\caption{Data distillation: the relationship between Principal Component Analysis (PCA), Canonical Correlation Analysis (CCA), nonlinear autoencoders and nonlinear latent representations.
\label{ComparisonTable}
}
\end{table}

If we distill $X$ as a whole, then we would ideally like to find a function $f$ such that
the so-called latent representation $Z=f(X)$ retains the full entropy
$H(X)=H(Z)=\sum H(Z_i)$, decomposed into independent\footnote{When implementing any distillation algorithm in practice, there is always a one-parameter tradeoff between compression and information retention which defines a Pareto frontier. A key advantage of the latent variables (or variable pairs) being statistically independent is that this allows the Pareto frontier to be trivially computed, by simply sorting them by decreasing information content and varying the number retained.
}
parts with vanishing mutual infomation: 
$I(Z_i,Z_j)=\delta_{ij}H(Z_i).$
For the special case where $X=\x$ is a vector with a multivariate Gaussian distribution, 
the optimal solution is 
Principal Component Analysis (PCA) \cite{pearsonPCA1901}, which
has long been a workhorse of statistical physics and many other disciplines: here $f$ is simply a linear function mapping into the eigenbasis of the covariance matrix of $\x$.
The general case remains unsolved, and it is easy to see that it is hard: 
if $X=c(Z)$ where $c$ implements some state-of-the-art cryptographic code, then finding $f=c^{-1}$ (to recover the independent pieces of information and discard the useless parts) would generically require breaking the code. Great progress has nonetheless been made for many special cases, using techniques such as nonlinear autoencoders \cite{vincent2008extracting} and Generative Adversarial Networks (GANs) \cite{goodfellow2014generative}.

Now consider the case where we wish to distill $X$ and $Y$ separately,
into $Z\equiv f(X)$ and $Z'=g(Y)$, retaining the mutual information between the two parts. 
Then we ideally have
$I(X,Y)=\sum_i I(Z_i,Z'_i)$,
$I(Z_i,Z_j)=\delta_{ij}H(Z_i),$
$I(Z'_i,Z'_j)=\delta_{ij}H(Z'_i),$
$I(Z_i,Z'_j)=\delta_{ij}I(Z_i,Z'_j).$
This problem has attracted great interest, especially for time series where $X=\u_i$ and $Y=\u_j$ for some sequence of states $\u_k$ ($k=0,1, 2, ...$) in physics or other fields, 
where one typically maps the state vectors $\u_i$ into some lower-dimensional 
vectors $f(\u_i)$, after which the prediction is carried out in this latent space. 
For the special case where $X$ has a multivariate Gaussian distribution, 
the optimal solution is 
Canonical Correlation Analysis (CCA) \cite{hotellingCCA1936}: here both $f$ and $g$ are linear functions, computed via a  
singular-value decomposition (SVD) \cite{eckart1936SVD} of the cross-correlation matrix after prewhitening $X$ and $Y$.
The general case remains unsolved, and is obviously even harder than the above-mentioned 1-vector autoencoding problem. 
The recent work \cite{oord2018representation,clark2019unsupervised} review the state-of-the art as well as presenting  Contrastive Predictive Coding and Dynamic Component Analysis, powerful new distillation techniques for time series, following the long tradition of setting $f=g$ even though this is provably not optimal for the Gaussian case as shown in \cite{tegmark2019optimal}.

The goal of this paper is to make progress in the lower right quadrant of Table~\ref{ComparisonTable}. 
We will first show that if $Y\in \{1,2\}$ (as in binary classification tasks) and we can successfully train a classifier that correctly predicts the conditional probability distribution $p(Y|X)$, then it can be used to provide an exact solution to the distillation problem,
losslessly distilling $X$ into a single real variable $W=f(X)$. 
We will generalize this to an arbitrary classification problem $Y\in\{1,...,n\}$ by losslessly distilling $X$ into a vector $W=f(X)\in \mathbb{R}^{n-1}$, although in this case, the components of the vector $W$ may not be independent.
We will then return to the binary classification case and provide a family of binnings that map $W$ into an integer $Z$, allowing us to scan the full Pareto frontier reflecting the tradeoff between retained entropy and class information, illustrating the end-to-end procedure with the CIFAR-10, MNIST and Fashion-MNIST datasets. 
This is related to the work of \cite{kurkoski2014quantization} which maximizes $I(Z,Y)$ for a fixed number of bins (instead of for a fixed entropy), which corresponds to the ``corners'' seen in \fig{paretoAnalyticFig}.

This work is closely related to the Information Bottleneck (IB) method \cite{tishby2000information}, which provides an insightful, principled approach for balancing compression against prediction \cite{tan2019renormalization}. 
Just as in our work, the IB method aims to find a random variable $Z=f(X)$ that loosely speaking retains as much information as possible about $Y$ and as little other information as possible.
The IB method implements this by maximizing the IB-objective
\beq{IBeq}
{\cal L}_{\rm IB}= I(Z,Y)-\beta I(Z,X)
\eeq
where the Lagrange multiplier $\beta$ tunes the balance between knowing about $Y$ and forgetting about $X$. 
\cite{strouse2017deterministic} considered the alternative Deterministic Information Bottleneck (DIB) objective 
\beq{IBeq2}
{\cal L}_{\rm DIB}= I(Z,Y)-\beta H(Z),
\eeq
to close the loophole where $Z$ retains random information that is independent of both $X$ and $Y$ (which is possible if $f$ is function that contains random components rather than fully deterministic\footnote{If $Z=f(X)$ for some deterministic function $f$, 
which is typically not the case in the popular 
variational IB-implementation \cite{alemi2016deep,chalk2016relevant,fischer2018the},
then $H(Z|X)=0$, so $I(Z,X) \equiv H(Z)-H(Z|X) = H(Z)$, 
which means the two objectives\eqn{IBeq} 
and\eqn{IBeq2} are identical.}).
However, there is a well-known problem with this objective that occurs when $Z\in \mathbb{R}^n$ is continuous \cite{amjad2019learning}: 
$H(Z)$ is strictly speaking infinite, since it requires an infinite amount of information to store the infinitely many decimals of a generic real number. 
Although this infinity is normally regularized away by only defining $H(Z)$ up to an additive constant, which is irrelevant when minimizing \eqn{IBeq2}, the problem is that we can define a new rescaled random variable
\beq{RescalingEq}
Z'=aZ
\eeq
for a constant $a\ne 0$ and obtain\footnote{Throughout this paper, we take $\log$ to denote the logarithm in base $2$, so that entropy and mutual information are measured in bits.}
\beq{IscalingEq}
I(Z',X)=I(Z,X)
\eeq
and
\beq{HscalingEq}
H(Z')=H(Z)+n\log|a|.
\eeq
This means that by choosing $|a|\ll 1$, we can make $H(Z')$ arbitrarily negative while keeping $I(Z',X)$ unchanged, thus making ${\cal L}_{\rm DIB}$ arbitrarily negative.
The objective ${\cal L}_{\rm DIB}$ is therefore not bounded from below, and trying to minimize it will not produce an interesting result.
We will eliminate this $Z$-rescaling problem by making $Z$ discrete rather than continuous, so that $H(Z)$ is always well-defined and finite. 
Another challenge with the DIB objective of \eq{IBeq2}, which we will also overcome, is that it maximizes a linear combination of the two axes in \fig{paretoAnalyticFig}, and can therefore  only discover concave parts of the Pareto frontier, not convex ones (which are seen to dominate in \fig{paretoAnalyticFig}).

The rest of this paper is organized as follows:
In \Sec{MethodSecA}, we will provide an exact solution for the binary classification problem where $Y\in\{1,2\}$ by losslessly distilling $X$ into a single real variable $Z=f(X)$. We also generalize this to an arbitrary classification problem  $Y\in\{1,...,n\}$ by losslessly distilling $X$ into a vector $W=f(X)\in \mathbb{R}^{n-1}$, although the components of the vector $W$ may not be independent.
In \Sec{MethodSecB}, we return to the binary classification case and provide a family a binnings that map $Z$ into an integer, allowing us to scan the full Pareto frontier reflecting the tradeoff between retained entropy and class information. We apply our method to various image datasets in \Sec{ResultsSec} and discuss our conclusions in \Sec{ConclusionsSec}

\section{Method}
\label{LatentSec}

Our algorithm for mapping the Pareto frontier transforms our original data set $X$ in a series of steps which will be describe in turn below:
\beq{MappingsEq}
X\overset{w}{\mapsto} W\mapsto W_{\rm uniform}\mapsto W_{\rm binned}\mapsto W_{\rm sorted}\overset{B}{\mapsto} Z.
\eeq
As we will show, the first, second and fourth transformations retain all mutual information with the label $Y$, and the information loss about $Y$ can be kept arbitrarily small in the third step. In contrast, the last step treats the information loss as a tuneable parameter that parameterizes the Pareto frontier.

\subsection{Lossless distillation for classification tasks}
\label{MethodSecA}

Our first step is to compress $X$ (an image, say) into $W$, a set of $n-1$ real numbers, in such a way that no class information is lost about 
$Y\in\{1,...,n\}$. 

\begin{theorem}
\label{Wtheorem}
{\bf (Lossless Distillation Theorem):} For an arbitrary random variable $X$ and a categorical random variable $Y\in\{1,...,n\}$,
we have 
\beq{Theorem1Eq}
P(Y|X) = P(Y|W),
\eeq
where $W\equiv w(X)\in \mathbb{R}^{n-1}$
is defined by\footnote{Note that we ignore the $n^{\rm th}$ component since it is redundant:
$w_n(X)=1-\sum_i^{n-1} w_i(X)$.}
\beq{wDefEq}
w_i(X) \equiv P(Y=i | X).
\eeq
\end{theorem}

\begin{proof}
Let $S$ denote the domain of $X$, \ie, $X\in S$, and define the set-valued function 
$$s(W)\equiv \{x\in S: w(x)=W\}.$$
These sets $s(W)$ form a partition of $S$ parameterized by $W$, since they are disjoint and 
\beq{UnionEq}
\cup_{W\in\mathbb{R}^{n-1}} \>s(W) = S.
\eeq
For example, if $S=\mathbb{R}^2$ and $n=2$, then the sets $s(W)$ are simply contour curves of the 
conditional probability $W\equiv P(Y=1|X)\in \mathbb{R}$.
This partition enables us to uniquely specify $X$ as  the pair $\{W,X_W\}$ by first specifying which set $s[f(X)]$ it belongs to
(determined by $W=f(X)$), and then specifying the particular element within that set, which we denote
$X_W\in S(W)$.
This implies that 
\beq{Proof1Eq}
P(Y|X) = P(Y | W,X_W) = P(Y|W),
\eeq
completing the proof. 
The last equal sign follows from the fact that the conditional probability $P(Y|X)$ is independent of $X_W$, since it is by definition constant throughout the set $s(W)$.
\end{proof}

The following corollary implies that $W$ is an optimal distillation of the information $X$ has about $Y$, in the sense that it constitutes a lossless compression of said information:
$I(W,Y)=I(X,Y)$ as shown, and the total information content (entropy) in $W$ cannot exceed that of $X$ since it is a deterministic function thereof.
\begin{corollary}
With the same notation as above, we have
\beq{Theorem2Eq}
I(X,Y) = I(W,Y).
\eeq
\end{corollary}
\begin{proof}
For any two random variables, we have the identity $I(U,V)=H(V)-H(V|U)$, where 
 $I(U,V)$ is their mutual information and $H(V|U)$ denotes conditional entropy.
We thus obtain 
\beqa{Proof2Eq}
I(X,Y)&=&H(Y)-H(Y|X) = H(Y)+\expec{\log P(Y|X)}_{X,Y}\nonumber\\
	&=&H(Y)+\expec{\log P(Y|W)}_{W,X_W,Y}\nonumber\\
	&=&H(Y)+\expec{\log P(Y|W)}_{W,Y} \nonumber\\
	&=&H(Y)-H(Y|W) =I(W,Y),
\eeqa
which completes the proof.
We obtain the second line by using $P(Y|X) = P(Y|W)$ from Theorem~1 and specifying $X$ by $W$ and $X_W$, and the third line since $P(Y|W)$ is independent of $X_W$, as above.
\end{proof}

In most situations of practical interest, the conditional probability distribution $P(Y|X)$ is not precisely known, but can be approximated by training a neural-network-based classifier that outputs the probability distribution for $Y$ given any input $X$. We present such examples in \Sec{ResultsSec}. 
The better the classifier, the smaller the information loss $I(X,Y)-I(W,Y)$ will be, approaching zero in the limit of an optimal classifier.

\subsection{Pareto-optimal compression for binary classification tasks}
\label{MethodSecB}

 \begin{figure}[ht!]
\begin{center}
\includegraphics[width=0.55\columnwidth]{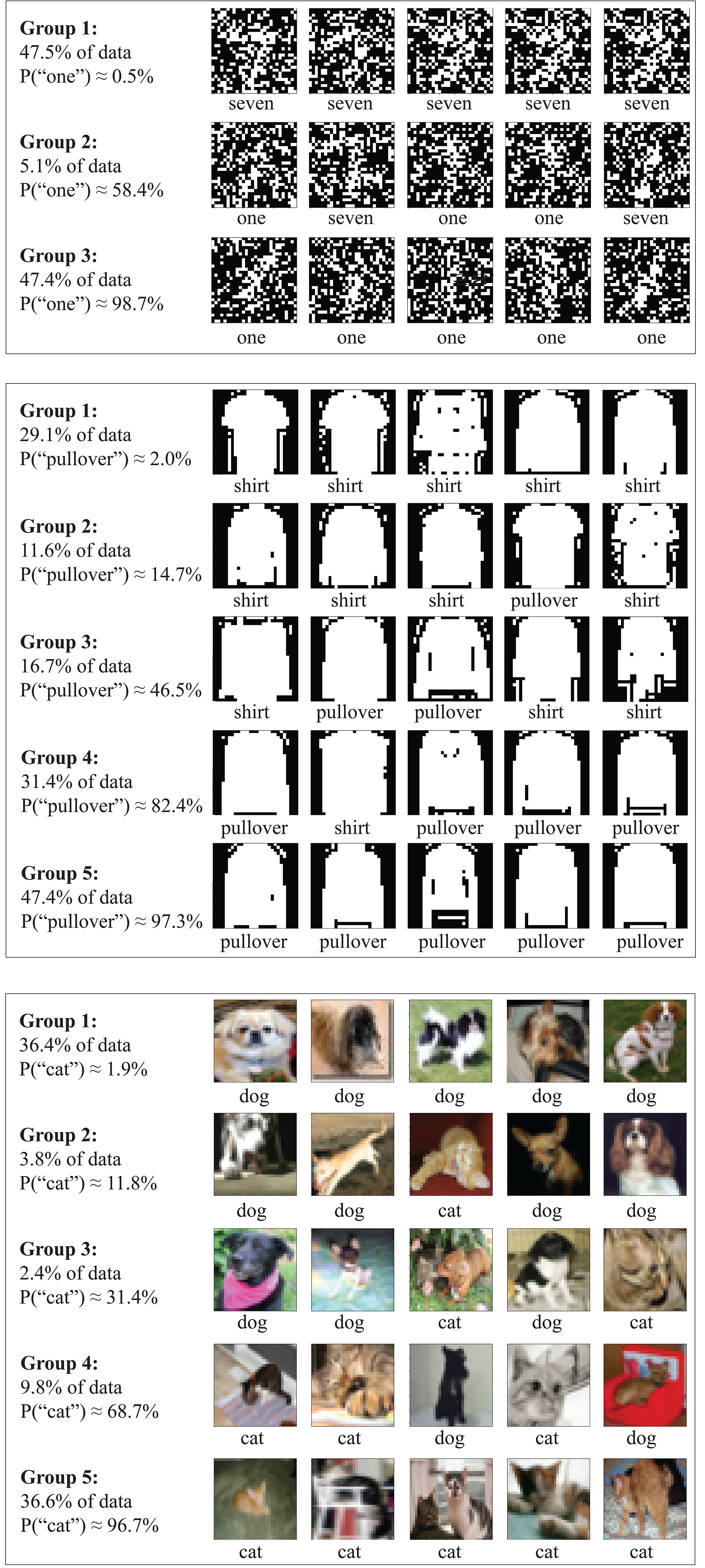}
\end{center}
\vskip-5mm
\caption{Sample data from \protect\Sec{ResultsSec}.
Images from MMNIST (top), Fashion-MNIST (middle) and CIFAR-10 are 
mapped into integers (group labels) $Z=f(X)$ retaining maximum mutual information with the class variable $Y$ (ones/sevens, shirts/pullovers and cats/dogs, respectively) for 
3, 5 and 5 groups, respectively. These mappings $f$ correspond to Pareto frontier ``corners". 
}
\label{classesFig}
\end{figure}

Let us now focus on the special case where $n=2$, \ie, binary classification tasks. 
For example, $X$ may correspond to images of equal numbers of felines and canines to be classified despite challenges with variable lighting, occlusion, {\etc} as in \fig{classesFig}, and $Y\in\{1,2\}$ may correspond to the labels ``cat" and ``dog".
In this case, $Y$ contains $H(Y)=1$ bit of information of which $I(X,Y)\le 1$ bit is contained in $X$.
Theorem~\ref{Wtheorem} shows that for this case, all of this 
information about whether an image contains a cat or a dog can be compressed into a single number $W$ which is not a bit like $Y$, but a real number between zero and one. 

The goal of this section is find a class of functions $g$ that perform Pareto-optimal lossy compression of $W$, mapping it into an 
integer $Z\equiv g(W)$ that maximizes $I(Z,Y)$ for a fixed entropy $H(Z)$.\footnote{Throughout this paper, we will use the term ``Pareto-optimal" or ``optimal" in this sense, \ie, maximizing $I(X,Y)$ for a fixed $H(Z)$.}
The only input we need for our work in this section is the joint probability distribution $f_i(w)=P(Y\hbox{=}i,W\hbox{=}w)$, whose marginal distributions are the discrete probability distribution for $P^Y_i$ for
$Y$ and the probability distribution $f$ for $W$, which we will henceforth assume to be continuous:
\beqa{WYmargEq1}
f(w)&\equiv&\sum_{i=1}^2 f_i(w),\\
P^Y_i\equiv P(Y\hbox{=}i)&=&\int_0^1 f_i(w)dw.
\eeqa

\subsubsection{Uniformization of $W$}

For convenience and without loss of generality, we will henceforth assume that $f(w)=1$, \ie, that $W$ has a uniform distribution on the unit interval $[0,1]$. We can do this because if $W$ were not uniformly distributed, we could make it so by using the standard statistical technique of applying its cumulative probability distribution function to it:
\beq{UniformizationEq}
W \mapsto W' \equiv F(W),\quad F(w)\equiv\int_0^w f(w')dw',
\eeq
retaining all information ---  $I(W',Y)=I(W,Y)$ ---
since this procedure is invertible almost everywhere.

\begin{figure}[t]
\begin{center}
\includegraphics[width=0.7\columnwidth]{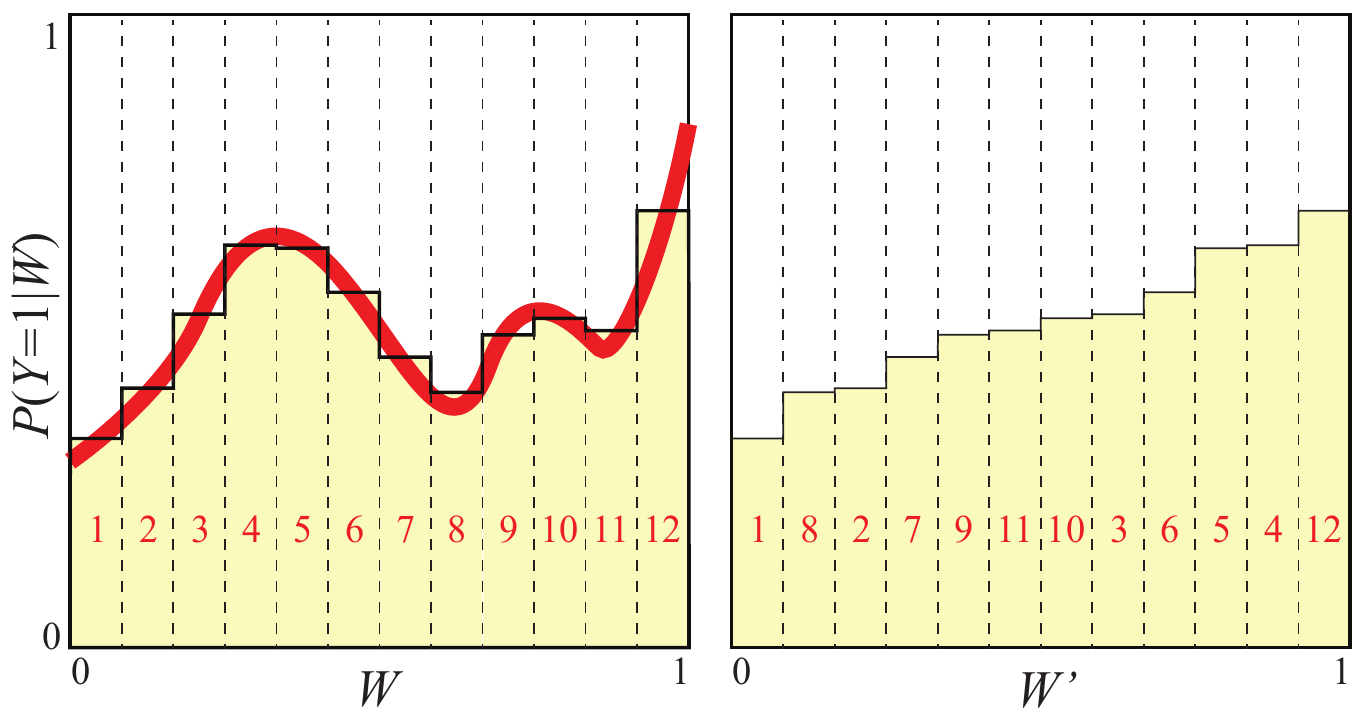}
\end{center}
\caption{Essentially all information about $Y$ is retained if $W$ is binned into sufficiently narrow bins.
Sorting the bins (left) to make the conditional probability monotonically increasing (right) changes neither this information not the entropy.
}
\label{fineBinningFig}
\end{figure}

\subsubsection{Binning $W$}
Given a set of bin boundaries $b_1<b_2<...<b_{n-1}$ grouped into a vector $\b$,
we define the integer-value {\it contiguous binning function} 
\beq{BinningFuncDef}
B(x,\b)\equiv
\left\{
\begin{tabular}{l}
$1$\quad if $x<b_1$\\
$k$\quad if $b_{k-1}<x\le b_k$\\
$n$\quad if $x\ge b_{N-1}$\\
\end{tabular}
\right.
\eeq
$B(x,\b)$ can thus be interpreted as the ID of the bin into which $x$ falls. 
Note that $B$ is a monotonically increasing piecewise constant function of $x$ that is shaped like an $N$-level staircase with $n-1$ steps at $b_1,...,b_{N-1}$. 

Let us now bin $W$ into $N$ equispaced bins, by
mapping it into an integer $W'\in\{1,...,N\}$ (the bin ID) defined by 
\beq{WprimeDefEq}
W'\equiv W_{\rm binned}\equiv B(W,\b_N),
\eeq
where $\b$ is the vector with elements $b_j=j/N$, $j=1,...,N-1$.
As illustrated visually in \fig{fineBinningFig} and mathematically in Appendix~\ref{LosslessBinningAppendix}, binning $W\mapsto W'$ corresponds
to creating a new random variable for which the conditional distribution
$p_1(w)=P(Y\hbox{=}1|W\hbox{=}w)$ is replaced by a piecewise constant function $\pbar_1(w)$, 
replacing the values in each bin by their average.
The binned variable $W'$ thus retains only information about which bin $W$ falls into, discarding all information about the precise location within that bin.
In the $N\to\infty$ limit of infinitesimal bins,  
$\pbar_1(w)\to p_1(w)$, and we expect the above-mentioned discarded information to become negligible.
This intuition is formalized by 
\ref{LosslessBinningTheorem} in Appendix~\ref{LosslessBinningAppendix}, which under mild smoothness assumptions ensuring that $p_1(w)$ is not pathological 
shows that  
\beq{binningLimitEq}
I(W',Y) \to I(W,Y)\quad\hbox{as}\quad N\to\infty,
\eeq
\ie, that we can make the binned data $W'$ retain essentially all the class information from 
$W$ as long as we use enough bins. 

In practice, such as for the numerical experiments that we will present in Section~\ref{ResultsSec}, training data is never infinite and the conditional probability function $p_1(w)$ is never known to perfect accuracy.
This means that the pedantic distinction between 
$I(W',Y)=I(W,Y)$ and $I(W',Y)\approx I(W,Y)$ for very large $N$ is completely irrelevant in practice. In the rest of this paper, we will therefore work with the
unbinned ($W$) and binned ($W'$) data somewhat interchangeably below for convenience, occasionally dropping the apostrophy $'$ from $W'$ when no confusion is caused. 

\subsubsection{Making the conditional probabilty monotonic}

For convenience and without loss of generality, we can assume that the conditional probability distribution $\pbar_1(w)$ is a monotonically increasing function.
We can do this because if this were not the case, we could make it so by sorting the bins by increasing conditional probability, as illustrated in 
\fig{fineBinningFig}, because both the entropy $H(W')$ and the mutual information $I(W',Y)$ are left invariant by this renumbering/relabeling of the bins. The ``cat" probability $P(Y\hbox{=}1)$ (the total shaded area in \fig{fineBinningFig}) is of course also left unchanged by both this sorting and by the above-mentioned binning.

\subsubsection{Proof that Pareto frontier is spanned by contiguous binnings}

We are now finally ready to tackle the core goal of this paper: mapping the Pareto frontier $(H_*,I_*)$ of optimal data compression $X\mapsto Z$ that reflects the tradeoff between $H(Z)$ and $I(Z,Y)$.
While fine-grained binning has no effect on the entropy $H(Y)$ and negligible effect on $I(W,Y)$, it dramatically reduces the entropy of our data.
Whereas $H(W)=\infty$ since $W$ is continuous\footnote{While this infinity, which reflects the infinite number of bits required to describe a single generic real number, is customarily eliminated by defining  entropy only up to an overall additive constant, we will not follow that custom here, for the reason explained in the introduction.}, $H(W')=\log N$ is finite, approaching infinity only in the limit of infinitely many infinitesimal bins.
Taken together, these scalings of $I$ and $H$ imply that the leftmost part of the Pareto frontier $I_*(H_*)$, defined by \eq{ParetoDefEq} and illustrated in \fig{paretoAnalyticFig}, asymptotes  to a horizontal line of height $I_*=I(X,Y)$ as $H_*\to\infty$.

To reach the interesting parts of the Pareto frontier further to the right, we must destroy some information
about $Y$. We do this by defining
\beq{ZdefEq}
Z = g(W'),
\eeq
where the function $g$ groups the tiny bins indexed by $W' \in\{1,...,N\}$ into fewer ones indexed by
$Z\in \{1,...,M\}$, $M<N$. 
There are vast numbers of such possible groupings, since each group corresponds to one of the
$2^N-2$ nontrivial subsets of the tiny bins. Fortunately, as we will now prove, we need only consider 
the $\mathcal{O}(N^M)$ {\it contiguous} groupings, since non-contiguous ones are inferior and cannot lie on the Pareto frontier. Indeed, we will see that for the examples in \Sec{ResultsSec}, $M\simlt 5$ 
suffices to capture the most interesting information. 

\begin{figure}[ht]
\begin{center}
\includegraphics[width=0.5\columnwidth]{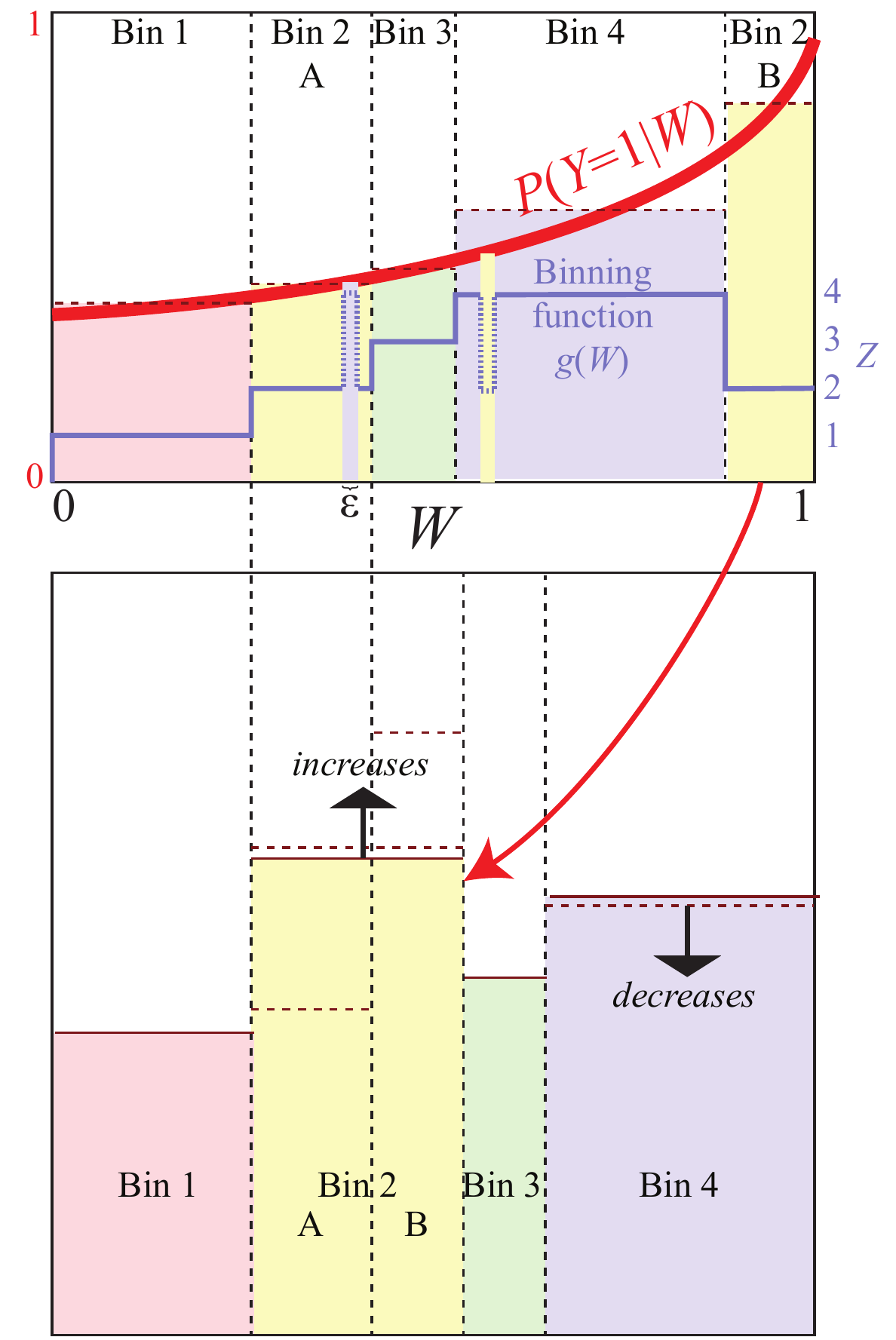}
\end{center}
\caption{The reason that the Pareto frontier can never be reached using non-contiguous bins is that a swapping 
parts of them against parts of an intermediate bin can increase $I(Z,X)$ while keeping $H(Z)$ constant. 
In this example, the binning function $g$ assigns two separate $W$-intervals (top panel) to the same bin (bin 2) as seen is the bottom panel. The shaded rectangles have widths $P_i$, heights $p_i$ and areas
$P_{i1}=P_i p_1$. In the upper panel, the conditional probabilities $p_i$ are monotonically increasing because they are averages of the monotonically increasing curve $p_1(w)$.
}
\label{binningFig}
\end{figure}

\begin{theorem}
\label{ContiguousBinningTheorem}
{\bf (Contiguous Binning Theorem):} If $W$ has a uniform distribution and the conditional probability distribution $P(W|Y\hbox{=}1)$ is monotonically increasing,
then all points $(H_*,I_*)$ on the Pareto frontier correspond to binning $W$ into contiguous intervals, \ie, if
\beq{IstarDefEq}
I(H_*) \equiv\sup_{\{g: H[g(W)]\le H_*\}} I[g(W),Y],
\eeq
then there exists a set of bin boundaries $b_1<...<b_{n-1}$
such that the binned variable $Z\equiv B(W,\b)\in\{1,...,M\}$ satisfies $H(Z)=H_*$ and $I(Z,Y)=I_*$.
\end{theorem}
\begin{proof}
We prove this by contradiction:
we will assume that there is a point 
$(H_*,I_*)$
on the Pareto frontier to which we can come arbitrarily close with
$\left(H(Z),I(Z,Y)\right)$ for $Z\equiv g(X)$ for a 
compression function $g: \mathbb{R}\mapsto \{1,...,M\}$
that is not a contiguous binning function, and obtain a contradiction by using $g$ to construct another
compression function $g'(W)$ lying above the Pareto frontier, 
with $H[g'(W)]=H_*$ and $I[g'(W),Y])>I_*$.
The joint probability distribution $P_{ij}$ for the $Z$ and $Y$ is given by the Lebesgue integral
\beq{LebesgueEq}
P_{ij}\equiv P(Z\hbox{=}i,Y\hbox{=}j) = \int f_j d\mu_i,
\eeq
where $f_j(w)$ is the  joint probability distribution for $W$ and $Y$ introduced earlier and 
$\mu_j$ is the set
$\mu\equiv\{w\in [0,1]: g(w)=i\}$, \ie, the set of $w$-values that are grouped together into the
$i^{\rm th}$ large bin.
We define the marginal and conditional probabilities
\beq{MargCondProbDef}
P_i\equiv P(Z\hbox{=}i)=P_{i1}+P_{i2}, \quad	p_i\equiv P(Y\hbox{=}1|Z\hbox{=}i)={P_{i1}\over P_i}.
\eeq
\fig{binningFig} illustrates the case where the binning function $g$ corresponds to $M=4$ large bins, the second of which consists of two non-contiguous regions that are grouped together; the shaded rectangles in the bottom panel have width $P_i$, height $p_i$ and area $P_{ij}=P_i p_i$.

According to Theorem~\ref{informationTheorem} in the Appendix, we obtain the contradiction required to complete our proof
(an alternative compression $Z'\equiv g'(W)$ above the Pareto frontier with $H(Z')=H_*$
and $I(Z',Y)>I_*$) if there are two different conditional probabilities $p_k\ne p_l$, 
and we can change $g$ into $g'$ so that the joint distribution $P'_{ij}$ of $Z'$ and $Y$ changes in the following way:
\begin{enumerate}
\item Only $P_{kj}$ and $P_{lj}$ change,
\item both marginal distributions remain the same,
\item the new conditional probabilities $p'_k$ and $p'_l$ are further apart.
\end{enumerate}
\fig{binningFig} shows how this can be accomplished for non-contiguous binning:
let $k$ be a bin with non-contiguous support set $\mu_k$ (bin 2 in the illustrated example), let $l$ be a bin whose support $\mu_l$ (bin 4 in the example) contains a positive measure subset 
$\mu_l^{\rm mid}\subset\mu_l$ within two parts $\mu_k^{\rm left}$ and $\mu_k^{\rm right}$ of $\mu_k$, 
and define a new binning function $g'(w)$ that differs from $g(w)$ only by swapping a set $\mu^\epsilon\subset \mu_l^{\rm mid}$ against a subset of either $\mu_k^{\rm left}$ or $\mu_k^{\rm right}$ of measure $\epsilon$ (in the illustrated example, the binning function change implementing this subset is shown with dotted lines).
This swap leaves the total measure of both bins (and hence the marginal distribution $P_i$) unchanged, and also leaves $P(Y\hbox{=}1)$ unchanged. 
If $p_k<p_l$, we perform this swap between $\mu_l^{\rm mid}$ an $\mu_k^{\rm right}$ (as in the figure), 
and if $p_k>p_l$, we instead perform this swap between $\mu_l^{\rm mid}$ an $\mu_k^{\rm left}$, in both cases guaranteeing that $p_l$ and $p_k$ move further apart (since $p(w)$ is monotonically increasing).
This completes our proof by contradiction except for the case where $p_k=p_l$; in this case, we swap to entirely eliminate the discontiguity, and repeat our swapping procedure between other bins until we increase the entropy (again obtaining a contradiction) or end up with a fully contiguous binning (if needed, $g(w)'$ can be changed to eliminate any measure-zero subsets that ruin contiguity, since they leave the Lebesgue integral in \eq{LebesgueEq} unchanged.)
\end{proof}

\subsection{Mapping the frontier}

Theorem~\ref{ContiguousBinningTheorem} implies that we can in practice find the Pareto frontier for any random variable $X$ by searching the space of contiguous binnings of $W=w(X)$ after uniformization, binning and sorting. 
In practice, we can first try the 2-bin case by scanning the bin boundary $0<b_1<1$, 
then trying the 3-bin case by trying bin boundaries $0<b_1<b_2<1$, then trying the 4-bin case, {\etc}, as illustrated in \fig{paretoAnalyticFig}.
Each of these cases corresponds to a standard multi-objective optimization problem aiming to maximize the two objectives $I(Z,Y)$ and $H(Z)$. We perform this optimization numerically with the AWS algorithm of \cite{kim2005adaptive} as described in the next section. 

Although the uniformization, binning and sorting procedures are helpful in practice as well as for for simplifying proofs, they are not necessary in practice.
Since what we really care about is grouping into integrals containing similar conditional probabilities $p_1(w)$, not similar $w$-values,
it is easy to see that binning horizontally after sorting is equivalent to binning vertically before sorting. 
In other words, we can eliminate the binning and sorting steps if we replace ``horizontal" binning
$g(W) = B(W,\b)$
by ``vertical" binning 
\beq{verticalBinningEq}
g(W) = B[p_1(W),\b],
\eeq
where $p_1$ denotes the conditional probability as before.

\section{Results}
\label{ResultsSec}

We will now test our algorithm for Pareto frontier mapping using some well-known datasets: 
the CIFAR-10 image database \cite{krizhevsky2014cifar}, the MNIST database of hand-written digits \cite{lecun2010mnist} and the Fashion-MNIST database of garment images \cite{xiao2017fashion}. Before doing this, however, let us build intuition for how it works by testing on a much simpler toy model that is analytically solvable, where the accuracy of all approximations can be exactly determined.

\subsection{Analytic warmup example}

Let the random variables $X=(x_1,x_2)\in[0,1]^2$ and 
$Y\in\{1,2\}$ be defined by the bivariate probability distribution
\beq{xProbDistEq}
f(X,Y) =  
\left\{
\begin{tabular}{ll}
$2 x_1 x_2$				&if $Y=1$,\\
$2(1-x_1)(1-x_2)$			&if $Y=2$,
\end{tabular}
\right.
\eeq
which corresponds to $x_1$ and $x_2$ being two independent and identically distributed random variables with triangle distribution $f(x_i)=x_i$ if $Y=1$, but flipped $x_i\mapsto 1-x_i$ if $Y=2$.
This gives $H(Y)=1$ bit and mutual information
\beq{AnalyticIeq}
I(X,Y)=1-{\pi^2-4\over 16\ln 2}\approx 0.4707\>\hbox{bits}.
\eeq
The compressed random variable $W=w(X)\in\mathbb{R}$ defined by \eq{wDefEq} is thus
\beq{analyticgEq}
W=
P(Y\hbox{=}1|X) 
= {x_1 x_2\over x_1 x_2+(1-x_1)(1-x_2)}.
\eeq
After defining $Z\equiv B(W,\b)$ for a vector $\b$ of bin boundaries, 
a straightforward calculation shows that the joint probability distribution of $Y$ and the binned variable $Z$ is given by
\beq{AnalyticPeq}
P_{ij}\equiv P(Z\hbox{=}i,Y\hbox{=}j) = 
F_j(b_{i+1})-F_j(b_i),
\eeq
where the cumulative distribution function
$F_j(w)\equiv P(W\hbox{$<$}w,Y\hbox{=}j)$
is given by
\beqa{analyticFwEq}
F_1(w)&=&{w^2\left[(2 w - 1) (5 - 4 w) + 
     2 (1 - w^2) \log(w^{-1} - 1)\right]\over 2 (2 w-1)^4},\nonumber\\
     F_2(w)&=&{1\over 2} - F_1(1-w).
\eeqa   
Computing $I(W,Y)$ using this probability distribution recovers exactly the same mutual information $I\approx 0.4707\>$bits as in \eq{AnalyticIeq}, as we proved in Theorem~\ref{Wtheorem}.

\subsection{The Pareto frontier}

Given any binning vector $\b$, we can plot a corresponding point 
$(H[Z],I[Z,Y])$ in \fig{paretoAnalyticFig} by computing 
$I(Z,Y)=H(Z)+H(Y)-H(Z,Y)$,\\
$H(Z,Y)=-\sum P_{ij}\log P_{ij}$, {\etc},
where $P_{ij}$ is given by \eq{AnalyticPeq}.

The figure shows 6,000 random binnings each for $M=3,...,8$ bins; as we have proven, the upper envelope of points corresponding to all possible (contiguos) binnings defines the Pareto frontier.
The Pareto frontier begins with the black dot at $(0,0)$ (the lower right corner), since
$M=1$ bin obviously destroys all information. The $M=2$ bin case corresponds to a 1-dimensional closed curve parametrized by the single parameter $b_1$ that specifies the boundary between the two bins: 
it runs from $(0,0)$ when $b_1=1$, moves to the left until $H(Z)=1$ when $b_1=0.5$, and returns to $(0,0)$ when $b_1=1$.  The $b_1<0.5$ and $b_1>0.5$ branches are indistinguishable in \fig{paretoAnalyticFig} because of the symmetry of our warmup problem, but in generic cases, a closed loop can be seen where only the upper part defines the Pareto frontier.

More generally, we see that the set of all binnings into $M>2$ bins maps the vector $\b$ of $M-1$ bin boundaries into a contiguous region in \fig{paretoAnalyticFig}. The inferior white region region below  can also be reached if we use non-contiguous binnings.

The Pareto Frontier is seen to resemble the  top of a circus tent, with convex segments separated by ``corners" where the derivative vanishes, corresponding to a change in the number of bins. We can understand the origin of these corners by considering what happens when adding a new bin of infinitesimal size $\epsilon$. 
As long as $p_i(w)$ is continuous, this changes all probabilites $P_{ij}$ by amounts 
$\delta P_{ij}=\mathcal{O}(\epsilon)$, and the probabilities corresponding to the new bin
(which used to vanish) will now be ${O}(\epsilon)$.
The function $\epsilon\log\epsilon$ has infinite derivative at $\epsilon=0$, blowing up as 
$\mathcal{O}(\log\epsilon)$, which implies that 
the entropy increase $\delta H(Z)=\mathcal{O}(-\log\epsilon)$.
In contrast, a straightforward calculation shows that all $\log\epsilon$-terms cancel when computing the mutual information, which changes only by 
$\delta I(Z,Y)=\mathcal{O}(\epsilon)$. As we birth a new bin and move leftward from one of the black dots in \fig{paretoAnalyticFig}, the initial slope of the Pareto frontier is thus
\beq{ParetoSlopeEq}
\lim_{\epsilon\to 0}\>{\delta I(Z,Y)\over \delta H(Z)}=0.
\eeq
In other words, the Pareto frontier starts out {\it horizontally} to the left of each of its corners in \fig{paretoAnalyticFig}. Indeed, the corners are ``soft" in the sense that the derivative of the Pareto Frontier is continuous and vanishes at the corners: for a given number of bins, $I(X,Z)$ by definition takes its global maximum at the corresponding corner, so the derivative $\partial I(Z,Y)/\partial H(Z)$ vanishes also as we approach the corner from the right.\footnote{The first corner (the transition from 2 to 3 bins) can nonetheless look fairly sharp because the 2-bin curve turns around rather abruptly, and the right derivative does not vanish in the limit where a symmetry causes the upper and lower parts of the 2-bin loop to coincide.}

Our theorems imply that in the $M\to\infty$ limit of infinitely many bins, 
successive corners become gradually less pronounced (with ever smaller derivative discontinuities), because the left asymptote of the Pareto frontier simply approaches the horizontal line $I_*=I(Y,Z)$.

\subsubsection{Approximating $w(X)$}

For our toy example, we knew the conditional probability distribution $P(Y|X)$ and could therefore compute $W=w(X)=P(Y\hbox{=}1|X)$ exactly. For practical examples where this is not the case, we can instead train a neural network to implement a function $\what(X)$ that approximates 
$P(Y\hbox{=}1|X)$. For our toy example, we train a fully connected feedforward neural network to predict $Y$ from $X$ using cross-entropy loss;  it has 2 hidden layers, each with 256 neurons with ReLU activation, and a final linear layer with softmax activation, whose first neuron defines $\what(X)$.
A illustrated in \fig{contourFig}, the network prediction for the conditional probability 
$\what(X)\equiv P(Y\hbox{=}1)$ is fairly accurate, but slightly over-confident, 
tending to err on the side of predicting more extreme probabilities (further from $0.5$). The average KL-divergence between the predicted and actual conditional probability distribution $P(Y|X)$ is  about $0.004$, which causes negligible loss of information about $Y$.

\begin{figure}[t]
\begin{center}
\includegraphics[width=0.5\columnwidth]{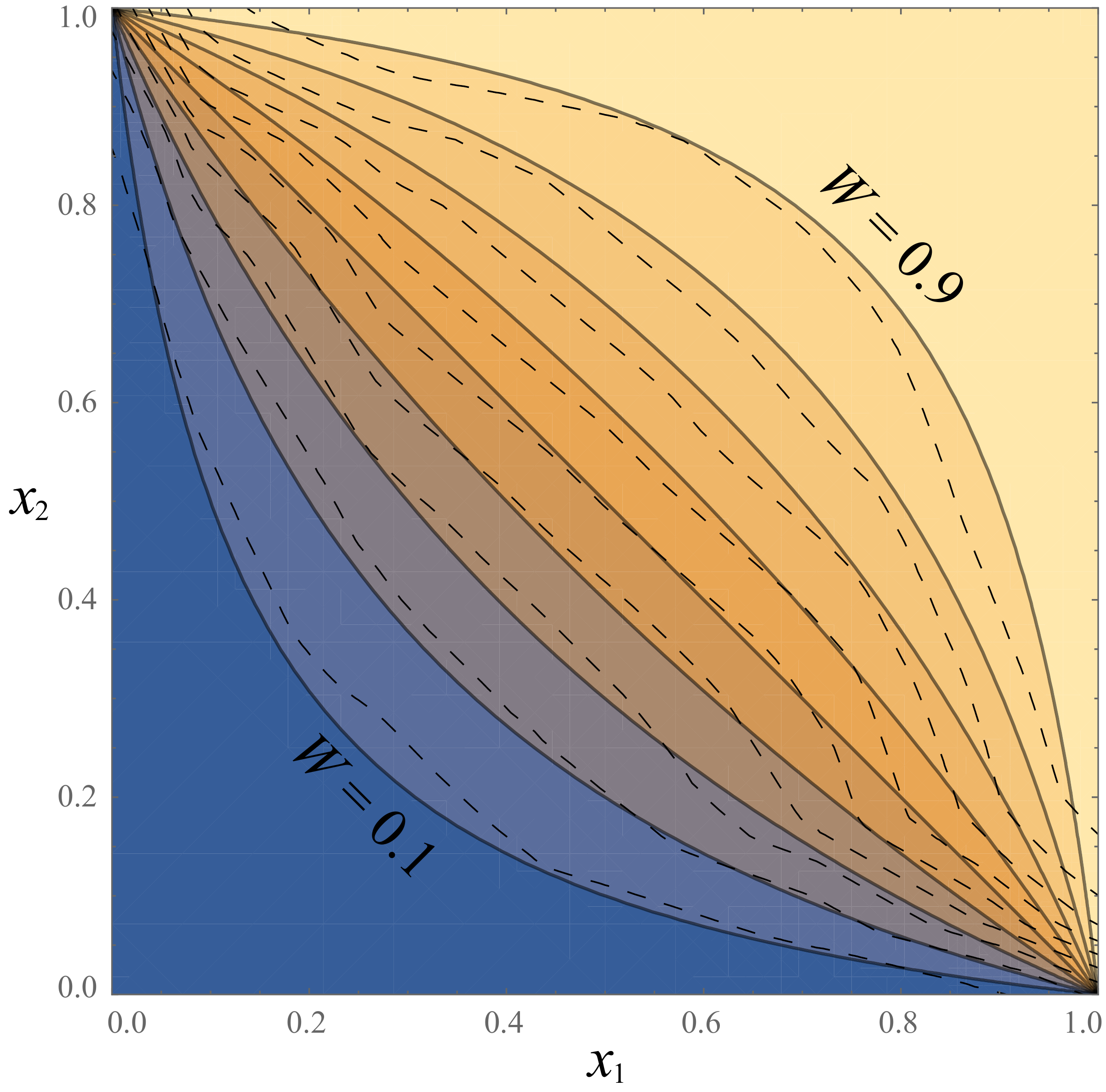}
\end{center}
\vskip-5mm
\caption{Contour plot of the function $W(x_1,x_2)$ computed both exactly using 
equation~\ref{analyticgEq} (solid curves) and approximately using a neural network (dashed curves).
}
\label{contourFig}
\end{figure}

\subsubsection{Approximating $f_1(W)$}

For practical examples where the conditional joint probability distribution $P(W,Y)$ cannot be computed analytically, we need to estimate it from the observed distribution of $W$-values output by the neural network. 
For our examples, we do this by fitting 
each probability distribution by a beta-distribution times the exponential of a polynomial of degree $d$: 
\beq{parametrizedFuncDefEq}
f(w,\a)\equiv 
\exp\left[\sum_{k=0}^d a_k x^k\right]x^{a_{d+1}}(1-x)^{a_{d+2}}, 
\eeq
where the coefficient $a_0$ is fixed by the normalization requirement $\int_0^1 f(w,\a)dw=1$.
We use this simple parametrization because it 
can fit any smooth distribution arbitrarily well for sufficiently large $d$, and 
provides accurate fits for the probability distributions in our examples using quite modest $d$; for example, $d=3$ gives 
$d_{KL}[ f_1(w),f(w,\a)]\approx 0.002$
for
\beqa{pFitLikelihoodEq}
\a&\equiv&\underset{\a'}{\rm argmin}\>\>d_{\rm KL}[f_1(w),f(w,\a')]\\
    &=&(-1.010, 2.319, -5.579, 4.887, 0.308, -0.307),\nonumber
\eeqa
which causes rather negligible loss of information about $Y$.
For our examples below where we do not  know the exact distribution $f_1(w)$ and merely have samples $W_i$ drawn from it, one for each element of the data set, we
instead perform the fitting by the standard technique of minimizing the cross entropy loss, \ie, 
\beq{pFitLikelihoodEq2}
\a\equiv \underset{\a'}{\rm argmin} \>\>-\sum_{k=1}^{n}\log f(W_k,\a').
\eeq
Table~\ref{ProbFitTable} lists the fitting coefficients used, and \fig{ProbFitsFig} illustrates the fitting accuracy.

\begin{table*}[]
{\footnotesize
\begin{tabular}{|l|l|r|r|r|r|r|r|r|}
\hline
Experiment   & Y &$a_0$& $a_1$ & $a_2$ & $a_3$  &$a_4$ &$a_5$&$a_6$\\
\hline                            
Analytic		&1	&0.0668	&-4.7685	&16.8993	&-25.0849		&13.758	&0.5797	&-0.2700\\
			&2	&0.4841	&-5.0106	&5.7863	&-1.5697		&-1.7180	&-0.3313	&-0.0030\\
\hline
Fashion-MNIST	&Pullover	&0.2878	&-12.9596 &44.9217	&-68.0105 &37.3126	&0.3547	&-0.2838\\
			&Shirt	&1.0821	&-23.8350	&81.6655	&-112.2720	&53.9602	&-0.4068	&0.4552\\
\hline
CIFAR-10		&Cat		&0.9230	&0.2165	&0.0859	&6.0013	&-1.0037	&0.8499\\
			&		&		&0.6795	&0.0511	&0.6838	&-1.0138	&0.9061\\ 
  			&Dog	&0.8970	&0.2132	&0.0806	&6.0013	&-1.0039	&0.8500\\
			&		&		&0.7872	&0.0144	&0.7974	&-0.9440	&0.7237\\			
\hline
MNIST		&One	&3.1188	&-65.224	&231.4	&-320.054	&150.779	&1.1226	&-0.6856\\
			&Seven	&-1.0325	&-47.5411	&189.895	&-269.28	&127.363	&-0.8219	&0.1284\\
\hline
\end{tabular}
\caption{Fits to the conditional probability distributions $P(W|Y)$ for our experiments, 
in terms of the parameters $a_i$ defined by \protect\eq{parametrizedFuncDefEq}.
\label{ProbFitTable}
}
}
\end{table*}

\subsection{MNIST, Fashion-MNIST and CIFAR-10}

The MNIST database consists of 28x28 pixel greyscale images of handwritten digits: 
60,000 training images and 10,000 testing images \cite{lecun2010mnist}.
We use the digits 1 and 7, since they are the two that are most frequently confused, relabeled as $Y=1$ (ones) and $Y=2$ (sevens). To increase difficulty, we inject 30\% of pixel noise, i.e., randomly flip each pixel with 30\% probability (see examples in \fig{classesFig}). For easy comparison with the other cases, we use the same number of samples for each class. 

The Fashion-MNIST database has the exact same format (60,000 + 10,000 28x28 pixel greyscale images), depicting not digits but 10 classes of clothing \cite{xiao2017fashion}. Here we again use the two most easily confused classes: pullovers ($Y=1$) and shirts ($Y=2$);  see \fig{classesFig} for examples.

The architecture of the neural network classifier we train on the above two datasets is adapted from here\footnote{We use the neural network architecture from \href{https://github.com/pytorch/examples/blob/master/mnist/main.py}{github.com/pytorch/examples/blob/master/mnist/main.py}; the only difference in architecture is that our output number of neurons is 2 rather than 10.}: two convolutional layers (kernel size 5, stride 1, ReLU activation) with 20 and 50 features, respectively, each of which is followed by max-pooling with kernel size 2. This is followed by a fully connected layer with 500 ReLU neurons and finally a softmax layer that produces the predicted probabilities for the two classes.  After training, we apply the trained model to the test set to obtain $W_i=P(Y|X_i)$ for each dataset.

CIFAR-10 \cite{cifar} is one of the most widely used datasets for machine learning research, and contains 60,000 $32\times 32$ color images in 10 different classes.
We use only the cat ($Y=1$) and dog ($Y=2$) classes, which are the two that are empirically hardest to discriminate;  see \fig{classesFig} for examples.
We use a ResNet18 architecture\footnote{The architecture is adapted from \href{https://github.com/kuangliu/pytorch-cifar}{github.com/kuangliu/pytorch-cifar}, for which we use its ResNet18 model; the only difference in architecture is that we use 2 rather than 10 output neurons.} \cite{he2016deep}. We train with a learning rate of 0.01 for the first 150 epochs, 0.001 for the next 100, and 0.0001 for the final 100 epochs; we keep all other settings the same as in the original repository.

\begin{figure}[ht]
\begin{center}
\includegraphics[width=0.6\columnwidth]{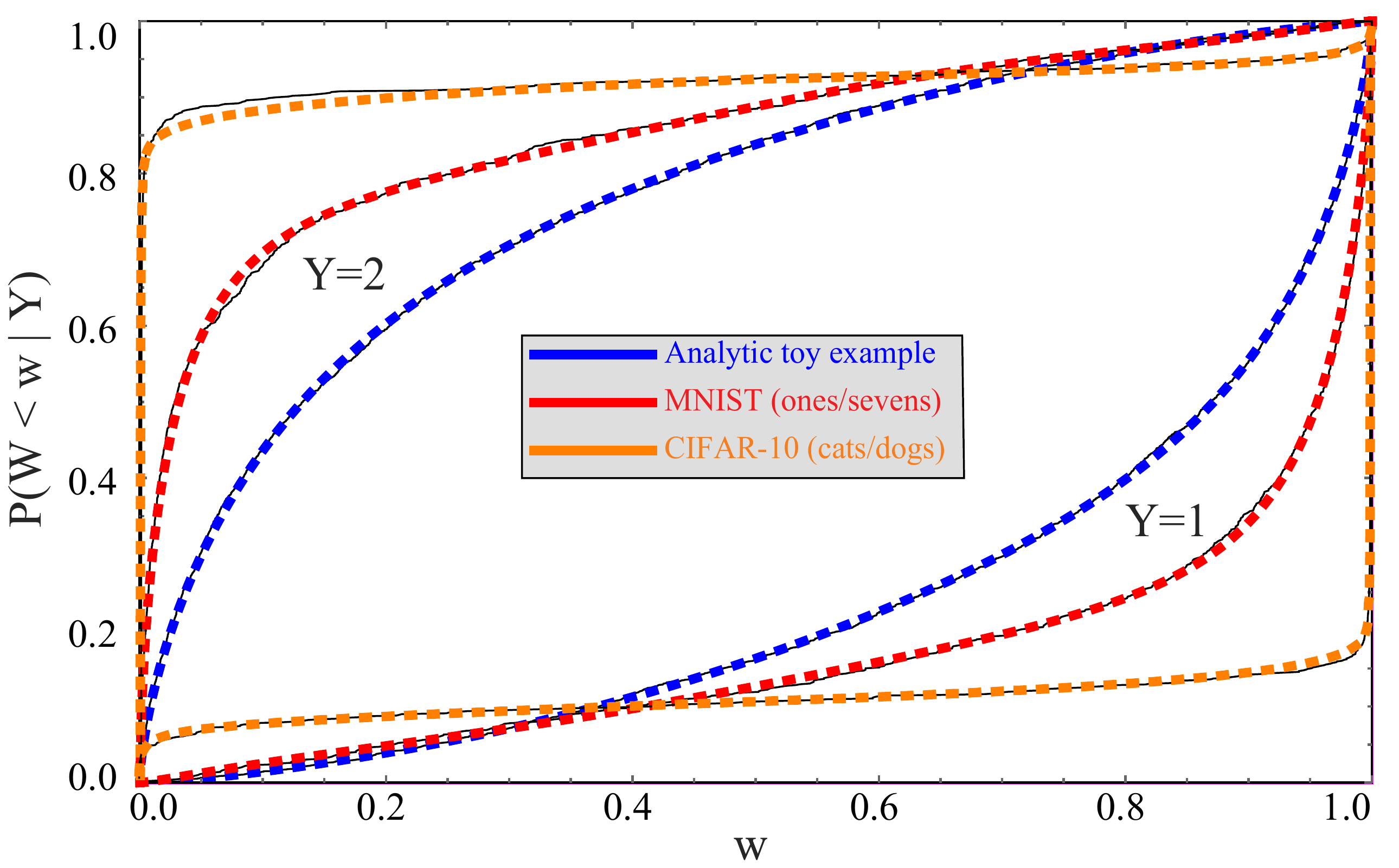}
\end{center}
\vskip-5mm
\caption{Cumulative distributions $F_i(w)\equiv P(W\hbox{$<$}w|Y\hbox{$=$}i)$ are shown for the analytic (blue/dark grey), Fashion-MNIST (red/grey) and CIFAR-10 (orange/light grey) examples. Solid curves show the observed cumulative histograms of $W$ from the neural network, and dashed curves show the fits defined by \eq{parametrizedFuncDefEq} and
Table~\ref{ProbFitTable}.
}
\label{ProbFitsFig}
\end{figure}

\fig{ProbFitsFig} shows observed cumulative distribution functions $F_i(w)$ (solid curves) for the $W_i=P(Y=1|X_i)$ generated by the neural network classifiers, together with our above-mentioned analytic fits (dashed curves).\footnote{In the case of CIFAR-10, the observed distribution $f(w)$ was so extremely peaked near the endpoints that we replaced \eq{parametrizedFuncDefEq} by the more accurate fit 
\beqa{CIFARfitEq1}
f(w)&\equiv		& F'(w),\\
F(w)&\equiv		&\left\{
\begin{tabular}{lc}
$\a^A_0 F_*[w, \a^A]$			&if $w<1/2$,\\
$1 - (1 - \a^A_0) F_*[1 - w, \a^B]]$	&otherwise,
\end{tabular}
\right.\\
F_*(x)&\equiv&G\left[{(2 x)^{a_1}\over 2}\right],\\
G(x)&\equiv&\left[\left({x\over a_2}\right)^{a_3 a_4} + (a_5 + a_6 x)^{a_4}\right]^{1/a_4},\\
a_6&\equiv&2\left[(1 - (2 a_2)^{-a_3 a_4})^{1/a_4} - a_5\right],
\eeqa
where the parameters vectors $\a^A$ and $\a^B$ are given in Table~\ref{ProbFitTable} for both cats and dogs. For the cat case, this fit gives not $f(w)$ but $f(1-w)$.
Note that $F_*(0) = 0$, $F_*(1/2)=1$.
}
\fig{CondProbFig} shows the corresponding conditional probability curves $P(Y=1|W)$ after remapping $W$ to have a uniform distribution as described above.  \fig{ProbFitsFig} shows that the original $W$-distributions are strongly peaked around $W\approx 0$ and $W\approx 1$, so this remapping stretches the $W$-axis so as to shift probability toward more central values.
  
 \begin{figure}[h!]
\begin{center}
\includegraphics[width=0.6\columnwidth]{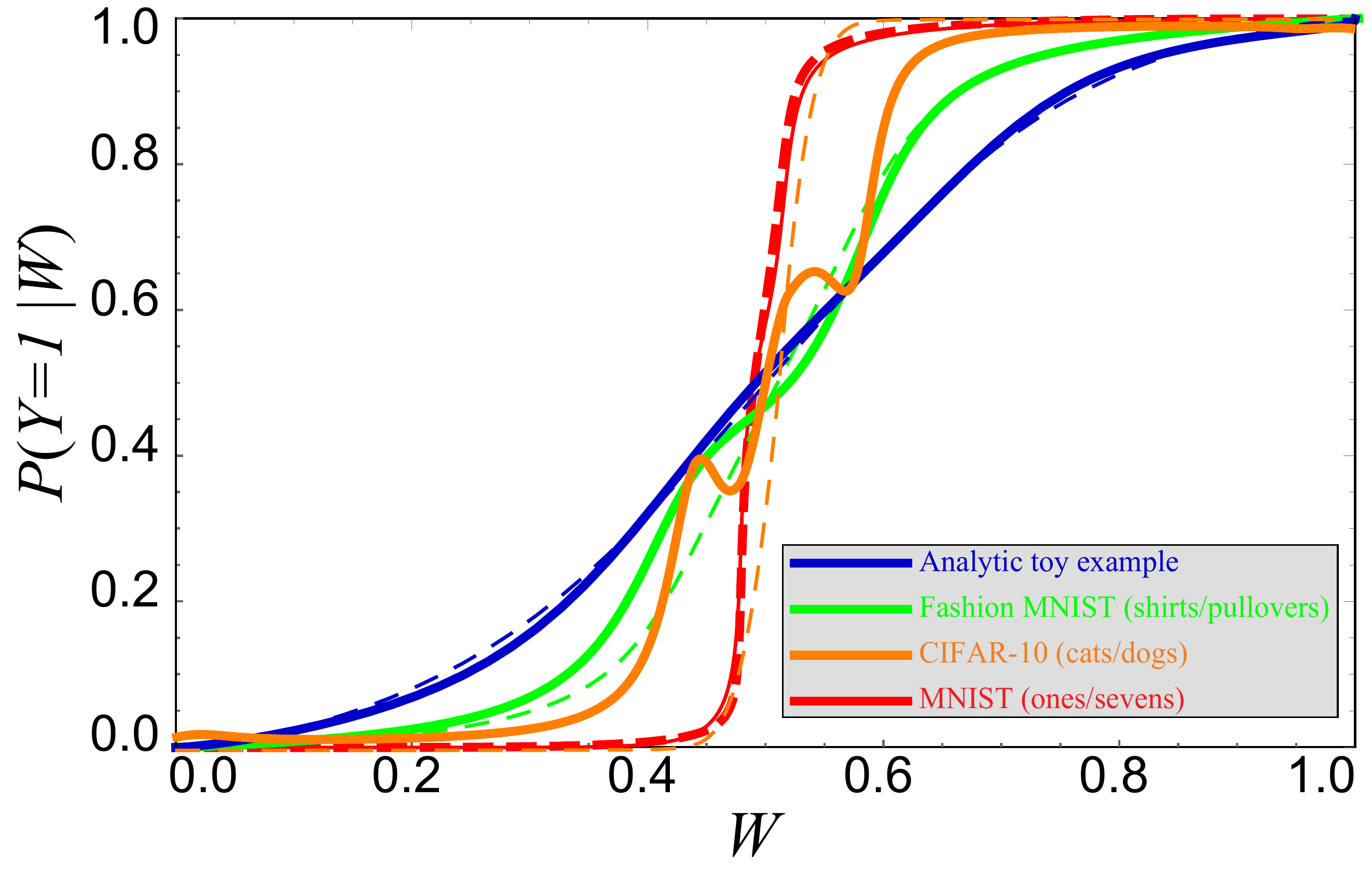}
\end{center}
\vskip-4mm
\caption{The solid curves show the actual conditional probability $P(Y\hbox{=}1|W)$ for CIFAR-10 (where the labels Y=1 and 2 correspond to ``cat" and ``dog") and MNIST with 20\% label noise (where the labels Y=1 and 2 correspond to ``1" and ``7") , respectively.
The color-matched dashed curves show the conditional probabilities predicted by the neural network; the reason that they are not diagonal lines  $P(Y\hbox{=}1|W)=W$ is that $W$ has been reparametrized to have a uniform distribution.
If the neural network classifiers were optimal, then solid and dashed curves would coincide.}
\label{CondProbFig}
\end{figure}

\begin{figure}[ht]
\begin{center}
\includegraphics[width=0.7\columnwidth]{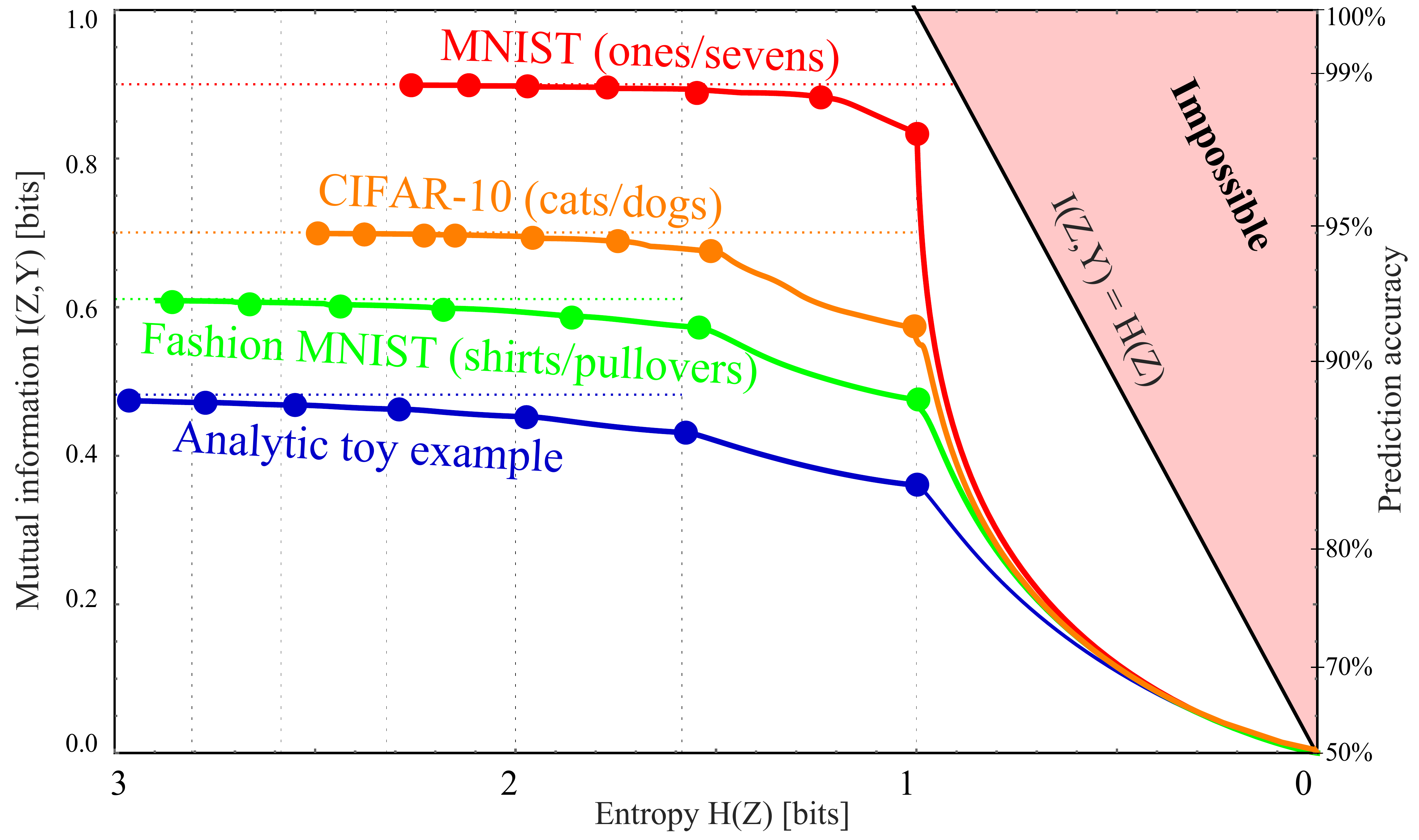}
\end{center}
\vskip-5mm
\caption{The Pareto frontier for compressed versions $Z=g(X)$ of our four datasets $X$, showing the maximum attainable class information $I(Z,Y)$ for a given entropy $H(Z)$.
The ``corners" (dots) correspond to the maximum $I(Z,Y)$ attainable when binning the likelihood $W$ into a given number of bins (2, 3, ...,8 from right to left). The horizontal dotted lines show the maximum available information $I(X,Y)$ for each case, reflecting that there is simply less to learn in some examples than in others.
}
\label{paretoFig}
\end{figure}

The final result of our calculations is shown in \Fig{paretoFig}: the Pareto frontiers for our four datasets, computed using our method.

\subsection{Interpretation of our results}

To build intuition for our results, let us consider our CIFAR-10 example of images $X$ depicting cats ($Y\hbox{=}1$) and dogs  ($Y\hbox{=}2$) as in \fig{classesFig} and ask what aspects $Z=g(X)$ of an image $X$ capture the most information about the species $Y$.
Above, we estimated that 
$I(X,Y)\approx 0.69692$ 
bits, so what $Z$ captures the largest fraction of this information for a fixed entropy?
Given a good neural network classifier, a natural guess might be the single bit $Z$ containing its best guess, say ``it's probably a cat".
This corresponds to defining $Z=1$ if $P(Y\hbox{=}1|X)>0.5$, $Z=2$ otherwise, and gives the joint distribution 
$P_{ij}\equiv P(Y\hbox{=}i,Z\hbox{=}j$)
$$
\P=\left(
\begin{tabular}{cc}
0.454555		&0.045445\\
0.042725		&0.457275
\end{tabular}
\right)
$$
corresponding to $(Z,Y)\approx 0.56971$ bits.
But our results show that we can improve things in two separate ways.

First of all, if we only want to store one bit $Z$, then we can do better, corresponding to the first ``corner" in \fig{paretoFig}:
moving the likelihood cutoff from $0.5$ to $0.51$, \ie, redefining $Z=1$ if $P(Y|X)>0.51$,
increases the mutual information to $I(Z,Y)\approx 0.56974$ bits.

More importantly, we are still falling far short of the $0.69692$ bits of information we had without data compression, capturing only 88\% of the available species information.
Our Theorem~\ref{Wtheorem} showed that we can retain {\it all} this information if we instead define $Z$ as the cat probability itself:
$Z\equiv W\equiv P(Y|X)$. For example, a given image might be compressed not into {\it ``It's probably a cat"} 
but into {\it ``I'm 94.2477796\% sure it's a cat"}.
However, it is clearly impractical to report the infinitely many decimals required to retain all the species information, which would make $H(Z)$ infinite.
Our results can be loosely speaking interpreted as the optimal way to round $Z$, so that the information  $H(Z)$ required to store it becomes finite.
We found that simply rounding to a fixed number of decimals is suboptimal; for example, if we pick 2 decimals and say
{\it ``I'm 94.25\% sure it's a cat"}, then we have effectively binned the probability $W$ into 10,000 bins of equal size, even though we can often do much better with bins of unequal size, as illustrated in the bottom panel of 
\fig{paretoAnalyticFig}. Moreover, when the probability $W$ is approximated by a neural network, we found that what should be optimally binned is not $W$ but the conditional probability $P(Y\hbox{=}1|W)$ illustrated in \fig{CondProbFig}
(``vertical binning").

It is convenient to interpret our Pareto-optimal data compression $X\mapsto Z$ as {\it clustering}, \ie, as a method of grouping our images or other data $X_i$ into clusters based on what information they contain about $Y$.
For example, \fig{classesFig} illustrates CIFAR-10 images clustered by their degree of ``cattiness"
into 5 groups $Z=1,...,5$ that might be nicknamed ``1.9\% cat", ``11.8\% cat", ``31.4\% cat", ``68.7\% cat" and ``96.7\% cat".  
This gives the joint distribution $P_{ij}\equiv P(Y\hbox{=}i,Z\hbox{=}j$) where
$$
\P=\left(
\begin{tabular}{ccccc}
0.350685	&0.053337	&0.054679	&0.034542	&0.006756\\
0.007794	 &0.006618	&0.032516	&0.069236	&0.383836
\end{tabular}
\right)
$$
and gives $I(Z,Y)\approx 0.6882$, thus increasing the fraction of species information retained from 
82\% to 99\%.

This is a striking result: we can group the images into merely five groups and discard all information about all images except which group they are in, yet retain 99\% of the information we cared about.
Such grouping may be helpful in many contexts. For example, given a large sample of labeled medical images of potential tumors, they can be used to define say five optimal clusters, after which future images can be classified into five degrees of cancer risk that collectively retain virtually all the malignancy information in the original images.

Given that the Pareto Frontier is continuous and corresponds to an infinite family of possible clusterings, which one is most useful in practice? Just as in more general multi-objective optimization problems, the most interesting points on the frontier are arguably its ``corners", indicated by dots in \fig{paretoFig}, where we do notably well on both criteria. 
This point was also made in the important paper \cite{strouse2019information} in the context of the DIB-frontier discussed below.
We see that the parts of the frontier between corners tend to be convex and thus rather unappealing, since any weighted average of $-H(Z)$ and $I(Z,Y)$ will be maximized at a corner.
Our results show that these corners can conveniently be computed without numerically tedious multiobjective optimization, by simply maximizing the mutual information $I(Z,Y)$ for $m=2, 3, 4, ...$ bins. The first corner, at $H(Z)=1$bit, corresponds to the learnability phase transition for DIB, \ie, the largest $\beta$ for which DIB is able to learn a non-trivial representation. In contrast to the IB learnability phase transition \citep{wu2019learnabilityEntropy} where $I(Z,Y)$ increases continuously from 0, here the $I(Y;Z)$ has a jump from 0 to a positive value, due to the non-concave nature of the Pareto frontier.

Moreover, all the examples in \fig{paretoFig} are seen to get quite close to the $m\to\infty$ asymptote 
$I(Z,Y)\to I(X,Y)$ for $m\simgt 5$, so the most interesting points on the Pareto frontier are simply the first handful of corners. For these examples, we also see that the greater the mutual information is, the fewer bins are needed to capture most of it. 

An alternative way if interpreting the Pareto plane in \fig{paretoFig} 
is as a traveoff between two evils: 
\beqa{InfoBloatLossEq}
\hbox{\bf Information bloat:~} H(Z|Y)&\equiv&H(Z)-I(Z,Y)\ge 0,\nonumber\\
\hbox{\bf Information loss:~~~~~~~~$\>$} \Delta I&\equiv&I(X,Y)-I(Z,Y)\ge 0.\nonumber
\eeqa
What we are calling the {\it information bloat} has also been called ``causal waste" \cite{thompson2018causal}. It is simply the conditional entropy of $Z$ given $Y$, and represents the excess bits we need to store in order to retain the desired information about $Y$. 
Geometrically, it is the horizontal distance to the impossible region to the right in \fig{paretoFig}, and we see that for MNIST, it takes local minima at the corners for both 1 and 2 bins. The {\it information loss} is simply the information discarded by our lossy compression of $X$. Geometrically, it is the vertical distance to the impossible region at the top of \fig{paretoAnalyticFig} (and, in \fig{paretoFig}, it is the vertical distance to the corresponding dotted horizontal line). As we move from corner to corner adding more bins, we typically reduce the information loss at the cost of increased information bloat. 
For the examples in \fig{paretoFig}, we see that going beyond a handful of bins essentially just adds bloat without significantly reducing the information loss.

\subsection{Real-world issues}

We just discussed how lossy compression is a tradeoff between information bloat and information loss. Let us now elaborate on the latter, for the real-world situation where 
$W\equiv P(Y\hbox{=}1|X)$ is approximated by a neural network. 

If the neural network learns to become perfect, then the function $w$ that it implements will be such that $W\equiv w(X)$ satisfies 
$P(Y=1|W)=W$, which corresponds to the dashed curves in \fig{CondProbFig} being identical to the solid curves. Although we see that this is close to being the case for the analytic and MNIST examples, the neural networks are further from optimal for Fashion-MNIST and CIFAR-10.
The figure illustrates that the general trend is for these neural networks to overfit and therefore be overconfident, predicting probabilities that are too extreme.

This fact that $P(Y\hbox{=}1|W)\ne W$ probably indicates that our Fashion-MNIST and CIFAR-10 classifiers $W=w(X)$ destroy information about $X$, but it does not prove this, because if we had a perfect lossless classifier $W\equiv w(X)$ satisfying $P(Y\hbox{=}1|W)=W$, then we could define an overconfident lossless classifier by an invertible (and hence information-preserving) reparameterization such as $W'\equiv W^2$ that violates the condition $P(Y\hbox{=}1|W')=W'$.

So how much information does $X$ contain about $Y$? One way to lower-bound $I(X;Y)$ uses the classification accuracy:
if we have a classification problem where $P(Y\hbox{=}1)=1/2$ and compress $X$ into a single classification bit $Z$ (corresponding to a binning of $W$ into two bins), then we can write the 
joint probability distribution for $Y$ and the guessed class $Z$ as
$$
P=\left(
\begin{tabular}{cc}
${1\over 2}-\epsilon_1$&$\epsilon_1$\\
$\epsilon_2$	&${1\over 2}-\epsilon_2$
\end{tabular}
\right).
$$
For a fixed total error rate $\epsilon\equiv\epsilon_1+\epsilon_2$,
Fano's Inequality implies that the  mutual information takes a minimum
\beq{PredictionAccuracyEq}
I(Z,Y)=1 + \epsilon\log\epsilon + (1 - \epsilon)\log(1 - \epsilon)
\eeq
when $\epsilon_1=\epsilon_2=\epsilon/2$, 
so if we can train a classifier that gives an error rate $\epsilon$, 
then the right-hand-side of \eq{PredictionAccuracyEq} places a lower bound on the mutual information $I(X,Y)$.
The prediction accuracy $1-\epsilon$ is shown for reference on the right side of \fig{paretoFig}. Note that getting close to one bit of mutual information requires extremely high accuracy; for example, 99\% prediction accuracy corresponds to only 0.92 bits of mutual information.
 
We can obtain a stronger estimated lower bound on $I(X,Y)$ from the cross-entropy loss function $\Ell$ used to train our classifiers:
\beq{LossEq}
\expec{\Ell} = -\left<\log P(Y\hbox{=}Y_i|X\hbox{=}X_i)\right> = H(Y|X) + \dKL,
\eeq
where $\dKL\ge 0$ denotes the average KL-divergence between true and predicted conditional probability distributions, and $\expec{\cdot}$ denotes ensemble averaging over data points, which implies that
\beqa{IboundEq}
I(X,Y)&=&H(Y) - H(Y|X) = H(Y) - \expec{\Ell} -\dKL\nonumber\\
	&\ge& H(Y) - \expec{\Ell}.
\eeqa
If $P(Y\hbox{=}1|W)\ne W$ as we discussed above, then $\dKL$ and hence the loss can be further reduced be recalibrating $W$ as we have done, which increases the information bound from 
\eq{IboundEq} up to the the value computed directly from the observed joint distribution $P(W,Y)$.

Unfortunately, without knowing the true probability $p(Y|X)$, there is no rigorous and practically useful {\it upper} bound on the mutual information other than the trivial inequality $I(X,Y)<H(Y)=1$ bit, as the following simple counterexample shows: suppose our images $X$ are encrypted with some encryption algorithm that is extremely time-consuming to crack, rendering the images for all practical purposes indistinguishable from random noise. Then any reasonable neural network will produce a useless classifier giving $I(W,Y)\approx 0$ even though the true mutual information $I(X,Y)$ could be as large as one bit. In other words, we generally cannot know the true information loss caused by compressing $X\mapsto W$, so the best we can do in practice is to pick a corner reasonably close to the upper asymptote in \fig{paretoFig}.

\subsection{Performance compared with Blahut-Arimoto method}

\begin{figure}[t]
\begin{center}
\includegraphics[width=0.5\columnwidth]{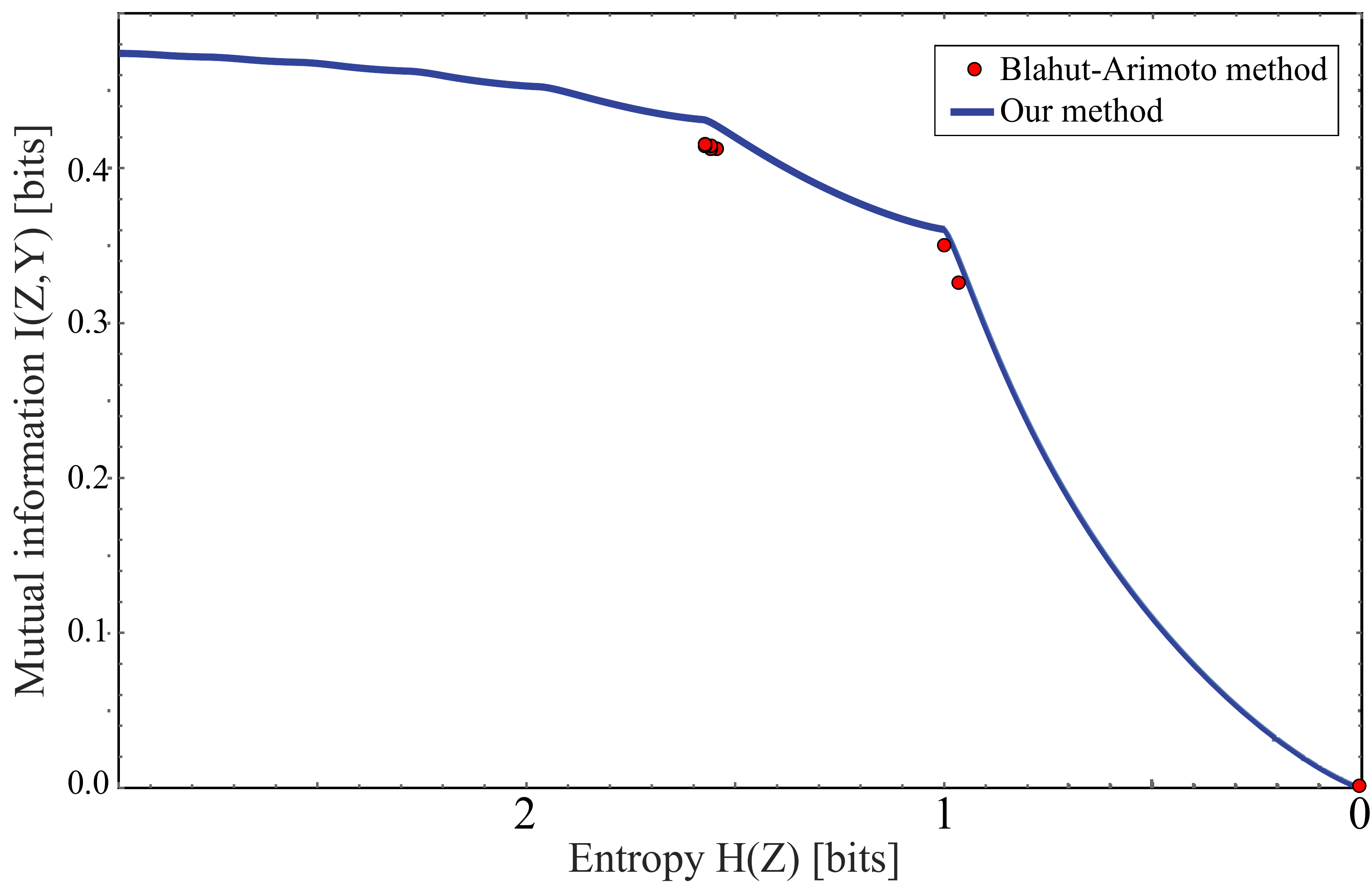}
\end{center}
\vskip-5mm
\caption{The Pareto frontier our analytic example is computed exactly with our method (solid curve) and approximately with the Blahut-Arimoto method (dots).}
\label{BlahutArimotoFig}
\end{figure}

The most commonly used technique to date for finding the Pareto frontier is the Blahut-Arimoto (BA) method 
\cite{blahut1972computation,arimoto1972algorithm}
applied to the DIB objective of \eq{IBeq2} as described in  \cite{strouse2017deterministic}. 
\fig{BlahutArimotoFig} shows the BA method implemented as in \cite{strouse2019information}, applied to our above-mentioned analytic toy example, after binning using 2,000 equispaced $W$-bins and $Z\in{1,...,8}$, scanning the $\beta$-parameter
from \eq{IBeq2} from $10^{-10}$ to $1$ in 20,000 logarithmically equispaced steps. Our method is seen to improve on 
the BA method in two ways. First, our method finds the entire continuous frontier, whereas the BA method finds only discrete disconnected points. This is because the BA-method tries to maximize the the DIB-objective from \eq{IBeq2} and thus cannot discover points where the Pareto frontier is convex as discussed above.
Second, our method finds the exact frontier, whereas the BA-method finds only approximations, which are seen to all lie below the true frontier.

\section{Conclusions}
\label{ConclusionsSec}

We have presented a method for mapping out the Pareto frontier for classification tasks (as in \fig{paretoFig}), reflecting the tradeoff between retained entropy and class information. 
We first showed how a random variable $X$ (an image, say) drawn from a class $Y\in\{1,...,n\}$ can be distilled into a vector $W=f(X)\in \mathbb{R}^{n-1}$ losslessly, so that $I(W,Y)=I(X,Y)$.
For the $n=2$ case of binary classification, we then showed how the Pareto frontier is swept out by a one-parameter family of binnings of $W$ into a discrete variable $Z=g_\beta(W)\in\{1,...,m_\beta\}$ that corresponds to binning $W$ into $m_\beta=2,3,...$ bins, such that $I(Z,Y)$ is maximized for each fixed entropy $H(Z)$.
Our method efficiently finds the exact Pareto frontier, significantly outperforming the 
Blahut-Arimoto (BA) method \cite{blahut1972computation,arimoto1972algorithm}.

\subsection{Relation to Information Bottleneck}

As mentioned in the introduction, the Discrete Information Bottleneck (DIB) method \cite{strouse2017deterministic}
maximizes a linear combination $I(Z,Y)-\beta H(Z)$ of the two axes in \fig{paretoFig}.
We have presented a method solving a generalization of the DIB problem.
The generalization lies in 
switching the objective  from \eq{IBeq2} to \eq{ParetoDefEq}, which has the advantage of discovering the full Pareto frontier in \fig{paretoFig} instead of merely the corners and concave parts (as mentioned, the DIB objective cannot discover convex parts of the frontier).
The solution lies in our proof that the frontier is spanned by binnings of the likelihood into $2, 3, 4, ...$ bins, which enables it to be computed more efficiently than with the
iterative/variational method of \cite{strouse2017deterministic}.

The popular original Information Bottleneck (IB) method \cite{tishby2000information} generalizes DIB by allowing the compression function $g(X)$ to be non-deterministic, thus adding noise that is independent of $X$. Starting with a Pareto-optimal $Z\equiv g(X)$ and adding such noise will simply shift us straight to the left in \fig{paretoFig}, away from the frontier (which is by definition monotonically decreasing) and into the Pareto-suboptimal region in the $I(Y;Z)$ vs. $H(Z)$ plane. As shown in \cite{strouse2017deterministic}, IB-compressions tend to altogether avoid the rightmost part of \fig{paretoFig}, with an entropy $H(Z)$ that never drops below some fixed value independent of $\beta$.

 \subsection{Relation to phase transitions in DIB learning}

Recent work has revealed interesting phase transitions that occur during information bottleneck learning \cite{chechik2005information,strouse2017deterministic,wu2019learnabilityEntropy},
as well as phase transitions in other objectives, e.g. $\beta$-VAE \cite{rezende2018taming}, infoDropout \cite{achille2018emergence}.
Specifically, when the $\beta$-parameter that controls the tradeoff between information retention and model simplicity is continuously adjusted, the resulting point in the IB-plane can sometimes ``get stuck" or make discontinuous jumps.
For the DIB case,  our results provide an intuitive understanding of these phase transitions in terms of the geometry of the Pareto frontier.

Let us consider \fig{paretoAnalyticFig} as an example. 
The DIB maximiziation of $I(Z,Y)-\beta H(Z)$ geometrically corresponds to finding a tangent line of the Pareto frontier of slope $-\beta$. 

If the Pareto frontier $I_*(H)$ were everywhere continuous and concave, so that $I_*''(H)<0$,
then its slope would range from some steepest value $-\beta_*$ 
at the right endpoint $H=0$ and continuously flatten out as we move leftward, asymptotically approaching 
zero slope as $H\to\infty$.
The learnability phase transition studied in \cite{wu2019learnabilityEntropy} would then occur when $\beta=\beta_*$:
for and $\beta\ge\beta_*$, the DIB method learns nothing, \eg, discovers as optimal the point 
$(H,I)=(0,0)$ where $Z$ retains no information whatsoever about $Y$.
As $\beta\le\beta_*$ is continuously reduced, the DIB-discovered point would then continuously move up and to the left along the Pareto frontier.

This was for the case of an everywhere concave frontier, but Figures~\ref{paretoAnalyticFig} and~\ref{paretoFig} show that actual Pareto frontiers need {\it not} be concave --- indeed, none of the frontiers that we have computed are concave. Instead, they are seen to consist of long convex segments joint together by short concave pieces near the ``corners".  
This means that as $\beta$ is continuously increased, the DIB solution exhibits first-order phase transitions, making discontinuous jumps from corner to corner at certain critical $\beta$-values; these phase transitions correspond to increasing the number of clusters into which the data $X$ is grouped.

\subsection{Outlook}

Our results suggest a number of opportunities for further work, ranging from information theory to machine learning, neuroscience and physics.

As to information theory, it will be interesting to try to generalize our method from binary classification into classification into more than two classes. Also, one can ask if there is a way of pushing 
the general information distillation problem all the way to bits.
It is easy to show that a discrete random variable $Z\in\{1,...,m\}$ can always be encoded as 
$m-1$ independent random bits (Bernoulli variables) $B_1,...,B_{m-1}\in\{0,1\}$, defined by\footnote{The mapping $z$ from bit strings $\B$ to integers $Z\equiv z(\B)$
is defined so that $z(\B)$ is the position of the last bit that equals one when $\B$ is preceded by a one. For example, for $m=4$, the mapping from length-3 bit strings 
$\B\in\{0,1\}^3$ to integers $Z\in\{1,...,4\}$
is 
$z(001)=z(011)=z(101)=z(111)=4$, 
$z(010)=z(110)=3$,
$z(100)=2$,
$z(000)=1$.
}
\beq{BitProbEq}
P(B_k\hbox{=}1)=P(Z\hbox{=}k+1)/P(Z\le k+1),
\eeq
although this generically requires some information bloat.
So in the spirit of the introduction, is there some useful way of generalizing 
PCA, autoencoders, CCA and/or the method we have presented so that the quantities
$Z_i$ and $Z'_i$ in Table~\ref{ComparisonTable} are not real numbers but bits?

As to neural networks, it is interesting to explore novel classifier architectures that reduce the overfitting and resulting overconfidence revealed by \fig{CondProbFig}, as this might significantly increase the amount of information we can distill into our compressed data. It is important not to complacently declare victory just because classification accuracy is high; as mentioned, even 99\% binary classification accuracy can waste 8\% of the information.

As to neuroscience, our discovery of optimal ``corner" binnings begs the question of whether
evolution may have implemented such categorization in brains.
For example, if some binary variable $Y$ that can be inferred from visual imagery is evolutionarily important for a given species (say, whether potential food items are edible), might our method help predict how many distinct colors $m$ their brains have evolved to classify hues into? In this example, $X$ might be a triplet of real numbers corresponding to light intensity recorded by three types of retinal photoreceptors, and the integer $Z$ might end up corresponding so some definitions of yellow, orange, {\etc}
A similar question can be asked for other cases where brains define finite numbers of categories, for example categories defined by distinct words. 

As to physics, it has been known even since the introduction of Maxwell's Demon that a physical system can use information about its environment to extract work from it. If we view an evolved life form as an intelligent agent seeking to perform such work extraction, then it faces a tradeoff between retaining too little relevant infomation (consequently extrating less work) and retaining too much (wasting energy on information processing and storage). 
Susanne Still recently proved the remarkable physics result \cite{still2017thermodynamic} that the lossy data compression optimizing such work extraction efficiency is precisely that prescribed by the above-mentioned Information Bottleneck method \cite{tishby2000information}. As she puts it, an intelligent data representation strategy emerges from the optimization of a fundamental physical limit to information processing.
This derivation made minimal and reasonable seeming assumptions about the physical system, but did not include an energy cost for information encoding. We conjecture that this can be done such that an extra Shannon coding term proportional to $H(Z)$ gets added to the loss function, which means that when this term dominates, the generalized Still criterion would instead prefer the Deterministic Information Bottleneck or one of our Pareto-optimal data compressions. 

Although noise-adding IB-style data compression may turn out to be commonplace in many biological settings, it is striking that the types of data compression we typically associate with human perception intelligence appears more deterministic, in the spirit of DIB and our work. For example, when we compress visual input into ``this is a probably a cat", we do not typically add noise by deliberately flipping our memory to ``this is probably a dog". Similarly, 
the popular jpeg image compression algorithm dramatically reduces image sizes while retaining essentially all information that we humans find relevant, and does so deterministically, without adding noise.

It is striking that simple information-theoretical principles such as IB, DIB and Pareto-optimality appear relevant
across the spectrum of known intelligence, ranging from extremely simple 
physical systems as in Still's work all the way up to high-level human perception and cognition.
This motivates further work on the exciting quest for a deeper understanding of Pareto-optimal data compression 
and its relation to neuroscience and physics.

\chapter{Learning causal relations from observations}
\label{chap6:causal}

Identifying\footnote{The paper \href{http://roseyu.com/time-series-workshop/submissions/2019/timeseries-ICML19_paper_43.pdf}{``Nonlinear Causal Discovery with Minimum Predictive Information Regularization"} was presented at ICML 2019 Time Series workshop \cite{wu2020discovering}, and awarded the \emph{Best Poster Award}. Authors: Tailin Wu, Thomas Breuel, Michael Skuhersky and Jan Kautzin.}\footnote{The code for the methods and experiments is open-sourced at \href{https://github.com/tailintalent/causal}{github.com/tailintalent/causal}.} the underlying directional/causal relations from observational time series with nonlinear interactions and complex relational structures is key to a wide range of applications, yet remains a hard problem. In this chapter, we introduce a novel minimum predictive information regularization method to infer directional relations from time series, allowing deep learning models to discover nonlinear relations. Our method substantially outperforms other methods for learning nonlinear relations in synthetic datasets, and discovers the directional relations in a video game environment and a heart-rate vs. breath-rate dataset.

\section{Introduction and Related Work}

Imagine a dataset with tens or hundreds of observational time series. There may exist interesting directional relations between the time series which we want to uncover, but their relation graph may be complicated, and the relation may be nonlinear as we do not know its functional form. How can we discover the underlying relations of those challenging scenarios in an efficient way, or at least identify candidate relations that are worth further investigation by a researcher? Problems of this type are omnipresent and important in a variety of scientific endeavors and applications, e.g., gene regulatory networks \cite{lozano2009grouped}, neuroscience \cite{neves2008synaptic,seth2015granger}, economics \cite{granger1969investigating,stock1989interpreting} and finance \cite{hiemstra1994testing, granger2000bivariate}.

To address this question, the field of causal learning has proposed a large class of methods to discover or quantify causal relations. These methods have certain limitations in regards to capability of handling nonlinearity, and/or scalability and efficiency to large numbers of time series.
Pearl \cite{pearl2002causality,pearl2009causality,pearl2009causal} defines causality in terms of intervention and structural dependence, under the structural equation models (SEM). However, in our problem, where only observational time series is available, Pearl's definition may not be applicable. In his seminal work, Granger \cite{granger1969investigating,granger1980testing} defines causality via prediction: if the prediction of the future Y via a linear model can be improved by including the information of X, then X causes Y in the Granger sense. The original Granger causality is limited to linear causal models. Although later works also extend Granger causality to kernel methods \cite{ancona2004radial,marinazzo2008kernel, marinazzo2008kernel2, sindhwani2012scalable}, they may still be insufficient to model and discover the nonlinear causal relations in real data. On the other hand, the causal measures of transfer entropy \cite{schreiber2000measuring} and causal influence \cite{janzing2013quantifying} are in theory able to handle any nonlinearity. However, both measures require density estimation of the joint distribution for the full $N$ time series ($N$ is the number of time series), which is difficult and data-hungry when $N$ is large. Constraint-based methods \cite{spirtes2000causation, harris2013pc, pearl2002causality, spirtes2000causation} require repetitive conditional independence tests, where the number of tests will grow large when $N$ is large and the underlying causal graph is dense. Score-based methods search for the structure that yields the optimal score w.r.t. the data, generally using greedy search methods, for example GES \cite{chickering2002optimal}, rankGES \cite{nandy2018high} and GIES \cite{hauser2012characterization}. This in general requires $\Theta(N^2)$ steps, and the number of neighboring states may grow very large at each step. Another closely related field is sparse learning/feature selection methods. Some important classes are Lasso \cite{tibshirani1996regression} and elastic net \cite{zou2005regularization}, which are effective but subject to the limitations of linear models. For nonlinear models, although L1 and group L1 regularization \cite{meier2008group,scardapane2017group,tank2018neural} can induce sparsity in the model parameters, they are model and input dependent. 

To handle the nonlinear relations in time series, a promising tool is neural nets. Not only are neural nets universal function approximators \cite{hornik1991approximation}, a deep neural net also provides exponentially large expressive power \cite{rolnick2018the}, making it particularly suitable for modeling the unknown nonlinear relations in time series. Recently there has been an increasing amount of work on learning the dynamic models of interacting systems \cite{battaglia2016interaction,chang2016compositional,guttenberg2016permutation,watters2017visual,hoshen2017vain,van2018relational}. However, their main focus is to make better predictions, using implicit interaction models (e.g. using fully connected graph networks). In this paper, we are mainly interested in discovering the underlying directional relations in an explicit form, utilizing the expressive power of neural nets.

To discover nonlinear directional relations from potentially large number of time series in an efficient way, the contribution of our work is as follows:
\begin{itemize}[topsep=0pt,partopsep=0pt,leftmargin=*]

\item We introduce a novel relational learning with Minimum Predictive Information Regularization (MPIR) method for exploratory discovery of nonlinear directional relations from observational time series. It is based on minimizing a mutual information-regularized risk with learnable input noise of a prediction model, which allows function approximators such as neural nets to learn nonlinear relations, combining the benefits of the Granger causality paradigm with deep learning models. At the minimization of the objective, the minimum predictive information term quantifies the directional predictive strength between each pair of time series given other time series. For discovering the directional relations among $N$ time series, it only has to learn $N$ models, and does not requires density estimation for the joint $N$ time series.

\item We prove that the minimum predictive information is able to differentiate dependence or independence between pairs of time series, which allows for statistical test. Moreover, we prove that the minimum predictive information is invariant to the scaling of input and reparameterization of the model. We further provide intuition that under certain conditions, our method is likely to discover direct relations instead of indirect associations.

\item We demonstrate on nonlinear synthetic datasets that our method outperforms other methods in discovering true causal relations with larger $N$, and discovers the directional relations in video game environment and real-world heart-rate vs. breath-rate datasets.

\end{itemize}

\section{Method}

\subsection{Problem setup}
\label{sec:problem_definition}

We consider $N$ time series $x^{(1)}, x^{(2)}, ...x^{(N)}$, where each time series $x^{(i)}=(x^{(i)}_1,x^{(i)}_2,...x^{(i)}_t,...)$ and each $x^{(i)}_t\in \R^M$ is an $M$-dimensional vector. Denote $X_{t-1}^{(i)}=(x^{(i)}_{t-K},x^{(i)}_{t-K+1},...x^{(i)}_{t-1})$ with maximum time horizon of $K$, and $\mathbf{X}_{t-1}=\{X_{t-1}^{(i)}\}, i=1,2,...N$. We also denote $\mathbf{X}_{t-1}^{(\hat{j})}=\mathbf{X}_{t-1}\texttt{\textbackslash} X^{(j)}_{t-1}$ ($\mathbf{X}_{t-1}$ excluding $X^{(j)}_{t-1}$).
We assume that $x^{(1)}, x^{(2)}, ...x^{(N)}$ are generated by stationary response functions $h_i$ that are unknown to the learner:
\begin{equation}
\label{eq:response_function}
    \begin{cases}
      x^{(1)}_t:=h_1(\mathbf{X}_{t-1},u_1) \\
      x^{(2)}_t:=h_2(\mathbf{X}_{t-1},u_2) \\
     ...\\
      x^{(N)}_t:=h_N(\mathbf{X}_{t-1},u_N) 
    \end{cases}
  \end{equation}
for $t=K+1,K+2,...$ . 
Here $u_i\in \R^M, i=1,2,...N$ are noise variables that are mutually independent, are independent of any $X^{(i)}_{t-1}, x^{(i)}_t$, $i\in\{1,2,...N\}$.
For any $i,j\in\{1,2,...N\}$, we assume that the variables $(\mathbf{X}_{t-1}^{(\hat{j})}, X_{t-1}^{(j)}, x_{t}^{(i)})$ have probability density function $P(\mathbf{X}_{t-1}^{(\hat{j})}, X_{t-1}^{(j)}, x_{t}^{(i)})$.

Our method is inspired by Granger causality \cite{granger1969investigating,granger1980testing}, which defines causality via predictions, making it especially suitable for relational inference of observational time series. Adapting to our notation:

\textbf{Granger causality} \cite{granger1980testing}:\textit{
Assuming causal sufficiency \cite{peters2017elements}, we say $X^{(j)}_{t-1}, j\neq i$ does not Granger-cause $x^{(i)}_{t}$, if $P(x^{(i)}_{t}|X^{(j)}_{t-1}, \mathbf{X}_{t-1}^{(\hat{j})})=P(x^{(i)}_{t}|\mathbf{X}_{t-1}^{(\hat{j})})$. Otherwise, we say $X^{(j)}_{t-1}$ Granger-causes $x^{(i)}_{t}$.
}

In practice, we say that time series $j$ Granger-causes time series $i$, if it can be shown via significance tests that the null hypothesis of $P(x^{(i)}_{t}|X^{(j)}_{t-1}, \mathbf{X}_{t-1}^{(\hat{j})})=P(x^{(i)}_{t}|\mathbf{X}_{t-1}^{(\hat{j})})$ is rejected, i.e. $X_{t-1}^{(j)}$ provides statistically significant information for predicting $x_t^{(i)}$.

In his original work, Granger \cite{granger1969investigating} investigates causality with linear function predictors. Later works have extended it to kernel methods \cite{ancona2004radial, marinazzo2008kernel, marinazzo2008kernel2, sindhwani2012scalable}, which essentially estimate linear Granger causality on the feature space of the kernel. To learn potentially highly nonlinear response functions, it may be desirable to use expressive and universal function approximators \cite{hornik1991approximation} such as neural nets. Neural nets are much more flexible than linear models, and do not require kernel selection as in kernel methods.

\subsection{Our method}
\label{sec:our_method}
Based on the definition of Granger causality, a naïve way to combine it with neural net is: for each $j\to i$, train two neural nets, one predicting $x_t^{(i)}$ based on $\X_{t-1}^{(\hat{j})}$, another predicting $x_t^{(i)}$ based on the full $\X_{t-1}=(\X_{t-1}^{(\hat{j})}, X_{t-1}^{(j)})$, and test whether former MSE is significantly larger than the latter. This method suffers from two major drawbacks: (1) instability: different training of the neural net may end up in different local minima, so that the two MSEs have large variance, which is observed in our initial explorations; (2) inefficiency: to discover the relations among $N$ time series, it has to train at least $N^2$ models (for each $x^{(i)}_t$, train $N-1$ models with one $X_{t-1}^{(j)}$ removed, and $Q$ models ($Q\ge1$) with full $\X_{t-1}$ for accumulating statistics). On the other hand, these two drawbacks exactly inspire our method. Instead of predicting $x^{(i)}_t$ with one $X_{t-1}^{(j)}$ missing at a time, what if we let each $X_{t-1}^{(j)}$ have \textit{learnable} corruption, and encourage each $X_{t-1}^{(j)}$ to provide as little information to $x^{(i)}_t$ as possible while maintaining good prediction? In this way, we have a \emph{single} shared model that can span the full product space of $[\text{total corruption}, \text{no corruption}]^{\bigotimes N}$ for $N$ input time series, which is more stable and efficient than the removing one $X_{t-1}^{(j)}$ at a time and training $N$ models. To achieve this, we add independent noise with learnable amplitudes to each input $X_{t-1}^{(j)}$, and measure the corruption by the mutual information between the input and the corrupted input. We then define the following risk:

\begin{equation}
\begin{aligned}
\label{eq:learnable_risk}
R_{\mathbf{X},x^{(i)}}[f_\theta,&\boldsymbol{\eta}]=\mathbb{E}_{\mathbf{X}_{t-1},x_t^{(i)},\boldsymbol{\epsilon}}\left[\left(x_t^{(i)}-f_\theta(\tilde{\mathbf{X}}^{(\boldsymbol{\eta})}_{t-1})\right)^2\right]+\lambda\cdot \sum_{j=1}^{N}I(\tilde{X}^{(j)(\eta_j)}_{t-1};X^{(j)}_{t-1})
\end{aligned}
\end{equation}
where $\tilde{\mathbf{X}}^{(\boldsymbol{\eta})}_{t-1}:=\mathbf{X}_{t-1}+\boldsymbol{\eta}\odot\boldsymbol{\epsilon}$ (or element-wise, $\tilde{X}_{t-1}^{(j)(\eta_j)}:=X_{t-1}^{(j)}+\eta_j\cdot \epsilon_j$, $j=1,2,...N$) is the noise-corrupted inputs with \emph{learnable} noise amplitudes $\eta_j\in \R^{KM}$, and $\epsilon_j\sim N(\mathbf{0},\mathbf{I})$. $\lambda >0$ is a positive hyperparameter for the mutual information $I(\cdot, \cdot)$. 
Intuitively, the minimization of the second term $I(\tilde{X}^{(j)(\eta_j)}_{t-1};X^{(j)}_{t-1})$ requires the noise amplitude $\eta_j$ to go up. The minimization of the first term requires the noise amplitude $\eta_j$ to go down, and the larger causal strength from $X_{t-1}^{(j)}$ to $x_t^{(i)}$, the larger this force. The minimization of the two terms strikes a balance, at which point the $I(\tilde{X}^{(j)(\eta_j)}_{t-1};X^{(j)}_{t-1})$ measures the \emph{minimum} number of bits of information the time series $j$ need to provide to the learner, without compromising the prediction. 
\vskip -0.05in
At the minimization of $R_{\mathbf{X},x^{(i)}}[f_\theta,\boldsymbol{\eta}]$, we define $W_{ji}=I\left(\tilde{X}^{(j)(\eta_j^*)}_{t-1};X^{(j)}_{t-1}\right)$, which we term \emph{minimum predictive information}, where $(f_{\theta^*},\boldsymbol{\eta}^*)=\text{argmin}_{(f_{\theta},\boldsymbol{\eta})}R_{\mathbf{X},x^{(i)}}[f_\theta,\boldsymbol{\eta}]$.
Essentially, $W_{ji}$ measures the \emph{predictive strength} of time series $j$ for predicting time series $i$, \textit{conditioned} on all the other observed time series. We have that
$W_{ji}$ satisfies the following properties: 

\begin{enumerate}[label={(\arabic*)}]
\item If $x^{(j)}\independent x^{(i)}$, then $W_{ji}=0$.
\item $W_{ji}$ is invariant to affine transformation of each individual $X_{t-1}^{(k)},k=1,2,...N$.
\item $W_{ji}$ is invariant to reparameterization of $\theta$ in $f_\theta$ (the mapping remains the same).
\end{enumerate}

The proofs are provided in Appendix \ref{app:W_proof}. Property 1 shows that $W_{ji}$ is able to differentiate time series that are dependent or independent with the target time series $i$. Empirically, to perform statistical tests, we can let the null hypothesis be  $x^{(j)}\independent x^{(i)}$. Before training, we append to $\X_{t-1}$ some fake time series $v^{(s)}_{t-1},s=1,2,...S$ (e.g. by randomly permuting $X_{t-1}^{(j)}$) so that $v^{(s)}_{t-1}\independent x^{(i)}_t$. After optimizing w.r.t. to the augmented dataset, the values of $W_{si}$ between $v^{(s)}_{t-1}$ and $x_t^{(i)}$ form a distribution for which we know that the null hypothesis is true. Then if certain $W_{ji}$ is greater than the $1-\alpha$ quantile (e.g. $\alpha=0.05$) of the distribution, we can reject the null hypothesis of independence. Properties 2 and 3 show the benefit of our method which essentially regularizes the \textit{input information}, compared with L1 and group L1 \cite{meier2008group,scardapane2017group,tank2018neural} which regularize the \textit{model} and thus do not satisfy these two properties.

Moreover, in Appendix \ref{app:W_proof} we further provide intuition that under certain conditions, $W_{ji}$ is likely to favor the time series  that \textit{directly} causes time series $i$, compared with the time series that relate to $i$ via the direct causal connections. 
Note that our method is not \emph{guaranteed} to identify \emph{direct causal} relations (in Granger \cite{granger1980testing} or Pearl \cite{pearl2002causality} sense), which is a very hard problem given the potential large number of time series and nonlinearity present. However, our method provides an effective data exploratory tool to identify time series that are \emph{predictive} of one another, \emph{conditioned} on all the other observed time series, whose identified directional relations can be investigated further by a researcher. As stated above, under certain conditions, our method does favor the direct causal relations. And in the experiment section, we will compare the estimated $W_{ji}$ with true causal relations if available.

\begin{algorithm}[t]
   \caption{\textbf{Relational Learning with Minimum Predictive Information Regularization}}
\label{alg:learnable_noise}
\begin{algorithmic}
   \STATE {\bfseries Require} $x^{(i)}_{t}, \mathbf{X}_{t-1}$, for $i\in\{1,2,...N\},t\in \mathbf{T}=\{K+1,K+2,...\}$.
   \STATE {\bfseries Require $\eta_0$}: a small value for initialization of $\boldsymbol{\eta}$.
   \STATE {\bfseries Require $\lambda$}: coefficient for the mutual information term.
   \STATE {\bfseries Require $S$}: number of fake time series.
   \STATE {\bfseries Require $\alpha$}: significance level.
\STATE 1:\ \ \  Randomly select $S$ indices $i_1,i_2,...i_S$ from $\{1,2,...N\}$
\STATE 2:\ \ \  $v_{t-1}^{(s)}\gets \text{Permute-examples}_t(X_{t-1}^{(i_s)})$ \textbf{for} $s=1,2,...S$    \ \ \ \ // \textit{Permuting on the example dimension}
\STATE 3:\ \ \  $\X^{(\text{aug})}_{t-1}\gets[\X_{t-1}, \mathbf{v}_{t-1}]$, where $\mathbf{v}_{t-1}=[v_{t-1}^{(1)},...v_{t-1}^{(S)}]$ and $[\cdot,...,\cdot]$ denotes  
\STATE \ \ \ \ \ \ \ \ \ concatenation along the dimension of $N$ (thus $\X^{(\text{aug})}_{t-1}$ consists of $N+S$ time series)
\STATE 4:\ \ \ \textbf{for} $i$ in $\{1,2,...N\}\ \textbf{do:}$
\STATE 5:\ \ \ \ \ \ \ \ \ \ Initialize function approximator $f_\theta$.
\STATE 6:\ \ \ \ \ \ \ \ \ \ Initialize $\boldsymbol{\eta}=(\eta_1,\eta_2,...\eta_N)=(\eta_0\boldsymbol{1},\eta_0\boldsymbol{1},...\eta_0\boldsymbol{1})$, where each element $\eta_0\boldsymbol{1}$ is a    
 \STATE \ \  \ \ \ \ \ \ \ \ \ \ \ \ \ \ $KM$-dimensional vector, same dimension as $X_{t-1}^{(j)}$.
\STATE 7:\ \ \ \ \ \ \ \ \ \  $(f_{\theta^*},\boldsymbol{\eta}^*)\gets\text{Minimize}_{(f_\theta,\boldsymbol{\eta})}\hat{R}_{\X^{(\text{aug})},x^{(i)},\boldsymbol{\epsilon}}[f_\theta,\boldsymbol{\eta}]$ (Eq. \ref{eq:empirical}) with e.g. gradient descent.
\STATE 8:\ \ \ \ \ \ \ \ \ \  $W_{ji}\gets I(\tilde{X}^{(j)(\eta_j^*)}_{t-1};X^{(j)}_{t-1})$, for $j=1,2,...N, N+1,...N+S$.
\STATE 9:\ \ \ \textbf{end for}
\STATE 10: Accumulate the values of $W_{si}$ between all $v_{t-1}^{(s)},s=1,2,...S$ and $x_t^{(i)},i=1,2,...N$,  
\STATE \ \ \ \ \ \  and obtain the $1-\alpha$ quantile as the threshold.
\STATE 11: Zero the $W_{ji}$ elements ($j,i=1,2,...N$) whose value are below the threshold.
\STATE 12: \textbf{return} $W$ \ \ // \textit{Return the main $N\times N$ matrix}
\end{algorithmic}
\end{algorithm}

Empirically, we minimize the following empirical risk:
\begin{equation}
\label{eq:empirical}
\begin{aligned}
\hat{R}_{\mathbf{X},x^{(i)},\boldsymbol{\epsilon}}[f_\theta,&\boldsymbol{\eta}]=\frac{1}{|\mathbf{T}|}\sum_{t\in \mathbf{T}}\left(x_t^{(i)}-f_\theta(\tilde{\mathbf{X}}^{(\boldsymbol{\eta})}_{t-1})\right)^2+\lambda \sum_{j=1}^{N}I(\tilde{X}^{(j)(\eta_j)}_{t-1};X^{(j)}_{t-1})
\end{aligned}
\end{equation}

In general, it may be inefficient to estimate the mutual information $I(\tilde{X}^{(j)(\eta_j)}_{t-1};X^{(j)}_{t-1})$ with large dimension of $X^{(j)}_{t-1}$ such that the expression is also differentiable w.r.t. $\eta_j$. Utilizing the property of Gaussian channels, in Appendix \ref{app:Gaussian_channel_upper_bound} we prove that $I(\tilde{X}^{(j)(\eta_j)}_{t-1};X^{(j)}_{t-1})\leq \frac{1}{2}\sum_{l=1}^{KM}\text{log}\left(1+\frac{\text{Var}(X_{t-1,l}^{(j)})}{\eta_{j,l}^2}\right)$, where $l$ denotes the $l^{\text{th}}$ element of a vector, and $\text{Var}(X_{t-1,l}^{(j)})$ is the variance of $X_{t-1,l}^{(j)}$ across $t$. Therefore, in practice to improve efficiency, we can optimize an \emph{upper bound} of the risk:
\begin{equation}
\label{eq:empirical_upper_bound}
\begin{aligned}
\hat{R}^{\text{upper}}_{\mathbf{X},x^{(i)},\boldsymbol{\epsilon}}[f_\theta,&\boldsymbol{\eta}]=\frac{1}{|\mathbf{T}|}\sum_{t\in \mathbf{T}}\left(x_t^{(i)}-f_\theta(\tilde{\mathbf{X}}^{(\boldsymbol{\eta})}_{t-1})\right)^2+\frac{\lambda}{2} \sum_{j=1}^{N}\sum_{l=1}^{KM}\text{log}\left(1+\frac{\text{Var}(X_{t-1,l}^{(j)})}{\eta_{j,l}^2}\right)
\end{aligned}
\end{equation}

When the dimension of $X_{t-1}^{(j)}$ is large, a differentiable estimate of the mutual information (e.g. MINE \cite{belghazi2018mine}) can be applied. We provide Algorithm~\ref{alg:learnable_noise} to empirically estimate $W_{ji}$, which we term relational learning with Minimum Predictive Information Regularization (MPIR). The steps 1-3 construct fake input time series $v_{t-1}^{(s)}, s=1,2,...S$ (which we know the null hypothesis of $v_{t-1}^{(s)}\independent{}x_t^{(i)}$ is true) to append to $\X_{t-1}$. Steps 4-9 optimize the objective w.r.t. the augmented dataset, and obtain a $(N+S)\times N$ matrix $W_{ji}$. Steps 10-11 performs significance test and only preserve the $W_{ji}$ values in the main $N\times N$ matrix that are statistically significant. Finally the main matrix is returned.

To select an appropriate hyperparameter $\lambda$, we can additionally append to the target $x_t^{(i)}$ a few time series $w_t$ constructed from $\X_{t-1}$. We then select $\lambda$ such that the estimated causal strength between $\X_{t-1}$ and $w_t$ (for which we know the causal relations) is at least $4\sigma$ away from the estimated causal strength between $v_{t-1}$ and $w_t$ (for which we know that they are independent). See Appendix \ref{app:select_lambda} for details.

\section{Experiments}
\label{sec:causal_experiments}

To demonstrate that our proposed method is able to discover interesting underlying directional (possibly causal) relations, we test it on both synthetic and real datasets. We first use synthetic datasets, where we know the underlying causal structure and compare with other methods. We then test whether our algorithm can infer directional relations among trajectories of objects from watching an agent playing video games. Finally, we apply our algorithm to a real-world heart-rate vs. breath-rate dataset and a rat EEG dataset to test its effectiveness. We use the $\hat{R}^{\text{upper}}_{\mathbf{X},x^{(i)},\boldsymbol{\epsilon}}[f_\theta,\boldsymbol{\eta}]$ (Eq. \ref{eq:empirical_upper_bound}) for optimization for all experiments.

\subsection{Synthetic experiment with log-normal causal strengths}
\label{sec:synthetic}

In this experiment, we evaluate our method together with other methods using a nonlinear synthetic dataset generated to have a known causal structure (hidden to the methods being compared). We study performance with varying number $N$ of time series, with $N$ up to $30$. To generate the data, we let each $x_t^{(i)}$ have dimension $M=1$, and also set the maximum time horizon $K=3$, so each $X_{t-1}^{(j)}$ is a $K\times M=3\times 1$ matrix. We use the following realization of the response function $h_i$ in Eq.~(\ref{eq:response_function}):
\begin{equation}
\label{eq:synthetic}
\begin{aligned}
    x_t^{(i)} = h_i(\mathbf{X}_{t-1}, u_t)=&\text{H}_1\left(\sum_{j=1}^N\left[A_{ji} \odot\text{H}_2(B_j\odot X_{t-1}^{(j)})\right]\right)+u_t,i=1,2,...N
\end{aligned}
\end{equation}
where $u_t\sim N(\mathbf{0},\mathbf{I})\in \R^M$, $\odot$ denotes element-wise multiplication, and $\text{H}_1$ and $\text{H}_2$ are two nonlinear functions to make the response functions nonlinear. In this experiment, we use $\text{H}_1(x)=\text{softplus}(x)=\text{log}(1+e^x)$, and $\text{H}_2(x)=\text{tanh}(x)$. $B_j$ is a $K\times M$ random matrix, whose element is sampled from $U[-1, 1]$. $A_{ji}$ is a $K\times M$ matrix, with 0.5 probability of being a zero matrix and 0.5 probability of being a nonzero random matrix, characterizing the underlying causal strength from $j$ to $i$. Crucially, to reflect that the causal strength may span different orders of magnitude, if $A_{ji}$ is sampled to be a nonzero matrix, then the amplitude of each of its element is sampled from a log-normal distribution with $\mu=1,\sigma=0$, their sign sampling from $U\{-1,1\}$. Denote $\mathbbm{1}(A)$ as the 0-1 indicator matrix of causality ($\mathbbm{1}(A)_{ji}=1\ \text{if}\ |A_{ji}|>0; 0\ \text{otherwise}$). The goal of each algorithm being evaluated is to produce an $N\times N$ score matrix $\tilde{A}$, where each entry $\tilde{A}_{ji}$ characterizes the directional strength from $j$ to $i$. Then the flattened $\tilde{A}$ is evaluated against the flattened $\mathbbm{1}(A)$ (excluding diagonal elements of the matrices) via different metrics. Fig. S\ref{fig:synthetic_example_figure} in Appendix \ref{app:synthetic_exp} shows example snapshots of the time series.

In general, for a large $N$, the number of possible causal graphs grows double exponentially: there are $2^{N^2}$ possible matrix of $\mathbbm{1}(A)$. To give an estimate, for $N=3,4,5,8,10,20,30$, there are $512, 6.6\times10^4, 3.3\times10^7,1.8\times10^{19}, 1.2\times10^{30}, 2.6\times 10^{120}, 8.5\times10^{270}$ number of possible graphs, respectively. Therefore, estimating the underlying causal graph is in general a non-trivial task when $N$ is large. We compare our algorithm with previous methods including transfer entropy \cite{schreiber2000measuring}, causal influence \cite{janzing2013quantifying}, linear Granger causality \cite{granger1969investigating, ding2006granger}, kernel Granger causality \cite{marinazzo2008kernel,marinazzo2008kernel2}, and three baselines: (1) mutual information $\tilde{A}_{ji}=I(X_{t-1}^{(j)};x_t^{(i)})$ (which gives $\tilde{A}_{ji}=\tilde{A}_{ij}$), (2) a sparse feature selection method, elastic net \cite{zou2005regularization}, and (3) a random matrix, each element of which is drawn from a standard Gaussian distribution. For each $N$, we sample 10 datasets with different $A_{ji}$ and $B_j$ matrices, and compare each method's average performance over 10 datasets together with their standard deviation. The implementation details for each method and each experiment are provided in Appendix \ref{app:algorithm_implementation} and \ref{app:synthetic_exp}, respectively. Since many of the methods do not provide a threshold or significance test, we use the standard metrics of area under the precision-recall curve (AUC-PR) \cite{davis2006relationship} (Table \ref{table:synthetic_larger_N_AUC_PR} below) and area under the ROC curve (AUC-ROC) (Table S\ref{table:synthetic_larger_N_AUC_ROC} in Appendix \ref{app:synthetic_ROC_AUC}) to compare their performance.

\begin{table*}[t]
\caption{Mean and standard deviation of AUC-PR (\%) vs. $N$, over 10 random sampling of datasets. Bold font marks the top method for each $N$.}
\resizebox{1\linewidth}{!}{%
\begin{tabular}{p{3cm}p{1cm}p{1cm}p{1cm}p{1cm}p{1cm}p{1cm}p{1cm}p{1cm}}
\toprule
 \ \ \ \ \ \ \ \ \ \ \ \ \ \ \ \ \ \ \ \ \ \ \ \ \ \ $N$ &     3  &    4  &    5  &    8  &    10 &    15 &    20 &    30 \\
method             &        &       &       &       &       &       &       &       \\
\midrule
\textbf{MPIR (ours)}        &   97.5{\tiny$\pm$5.3} &  98.4{\tiny$\pm$2.5} &  \textbf{97.6}{\tiny$\pm$2.7} &  \textbf{96.1}{\tiny$\pm$2.4} &  \textbf{93.5}{\tiny$\pm$3.7} &  \textbf{91.3}{\tiny$\pm$3.0} &  \textbf{85.9}{\tiny$\pm$2.4} &  \textbf{76.3}{\tiny$\pm$1.5} \\
Mutual Information &   90.5{\tiny$\pm$13.7} &  93.3{\tiny$\pm$3.8} &  90.0{\tiny$\pm$4.3} &  82.4{\tiny$\pm$5.1} &  76.9{\tiny$\pm$9.3} &  76.8{\tiny$\pm$4.8} &  71.9{\tiny$\pm$3.8} &  70.6{\tiny$\pm$3.1} \\
Transfer Entropy   &   93.5{\tiny$\pm$7.7} &  97.3{\tiny$\pm$3.3} &  91.6{\tiny$\pm$8.2} &  83.7{\tiny$\pm$7.2} &  76.2{\tiny$\pm$5.7} &  67.1{\tiny$\pm$4.2} &  61.2{\tiny$\pm$4.3} &   55.7{\tiny$\pm$2.5} \\
Linear Granger     &   \textbf{99.4}{\tiny$\pm$1.8} &  97.8{\tiny$\pm$2.5} &  92.0{\tiny$\pm$8.3} &  83.1{\tiny$\pm$8.8} &  79.4{\tiny$\pm$9.2} &  71.0{\tiny$\pm$10.0} &  63.7{\tiny$\pm$8.8} &  52.4{\tiny$\pm$1.7} \\
Kernel Granger     &   99.3{\tiny$\pm$2.3} &  \textbf{99.3}{\tiny$\pm$1.5} &  96.5{\tiny$\pm$4.8} &  92.5{\tiny$\pm$3.4} &  90.0{\tiny$\pm$3.3} &  86.0{\tiny$\pm$2.4} &  81.0{\tiny$\pm$4.0} &  73.1{\tiny$\pm$1.8} \\
Elastic Net        &   99.1{\tiny$\pm$2.9} &  98.5{\tiny$\pm$2.0} &  95.7{\tiny$\pm$4.2} &  88.9{\tiny$\pm$6.2} &  83.6{\tiny$\pm$4.6} &  79.1{\tiny$\pm$3.0} &  75.3{\tiny$\pm$3.6} &  69.1{\tiny$\pm$5.8} \\
Causal Influence   &   67.5{\tiny$\pm$26.7} &  60.2{\tiny$\pm$24.1} &  59.3{\tiny$\pm$15.3} &  44.1{\tiny$\pm$8.9} &  42.7{\tiny$\pm$7.8} &  47.0{\tiny$\pm$3.1} &  44.5{\tiny$\pm$4.1} &   44.6{\tiny$\pm$2.1} \\
Gaussian random    &   60.0{\tiny$\pm$14.7} &  57.9{\tiny$\pm$12.9} &  51.6{\tiny$\pm$8.0} &  44.5{\tiny$\pm$5.6} &  41.3{\tiny$\pm$6.2} &  44.6{\tiny$\pm$4.0} &  44.0{\tiny$\pm$2.4} &  44.3{\tiny$\pm$2.4} \\
\bottomrule
\end{tabular}
}
\label{table:synthetic_larger_N_AUC_PR}
\end{table*}

We see that for smaller $N$ ($N\le4$), methods with smaller expressivity (linear Granger, kernel Granger) performs slightly better. However, as $N$ becomes larger, our method outperforms other methods with increasing margin, demonstrating our method's capability to infer complex relational structures from interacting time series. Particularly, although two linear methods, linear Granger and elastic net, have relatively strong performance with $N\leq5$, they quickly degrade with larger $N$ due to more nonlinearity present in the data. With the help of kernels, kernel Granger degrades slower, but can not compete in larger $N$ with our method which allows expressive neural nets to model complex nonlinear interactions. For the Causal Influence method, although it has very good mathematical properties, it may be impractical in practice, as is also shown in the table. This is due to that it is defined as the KL-divergence between $(\mathbf{X}_{t-1}, x_{t-1}^{(i)})$ and its counterpart (whose causal arrows to and from time series $j$ are cut), each of which is an $(NK+1)M-$dimensional vector, which can quickly go to high dimensions, where density estimation required to calculate KL-divergence is in general data-hungry and difficult. In comparison, our method that estimates predictive strength via minimizing prediction errors is comparatively easier in high dimensions.

\subsection{Experiments with video games}
\label{sec:video_games}

\begin{figure}[ht]
    \centering
    \begin{subfigure}{.24\linewidth}
        \includegraphics[scale=0.14]{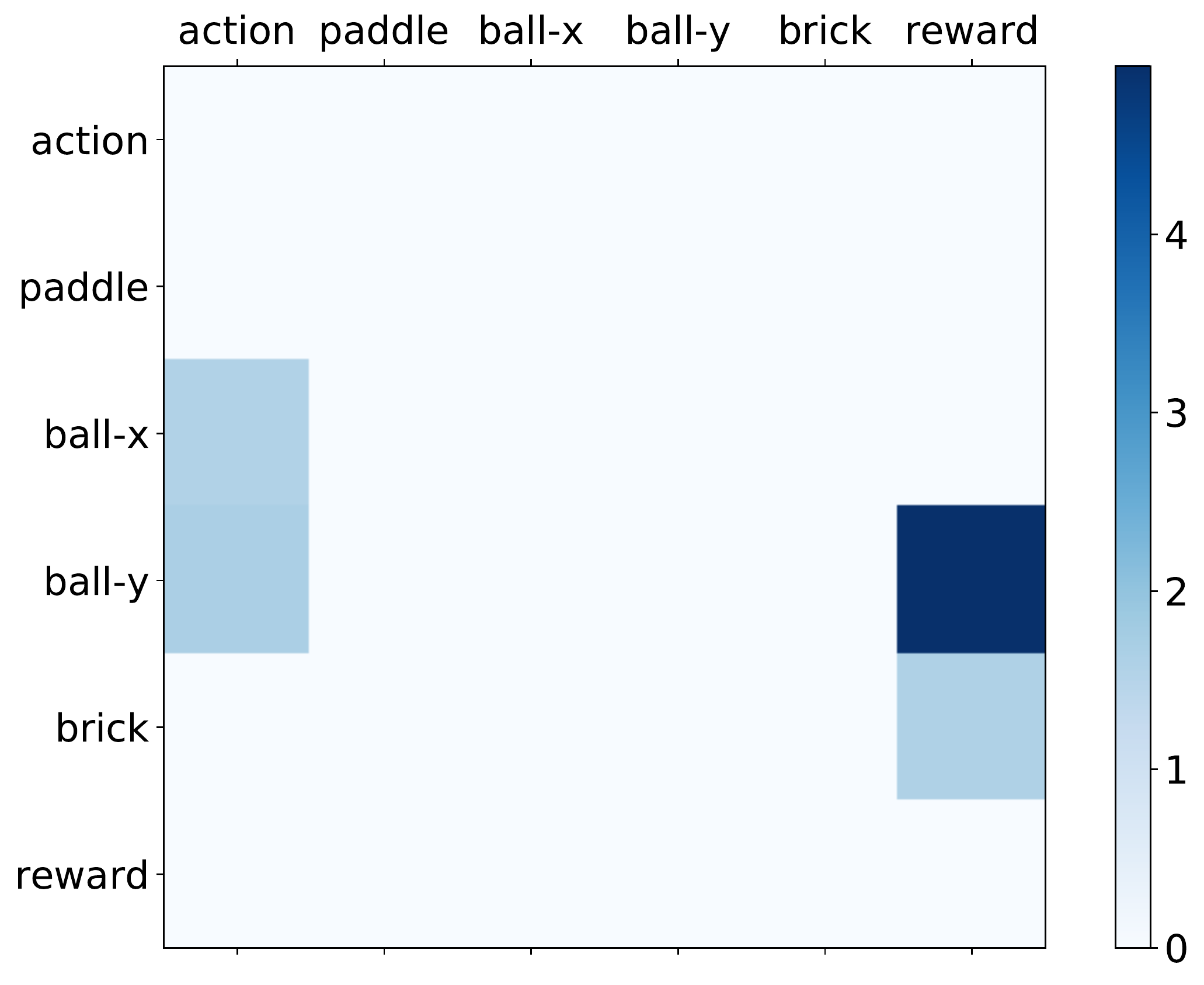}
        \caption{}
    \end{subfigure}
    \rulesep
    \begin{subfigure}{.24\linewidth}
        \includegraphics[scale=0.14]{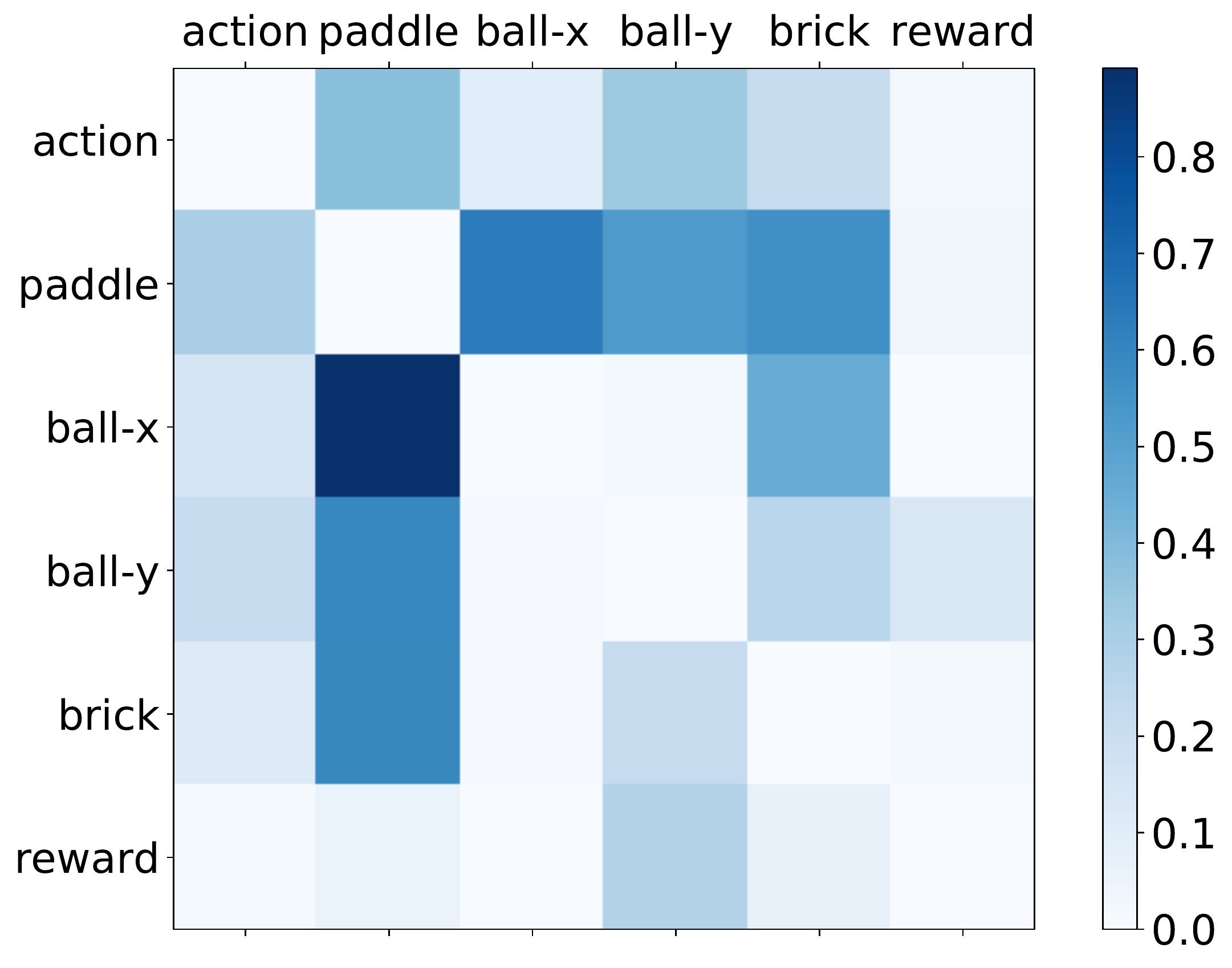}
        \caption{}
    \end{subfigure}
    \begin{subfigure}{.24\linewidth}
        \includegraphics[scale=0.14]{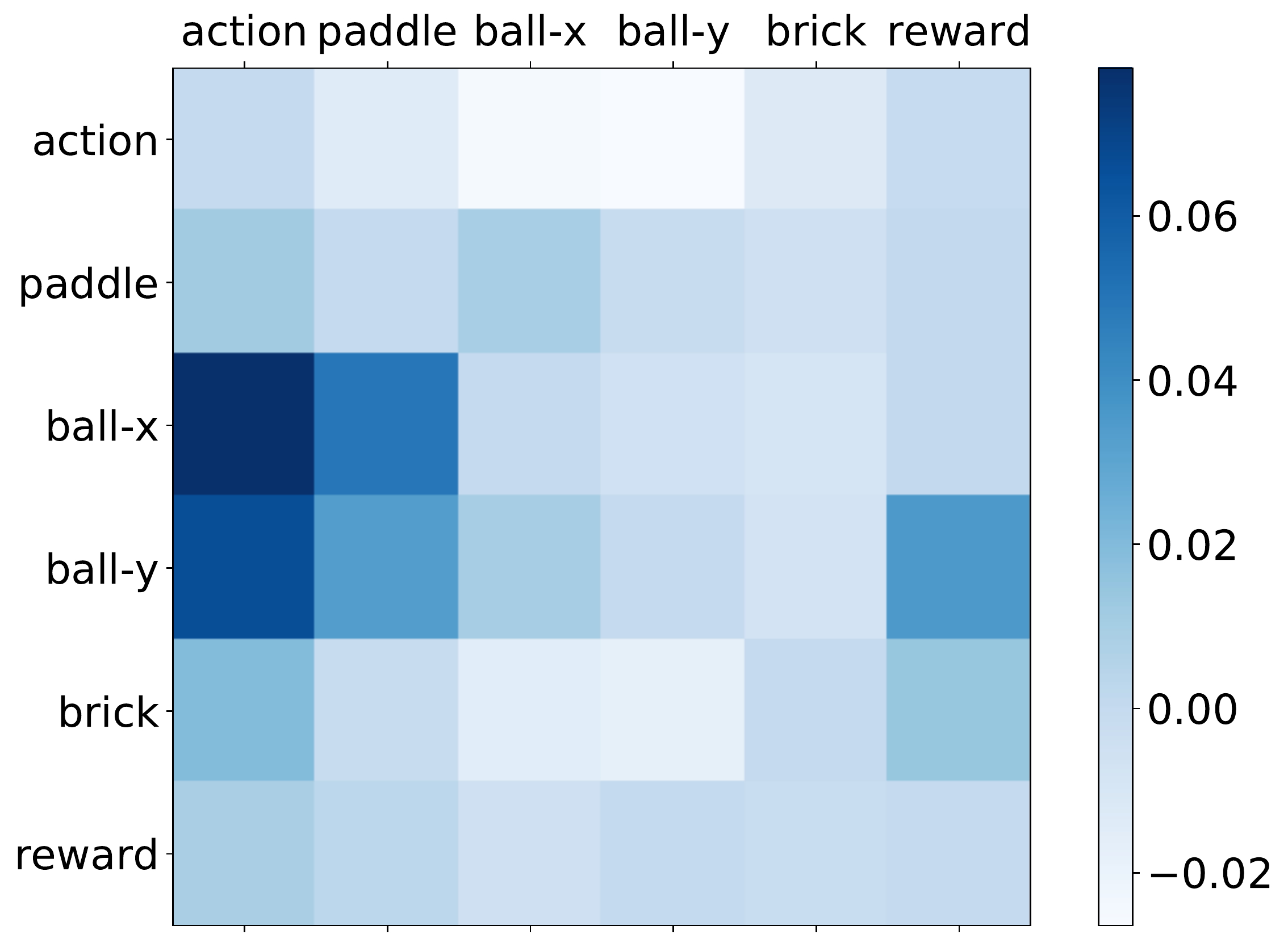}
        \caption{}
    \end{subfigure}
    \begin{subfigure}{.24\linewidth}
        \includegraphics[scale=0.14]{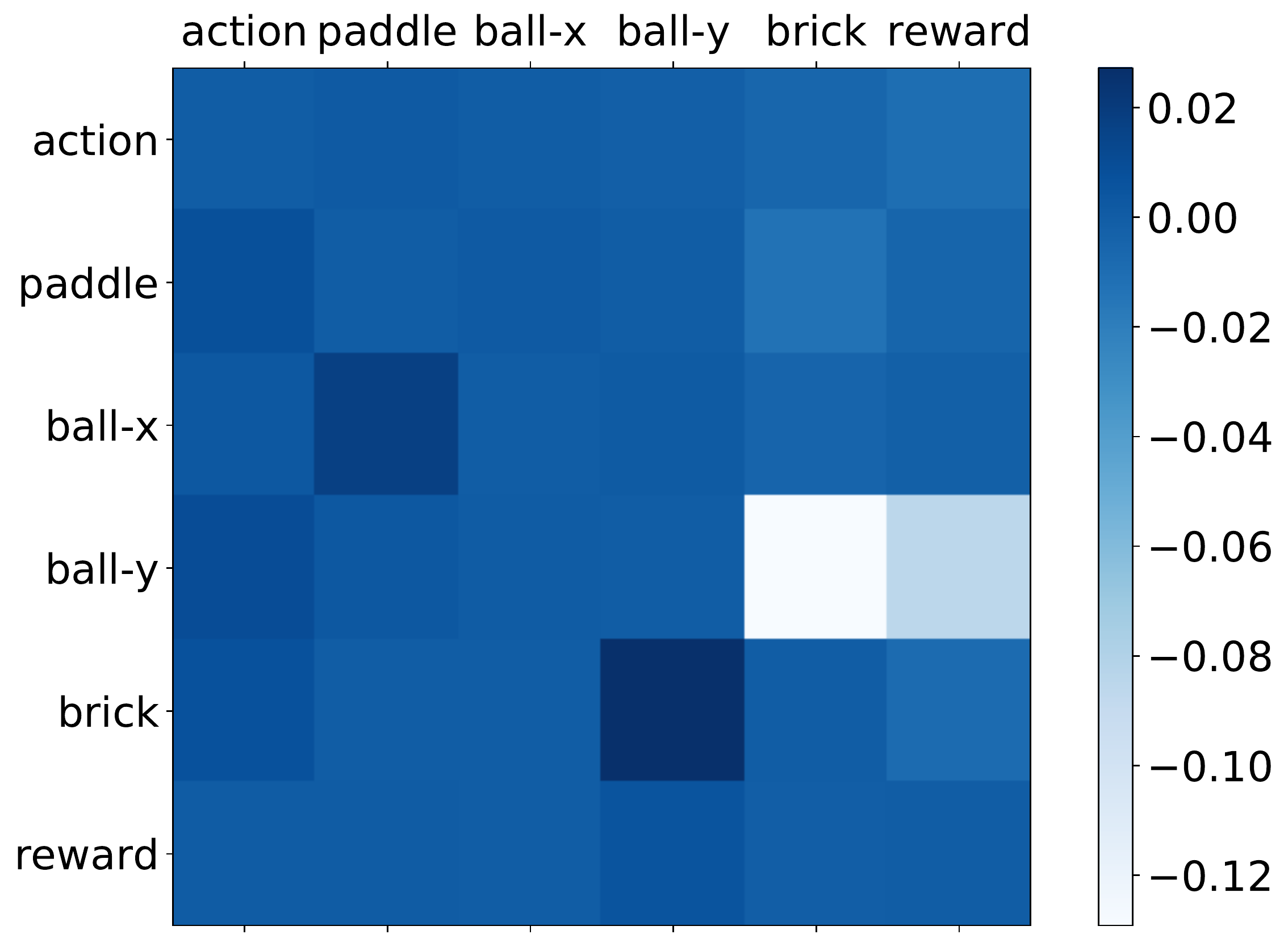}
        \caption{}
    \end{subfigure}
    \begin{subfigure}{.24\linewidth}
        \includegraphics[scale=0.14]{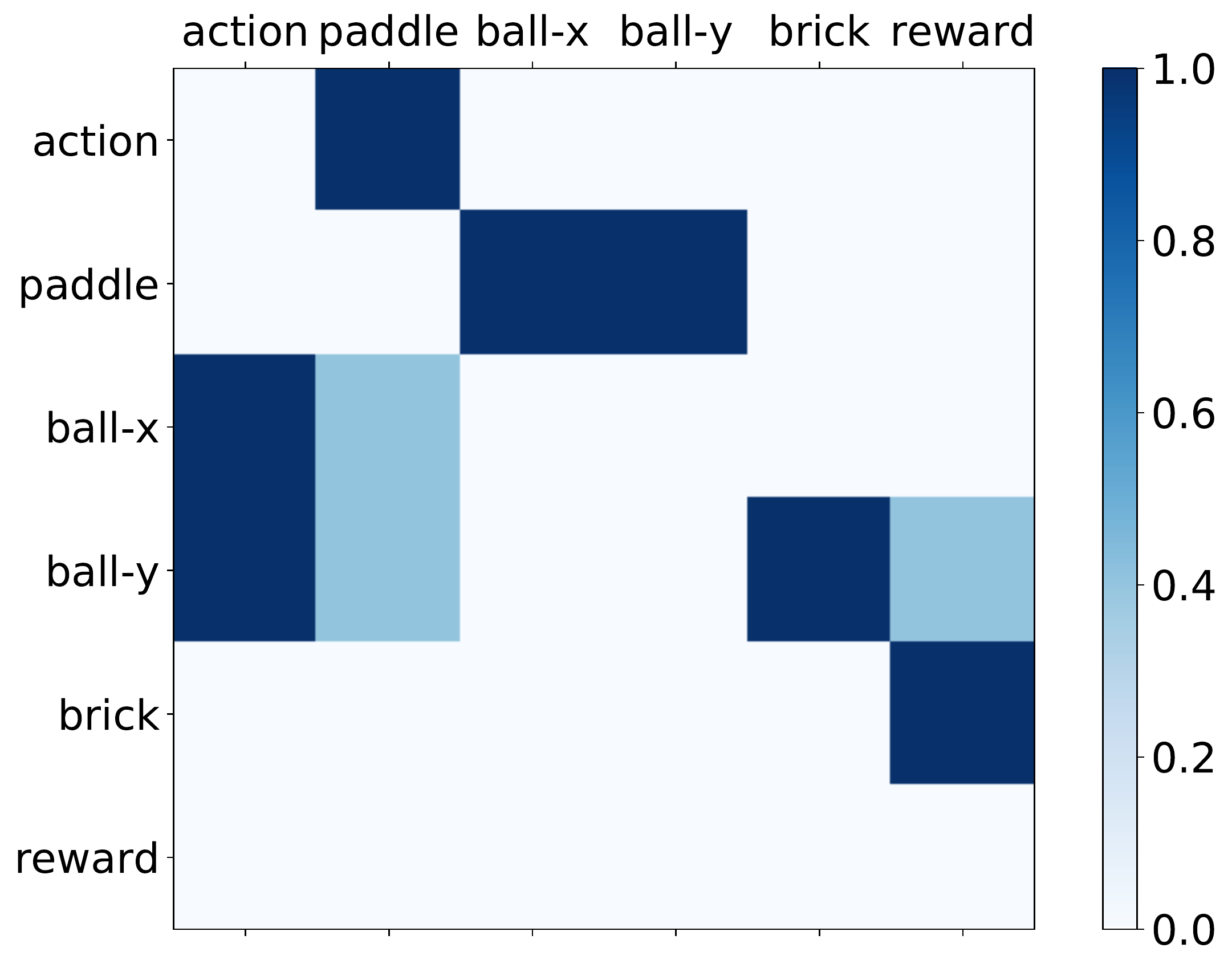}
        \caption{}
    \end{subfigure}
    \rulesep
    \begin{subfigure}{.24\linewidth}
        \includegraphics[scale=0.14]{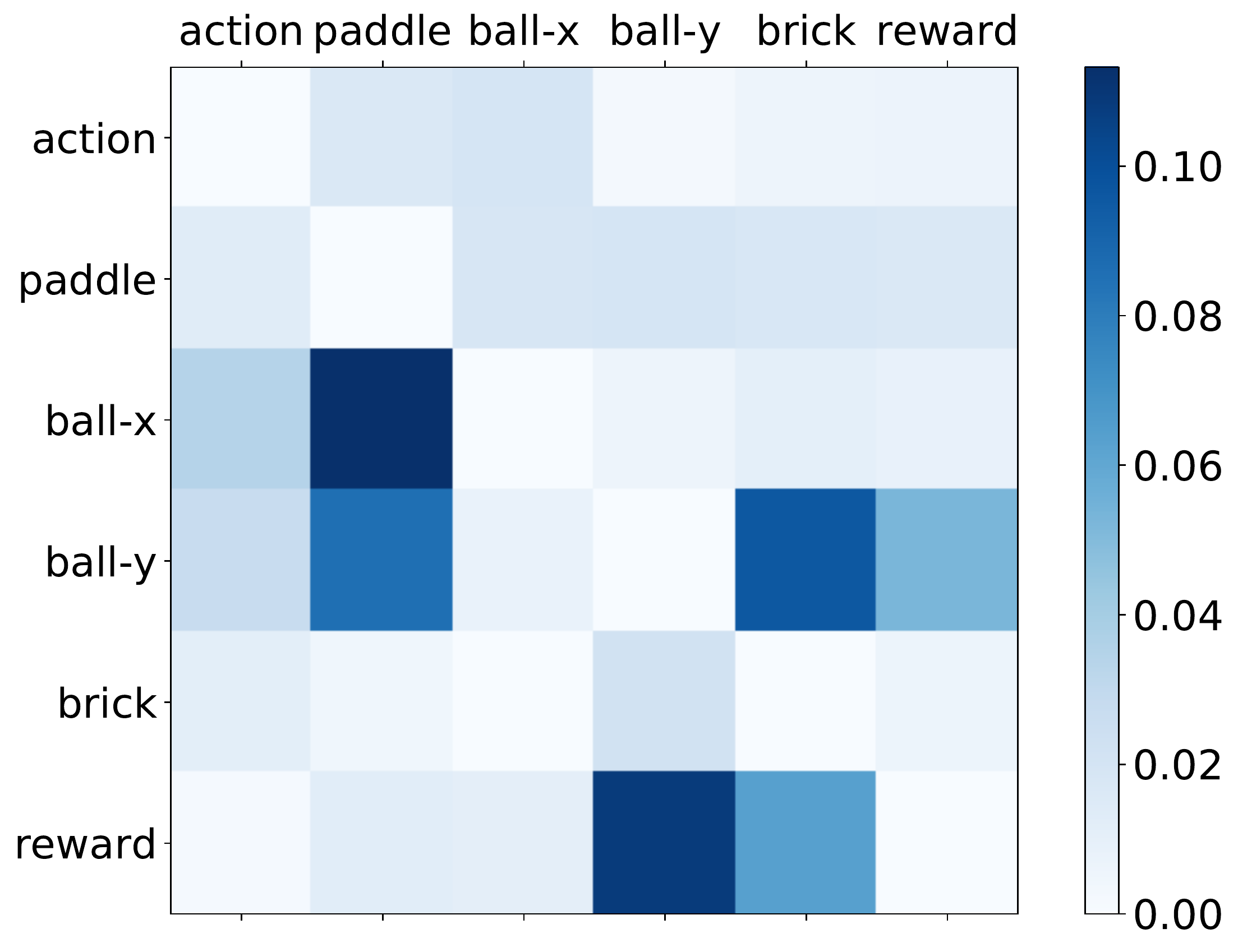}
        \caption{}
    \end{subfigure}
    \begin{subfigure}{.24\linewidth}
        \includegraphics[scale=0.14]{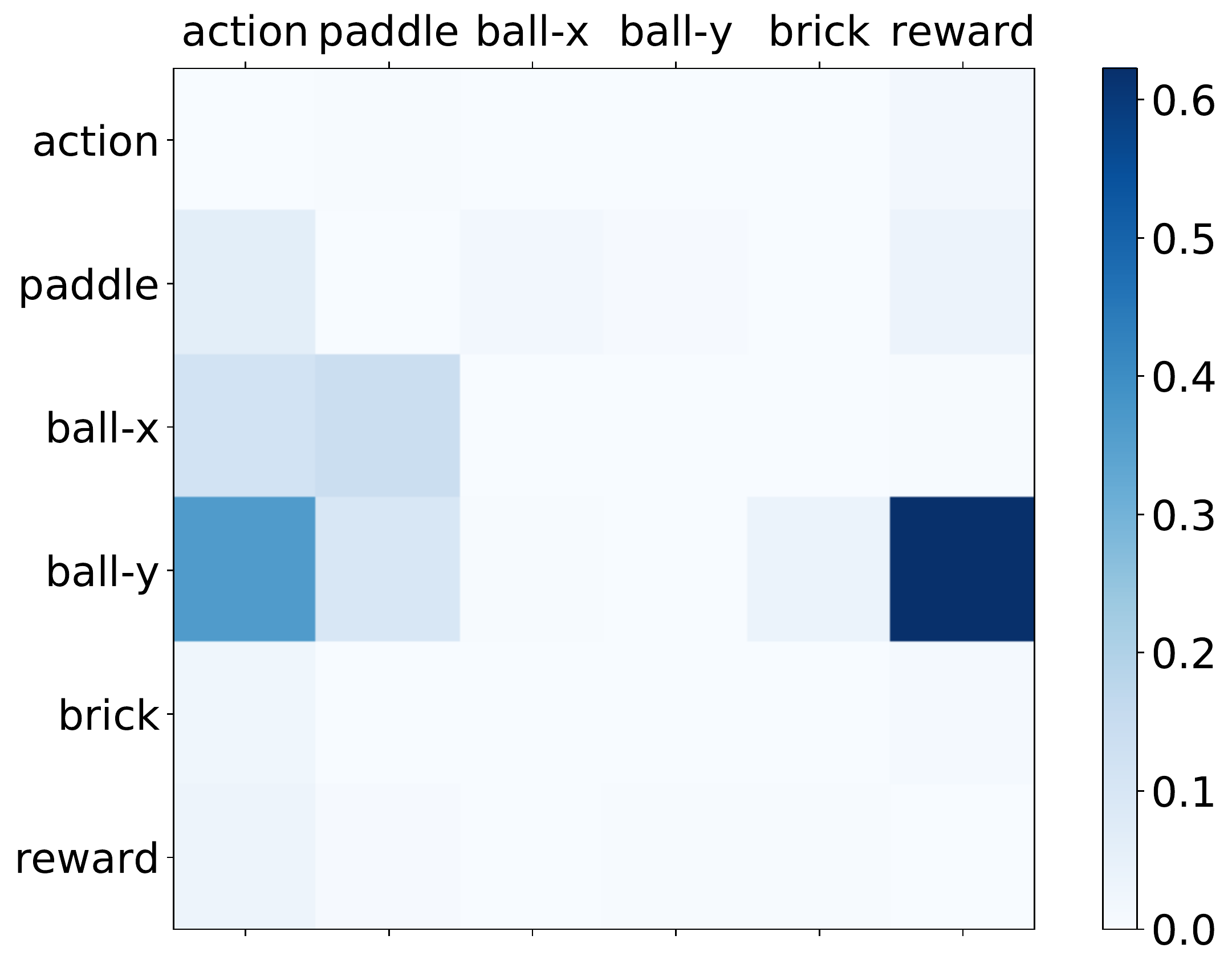}
        \caption{}
    \end{subfigure}
    \begin{subfigure}{.24\linewidth}
        \includegraphics[scale=0.14]{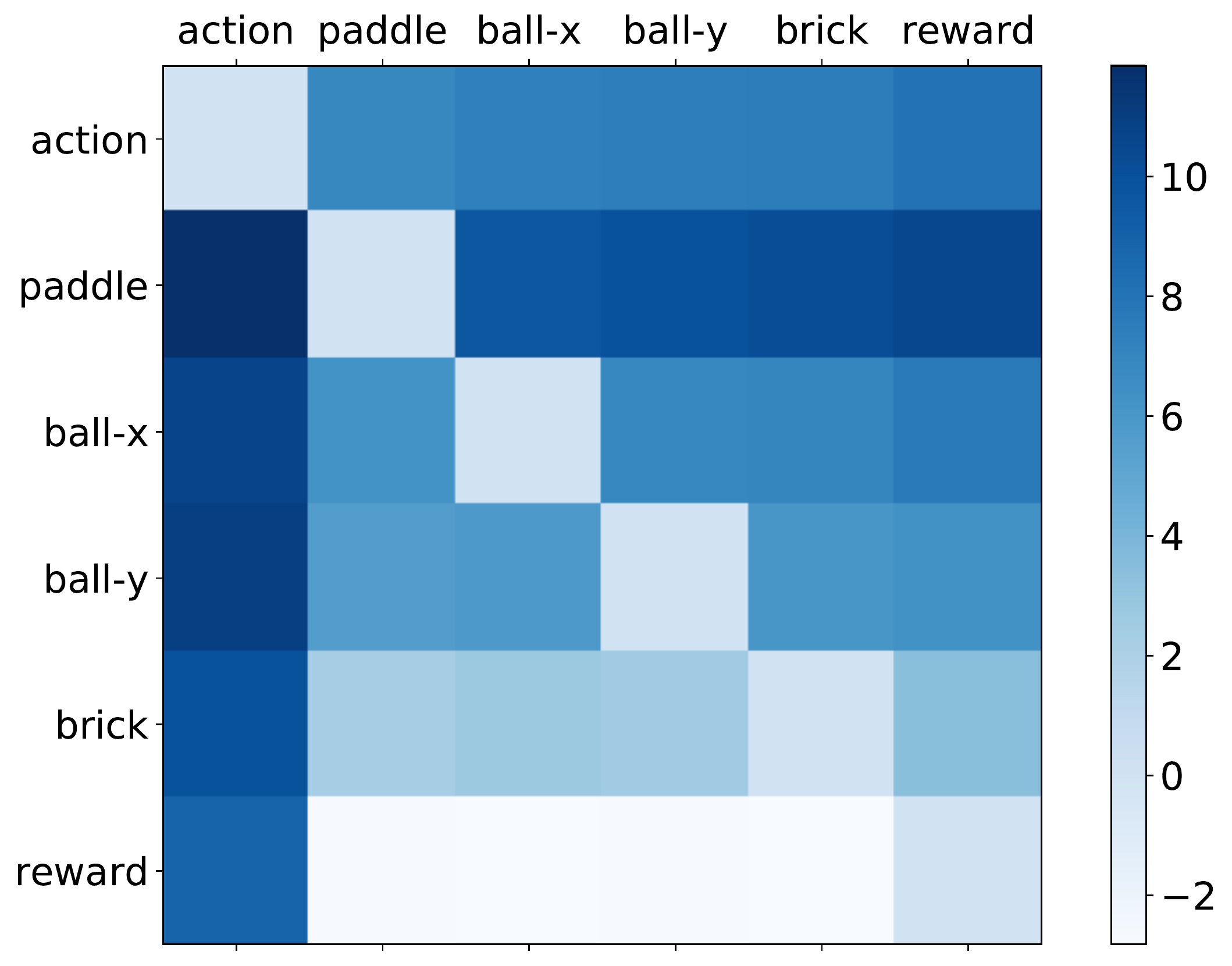}
        \caption{}
    \end{subfigure}
    \caption{(a) Predictive strength $W_{ji}$ inferred by our method in Section \ref{sec:video_games}. The $(j,i)$ element denotes the inferred causal strength from $j$ to $i$. (e) True underlying causal relations are marked dark, with light color marking competing causal relations that are indistinguishable from data. Other subfigures are: directional strength inferred by (b) mutual information (c) transfer entropy (d) linear Granger (f) kernel Granger (g) elastic net (h) causal influence.}
\label{fig:breakout_comparison}
\end{figure}

To see how our method can discover the directional (possibly causal) relations in real video games, and potentially improve reinforcement learning (RL) or imitation learning (IL), we apply our method to the relational inference between the trajectories of different objects from a trained CNN RL-agent playing Atari Breakout games (\cite{bellemare2013arcade}, implementation details see Appendix \ref{app:breakout}). Fig.~\ref{fig:breakout_comparison} shows the inferred $W_{ji}$ matrix for our method and compared methods, respectively. The true underlying causal chain is marked in dark color in Fig.~\ref{fig:breakout_comparison}e, with light color marking the competing causal relations that are indistinguishable from data (e.g. decrease of bricks and increase of reward happen at the same time step, so we cannot distinguish ball-y$\to$brick and ball-y$\to$reward). 
Compared with other methods, we see that our method is able to discover comparatively most of the causal relations without finding false positives. Specifically, it correctly discovers a prominent causal direction from the ball's $y$ position to the reward, as well as brick $\to$ reward, ball-x $\to$ action, ball-y $\to$ action. The latter two show that the ball's $x$ and $y$ positions also have influences on the trained agent's action: in order that the ball does not fall to the bottom, the agent has to position itself at the right position depending on the $x$ and $y$ positions of the ball. 

In comparison, mutual information (Fig. \ref{fig:breakout_comparison}b) gives a symmetric matrix that does not differentiate the two possible directions, and also misses the arrows ball-y$\to$brick$\to$reward. For transfer entropy (Fig. \ref{fig:breakout_comparison}c), although it correctly discovers a number of  causal arrows, it also gives relatively high scores for some incorrect arrows: brick$\to$ action, ball-y$\to$ball-x. For kernel Granger (Fig. \ref{fig:breakout_comparison}f), although it correctly discovers four causal  relations, it also incorrectly finds reward$\to$ball-y and reward$\to$brick. For elastic net (Fig. \ref{fig:breakout_comparison}g), it correctly discovers two prominent causal relations: ball-y$\to$action and ball-y$\to$reward, but misses a few others. Linear Granger (Fig. \ref{fig:breakout_comparison}d) and causal influence (Fig. \ref{fig:breakout_comparison}h) fail to discover useful causal arrows.

\subsection{Experiment with heart-rate vs. breath-rate and rat brain EEG datasets}
\label{sec:heart_rate}

\begin{figure}[hb]
\centering
\begin{subfigure}[b]{.32\linewidth}
\includegraphics[scale=0.24]{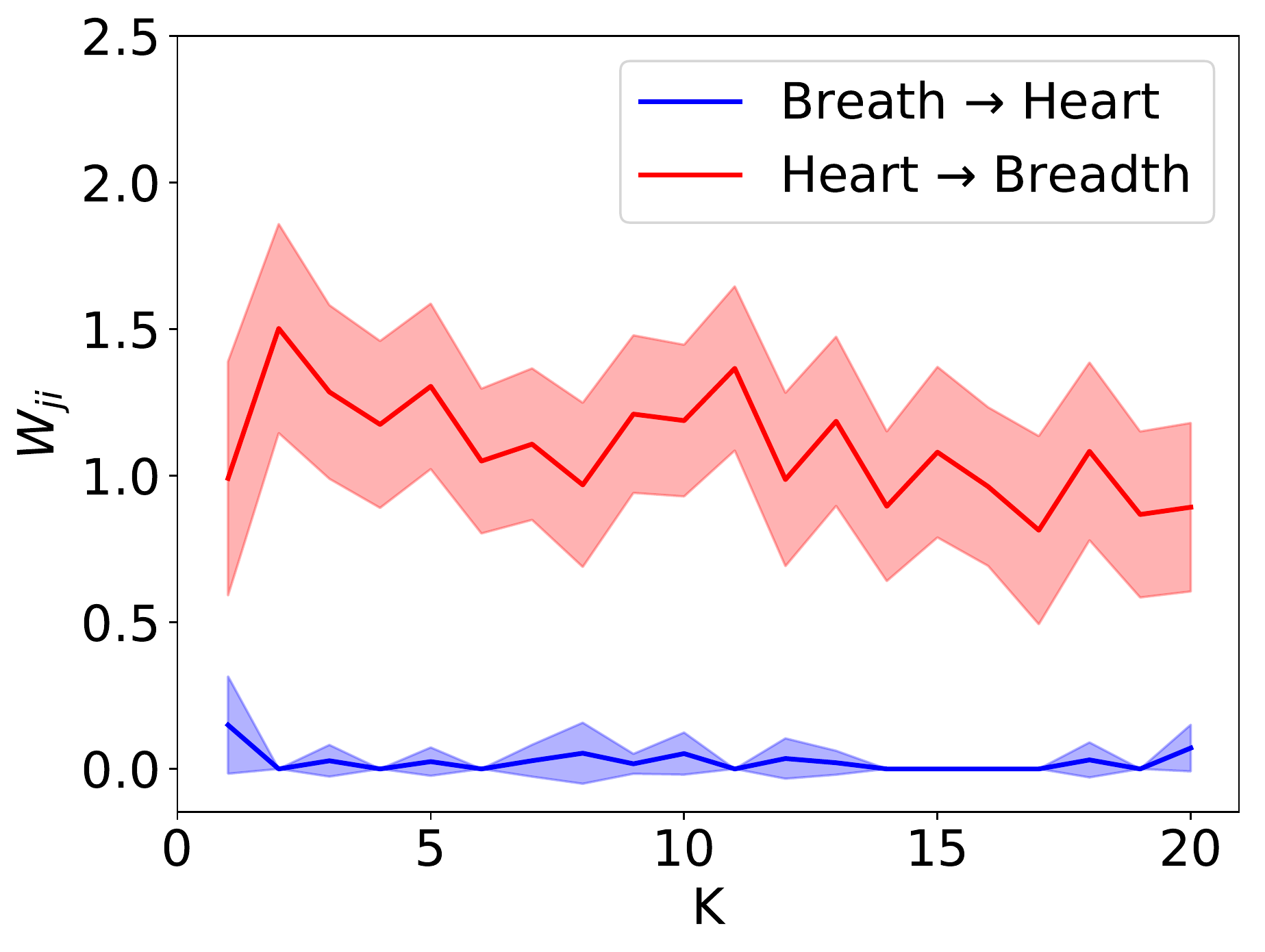}
\vspace{-15pt}\caption{}
\end{subfigure}
\begin{subfigure}[b]{.32\linewidth}
\includegraphics[scale=0.34]{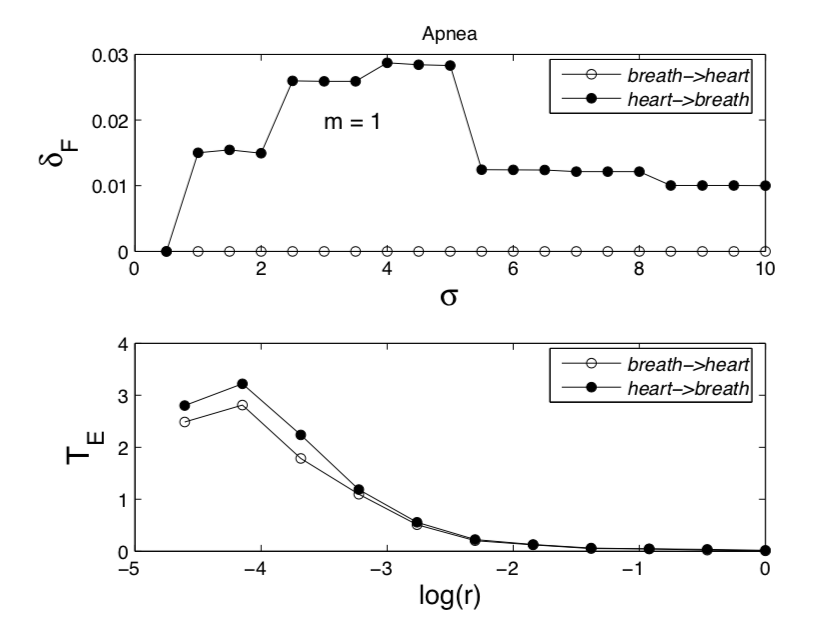}
\vspace{-15pt}\caption{}
\end{subfigure}
\begin{subfigure}[b]{.32\linewidth}
\includegraphics[scale=0.34]{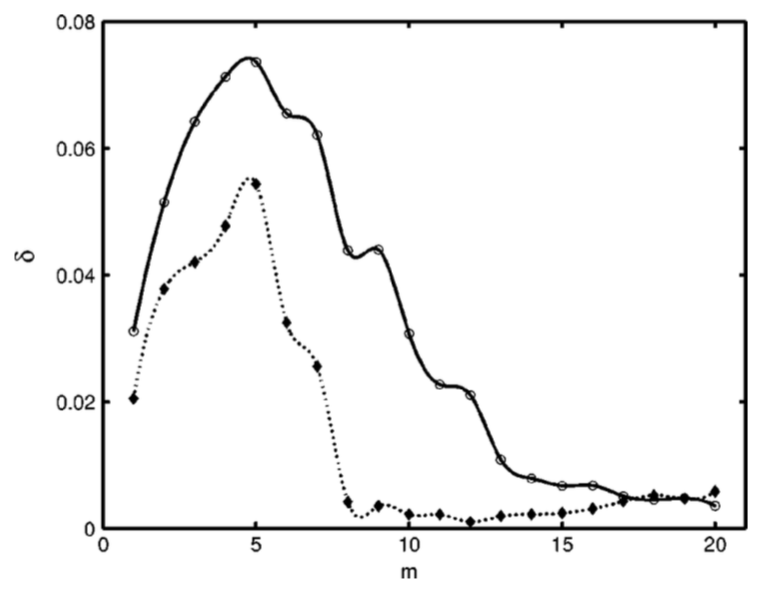}
\vspace{-15pt}\caption{}
\end{subfigure}
\vspace{-4pt}
\caption{(a) Predictive strength $W_{ji}$ inferred by our method with the heart-rate vs. breath-rate dataset, averaged over 50 initializations of $f_\theta$. The shaded areas are the 95\% confidence interval. (b) Upper: the filtered causality index  vs. varying width of Gaussian kernel $\sigma$ \cite{marinazzo2008kernel}; lower: transfer entropy vs. $r$, the length scale \cite{schreiber2000measuring}; (c) The causality index for breath$\to$heart (lower) and heart$\to$breath (upper) in \cite{ancona2004radial}, where $m$ is the maximum time lag (equivalent to our $K$).}
\label{fig:apnea}%
\vskip -0.08in
\end{figure}

Now we test our algorithm with real-world datasets. As a common dataset studied in previous causal works, we use the time-series of the breath rate and instantaneous heart rate of a sleeping patient suffering from sleep apnea (samples 2350-3550 of data set B from Santa Fe Institute time series contest held in 1991, available in \cite{Physionet}). We apply our method to infer the directional relations between the breath rate and heart rate, with different maximum time horizon $K$. The result is shown in Fig.~\ref{fig:apnea}. We see that the predictive strength $W_{ji}$ from heart to breath is significantly higher than the reverse direction that is basically 0, consistent with the results from previous causal inference methods \cite{schreiber2000measuring, ancona2004radial,marinazzo2008kernel} as also shown in Fig.~\ref{fig:apnea}(b)(c). Notably, the $W_{ji}$ from heart to breath estimated by our method remains at roughly the same level for different $K$s, in contrast to the decaying causality index w.r.t. increasing history length in (\cite{ancona2004radial}, Fig. \ref{fig:apnea} (c)), showing a merit of our method in estimating directional strength across different time-horizons, aided by the flexibility of neural nets in extracting the right information to predict the future. The implementation details are provided in Appendix \ref{app:real_dataset}. In addition, in Appendix \ref{app:rat_EEG_experiment} we test our algorithm on a rat EEG dataset, and obtain consistent result with previous works.

\section{Discussion and conclusion}

In this paper, we have introduced a novel relational learning with Minimum Predictive Information Regularization (MPIR) method for exploratory discovery of nonlinear directional relations from observational time series. It allows functional approximators like neural nets to learn complex directional relations from time series data. We prove its three theoretical properties,
and provide intuition that it favors variables that directly cause the variable of interest. We demonstrate in synthetic datasets, a video game environment and heart-rate vs. breath-rate datasets, that our method has better capability to handle nonlinearity, and can scale to large numbers of time series. We believe our work endows practitioners with a useful tool for deciphering the directional relations in complex systems, and are excited to see it in broader applications.

\chapter{Meta-learning autoencoders for few-shot prediction}
\label{chap7:mela}

Compared\footnote{The paper ``Meta-learning autoencoders for few-shot prediction'' is under review. Authors: Tailin Wu, John Peurifoy, Isaac L. Chuang, Max Tegmark \cite{wu2018meta}. } to humans, machine learning models generally require significantly more training examples and fail to extrapolate from experience to solve previously unseen physical prediction tasks.  To help close this performance gap, we augment single-task neural networks with a meta-recognition model which learns a succinct model code via its autoencoder structure, using just a few informative examples. The model code is then employed by a meta-generative model to construct parameters for the task-specific model. We demonstrate that for previously unseen tasks, without additional training, this {\it Meta-Learning Autoencoder} (MeLA) framework can build models that closely match the true underlying models, with losses significantly lower than fine-tuned baseline networks, and performance that compares favorably with state-of-the-art meta-learning algorithms.  MeLA also adds the ability to identify influential training examples and predict which additional unseen data will be most valuable to improve model prediction.

\section{Introduction}
Physical reasoning with few examples is an essential part of human-like intelligence. Humans are not only able to develop physical models that can describe the dynamics of objects in a single environment, we can generalize to a continuum of models in unseen environments with few observations (known as few-shot learning). For example, after seeing several moving objects accelerated by different forces in different environments,
humans are able to develop a meta-model that can generalize to a continuum of unseen accelerated objects. Upon arriving at a new environment and seeing an object moving for only a short time, s/he can quickly propose a new model for the moving object, including a good estimate of its acceleration. Incorporating this ability of generalization beyond the initial environments (training data) for quick recognition from few examples remains an important challenge in machine learning.

Great progress has been made in recent years towards developing machine learning models for physical systems. For example, disentangling recognition and dynamics models\cite{NIPS2017_69510}, visual de-animation\cite{NIPS2017_6620}, modeling physical interactions\cite{chang2016compositional,battaglia2016interaction,NIPS2017_7040}. However, most works aim to learn models that perform well in a single environment, without considering generalization to unseen environments where the dynamics or the environment constraint is novel. Hence when encountering new environments with different dynamics or environment constraints, the model has to be relearned. Moreover, the training of the models usually requires a large number of examples, posing a performance gap compared with humans.

In this work, we tackle the above problems by proposing a novel class of neural network architecture for few-shot/meta-learning of physical models, which learns to generalize across environments so that the learning of the dynamics in an unseen environment only requires a few examples. Although much progress has been made in recent years in few-shot/meta-learning, a large number of works are specifically designed for classification \cite{NIPS2016_6385,koch2015siamese,snell2017prototypical,edwards2016towards}, and the regression benchmark is only a simple trigonometric sine regression problem. Our work fills this gap, by introducing a \emph{Meta-Learning Autoencoders} (MeLA) architecture, designed for few-shot physical prediction/regression. 

At its core, MeLA consists of a learnable meta-recognition model that can for each (unseen) task distill a few input-output examples into a model code vector parameterizing the task's functional relationship, and a learnable meta-generative model that maps this model code into the weight and bias parameters of a neural network implementing this function. This architecture forces the meta-recognition model to discover and encode the important variations of the functional mappings for different tasks, and the meta-generative model to decode the model codes to corresponding task-specific models with a common model-generating network. 

This brings the key innovation of MeLA: for a class of tasks, MeLA does not attempt to learn a \emph{single} good initialization for multiple tasks\cite{pmlr-v70-finn17a}, or learn an update function\cite{schmidhuber1987evolutionary,bengio1992optimization,andrychowicz2016learning}, or learn an update function together with a \emph{single} good initialization\cite{li2017meta,ravi2016optimization}. Instead, it learns to map the few examples from different datasets into \emph{different} models, which not only allows for more diverse model parameters tailored for each individual tasks, but also obviates the need for fine-tuning.
Moreover, by encoding each function as a vector in a single low-dimensional latent space, MeLA is able to generalize beyond the training datasets, by both interpolating between and extrapolating beyond learned models into a continuum of models. 
We will demonstrate that the meta-learning autoencoder has the following 3 important capabilities:
\begin{enumerate}
    \item \textbf{Augmented model recognition:} MeLA strategically builds on a pre-existing, single-task trained network for physical prediction, augmenting it with a second network used for meta-recognition. It achieves lower loss in unseen environments at zero and few gradient steps, compared with both the original architecture upon which it is based and state-of-the-art meta-learning algorithms.
    \item \textbf{Influence identification:} MeLA can identify which examples are most useful for determining the model
  (for example, a rectangle's vertices have far greater influence in determining its position and size than its inferior points).
  
    \item \textbf{Interactive learning:} MeLA can actively request new samples which maximize its ability to learn models.
\end{enumerate}
\section{Methods}
\label{sec:Methods}

\subsection{Meta-learning problem setup}

We are interested in modeling a set of vector-valued functions $\h_{\alpha}$ (which we will refer to as {\it models}), that each map an $m$-dimensional input vector $\x$ into an $n$-dimensional output vector $\y$.
Let's first consider the case for a single dataset. Given many input-output pairs 
$\y_i=\h_\alpha(\x_i)$ linked by the same function $\h_{\alpha}$, we group the corresponding vectors into matrices 
$\XX$ and $\YY$ whose $i^{th}$ rows are the vectors $\x_i$ and $\y_i$. Since we focus on physical prediction/regression, the target $\y\in\R^n$ is continuous. This class of problems includes a wide range of scenarios, {\eg}, modeling time series data, learning physics and dynamics, and frame-to-frame prediction of videos.

The meta-learning problem we tackle is as follows. Suppose that we are given an ensemble of datasets 
$\{\D^\alpha\}=\{(\XX^\alpha,\YY^\alpha)\}$, 
$\alpha=1,2,...$, each of which is generated by a corresponding function $\h_\alpha$. 
In the single-task scenario, we want to train a model $\f$ 
that predicts all output vectors for the corresponding input vectors from a single dataset, 
to minimize some loss function $\ell$ that quantifies the prediction errors. 
The meta-learning goal is, after training on an ensemble of training datasets 
$\D^1$,...,$\D^a$,
to be able to quickly learn from few examples from held-out datasets 
$\D^{a+1}$,...,
and obtain a low loss on them.


\begin{figure}[t]
\begin{center}
\centerline{\includegraphics[width=0.8\columnwidth]{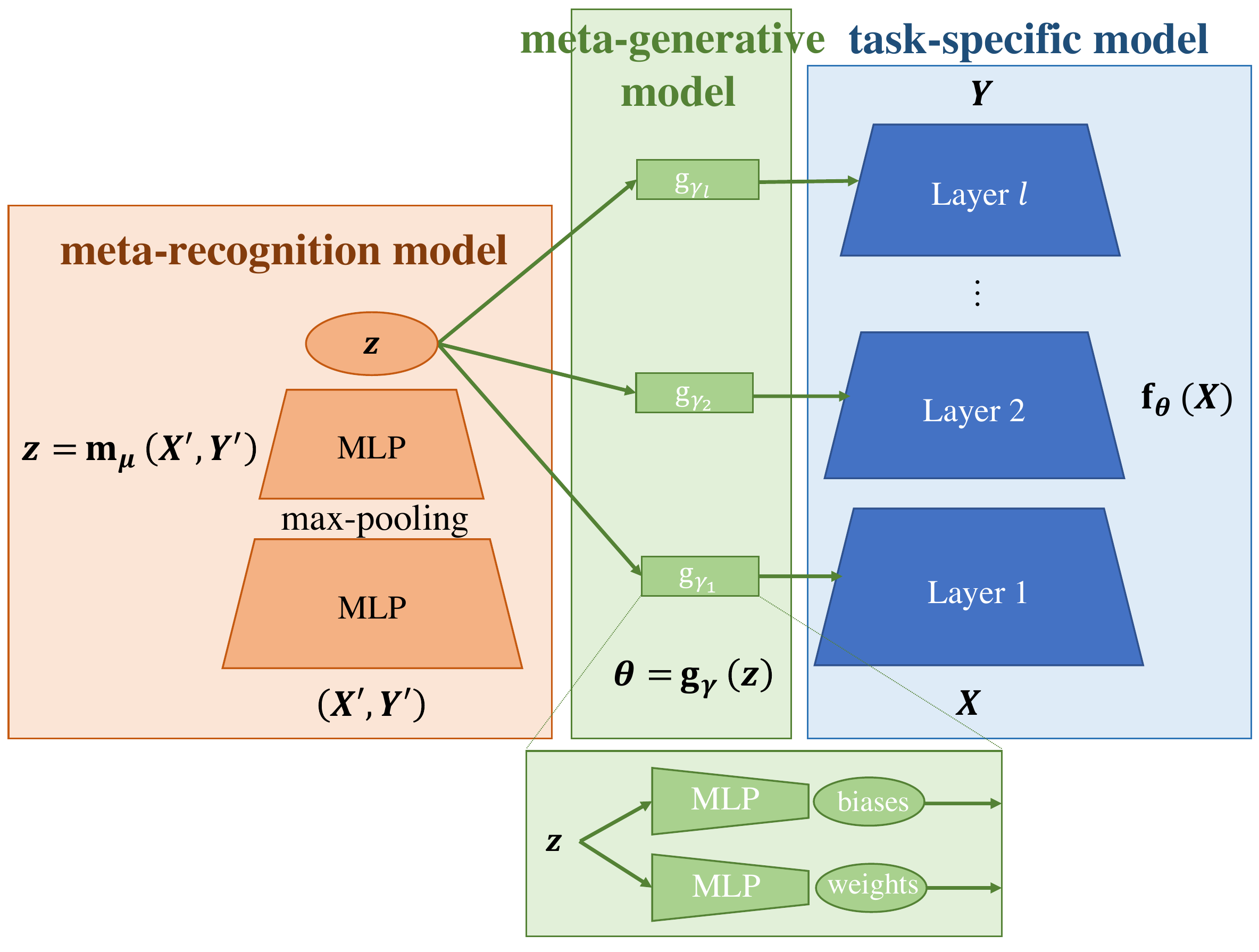}}
\caption{Architecture of our Meta-Learning Autoencoder (MeLA). MeLA augments a pre-existing neural network architecture $\f_{\phivec}$ (right) with a meta-recognition model (left) that generates the model code $\z$ based on a few examples $\XX^{\prime}$, $\YY^{\prime}$, and a meta-generative model (middle) that generates the parameters $\phivec$ of model $\f_{\phivec}$ based on the model code. $\f_{\phivec}$, $\g_{\gammavec}$ and $\m_{\muvec}$ are implemented as multilayer perceptron (MLP).}
\label{ArchitectureFig}
\end{center}
\vskip -0.15in
\end{figure}

\subsection{Meta-learning autoencoder architecture}
\label{sec:architecture}

The architecture of our Meta-Learning Autoencoder (MeLA) is illustrated in \fig{ArchitectureFig}.
It is defined by three vector-valued functions $\f_{\phivec}$, $\g_{\gammavec}$ and $\m_{\muvec}$ that are defined by feedforward neural networks parametrized by vectors $\phivec$, $\gammavec$ and $\muvec$, respectively.
In contrast to prior methods for learning to quickly adapt to different datasets \cite{pmlr-v70-finn17a} or using memory-augmented setup \cite{pmlr-v48-santoro16,NIPS2016_6385}, the MeLA takes full advantage of the prior that the datasets are generated by a hidden model class, where the functions $\h$ lie in a relatively low-dimensional submanifold of the space of all functions. Based on this prior, we use a meta-recognition model $\m_{\muvec}$ that maps a whole dataset $\D=(\XX,\YY)$ to a model code vector $\z$, and a meta-generative model $\g_{\gammavec}$ that maps $\z$ to the parameters vector $\phivec$ of the network implementing the function $\f_{\phivec}$. In other words, $\phivec=\g_{\gammavec}(\z)$, and for a specific dataset $(\XX,\YY)$, $\f_{\phivec}$ can be instantiated by 
\begin{equation}
{\f_{\phivec}}
={\f_{{\g_{\gammavec}}(\z)}}
=\f_{{\g_{\gammavec}}({\m_{\muvec}}(\XX,\YY))}.
\end{equation}


This architecture is designed so that it can easily transform a neural network that is originally intended to learn from a single task into an architecture that can perform meta/few-shot learning on a number of tasks, combining the knowledge of individual task architectures with MeLA's meta-learning power. If the original single-task model is $\f_{\phivec}$, then without changing the architecture of $\f_{\phivec}$, we can simply attach a meta-recognition model $\m_{\muvec}$ and a meta-generative model $\g_{\gammavec}$ that generates the parameters of $\f_{\phivec}$, and train on an ensemble of tasks. 

\paragraph{Network architecture examples} 

Although the MeLA architecture described above can be implemented with any 
choices whatsoever for the three feedforward neural networks that define the functions 
$\f_{\phivec}$, $\g_{\gammavec}$ and $\m_{\muvec}$,
let us consider simple specific implementations to build intuition and get ready for the numerical experiments.

Suppose we implement the main model $\f_{\phivec}$ as a network with two hidden layers with $s_1$ and $s_2$ neurons, respectively. Its input size is $s_0$ and its output size is $\sout$. 
The meta-recognition model $\m_{\muvec}$ takes as input $\XX$ and $\YY$ concatenated horizontally into a single
$N\times(s_0+\sout)$ matrix, where $N$ is the number of training examples at hand. 
The feedforward neural network implementing $\m_{\muvec}$ has 
two parts: the first is a series of layers that collectively transform the $N\times(s_0+\sout)$ input matrix into an $N\times \spool$ matrix, where $\spool$ is the number of output neurons in this first block (we typically use 200 to 400 below). Then a max-pooling operation is applied over the $N$ examples, transforming 
this $N\times \spool$ matrix into a single vector of length $\spool$. 
The meta-recognition model $\m_{\muvec}$ is thus defined independently of the number of  training examples $N$.
As will be explained in the ``Influence identification" subsection below, max-pooling is key to MeLA, forcing the meta-recognition model to learn to capture key characteristics in a few representative examples. 
The second block of the $\m_{\muvec}$ network is a multilayer feedforward neural network, which takes as input the max-pooled vector, and 
transforms it into a 
$\slat$-dimensional model code $\z$ that parametrizes the functional relationship between $\x$ and $\y$.

The meta-generative model $\g_{\gammavec}$ takes as input the model code $\z$, 
and for each layer in the main model $\f_{\phivec}$, it has two separate neural networks that map $\z$ to all the weight and bias parameters of that layer. 
We typically implement each of these subnetworks of $\g_{\gammavec}$ using 2-3 hidden layers with 60 neurons each. The number of parameters for the new model $\f_{\g_\gamma(z)}$ is linear w.r.t. the number of parameters for the original model $\f_{\phivec}$, independent of the number of tasks.

\subsection{MeLA's meta-training and evaluation}

The extension from the training on a single-task $\f_{\phivec}$ to MeLA is straightforward. Suppose that the loss function for the single-task is $\ell(\yhat,\y)$, with expected risk
$R_{\ell,\D_k}(\f_{\phivec})\equiv\mathbb{E}_{(\XX,\YY)\sim \D_k}\left[\ell(\f_{\phivec}(\XX),\YY)\right]$. Then the meta-expected risk for MeLA is

\begin{equation}
\label{eqn:meta_risk}
    R_{\ell,p(\D)}(\m_{\muvec}, \g_{\gammavec})=\mathbb{E}_{\D_k\sim p(\D)}\left[\mathbb{E}_{(\XX,\YY)\sim \D_k}\left[\ell(\f_{\g_{\gammavec}(\z)}(\XX),\YY)\right]\right]
\end{equation}
where $p(\D)$ is the distribution for datasets $\{\D_k\}$ generated by the hidden model class $\h$. The goal of meta-training is to learn the parameters for the meta-recognition model $\m_{\muvec}$ and meta-generative model $\g_{\gammavec}$ such that $R_{\ell,p(\D)}(\m_{\muvec},\g_{\gammavec})$ is minimized:

\begin{equation}
\label{eqn:minimize_meta_risk}
(\muvec,\gammavec)=\argmin{(\muvec,\gammavec)}R_{\ell,p(D)}(\m_{\muvec}, \g_{\gammavec})
\end{equation}

\begin{algorithm}[t]
   \caption{\textbf{Meta-Training for MeLA}}
   \label{alg:standard}
\begin{algorithmic}
   \STATE {\bfseries Require} datasets $\{\D^\alpha\}=\{(\XX^\alpha, \YY^\alpha)\}$, $\alpha=1,2,...a$ 
   \STATE {\bfseries Require $n$}: number of meta-iterations
   \STATE {\bfseries Require $\beta$}: learning rate hyperparameter
   \STATE 1: Initialize random parameters for $\m_{\muvec}, \g_{\gammavec}$.
   \STATE 2: $i\gets 0$
   \STATE 3: \textbf{while} $i<n$:
   \STATE 4: \ \ \ \ \ \ $\{\D^{\alpha^\prime}\}\gets \text{permute}(\{\D^\alpha\})$   //\textit{Randomly permute the order of datasets.}
   \STATE 5: \ \ \ \ \ \ \textbf{for} $\D_j$ \textbf{in} $\{\D^{\alpha^\prime}\}$ \textbf{do}
   \STATE 6: \ \ \ \ \ \ \ \ \ \ \ \ Split $\D_j=(\XX_j,\YY_j)$ into training examples $(\XX_j^{\rm train},\YY_j^{\rm train})$ and \STATE \ \ \ \ \ \ \ \ \ \ \ \ \ \ \ testing examples $(\XX_j^{\rm test},\YY_j^{\rm test})$
   \STATE 7: \ \ \ \ \ \ \ \ \ \ \ \ $\z\gets \m_{\muvec}(\XX_j^{\rm train},\YY_j^{\rm train})$
   \STATE 8: \ \ \ \ \ \ \ \ \ \ \ \ $\phivec\gets \g_{\gammavec}(\z)$
   \STATE 9: \ \ \ \ \ \ \ \ \ \ \ \ Update $\muvec\gets\muvec-\beta\nabla_{\muvec}\ell\left[\f_{\g_{\gammavec}(\z)}(\XX_j^{\rm test}),\YY_j^{\rm test}\right]$ 
   \STATE \ \ \ \ \ \ \ \ \ \ \ \ \ \ \ \ \ \ \ \ \ \ \ \ \ \  $\gammavec\gets\gammavec-\beta\nabla_{\gammavec}\ell\left[\f_{\g_{\gammavec}(\z)}(\XX_j^{\rm test}),\YY_j^{\rm test}\right]$ 
   \STATE 10: \ \ \ \ \ \textbf{end for}
   \STATE 11: \ \ \ \ \ $i\gets i+1$
   \STATE 12: \textbf{end while} 
\end{algorithmic}
\end{algorithm}

Algorithm \ref{alg:standard} illustrates the step-by-step meta-training process for MeLA implementing an empirical meta-risk minimization for Eq. (\ref{eqn:minimize_meta_risk}). In each iteration, the training dataset ensemble is randomly permuted, from which each dataset is selected once for inner-loop task-specific training. Inside the task-specific training, the training examples for each dataset are used for calculating the model code $\z=\m_{\mu}(\XX^{\rm train},\YY^{\rm train})$, after which the model parameter vector $\phivec=\g_{\gammavec}(\z)$ and the testing examples are used to calculate the task-specific testing loss $\ell(\f_{\g_{\gammavec}(\z)}(\XX^{\rm test}),\YY^{\rm test})$, from which the gradients w.r.t.~$\muvec$ and $\gammavec$ are computed and used for one-step of gradient descent for the meta-recognition model and meta-generative model. Note that here the task-specific testing loss in the training datasets serves as the training loss in the meta-training.

During the evaluation of MeLA, we use the held-out datasets unseen during the meta-training. For each held-out dataset, we split it into training and testing examples. The training examples is fed to MeLA and a task-specific model is generated without any gradient descent. Then we evaluate the task-specific model on the testing examples in the held-out datasets. We also evaluate whether the task-specific model can further improve with a few more steps of gradient descent.

\subsection{Influence identification}
\label{sec:influence}

The max-pooling over examples in the meta-recognition model $\m_{\muvec}$ is key to MeLA, and also provides a natural way to identify the influence of each example on the model $\f_{\phivec}$. 
Typically, some examples are more useful than others in in determining the model. 
For example, suppose that we try to learn a function $\f_{\phivec}$ defined on $\R^2$ that equals 1 inside a polygon and 0 outside, 
with different polygons corresponding to different models parameterized by $\phivec$.
Then data points near the polygon vertices carry far more information about $\phivec$ than do points in the deep interior,
and the max-pooling over the dimension of examples forces the meta-recognition model to recognize those influential points, and based on them perform computation that returns a model code that determines the whole polygon. 
Recall that max-pooling compresses $N\times\spool$ numbers into merely $\spool$, which means that for each column 
of the $N\times\spool$ matrix, only one of the $N$ examples takes the maximum value and hence contributes to this feature.
We therefore define the \emph{influence} of an example as

\begin{equation}
\label{eqn:influence}
\text{Influence}=\frac{\text{Number of columns where the example is maximal}}{\spool}
\end{equation}

The influence of each example can be interpreted as a percentage, since it lies in $[0,1]$, and the influences sum to 1 for all the examples in the dataset fed to the meta-recognition model.

\subsection{Interactive learning}
\label{sec:propose}

In some situations, measurements are hard or costly to obtain. 
It is then helpful if we can do better than merely acquiring random examples, and instead determine in advance at which data points $\x_i$ to collect measurements $\y_i$ to glean as much information as possible about the correct function $\f$.
Specifically, suppose that we want to predict $\hat{\y}^*=\f_{\phivec}(\x^*)$ as accurately as possible at a given input point $\x^*$ where we have no training data. If before making our prediction, we have the option to measure $\y$ at one of 
several candidate points $\x_1^{\prime},\x_2^{\prime}, ...\x_n^{\prime}$, then which point shall we choose?

The MeLA architecture provides a natural way to answer this question. 
We can first use $\f_{\phivec}$ to calculate the current predictions for $\y_1^{\prime},\y_2^{\prime}, ...\y_n^{\prime}$ at $\x_1^{\prime},\x_2^{\prime}, ...\x_n^{\prime}$ based on current model generated by $\phivec=\g_{\gammavec}(\z)$ and $\z=\m_{\muvec}(\XX^{\prime},\YY^{\prime})$, where $\XX^{\prime}$ and $\YY^{\prime}$ are the examples that are already given. Then we can fix the meta-parameters $\muvec$ and $\gammavec$, and calculate the \emph{sensitivity matrix} of $\y^*$ w.r.t.~each current prediction $\y_i^\prime$:


\begin{equation}
\label{eqn:sensitivity}
\frac{\partial\y^*}{\partial \y^\prime_i}=
\J
\frac{\partial\z}{\partial \y^\prime_i},
\quad\hbox{where}
\quad\J\equiv\frac{\partial \f_{\g_{\gammavec}(\z)}(\x^*)}{\partial\z}.
\end{equation}

We can select the candidate point whose sensitivity matrix has the largest determinant, \ie, the point for which the measured data
carries the most information about the answer $\y^*$ that we want:
\begin{equation}
    \y_i^\prime=\argmax{y_i^\prime}\left|\frac{\partial\y^*}{\partial \y^\prime_i}\right|
\end{equation}
If we model our uncertainty about $\y^*$ as a multivariate Gaussian distribution, then this criterion maximizes the entropy reduction, \ie, the number of bits of information learned about 
$\y^*$ from the new measurement.
Note that with $\gammavec$ fixed and for a given $\x^*$, the Jacobian matrix $\J$
is independent of the different candidate inquiry inputs $\x_1^{\prime},\x_2^{\prime}, ...\x_n^{\prime}$. 
This means that  we can simply select the  candidate point that has the largest ``projection" of $\left|\J\frac{\partial\z}{\partial \y^\prime_i}\right|$ onto $\J$, requiring in total only one forward and one backward pass for all the candidate examples to obtain the gradient. This factorization emerges naturally from MeLA's architecture.
\section{Related work}

MeLA addresses few-shot learning of physical dynamics models. Past works of learning physical dynamics models generally consider learning in a single environment, where the training and testing tasks have the same dynamics or environment constraint\cite{NIPS2017_69510,NIPS2017_6620,chang2016compositional,battaglia2016interaction,NIPS2017_7040}. Our MeLA architecture is complementary in that it can boost the architecture designed for single-task learning into one that can quickly adapt to new tasks with few examples.

MeLA addresses few-shot/meta-learning \cite{thrun2012learning,schmidhuber1987evolutionary,287172}, whose goal is to quickly adapt to new tasks with one-shot or few-shot examples. A recent innovative meta-learning method MAML\cite{pmlr-v70-finn17a} optimizes the parameters of the model so that it is easy to fine-tune to individual tasks in a few gradient steps. Another class of methods focuses on learning a learning rule or update functions\cite{schmidhuber1987evolutionary,bengio1992optimization,andrychowicz2016learning}, or learning an update function from a single good initialization\cite{li2017meta,ravi2016optimization}. Compared to these methods that only learn a \emph{single} good initialization point or how to update from a single initialization point, our method learns recognition and generative models that can quickly determine the model code for the model, and directly propose the appropriate neural network parameters tailored for each task without the need of fine-tuning.

Another interesting class of few-shot learning methods uses memory-augmented networks. \citet{NIPS2016_6385} proposes matching nets for one-shot classification, which generates the probability distribution for the test example based on the support set using attention mechanisms, essentially learning a ``similarity" metric between the test example and the support set. \citet{pmlr-v48-santoro16} utilizes a neural Turing machines for few-shot learning, and  \citet{duan2016rl,wang2016learning} learn fast reinforcement learning agents with recurrent policies using memory-augmented nets. In contrast to memory-augmented approaches, our model learns to distill features from representative examples and produces a model code, based on which it directly generates the parameters of the main model. This eliminates the need to store the examples for the support set, and allows a continuous generation of models, which is especially suitable for generating a continuum of regression models. Other few-shot learning techniques include using Siamese structures \cite{koch2015siamese} and evolutionary methods \cite{mengistu2016evolvability}.

Autoencoders are typically used for representation learning in a single dataset, and have only recently been applied to multiple datasets. The recent neural statistician work  \cite{edwards2016towards} applies the variational autoencoder approach to the encoding and generation of datasets. Compared to their work, our MeLA differs in the following aspects. Firstly, the problem is different. While in neural statistician, each example in the dataset is an instance of a class, in MeLA, we are dealing with datasets whose examples are $(\x,\y)$ pairs, where we don't know \emph{a priori} where the input $\x$ will be in testing time. Therefore, direct autoencoding of datasets is not enough for prediction, especially for regression tasks. Therefore, instead of using autoencoding to generate the \emph{dataset}, our MeLA uses autoencoding to generate the \emph{model} that can generate the dataset given test inputs $\X$, which is a more compact way to express the relationship between $\X$ and $\Y$.

\section{Experiments}
\label{Experiments}
Here we examine the core value of MeLA: can it transform a model that is originally intended for single-task learning into one that can quickly adapt to new tasks with few examples without training, and continue to improve with a few gradient steps?
The baseline we compare with is a single network pretrained to fit to all tasks, which during testing is fine-tuned to each individual task through further training. MeLA has the same main network architecture $\f_{\phivec}$ as this baseline network, supplemented by the meta-recognition and meta-generative models trained via Algorithm \ref{alg:standard}. 
We also compare with the state-of-the-art meta-learning algorithm MAML\cite{pmlr-v70-finn17a}, with the same network architecture $\f_{\phivec}$.
In addition, we explore the two other MeLA capabilities: influence identification and interactive learning.

For all experiments, the true model $\h$ and its parameters are hidden from all algorithms, except for an {\it oracle} model which ``cheats" by getting access to 
the true model parameters for each example, thus providing an  upper bound on performance. The performance of each algorithm is then evaluated on previously unseen test datasets. 
For all experiments, the Adam optimizer\cite{kingma2014adam} with default parameters is used for training and fine-tuning during evaluation. \footnote[1]{The code for MeLA, the dataset and experiments will be open-sourced upon publication of the paper.}

\subsection{Simple regression problem}

We first demonstrate the 3 capabilities of MeLA via the same  simple regression benchmark previously studied with MAML \cite{pmlr-v70-finn17a}, where
the hidden function class is
$h(x)=c_1\sin(x + c_2)$,
and the parameters $c_1\sim U[0.1,5.0]$, $c_2\sim U[0,\pi]$
are randomly generated for each dataset. For each dataset, 10 input points $x_i$ are sampled from $U[-5,5]$ as training examples and another 10 are sampled as testing examples. 
100 such datasets are presented for the algorithms during training. The baseline model is a 3-layer network where each hidden layer has 40 neurons with leakyReLU activation.

\begin{figure}[ht] 
  \centering
   \includegraphics[width=0.9\linewidth]{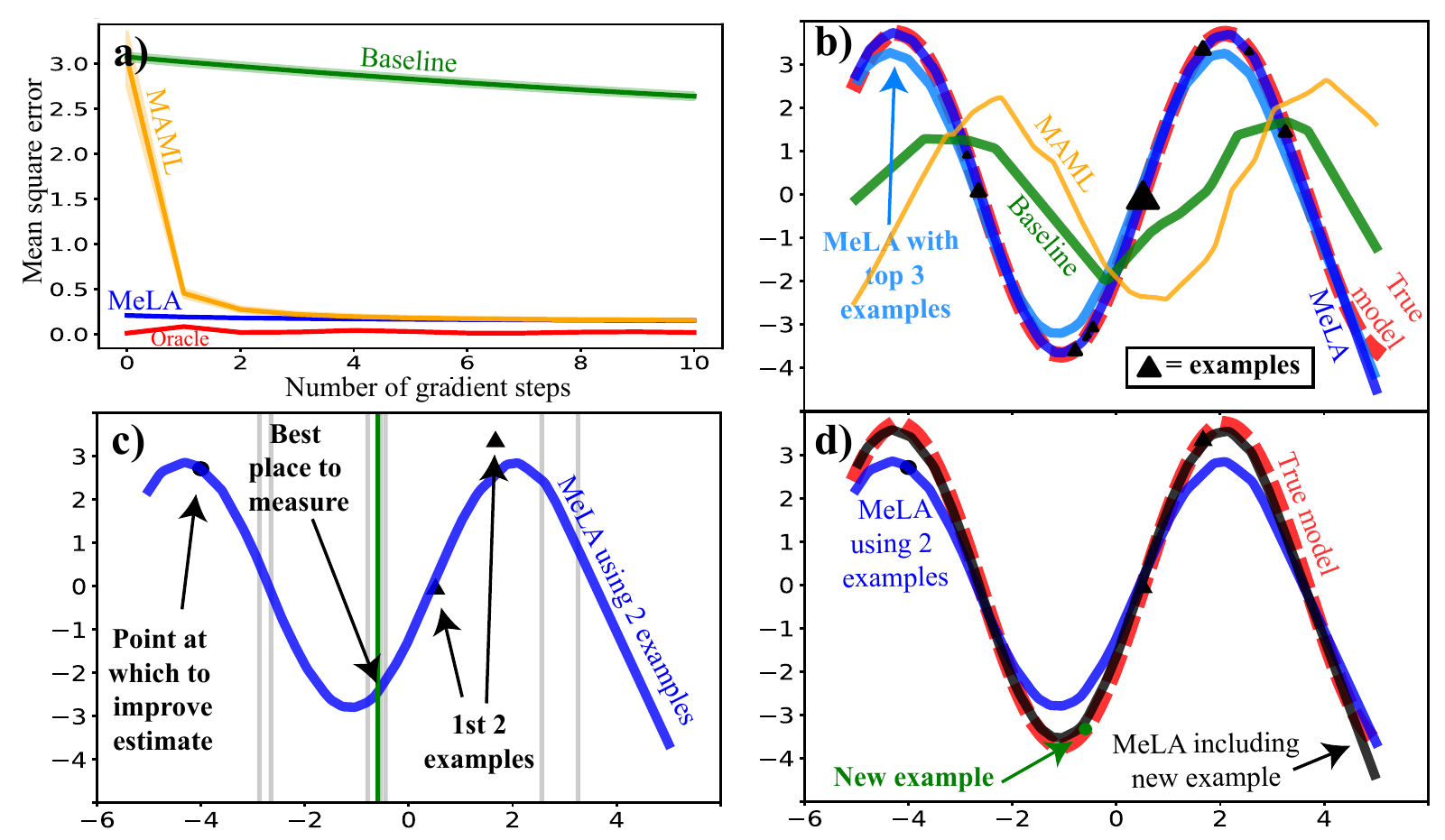} 
  \caption{(a) 
  MSE vs. number of gradient steps (with learning rate = 0.001, Adam optimizer) on 20,000 randomly sampled testing datasets, for MeLA, MAML, baseline (pretrained) and oracle. MeLA starts at MSE of 0.208 and gets down to 0.129 after 10  steps, while MAML starts at 3.05 and gets down to 0.208 after 5 steps. (b) Predictions after 0 gradient steps for an example test dataset (MAML is after 1 gradient step). The markers' size is proportional to the influence identified by MeLA. Also plotted is MeLA's prediction given only the top 3 influential examples. (c) To get a better prediction at $x^*=-4$ using only two examples at hand, MeLA requests the example at $x=-0.593$ from 8 candidate positions (vertical lines). (d) Improved estimate at $x^*=-4$ after obtaining the requested example. }
  \label{fig:sin} 
\end{figure}

The results are shown in Fig.~\ref{fig:sin}. Panel a) plots the mean squared error vs.~number of gradient steps on unseen randomly generated testing datasets, showing that MeLA outclasses the baseline model at all stages. 
It also shows that MeLA asymptotes to the same performance as MAML but learns much faster, starting with a low loss that MAML needs 5 gradient steps to surpass. 
Panel b) compares predictions with 0 gradient steps. MeLA not only proposes a model that accurately matches the true model, but also identifies each examples' influence on the model generation, and obtains good prediction if only the top 3 influential examples are given. Panels c) and d) show MeLA's capability of actively requesting informative examples by predicting which additional example will help improve the prediction the most.

\subsection{Ball bouncing with state representation}
\label{sec:bounce-states}

Next, we test MeLA's capability in simple but challenging physical
environments, where it is desirable that an algorithm quickly adapts to each new environment with few observations of states or frames. This is a much harder task than the previous simple ``Sin" regression benchmark in few-shot/meta-learning, and is another contribution to the community. Each environment, implemented as a custom Gym environment\cite{1606.01540}, consists of a room with 4 walls, whose frictionless floor is a random 4-sided convex polygon inside the  2-dimensional unit square $[0,1]\times[0,1]$ (Fig. \ref{fig:bounce_visualize}(a)), and a ball of radius 0.075
that bounces elastically off of these walls and otherwise moves with constant velocity. 
Because the different room geometries give the ball conflicting bouncing dynamics in different environments, a model trained well in one environment may not necessarily perform well in another, providing an ideal test bed for few-shot/meta-learning. During training, all models take as input 3 consecutive time steps of ball's state 
 ($x-$ and $y-$ coordinates), recorded every time it has moved a distance 0.1.
The oracle model is also given as input the coordinates of the floor's 4 corners.

\begin{figure}[t]
\centering
\begin{subfigure}{.48\textwidth}
  \centering
  \includegraphics[width=0.85\linewidth]{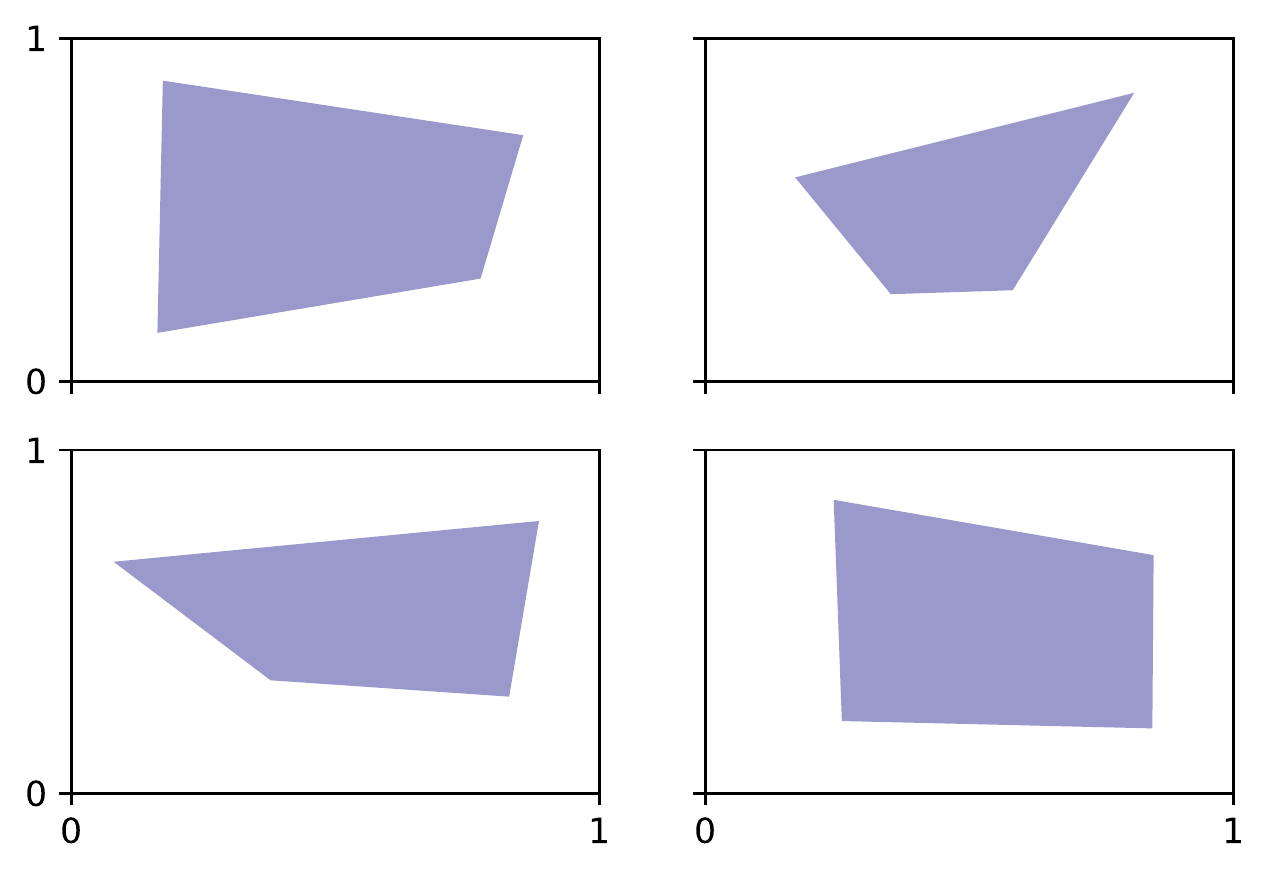}
  \caption{}
  \label{fig:bounce_visualize_1}
\end{subfigure}%
\begin{subfigure}{.52\textwidth}
  \centering
  \includegraphics[width=1.05\linewidth]{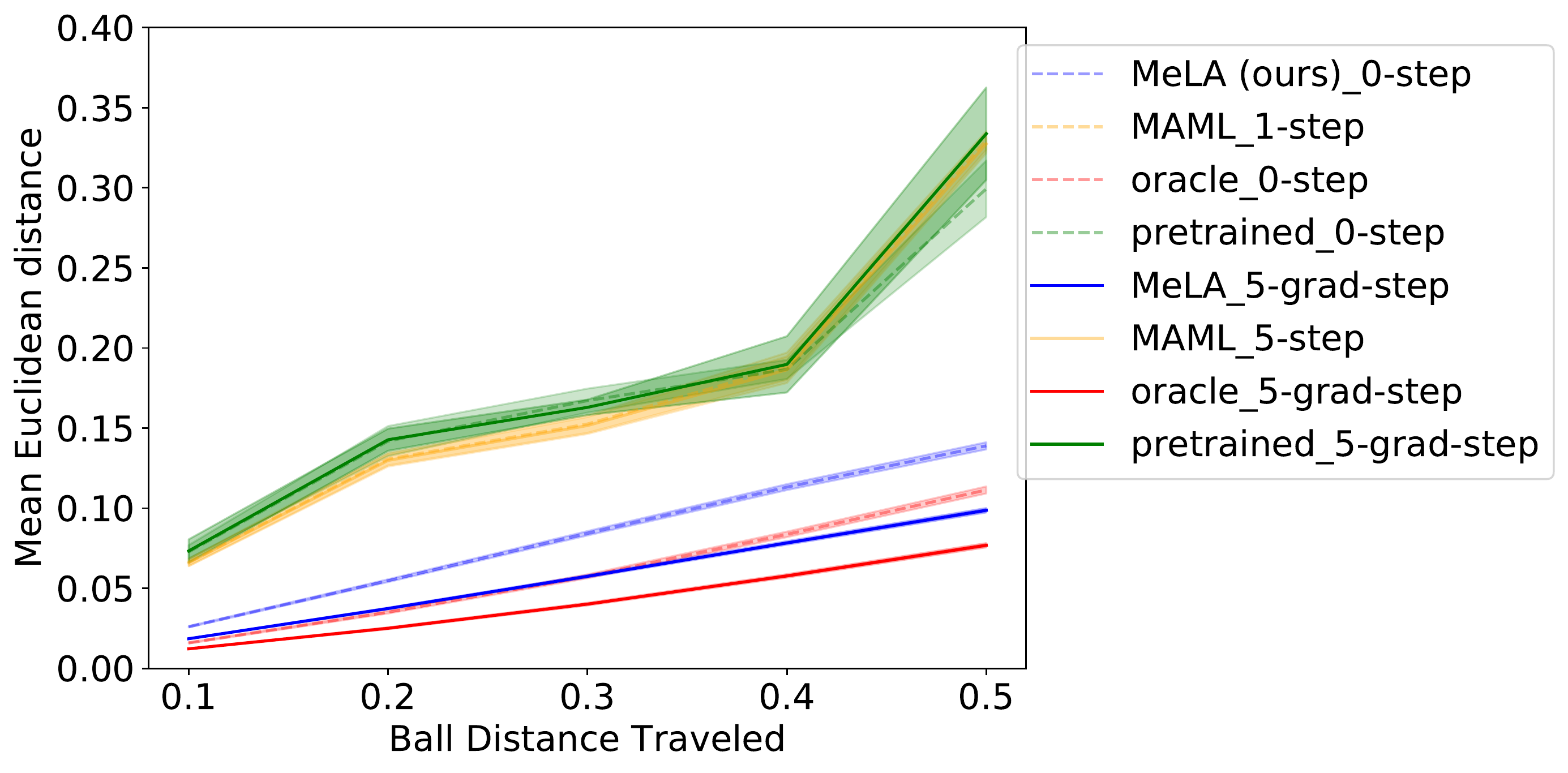}
  \caption{}
  \label{fig:bounce_visualize_2}
\end{subfigure}
\caption{(a) Examples of the polygon "bouncy-house" environments. (b) Mean Euclidean distance between target and prediction vs. rollout distance traveled 
on 1000 randomly generated testing environments.}
\label{fig:bounce_visualize}
\end{figure}

Fig. \ref{fig:bounce_visualize} (b) plots the mean Euclidean distance of the models' predictions vs. rollout distance traveled. We can see that MeLA outperforms pretrained and MAML for both 0 and 5 gradient steps. Moreover, what MeLA identifies as influential examples (Fig. \ref{fig:bounce_influential}) lies near the vertices of the polygon, showing that MeLA essentially learns to capture the convex hull of all the trajectories when proposing the model.

\begin{figure}[t]
\begin{center}
\vskip-3mm
\centerline{\includegraphics[width=0.6\columnwidth]{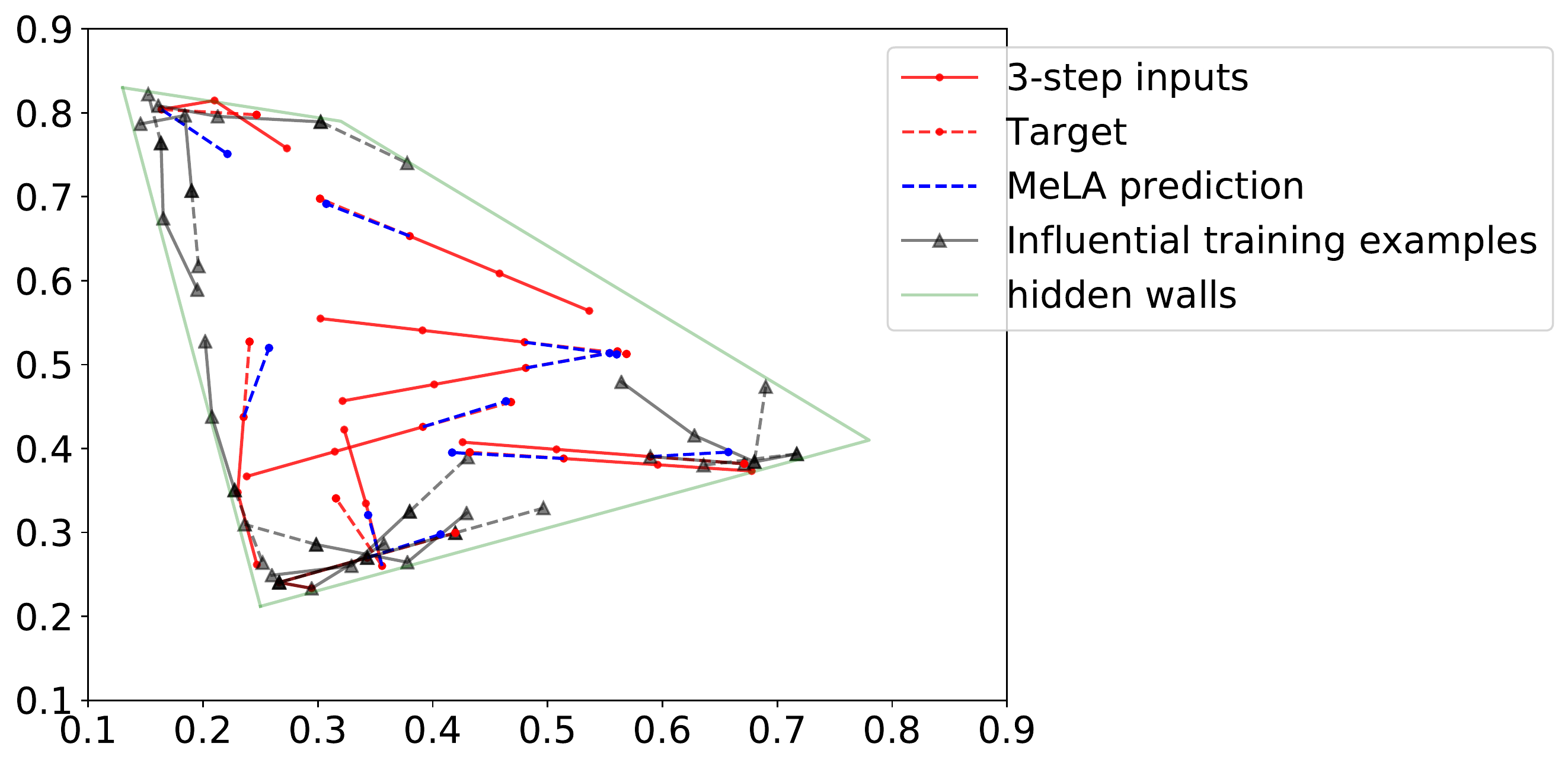}}
\vskip -0.05in
\caption{Ball bouncing prediction by MeLA for an example testing dataset. Also plotted are the top 10 most influential training  trajectories identified by MeLA, which are all near the vertices.}
\label{fig:bounce_influential}
\end{center}
\end{figure}

\subsection{Video prediction}

To test MeLA's ability to integrate into other end-to-end architectures that deal with high-dimensional inputs, we present it with an ensemble of video prediction tasks, each of which has a ball bouncing inside randomly generated polygon walls. The environment setup is the same as in section \ref{sec:bounce-states}, except that the inputs are 3 consecutive frames of 39 x 39 pixel snapshots, and the target is a 39 x 39 snapshot of the next time step. For all the models, a convolutional autoencoder is used for autoencoding the frames, and the models differ only in the latent dynamics model that predicts the future latent variable based on the 3 steps of latent variables encoded by the autoencoder. For the pretrained model, a single 4-layer network with 40 neurons in each hidden layer is used for the latent dynamics model, training on all tasks. MAML and MeLA also have/generate the same architecture for the latent dynamics model. For the oracle model, the coordinates of the vertices are concatenated with the latent variables as inputs.

Fig.~\ref{fig:bounce-images_perform} plots the mean Euclidean distance of the center of mass of the models' predictions vs. rollout distance. We see that MeLA again greatly reduces the prediction error compared to the baseline model which has to use a single model to predict the trajectory in all environments. MeLA's accuracy is seen to be near that of the oracle, demonstrating that MeLA is learning to quickly recognize and model each environment and propose reasonable models.

\begin{figure}[t]
\centering
\vskip -0.1in
\begin{subfigure}{.5\textwidth}
  \centering
  \includegraphics[width=0.98\linewidth]{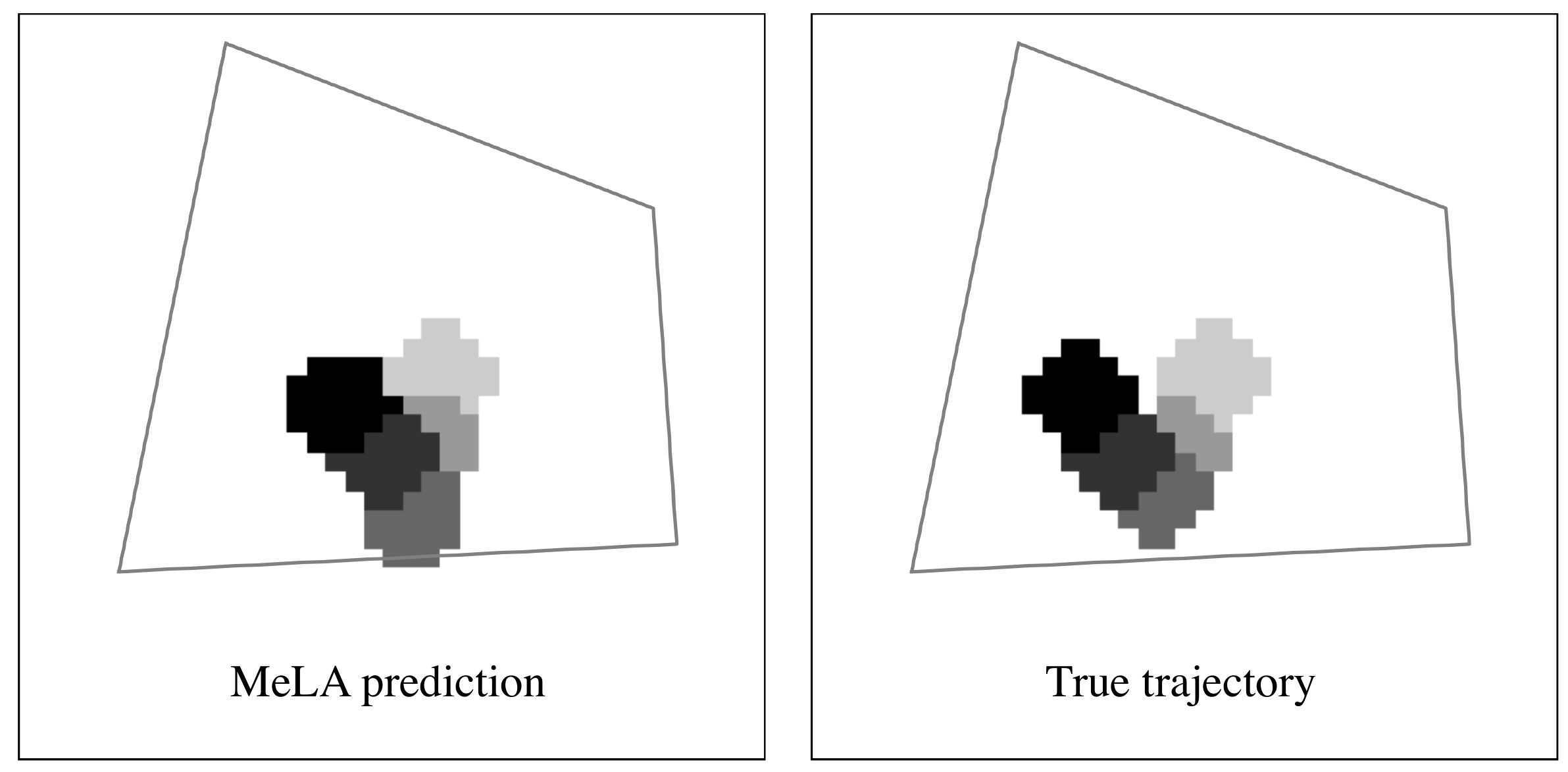}
  \caption{}
  \label{fig:sub1}
\end{subfigure}%
\begin{subfigure}{.5\textwidth}
  \centering
  \includegraphics[width=0.98\linewidth]{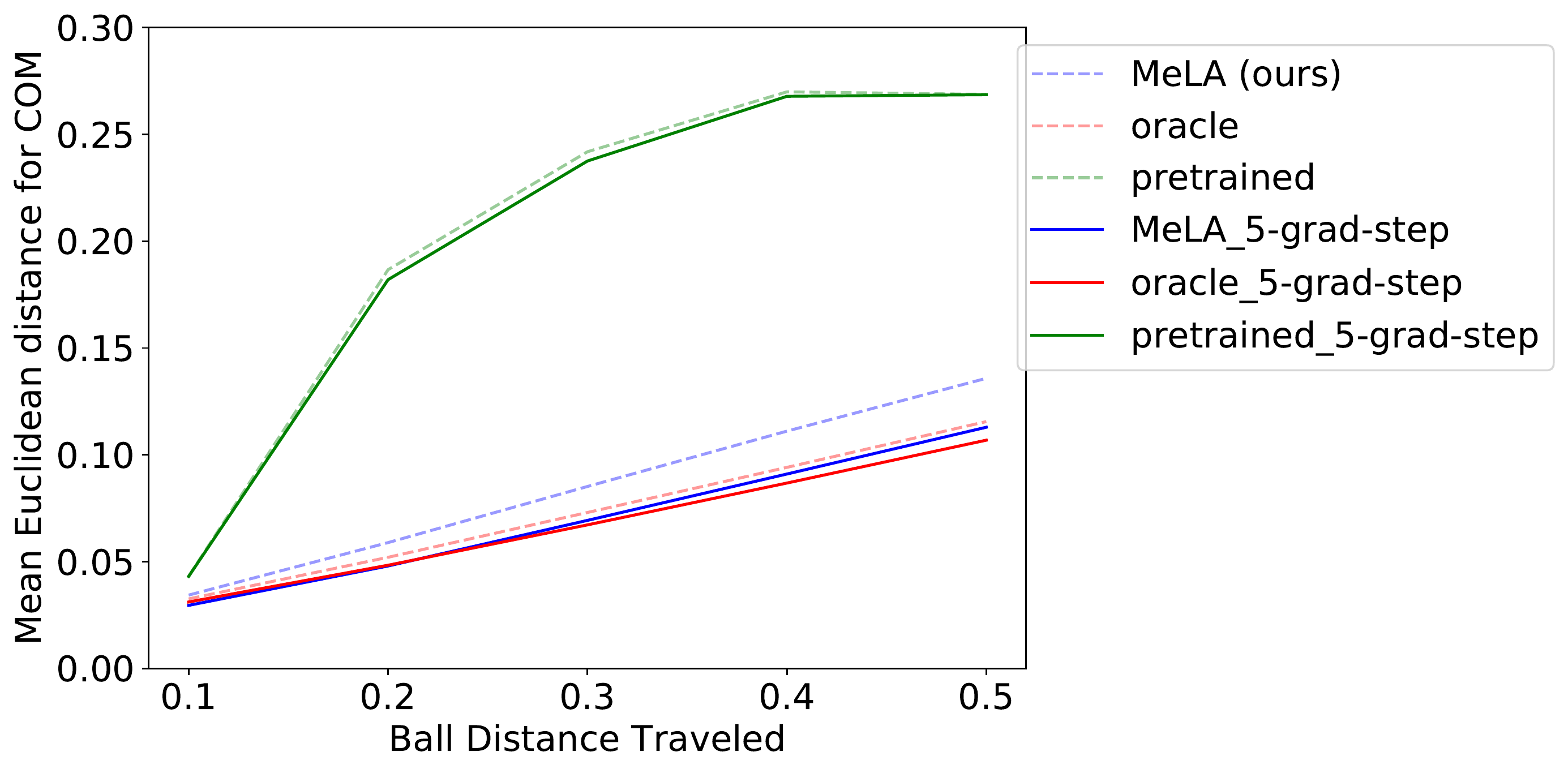}
  \caption{}
  \label{fig:bounce-images_perform}
\end{subfigure}
\caption{(a) Example MeLA prediction vs. true trajectory for 5 rollout steps, with an unseen testing environment without gradient steps. (b) Mean Euclidean distance between the center of mass (COM) of true trajectory and prediction vs. rollout distance traveled for MeLA, pretrained and oracle on 100 randomly generated testing environments.}
\label{fig:bounce_image_quicklearn}
\end{figure}

\section{Conclusions}

In this paper, we have introduced MeLA, an algorithm for rapid recognition and determination of physical models in few-shot/meta-learning. We demonstrate that MeLA can transform a model intended for single-task learning into one that can quickly adapt with a few examples to a new task. Further, we show that MeLA allows the model to improve with a few gradient steps, for fast few-shot learning. We demonstrate how MeLA learns more accurately and with fewer examples than both the original model it is based on, and the state-of-the-art meta-learning algorithm MAML. We also demonstrate two additional capabilities of MeLA: its ability to identify influential examples, and how MeLA can interactively request informative examples to optimize learning.

A core enabler of human's skill to handle novel tasks is
our ability to quickly recognize and propose models in new environments, based on previously learned physical models. We believe that by incorporating this ability, machine learning models will become more adaptive and capable for new environments and unsolved problems.



\chapter{Rank Pruning for robust learning with noisy labels}
\label{chap8:rankpruning}

$\tilde{P}\tilde{N}$ learning\footnote{The paper ``Learning with Confident Examples: Rank Pruning for Robust Classification with Noisy Labels'' is published at Conference on Uncertainty in Artificial Intelligence
(\emph{UAI} 2017) \cite{northcutt2017learning}. Authors: Curtis G. Northcutt*, Tailin Wu*, Isaac L. Chuang, where * dentes equal contributions.} is the problem of binary classification when training examples may be mislabeled (flipped) uniformly with noise rate $\rho_1$ for positive examples and $\rho_0$ for negative examples. We propose Rank Pruning (RP) to solve $\tilde{P}\tilde{N}$ learning and the open problem of estimating the noise rates, i.e. the fraction of wrong positive and negative labels. 
Unlike prior solutions, RP is time-efficient and general, requiring $\mathcal{O}(T)$ for any unrestricted choice of probabilistic classifier with $T$ fitting time. We prove RP has consistent noise estimation and equivalent expected risk as learning with uncorrupted labels in ideal conditions, and derive closed-form solutions when conditions are non-ideal. RP achieves state-of-the-art noise estimation and F1, error, and AUC-PR for both MNIST and CIFAR datasets, regardless of the amount of noise and performs similarly impressively when a large portion of training examples are noise drawn from a third distribution. To highlight, RP with a CNN classifier can predict if an MNIST digit is a \emph{one} or \emph{not} with only $0.25\%$ error, and $0.46\%$ error across all digits, even when 50\% of positive examples are mislabeled and 50\% of observed positive labels are mislabeled negative examples.

\section{Introduction}
\label{sec:intro}

Consider a student with no knowledge of animals tasked with learning to classify whether a picture contains a dog. A teacher shows the student example pictures of lone four-legged animals, stating whether the image contains a dog or not. Unfortunately, the teacher may often make mistakes, asymmetrically, with a significantly large false positive rate, $\rho_1 \in [0,1]$, and significantly large false negative rate, $\rho_0 \in [0,1]$. The teacher may also include ``white noise" images with a uniformly random label. This information is unknown to the student, who only knows of the images and corrupted labels, but suspects that the teacher may make mistakes. Can the student (1) estimate the mistake rates, $\rho_1$ and $\rho_0$, (2) learn to classify pictures with dogs accurately, and (3) do so efficiently (e.g. less than an hour for 50 images)? This allegory clarifies the challenges of $\tilde{P}\tilde{N}$ learning for any classifier trained with corrupted labels, perhaps with intermixed noise examples. We elect the notation $\tilde{P}\tilde{N}$ to emphasize that both the positive and negative sets may contain mislabeled examples, reserving $P$ and $N$ for uncorrupted sets.

This example illustrates a fundamental reliance of supervised learning on training labels \citep{Michalski:1986:MLA:21934}. Traditional learning performance degrades monotonically with label noise \citep{Aha1991, Nettleton2010}, necessitating semi-supervised approaches \citep{Blanchard:2010:SND:1756006.1953028}. Examples of noisy datasets are medical \citep{raviv_heart}, human-labeled \citep{mech_turk_quality}, and sensor \citep{lane2010survey} datasets. The problem of uncovering the same classifications as if the data was not mislabeled is our fundamental goal.

Towards this goal, we introduce Rank Pruning\footnote{ Rank Pruning is open-source and available at \href{http://bit.ly/2pbtjR6}{https://github.com/cgnorthcutt/rankpruning}}, an algorithm for $\tilde{P}\tilde{N}$ learning composed of two sequential parts: (1) estimation of the asymmetric noise rates $\rho_1$ and $\rho_0$ and (2) removal of mislabeled examples prior to training. The fundamental mantra of Rank Pruning is \emph{learning with confident examples}, i.e. examples with a predicted probability of being positive \emph{near} $1$ when the label is positive or $0$ when the label is negative. If we imagine non-confident examples as a noise class, separate from the confident positive and negative classes, then their removal should unveil a subset of the uncorrupted data. 

An ancillary mantra of Rank Pruning is \emph{removal by rank} which elegantly exploits ranking without sorting. Instead of pruning non-confident examples by predicted probability, we estimate the number of mislabeled examples in each class. We then remove the $k^{th}$-most or $k^{th}$-least examples, \emph{ranked} by predicted probability, via the BFPRT algorithm \citep{Blum:1973:TBS:1739940.1740109} in $\mathcal{O}(n)$ time, where $n$ is the number of training examples. \emph{Removal by rank} mitigates sensitivity to probability estimation and exploits the reduced complexity of learning to rank over probability estimation \citep{Menon2012PredictingAP}. Together, \emph{learning with confident examples} and \emph{removal by rank} enable robustness, i.e. invariance to erroneous input deviation.

Beyond prediction, confident examples help estimate $\rho_1$ and $\rho_0$. Typical approaches require averaging predicted probabilities on a holdout set \citep{Liu:2016:CNL:2914183.2914328, Elkan:2008:LCO:1401890.1401920} tying noise estimation to the accuracy of the predicted probabilities, which in practice may be confounded by added noise or poor model selection. Instead, we estimate $\rho_1$ and $\rho_0$ as a fraction of the predicted counts of confident examples in each class, encouraging robustness for variation in probability estimation.

\subsection{Related Work}

Rank Pruning bridges framework, nomenclature, and application across $PU$ and $\tilde{P}\tilde{N}$ learning. In this section, we consider the contributions of Rank Pruning in both.

\subsubsection{\texorpdfstring{$PU$}{PU} Learning}

Positive-unlabeled ($PU$) learning is a binary classification task in which a subset of positive training examples are labeled, and the rest are unlabeled. For example, co-training \citep{Blum:1998:CLU:279943.279962, Nigam00understandingthe} with labeled and unlabeled examples can be framed as a $PU$ learning problem by assigning all unlabeled examples the label `0'. $PU$ learning methods often assume corrupted negative labels for the unlabeled examples $U$ such that $PU$ learning is $\tilde{P}\tilde{N}$ learning with no mislabeled examples in $P$, hence their naming conventions.

Early approaches to $PU$ learning modified the loss functions via weighted logistic regression \citep{lee2003PUlearning_weightedlogreg} and biased SVM \citep{Liu:2003:BTC:951949.952139} to penalize more when positive examples are predicted incorrectly. Bagging SVM \citep{Mordelet:2014:BSL:2565612.2565683} and RESVM \citep{Claesen201573} extended biased SVM to instead use an ensemble of classifiers trained by resampling $U$ (and $P$ for RESVM) to improve robustness \citep{Breiman:1996:BP:231986.231989}. RESVM claims state-of-the-art for $PU$ learning, but is impractically inefficient for large datasets because it requires optimization of five parameters and suffers from the pitfalls of SVM model selection \citep{Chapelle:1999:MSS:3009657.3009690}. \cite{Elkan:2008:LCO:1401890.1401920} introduce a formative time-efficient probabilistic approach (denoted \emph{Elk08}) for $PU$ learning that directly estimates $1 - \rho_1$ by averaging predicted probabilities of a holdout set and dividing all predicted probabilities by $1 - \rho_1$. On the SwissProt database, \emph{Elk08} was 621 times faster than biased SVM, which only requires two parameter optimization. However, \emph{Elk08} noise rate estimation is sensitive to inexact probability estimation and both RESVM and \emph{Elk08} assume $P$ = $\tilde{P}$ and do not generalize to $\tilde{P}\tilde{N}$ learning.  Rank Pruning leverages \emph{Elk08} to initialize $\rho_1$, but then re-estimates $\rho_1$ using confident examples for both robustness (RESVM) and efficiency (\emph{Elk08}).

\begin{table*}[t]
\caption{Variable definitions and descriptions for $\tilde{P}\tilde{N}$ learning and  PU learning. Related work contains a prominent author using each variable. $\rho_1$ is also referred to as \emph{contamination} in PU learning literature.}
 \vskip -0.15in
\label{t:definitions}
\begin{center}
\begin{small}
\begin{sc}

\resizebox{\textwidth}{!}{
\begin{tabular}{lcccr}
\toprule
\abovestrut{0.13in}\belowstrut{0.05in}
\textbf{Variable} & \textbf{Conditional} & \textbf{Description} & \textbf{Domain} & \textbf{Related} Work \\
\midrule
\abovestrut{0.1in}

$\rho_0$ & $P(s=1|y=0)$ & Fraction of $N$ examples mislabeled as positive & $\tilde{P}\tilde{N}$ & Liu \\
$\rho_1$ & $P(s=0|y=1)$ & Fraction of $P$ examples mislabeled as negative & $\tilde{P}\tilde{N}$, PU & Liu, Claesen \\
$\pi_0$ & $P(y=1|s=0)$ & Fraction of mislabeled examples in $\tilde{N}$ & $\tilde{P}\tilde{N}$ & Scott \\
$\pi_1$ & $P(y=0|s=1)$ & Fraction of mislabeled examples in $\tilde{P}$ & $\tilde{P}\tilde{N}$ & Scott \\
$c = 1- \rho_1$ & $P(s=1|y=1)$ & Fraction of correctly labeled $P$ if $P(y=1|s=1) = 1$ & PU & Elkan \\

\bottomrule
\end{tabular}
}
\end{sc}
\end{small}
\end{center}
 \vskip -0.25in
\end{table*}

\begin{table*}[b]
\renewcommand{\arraystretch}{.5}
\caption{Summary of state-of-the-art and selected general solutions to $\tilde{P}\tilde{N}$ and $PU$ learning.}
\vskip -.17in
\label{t:related_works}
\begin{center}
\begin{small}
\begin{sc}
\resizebox{\textwidth}{!}{
\begin{tabular}{lcccccccc}
\toprule
\abovestrut{0.05in}\belowstrut{0.05in}
\textbf{Related Work} & \textbf{Noise} & \textbf{$\tilde{P}\tilde{N}$} & \textbf{$PU$} & \textbf{Any Prob.} & \textbf{Prob Estim.} &                                     \textbf{Time}         & \textbf{Theory} & \textbf{Added} \\
                      & \textbf{Estim.} &                               &               & \textbf{Classifier} & \textbf{Robustness} &               \textbf{Efficient} & \textbf{Support} & \textbf{Noise} \\
\midrule
\abovestrut{0.1in}

\cite{Elkan:2008:LCO:1401890.1401920} & \checkmark &            & \checkmark & \checkmark &            & \checkmark & \checkmark & \\
\midrule
\cite{Claesen201573} &            &            & \checkmark &            & \checkmark &            &   &         \\
\midrule
\cite{ScottBH13} & \checkmark &            &            & \checkmark & \checkmark &            & \checkmark &\\
\midrule
\cite{NIPS2013_5073} &            & \checkmark & \checkmark & \checkmark & \checkmark & \checkmark & \checkmark &\\
\midrule
\cite{Liu:2016:CNL:2914183.2914328} &            & \checkmark & \checkmark & \checkmark &            & \checkmark & \checkmark &\\
\midrule
\midrule
\textbf{Rank Pruning} & \checkmark & \checkmark & \checkmark & \checkmark & \checkmark & \checkmark & \checkmark & \checkmark \\

\bottomrule
\end{tabular}
}
\end{sc}
\end{small}
\end{center}
\end{table*}

\subsubsection{\texorpdfstring{$\tilde{P}\tilde{N}$}{Noisy PN} Learning}

Theoretical approaches for $\tilde{P}\tilde{N}$ learning often have two steps: (1) estimate the noise rates, $\rho_1$, $\rho_0$, and (2) use $\rho_1$, $\rho_0$ for prediction. To our knowledge, Rank Pruning is the only time-efficient solution for the open problem \mbox{\citep{Liu:2016:CNL:2914183.2914328, ICML2012Yang_127}} of noise estimation.

We first consider relevant work in noise rate estimation. \cite{ScottBH13} established a lower bound method for estimating the \emph{inversed} noise rates $\pi_1$ and $\pi_0$ (defined in Table \ref{t:definitions}). However, the method can be intractable due to unbounded convergence and assumes that the positive and negative distributions are mutually irreducible. Under additional assumptions, \cite{scott2015rate} proposed a time-efficient method for noise rate estimation, but \cite{Liu:2016:CNL:2914183.2914328} reported poor performance. \cite{Liu:2016:CNL:2914183.2914328} used the minimum predicted probabilities as the noise rates, which often yields futile estimates of min = $0$. \cite{NIPS2013_5073} provide no method for estimation and view the noise rates as parameters optimized with cross-validation, inducing a sacrificial accuracy, efficiency trade-off. In comparison, Rank Pruning noise rate estimation is time-efficient, consistent in ideal conditions, and robust to imperfect probability estimation.

\cite{NIPS2013_5073} developed two methods for prediction in the $\tilde{P}\tilde{N}$ setting which modify the loss function. The first method constructs an unbiased estimator of the loss function for the true distribution from the noisy distribution, but the estimator may be non-convex even if the original loss function is convex. If the classifier's loss function cannot be modified directly, this method requires splitting each example in two with class-conditional weights and ensuring split examples are in the same batch during optimization. For these reasons, we instead compare Rank Pruning with their second method (\emph{Nat13}), which constructs a label-dependent loss function such that for 0-1 loss, the minimizers of \emph{Nat13}'s risk and the risk for the true distribution are equivalent.

\cite{Liu:2016:CNL:2914183.2914328} generalized \emph{Elk08} to the $\tilde{P}\tilde{N}$ learning setting by modifying the loss function with per-example importance reweighting (\emph{Liu16}), but reweighting terms are derived from predicted probabilities which may be sensitive to inexact estimation. To mitigate sensitivity, \cite{Liu:2016:CNL:2914183.2914328} examine the use of density ratio estimation \citep{Sugiyama:2012:DRE:2181148}. Instead, Rank Pruning mitigates sensitivity by learning from confident examples selected by rank order, not predicted probability. For fairness of comparison across methods, we compare Rank Pruning with their probability-based approach.

Assuming perfect estimation of $\rho_1$ and $\rho_0$, we, \cite{NIPS2013_5073}, and \cite{Liu:2016:CNL:2914183.2914328} all prove that the expected risk for the modified loss function is equivalent to the expected risk for the perfectly labeled dataset. However, both \cite{NIPS2013_5073} and \cite{Liu:2016:CNL:2914183.2914328} effectively "flip" example labels in the construction of their loss function, providing no benefit for added random noise. In comparison, Rank Pruning will also remove added random noise because noise drawn from a third distribution is unlikely to appear confidently positive or negative. Table \ref{t:related_works} summarizes our comparison of $\tilde{P}\tilde{N}$ and $PU$ learning methods.

Procedural efforts have improved robustness to mislabeling in the context of machine vision \citep{Xiao2015LearningFM}, neural networks \citep{noisy_boostrapping_google}, and face recognition \citep{angelova2005pruning}. Though promising, these methods are restricted in theoretical justification and generality, motivating the need for Rank Pruning.

\subsection{Contributions}

In this paper, we describe the Rank Pruning algorithm for binary classification with imperfectly labeled training data. In particular, we:

\vskip -0.05in
\begin{itemize}
\vskip -0.05in
\itemsep.05in 
        \item Develop a robust, time-efficient, general solution for both $\tilde{P}\tilde{N}$ learning, i.e. binary classification with noisy labels, and estimation of the fraction of mislabeling in both the positive and negative training sets.
        \item Introduce the \emph{learning with confident examples} mantra as a new way to think about robust classification and estimation with mislabeled training data.
        \item Prove that under assumptions, Rank Pruning achieves perfect noise estimation and equivalent expected risk as learning with correct labels. We provide closed-form solutions when those assumptions are relaxed.
        \item Demonstrate that Rank Pruning performance generalizes across the number of training examples, feature dimension, fraction of mislabeling, and fraction of added noise examples drawn from a third distribution.
        \item Improve the state-of-the-art of $\tilde{P}\tilde{N}$ learning across F1 score, AUC-PR, and Error. In many cases, Rank Pruning achieves nearly the same F1 score as learning with correct labels when 50\% of positive examples are mislabeled and 50\% of observed positive labels are mislabeled negative examples.
    \end{itemize}

\section{Framing the \texorpdfstring{$\tilde{P}\tilde{N}$}{Noisy PN} Learning Problem }
\label{sec:framing}

\vskip -0.04in

In this section, we formalize the foundational definitions, assumptions, and goals of the $\tilde{P}\tilde{N}$ learning problem illustrated by the student-teacher motivational example.

Given $n$ observed training examples $x \in \mathcal{R}^D$ with associated observed corrupted labels $s \in \{0,1\}$ and unobserved true labels $y \in \{0,1\}$, we seek a binary classifier $f$ that estimates the mapping $x\to y$. Unfortunately, if we fit the classifier using observed $(x, s)$ pairs, we estimate the mapping $x\to s$ and obtain $g(x)=P(\hat{s}=1|x)$. 

We define the observed noisy positive and negative sets as $\tilde{P} = \{x| s = 1\}, \tilde{N} = \{x| s = 0\}$ and the unobserved true positive and negative sets as $P = \{x| y = 1\}, N = \{x| y = 0\}$. Define the hidden training data as $D=\{(x_1, y_1), (x_2, y_2), ..., (x_n, y_n)\}$, drawn i.i.d. from some true distribution $\mathcal{D}$. 
We assume that a class-conditional Classification Noise Process (CNP) \citep{angluin1988learning} maps $y$ true labels to $s$ observed labels such that each label in $P$ is flipped independently with probability $\rho_1$ and each label in $N$ is flipped independently with probability $\rho_0$ ($s \leftarrow CNP(y, \rho_1, \rho_0)$). 
The resulting observed, corrupted dataset is $D_{\rho}=\{(x_1, s_1), (x_2, s_2), ..., (x_n, s_n)\}$. Therefore, $(s \independent x) | y$ and $P(s=s|y=y,x) = P(s=s|y=y)$. In recent work, CNP is referred to as the random noise classification (RCN) noise model \citep{Liu:2016:CNL:2914183.2914328, NIPS2013_5073}.

The noise rate $\rho_1 = P(s=0|y=1)$ is the fraction of $P$ examples mislabeled as negative and the noise rate $\rho_0 = P(s=1|y=0)$ is the fraction of $N$ examples mislabeled as positive. Note that $\rho_1 + \rho_0 < 1$ is a necessary condition, otherwise more examples would be mislabeled than labeled correctly. Thus, $\rho_0 < 1- \rho_1$. We elect a subscript of ``0" to refer to the negative set and a subscript of ``1" to refer to the positive set. Additionally, let $p_{s1} = P(s=1)$ be the fraction of corrupted labels that are positive and  $p_{y1} = P(y=1)$ be the fraction of true labels that are positive. It follows that the inversed noise rates are $\pi_1 = P(y=0|s=1) =  \frac{\rho_0 (1 - p_{y1})}{p_{s1}}$ and $\pi_0 = P(y=1|s=0) =  \frac{\rho_1 p_{y1}}{(1 - p_{s1})}$. Combining these relations, given any pair in $\{(\rho_0, \rho_1),(\rho_1, \pi_1),(\rho_0, \pi_0), (\pi_0, \pi_1)\}$, the remaining two and $p_{y1}$ are known. 

We consider five levels of assumptions for $P$, $N$, and $g$: \\
\textbf{Perfect Condition}: $g$ is a ``perfect" probability estimator iff $g(x) = g^*(x)$ where $g^*(x) = P(s=1|x)$. Equivalently, let $g(x) = P(s=1|x)+\Delta g(x)$. Then $g(x)$ is ``perfect" when $\Delta g(x) = 0$ and ``imperfect" when $\Delta g(x) \neq 0$. $g$ may be imperfect due to the method of estimation or due to added uniformly randomly labeled examples drawn from a third noise distribution.\\
\textbf{Non-overlapping Condition}: $P$ and $N$ have ``non-overlapping support" if $P(y=1|x) = \indicator{y=1}$, where the indicator function $\indicator{a}$ is $1$ if the $a$ is true, else $0$. \\
\textbf{Ideal Condition}\footnote{ Eq. (\ref{eq1}) is first derived in \citep{Elkan:2008:LCO:1401890.1401920} .}: $g$ is ``ideal" when both perfect and non-overlapping conditions hold and $(s \independent x) | y$ such that

\vskip -0.25in
\begin{equation} \label{eq1}
\begin{split}
g(x) = & g^*(x) = P(s=1|x) \\ 
 = & P(s=1|y=1,x) \cdot P(y=1|x) + P(s=1|y=0,x) \cdot P(y=0|x) \\
 = & (1 - \rho_1) \cdot \mathbbm{1}[[y=1]] + \rho_0 \cdot \mathbbm{1}[[y=0]]
\end{split}
\end{equation}
\vskip -0.1in

\textbf{Range Separability Condition} $g$ range separates $P$ and $N$ iff $\forall x_1\in P$ and $\forall x_2 \in N$, we have $g(x_1) >g(x_2)$. \\
\textbf{Unassuming Condition}: $g$ is ``unassuming" when perfect and/or non-overlapping conditions may not be true. 

Their relationship is:
$\textbf{Unassuming}\supset\textbf{Range}$ $\textbf{Separability}$ $\supset\textbf{Ideal}=\textbf{Perfect}\cap\textbf{Non-overlapping}$. 

We can now state the two goals of Rank Pruning for $\tilde{P}\tilde{N}$ learning. \textbf{Goal 1} is to perfectly estimate $\hat{\rho}_1 \estimates \rho_1$ and $\hat{\rho}_0  \estimates \rho_0$ when $g$ is ideal. When $g$ is not ideal, to our knowledge perfect estimation of $\rho_1$ and $\rho_0$ is impossible and at best \textbf{Goal 1} is to provide exact expressions for $\hat{\rho}_1$ and $\hat{\rho}_0$ w.r.t. $\rho_1$ and $\rho_0$. \textbf{Goal 2} is to use $\hat{\rho}_1$ and $\hat{\rho}_0$ to uncover the classifications of $f$ from $g$. Both tasks must be accomplished given only observed ($x,s$) pairs. $y, \rho_1, \rho_0, \pi_1$, and $\pi_0$ are hidden.

\section{Rank Pruning}\label{methodology}

We develop the Rank Pruning algorithm to address our two goals. In Section \ref{method:noise}, we propose a method for noise rate estimation and prove consistency when $g$ is ideal. An estimator is ``consistent" if it achieves perfect estimation in the expectation of infinite examples. In Section \ref{method:noise_non_ideal}, we derive exact expressions for $\hat{\rho}_1$ and $\hat{\rho}_0$ when $g$ is unassuming. In Section \ref{method:rp}, we provide the entire algorithm, and in Section \ref{method:risk}, prove that Rank Pruning has equivalent expected risk as learning with uncorrupted labels for both ideal $g$ and non-ideal $g$ with weaker assumptions. Throughout, we assume $n \rightarrow \infty$ so that $P$ and $N$ are the hidden distributions, each with infinite examples. This is a necessary condition for Theorems. \ref{thm:noise_ideal}, \ref{rho_conf_robust} and Lemmas \ref{thm:lemma1}, \ref{lemma3}.

\subsection{Deriving Noise Rate Estimators \texorpdfstring{$\hat{\rho}_1^{conf}$ and $\hat{\rho}_0^{conf}$}{}} \label{method:noise}

We propose the \emph{confident counts} estimators $\hat{\rho}_1^{conf}$ and $\hat{\rho}_0^{conf}$ to estimate $\rho_1$ and $\rho_0$ as a fraction of the predicted counts of confident examples in each class, encouraging robustness for variation in probability estimation. To estimate $\rho_1=P(s=0|y=1)$, we count the number of examples with label $s=0$ that we are ``confident" have label $y=1$ and divide it by the total number of examples that we are ``confident" have label $y = 1$. More formally, 

\vskip -0.2in
\begin{equation}\label{define_rho_conf}
\hat{\rho}_1^{conf}:=\frac{|\tilde{N}_{y=1}|}{|\tilde{N}_{y=1}|+|\tilde{P}_{y=1}|}, \quad
\hat{\rho}_0^{conf}:=\frac{|\tilde{P}_{y=0}|}{|\tilde{P}_{y=0}|+|\tilde{N}_{y=0}|}
\end{equation}
\vskip -0.1in

such that

\vskip -0.05in
\begin{equation}\label{define_threshold}
\begin{cases}
\tilde{P}_{y=1}=\{x\in\tilde{P}\mid g(x)\geq LB_{y=1}\}\\
\tilde{N}_{y=1}=\{x\in\tilde{N}\mid g(x)\geq LB_{y=1}\}\\
\tilde{P}_{y=0}=\{x\in\tilde{P}\mid g(x)\leq UB_{y=0}\}\\
\tilde{N}_{y=0}=\{x\in\tilde{N}\mid g(x)\leq UB_{y=0}\}
\end{cases}
\end{equation}
\vskip -0.1in

where $g$ is fit to the corrupted training set $D_{\rho}$ to obtain $g(x)=P(\hat{s}=1|x)$. The threshold $LB_{y=1}$ is the predicted probability in $g(x)$ above which we guess that an example $x$ has hidden label $y = 1$, and similarly for upper bound $UB_{y=0}$. $LB_{y=1}$ and $UB_{y=0}$ partition $\tilde{P}$ and $\tilde{N}$ into four sets representing a \emph{best guess} of a \emph{subset} of examples having labels (1) $s=1, y=0$, (2) $s=1, y=1$, (3) $s=0, y=0$, (4) $s=0, y=1$. The threshold values are defined as

\vskip -0.15in
\begin{equation*}
\begin{cases}
LB_{y=1}:=P(\hat{s}=1\mid s=1)=E_{x\in\tilde{P}}[g(x)]\\
UB_{y=0}:=P(\hat{s}=1\mid s=0)=E_{x\in\tilde{N}}[g(x)]\\
\end{cases}
\end{equation*}
\vskip -0.1in

where $\hat{s}$ is the predicted label from a classifier fit to the observed data. $|\tilde{P}_{y=1}|$ counts examples with label $s=1$ that are \emph{most} likely to be correctly labeled ($y=1$) because $LB_{y=1} = P(\hat{s}=1|s=1)$. The three other terms in Eq. (\ref{define_threshold}) follow similar reasoning. Importantly, the four terms do not sum to $n$, i.e. $|N| + |P|$, but $\hat{\rho}_1^{conf}$ and $\hat{\rho}_0^{conf}$ are valid estimates because mislabeling noise is assumed to be uniformly random. 
The choice of threshold values relies on the following two important equations:

\vskip -0.25in
\begin{align*}
LB_{y=1}=&E_{x\in\tilde{P}}[g(x)]=E_{x\in\tilde{P}}[P(s=1|x)]\\
=&E_{x\in\tilde{P}}[P(s=1|x,y=1)P(y=1|x)+P(s=1|x,y=0)P(y=0|x)]\\
=&E_{x\in\tilde{P}}[P(s=1|y=1)P(y=1|x)+P(s=1|y=0)P(y=0|x)]\\
=&(1-\rho_1)(1-\pi_1)+\rho_0\pi_1\numberthis \label{rh1_prob_eq}
\end{align*}
\vskip -0.05in

Similarly, we have

\vskip -0.15in
\begin{equation}\label{rh0_prob_eq}
UB_{y=0}=(1-\rho_1)\pi_0+\rho_0(1-\pi_0)
\end{equation}
\vskip -0.05in

To our knowledge, although simple, this is the first time that the relationship in Eq. (\ref{rh1_prob_eq}) (\ref{rh0_prob_eq}) has been published, linking the work of \cite{Elkan:2008:LCO:1401890.1401920}, \cite{Liu:2016:CNL:2914183.2914328}, \cite{ScottBH13} and \cite{NIPS2013_5073}. From Eq. (\ref{rh1_prob_eq}) (\ref{rh0_prob_eq}), we observe that $LB_{y=1}$ and $UB_{y=0}$ are linear interpolations of $1-\rho_1$ and $\rho_0$ and since $\rho_0 < 1- \rho_1$, we have that $\rho_0 < LB_{y=1} \leq 1-\rho_1$ and $\rho_0 \leq UB_{y=0} < 1-\rho_1$. When $g$ is ideal we have that $g(x)=(1-\rho_1)$, if $x \in P$ and $g(x)=\rho_0$, if $x \in N$. Thus  when $g$ is ideal, the thresholds $LB_{y=1}$ and  $UB_{y=0}$ in Eq. (\ref{define_threshold}) will perfectly separate $P$ and $N$ examples within each of $\tilde{P}$ and $\tilde{N}$. Lemma \ref{thm:lemma1} immediately follows.

\begin{lemma}\label{thm:lemma1}
\vskip -.1in
When $g$ is ideal,
\vskip -0.25in

\begin{align*}
\tilde{P}_{y=1} &= \{x\in P\mid s=1\},\quad
\tilde{N}_{y=1} = \{x\in P\mid s=0\},\\
\tilde{P}_{y=0} &= \{x\in N\mid s=1\},\quad
\tilde{N}_{y=0} = \{x\in N\mid s=0\}\numberthis
\end{align*}
\vskip -0.3in
\end{lemma}
\vskip -0.15in

Thus, when $g$ is ideal, the thresholds in Eq. (\ref{define_threshold}) partition the training set such that $\tilde{P}_{y=1}$ and $\tilde{N}_{y=0}$ contain the correctly labeled examples and $\tilde{P}_{y=0}$ and $\tilde{N}_{y=1}$ contain the mislabeled examples. Theorem \ref{thm:noise_ideal} follows (for brevity, proofs of all theorems/lemmas are in Appendix \ref{sec:A1}-\ref{sec:A5}).

\begin{theorem} \label{thm:noise_ideal}

When $g$ is ideal,
\vskip -0.15in
\begin{equation}
\hat{\rho}_1^{conf}=\rho_1,\quad\hat{\rho}_0^{conf}=\rho_0
\end{equation}
\end{theorem}

Thus, when $g$ is ideal, the \emph{confident counts} estimators $\hat{\rho}_1^{conf}$ and $\hat{\rho}_0^{conf}$ are consistent estimators for $\rho_1$ and $\rho_0$ and we set $\hat{\rho}_1:=\hat{\rho}_1^{conf}, \hat{\rho}_0:=\hat{\rho}_0^{conf}$. These steps comprise Rank Pruning noise rate estimation (see Alg. \ref{alg:rp}). There are two practical observations. First, for any $g$ with $T$ fitting time, computing $\hat{\rho}_1^{conf}$ and $\hat{\rho}_0^{conf}$ is $\mathcal{O}(T)$. Second, $\hat{\rho}_1$ and $\hat{\rho}_0$ should be estimated out-of-sample to avoid over-fitting, resulting in sample variations. In our experiments, we use 3-fold cross-validation, requiring at most $2T=\mathcal{O}(T)$. 

\subsection{Noise Estimation: Unassuming Case} \label{method:noise_non_ideal}
 
Theorem \ref{thm:noise_ideal} states that $\hat{\rho}_i^{conf} = \rho_i$, $\forall i \in \{0,1\}$ when $g$ is ideal. Though theoretically constructive, in practice this is unlikely. Next, we derive expressions for the estimators when $g$ is unassuming, i.e. $g$ may not be perfect and $P$ and $N$ may have overlapping support.

Define $\Delta p_o:=\frac{|P \cap N|}{|P \cup N|}$ as the fraction of overlapping examples in $\mathcal{D}$ and remember that $\Delta g(x) := g(x)-g^*(x)$. Denote $ LB_{y=1}^*=(1-\rho_1)(1-\pi_1)+\rho_0\pi_1, UB_{y=0}^*=(1-\rho_1)\pi_0+\rho_0(1-\pi_0)$. We have

\begin{lemma}\label{lemma3}
When $g$ is unassuming, we have

\vskip -0.1in
\begin{equation}\label{rho_imperfect_condition}
\begin{cases}
LB_{y=1}=LB_{y=1}^*+E_{x\in\tilde{P}}[\Delta g(x)]-\frac{(1-\rho_1-\rho_0)^2}{p_{s1}}\Delta p_o\\
UB_{y=0}=UB_{y=0}^*+E_{x\in\tilde{N}}[\Delta g(x)]+\frac{(1-\rho_1-\rho_0)^2}{1-p_{s1}}\Delta p_o\\
\hat{\rho}_1^{conf}=\rho_1+\frac{1-\rho_1-\rho_0}{|P|-|\Delta P_1| + |\Delta N_1|}|\Delta N_1|\\
\hat{\rho}_0^{conf}=\rho_0+\frac{1-\rho_1-\rho_0}{|N|-|\Delta N_0| + |\Delta P_0|}|\Delta P_0|
\end{cases}
\end{equation}
\vskip -0.1in

where

\vskip -0.15in
\begin{equation*}
\begin{cases}
\Delta P_1=\{x\in P\mid g(x)< LB_{y=1}\}\\
\Delta N_1=\{x\in N\mid g(x)\geq LB_{y=1}\}\\
\Delta P_0=\{x\in P\mid g(x)\leq UB_{y=0}\}\\
\Delta N_0=\{x\in N\mid g(x)> UB_{y=0}\}\\
\end{cases}
\end{equation*}
\vskip -0.35in
\end{lemma}
\vskip -0.1in

The second term on the R.H.S. of the $\hat{\rho}_i^{conf}$ expressions captures the deviation of $\hat{\rho}_i^{conf}$ from $\rho_i$, $i=0,1$. This term results from both imperfect $g(x)$ and overlapping support. Because the term is non-negative, $\hat{\rho}_i^{conf} \geq \rho_i$, $i=0,1$ in the limit of infinite examples. In other words, $\hat{\rho}_i^{conf}$ is an \emph{upper bound} for the noise rates $\rho_i$, $i=0,1$. From Lemma \ref{lemma3}, it also follows:

\begin{theorem}\label{rho_conf_robust}
Given non-overlapping support condition,

\text{If} $\forall x\in N, \Delta g(x)<LB_{y=1}-\rho_0$, then $\hat{\rho}_1^{conf}=\rho_1$.

If $\forall x\in P, \Delta g(x)>-(1-\rho_1-UB_{y=0}$), then $\hat{\rho}_0^{conf}=\rho_0$.
\end{theorem}

Theorem \ref{rho_conf_robust} shows that $\hat{\rho}_1^{conf}$ and $\hat{\rho}_0^{conf}$ are robust to imperfect probability estimation. As long as $\Delta g(x)$ does not exceed the distance between the threshold in Eq. (\ref{define_threshold}) and the perfect $g^*(x)$ value, $\hat{\rho}_1^{conf}$ and $\hat{\rho}_0^{conf}$ are consistent estimators for $\rho_1$ and $\rho_0$. Our numerical experiments in Section \ref{sec:experimental} suggest this is reasonable for $\Delta g(x)$. The average $|\Delta g(x)|$ for the MNIST training dataset across different ($\rho_1$, $\pi_1$) varies between 0.01 and 0.08 for a logistic regression classifier, 0.01$\sim$0.03 for a CNN classifier, and 0.05$\sim$0.10 for the CIFAR dataset with a CNN classifier. Thus, when  $LB_{y=1}-\rho_0$ and $1-\rho_1-UB_{y=0}$ are above 0.1 for these datasets, from Theorem  \ref{rho_conf_robust} we see that $\hat{\rho}_i^{conf}$ still accurately estimates $\rho_i$.

\subsection{The Rank Pruning Algorithm} \label{method:rp}

Using $\hat{\rho}_1$ and $\hat{\rho}_0$, we must uncover the classifications of $f$ from $g$. In this section, we describe how Rank Pruning selects confident examples, removes the rest, and trains on the pruned set using a reweighted loss function. 

First, we obtain the inverse noise rates $\hat{\pi}_1$, $\hat{\pi}_0$ from $\hat{\rho}_1$, $\hat{\rho}_0$:

\vskip -0.13in
\begin{equation}
\hat{\pi}_1=\frac{\hat{\rho}_0}{p_{s1}}\frac{1-p_{s1}-\hat{\rho}_1}{1-\hat{\rho}_1-\hat{\rho}_0},\quad
\hat{\pi}_0=\frac{\hat{\rho}_1}{1-p_{s1}}\frac{p_{s1}-\hat{\rho}_0}{1-\hat{\rho}_1-\hat{\rho}_0}
\end{equation}
\vskip -0.07in

Next, we prune the $\hat{\pi}_1|\tilde{P}|$ examples in $\tilde{P}$ with smallest $g(x)$ and the $\hat{\pi}_0|\tilde{N}|$ examples in $\tilde{N}$ with highest $g(x)$ and denote the pruned sets $\tilde{P}_{conf}$ and $\tilde{N}_{conf}$. To prune, we define $k_1$ as the $(\hat{\pi}_1|\tilde{P}|)^{th}$ smallest $g(x)$ for $x \in \tilde{P}$ and $k_0$ as the $(\hat{\pi}_0|\tilde{N}|)^{th}$ largest $g(x)$ for $x \in \tilde{N}$. BFPRT ($\mathcal{O}(n)$) \citep{Blum:1973:TBS:1739940.1740109} is used to compute $k_1$ and $k_0$ and pruning is reduced to the following $\mathcal{O}(n)$ filter:

\vskip -0.13in
\begin{equation}
\tilde{P}_{conf} := \{x \in \tilde{P} \mid g(x) \geq k_1 \},\quad 
\tilde{N}_{conf} := \{x \in \tilde{N} \mid g(x) \leq k_0 \}
\end{equation}
\vskip -0.05in

Lastly, we refit the classifier to $X_{conf} = \tilde{P}_{conf} \cup \tilde{N}_{conf}$ by class-conditionally reweighting the loss function for examples in $\tilde{P}_{conf}$ with weight $\frac{1}{1-\hat{\rho}_1}$ and examples in $\tilde{N}_{conf}$ with weight $\frac{1}{1-\hat{\rho}_0}$ to recover the estimated balance of positive and negative examples. The entire Rank Pruning algorithm is presented in Alg. \ref{alg:rp} and illustrated step-by-step on a synthetic dataset in Fig. \ref{rankpruning_illustration}.

We conclude this section with a formal discussion of the loss function and efficiency of Rank Pruning.  Define $\hat{y_i}$ as the predicted label of example $i$ for the classifier fit to $X_{conf}, s_{conf}$ and let $l(\hat{y_i}, s_i)$ be the original loss function for $x_i \in D_\rho$. Then the loss function for Rank Pruning is simply the original loss function exerted on the pruned $X_{conf}$, with class-conditional weighting:

\vskip -0.25in
\begin{align*}  \label{rank_pruning_loss_function}
\tilde{l}(\hat{y_i}, s_i)=&\frac{1}{1-\hat{\rho}_1}l(\hat{y_i}, s_i)\cdot\indicator{x_i\in \tilde{P}_{conf}}+ \frac{1}{1-\hat{\rho}_0}l(\hat{y_i}, s_i)\cdot\indicator{x_i\in \tilde{N}_{conf}}\numberthis 
\end{align*}

\medskip

Effectively this loss function uses a zero-weight for pruned examples. Other than potentially fewer examples, the only difference in the loss function for Rank Pruning and the original loss function is the class-conditional weights. These constant factors do not increase the complexity of the minimization of the original loss function. In other words, we can fairly report the running time of Rank Pruning in terms of the running time ($\mathcal{O}(T)$) of the choice of probabilistic estimator. Combining noise estimation ($\mathcal{O}(T)$), pruning ($\mathcal{O}(n)$), and the final fitting ($\mathcal{O}(T)$), Rank Pruning has a running time of $\mathcal{O}(T) + \mathcal{O}(n)$, which is $\mathcal{O}(T)$ for typical classifiers.

\subsection{Rank Pruning: A simple summary}

Recognizing that formalization can create obfuscation, in this section we describe the entire algorithm in a few sentences. Rank Pruning takes as input training examples $X$, noisy labels $s$, and a probabilistic classifier $clf$ and finds a subset of $X, s$ that is likely to be correctly labeled, i.e. a subset of $X, y$. To do this, we first find two thresholds, $LB_{y=1}$ and $UB_{y=0}$, to \emph{confidently} guess the correctly and incorrectly labeled examples in each of $\tilde{P}$ and $\tilde{N}$, forming four sets, then use the set sizes to estimate the noise rates $\rho_1 = P(s = 0 | y = 1)$ and $\rho_0 = P(s = 1 | y = 0)$. We then use the noise rates to estimate the number of examples with observed label $s = 1$ and hidden label $y = 0$ and remove that number of examples from $\tilde{P}$ by removing those with lowest predicted probability $g(x)$. We prune $\tilde{N}$ similarly. Finally, the classifier is fit to the pruned set, which is intended to represent a subset of the correctly labeled data.

\begin{algorithm}[t]
   \caption{\textbf{Rank Pruning}}
   \label{alg:rp}
\begin{algorithmic}
   \STATE {\bfseries Input:} Examples $X$, corrupted labels $s$, classifier clf
   \STATE \textbf{Part 1. Estimating Noise Rates:}
   \STATE (1.1)\ \ clf.fit($X$,$s$)
   \STATE \ \ \ \ \ \ \ \ \ $g(x)\leftarrow$clf.predict\_crossval\_probability($\hat{s}=1|x$)
   \STATE \ \ \ \ \ \  \ \ \ $p_{s1}=\frac{\text{count}(s=1)}{\text{count}(s=0 \lor s=1)}$
   \STATE \ \ \ \ \ \  \ \ \ $LB_{y=1}=E_{x\in\tilde{P}}[g(x)]$, $UB_{y=0}=E_{x\in\tilde{N}}[g(x)]$
   \STATE (1.2)\  $\hat{\rho}_1=\hat{\rho}_1^{conf}=\frac{|\tilde{N}_{y=1}|}{|\tilde{N}_{y=1}|+|\tilde{P}_{y=1}|}$, $\hat{\rho}_0=\hat{\rho}_0^{conf}=\frac{|\tilde{P}_{y=0}|}{|\tilde{P}_{y=0}|+|\tilde{N}_{y=0}|}$
   \STATE \ \ \ \ \ \ \ \ \ $\hat{\pi}_1=\frac{\hat{\rho}_0}{p_{s1}}\frac{1-p_{s1}-\hat{\rho}_1}{1-\hat{\rho}_1-\hat{\rho}_0}$, $\hat{\pi}_0=\frac{\hat{\rho}_1}{1-p_{s1}}\frac{p_{s1}-\hat{\rho}_0}{1-\hat{\rho}_1-\hat{\rho}_0}$
   \vskip .1in
   \STATE \textbf{Part 2. Prune Inconsistent Examples:}
   \STATE (2.1) Remove $\hat{\pi}_1|\tilde{P}|$ examples in $\tilde{P}$ with least $g(x)$, Remove $\hat{\pi}_0|\tilde{N}|$ examples in $\tilde{N}$ with greatest $g(x)$,
   \STATE \ \ \ \ \ \ \ \ \ Denote the remaining training set ($X_{conf}$, $s_{conf}$)
   \STATE (2.2) clf.fit($X_{conf}$, $s_{conf}$), with sample weight 
   $w(x)=\frac{1}{1-\hat{\rho}_1}\indicator{s_{conf}=1}$+$\frac{1}{1-\hat{\rho}_0}\indicator{s_{conf}=0}$
   \STATE {\bfseries Output:} clf
\end{algorithmic}
\end{algorithm}

\subsection{Expected Risk Evaluation} 
\label{method:risk}

In this section, we prove Rank Pruning exactly uncovers the classifier $f$ fit to hidden $y$ labels when $g$ range separates $P$ and $N$ and $\rho_1$ and $\rho_0$ are given.

Denote $f_{\theta}\in\mathcal{F}: x\to \hat{y}$ as a classifier's prediction function belonging to some function space $\mathcal{F}$, where $\theta$ represents the classifier's parameters. $f_{\theta}$ represents $f$, but without $\theta$ necessarily fit to the training data. $\hat{f}$ is the Rank Pruning estimate of $f$.

Denote the empirical risk of $f_{\theta}$ w.r.t. the loss function $\tilde{l}$ and corrupted data $D_{\rho}$ as $\hat{R}_{\tilde{l}, D_{\rho}}(f_{\theta})=\frac{1}{n}\sum_{i=1}^n\tilde{l}(f_{\theta}(x_i), s_i)$, and the expected risk of $f_{\theta}$ w.r.t. the corrupted distribution $\mathcal{D}_{\rho}$ as $R_{\tilde{l},\mathcal{D}_{\rho}}(f_{\theta})=E_{(x,s)\sim \mathcal{D}_{\rho}}[\hat{R}_{\tilde{l},\mathcal{D}_{\rho}}(f_{\theta})]$. Similarly, denote $R_{l,\mathcal{D}}(f_{\theta})$ as the expected risk of $f_{\theta}$ w.r.t. the hidden distribution $\mathcal{D}$ and loss function $l$. We show that using Rank Pruning, a classifier $\hat{f}$ can be learned for the hidden data $D$, given the corrupted data $D_{\rho}$, by minimizing the empirical risk:

\vskip -0.12in
\begin{equation}\label{risk_minimization}
\hat{f}=\argmin\limits_{f_{\theta}\in \mathcal{F}} \hat{R}_{\tilde{l},D_{\rho}}(f_{\theta})=\argmin\limits_{f_{\theta}\in \mathcal{F}} \frac{1}{n}\sum_{i=1}^n\tilde{l}(f_{\theta}(x_i), s_i)
\end{equation}

Under the \emph{range separability} condition, we have

\begin{figure}[t]
\begin{center}
\centerline{\includegraphics[width=1.0\columnwidth]{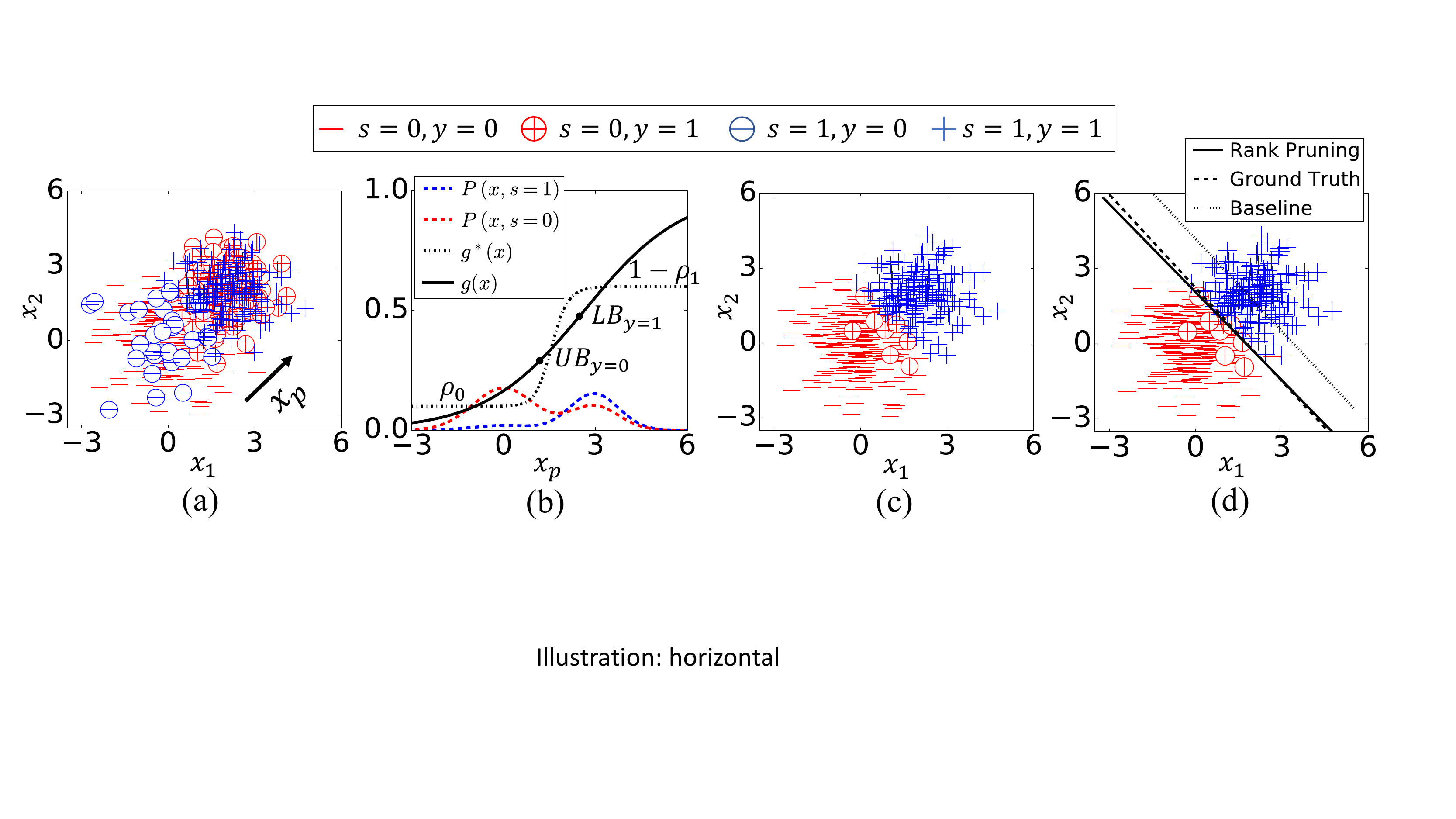}}
\caption{Illustration of Rank Pruning with a logistic regression classifier ($\mathcal{LR}_\theta$).  \textbf{(a)}: The corrupted training set $D_\rho$ with noise rates $\rho_1=0.4$ and $\rho_0=0.1$. Corrupted colored labels (\textcolor{blue}{$s=1$},\textcolor{red}{$s=0$}) are observed. $y$ ($+$,$-$) is hidden. \textbf{(b)}: The marginal distribution of $D_\rho$ projected onto the $x_p$ axis (indicated in (a)), and the $\mathcal{LR}_\theta$'s estimated $g(x)$, from which $\hat{\rho}_1^{conf}=0.4237$, $\hat{\rho}_0^{conf}=0.1144$ are estimated. \textbf{(c)}: The pruned $X_{conf}, s_{conf}$. \textbf{(d)}: The classification result by Rank Pruning ($\hat{f}$ = $\mathcal{LR}_\theta.\text{fit}(X_{conf}, s_{conf})$),
ground truth classifier ($f$ = $\mathcal{LR}_\theta.\text{fit}(X,y)$),
and baseline classifier ($g$ = $\mathcal{LR}_\theta.\text{fit}(X, s)$), with an accuracy of $94.16\%$, $94.16\%$ and $78.83\%$, respectively.}
\label{rankpruning_illustration}
\end{center}
\vskip -0.15in
\end{figure}

\begin{theorem}\label{Rank Pruning_with_range_separability}
If $g$ range separates $P$ and $N$ and $\hat{\rho}_i=\rho_i$, $i=0,1$, then for any classifier $f_{\theta}$ and any bounded loss function $l(\hat{y}_i,y_i)$, we have

\vskip -0.12in
\begin{equation}
R_{\tilde{l},\mathcal{D}_{\rho}}(f_{\theta})=R_{l,\mathcal{D}}(f_{\theta})
\end{equation}
\vskip -0.06in

where $\tilde{l}(\hat{y}_i,s_i)$ is Rank Pruning's loss function (Eq. \ref{rank_pruning_loss_function}).
\end{theorem}

The proof of Theorem \ref{Rank Pruning_with_range_separability} is in Appendix \ref{sec:A5}. Intuitively, Theorem \ref{Rank Pruning_with_range_separability} tells us that if $g$ range separates $P$ and $N$, then given exact noise rate estimates, Rank Pruning will exactly prune out the positive examples in $\tilde{N}$ and negative examples in $\tilde{P}$, leading to the same expected risk as learning from uncorrupted labels. Thus, Rank Pruning can exactly uncover the classifications of $f$ (with infinite examples) because the expected risk is equivalent for any $f_{\theta}$. Note Theorem \ref{Rank Pruning_with_range_separability} also holds when $g$ is ideal, since \emph{ideal} $\subset$ \emph{range separability}. In practice, \emph{range separability} encompasses a wide range of imperfect $g(x)$ scenarios, e.g. $g(x)$ can have large fluctuation in both $P$ and $N$ or have systematic drift w.r.t. to $g^*(x)$ due to underfitting.

\section{Experimental Results}
\label{sec:experimental}

In Section \ref{methodology}, we developed a theoretical framework for Rank Pruning, proved exact noise estimation and equivalent expected risk when conditions are ideal, and derived closed-form solutions when conditions are non-ideal. Our theory suggests that, in practice, Rank Pruning should (1) accurately estimate $\rho_1$ and $\rho_0$, (2) typically achieve as good or better F1, error and AUC-PR \citep{Davis:2006:RPR:1143844.1143874} as state-of-the-art methods, and (3) be robust to both mislabeling and added noise. 

In this section, we support these claims with an evaluation of the comparative performance of Rank Pruning in non-ideal conditions across thousands of scenarios. These include less complex (MNIST) and more complex (CIFAR) datasets,  simple (logistic regression) and complex (CNN) classifiers, the range of noise rates, added random noise, separability of $P$ and $N$, input dimension, and number of training examples to ensure that Rank Pruning is a general, agnostic solution for $\tilde{P}\tilde{N}$ learning.

\begin{figure}[t]
\begin{center}
\centerline{\includegraphics[width=1.0\columnwidth]{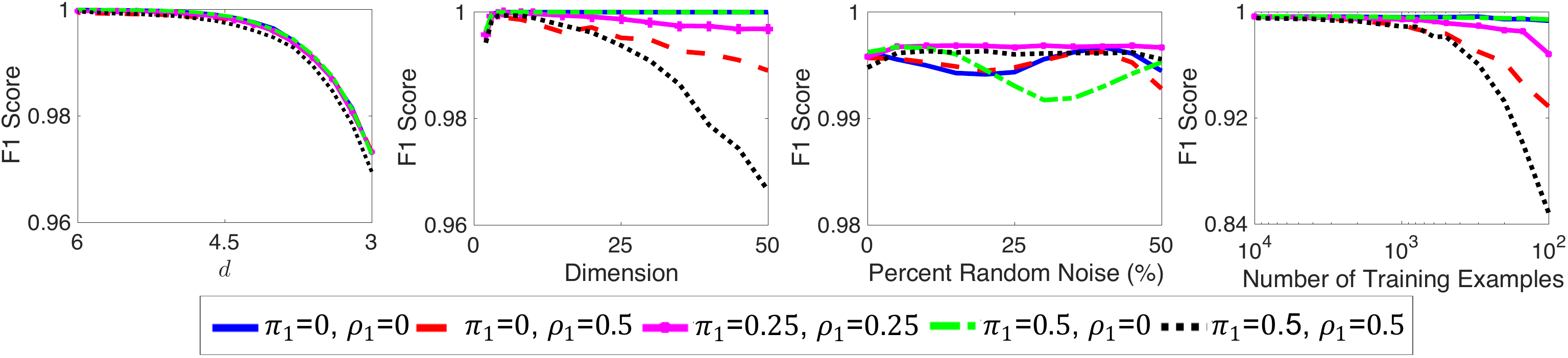}}
\caption{Comparison of Rank Pruning with different noise ratios $(\pi_1, \rho_1)$ on a synthetic dataset for varying separability $d$, dimension, added random noise and number of training examples. Default settings for Fig. \ref{synthetic_F1}, \ref{synthetic_diff_tp} and \ref{synthetic_comparison}: $d=4$, 2-dimension, $0\%$ random noise, and 5000 training examples with $p_{y1}=0.2$. The lines are an average of 200 trials.}
\label{synthetic_F1}
\end{center}
\vskip -0.3in
\end{figure}

\begin{figure}[t]
\begin{center}
\centerline{\includegraphics[width=0.75\columnwidth]{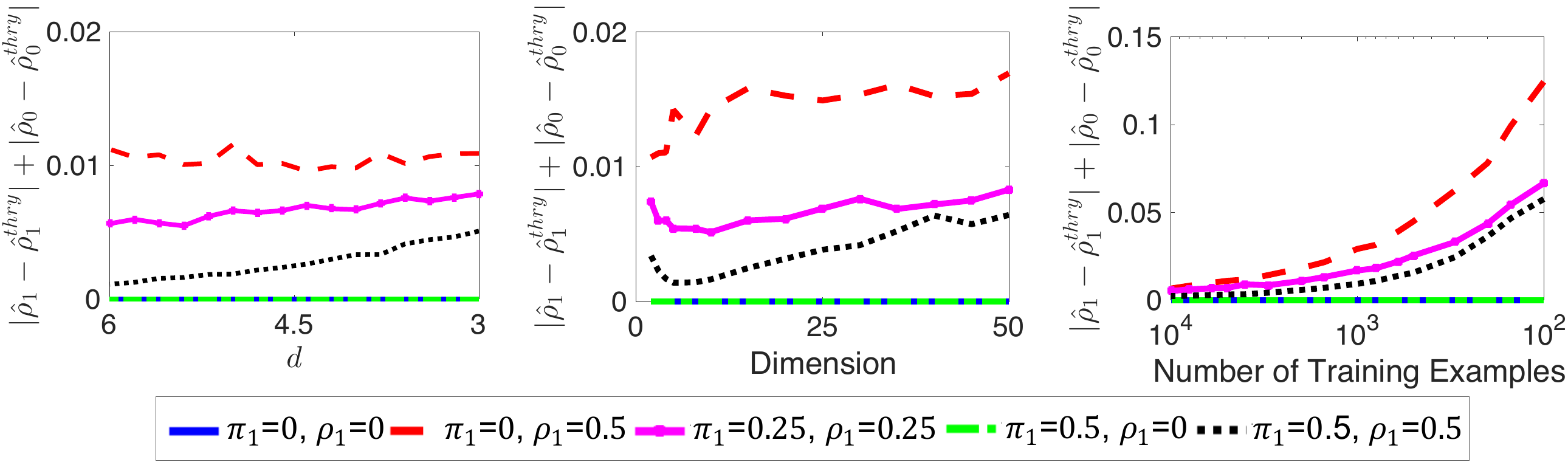}}
\caption{Sum of absolute difference between theoretically estimated $\hat{\rho}_i^{thry}$ and empirical $\hat{\rho}_i$, $i=0,1$, with five different $(\pi_1, \rho_1)$, for varying separability $d$, dimension, and number of training examples. Note that no figure exists for percent random noise because the theoretical estimates in Eq. (\ref{rho_imperfect_condition}) do not address added noise examples.}
\label{synthetic_diff_tp}
\end{center}
\vskip -0.3in
\end{figure}

In our experiments, we adjust $\pi_1$ instead of $\rho_0$ because binary noisy classification problems (e.g. detection and recognition tasks) often have that $|P| \ll |N|$. This choice allows us to adjust both noise rates with respect to $P$, i.e. the fraction of true positive examples that are mislabeled as negative ($\rho_1$) and the fraction of observed positive labels that are actually mislabeled negative examples ($\pi_1$). The $\tilde{P}\tilde{N}$ learning algorithms are trained with corrupted labels $s$, and tested on an unseen test set by comparing predictions $\hat{y}$ with the true test labels $y$ using F1 score, error, and AUC-PR metrics. We include all three to emphasize our apathy toward tuning results to any single metric. We provide F1 scores in this section with error and AUC-PR scores in Appendix \ref{sec:tables_appendix}.

\subsection{Synthetic Dataset}
The synthetic dataset is comprised of a Guassian positive class and a Guassian negative classes such that negative examples ($y=0$) obey an $m$-dimensional Gaussian distribution $N(\mathbf{0}, \mathbf{I})$ with unit variance $\mathbf{I}=diag(1,1,...1)$, and positive examples obey $N(d\mathbf{1}, 0.8\mathbf{I})$, where $d\mathbf{1}=(d,d,...d)$ is an $m$-dimensional vector, and $d$ measures the separability of the positive and negative set.

We test Rank Pruning by varying 4 different settings of the environment: separability $d$, dimension, number of training examples $n$, and percent (of $n$) added random noise drawn from a uniform distribution $U([-10,10]^m)$. In each scenario, we test 5 different $(\pi_1,\rho_1)$ pairs: $(\pi_1,\rho_1)\in\{(0,0),(0,0.5),(0.25,0.25)$, $(0.5,0),(0.5,0.5)\}$. From Fig. \ref{synthetic_F1}, we observe that across these settings, the F1 score for Rank Pruning is fairly agnostic to magnitude of mislabeling (noise rates). As a validation step, in Fig. \ref{synthetic_diff_tp} we measure how closely our empirical estimates match our theoretical solutions in Eq. (\ref{rho_imperfect_condition}) and find near equivalence except when the number of training examples approaches zero.

For significant mislabeling ($\rho_1=0.5$, $\pi_1=0.5$), Rank Pruning often outperforms other methods (Fig. \ref{synthetic_comparison}). In the scenario of different separability $d$, it achieves nearly the same F1 score as the ground truth classifier. Remarkably, from Fig. \ref{synthetic_F1} and Fig. \ref{synthetic_comparison}, we observe that when added random noise comprises $50\%$ of total training examples, Rank Pruning still achieves F1 $>$ 0.85, compared with F1 $<$ 0.5 for all other methods. This emphasizes a unique feature of Rank Pruning, it will also remove added random noise because noise drawn from a third distribution is unlikely to appear confidently positive or negative.

\subsection{MNIST and CIFAR Datasets}

We consider the binary classification tasks of one-vs-rest for the MNIST \citep{lecun-mnisthandwrittendigit-2010} and CIFAR-10 (\cite{cifar10}) datasets, e.g. the ``car vs rest" task in CIFAR is to predict if an image is a ``car" or ``not". $\rho_1$ and $\pi_1$ are given to all $\tilde{P}\tilde{N}$ learning methods for fair comparison, except for $RP_\rho$ which is Rank Pruning including noise rate estimation. $RP_\rho$ metrics measure our performance on the unadulterated $\tilde{P}\tilde{N}$ learning problem.

\begin{figure}[t]
\begin{center}
\centerline{\includegraphics[width=1.0\columnwidth]{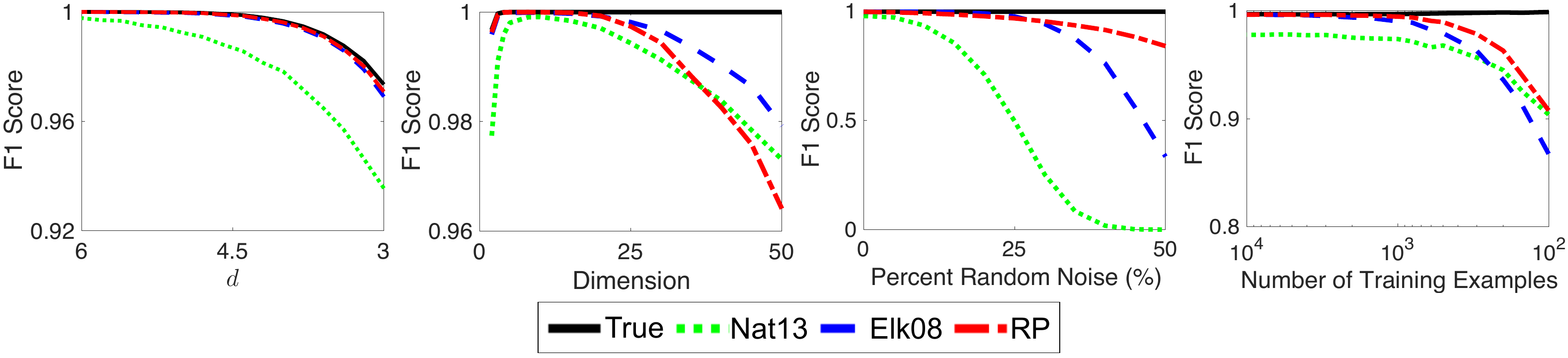}}
\caption{Comparison of $\tilde{P}\tilde{N}$ methods for varying separability $d$, dimension, added random noise, and number of training examples for $\pi_1=0.5$, $\rho_1=0.5$ (given to all methods).
}
\label{synthetic_comparison}
\end{center}
\vskip -0.15in
\end{figure}

\begin{figure}[b!]
\begin{center}
\centerline{\includegraphics[width=1.0\columnwidth]{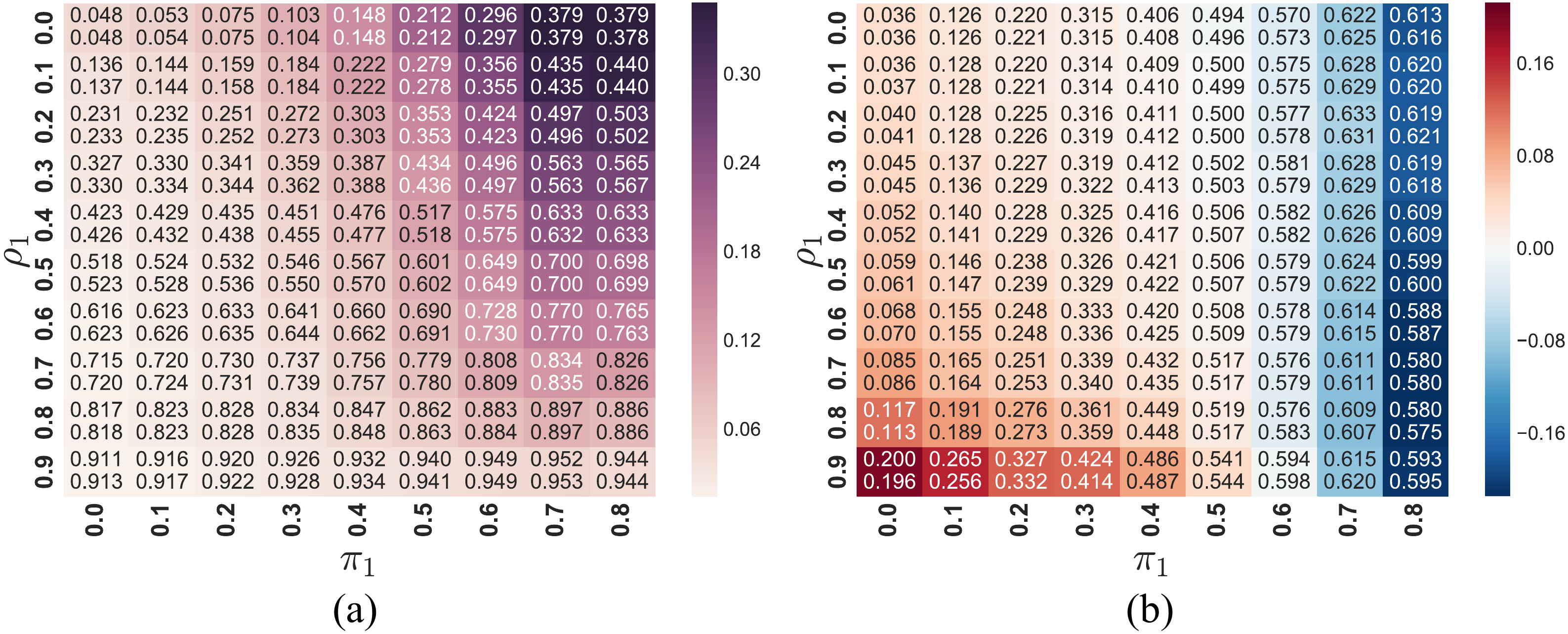}}
\vskip -0.08in
\caption{Rank Pruning $\hat{\rho}_1$ and $\hat{\pi}_1$ estimation consistency, averaged over all digits in MNIST. \textbf{(a)} Color depicts $\hat{\rho}_1-\rho_1$ with $\hat{\rho}_1$ (upper) and theoretical $\hat{\rho}_1^{thry}$ (lower) in each block. \textbf{(b)} Color depicts $\hat{\pi}_1-\pi_1$ with $\hat{\pi}_1$ (upper) and $\hat{\pi}_1^{thry}$ (lower) in each block. }
\label{rho1_pi1_mnist_logreg}
\end{center}
\vskip -0.1in
\end{figure} 

As evidence that Rank Pruning is dataset and classifier agnostic, we demonstrate its superiority with both (1) a linear logistic regression model with unit L2 regularization and (2) an AlexNet CNN variant with max pooling and dropout, modified to have a two-class output. The CNN structure is adapted from \cite{mnist_cnn_structure} for MNIST and \cite{cifar_cnn_structure} for CIFAR. CNN training ends when a 10\% holdout set shows no loss decrease for 10 epochs (max 50 for MNIST and 150 for CIFAR).

We consider noise rates \begin{small}{$\pi_1, \rho_1\in\{(0, 0.5),(0.25,0.25),$
$(0.5,0),(0.5,0.5)\}$}\end{small} for both MNIST and CIFAR, with additional settings for MNIST in Table \ref{table:mnist_cifar_logreg_f1} to emphasize Rank Pruning performance is noise rate agnostic. The $\rho_1=0$, $\pi_1=0$ case is omitted because when given $\rho_1$, $\pi_1$, all methods have the same loss function as the ground truth classifier, resulting in nearly identical F1 scores. Note that in general, Rank Pruning does not require perfect probability estimation to achieve perfect F1-score. As an example, this occurs when $P$ and $N$ are range-separable, and the rank order of the sorted $g(x)$ probabilities in $P$ and $N$ is consistent with the rank of the perfect probabilities, regardless of the actual values of $g(x)$.

For MNIST using logistic regression, we evaluate the consistency of our noise rate estimates with actual noise rates and theoretical estimates (Eq. \ref{rho_imperfect_condition}) across $\pi_1\in[0, 0.8] \times \rho_1\in[0,0.9]$. The computing time for one setting was $\sim10$ minutes on a single CPU core. The results for $\hat{\rho}_1$ and $\hat{\pi}_1$ (Fig. \ref{rho1_pi1_mnist_logreg}) are satisfyingly consistent, with mean absolute difference MD$_{\hat{\rho}_1, \rho_1}=0.105$ and  MD$_{\hat{\pi}_1, \pi_1}=0.062$, and validate our theoretical solutions (MD$_{\hat{\rho}_1,\hat{\rho}_1^{thry}}=0.0028$, MD$_{\hat{\pi}_1,\hat{\pi}_1^{thry}}=0.0058$). The deviation of the theoretical and empirical estimates reflects the assumption that we have infinite examples, whereas empirically, the number of examples is finite.

We emphasize two observations from our analysis on CIFAR and MNIST. First, Rank Pruning performs well in nearly every scenario and boasts the most dramatic improvement over prior state-of-the-art in the presence of extreme noise ($\pi_1=0.5$, $\rho_1=0.5$). This is easily observed in the right-most quadrant of Table \ref{table:mnist_cifar_cnn_f1}. The $\pi_1=0.5$, $\rho_1=0$ quadrant is nearest to $\pi_1=0$, $\rho_1=0$ and mostly captures CNN prediction variation because $|\tilde{P}| \ll |\tilde{N}|$.

Second, RP$_{\rho}$ often achieves equivalent (MNIST in Table \ref{table:mnist_cifar_cnn_f1}) or significantly higher (CIFAR in Tables \ref{table:mnist_cifar_logreg_f1} and \ref{table:mnist_cifar_cnn_f1}) F1 score than Rank Pruning when $\rho_1$ and $\pi_1$ are provided, particularly when noise rates are large. This effect is exacerbated for harder problems (lower F1 score for the ground truth classifier) like the ``cat” in CIFAR or the ``9” digit in MNIST likely because these problems are more complex, resulting in less confident predictions, and therefore more pruning. 

Remember that $\rho^{conf}_1$ and $\rho^{conf}_0$ are upper bounds when $g$ is unassuming. Noise rate overestimation accounts for the complexity of harder problems. As a downside, Rank Pruning may remove correctly labeled examples that ``confuse” the classifier, instead fitting only the confident examples in each class. We observe this on CIFAR in Table \ref{table:mnist_cifar_logreg_f1} where logistic regression severely underfits so that RP$_{\rho}$ has significantly higher F1 score than the ground truth classifier. Although Rank Pruning with noisy labels seemingly outperforms the ground truth model, if we lower the classification threshold to 0.3 instead of 0.5, the performance difference goes away by accounting for the lower probability predictions.

\begin{table*}[t]

\setlength\tabcolsep{2pt} 
\renewcommand{\arraystretch}{0.9}
\caption{Comparison of F1 score for one-vs-rest MNIST and CIFAR-10 (averaged over all digits/images) using logistic regression. Except for $RP_\rho$, $\rho_1$, $\rho_0$ are given to all methods. Top model scores are in bold with $RP_\rho$ in red if greater than non-RP models. Due to sensitivity to imperfect $g(x)$, \emph{Liu16} often
predicts the same label for all examples.
} 
\vskip -0.17in
\label{table:mnist_cifar_logreg_f1}
\begin{center}
\begin{small}
\begin{sc}

\resizebox{\textwidth}{!}{ 
\begin{tabular}{l|cccc|ccc|cccc|cccc|cccc}
\toprule

\multicolumn{0}{r|}{\textbf{Dataset}} & \multicolumn{4}{c|}{\textbf{CIFAR}} & \multicolumn{15}{c}{\textbf{MNIST}} \\

\multicolumn{0}{r|}{$\pi_1=$} & \textbf{0.0} & \textbf{0.25} & \textbf{0.5} & \textbf{0.5} &
\multicolumn{3}{c}{$\pi_1$\textbf{ = 0.0}}  &   
\multicolumn{4}{c}{$\pi_1$\textbf{ = 0.25}}  & 
\multicolumn{4}{c}{$\pi_1$\textbf{ = 0.5}}  & 
\multicolumn{4}{c}{$\pi_1$\textbf{ = 0.75}} 
\\

\textbf{Model,$\rho_1 = $} &   \textbf{0.5} &    \textbf{0.25} &    \textbf{0.0} &   \textbf{0.5} &  \textbf{0.25} &    \textbf{0.5} &   \textbf{0.75} &   \textbf{0.0} &    \textbf{0.25} &    \textbf{0.5} &   \textbf{0.75} &   \textbf{0.0} &    \textbf{0.25} &    \textbf{0.5} &   \textbf{0.75} &   \textbf{0.0} &    \textbf{0.25} &    \textbf{0.5} &   \textbf{0.75}    \\

\midrule

\textbf{True}   &  0.248 & 0.248 & 0.248 & 0.248 & 0.894 &  0.894 &  0.894 &  0.894 &  0.894 &  0.894 &  0.894 &  0.894 &  0.894 &  0.894 &  0.894 &  0.894 &  0.894 &  0.894 &  0.894 \\
\textbf{RP}$_{\rho}$ &  \textcolor{red}{\textbf{0.301}}  & \textcolor{red}{\textbf{0.316}}  & \textcolor{red}{\textbf{0.308}}  & \textcolor{red}{\textbf{0.261}}  &  \textcolor{red}{\textbf{0.883}} &  \textcolor{red}{\textbf{0.874}} &  \textcolor{red}{\textbf{0.843}} &  \textcolor{red}{\textbf{0.881}} &  \textcolor{red}{\textbf{0.876}} &  \textcolor{red}{\textbf{0.863}} &  \textcolor{red}{\textbf{0.799}} &  0.823 &  0.831 &  \textcolor{red}{\textbf{0.819}} &  \textcolor{red}{\textbf{0.762}} &  0.583 &  0.603 &  0.587 &  0.532 \\
\textbf{RP} & \textbf{0.256} & \textbf{0.262} & \textbf{0.244} & 0.209 & \textbf{0.885} &  \textbf{0.873} &  \textbf{0.839} &  \textbf{0.890} &  \textbf{0.879} &  \textbf{0.863} &  \textbf{0.812} &  \textbf{0.879} &  \textbf{0.862} &  \textbf{0.838} &  \textbf{0.770} &  \textbf{0.855} &  \textbf{0.814} &  \textbf{0.766} &  0.617 \\
\textbf{Nat13} & 0.226 & 0.219 & 0.194 & 0.195 &  0.860 &  0.830 &  0.774 &  0.865 &  0.836 &  0.802 &  0.748 &  0.839 &  0.810 &  0.777 &  0.721 &  0.809 &  0.776 &  0.736 &  \textbf{0.640} \\
\textbf{Elk08} & 0.221  & 0.226  & 0.228  & \textbf{0.210}  &  0.862 &  0.830 &  0.771 &  0.864 &  0.847 &  0.819 &  0.762 &  0.843 &  0.835 &  0.814 &  0.736 &  0.674 &  0.669 &  0.599 &  0.473 \\
\textbf{Liu16} & 0.182 & 0.182 & 0.000 & 0.182 &   0.021 &  0.000 &  0.000 &  0.000 &  0.147 &  0.147 &  0.073 &  0.000 &  0.164 &  0.163 &  0.163 &  0.047 &  0.158 &  0.145 &  0.164 \\

\bottomrule
\end{tabular}
}
\end{sc}
\end{small}
\end{center}
\vskip -0.1in
\end{table*}

\begin{table*}[t]

\setlength\tabcolsep{1pt} 
\renewcommand{\arraystretch}{0.85}
\caption{F1 score comparison on MNIST and CIFAR-10 using a CNN. Except for $RP_\rho$, $\rho_1$, $\rho_0$ are given to all methods.
} 
\vskip -0.17in
\label{table:mnist_cifar_cnn_f1}
\begin{center}
\begin{small}
\begin{sc}

\resizebox{\textwidth}{!}{
\begin{tabular}{l|c|ccccc|ccccc|ccccc|ccccc}
\toprule

\multicolumn{2}{l|}{\textbf{MNIST/CIFAR}} &  
\multicolumn{5}{c|}{$\pi_1$\textbf{ = 0.0}}   &   
\multicolumn{5}{c|}{$\pi_1$\textbf{ = 0.25}}  & 
\multicolumn{10}{c}{$\pi_1$\textbf{ = 0.5}}   \\

\multicolumn{1}{l}{\textbf{IMAGE}} &  
\multicolumn{1}{c|}{} & 
\multicolumn{5}{c|}{$\rho_1$\textbf{ = 0.5}} & 
\multicolumn{5}{c|}{$\rho_1$\textbf{ = 0.25}}  & 
\multicolumn{5}{c}{$\rho_1$\textbf{ = 0.0}} & 
\multicolumn{5}{c}{$\rho_1$\textbf{ = 0.5}}    \\

{\textbf{CLASS}} &   \textbf{True} & \textbf{RP}$_{\rho}$ &    \textbf{RP} & \textbf{Nat13} & \textbf{Elk08} & \textbf{Liu16} & \textbf{RP}$_{\rho}$ &    \textbf{RP} & \textbf{Nat13} & \textbf{Elk08} & \textbf{Liu16} & \textbf{RP}$_{\rho}$ &    \textbf{RP} & \textbf{Nat13} & \textbf{Elk08} & \textbf{Liu16} & \textbf{RP}$_{\rho}$ &    \textbf{RP} & \textbf{Nat13} & \textbf{Elk08} & \textbf{Liu16} \\
\midrule
\textbf{0}     &  0.993 &  \textcolor{red}{\textbf{0.991}} &  \textbf{0.988} &  0.977 &  0.976 &  0.179 &  \textcolor{red}{\textbf{0.991}} &  \textbf{0.992} &  0.982 &  0.981 &  0.179 &  \textcolor{red}{\textbf{0.991}} &  \textbf{0.992} &  0.984 &  0.987 &  0.985 &  \textcolor{red}{\textbf{0.989}} &  \textbf{0.989} &  0.937 &  0.964 &  0.179 \\
\textbf{1}     &  0.993 &  \textcolor{red}{\textbf{0.990}} &  \textbf{0.991} &  0.989 &  0.985 &  0.204 &  \textcolor{red}{\textbf{0.992}} &  \textbf{0.992} &  0.984 &  0.987 &  0.204 &  0.990 &  0.991 &  0.992 &  \textbf{0.993} &  0.990 &  \textcolor{red}{\textbf{0.989}} &  \textbf{0.989} &  0.984 &  0.988 &  0.204 \\
\textbf{2}     &  0.987 &  \textcolor{red}{\textbf{0.973}} &  \textbf{0.976} &  0.972 &  0.969 &  0.187 &  \textcolor{red}{\textbf{0.984}} &  \textbf{0.983} &  0.978 &  0.975 &  0.187 &  0.985 &  0.986 &  0.985 &  0.986 &  \textbf{0.988} &  \textcolor{red}{\textbf{0.971}} &  \textbf{0.975} &  0.968 &  0.959 &  0.187 \\
\textbf{3}     &  0.990 &  \textcolor{red}{\textbf{0.984}} &  \textbf{0.984} &  0.972 &  0.981 &  0.183 &  \textcolor{red}{\textbf{0.986}} &  \textbf{0.986} &  0.978 &  0.978 &  0.183 &  \textcolor{red}{\textbf{0.990}} &  0.987 &  \textbf{0.989} &  \textbf{0.989} &  0.984 &  \textcolor{red}{\textbf{0.981}} &  \textbf{0.979} &  0.957 &  0.971 &  0.183 \\
\textbf{4}     &  0.994 &  \textcolor{red}{\textbf{0.981}} &  0.979 &  \textbf{0.981} &  0.977 &  0.179 &  \textcolor{red}{\textbf{0.985}} &  \textbf{0.987} &  0.971 &  0.964 &  0.179 &  0.987 &  \textbf{0.990} &  \textbf{0.990} &  0.989 &  0.985 &  \textcolor{red}{\textbf{0.977}} &  \textbf{0.982} &  0.955 &  0.961 &  0.179 \\
\textbf{5}     &  0.989 &  \textcolor{red}{\textbf{0.982}} &  \textbf{0.980} &  0.978 &  0.979 &  0.164 &  \textcolor{red}{\textbf{0.985}} &  \textbf{0.982} &  0.964 &  0.965 &  0.164 &  \textcolor{red}{\textbf{0.988}} &  \textbf{0.987} &  \textbf{0.987} &  0.984 &  \textbf{0.987} &  \textcolor{red}{\textbf{0.965}} &  \textbf{0.968} &  0.962 &  0.957 &  0.164 \\
\textbf{6}     &  0.989 &  \textcolor{red}{\textbf{0.986}} &  \textbf{0.985} &  0.972 &  0.982 &  0.175 &  \textcolor{red}{\textbf{0.985}} &  \textbf{0.987} &  0.978 &  0.981 &  0.175 &  0.985 &  0.985 &  \textbf{0.988} &  0.987 &  0.985 &  \textcolor{red}{\textbf{0.983}} &  \textbf{0.982} &  0.946 &  0.959 &  0.175 \\
\textbf{7}     &  0.987 &  \textcolor{red}{\textbf{0.981}} &  \textbf{0.980} &  0.967 &  0.948 &  0.186 &  \textcolor{red}{\textbf{0.976}} &  \textbf{0.975} &  0.971 &  0.971 &  0.186 &  0.976 &  0.980 &  \textbf{0.985} &  0.982 &  0.983 &  \textcolor{red}{\textbf{0.973}} &  \textbf{0.968} &  0.942 &  0.958 &  0.186 \\
\textbf{8}     &  0.989 &  \textcolor{red}{\textbf{0.975}} &  \textbf{0.978} &  0.943 &  0.967 &  0.178 &  \textcolor{red}{\textbf{0.982}} &  \textbf{0.981} &  0.967 &  0.951 &  0.178 &  0.982 &  \textbf{0.984} &  0.982 &  0.979 &  0.983 &  \textcolor{red}{\textbf{0.977}} &  \textbf{0.975} &  0.864 &  0.959 &  0.178 \\
\textbf{9}     &  0.982 &  0.966 &  \textbf{0.974} &  0.972 &  0.935 &  0.183 &  \textcolor{red}{\textbf{0.976}} &  \textbf{0.974} &  0.967 &  0.967 &  0.183 &  0.976 &  0.975 &  0.974 &  \textbf{0.978} &  0.970 &  \textcolor{red}{\textbf{0.959}} &  0.940 &  0.931 &  \textbf{0.942} &  0.183 \\
\midrule
\textbf{AVG$_{MN}$} &  0.989 &  \textcolor{red}{\textbf{0.981}} &  \textbf{0.981} &  0.972 &  0.970 &  0.182 &  \textcolor{red}{\textbf{0.984}} &  \textbf{0.984} &  0.974 &  0.972 &  0.182 &  0.985 &  \textbf{0.986} &  \textbf{0.986} &  0.985 &  0.984 &  \textcolor{red}{\textbf{0.976}} &  \textbf{0.975} &  0.945 &  0.962 &  0.182 \\

\midrule

\textbf{plane}   &  0.755 &  \textcolor{red}{\textbf{0.689}} &  \textbf{0.634} &  0.619 &  0.585 &  0.182 &  \textcolor{red}{\textbf{0.695}} &  \textbf{0.702} &  0.671 &  0.640 &  0.182 &  \textcolor{red}{\textbf{0.757}} &  \textbf{0.746} &  0.716 &  0.735 &    0.000 &  \textcolor{red}{\textbf{0.628}} &  \textbf{0.635} &  0.459 &  0.598 &  0.182 \\
\textbf{auto} &  0.891 &  \textcolor{red}{\textbf{0.791}} &  \textbf{0.785} &  0.761 &  0.768 &  0.000 &  \textcolor{red}{\textbf{0.832}} &  \textbf{0.824} &  0.771 &  0.783 &  0.182 &  0.862 &  0.866 &  \textbf{0.869} &  0.865 &    0.000 &  \textcolor{red}{\textbf{0.749}} &  \textbf{0.720} &  0.582 &  0.501 &  0.182 \\
\textbf{bird}       &  0.669 &  \textcolor{red}{\textbf{0.504}} &  \textbf{0.483} &  0.445 &  0.389 &  0.182 &  \textcolor{red}{\textbf{0.543}} &  \textbf{0.515} &  0.469 &  0.426 &  0.182 &  \textcolor{red}{\textbf{0.577}} &  \textbf{0.619} &  0.543 &  0.551 &    0.000 &  \textcolor{red}{\textbf{0.447}} &  \textbf{0.409} &  0.366 &  0.387 &  0.182 \\
\textbf{cat}        &  0.487 &  \textcolor{red}{\textbf{0.350}} &  0.279 &  0.310 &  \textbf{0.313} &  0.000 &  \textcolor{red}{\textbf{0.426}} &  0.317 &  \textbf{0.350} &  0.345 &  0.182 &  \textcolor{red}{\textbf{0.489}} &  \textbf{0.433} &  0.426 &  0.347 &    0.000 &  \textcolor{red}{\textbf{0.394}} &  0.282 &  0.240 &  \textbf{0.313} &  0.182 \\
\textbf{deer}       &  0.726 &  \textcolor{red}{\textbf{0.593}} &  \textbf{0.540} &  0.455 &  0.522 &  0.182 &  \textcolor{red}{\textbf{0.585}} &  0.554 &  0.480 &  \textbf{0.569} &  0.182 &  0.614 &  0.630 &  \textbf{0.643} &  0.633 &    0.000 &  \textcolor{red}{\textbf{0.458}} &  0.375 &  0.310 &  \textbf{0.383} &  0.182 \\
\textbf{dog}        &  0.569 &  \textcolor{red}{\textbf{0.544}} &  \textbf{0.577} &  0.429 &  0.456 &  0.000 &  \textcolor{red}{\textbf{0.579}} &  0.559 &  0.569 &  \textbf{0.576} &  0.182 &  0.647 &  0.637 &  \textbf{0.667} &  0.630 &    0.000 &  \textcolor{red}{\textbf{0.516}} &  0.461 &  0.412 &  \textbf{0.465} &  0.182 \\
\textbf{frog}       &  0.815 &  \textcolor{red}{\textbf{0.746}} &  0.727 &  \textbf{0.733} &  0.718 &  0.000 &  \textcolor{red}{\textbf{0.729}} &  \textbf{0.750} &  0.630 &  0.584 &  0.182 &  0.767 &  \textbf{0.782} &  0.777 &  0.770 &    0.000 &  \textcolor{red}{\textbf{0.635}} &  \textbf{0.615} &  0.589 &  0.524 &  0.182 \\
\textbf{horse}      &  0.805 &  \textcolor{red}{\textbf{0.690}} &  0.670 &  0.624 &  \textbf{0.672} &  0.182 &  \textcolor{red}{\textbf{0.710}} &  0.669 &  \textbf{0.683} &  0.627 &  0.182 &  0.761 &  \textbf{0.776} &  0.769 &  0.753 &    0.000 &  \textcolor{red}{\textbf{0.672}} &  \textbf{0.569} &  0.551 &  0.461 &  0.182 \\
\textbf{ship}       &  0.851 &  \textcolor{red}{\textbf{0.791}} &  \textbf{0.783} &  0.719 &  0.758 &  0.182 &  \textcolor{red}{\textbf{0.810}} &  \textbf{0.801} &  0.758 &  0.723 &  0.182 &  0.816 &  0.822 &  0.830 &  \textbf{0.831} &    0.000 &  \textcolor{red}{\textbf{0.715}} &  \textbf{0.738} &  0.569 &  0.632 &  0.182 \\
\textbf{truck}      &  0.861 &  \textcolor{red}{\textbf{0.744}} &  \textbf{0.722} &  0.655 &  0.665 &  0.182 &  \textcolor{red}{\textbf{0.814}} &  \textbf{0.826} &  0.798 &  0.774 &  0.182 &  0.812 &  0.830 &  \textbf{0.826} &  0.824 &    0.000 &  \textcolor{red}{\textbf{0.654}} &  0.543 &  0.575 &  \textbf{0.584} &  0.182 \\
\midrule
\textbf{AVG$_{CF}$}      &  0.743 &  \textcolor{red}{\textbf{0.644}} &  \textbf{0.620} &  0.575 &  0.585 &  0.109 &  \textcolor{red}{\textbf{0.672}} &  \textbf{0.652} &  0.618 &  0.605 &  0.182 &  \textcolor{red}{\textbf{0.710}} &  \textbf{0.714} &  0.707 &  0.694 &    0.000 &  \textcolor{red}{\textbf{0.587}} &  \textbf{0.535} &  0.465 &  0.485 &  0.182 \\
\bottomrule
\end{tabular}
}
\end{sc}
\end{small}
\end{center}
\vskip -0.1in
\end{table*}

\section{Discussion}

To our knowledge, Rank Pruning is the first time-efficient algorithm, w.r.t. classifier fitting time, for $\tilde{P}\tilde{N}$ learning that achieves similar or better F1, error, and AUC-PR than current state-of-the-art methods across practical scenarios for synthetic, MNIST, and CIFAR datasets, with logistic regression and CNN classifiers, across all noise rates, $\rho_1, \rho_0$, for varying added noise, dimension, separability, and number of training examples. By \emph{learning with confident examples}, we discover provably consistent estimators for noise rates, $\rho_1$, $\rho_0$, derive theoretical solutions when $g$ is unassuming, and accurately uncover the classifications of $f$ fit to hidden labels, perfectly when $g$ range separates $P$ and $N$.

We recognize that disambiguating whether we are in the unassuming or range separability condition may be desirable. Although knowing $g^*(x)$ and thus $\Delta g(x)$ is impossible, if we assume randomly uniform noise, and toggling the $LB_{y=1}$ threshold does not change $\rho^{conf}_1$, then $g$ range separates $P$ and $N$. When $g$ is unassuming, Rank Pruning is still robust to imperfect $g(x)$ within a range separable subset of $P$ and $N$ by training with confident examples even when noise rate estimates are inexact.

An important contribution of Rank Pruning is generality, both in classifier and implementation. The use of logistic regression and a generic CNN in our experiments emphasizes that our findings are not dependent on model complexity. We evaluate thousands of scenarios to avoid findings that are an artifact of problem setup. A key point of Rank Pruning is that we only report the simplest, non-parametric version. For example, we use 3-fold cross-validation to compute $g(x)$ even though we achieved improved performance with larger folds. We tried many variants of pruning and achieved significant higher F1 for MNIST and CIFAR, but to maintain generality, we present only the basic model.

At its core, Rank Pruning is a simple, robust, and general solution for noisy binary classification by \emph{learning with confident examples}, but it also challenges how we think about training data. For example, SVM showed how a decision boundary can be recovered from only support vectors. Yet, when training data contains significant mislabeling, confident examples, many of which are far from the boundary, are informative for uncovering the true relationship $P(y=1|x)$. Although modern affordances of ``big data" emphasize the value of \emph{more} examples for training, through Rank Pruning we instead encourage a rethinking of learning with \emph{confident} examples.

\chapter{Conclusion and Prospects}
\label{chap9:conclusion}

\section{Conclusions}
In this thesis, I have addressed several key aspects of intelligence: few-shot learning, representation learning, causal learning, lifelong learning, improving robustness, and improving intelligibility. Different works usually involve more than one aspect, and use one aspect to improve some others. Also, physics and information play pivotal roles (Fig. \ref{fig:thesis_for_aspects}). In AI Physicist (Chapter \ref{chap2:AI_physicist}), we apply four physicist strategies to build an agent that has the capability of predicting the future, few-shot learning, lifelong learning, and interpretability. The agent we develop can in turn also learn physics theories in prototypical environments. In MeLA (Chapter \ref{chap7:mela}), we show that learning good representation can help predicting the future in a few-shot way. To understand the two-term tradeoff in representation learning, inspired by phase transitions in physics, we study the phase transitions in the Information Bottleneck (IB) (Chapter \ref{chap3:IB} and Chapter \ref{chap4:IB_phase_transition}). We  derive formulas for giving the condition for IB phase transitions. Based on the formulas, we show the close interplay between the information objective, the dataset and the learned representation, by revealing that each phase transition corresponds to learning of a new component of nonlinear maximum correlation between the input and the target. We also show how the phase transitions depend on the model capacity. In addition, for binary classification, we study the mutual information with target vs. entropy of the representation tradeoff (Chapter \ref{chap5:distillation}), by proving that we can reach the Pareto frontier by binning a uniformized sorted probability of the target given the input, and illustrate how it can be interpreted as an information-theoretically optimal image clustering algorithm. To enable machines to understand causality from observations, we introduce an algorithm that combines predicting the future with minimizing information from the input, for exploratory causal discovery of observational time series, and demonstrate its effectiveness in synthetic, video game, breath rate vs. heart rate and \emph{C.elegans} datasets. To improve robustness of classifiers to noisy labels, we introduce Rank Pruning for learning under noisy labels. Under mild assumptions, we prove that it can achieve the same accuracy as if the labels are not corrupted. We also demonstrate that it improves the state-of-the-art in noisy label classification. I believe this progress will bring us one step closer to building intelligent machines that can make sense of the world and become better at learning.

\section{Prospects}

Looking ahead, 
I believe there are vast opportunities ahead. The graph in Fig. \ref{fig:thesis_for_aspects} is far from a complete graph. And I believe an ideal graph will look like Fig. \ref{fig:prospects_aspects_combined}, where all the aspects work integrally together. As a first step, any edge, and with either of the two directions, implies an opportunity. Take some aspects which have not been addressed by my thesis for example: learning good representations can improve robustness; learning good representation can help lifelong learning; causal learning can improve robustness, improving interpretability can help few-shot learning; to name just a few. Besides, there exists many other aspects either mentioned or not mentioned in Section \ref{sec:what_is_intelligence}, which we can also draw many edges with. Suppose that there are $N$ aspects. Then we have $N(N-1)$ directed edges above. Moreover, we can try to combine three or more aspects together, just like my AI Physicist work, which implies on the order of $N^3$ or larger number of combinations.

Ultimately, we want an agent that can combine all the aspects integrally together in an elegant way. Although daunting, I think it is possible, and it may not be so far away. As we have seen, physics and information have provided valuable tools underlying my thesis and many other works, and I believe they will certainly play a central role, in combining theses aspects integrally together, to build an agent that can make sense of the world, and help solve many problems in the society.

\begin{figure}[pbt]
\centerline{\includegraphics[width=1\columnwidth]{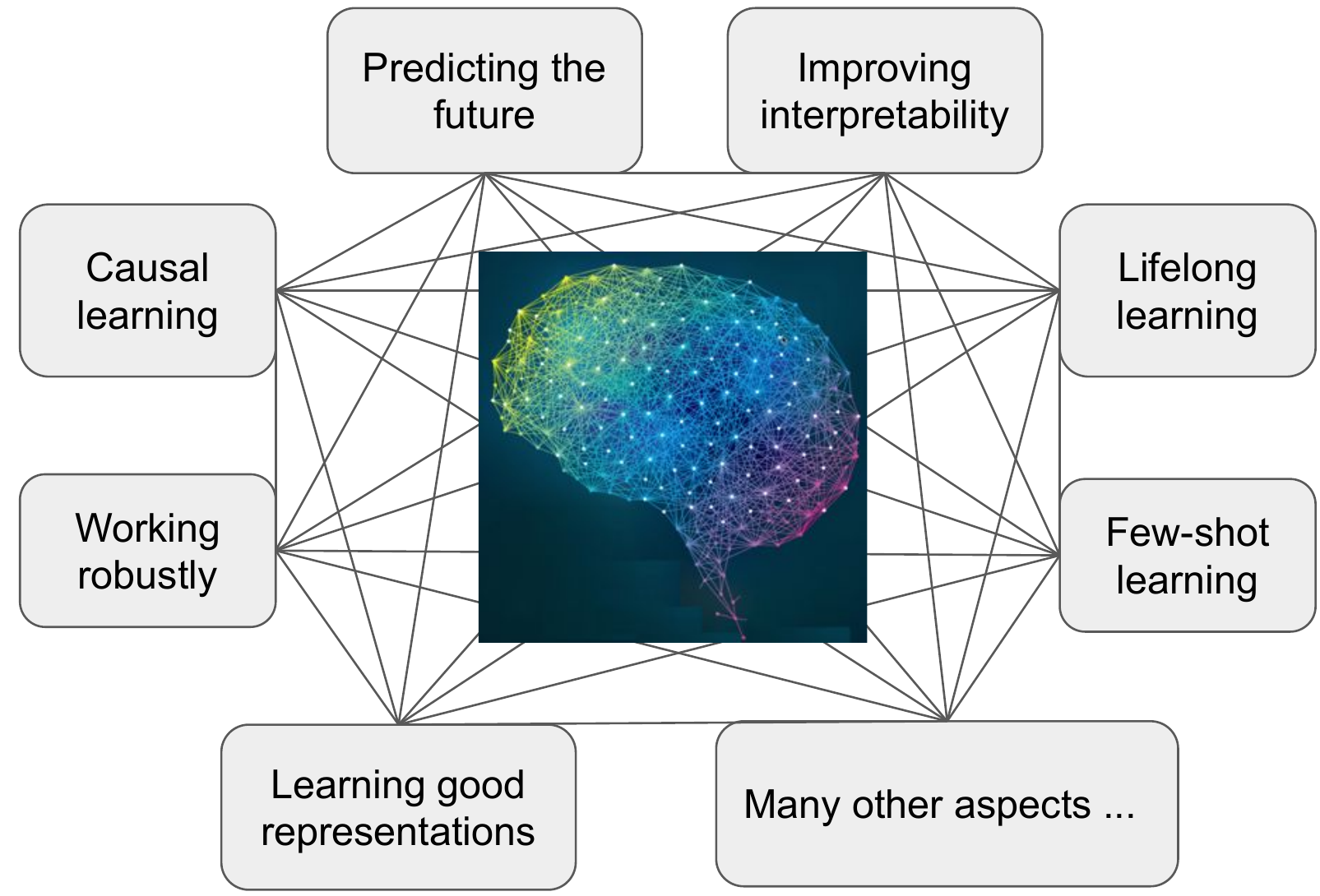}}
\caption{
Prospect of different aspects combined integrally together.
\label{fig:prospects_aspects_combined}
}
\end{figure}

\appendix
\chapter{Appendix}
\section{Appendix for Chapter \ref{chap2:AI_physicist}}

\subsection{AI Physicist Algorithm}
\label{AIphysicistAlgo}

The detailed AI Physicist algorithm is presented\footnote{The full code is open-sourced at \href{https://github.com/tailintalent/AI_physicist}{github.com/tailintalent/AI\_physicist}.} in Algorithm~\ref{alg:overall_algorithm}, with links to each of the individual sub-algorithms.
Like most numerical methods, the algorithm contains a number of hyperparameters that can be tuned to optimize performance; Table~\ref{hyperparameter_table} lists them and their settings for our numerical experiments.

\begin{algorithm}[h!]
\caption{\textbf{AI Physicist: Overall algorithm}}
\label{alg:overall_algorithm}
\begin{algorithmic}
\STATE Given observations $D=\{(\x_t, \y_t)\}$ from new environment:
\STATE 1: $\mathbfcal{T}_{M_0}\gets \text{\textbf{Hub}.propose-theories}(D, M_0)$ (Alg.~\ref{alg:theory_propose})
\STATE 2: $\mathbfcal{T}\gets\text{differentiable-divide-and-conquer}(D, \mathbfcal{T}_{M_0})$(Alg.~\ref{alg:divide_and_conquer})
\STATE 3: \textbf{Hub}.add-theories($\mathbfcal{T},D$) (Alg.~\ref{alg:add_theory})
\STATE
\STATE Organizing theory hub:
\STATE $\mathbfcal{T}\gets$\textbf{Hub}.Occam's-Razor-with-MDL($\mathbfcal{T},D$) (Alg.~\ref{alg:MDL_simplification})
\STATE $\mathbfcal{T}\gets$\textbf{Hub}.unify($\mathbfcal{T}$) (Alg.~\ref{alg:unification})
\end{algorithmic}
\end{algorithm}

\subsection{The Differentiable Divide-and-Conquer (DDAC) Algorithm}
\label{DivideAndConquerAlgo}

\begin{algorithm}[h!]
\caption{\textbf{AI Physicist: Differentiable Divide-and-Conquer with Harmonic Loss}}
\label{alg:divide_and_conquer}
\begin{algorithmic}
   \STATE {\bfseries Require} Dataset $\D=\{(\x_t, \y_t)\}$
   \STATE {\bfseries Require $M$}: number of initial total theories for training
   \STATE {\bfseries Require $\mathbfcal{T}_{M_0}=\{(\f_i,c_i)\}, i=1,...,M_0, 0\leq M_0\leq M$}: theories proposed from theory hub
   \STATE {\bfseries Require $K$}: number of gradient iterations
   \STATE {\bfseries Require $\beta_\f, \beta_\c$}: learning rates
   \STATE {\bfseries Require $\epsilon_0$}: initial precision floor
   \STATE
   \STATE 1: Randomly initialize $M-M_0$ theories $\T_i, i=M_0+1,$
   \STATE \ \ \ \ $...,M$. Denote $\mathbfcal{T}=(\T_1,,...,\T_M)$, $\mathbf{f_\vtheta}=(\f_1,...,\f_M)$,
   \STATE \ \ \ \ $\mathbf{c_\vphi}=(c_1,...c_M)$ with learnable parameters $\vtheta$ and $\vphi$.
  \STATE
  \STATE // \textit{Harmonic training with DL loss:}
   \STATE 2: $\epsilon\gets\epsilon_0$
   \STATE 3: \textbf{for} $\mathit{k}$ \textbf{in} $\{1,2,3,4,5\}$ \textbf{do}:
   \STATE 4: \ \ \ \ $\mathbfcal{T}\gets \textbf{IterativeTrain}(\mathbfcal{T}, \D, \ell_{\DL, \epsilon}, \Ell_{-1})$, where
   \STATE \ \ \ \ \ \ \ \  $\Ell_{-1}\equiv\sum_t\left(\frac{1}{M}\sum_{i=1}^M \ell[\f_i(\x_t), \y_t]^{-1}\right)^{-1}$ (Eq. \ref{generalized_mean_risk_empirical})
 \STATE 5: \ \ \ \  $\epsilon\gets\text{set\_epsilon}(\mathbfcal{T},\D)$ //
   \textit{median prediction error}
    \STATE 6: \textbf{end for}
  \STATE
   \STATE // \textit{Fine-tune each theory and its domain:}
   \STATE 7: \textbf{for} $\mathit{k}$ \textbf{in} $\{1,2\}$ \textbf{do}:
   \STATE 8: \ \ \ \ $\mathbfcal{T}\gets \textbf{IterativeTrain}(\mathbfcal{T}, \D, \ell_{\DL, \epsilon}, \Ell_\text{dom})$, where
   \STATE \ \ \ \ \ \ \ \ $\Ell_{\text{dom}}\equiv\sum_t \ell[\f_{i_t}(\x_t), \y_t]$ with $i_t=\text{arg max}_i\c_i(\x_t)$ 
 \STATE 9: \ \ \ \  $\epsilon\gets\text{set\_epsilon}(\mathbfcal{T},\D)$ //
   \textit{median prediction error}
    \STATE 10: \textbf{end for}
   \STATE 11: \textbf{return} $\mathbfcal{T}$
   \STATE
   \STATE \textbf{subroutine IterativeTrain}$(\mathbfcal{T}, \D, \ell, \Ell):$
   
   \STATE s1: \textbf{for} $k$ in $\{1,...,K\}$ \textbf{do}:
   \STATE \textit{\ \ \ \ \ \ \ //   Gradient descent on $\mathbf{f_\vtheta}$ with loss $\Ell$:}
   \STATE s2: \ \ \ $\g_\f\gets\nabla_\vtheta\Ell[\mathbfcal{T},D,\ell]$
   \STATE s3:\ \ \ \ \ Update $\vtheta$ using gradients $\mathbf{g}_\f$ (e.g. Adam \cite{kingma2014adam} or SGD)
   \STATE \textit{\ \ \ \ \ \ \ // Gradient descent on $\c_\vphi$ with the best performing}
   \STATE \textit{\ \ \ \ \ \ \ \ \ \ \ theory index as target:}
   \STATE s4: \ \ \ $b_t\gets \argmin_i\{\ell[\f_i(\x_t), \y_t]\}, \forall t$
   \STATE s5: \ \ \ \ $\g_\c\gets \nabla_\vphi\sum_{(\x_t, \cdot)\in \D}\text{CrossEntropy}[\text{softmax}(\c_\vphi(\x_t)), b_t]$
   \STATE s6:\ \ \ \ \ Update $\vphi$ using gradients $\g_\c$ (e.g. Adam \cite{kingma2014adam} or SGD)
   \STATE s7: \textbf{end for}
   \STATE s8: $\mathbfcal{T}\gets \text{AddTheories}(\mathbfcal{T},D,\ell,\Ell)$\ //Optional
   \STATE s9: $\mathbfcal{T}\gets \text{DeleteTheories}(\mathbfcal{T},D,\ell)$\ //Optional
   \STATE s10: \textbf{return} $\mathbfcal{T}$
\end{algorithmic}
\end{algorithm}

Here we elaborate on our differentiable divide-and-conquer (DDAC) algorithm with generalized-mean loss $\Ell_\gamma$ (Eq. (\ref{generalized_mean_risk_empirical})). This loss with $\gamma <0$ works with a broad range of distance functions $\l$ satisfying Theorem~\ref{thm:theorem_gradient}. Since the goal of our AI Physicist is to minimize the overall description length (DL) from \eq{description_length}, we choose $\l$ to be the DL loss function of \eq{RealDLeq} together with $\gamma=-1$ (harmonic loss), which works quite well in practice.

Algorithm~\ref{alg:divide_and_conquer} describes our differentiable divide-and-conquer implementation, which consists of two stages. In the first stage (steps 2-6), it applies the subroutine $\text{IterativeTrain}(\mathbfcal{T},D,\ell_{\DL,\epsilon},\Ell_{-1})$ with harmonic loss $\Ell_{-1}$ to train the theories $\mathbfcal{T}$ a few times with the precision floor $\epsilon$ gradually lowered according to the following annealing schedule. We set the initial precision floor $\epsilon$ to be quite large so that $\ell$ initially approximates an MSE loss function.
After each successive iteration, we reset $\epsilon$ to the median prediction error.

The DL loss function from \eq{RealDLeq} is theoretically desirable but tricky to train,  both because it is non-convex and because it is quite flat and uninformative far from its minimum. 
Our annealing schedule helps overcome both problems: initially when $\epsilon$ is large, it approximates MSE-loss which is convex and guides the training to a good approximate minimum, which further training accurately pinpoints as $\epsilon$ is reduced.

The subroutine $\text{IterativeTrain}$ forms the core of the algorithm. In the first stage (steps 2-6), it uses the harmonic mean of the DL-loss of multiple prediction functions $\f_\vtheta=(\f_1,...,\f_M)$
(\ie, \eq{generalized_mean_risk_empirical} with $\gamma=-1$ and $\ell=$DL)
to simultaneously train these functions, encouraging them to each specialize in the domains where they predict best (as proven by Theorem \ref{thm:theorem_gradient}), and simultaneously trains the domain classifier $\c_\vphi=(c_1,...,c_M)$ using each example's best-performing prediction function as target, with categorical cross-entropy loss. 
After several rounds of $\text{IterativeTrain}$ with successively lower precision floors, each prediction function typically becomes good at predicting part of the dataset, and the domain classifier becomes good at predicting for each example which prediction function will predict best. 

AddTheories($\mathbfcal{T}, D,\ell,\Ell$) inspects each theory $\mathcal{T}_i$ describing at least a large fraction $\eta_{\rm insp}$ (we use $30\%$) of the examples to see if a non-negligible proportion $\eta_{\rm split}$ of examples (we use a threshold of $5\%$) of the examples inside its domain have MSE larger than a certain limit $\epsilon_{\rm add}$ (we use $2\times10^{-6}$) and thus warrant splitting off into a separate domain.

If so, it uses those examples to initialize a new theory $\mathcal{T}_{M+1}$, and performs tentative training together with other theories using IterativeTrain without steps s8 and s9 (it is also possible to allow steps s8 and s9 in this recursive calling of IterativeTrain, which will enable a recursive adding of theories for not-well-explained data, and may enable a more powerful DDAC algorithm). If the resulting loss $\Ell$ is smaller than before adding the new theory, $\mathcal{T}_{M+1}$ is accepted and retained, otherwise it is rejected and training reverts to the checkpoint before adding the theory. 
DeleteTheories($\mathbfcal{T},D,\ell$) deletes theories whose domain or best-predicted examples cover a negligible fraction of the examples (we use a delete threshold $\eta_{\rm del}=0.5\%$).

In the second stage (steps 7-10), the IterativeTrain is applied again, but the loss for each example $(\x_t,\y_t)$ is using only the theory that the domain classifier $\c_\vphi=(c_1,c_2,...c_M)$ predicts (having the largest logit). In this way, we iteratively fine-tune the prediction functions $\{\f_i\}$ w.r.t. each of its domain, and fine-tune the domain to the best performing theory at each point.
The reason that we assign examples to domains using our 
domain classifier rather than prediction accuracy is that the trained domains are likely to be simpler and more contiguous, thus generalizing better to unseen examples than, \eg, the nearest neighbor algorithm.

We now specify the default hyperparameters used for Algorithm \ref{thm:theorem_gradient} in our experiments (unless otherwise specified). We set the initial total number of theories $M=4$, from which $M_0=2$ theories are proposed from the theory hub. The initial precision floor $\epsilon_0=10$ and the number of gradient iterations $K=10000$. We use the Adam \cite{kingma2014adam} optimizer with default parameters for the optimization of both the prediction function and the domain classifier. We randomly split each dataset $D$ into $D_{train}$ and $D_{test}$ with 4:1 ratio. The $D_{test}$ is used only for evaluation of performance. The batch size is set to min(2000, $|D_{train}|$). We set the initial learning rate $\beta_\f=5\times 10^{-3}$ for the prediction functions $\f_\vtheta$ and
 $\beta_\c=10^{-3}$ for the domain classifier $\c_\vphi$. We also use a learning rate scheduler that monitors the validation loss every 10 epochs, 
 and divides the learning rate by 10 if the validation loss has failed to decrease after 40 monitoring points and stops training early if there is no decrease after 200 epochs --- or if 
the entire MSE loss for all the theories in their respective domains drops below $10^{-12}$.

To the main harmonic loss $\Ell_\gamma$, we add two regularization terms. One is $L_1$ loss whose strength increases quadratically from 0 to $\epsilon_{L1}$ during the first 5000 epochs and remains constant thereafter.
The second regularization term is a very small MSE loss of strength $\epsilon_{MSE}$, to encourage the prediction functions to remain not too far away from the target outside their domain.

\subsection{Occam's Razor with MDL Algorithm}
\label{appendix:OccamsRazorAlgo}

Pushing on after the DDAC algorithm with harmonic loss that minimizes the $\sum_t \DL(\u_t)$ term in Eq. (\ref{description_length}), the AI Physicist then strives to minimize the $\DL(\mathbfcal{T})$ term, which can be decomposed as $\DL(\mathbfcal{T})=\DL(\mathbf{f}_\vtheta)+\DL(\c_\vphi)$, where $\mathbf{f}_\vtheta=(\f_1,...\f_M)$ and $\c_\vphi=(c_1,...c_M)$. We focus on minimizing $\DL(\mathbf{f}_\vtheta)$, since in different environments the prediction functions $\f_i$ can often be reused, while the domains may differ. 
As mentioned, we define $\DL(\f_\vtheta)$ simply as the sum of the description lengths of the numbers parameterizing $\f_\vtheta$:
\beq{DL_for_f}
\DL(\mathbf{f}_\vtheta)=\sum_j \DL(\theta_j).
\eeq
This means that 
$\DL(\mathbf{f}_\vtheta)$ can be significantly reduced if an irrational parameter is replaced by a simpler rational number.

\begin{algorithm}[hp!]
\caption{\textbf{AI Physicist: Occam's Razor with MDL}}
\label{alg:MDL_simplification}
\begin{algorithmic}
   \STATE {\bfseries Require} Dataset $\D=\{(\x_t, \y_t)\}$
   \STATE {\bfseries Require $\mathbfcal{T}=\{(\f_i,c_i)\}, i=1,...,M$}: theories trained after Alg.~\ref{alg:divide_and_conquer}
   \STATE {\bfseries Require $\epsilon$}: Precision floor for $\ell_{\DL,\epsilon}$
   \STATE 1: \textbf{for} $i$ in $\{1,...,M\}$ \textbf{do}:
  \STATE 2: \ \ \ \ \ $\D^{(i)}\gets\{(\x_t,\y_t)|\argmax_j\{c_j(\x_t)\}=i\}$
  \STATE 3:\ \ \ \ \ \  $\f_i\gets\textbf{MinimizeDL}(\text{collapseLayers},\f_i, D^{(i)},\epsilon)$
  \STATE 4:\ \ \ \ \ \   $\f_i\gets\textbf{MinimizeDL}(\text{localSnap},\f_i, D^{(i)},\epsilon)$
  \STATE 5:\ \ \ \ \ \  $\f_i\gets\textbf{MinimizeDL}(\text{integerSnap},\f_i, D^{(i)},\epsilon)$
  \STATE 6:\ \ \ \ \ \  $\f_i\gets\textbf{MinimizeDL}(\text{rationalSnap},\f_i, D^{(i)},\epsilon)$
  \STATE 7:\ \ \ \ \ \  $\f_i\gets\textbf{MinimizeDL}(\text{toSymbolic},\f_i, D^{(i)},\epsilon)$
  \STATE 8: \textbf{end for}
  \STATE 9: \textbf{return} $\mathbfcal{T}$
  \STATE
  \STATE \textbf{subroutine MinimizeDL}(transformation, $\f_i$, $D^{(i)}$,$\epsilon$):
  \STATE s1:\ \textbf{while} transformation.is\_applicable($\f_i$) \textbf{do}:
  \STATE s2:\ \ \ \ \ \  $\text{dl}_0\gets\DL(\f_i)+\sum_{(\x_t,\y_t)\in D^{(i)}}\ell_{\DL,\epsilon}[\f_i(\x_t), \y_t]$
  \STATE s3:\ \ \ \ \ \  $f_{\text{clone}}\gets \f_i$  // \textit{clone $\f_i$ 
  in case transformation fails}
  \STATE s4:\ \ \ \ \ \  $\f_i\gets \text{transformation}(\f_i)$
  \STATE s5:\ \ \ \ \ \  $\f_i\gets\text{Minimize}_{\f_i}\sum_{(\x_t, \y_t)\in \D^{(i)}}\ell_{\DL,\epsilon}[\f_i(\x_t),\y_t]$
  \STATE s6:\ \ \ \ \ \  $\text{dl}_1\gets\DL(\f_i)+\sum_{(\x_t,\y_t)\in D^{(i)}}\ell_{\DL,\epsilon}[\f_i(\x_t), \y_t]$ 
\STATE s7:\ \ \ \ \ \  \textbf{if} $\text{dl}_1>\text{dl}_0$ $\textbf{return}$ $f_{\text{clone}}$
  \STATE s8:\ \textbf{end while}
  \STATE s9:\ \textbf{return} $\f_i$
\end{algorithmic}
\end{algorithm}

If a physics experiment or neural net training produces a parameter $p=1.999942$, it would be natural to interpret this as a hint, and to check if $p=2$ gives an equally acceptable fit to the data. We formalize this by replacing any real-valued parameter $p_i$ in our theory $\T$ by its nearest integer if this reduces the total description length in \eq{description_length}, as detailed below. We start this search for  integer candidates with the parameter that is closest to an integer, refitting for the other parameters after each successful ``integer snap".
 
 What if we instead observe a parameter $p=1.5000017$? Whereas generic real numbers have a closest integer, they lack a closest rational number. Moreover, as illustrated in \fig{RationalComplexityFig}, we care not only about closeness (to avoid increasing the second term in \eq{description_length}), but also about simplicity (to reduce the first term).  
  To rapidly find the best ``rational snap" candidates (dots in \fig{RationalComplexityFig} that lie both near $p$ and far down), we perform a continued fraction expansion of $p$ and use each series truncation as a rational candidate.    
We repeat this for all parameters in the theory $\T$, again accepting only those snaps that reduce the total description length. 
We again wish to try the most promising snap candidates first; to rapidly identify promising candidates without having to recompute the second term in \eq{description_length}, we evaluate all truncations of all parameters as in \fig{NumberMDLfig}, comparing the description length of the rational approximation $q=m/n$ with the description length of the approximation error $|p-q|$. The most promising candidate minimizes their sum, \ie, lies furthest down to the left of the diagonal in the figure. The figure illustrates how, given the parameter vector $\p=\{\pi,\sqrt{2},3.43180632382353\}$, the first snap to be attempted will replace the third parameter by $53/17$.

We propose Algorithm \ref{alg:MDL_simplification} to implement the above minimization of $\DL(\f_\vtheta)$ without increasing $\DL(\mathbfcal{T},D)$ (Eq. \ref{description_length}). For each theory $\T_i=(\f_i, c_i)$, we first extract the examples $D^{(i)}$ inside its domain, then perform a series of tentative transformations (simplifications) of the prediction function $\f_i$ using the MinimizeDL subroutine. This subroutine takes $\f_i$, the transformation, and $D^{(i)}$ as inputs and repeatedly applies the transformation to $\f_i$. After each such transformation, it fine-tunes the fit of $\f_i$ to $D^{(i)}$ using gradient descent. For determining whether to accept the transformation, Algorithm \ref{alg:MDL_simplification} presents the simplest 0-step patience implementation: if the description length $\text{dl}=\DL(\f_i)+\sum_{(\x_t,\y_t)\in D^{(i)}}\ell_{\DL,\epsilon}(\f_i(\x_t), \y_t)$ for theory $i$ decreases, then apply the transformation again if possible, otherwise exit the loop. In general, to allow for temporary increase of DL during the transformations, a non-zero patience can be adopted: at each step, save the best performing model as the pivot model, and if DL does not decrease during $n$ consecutive transformations inside MinimizeDL, exit the loop. In our implementation, we use a 4-step patience.

We now detail the five transformations used in Algorithm \ref{alg:MDL_simplification}. The collapseLayer transformation finds all successive layers of a neural net where the lower layer has linear activation, and combines them into one. The toSymbolic transformation transforms $\f_i$ from the form of a neural net into a symbolic expression (in our implementation, from a PyTorch net to a SymPy symbolic lambda expression). These two transformations are one-time transformations (for example, once $\f_i$ has been transformed to a symbolic expression,  toSymbolic cannot be applied to it again.) The localSnap transformation successively sets the incoming weights in the first layer to 0, thus favoring inputs that are closer to the current time step. 
The integerSnap transformation finds the (non-snapped) parameters in $\f_i$ that is closest to an integer, and snaps it to that integer. The rationalSnap transformation finds the (non-snapped) parameter in $\f_i$ that has the lowest bit sum when replaced by a rational number, as described in section \ref{sec:Occams_Razor}, and snaps it to that rational number. The latter three transformations can be applied multiple times to $\f_i$, until there are no more parameters to snap in $\f_i$, or the transformation followed by fine-tuning fails to reduce the description length.

In the bigger picture, Algorithm~ \ref{alg:MDL_simplification} is an implementation of minimizing the $\text{DL}(\f_\vtheta)$ without increasing the total $\text{DL}(\mathbfcal{T},D)$, if the description length of $\f_\vtheta$ is given by Eq. (\ref{DL_for_f}). There can be other ways to encode $\mathbfcal{T}$ with a different formula for $\text{DL}(\mathbfcal{T})$, in which case the transformations for decreasing $\text{DL}(\mathbfcal{T})$ may be different. But the structure of the Algorithm \ref{alg:MDL_simplification} remains the same, with the goal of minimizing $\text{DL}(\f_\vtheta)$ without increasing $\text{DL}(\mathbfcal{T},D)$ w.r.t. whatever DL formula it is based on.

In the still bigger picture, Algorithm~ \ref{alg:MDL_simplification} is a computationally efficient approximate implementation of the MDL formalism, involving the following two approximations:
\begin{enumerate}
    \item The description lengths $\DL(x)$ for various types of numbers are approximate, for convenience. For example, the length of the shortest self-terminating bit-string encoding an arbitrary natural number $n$ grows slightly faster than our approximation $\log_2 n$, because self-termination requires storing not only the binary digits of the integer, but also the length of said bit string, recursively, requiring 
    $\log_2 n + \log_2\log_2 n + \log_2\log_2 n + ...$, where only the positive terms are included 
       \cite{rissanen1983universal}. Slight additional overhead is required to upgrade the encodings to actual programs in some suitable language, including encoding of whether bits encode integers, rational numbers, floating-point numbers, {\etc}. 
    \item If the above-mentioned $\DL(x)$-formulas were made exact, they would be mere {\it upper bounds} on the true minimum description length. For example, our algorithm gives a gigabyte description length for $\sqrt{2}$ with precision $\epsilon=256^{-10^9}$, even though it can be computed by a rather short program, and there is no simple algorithm for determining which numbers can be accurately approximated by algebraic numbers. Computing the true minimum description length is a famous numerically intractable problem.
\end{enumerate}

\subsection{Unification Algorithm}
\label{UnificationAlgo}

\begin{algorithm}[t]
\caption{\textbf{AI Physicist: Theory Unification}}
\label{alg:unification}
\begin{algorithmic}
\STATE {\bfseries Require \textbf{Hub}}: theory hub
\STATE {\bfseries Require $C$}: initial number of clusters
\STATE 1:\ \textbf{for} $(\f_i,c_i)$ \textbf{in} \textbf{Hub}.all-symbolic-theories \textbf{do}:
\STATE 2:\ \ \ \ \ \ $\text{dl}^{(i)}\gets \DL(\f_i)$
\STATE 3:\ \textbf{end for}
\STATE 4:\ $\{S_k\}\gets$Cluster $\{\f_i\}$ into $C$ clusters based on $\text{dl}^{(i)}$
\STATE 5:\ \textbf{for} $S_k$ \textbf{in} $\{S_k\}$ \textbf{do}:
\STATE 6:\ \ \ \ \ $(\g_{i_k},\h_{i_k})\gets\textbf{Canonicalize}(\f_{i_k})$, $\forall \f_{i_k}\in S_k$
\STATE 7:\ \ \ \ \ $\h_k^*\gets$ Mode of $\{\h_{i_k}|\f_{i_k}\in S_k\}$.
\STATE 8:\ \ \ \ \ $G_k\gets\{\g_{i_k}|\h_{i_k}=\h_k^*\}$
\STATE 9:\ \ \ \ \ \ \ $\g_{\p_k}\gets$Traverse all $\g_{i_k}\in G_k$ with synchronized steps, 
\STATE \ \ \ \ \ \ \ \ \ \ \ \ \ \ \ \ \ replacing the coefficient by a $\p_{j_k}$ when not all  
\STATE \ \ \ \ \ \ \ \ \ \ \ \ \ \ \ \ coefficients at the same position are identical.
\STATE 10:\ \ \ \ $\f_{\p_k}\gets \text{toPlainForm}(\g_{\p_k})$
\STATE 11:\ \textbf{end for}
\STATE 12:\ $\mathscr{T}\gets \{(\f_{\p_k},\cdot)\}$, $k=1,2,...C$
\STATE 13:\ $\mathscr{T}\gets\text{MergeSameForm}(\mathscr{T})$
\STATE 14:\ \textbf{return} $\mathscr{T}$

\STATE
\STATE \textbf{subroutine Canonicalize}($\f_i$):
\STATE s1:\ $\g_i\gets\text{ToTreeForm}(\f_i)$
\STATE s2:\ $\h_i\gets\text{Replace all non-input coefficient by a symbol $s$}$
\STATE \textbf{return} $(\g_i,\h_i)$

\end{algorithmic}
\end{algorithm}

\bigskip

The unification process takes as input the symbolic prediction functions $\{(\f_i,\cdot)\}$, and outputs master theories $\mathscr{T}=\{(\f_\p,\cdot)\}$ such that by varying each $\p$ in $\f_\p$, we can generate a continuum of prediction functions $\f_i$ within a certain class of prediction functions. The symbolic expression consists of 3 building blocks: operators (e.g. $+$,$-$,$\times$,$/$), input variables (e.g. $x_1, x_2$), and coefficients that can be either a rational number or irrational number. The unification algorithm first calculates the DL $\text{dl}^{(i)}$ of each prediction function, then clusters them into $K$ clusters using e.g. K-means clustering. Within each cluster $S_k$, it first canonicalizes each $\f_{i_k}\in S_k$ into a 2-tuple $(\g_{i_k},\h_{i_k})$, where $\g_{i_k}$ is a tree-form expression of $\f_{i_k}$ where each internal node is an operator, and each leaf is an input variable or a coefficient. When multiple orderings are equivalent (e.g. $x_1+x_2+x_3$ vs. $x_1+x_3+x_2$), it always uses a predefined partial ordering. $\h_{i_k}$ is the structure of $\g_{i_k}$ where all coefficients are replaced by an $s$ symbol. Then the algorithm obtains a set of $\g_{i_k}$ that has the same structure $\h_{i_k}$ with the largest cardinality (steps 7-8). This will eliminate some expressions within the cluster that might interfere with the following unification process. Step 9 is the core part, where it traverses each $\g_{i_k}\in G_k$ with synchronized steps using e.g. depth-first search or breath-first search. This is possible since each $\g_{i_k}\in G_k$ has the same tree structure $h_k^*$. During traversing, whenever encountering a coefficient and not all coefficients across $G_k$ at this position are the same, replace the coefficients by some symbol $\p_{j_k}$ that has not been used before. Essentially, we are turning all coefficients that varies across $G_k$ into a parameter, and the coefficients that do not vary stay as they are. In this way, we obtain a master prediction function $\f_{\p_k}$. Finally, at step 13, the algorithm merges the master prediction functions in $\mathscr{T}=\{(\f_{\p_k},\cdot)\}$ that have the exact same form, and return $\mathscr{T}$. The domain classifier is neglected during the unification process, since at different environments, each prediction function can have vastly different spacial domains. It is the prediction function (which characterizes the equation of motion) that is important for generalization.

\subsection{Adding and Proposing Theories}
\label{appendix:theory_proposal_adding}

Here we detail the algorithms adding theories to the hub and proposing them for use in new environments.
Alg.~\ref{alg:theory_propose} provides a simplest version of the theory proposing algorithm. Given a new dataset $D$, the theory hub inspects all theories $i$, and for each one, 
counts the number $n_i$ of data points
where it outperforms all other theories. The top $M_0$ theories with largest $n_i$ 
are then proposed. 

For theory adding after training with DDAC (Alg.~\ref{alg:divide_and_conquer}), each theory $i$ calculates its description length $\text{dl}^{(i)}$ inside its domain. If its $\text{dl}^{(i)}$ is smaller than a threshold $\eta$, then the theory $(\f_i,c_i)$ with its corresponding examples $D^{(i)}$ are added to the theory hub. The reason why the data $D^{(i)}$ are also added to the hub is that $D^{(i)}$ gives a reference for how the theory $(\f_i,c_i)$ was trained, and is also needed in the Occam's razor algorithm.

\begin{algorithm}[t]
\caption{\textbf{AI Physicist: Theory Proposing from Hub}}
\label{alg:theory_propose}
\begin{algorithmic}
\STATE {\bfseries Require \textbf{Hub}}: theory hub
\STATE {\bfseries Require} Dataset $\D=\{(\x_t, \y_t)\}$
\STATE {\bfseries Require $M_0$}: number of theories to propose from the hub

\STATE 1: $\{(\f_i,c_i)\}\gets \textbf{Hub}.\text{retrieve-all-theories()}$
\STATE 2: $D_{\text{best}}^{(i)}\gets\{(\x_t, \y_t)|\text{argmin}_j\ell_{\DL,\epsilon}[\f_j(\x_t), \y_t]=i\}$, $\forall i$ 
\STATE 3: $\mathbfcal{T}_{M_0}\gets\big\{(\f_i, c_i)\big| D_{\text{best}}^{(i)}\text{ ranks among } M_0\ \text{largest sets in}\ \{D_{\text{best}}^{(i)}\}\big\}$
\STATE 4:\ \textbf{return} $\mathbfcal{T}_{M_0}$
\end{algorithmic}
\end{algorithm}

\begin{algorithm}[t]
\caption{\textbf{AI Physicist: Adding Theories to Hub}}
\label{alg:add_theory}
\begin{algorithmic}
\STATE {\bfseries Require \textbf{Hub}}: theory hub
\STATE {\bfseries Require $\mathbfcal{T}=\{(\f_i,c_i\}$}: Trained theories from Alg.~\ref{alg:divide_and_conquer}
\STATE {\bfseries Require} Dataset $\D=\{(\x_t, \y_t)\}$
\STATE {\bfseries Require $\eta$}: DL threshold for adding theories to hub
\STATE 1: $D^{(i)}\gets\{(\x_t,\y_t)|\argmax_j\{c_j(\x_t)\}=i\}, \forall i$
\STATE 2: $\text{dl}^{(i)}\gets\frac{1}{|D^{(i)}|}\sum_{(\x_t,\y_t)\in D^{(i)}} \ell_{\DL,\epsilon}(\f_i(\x_t),\y_t), \forall i$
\STATE 3: \textbf{for} $i$ \textbf{in} $\{1,2,...|\mathbfcal{T}|\}$ \textbf{do}:
\STATE 4: \ \ \ \  \textbf{if} $\text{dl}^{(i)}<\eta$ \textbf{do}
\STATE 5: \ \ \ \ \ \ \ \ \  \textbf{Hub}.addIndividualTheory$((\f_i, c_i), D^{(i)})$ 
\STATE 6: \ \ \ \ \textbf{end if}
\STATE 7: \textbf{end for}
\end{algorithmic}
\end{algorithm}

\subsection{Time complexity}
\label{ComplexitySec}

\def\ndom{n_{\rm dom}}
\def\ndata{n_{\rm data}}
\def\ntheo{n_{\rm theo}}
\def\npar{n_{\rm par}}
\def\nmyst{n_{\rm myst}}

In this appendix, we list crude estimates of the time complexity of our AI physicist algorithm, \ie, of how its runtime scales with key parameters.

 \textbf{DDAC}, the differentiable divide-and-conquer algorithm, algorithm (Alg.~\ref{alg:divide_and_conquer}), is run once for each of the $\nmyst$ different mystery worlds,  with a total runtime scaling roughly as
$$\mathcal{O}(\nmyst\npar\ndata\ndom^2),$$
where  
$\npar$ is the average number of neural-network parameter in a theory, 
$\ndata$ is the average number of data points (time steps) per mystery
and
$\ndom$ is the number of discovered domains (in our case $\le 4$). The power of two for $\ndom$ appears because the time to evaluate the loss function 
scales as $\ndom$, and we need to perform of  order $\ndom$ training cycles to add the right number of theories. The $\npar$ scaling is due to that the forward and backward propagation of neural net involves successive matrix multiplied by a vector, which scales as $O(n^2)$ where $n\simeq \sqrt{\npar/N_\text{lay}^f}$ is the matrix dimension for each layer
and $N_\text{lay}^f$ is the number of layers.
Accumulating all layers we have $N_\text{lay}^f n^2=\npar$. We make no attempt to model how the number of training epochs needed to attain the desired accuracy depends on parameters.

Our \textbf{Lifelong learning} algorithm is also run once per mystery, with a time cost dominated by that for proposing new theories (Alg.~\ref{alg:theory_propose}), which scales as
$$\mathcal{O}(\nmyst\ndata\ntheo).$$
Here $\ntheo$ is the number of theories in theory hub.

In contrast, our {\bf Occam's Razor} algorithm (Alg. \ref{alg:MDL_simplification}) and {\bf unification} algorithm (Alg.~\ref{alg:unification}) are run once per learned theory, not once per mystery.
For Occam's Razor, the total runtime is dominated by that for snapping to rational numbers, 
which scales as
$$\mathcal{O}(\npar\ndata\ntheo).$$
For the unification, the total runtime scales as 
$\mathcal{O}(\npar\ntheo)$, which can be neglected relative to the cost of Occam's razor.

We note that all these algorithms have merely polynomial time complexity. 
The DDAC algorithm dominates the time cost; our mystery worlds were typically solved in about 1 hour each on a single CPU.
If vast amounts of data are available, it may suffice to analyze a random subset of much smaller size.

\subsection{Proof of Theorem 1 and Corollary}
\label{proof_theorem_gradient}
Here we give the proof for Theorem \ref{thm:theorem_gradient}, restated here for convenience.

\textbf{Theorem 1}\ 
\textit{Let $\hat{\y}^{(i)}_t\equiv \f_i(\x_t)$ denote the prediction of the target $\y_t$ by the function $\f_i$, $i=1,2,...M$.
Suppose that $\gamma<0$ and $\ell(\hat{\y}_t, \y_t) = \ell(|\hat{\y}_t - \y_t|)$ for a monotonically increasing function 
$\ell(u)$ that vanishes on $[0,u_0]$ for some $u_0\geq 0$, with $\ell(u)^\gamma$ differentiable and strictly convex for $u>u_0$. \\
Then if $0<\ell(\hat{\y}_t^{(i)}, \y_t) < \ell(\hat{\y}_t^{(j)}, \y_t)$, we have 
\beq{gradient_greater_appendix}
\left|\frac{\partial\Ell_\gamma}{\partial u^{(i)}_t}\right| > \left|\frac{\partial \Ell_\gamma}{\partial u^{(j)}_t}\right|,
\eeq
 where $u^{(i)}_t\equiv|\hat{\y}_t^{(i)} - \y_t|$}.

\textit{Proof}.
Since $u^{(i)}_t\equiv|\hat{\y}_t^{(i)} - \y_t|$ and $\ell(\hat{\y}_t, \y_t) = \ell(|\hat{\y}_t - \y_t|)$, 
the generalized mean loss $L_\gamma$ as defined in Eq. \ref{gradient_greater} can be rewritten as 
\beq{L_gamma_loss}
\Ell_\gamma=\sum_t\left(\frac{1}{M}\sum_{k=1}^M \ell(u^{(k)}_t)^{\gamma}\right)^{\frac{1}{\gamma}},
\eeq
which implies that
\begin{equation*}
\begin{aligned}
\left|\frac{\partial \Ell_\gamma}{\partial u^{(i)}_t}\right|&=\left|\frac{1}{\gamma M}\left(\frac{1}{M}\sum_{k=1}^M \ell(u^{(k)}_t)^{\gamma}\right)^{\frac{1}{\gamma}-1}\>\frac{d\ell(u_t^{(i)})^{\gamma}}{du_t^{(i)}}\right|\\
&=
\frac{1}{|\gamma| M}\left(\frac{1}{M}\sum_{k=1}^M \ell(u^{(k)}_t)^{\gamma}\right)^{\frac{1}{\gamma}-1}\>\left|\frac{d\ell(u_t^{(i)})^{\gamma}}{du_t^{(i)}}\right|.
\end{aligned}
\end{equation*}
Since only the last factor depends on $i$, 
proving \eq{gradient_greater_appendix} 
is equivalent to proving that
\beq{ProofEq2}
\left|\frac{d \ell(u_t^{(i)})^{\gamma}}{d u_t^{(i)}}\right|>\left|\frac{d \ell(u_t^{(j)})^{\gamma}}{d u_t^{(j)}}\right|.
\eeq

Let us henceforth consider only the case $u>u_0$, since the conditions $\ell(u_t^{(j)})>\ell(u_t^{(i)})>0$ imply
$u_t^{(j)}>u_t^{(i)}>u_0$.
Since $\gamma<0$, $\ell(u)>0$ and
$\ell'(u)\geq 0$, we have
$\frac{d \ell(u)^{\gamma}}{d u}=\gamma \ell(u)^{\gamma-1}\ell'(u)\leq0$,
so that 
$\left|\frac{d \ell(u)^{\gamma}}{d u}\right|=-\frac{d \ell(u)^{\gamma}}{d u}$.
Because $\ell(u)^\gamma$ is differentiable and strictly convex, its derivative $\frac{d \ell(u)^{\gamma}}{d u}$ is monotonically increasing,
implying that $\left|\frac{d \ell(u)^{\gamma}}{d u}\right|=-\frac{d \ell(u)^{\gamma}}{d u}$ is monotonically decreasing. 
Thus  $\left|\frac{d\ell(u_1)^{\gamma}}{du_1}\right|>\left|\frac{d\ell(u_2)^{\gamma}}{du_2}\right|$
whenever $u_1<u_2$.
Setting $u_1=|\hat{\y}_t^{(i)}-\y_t|$ and $u_2=|\hat{\y}_t^{(j)}-\y_t|$ therefore implies \eq{ProofEq2}, which completes the proof.

\bigskip

The following corollary \ref{corollary:loss_satisfy} demonstrates that the theorem applies to several popular loss functions as well as our two description-length loss functions.
\begin{corollary}
\label{corollary:loss_satisfy}
Defining $u\equiv |\hat{\y}-\y|$, the following loss functions which depend only on $u$ satisfy the conditions for Theorem~\ref{thm:theorem_gradient}:
\begin{enumerate}
    \item
    $\ell(u)=u^r$ for any $r>0$, which includes MSE loss ($r=2$) and mean-absolute-error loss ($r=1$).
    \item Huber loss: $\ell_\delta(u)=\begin{cases}
    \frac{1}{2}u^2, &  u\in[0,\delta]\\
    \delta(u-\frac{\delta}{2}),  & \text{otherwise},
    \end{cases}$\\
    where $\delta>0$.
    \item Description length loss\\
    $\ell_{\DL,\epsilon}(u)=\frac{1}{2}\log_2\left(1+\left(\frac{u}{\epsilon}\right)^2\right)$.
    \item Hard description length loss\\
     $\ell_{\text{DLhard},\epsilon}(u)=\log_2\max\left(1,\frac{u}{\epsilon}\right)$.
\end{enumerate}
\end{corollary}

\textit{Proof}.
We have $u_0=0$ for (1), (2), (3), and $u_0=\epsilon$ for (4). All  four functions $\l$ are monotonically increasing, satisfy $\ell(0)=0$ and are differentiable for $u>u_0$, so all that remains to be shown is that 
$\ell(u)^\gamma$ is strictly convex for $u>u_0$, \ie, that $\frac{d^2 \ell(u)^\gamma}{d u^2}>0$ when $u>u_0$.

(1) For $\ell(u)=u^r$ and $u>0$,
we have $\frac{d^2 \ell(u)^\gamma}{d u^2}=\gamma r(\gamma r -1)u^{\gamma r - 2}>0$, since
$\gamma<0$ and $r>0$ implies that  $\gamma r<0$ and $\gamma r-1<0$, so 
$\ell(u)^\gamma$ is strictly convex for $u>0$.

(2) The Huber loss $\ell_\delta(u)$ is continuous with a continuous derivative.
It satisfies
$\frac{d^2 \ell(u)^\gamma}{d u^2}>0$ both for $0<u<\delta$ and for $\delta<u$ according to the above proof of (1), since 
$\ell_\delta(u)$ is proportional to $\ell^r$ in these two intervals with $r=2$ and $r=1$, respectively.
At the transition point $u=\delta$, this second derivative is discontinuous, but takes positive value approaching both from the left and from the right, so  $\ell(u)^\gamma$ is strictly convex.
More generally, any function $\ell(u)$ built by smoothly connecting functions $\ell_i(u)$ in different intervals will satisfy our theorem if the functions $\ell_i(u)$ do.

 (3) Proving strict convexity of $\ell(u)^\gamma$ when $\ell$ is the description length loss $\ell_{\DL,\epsilon}(u)=\frac{1}{2}\text{log}_2\left[1+\left(\frac{u}{\epsilon}\right)^2\right]$
 is equivalent to proving it when $\ell(u)=\ro(u)\equiv \ln(1+u^2)$, since convexity is invariant under horizontal and vertical scaling.
We thus need to prove that  
$$\frac{d^2 \ro(u)^\gamma}{d u^2}
=-\frac{2\gamma[\ln(1+u^2)]^{\gamma-2}}{(1+u^2)^2 [2u^2(1-\gamma)+(u^2-1)\ln(1+u^2)]}$$ 
is positive when $u>0$.
The factor 
$\frac{-2\gamma[\ln(1+u^2)]^{\gamma-2}}{(1+u^2)^2}$ is always positive. The other factor 
$$2u^2(1-\gamma)+(u^2-1)\text{log}(1+u^2)>2u^2+(u^2-1)\text{log}(1+u^2),$$ 
since $\gamma<0$. Now we only have to prove that the function
$$\chi(u)\equiv2u^2+(u^2-1)\text{log}(1+u^2)>0$$ 
when $u>0$. We have $\chi(0)=0$ and 
$$\chi'(u)=2u\left[\frac{1+3u^2}{1+u^2}+\text{log}(1+u^2)\right]>0$$ 
when $u>0$. Therefore $\chi(u)=\chi(0)+\int_0^u\chi'(u')du'>0$ when $u>0$, which completes the proof that $\ell_{\DL,\epsilon}(u)^\gamma$ is strictly convex for $u>0$.

(4) For the hard description length loss $\ell_{\text{DLhard},\epsilon}(u)=\log_2\max\left(1,\frac{u_t}{\epsilon}\right)$, 
we have $u_0=\epsilon$. 
When $u>\epsilon$, we have 
$\ell'_{\text{DLhard},\epsilon}(u)=\frac{1}{u\text{ln}2}>0$ and 
$$\frac{d^2 \ell^\gamma_{\text{DLhard},\epsilon}(u)}{d u^2}=
{\gamma\over\ln 2}\left(-1+\gamma-\text{ln}\frac{u}{\epsilon}\right)\frac{\left(\ln{u\over\epsilon}\right)^{\gamma-2}}{u^2} .$$
For $u>\epsilon$, the factor $\frac{\left(\ln{u\over\epsilon}\right)^{\gamma-2}}{u^2}$ is always positive, as is the factor  $\gamma(-1+\gamma-\text{ln}\frac{u}{\epsilon})$, since $\gamma<0$.
$\ell^\gamma_{\text{DLhard},\epsilon}(u)$ is therefore strictly convex for $u>\epsilon$.

\subsection{Eliminating Transition Domains}
\label{DomainElimination}

In this appendix, we show how the only hard problem our AI Physicist need solve is to determine the laws of motion far from domain boundaries, because once this is done, the exact boundaries and transition regions can be determined automatically.

Our AI Physicist tries to predict the next position vector $\y_t\in R^d$ from the
concatenation $\x_t =(\y_{t-T},...,\y_{t-1})$  of the last $T$ positions vectors. 
Consider the example shown in \fig{BoundaryPointFig}, where motion is predicted from the last $T=3$ positions in a space with $d=2$ dimensions containing $n=2$ domains with different physics (an electromagnetic field in the upper left quadrant and free motion elsewhere), as well as perfectly reflective boundaries. Although there are only two physics domains in the 2-dimensional space, there are many more types of domains in the $Td=6$-dimensional space of $\x_t$ from which the AI Physicist makes its predictions of $\y_t$. When a trajectory crosses the boundary between the two spatial regions, there can be instances where $\x_t$ contains 3, 2, 1 or 0 points in the first domain and correspondingly 0, 1, 2 or 3 points in the second domain. Similarly, when the ball bounces, there can be instances where $\x_t$ contains 3, 2, 1 or 0 points before the bounce and correspondingly 0, 1, 2 or 3 points after. Each of these situations involves a different function $\x_t\mapsto\y_t$ and a corresponding 6-dimensional domain of validity for the AI Physicist to learn.

Our numerical experiments showed that the AI Physicist typically solves the big domains (where all vectors in $\x_t$ lie in the same spatial region), but occasionally fails to find an accurate solution in some of the many small transition domains involving boundary crossings or bounces, where data is insufficient. Fortunately, simple post-processing can automatically eliminate these annoying transition domains with an algorithm that we will now describe.

\begin{figure}[pbt]
\centerline{\includegraphics[width=88mm]{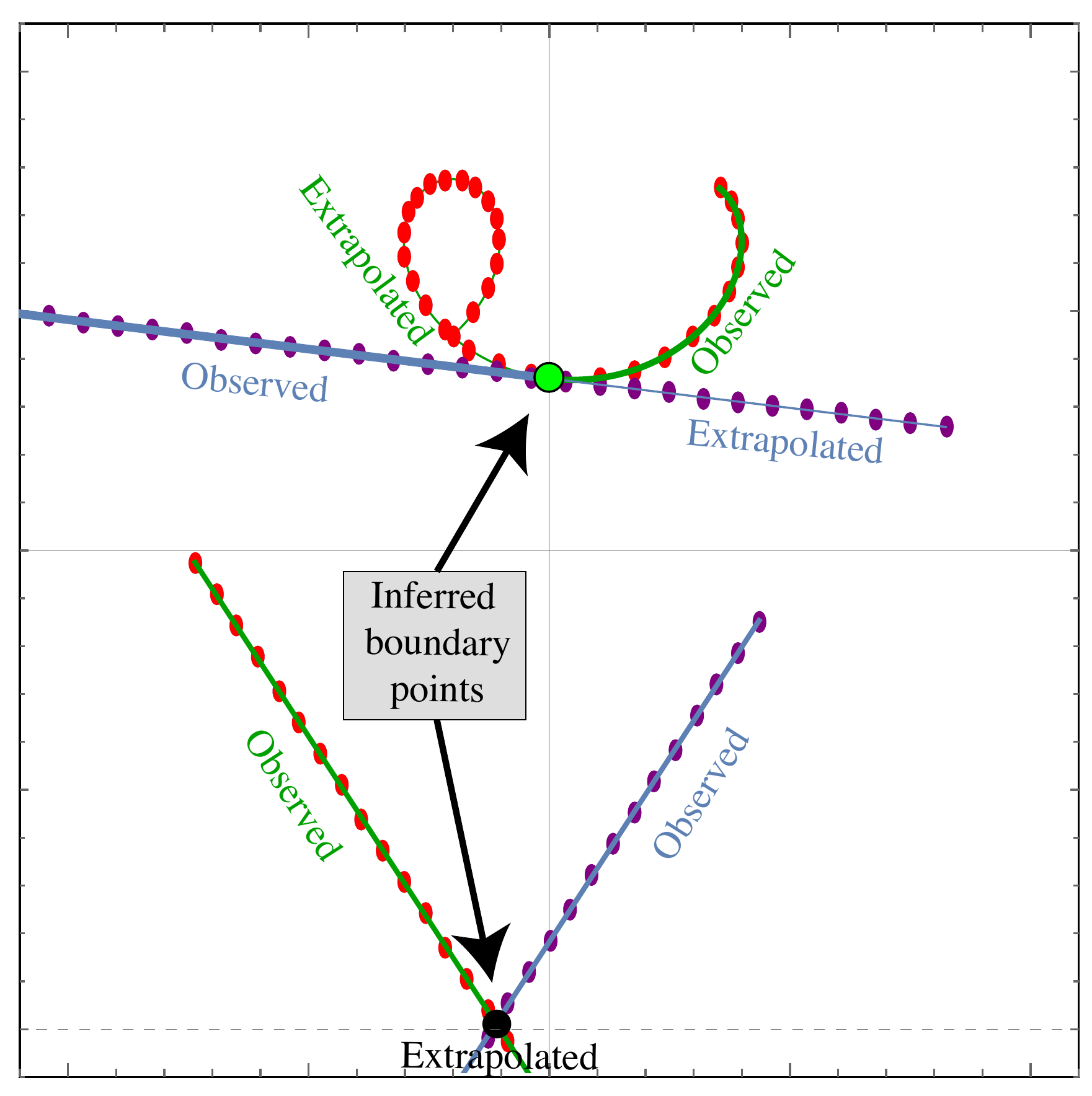}}
\caption{Points where forward and backward extrapolations agree (large black dots) are boundary points. The tangent vectors agree for region boundaries (upper example), but not for bounce boundaries (lower example). 
\label{BoundaryPointFig}
}
\end{figure}

The first step of the algorithm is illustrated in \fig{BoundaryPointFig}.
For each big domain where our AI Physicist has discovered the future-predicting function 
$\x_t\mapsto\y_t$, we determine the corresponding function that predicts the {\it past}
($\x_t\mapsto\y_{t-T-1}$) by fitting to forward trajectories generated with random initial conditions.
Now whenever a trajectory passes from a big domain through a transition region into another big domain, two different extrapolations can be performed: forward in time from the first big domain or backward in time from the second big domain. Using cubic spline interpolation, we fit continuous functions $\y_f(t)$ and $\y_b(t)$ (smooth curves in \fig{BoundaryPointFig}) to these forward-extrapolated and backward-extrapolated trajectories, and numerically find the time
\beq{BoundarpointEq}
t_*\equiv\argmin |\y_f(t)-\y_b(t)|
\eeq
when the distance between the two predicted ball positions is minimized.
If this minimum is numerically consistent with zero, so that $\y_f(t_*)\approx \y_b(t_*)$,
then we record this as being a boundary point.
If both extrapolations have the same derivative there, \ie, 
if $\y'_f(t_*)\approx \y'_b(t_*)$, then it is an interior boundary point between two different regions (\fig{BoundaryPointFig}, top), otherwise it is an external boundary point where the ball bounces (\fig{BoundaryPointFig}, bottom). 

\fig{BoundaryPointsFig} show these two types of automatically computed boundary points in green and black, respectively. These can now be used to retrain the domain classifiers to extend the big domains to their full extent, eliminating the transition regions.

Occasionally the boundary point determinations fill fail because of multiple transitions within $T$ time steps, \Fig{BoundaryPointsFig} illustrates that these failures (red dots) forces us to discard merely a tiny fraction of all cases, thus having a negligible affect on the ability to fit for the domain boundaries.

\begin{figure}[phbt]
\centerline{\includegraphics[width=88mm]{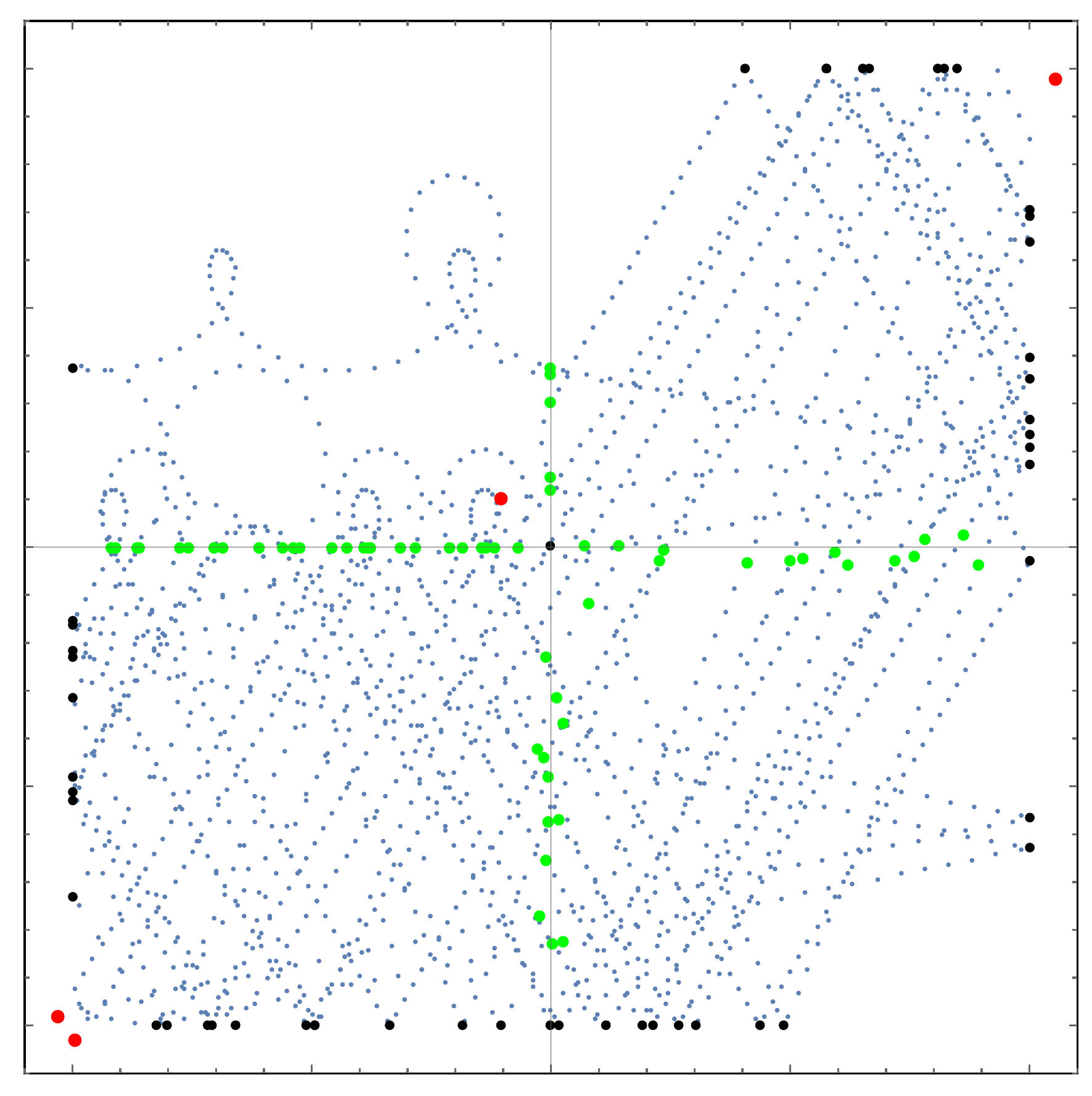}}
\caption{Example of automatically determined boundary points, for region boundary points (green), bounce boundary points (black) and failed cases (red).
\label{BoundaryPointsFig}
}
\end{figure}

\subsection{Numerical Experiment Details}
\label{DetailedResultsSec}

In this appendix, we provide supplementary details on our benchmark problems.

\subsubsection{Mystery Worlds}
\label{appendix:mystery_worlds}

\textbf{World generation} Our mystery worlds consist of a ball elastically bouncing against the square boundary of the
two-dimensional spatial region where  $|x|\le 2$ and $|y|\le 2$ (see \fig{WorldExampleFig}).
In each of the four quadrants, one of the following laws of physics are selected, together with their parameters sampled from distributions as follows:
\begin{enumerate}
\item Free motion.
\item A uniform gravitational field ${\bf g}=(g_x,g_y,0)$ with $g_x, g_y$ drawn from a uniform distribution: $g_x, g_y\sim U[-5,5]$.
\item Harmonic motion with frequency $\omega$ around a line a distance $a$ from the origin, making an angle $\phi$ with the $x$-axis; 
$\omega\sim U[1,4]$, 
$a\sim U[0.2, 0.5]$,
$\phi\sim U[0, 2\pi]$.
\item A uniform electric field $\E=(E_x,E_y,0)$ and magnetic field $\B=(0,0,B_z)$; $E_x$, $E_y\sim U[-5,5]$, $B_z\sim U[0,10]$.
\end{enumerate}
To control the difficulty of the tasks and avoid near-degenerate scenarios, we keep only mystery worlds satisfying the following two criteria: (1) At least 0.01 separation between all equations of motion (EOM) in the same world, defined as the Euclidean distance between the vectors of coefficients specifying the EOM difference equations, and
(2)  at least 0.0015 of any non-integer parameter from its nearest integer.

\textbf{Trajectory simulation}
Within each world, we initialize the ball with a random position $(x, y)\sim U[-1,1]^2$ and velocity $(v_0\cos\theta_0, v_0 \sin\theta_0, 0)$; $v_0\sim U[0.1,0.5]$, $\theta_0\sim U[0,2\pi]$. We then compute its position for $N=4,000$ times steps $t=1,2,...,N$ with time interval 0.05.

Although the above-mentioned laws of physics are linear, the mapping
from past points $(\y_{t-T},...,\y_{t-1})$ to the next points $\y_t$ is generally non-linear because of region boundaries where the ball either bounces or transitions to a different physics region. An exception is when three successive points lie within the same region (with the same physics), which happens far from boundaries: in this case, 
the mapping from  $(\y_{t-2},\y_{t-1})\mapsto\y_t$ is deterministic and linear thanks to the differential equations of motion being second-order and linear.

\textbf{Architecture} For the Newborn and AI Physicist agents, each prediction function $\f_i$ is implemented as a $N_{lay}^f$-layer neural network with linear activation, with $N_{neur}^f$-neuron hidden layers (we use $N_{lay}^f=3$, $N_{neur}^f=8$ for our main experiments; see Table~\ref{hyperparameter_table}).
Each domain sub-classifier $c_i$ is implemented as a $N_{lay}^c$-layer neural net, with two hidden $N_{neur}^c$-neuron layers with leakyReLU activation
$\sigma(x)=\max\{0.3x,x\}$,
and the output layer having linear activation (we use $N_{lay}^c=3$, $N_{neur}^c=8$ for our main experiments). The baseline model is implemented as a single $N_{lay}^f$-layer neural net with two hidden 16-neuron layers with leakyReLU activation followed by a linear output layer.
Note that for a fair comparison, the baseline model has more hidden neurons, to roughly compensate for the Newborn and AI Physicist agents typically having multiple theories. The baseline network is nonlinear to boost its expressive power for  modeling the nonlinear prediction function of each world as a whole. For the domain classifier $\c=(c_1,c_2,...c_M)$, it is a $N_{lay}^c$-layer neural net where each hidden layer has $N_{neur}^c=8$ neurons and leakyReLU activation. The last layer has linear activation. See Table \ref{hyperparameter_table} for a list of hyperparameters.

\textbf{Evaluation} The unsupervised classification accuracy is defined as the fraction of correctly classified points, using the permutation of the learned domain labels that best matches the hidden ground truth domain labels. It is ``unsupervised" in the sense that there is no external supervision signal as to which domain label each point should be assigned to: the AI Physicist has to figure out the number of domains and their boundaries and assign each point to a domain, which is a difficult task.

We define a domain as {\it solved} if the agent discovers the its law of motion as difference equation 
 (prediction function) within the following stringent tolerance: all rational coefficients in the difference equation
 are exactly matched, and all irrational coeffients agree to an accuracy better than $10^{-4}$. 
Because of the nature of the physics problems, some of these difference equation coefficients take on the values $0$, $-1$, or $2$, so solving a region requires successful integer snapping as described in Section~\ref{sec:Occams_Razor}.
To make the problem even harder, we also fine-tune the magnetic field in five of the electromagnetic regions to make some of the coefficients simple fractions such as $1/3$ and $1/4$, thus making solving those regions contingent on successful rational snapping as described in Section~\ref{sec:Occams_Razor}. Domain solving can fail either by 
``undersnapping" (failing to approximate a floating-point number by a rational number) or
 `oversnapping" (mistakenly rounding to a rational number).
All our mystery worlds can be downloaded at \href{http://space.mit.edu/home/tegmark/aiphysicist.html}{space.mit.edu/home/tegmark/aiphysicist.html}.

As shown in Appendix \ref{DomainElimination}, the only hard problem our AI Physicist or other algorithms need to solve is to determine the laws of motion away from domain boundaries. Therefore, we evaluate, tabulate and compare the performance of the algorithms only on interior points, \ie, excluding data points $(\x_t,\y_t)$ straddling a boundary encounter. 

\begin{table*}
\begin{center}
\resizebox{1\linewidth}{!}{%
\begin{tabular}{|l|rrr|rrr|rrr|rrr|}
\hline
&\multicolumn{3}{c|}{$\log_{10}$ MSE}
&\multicolumn{3}{c|}{Classification accuracy}
&\multicolumn{3}{c|}{Unsolved domains}
&\multicolumn{3}{c|}{Description length}\\
Regions&Base-&New-&AI&Base-&New-&AI&Base-&New-&AI&Base-&New-&AI\\
&line&born&phys&line&born&phys&line&born&phys&line&born&phys\\
\hline
Free + gravity		       &	      -4.12 &  -14.07 &      -14.08 &	    72.88\%&  100.00\%&      100.00\%&       2 &     0 &	 0 &		11310.4 &    59.4 &	   73.5 \\
Free + gravity		       &	      -4.21 &  -14.02 &      -14.04 &	    88.59\%&  100.00\%&      100.00\%&       2 &     0 &	 0 &		11271.5 &    60.3 &	   60.3 \\
Free + gravity		       &	      -3.69 &  -14.03 &      -14.03 &	    67.65\%&  100.00\%&      100.00\%&       2 &     0 &	 0 &		11364.2 &    60.2 &	   41.9 \\
Free + gravity		       &	      -4.18 &  -13.98 &      -13.98 &	    80.98\%&  100.00\%&      100.00\%&       2 &     0 &	 0 &		11341.7 &    60.6 &	   57.6 \\
Free + gravity		       &	      -4.51 &  -14.06 &      -14.07 &	    87.66\%&  100.00\%&      100.00\%&       2 &     0 &	 0 &		11289.3 &     5.2 &	   59.8 \\
Free + harmonic 	       &	      -3.77 &  -13.99 &      -13.94 &	    73.54\%&  100.00\%&      100.00\%&       2 &     0 &	 0 &		11333.8 &    94.4 &	  139.9 \\
Free + harmonic 	       &	      -3.60 &  -14.05 &      -13.89 &	    66.92\%&  100.00\%&      100.00\%&       2 &     0 &	 0 &		11337.4 &   173.0 &	  172.8 \\
Free + harmonic 	       &	      -3.77 &  -14.04 &      -13.95 &	    59.46\%&  100.00\%&      100.00\%&       2 &     0 &	 0 &		11317.5 &   156.0 &	  173.8 \\
Free + harmonic 	       &	      -5.32 &  -10.48 &      -13.14 &	    80.29\%&  100.00\%&      100.00\%&       2 &     1 &	 0 &		11219.5 &    91.6 &	   90.5 \\
Free + harmonic 	       &	      -3.64 &  -14.00 &      -13.89 &	    71.70\%&  100.00\%&      100.00\%&       2 &     0 &	 0 &		11369.6 &   143.7 &	  136.6 \\
Free + EM		       &	      -3.62 &  -13.95 &      -13.96 &	    82.77\%&  100.00\%&      100.00\%&       2 &     0 &	 0 &		11397.5 &   142.8 &	  284.9 \\
Free + EM		       &	      -4.13 &  -13.82 &      -13.67 &	    76.55\%&  100.00\%&      100.00\%&       2 &     0 &	 0 &		11283.0 &   306.2 &	  306.2 \\
Free + EM		       &	      -4.03 &  -13.45 &      -13.47 &	    74.56\%&   99.97\%&       99.97\%&       2 &     0 &	 0 &		11388.1 &   305.9 &	  307.9 \\
Free + EM		       &	      -4.31 &  -13.77 &      -13.62 &	    86.68\%&   99.91\%&       99.91\%&       2 &     0 &	 0 &		11257.7 &   152.0 &	  133.5 \\
Free + EM		       &	      -4.32 &  -14.00 &      -14.05 &	    84.55\%&  100.00\%&      100.00\%&       2 &     0 &	 0 &		11258.9 &   303.7 &	  303.8 \\
Free + EM rational	       &	      -3.45 &  -13.96 &      -13.95 &	    77.88\%&   99.96\%&       99.93\%&       2 &     0 &	 0 &		11414.9 &   194.2 &	  195.8 \\
Free + EM rational	       &	      -3.90 &  -13.96 &      -13.91 &	    71.13\%&  100.00\%&      100.00\%&       2 &     0 &	 0 &		11340.0 &   199.0 &	  199.0 \\
Free + EM rational	       &	      -4.12 &  -13.97 &      -13.90 &	    72.78\%&  100.00\%&      100.00\%&       2 &     0 &	 0 &		11330.7 &   198.8 &	  198.8 \\
Free + EM rational	       &	      -4.02 &  -14.07 &      -14.00 &	    77.92\%&  100.00\%&      100.00\%&       2 &     0 &	 0 &		11323.5 &   197.8 &	  197.8 \\
Free + EM rational	       &	      -4.83 &  -13.87 &      -13.86 &	    91.14\%&  100.00\%&      100.00\%&       2 &     0 &	 0 &		11247.1 &    10.3 &	   13.9 \\
Free + gravity + harmonic      &	      -4.08 &  -14.03 &      -13.95 &	    60.08\%&  100.00\%&      100.00\%&       3 &     0 &	 0 &		11269.0 &   191.8 &	  191.9 \\
Free + gravity + harmonic      &	      -4.31 &  -14.02 &      -13.66 &	    63.01\%&  100.00\%&      100.00\%&       3 &     0 &	 0 &		11334.2 &   170.4 &	   83.1 \\
Free + gravity + harmonic      &	      -4.01 &  -14.01 &      -13.99 &	    67.48\%&  100.00\%&      100.00\%&       3 &     0 &	 0 &		11351.0 &   168.7 &	  198.9 \\
Free + gravity + harmonic      &	      -3.64 &  -13.97 &      -13.88 &	    60.02\%&   99.97\%&       99.93\%&       3 &     0 &	 0 &		11374.6 &   225.7 &	  225.7 \\
Free + gravity + harmonic      &	      -4.11 &	-7.42 &       -7.43 &	    51.63\%&  100.00\%&       99.97\%&       3 &     1 &	 1 &		11313.7 &   193.5 &	  179.2 \\
Free + gravity + EM	       &	      -3.79 &  -13.93 &      -13.47 &	    57.89\%&  100.00\%&      100.00\%&       3 &     0 &	 0 &		11334.0 &   323.9 &	  346.8 \\
Free + gravity + EM	       &	      -4.18 &  -14.00 &      -14.00 &	    77.16\%&  100.00\%&      100.00\%&       3 &     1 &	 1 &		11301.0 &   277.9 &	   96.2 \\
Free + gravity + EM	       &	      -3.38 &  -13.58 &      -13.87 &	    53.33\%&  100.00\%&       99.96\%&       3 &     0 &	 0 &		11381.4 &   360.4 &	  364.0 \\
Free + gravity + EM	       &	      -3.46 &  -13.87 &      -13.89 &	    49.08\%&  100.00\%&      100.00\%&       3 &     0 &	 0 &		11370.1 &   354.0 &	  350.4 \\
Free + gravity + EM	       &	      -3.54 &  -13.69 &      -13.83 &	    51.28\%&  100.00\%&      100.00\%&       3 &     0 &	 0 &		11370.3 &   331.1 &	  320.7 \\
Free + harmonic + EM	       &	      -3.87 &  -13.82 &      -13.55 &	    67.27\%&  100.00\%&      100.00\%&       3 &     0 &	 0 &		11404.0 &   267.1 &	  275.4 \\
Free + harmonic + EM	       &	      -3.69 &  -13.87 &      -10.93 &	    56.02\%&   99.97\%&       99.94\%&       3 &     0 &	 0 &		11413.4 &   468.5 &	  464.9 \\
Free + harmonic + EM	       &	      -4.06 &  -13.39 &      -13.56 &	    70.87\%&  100.00\%&      100.00\%&       3 &     0 &	 0 &		11340.0 &   452.3 &	  452.3 \\
Free + harmonic + EM	       &	      -3.46 &  -13.94 &      -10.51 &	    59.02\%&   99.97\%&       99.93\%&       3 &     0 &	 0 &		11416.0 &   475.5 &	  471.9 \\
Free + harmonic + EM	       &	      -3.70 &  -13.75 &      -13.82 &	    61.67\%&  100.00\%&      100.00\%&       3 &     0 &	 0 &		11354.9 &   466.8 &	  466.8 \\
Free + gravity + harmonic + EM &	      -3.76 &  -13.82 &       -9.48 &	    27.93\%&  100.00\%&       99.94\%&       4 &     0 &	 0 &		11358.8 &   526.9 &	  530.4 \\
Free + gravity + harmonic + EM &	      -3.74 &  -13.00 &      -13.18 &	    40.80\%&  100.00\%&       99.97\%&       4 &     1 &	 1 &		11284.8 &   418.5 &	  389.1 \\
Free + gravity + harmonic + EM &	      -4.09 &  -13.97 &      -13.75 &	    35.69\%&  100.00\%&      100.00\%&       4 &     0 &	 0 &		11297.4 &   504.6 &	  504.6 \\
Free + gravity + harmonic + EM &	      -3.63 &  -13.80 &       -9.99 &	    31.61\%&  100.00\%&       99.97\%&       4 &     0 &	 0 &		11407.4 &   526.3 &	  526.2 \\
Free + gravity + harmonic + EM &	      -3.51 &	-6.37 &      -13.52 &	    32.97\%&  100.00\%&      100.00\%&       4 &     0 &	 0 &		11445.8 &   527.4 &	  527.5 \\
\hline
Median                         &	      -3.89 &  -13.95 &      -13.88 &	    67.56\%&  100.00\%&	     100.00\%&	    2.5&     0.00 &	  0.00 &	     11338.7 &   198.9 &       198.9 \\
Mean                           &	      -3.94 &  -13.44 &      -13.29 &	    65.51\%&   99.99\%&	      99.99\%&	    2.6&     0.10 &	  0.07 &	     11337.9 &   253.7 &       252.9 \\
\hline
\end{tabular}
}
\end{center}
\caption{Results for each of our first 40 mystery world benchmarks, as described in the section \ref{appendix:mystery_worlds}. Each number is the best out of ten trials with random initializations (using seeds 0, 30, 60, 90, 120, 150, 180, 210, 240, 270), and refers to big domains only. Based on the ``Unsolved domain" column, we count out of 40 worlds what's the percentage Baseline, Newborn and AI Physicist completely solve (has unsolved domain of 0), which goes to the ``Fraction of worlds solved" row in Table \ref{ResultsSummaryTable}.
\label{DetailedResultsTable}
}
\end{table*}

\begin{table*}
\begin{center}
\resizebox{1\linewidth}{!}{%

\begin{tabular}{|l|rrr|rrr|rrr|rrr|}
\hline
&\multicolumn{3}{c|}{Epochs to $10^{-2}$}
&\multicolumn{3}{c|}{Epochs to $10^{-4}$}
&\multicolumn{3}{c|}{Epochs to $10^{-6}$}
&\multicolumn{3}{c|}{Epochs to $10^{-8}$}\\
Regions&Base-&New-&AI-&Base-&New-&AI&Base-&New-&AI&Base-&New-&AI\\
&line&born&phys&line&born&physi&line&born&phys&line&born&phys\\
\hline
Free+gravity                   &                100 &      85 &          85 &                 8440 &     120 &         120 &                $\infty$ &    4175 &        3625 &             $\infty$ &    6315 &        4890 \\
Free+gravity                   &                100 &      70 &          10 &                 4680 &     190 &          35 &                $\infty$ &    2900 &        4650 &             $\infty$ &    2995 &        6500 \\
Free+gravity                   &                 85 &     100 &          15 &                  $\infty$ &     135 &          30 &                $\infty$ &    8205 &        3815 &             $\infty$ &    9620 &        6455 \\
Free+gravity                   &                 95 &      75 &          20 &                 7495 &     140 &          25 &                $\infty$ &    6735 &        1785 &             $\infty$ &    8040 &        2860 \\
Free+gravity                   &                110 &      75 &           0 &                 1770 &     295 &          35 &                $\infty$ &    3740 &        3240 &             $\infty$ &    7030 &        3460 \\
Free + harmonic                &                 80 &      75 &          20 &                  $\infty$ &     145 &          25 &                $\infty$ &    2725 &        4050 &             $\infty$ &    2830 &        6145 \\
Free + harmonic                &                 85 &      75 &          20 &                  $\infty$ &      80 &          25 &                $\infty$ &    7965 &        1690 &             $\infty$ &   10000 &        3400 \\
Free + harmonic                &                 95 &      75 &          30 &                  $\infty$ &     110 &          30 &                $\infty$ &    1805 &        3895 &             $\infty$ &    1855 &        3900 \\
Free + harmonic                &                 25 &      20 &           5 &                 1285 &     460 &          10 &                $\infty$ &    5390 &        1060 &             $\infty$ &    7225 &        6385 \\
Free + harmonic                &                 80 &      95 &           5 &                  $\infty$ &     110 &          20 &                $\infty$ &    4380 &        3300 &             $\infty$ &    4800 &        4035 \\
Free + EM                      &                 90 &      85 &          20 &                  $\infty$ &    1190 &         115 &                $\infty$ &    6305 &        3380 &             $\infty$ &    6590 &        3435 \\
Free + EM                      &                125 &     120 &           0 &                 6240 &     885 &          70 &                $\infty$ &    7310 &        1865 &             $\infty$ &    7565 &        1865 \\
Free + EM                      &                115 &     115 &          15 &                15260 &     600 &          70 &                $\infty$ &    2430 &        1225 &             $\infty$ &    2845 &        4435 \\
Free + EM                      &                145 &      90 &           0 &                 6650 &     140 &           0 &                $\infty$ &    3000 &        5205 &             $\infty$ &    4530 &        8735 \\
Free + EM                      &                 80 &      80 &          10 &                  965 &     200 &          25 &                $\infty$ &    4635 &        1970 &             $\infty$ &    4690 &        2870 \\
Free + EM rational           &                 80 &      75 &           0 &                  $\infty$ &     580 &          70 &                $\infty$ &    5415 &        4150 &             $\infty$ &    5445 &        4175 \\
Free + EM rational           &                100 &     100 &          10 &                  $\infty$ &     460 &          45 &                $\infty$ &    2560 &         965 &             $\infty$ &    2575 &        5760 \\
Free + EM rational           &                140 &      95 &          10 &                11050 &     455 &          65 &                $\infty$ &    1960 &        1150 &             $\infty$ &    6295 &        4005 \\
Free + EM rational           &                120 &     100 &           5 &                13315 &     325 &         175 &                $\infty$ &    3970 &        1290 &             $\infty$ &    4335 &        3560 \\
Free + EM rational           &                 35 &      30 &          35 &                 1155 &     335 &          35 &                $\infty$ &    3245 &        2130 &             $\infty$ &    5115 &        5610 \\
Free + gravity + harmonic      &                150 &      75 &          25 &                 9085 &     130 &          30 &                $\infty$ &    3870 &        6145 &             $\infty$ &    5555 &        6185 \\
Free + gravity + harmonic      &                145 &      90 &           5 &                 6915 &     140 &          25 &                $\infty$ &    4525 &        3720 &             $\infty$ &   10275 &        4430 \\
Free + gravity + harmonic      &                105 &     100 &          15 &                 6925 &     155 &          40 &                $\infty$ &    6665 &        6560 &             $\infty$ &    8915 &        6845 \\
Free + gravity + harmonic      &                 95 &      95 &           5 &                  $\infty$ &     120 &          30 &                $\infty$ &    5790 &       10915 &             $\infty$ &   18450 &       13125 \\
Free + gravity + harmonic      &                135 &      95 &          15 &                 7970 &     190 &          45 &                $\infty$ &   13125 &        7045 &             $\infty$ &     $\infty$ &         $\infty$ \\
Free + gravity + EM            &                130 &     100 &          20 &                  $\infty$ &     575 &          40 &                $\infty$ &    3215 &        5095 &             $\infty$ &    3215 &        5100 \\
Free + gravity + EM            &                125 &     110 &          15 &                 5650 &     160 &          30 &                $\infty$ &    6085 &        4720 &             $\infty$ &    8025 &        4980 \\
Free + gravity + EM            &                 80 &      65 &          15 &                  $\infty$ &     630 &         120 &                $\infty$ &    4100 &        6250 &             $\infty$ &    4100 &        6570 \\
Free + gravity + EM            &                 80 &      75 &           5 &                  $\infty$ &      90 &          45 &                $\infty$ &    5910 &        5815 &             $\infty$ &    7295 &        6090 \\
Free + gravity + EM            &                 80 &      85 &          20 &                  $\infty$ &    1380 &         465 &                $\infty$ &    2390 &       11425 &             $\infty$ &    7450 &       11510 \\
Free + harmonic + EM           &                 85 &      75 &          25 &                  $\infty$ &     600 &         150 &                $\infty$ &    3775 &        4525 &             $\infty$ &    4675 &        5070 \\
Free + harmonic + EM           &                 85 &      90 &          25 &                  $\infty$ &    1245 &         200 &                $\infty$ &    6225 &        2340 &             $\infty$ &    6390 &        3180 \\
Free + harmonic + EM           &                115 &      85 &          15 &                16600 &     190 &          35 &                $\infty$ &    6035 &        1515 &             $\infty$ &   10065 &        2110 \\
Free + harmonic + EM           &                 80 &      70 &          35 &                  $\infty$ &     720 &         195 &                $\infty$ &    6990 &        3895 &             $\infty$ &    6995 &        6115 \\
Free + harmonic + EM           &                 85 &      65 &          10 &                  $\infty$ &     985 &         165 &                $\infty$ &    5660 &        1670 &             $\infty$ &    5820 &        1820 \\
Free + gravity + harmonic + EM &                 90 &      75 &           0 &                  $\infty$ &     540 &         255 &                $\infty$ &    8320 &        7390 &             $\infty$ &    9770 &        7590 \\
Free + gravity + harmonic + EM &                 95 &      80 &          15 &                  $\infty$ &    1265 &         635 &                $\infty$ &    6520 &        6365 &             $\infty$ &    8475 &        6475 \\
Free + gravity + harmonic + EM &                130 &      85 &          10 &                 8620 &     575 &         105 &                $\infty$ &    6320 &        4035 &             $\infty$ &    9705 &        7685 \\
Free + gravity + harmonic + EM &                 75 &      80 &           0 &                  $\infty$ &     815 &         425 &                $\infty$ &    7575 &        8405 &             $\infty$ &   10440 &        8620 \\
Free + gravity + harmonic + EM &                 80 &      65 &          20 &                  $\infty$ &     735 &         280 &                $\infty$ &    6715 &        4555 &             $\infty$ &   12495 &        8495 \\
\hline
Median                           &                 95 &      83 &          15 &                $\infty$ &     330 &          45 &                $\infty$ &    5403 &        3895 &            $\infty$&    6590 &        5100 \\
Mean                             &                 98 &      82 &          15 &                 $\infty$ &     455 &         109 &                $\infty$ &    5217 &        4171 &            $\infty$ &    6892 &        5499 \\
\hline
\end{tabular}
}
\end{center}
\caption{Same as previous table, but showing number of training epochs required to reach various MSE prediction accuracies. We record the metrics every 5 epochs, so all the epochs are multiples of 5. Note that the AI Physicist has superseded $10^{-2}$ MSE already by 0 epochs for some environments, showing that thanks to the lifelong learning strategy which proposes previously learned theories in novel environments,  reasonably good predictions can sometimes be achieved even without gradient descent training.
\label{DetailedResultsTable2}
}
\end{table*}

\subsubsection{Double Pendulum}
\label{appendix:double_pendulum}

Our double pendulum is implemented as two connected pendulums with massless rods of length 1 and that each have a point charge of 1 at their end. As illustrated in \fig{PendulumFig}, the system state is fully determined by the 4-tuple
$\y=(\theta_1, \dot{\theta}_1,\theta_2, \dot{\theta}_2)$ and immersed in a piecewise constant electric field $\E$:
$\E=(0, -E_1)$ in the upper half plane $y\geq-1.05$, and 
$\E=(0, E_2)$
in the lower half plane $y<-1.05$, using coordinates where $y$ increases vertically and the origin is at the pivot point of the upper rod.

We generate 7 environments by 
setting $(E_1, E_2)$ equal to $(E_0,2E_0)$, $(E_0,1.5E_0)$, $(E_0,E_0)$, $(E_0,0.5E_0)$, $(2E_0,E_0)$, $(1.5E_0,E_0)$, and $(0.5E_0,E_0)$, where $E_0=9.8$.
We see that there are two different EOMs for the double pendulum system depending on which of the two fields the lower charge is in (the upper charge is always in $E_1$).
We use Runge-Kutta numerical integration to simulate $\y=(\theta_1, \dot{\theta}_1,\theta_2, \dot{\theta}_2)$ for 10,000 time steps with interval of 0.05, and the algorithms' task is to predict the future $(\y_{t+1}$) based on the past 
($\x_t\equiv\y_t$; history length $T= 1$), and simultaneously discover the two domains and their different EOMs unsupervised. 

In this experiment, we implement prediction function of the Baseline and Newborn both as a $N_{lay}^f$-layer neural net (we use $N_{lay}^f=6$) during DDAC. For the Newborn, each hidden layer has $N_{neur}^f=160$ neurons with hyperbolic tangent (tanh) activation, and for the Baseline, each hidden layer has $N_{neur}^f=320$ neurons with tanh activation for a fair comparison. For the Newborn, the optional AddTheories($\mathbfcal{T},D$) (step s8 in Alg. \ref{alg:divide_and_conquer}) is turned off to prevent unlimited adding of theories. The initial number $M$ of theories for Newborn is set to $M=2$ and $M=3$, each run with 10 instances with random initialization. Its domain classifier $\c=(c_1,c_2,...c_M)$ is a $N_{lay}^c$-layer neural net (we use $N_{lay}^c=3$) where each hidden layer has $N_{neur}^c=6$ neurons and leakyReLU activation. The last layer has linear activation.

\begin{table*}
\begin{center}
\resizebox{1\linewidth}{!}{%
\begin{tabular}{| l|l | l| c c c |}
\hline
&Hyperparameter&Environments& Baseline & Newborn & AI Physicist \\
\hline
$\gamma$ & Generalized-mean-loss exponent & All & -1 & -1 & -1\\
$\beta_\f$& Initial learning rate for $\f_\theta$&All& 0.005 & 0.005 & 0.005 \\
$\beta_c$&Initial learning rate for $\c_\phi$ &All& 0.001 & 0.001 & 0.001 \\
$K$& Number of gradient iterations &All& 10000  & 10000 & 10000\\
$\sigma_c$ & Hidden layer activation function in $\c_\phi$ & All & - & leakyReLU & leakyReLU\\
$N_{\rm lay}^c$& Number of layers in $\c_\phi$ &All & - & 3 & 3 \\
$C$ &Initial number of clusters in theory unification  & All & 4 & 4 & 4\\
$\epsilon_{MSE}$ & MSE regularization strength  & All & $10^{-7}$  & $10^{-7}$ & $10^{-7}$ \\
\hline
$\epsilon_{L1}$ & Final $L_1$ regularization strength  & Mystery worlds & $10^{-8}$  & $10^{-8}$ & $10^{-8}$ \\
 &   & Double Pendulum & $10^{-7}$  & $10^{-7}$ & $10^{-7}$ \\
\hline
$N_{\rm lay}^f$ & Number of layers in $\f_\theta$ &Mystery worlds & 3 & 3 & 3 \\
 &  &Double Pendulum & 6 & 6 & 6 \\
\hline
$N_{\rm neur}^f$&
Number of neurons in $\f_\theta$ &Mystery worlds & 16 & 8 & 8 \\
&						&Double Pendulum & 320 & 160 & - \\
\hline
$N_{\rm neur}^c$&
Number of neurons in $\c_\phi$ &Mystery worlds & - & 8 & 8 \\
&						&Double Pendulum & - & 6 & - \\
\hline
$T$&
Maximum time horizon for input &Mystery worlds & 2 & 2 & 2 \\
&						&Double Pendulum & 1 & 1 & 1 \\
\hline
$\sigma_f$& Hidden layer activation function in $\f_\theta$ &Mystery worlds  & leakyReLU  & linear  & linear \\
&						&Double Pendulum  & tanh  & tanh  & - \\
\hline
$M_0$&
Initial number of theories&Mystery worlds &1 &2&2\\
&					&Double Pendulum&1 &2 \& 3&-\\
\hline
$M$& Maximum number of theories&Mystery worlds & 1 & 4 & 4\\
&					&Double Pendulum& 1 & 2 \& 3 & - \\
\hline
$\epsilon_{\rm add}$& MSE threshold for theory adding &Mystery worlds & - & $2 \times 10 ^ {-6}$ & $2 \times 10^{-6}$\\
&					&Double Pendulum& - & $\infty$ & - \\
\hline
$\eta_{\rm insp}$& Inspection threshold for theory adding &Mystery worlds & - & 30\% & 30\%\\
&					&Double Pendulum& - & $\infty$ & - \\
\hline
$\eta_{\rm split}$& Splitting threshold for theory adding &Mystery worlds & - & 5\% & 5\%\\
&					&Double Pendulum& - & $\infty$ & - \\
\hline
$\eta_{\rm del}$& Fraction threshold for theory deletion &Mystery worlds & - & 0.5\% & 0.5\%\\
&					&Double Pendulum& - & 100\% & - \\
\hline
\end{tabular}
}
\caption{
Hyperparameter settings in the numerical experiments. For a fair comparison between Baseline and the other agents that can have up to 4 theories, the number of neurons in each layer of Baseline is larger  so that the total number of parameters is roughly the same for all agents. The Baseline agent in Mystery worlds has leakyReLU activation to be able to able to account for different domains.
\label{hyperparameter_table}
}
\end{center}
\end{table*}

\section{Appendix for Chapter \ref{chap3:IB}}

The structure of the Appendix is as follows.
In Section~\ref{app:variation_background}, we provide preliminaries for the first-order and second-order variations on functionals. We prove Theorem \ref{thm:homo_learnability} and Theorem \ref{thm:beta_monotonic} in Section \ref{app:homo_learnability} and \ref{app:beta_monotonic}, respectively.
In Section~\ref{app:suff_1}, we prove Theorem \ref{thm:suff_1}, the sufficient condition 1 for IB-Learnability.
In Section~\ref{app:first_second_variations}, we calculate the first and second variations of $\IB_\beta[p(z|x)]$ at the trivial representation $p(z|x)=p(z)$, which is used in proving the Sufficient Condition 2 for $\IB_\beta$-learnability (Section~\ref{app:suff_2}). In Appendix \ref{app:what_first_learns}, we prove Eq. (\ref{eq:what_first_learns}) at the onset of learning.
After these preparations, we prove the key result of this paper, Theorem~\ref{thm:suff_3}, in Section~\ref{app:suff_3}.
Then two important corollaries \ref{corollary:suff_3_class_conditional}, \ref{corollary:suff_3_2} are proved in Section~\ref{app:corollaries}.
In Section~\ref{app:maximum_corr} we explore the deep relation between $\beta_0$, $\beta_0[h(x)]$, the hypercontractivity coefficient, contraction coefficient and maximum correlation.
Finally in Section~\ref{app:experiment}, we provide details for the experiments.

\subsection{Preliminaries: first-order and second-order variations}
\label{app:variation_background}

Let functional $F[f(x)]$ be defined on some normed linear space $\mathscr{R}$. Let us add a perturbative function $\epsilon\cdot h(x)$ to $f(x)$, and now the functional 
$F[f(x)+\epsilon\cdot h(x)]$ can be expanded as

\begin{equation*}
\begin{aligned}
\Delta F[f(x)]&=F[f(x)+\epsilon \cdot h(x)]-F[f(x)]\\
&=\varphi_1[f(x)]+\varphi_2[f(x)]+\O(\epsilon^3 ||h||^2)
\end{aligned}
\end{equation*}

where $||h||$ denotes the norm of $h$, $\varphi_1[f(x)]=\epsilon \frac{d F[f(x)]}{d\epsilon}$ is a linear functional of $\epsilon \cdot h(x)$, and is called the \emph{first-order variation}, denoted as $\delta F[f(x)]$. $\varphi_2[f(x)]=\frac{1}{2}\epsilon^2 \frac{d^2 F[f(x)]}{d\epsilon^2}$ is a quadratic functional of $\epsilon \cdot h(x)$, and is called the \emph{second-order variation}, denoted as $\delta^2 F[f(x)]$.

If $\delta F[f(x)]=0$, we call $f(x)$ a stationary solution for the functional $F[\cdot]$.

If $\Delta F[f(x)]\geq0$ for all $h(x)$ such that $f(x)+\epsilon\cdot h(x)$ is at the neighborhood of $f(x)$, we call $f(x)$ a (local) minimum of $F[\cdot]$.

\subsection{Proof of Lemma \ref{thm:homo_learnability}}
\label{app:homo_learnability}
\begin{proof}
If $(X,Y)$ is $\IB_\beta$-learnable, then there exists $Z\in\mathcal{Z}$ given by some $p_1(z|x)$ such that $\IB_\beta(X,Y;Z)<\IB(X,Y;Z_{trivial})=0$, where $Z_{trivial}$ satisfies $p(z|x)=p(z)$. Since $X'=g(X)$ is a invertible map (if $X$ is continuous variable, $g$ is additionally required to be continuous), and mutual information is invariant under such an invertible map (\cite{kraskov2004estimating}), we have that $\IB_\beta(X',Y;Z)=I(X';Z)-\beta I(Y;Z)=I(X;Z)-\beta I(Y;Z)=\IB_\beta(X,Y;Z)<0=\IB(X',Y;Z_{trivial})$, so $(X',Y)$ is $\IB_\beta$-learnable. On the other hand, if $(X,Y)$ is not $\IB_\beta$-learnable, then $\forall Z$, we have $\IB_\beta(X,Y;Z)\ge\IB(X,Y;Z_{trivial})=0$. Again using mutual information's invariance under $g$, we have for all $Z$, $\IB_\beta(X',Y;Z)=\IB_\beta(X,Y;Z)\ge\IB(X,Y;Z_{trivial})=0$, leading to that $(X',Y)$ is not $\IB_\beta$-learnable. Therefore, we have that $(X,Y)$ and $(X',Y)$ have the same $\IB_\beta$-learnability. 

\end{proof}

\subsection{Proof of Theorem \ref{thm:beta_monotonic}}
\label{app:beta_monotonic}

\begin{proof}
At the trivial representation $p(z|x)=p(z)$, we have $I(X;Z)=0$, and $I(Y;Z)=0$ due to the Markov chain, so $\IB_\beta(X,Y;Z)\rvert_{p(z|x)=p(z)} = 0$ for any $\beta$.
Since $(X,Y)$ is $\IB_{\beta_1}$-learnable, there exists a $Z$ given by a $p_1(z|x)$ such that $\IB_{\beta_1}(X,Y;Z)\rvert_{p_1(z|x)} < 0$.
Since $\beta_2 > \beta_1$, and $I(Y;Z) \geq 0$, we have $\IB_{\beta_2}(X,Y;Z)\rvert_{p_1(z|x)} \leq \IB_{\beta_1}(X,Y;Z)\rvert_{p_1(z|x)} < 0 = \IB_{\beta_2}(X,Y;Z)\rvert_{p(z|x)=p(z)}$.
Therefore, $(X,Y)$ is $\IB_{\beta_2}$-learnable.
\end{proof}

\subsection{Proof of Theorem \ref{thm:suff_1}}
\label{app:suff_1}

\begin{proof}
To prove Theorem \ref{thm:suff_1}, we use the Theorem 1 of Chapter 5 of \citet{gelfand2000calculus} which gives a necessary condition for $F[f(x)]$ to have a minimum at $f_0(x)$.
Adapting to our notation, we have:

\begin{theorem}[\cite{gelfand2000calculus}]
\label{thm:necessary_minimum}
A necessary condition for the functional $F[f(x)]$ to have a minimum at $f(x)=f_0(x)$ is that for $f(x)=f_0(x)$ and all admissible $\epsilon\cdot h(x)$, 
$$\delta^2 F[f(x)]\geq0$$.
\end{theorem}

Applying to our functional $\IB_\beta[p(z|x)]$, an immediate result of Theorem \ref{thm:necessary_minimum} is that, if at $p(z|x)=p(z)$, there exists an $\epsilon \cdot h(z|x)$ such that $\delta^2 \IB_\beta[p(z|x)]<0$, then $p(z|x)=p(z)$ is not a minimum for $\IB_\beta[p(z|x)]$. Using the definition of $\IB_\beta$ learnability, we have that $(X,Y)$ is $\IB_\beta$-learnable.

\end{proof}

\subsection{First- and second-order variations of \texorpdfstring{$IB_\beta[p(z|x)]$}{Lg}}
\label{app:first_second_variations}

In this section, we derive the first- and second-order variations of $\IB_\beta[p(z|x)]$, which are needed for proving Lemma \ref{lemma:stationary} and Theorem \ref{thm:suff_2}.

\begin{lemma}
\label{lemma:first_second_variation_IB}
\text{Using perturbative function $h(z|x)$, we have}

\begin{equation*}
\begin{aligned}
&\delta \IB_\beta[p(z|x)]=\int dx dz p(x) h(z|x)\emph{\log}\frac{p(z|x)}{p(z)}-\beta \int dx dy dz p(x,y)h(z|x)\emph{\log}\frac{p(z|y)}{p(z)}\\
&\delta^2 \IB_\beta[p(z|x)]=\\
&\frac{1}{2}\bigg[\int dxdz\frac{p(x)^2}{p(x,z)}h(z|x)^2-\beta\int dx dx' dy dz\frac{p(x,y)p(x',y)}{p(y,z)}h(z|x)h(z|x')+\\
&(\beta-1)\int dx dx' dz \frac{p(x)p(x')}{p(z)}h(z|x)h(z|x')\bigg]
\end{aligned}
\end{equation*}
\end{lemma}

\begin{proof}
Since $\IB_\beta[p(z|x)]=I(X;Z)-\beta I(Y;Z)$, let us calculate the first and second-order variation of $I(X;Z)$ and $I(Y;Z)$ w.r.t. $p(z|x)$, respectively. Through this derivation, we use $\epsilon\cdot h(z|x)$ as a perturbative function, for ease of deciding different orders of variations. We assume that $h(z|x)$ is continuous, and there exists a constant $M$ such that $\big|\frac{h(z|x)}{p(z|x)}\big|<M$, $\forall (x,z)\in\X\times\mathcal{Z}$. We will finally absorb $\epsilon$ into $h(z|x)$.

Denote $I(X;Z)=F_1[p(z|x)]$. We have
\begin{equation*}
F_1[p(z|x)]=I(X;Z)=\int dx dz p(z|x)p(x)\log\frac{p(z|x)}{p(z)}
\end{equation*}

In this paper, we implicitly assume that the integral (or summing) are only on the support of $p(x,y,z)$.

Since
$$p(z)=\int p(z|x)p(x)dx$$
We have
$$p(z)\rvert_{p(z|x)+\epsilon h(z|x)}=p(z)\rvert_{p(z|x)}+\epsilon\int h(z|x)p(x)dx$$

Expanding $F_1[p(z|x)+\epsilon h(z|x)]$ to the second order of $\epsilon$, we have

\begin{equation*}
\begin{aligned}
&F_1[p(z|x)+\epsilon h(z|x)]\\
&=\int dx dz p(x)[p(z|x)+\epsilon h(z|x)]\log\frac{p(z|x)+\epsilon h(z|x)}{p(z)+\epsilon\int h(z|x')p(x')dx'}\\
&=\int dx dz p(x)p(z|x)\left(1+\epsilon\frac{ h(z|x)}{p(z|x)}\right)\log\frac{p(z|x)\left(1+\epsilon\frac{h(z|x)}{p(z|x)}\right)}{p(z)\left(1+\epsilon\frac{\int h(z|x')p(x')dx'}{p(z)}\right)}\\
&=\int dx dz p(x)p(z|x)\left(1+\epsilon\frac{ h(z|x)}{p(z|x)}\right)\log\bigg[\frac{p(z|x)}{p(z)}\bigg(1+\epsilon\frac{h(z|x)}{p(z|x)}\bigg)\bigg(1-\epsilon\frac{\int h(z|x')p(x')dx'}{p(z)}\\
&+\epsilon^2\bigg(\frac{\int h(z|x')p(x')dx'}{p(z)}\bigg)^2\bigg)\bigg]+\O(\epsilon^3)\\
&=\int dx dz p(x)p(z|x)\left(1+\epsilon\frac{ h(z|x)}{p(z|x)}\right)\log\bigg[\frac{p(z|x)}{p(z)}\bigg(1+\epsilon\bigg(\frac{h(z|x)}{p(z|x)}-\frac{\int h(z|x')p(x')dx'}{p(z)}\bigg)\\
&+\epsilon^2\left(\frac{\int h(z|x')p(x')dx'}{p(z)}\right)^2-\epsilon^2\frac{h(z|x)}{p(z|x)}\frac{\int h(z|x')p(x')dx'}{p(z)}\bigg)\bigg]+\O(\epsilon^3)\\
&=\int dx dz p(x)p(z|x)\left(1+\epsilon\frac{ h(z|x)}{p(z|x)}\right)\bigg[\log\frac{p(z|x)}{p(z)}+\epsilon\bigg(\frac{h(z|x)}{p(z|x)}-\frac{\int h(z|x')p(x')dx'}{p(z)}\bigg)\\
&+\epsilon^2\left(\frac{\int h(z|x')p(x')dx'}{p(z)}\right)^2-\epsilon^2\frac{h(z|x)}{p(z|x)}\frac{\int h(z|x')p(x')dx'}{p(z)}-\frac{1}{2}\epsilon^2\bigg(\frac{h(z|x)}{p(z|x)}-\frac{\int h(z|x')p(x')dx'}{p(z)}\bigg)^2\bigg]+\O(\epsilon^3)\\
\end{aligned}
\end{equation*}

Collecting the first order terms of $\epsilon$, we have
\begin{equation*}
\begin{aligned}
&\delta F_1[p(z|x)]\\
&=\epsilon\int dx dz p(x)p(z|x)\bigg(\frac{h(z|x)}{p(z|x)}-\frac{\int h(z|x')p(x')dx'}{p(z)}\bigg)+\epsilon\int dx dz p(x)p(z|x)\frac{ h(z|x)}{p(z|x)}\log\frac{p(z|x)}{p(z)}\\
&=\epsilon\int dx dz p(x)h(z|x)-\epsilon\int dx' dz p(x')h(z|x')+\epsilon\int dx dz p(x) h(z|x)\log\frac{p(z|x)}{p(z)}\\
&=\epsilon\int dx dz p(x) h(z|x)\log\frac{p(z|x)}{p(z)}\\
\end{aligned}
\end{equation*}

Collecting the second order terms of $\epsilon^2$, we have

\begin{equation*}
\begin{aligned}
&\delta^2 F_1[p(z|x)]\\
&=\epsilon^2\int dx dz p(x)p(z|x)\bigg[\left(\frac{\int h(z|x')p(x')dx'}{p(z)}\right)^2-\frac{h(z|x)}{p(z|x)}\frac{\int h(z|x')p(x')dx'}{p(z)}\\
&-\frac{1}{2}\bigg(\frac{h(z|x)}{p(z|x)}-\frac{\int h(z|x')p(x')dx'}{p(z)}\bigg)^2\bigg]\\
&+\epsilon^2\int dx dz p(x)p(z|x)\frac{ h(z|x)}{p(z|x)}\bigg(\frac{h(z|x)}{p(z|x)}-\frac{\int h(z|x')p(x')dx'}{p(z)}\bigg)\\
&=\frac{\epsilon^2}{2}\int dxdz\frac{p(x)^2}{p(x,z)}h(z|x)^2-\frac{\epsilon^2}{2}\int dx dx' dz \frac{p(x)p(x')}{p(z)}h(z|x)h(z|x')
\end{aligned}
\end{equation*}

Now let us calculate the first and second-order variation of $F_2[p(z|x)]=I(Z;Y)$. We have
\begin{equation*}
F_2[p(z|x)]=I(Y;Z)=\int dy dz p(z|y)p(y)\log\frac{p(y,z)}{p(y)p(z)}=\int dx dy dz p(z|y)p(x,y)\log\frac{p(y,z)}{p(y)p(z)}
\end{equation*}
Using the Markov chain $Z\gets X\leftrightarrow Y$, we have
$$p(y,z)=\int p(z|x)p(x,y)dx$$
Hence
$$p(y,z)\rvert_{p(z|x)+\epsilon h(z|x)}=p(y,z)\rvert_{p(z|x)}+\epsilon\int h(z|x)p(x,y)dx$$

Then expanding $F_2[p(z|x)+\epsilon h(z|x)]$ to the second order of $\epsilon$, we have

\begin{equation*}
\begin{aligned}
&F_2[p(z|x)+\epsilon h(z|x)]\\
&=\int dx dy dz p(x,y)p(z|x)\left(1+\epsilon\frac{ h(z|x)}{p(z|x)}\right)\log\frac{p(y,z)\left(1+\epsilon\frac{\int h(z|x')p(x',y)dx'}{p(y,z)}\right)}{p(y)p(z)(1+\epsilon\frac{\int h(z|x'')p(x'')dx''}{p(z)})}\\
&=\int dx dy dz p(x,y)p(z|x)\left(1+\epsilon\frac{ h(z|x)}{p(z|x)}\right)\bigg\{\log\frac{p(y,z)}{p(y)p(z)}\\
&+\epsilon\bigg(\frac{\int h(z|x')p(x',y)dx'}{p(y,z)}-\frac{\int h(z|x')p(x')dx'}{p(z)}\bigg)\\
&+\epsilon^2\bigg[\bigg(\frac{\int h(z|x')p(x')dx'}{p(z)}\bigg)^2-\frac{\int h(z|x')p(x',y)dx'}{p(y,z)}\frac{\int h(z|x'')p(x'')dx''}{p(z)}\\
&-\frac{1}{2}\bigg(\frac{\int h(z|x')p(x',y)dx'}{p(y,z)}-\frac{\int h(z|x')p(x')dx'}{p(z)}\bigg)^2\bigg]\bigg\}\\
&+\O(\epsilon^3)\\
\end{aligned}
\end{equation*}

Collecting the first order terms of $\epsilon$, we have
\begin{equation*}
\begin{aligned}
&\delta F_2[p(z|x)]\\
&=\epsilon\int dx dy dz p(x,y)h(z|x)\log\frac{p(y,z)}{p(y)p(z)}+\epsilon\int dx dy dz p(x,y)p(z|x)\frac{\int h(z|x')p(x',y)dx'}{p(y,z)}\\
&-\epsilon\int dx dy dz p(x,y)p(z|x)\frac{\int h(z|x')p(x')dx'}{p(z)}\\
&=\epsilon\int dx dy dz p(x,y)h(z|x)\log\frac{p(y,z)}{p(y)p(z)}+\epsilon\int dx'dy dz h(z|x')p(x',y)-\epsilon\int dz h(z|x')p(x')dx'\\
&=\epsilon\int dx dy dz p(x,y)h(z|x)\log\frac{p(z|y)}{p(z)}
\end{aligned}
\end{equation*}

Collecting the second order terms of $\epsilon$, we have

\begin{equation*}
\begin{aligned}
&\delta^2 F_2[p(z|x)]\\
&=\epsilon^2\int dx dy dz p(x,y)p(z|x)\bigg[\bigg(\frac{\int h(z|x')p(x')dx'}{p(z)}\bigg)^2-\frac{\int h(z|x')p(x',y)dx'}{p(y,z)}\frac{\int h(z|x'')p(x'')dx''}{p(z)}\bigg]\\
&-\frac{\epsilon^2}{2}\int dx dy dz p(x,y)p(z|x)\bigg(\frac{\int h(z|x')p(x',y)dx'}{p(y,z)}-\frac{\int h(z|x')p(x')dx'}{p(z)}\bigg)^2\\
&+\epsilon^2\int dx dy dz p(x,y)p(z|x)\frac{ h(z|x)}{p(z|x)}\bigg(\frac{\int h(z|x')p(x',y)dx'}{p(y,z)}-\frac{\int h(z|x')p(x')dx'}{p(z)}\bigg)\\
&=\frac{\epsilon^2}{2}\int dx dx' dy dz\frac{p(x,y)p(x',y)}{p(y,z)}h(z|x)h(z|x')-\frac{\epsilon^2}{2}\int dx dx' dz \frac{p(x)p(x')}{p(z)}h(z|x)h(z|x')
\end{aligned}
\end{equation*}

Finally, we have
\begin{equation}
\label{eq:delta_IB}
\begin{aligned}
\delta \IB_\beta[p(z|x)]&=\delta F_1[p(z|x)]-\beta \cdot\delta F_2[p(z|x)]\\
&=\epsilon\bigg(\int dx dz p(x) h(z|x)\log\frac{p(z|x)}{p(z)}-\beta \int dx dy dz p(x,y)h(z|x)\log\frac{p(z|y)}{p(z)}\bigg)
\end{aligned}
\end{equation}

\begin{equation*}
\begin{aligned}
\delta^2 \IB_\beta[p(z|x)]=&\delta^2 F_1[p(z|x)]-\beta \cdot\delta^2 F_2[p(z|x)]\\
=&\frac{\epsilon^2}{2}\int dxdz\frac{p(x)^2}{p(x,z)}h(z|x)^2-\frac{\epsilon^2}{2}\int dx dx' dz \frac{p(x)p(x')}{p(z)}h(z|x)h(z|x')\\
&-\beta\epsilon^2\bigg[\frac{1}{2}\int dx dx' dy dz\frac{p(x,y)p(x',y)}{p(y,z)}h(z|x)h(z|x')\\
&-\frac{1}{2}\int dx dx' dz \frac{p(x)p(x')}{p(z)}h(z|x)h(z|x')\bigg]\\
=&\frac{\epsilon^2}{2}\bigg[\int dxdz\frac{p(x)^2}{p(x,z)}h(z|x)^2\\
&-\beta\int dx dx' dy dz\frac{p(x,y)p(x',y)}{p(y,z)}h(z|x)h(z|x')\\
&+(\beta-1)\int dx dx' dz \frac{p(x)p(x')}{p(z)}h(z|x)h(z|x')\bigg]
\end{aligned}
\end{equation*}

Absorb $\epsilon$ into $h(z|x)$, we get rid of the $\epsilon$ factor and obtain the final expression in Lemma \ref{lemma:first_second_variation_IB}.

\end{proof}

\subsection{Proof of Lemma \ref{lemma:stationary}}
\label{app:stationary}

\begin{proof}
Using Lemma \ref{lemma:first_second_variation_IB}, we have
\begin{equation*}
\begin{aligned}
\delta \IB_\beta[p(z|x)]=\int dx dz p(x) h(z|x)\log\frac{p(z|x)}{p(z)}-\beta \int dx dy dz p(x,y)h(z|x)\log\frac{p(z|y)}{p(z)}
\end{aligned}
\end{equation*}
Let $p(z|x)=p(z)$ (the trivial representation), we have that $\log\frac{p(z|x)}{p(z)}\equiv0$. Therefore, the two integrals are both 0. Hence,
\begin{equation*}
\begin{aligned}
\delta \IB_\beta[p(z|x)]\big\rvert_{p(z|x)=p(z)}\equiv0
\end{aligned}
\end{equation*}
Therefore, the $p(z|x)=p(z)$ is a stationary solution for $\IB_\beta[p(z|x)]$.

\end{proof}

\subsection{Proof of Theorem \ref{thm:suff_2}}
\label{app:suff_2}

\begin{proof}

Firstly, from the necessary condition of $\beta>1$ in Section \ref{sec:learnability}, we have that any sufficient condition for $\IB_\beta$-learnability should be able to deduce $\beta>1$.

Now using Theorem \ref{thm:suff_1}, a sufficient condition for $(X,Y)$ to be $\IB_\beta$-learnable is that there exists $h(z|x)$ with $\int h(z|x)dx=0$ such that $\delta^2\IB_\beta[p(z|x)]<0$ at $p(z|x)=p(x)$.

At the trivial representation, $p(z|x)=p(z)$ and hence $p(x,z)=p(x)p(z)$. Due to the Markov chain $Z\gets X\leftrightarrow Y$, we have $p(y,z)=p(y)p(z)$. Substituting them into the $\delta^2\IB_\beta[p(z|x)]$ in Lemma \ref{lemma:first_second_variation_IB}, the condition becomes: there exists $h(z|x)$ with $\int h(z|x)dz=0$, such that
\begin{equation}
\label{eq:suff_2_app1}
\begin{aligned}
&0>\delta^2 \IB_\beta[p(z|x)]=\\
&\frac{1}{2}\bigg[\int dxdz\frac{p(x)^2}{p(x)p(z)}h(z|x)^2-\beta\int dx dx' dy dz\frac{p(x,y)p(x',y)}{p(y)p(z)}h(z|x)h(z|x')\\
&+(\beta-1)\int dx dx' dz \frac{p(x)p(x')}{p(z)}h(z|x)h(z|x')\bigg]
\end{aligned}
\end{equation}

Rearranging terms and simplifying, we have
\begin{equation*}
\begin{aligned}
\int \frac{dz}{p(z)}G[h(z|x)]=&\int \frac{dz}{p(z)}\bigg[\int dx h(z|x)^2p(x)-\beta\int\frac{dy}{p(y)}\bigg(\int dx h(z|x)p(x)p(y|x)\bigg)^2\\
&+(\beta-1)\bigg(\int dx h(z|x)p(x)\bigg)^2\bigg]<0
\end{aligned}
\end{equation*}

where 
$$G[h(x)]=\int dx h(x)^2p(x)-\beta\int\frac{dy}{p(y)}\bigg(\int dx h(x)p(x)p(y|x)\bigg)^2+(\beta-1)\bigg(\int dx h(x)p(x)\bigg)^2$$

Now we prove that the condition that $\exists h(z|x)$ s.t. $\int\frac{dz}{p(z)}G[h(z|x)]<0$ is equivalent to the condition that $\exists h(x)$ s.t. $G[h(x)]<0$.

If $\forall h(z|x)$, $G[h(z|x)]\ge0$, then we have $\forall h(z|x)$, $\int \frac{dz}{p(z)}G[h(z|x)]\ge0$. Therefore, if $\exists h(z|x)$ s.t. $\int\frac{dz}{p(z)}G[h(z|x)]<0$, we have that $\exists h(z|x)$ s.t. $G[h(z|x)]<0$. 
Since the functional $G[h(z|x)]$ does not contain integration over $z$, we can treat the $z$ in $G[h(z|x)]$ as a parameter and we have that $\exists h(x)$ s.t. $G[h(x)]<0$.

Conversely, if there exists an certain function $h(x)$ such that $G[h(x)]<0$, we can find some $h_2(z)$ such that $\int h_2(z)dz=0$ and $\int\frac{h^2_2(z)}{p(z)}dz>0$, and let $h_1(z|x)=h(x)h_2(z)$. Now we have

$$\int\frac{dz}{p(z)}G[h(z|x)]=\int\frac{h_2^2(z)dz}{p(z)}G[h(x)]=G[h(x)]\int\frac{h_2^2(z)dz}{p(z)}<0$$

In other words, the condition Eq. (\ref{eq:suff_2_app1}) is equivalent to requiring that there exists an $h(x)$ such that $G[h(x)]<0$
. Hence, a sufficient condition for $\IB_\beta$-learnability is that there exists an $h(x)$ such that

\begin{equation}
\label{eq:suff_2_app_0}
\begin{aligned}
G[h(x)]=\int dx h(x)^2p(x)-\beta\int\frac{dy}{p(y)}\bigg(\int dx h(x)p(x)p(y|x)\bigg)^2+(\beta-1)\bigg(\int dx h(x)p(x)\bigg)^2<0
\end{aligned}
\end{equation}

When $h(x)=C=\text{constant}$ in the entire input space $\X$, Eq. (\ref{eq:suff_2_app_0}) becomes:

\begin{equation*}
\begin{aligned}
C^2-\beta C^2+(\beta-1)C^2<0
\end{aligned}
\end{equation*}

which cannot be true. Therefore, $h(x)=\text{constant}$ cannot satisfy Eq. (\ref{eq:suff_2_app_0}).

Rearranging terms and simplifying, we have

\begin{equation}
\label{eq:suff_2_app_1}
\begin{aligned}
\beta\bigg[\int\frac{dy}{p(y)}\left(\int dx h(x)p(x)p(y|x)\right)^2-\left(\int dx h(x)p(x)\right)^2\bigg]>\int dx h(x)^2 p(x)-\left(\int dx h(x)p(x)\right)^2
\end{aligned}
\end{equation}

Written in the form of expectations, we have

\begin{equation}
\begin{aligned}
\label{eq:suff_2_app_2}
\beta\cdot \left(\E_{y \sim p(y)}\left[\left(\E_{x \sim p(x|y)} [h(x)]\right)^2\right] - \left(\E_{x\sim p(x)} [h(x)]\right)^2\right)>\E_{x \sim p(x)} [h(x)^2] - \left(\E_{x\sim p(x)} [h(x)]\right)^2
\end{aligned}
\end{equation}

Since the square function is convex, using Jensen's inequality on the L.H.S. of Eq. \ref{eq:suff_2_app_2}, we have

\begin{equation*}
\begin{aligned}
\E_{y\sim p(y)}\bigg[\bigg(\E_{x\sim p(x|y)} [h(x)]\bigg)^2\bigg]\ge\bigg(\E_{y\sim p(y)}\bigg[\E_{x\sim p(x|y)} [h(x)]\bigg]\bigg)^2 = \left(\E_{x\sim p(x)}[h(x)]\right)^2
\end{aligned}
\end{equation*}

The equality holds iff $\E_{x\sim p(x|y)} [h(x)]$ is constant w.r.t. $y$, i.e. $Y$ is independent of $X$. Therefore, in order for Eq. (\ref{eq:suff_2_app_2}) to hold, we require that $Y$ is not independent of $X$.

Using Jensen's inequality on the innter expectation on the L.H.S. of Eq. (\ref{eq:suff_2_app_2}), we have

\begin{equation}
\label{eq:suff_2_app_22}
\begin{aligned}
\E_{y\sim p(y)}\bigg[\bigg(\E_{x\sim p(x|y)} [h(x)]\bigg)^2\bigg]\le\E_{y\sim p(y)}\big[\E_{x\sim p(x|y)} [h(x)^2]\big] = \E_{x\sim p(x)}[h(x)^2]
\end{aligned}
\end{equation}

The equality holds when $h(x)$ is a constant. Since we require that $h(x)$ is not a constant, we have that the equality cannot be reached.

Similarly, using Jensen's inequality on the R.H.S. of Eq. \ref{eq:suff_2_app_2}, we have that $$\E_{x \sim p(x)} [h(x)^2] > \left(\E_{x\sim p(x)} [h(x)]\right)^2$$ 

where we have used the requirement that $h(x)$ cannot be constant.

Under the constraint that $Y$ is not independent of $X$, we can divide both sides of Eq. \ref{eq:suff_2_app_2}, and obtain the condition: there exists an $h(x)$ such that

\begin{equation*}
\begin{aligned}
\beta>\frac{\E_{x \sim p(x)} [h(x)^2] - \left(\E_{x\sim p(x)} [h(x)]\right)^2}{\E_{y \sim p(y)}\left[\left(\E_{x \sim p(x|y)} [h(x)]\right)^2\right] - \left(\E_{x\sim p(x)} [h(x)]\right)^2}
\end{aligned}
\end{equation*}

i.e.

\begin{equation*}
\begin{aligned}
\beta>\inf_{h(x)}\frac{\E_{x \sim p(x)} [h(x)^2] - \left(\E_{x\sim p(x)} [h(x)]\right)^2}{\E_{y \sim p(y)}\left[\left(\E_{x \sim p(x|y)} [h(x)]\right)^2\right] - \left(\E_{x\sim p(x)} [h(x)]\right)^2}
\end{aligned}
\end{equation*}

which proves the condition of Theorem \ref{thm:suff_2}. 

Furthermore, from Eq. (\ref{eq:suff_2_app_22}) we have

\begin{equation*}
\begin{aligned}
\beta_0[h(x)]>1
\end{aligned}
\end{equation*}

for $h(x)\not\equiv$ const, which satisfies the necessary condition of $\beta>1$ in Section \ref{sec:learnability}.

\textbf{Proof of lower bound of slope of the Pareto frontier at the origin:} 

Now we prove the second statement of Theorem \ref{thm:suff_2}. Since $\delta I(X;Z)=0$ and $\delta I(Y;Z)=0$ according to Lemma \ref{lemma:stationary}, we have $\left(\frac{\Delta I(Y;Z)}{\Delta I(X;Z)}\right)^{-1}=\left(\frac{\delta^2 I(Y;Z)}{\delta^2 I(X;Z)}\right)^{-1}$. Substituting into the expression of $\delta^2 I(Y;Z)$ and $\delta^2 I(X;Z)$ from Lemma \ref{lemma:first_second_variation_IB}, we have 

\begin{equation*}
\begin{aligned}
&\left(\frac{\Delta I(Y;Z)}{\Delta I(X;Z)}\right)^{-1}\\
&=\left(\frac{\delta^2 I(Y;Z)}{\delta^2 I(X;Z)}\right)^{-1}\\
&=\frac{\frac{\epsilon^2}{2}\int dxdz\frac{p(x)^2}{p(x)p(z)}h(z|x)^2-\frac{\epsilon^2}{2}\int dx dx' dz \frac{p(x)p(x')}{p(z)}h(z|x)h(z|x')}{\frac{\epsilon^2}{2}\int dx dx' dy dz\frac{p(x,y)p(x',y)}{p(y)p(z)}h(z|x)h(z|x')-\frac{\epsilon^2}{2}\int dx dx' dz \frac{p(x)p(x')}{p(z)}h(z|x)h(z|x')}\\
&=\frac{\left(\int dx p(x)h(x)^2-\int dx dx'  p(x)p(x')h(x)h(z|x')\right)\int \frac{h_2(z)^2}{p(z)}dz}{\left(\int dx dx' dy \frac{p(x,y)p(x',y)}{p(y)}h(x)h(z|x')-\int dx dx' p(x)p(x')h(x)h(z|x')\right)\int \frac{h_2(z)^2}{p(z)}dz}\\
&=\frac{\int dx p(x)h(x)^2-\int dx dx'  p(x)p(x')h(x)h(z|x')}{\int dx dx' dy \frac{p(x,y)p(x',y)}{p(y)}h(x)h(z|x')-\int dx dx' p(x)p(x')h(x)h(z|x')}\\
&=\frac{\E_{x\sim p(x)}[ h(x)^2]-\left(\E_{x\sim p(x)} [h(x)]\right)^2}{\E_{y\sim p(y)}\big[\left(\E_{x\sim p(x|y)} [h(x)]\right)^2\big]-\left(\E_{x\sim p(x)} [h(x)]\right)^2}\\
&=\frac{\frac{\E_{x\sim p(x)}[ h(x)^2]}{\left(\E_{x\sim p(x)} [h(x)]\right)^2}-1}{\E_{y\sim p(y)}\bigg[\left(\frac{\E_{x\sim p(x|y)} [h(x)]}{\E_{x\sim p(x)}[h(x)]}\right)^2\bigg]-1}\\
&=\beta_0[h(x)]
\end{aligned}
\end{equation*}

Therefore, $\left(\inf_{h(x)}\beta_0[h(x)]\right)^{-1}$ gives the largest slope of $\Delta I(Y;Z)$ vs. $\Delta I(X;Z)$ for perturbation function of the form $h_1(z|x)=h(x)h_2(z)$ satisfying $\int h_2(z)dz=0$ and $\int\frac{h_2^2(z)}{p(z)}dz>0$, which is a lower bound of slope of $\Delta I(Y;Z)$ vs. $\Delta I(X;Z)$ for all possible perturbation function $h_1(z|x)$. The latter is the slope of the Pareto frontier of the $I(Y;Z)$ vs. $I(X;Z)$ curve at the origin.

\textbf{Inflection point for general $Z$:} If we \emph{do not} assume that $Z$ is at the origin of the information plane, but at some general stationary solution $Z^*$ with $p(z|x)$, we define

\begin{equation*}
\begin{aligned}
\beta^{(2)}[h(x)]&=\left(\frac{\delta^2 I(Y;Z)}{\delta^2 I(X;Z)}\right)^{-1}\\
&=\frac{\frac{\epsilon^2}{2}\int dxdz\frac{p(x)^2}{p(x,z)}h(z|x)^2-\frac{\epsilon^2}{2}\int dx dx' dz \frac{p(x)p(x')}{p(z)}h(z|x)h(z|x')}{\frac{\epsilon^2}{2}\int dx dx' dy dz\frac{p(x,y)p(x',y)}{p(y,z)}h(z|x)h(z|x')-\frac{\epsilon^2}{2}\int dx dx' dz \frac{p(x)p(x')}{p(z)}h(z|x)h(z|x')}\\
&=\frac{\int dxdz\frac{p(x)^2}{p(x,z)}h(z|x)^2-\int dx dx' dz \frac{p(x)p(x')}{p(z)}h(z|x)h(z|x')}{\int dx dx' dy dz\frac{p(x,y)p(x',y)}{p(y,z)}h(z|x)h(z|x')-\int dx dx' dz \frac{p(x)p(x')}{p(z)}h(z|x)h(z|x')}\\
&=\frac{\int \frac{dz}{p(z)}\left[\int dx\frac{p(x)^2}{p(x|z)}h(z|x)^2-\left(\int dx p(x)h(z|x)\right)^2\right]}{\int\frac{dz}{p(z)}\left[\int \frac{dy}{p(y|z)} \left(\int dx p(x,y)h(z|x)\right)^2-\left(\int dx p(x)h(z|x)\right)^2\right]}\\
&=\frac{\int \frac{dz}{p(z)}\left[\frac{\int dx\frac{p(x)^2}{p(x|z)}h(z|x)^2}{\left(\int dx p(x)h(z|x)\right)^2}-1\right]}{\int\frac{dz}{p(z)}\left[\frac{\int \frac{dy}{p(y|z)} \left(\int dx p(x,y)h(z|x)\right)^2}{\left(\int dx p(x)h(z|x)\right)^2}-1\right]}\\
&=\frac{\int dz\left[\frac{\int dx\frac{p(x)}{p(z|x)}h(z|x)^2}{\left(\int dx p(x)h(z|x)\right)^2}-\frac{1}{p(z)}\right]}{\int dz\left[\frac{\int \frac{dy}{p(z|y)p(y)} \left(\int dx p(x,y)h(z|x)\right)^2}{\left(\int dx p(x)h(z|x)\right)^2}-\frac{1}{p(z)}\right]}\\
&=\frac{\int dz\left[\int dx\frac{p(x)}{p(z|x)}h(z|x)^2-\frac{1}{p(z)}(\int dx p(x)h(z|x))^2\right]}{\int dz\left[\int \frac{dy}{p(z|y)p(y)} \left(\int dx p(x,y)h(z|x)\right)^2-\frac{1}{p(z)}\left(\int dx p(x)h(z|x)\right)^2\right]}\\
\end{aligned}
\end{equation*}

which reduces to $\beta_0[h(x)]$ when $p(z|x)=p(z)$.
When
\begin{equation}
\label{eq:general_beta}
\beta>\inf_{h(z|x)}\beta^{(2)}[h(z|x)]
\end{equation}
it becomes a non-stable solution (non-minimum), and we will have other $Z$ that achieves a better $\IB_\beta(X,Y;Z)$ than the current $Z^*$. 

\end{proof}

\subsection{What IB first learns at its onset of learning}
\label{app:what_first_learns}

In this section, we prove that at the onset of learning, if letting $h(z|x)=h^*(x)h_2(z)$, we have

\begin{equation}
p_\beta(y|x)=p(y)+\epsilon^2 C_z (h^*(x)-\overline{h}^*_x)\int p(x,y)(h^*(x)-\overline{h}^*_x)dx
\end{equation}

where $p_\beta(y|x)$ is the estimated $p(y|x)$ by IB for a certain $\beta$, $h^*(x)=\inf_{h(x)}\beta_0[h(x)]$, $\overline{h}^*_x=\int h^*(x)p(x)dx$, $C_z=\int\frac{h_2^2(z)}{p(z)}dz$ is a constant.

\begin{proof}
In IB, we use $p_\beta(z|x)$ to obtain $Z$ from $X$, then obtain the prediction of $Y$ from $Z$ using $p_\beta(y|z)$. Here we use subscript $\beta$ to denote the probability (density) at the optimum of $\IB_\beta[p(z|x)]$ at a specific $\beta$. We have
\begin{equation*}
\begin{aligned}
p_\beta(y|x)&=\int p_\beta(y|z) p_\beta(z|x)dz \\
&=\int dz \frac{p_\beta(y,z) p_\beta(z|x)}{p_\beta(z)}\\
&=\int dz \frac{p_\beta(z|x)}{p_\beta(z)}\int p(x',y)p_\beta(z|x')dx'
\end{aligned}
\end{equation*}

When we have a small perturbation $\epsilon\cdot h(z|x)$ at the trivial representation, $p_\beta(z|x)=p_{\beta_0}(z)+\epsilon\cdot h(z|x)$, we have $p_\beta(z)=p_{\beta_0}(z)+\epsilon\cdot\int h(z|x'')p(x'')dx''$. Substituting, we have

\begin{equation*}
\begin{aligned}
p_\beta(y|x)&=\int dz \frac{p_{\beta_0}(z)\left(1+\epsilon\cdot\frac{h(z|x)}{p_{\beta_0}(z)}\right)}{p_{\beta_0}(z)\left(1+\epsilon\cdot\frac{\int h(z|x'')p(x'')dx''}{p_{\beta_0}(z)}\right)}\int p(x',y)p_{\beta_0}(z)\left(1+\epsilon\cdot\frac{h(z|x')}{p_{\beta_0}(z)}\right)dx'\\
&=\int dz \frac{1+\epsilon\cdot\frac{h(z|x)}{p_{\beta_0}(z)}}{1+\epsilon\cdot\frac{\int h(z|x'')p(x'')dx''}{p_{\beta_0}(z)}}\int p(x',y)p_{\beta_0}(z)\left(1+\epsilon\cdot\frac{h(z|x')}{p_{\beta_0}(z)}\right)dx'\\
\end{aligned}
\end{equation*}

The $0^{\text{th}}$-order term is $\int dz dx' p(x',y)p_{\beta_0}(z)=p(y)$. The first-order term is

\begin{equation*}
\begin{aligned}
\delta p_\beta(z|x)=&\epsilon\cdot\int dzdx'\left(h(z|x) + h(z|x')-\int h(z|x'')p(x'')dx''\right)p(x',y)\\
=&\epsilon\cdot\int dx' \left(\int dz h(z|x)+\int dz h(z|x')\right)-\epsilon\cdot\int dx'dx''p(x',y)p(x'')\int dz h(z|x'')\\
=&0-0\\
=&0
\end{aligned}
\end{equation*}

since we have $\int h(z|x)dz=0$ for any $x$.

For the second-order term, using $h(z|x)=h^*(x)h_2(z)$ and $C_z=\int\frac{dz}{p_{\beta_0}(z)}h_2^2(z)$, it is
\begin{equation*}
\begin{aligned}
\delta^2 p_\beta(y|x)=&\epsilon^2\cdot\int dz\left(\frac{\int h(z|x'')p(x'')dx''}{p_{\beta_0}(z)}\right)^2 \int p(x',y)p_{\beta_0}(z)dx'\\
&-\epsilon^2\cdot \int dz\frac{h(z|x)\int h(z|x'')p(x'')dx''}{(p_{\beta_0}(z))^2} \int p(x',y)p_{\beta_0}(z)dx'\\
&+\epsilon^2 \int dz\left(h(z|x)-\int h(z|x'')p(x'')dx\right)\int p(x',y)\frac{h(z|x')}{p_{\beta_0}(z)}dx'\\
=&\epsilon^2 C_z\cdot\left(\int h^*(x'')p(x'')dx''\right)^2 p(y)\\
&-\epsilon^2 C_z\cdot h^*(x)\int h^*(x'')p(x'')dx'' p(y)\\
&+\epsilon^2 C_z\cdot h^*(x)\int p(x',y)h^*(x')dx'\\
&-\epsilon^2 C_z\cdot\int h^*(x'')p(x'')dx\int p(x',y)h^*(x')dx'\\
=&\epsilon^2 C_z(h^*(x)-\overline{h}^*_x)\left[\left(\int p(x',y)h^*(x')dx'\right)-\overline{h}^*_x p(y)\right]\\
=&\epsilon^2 C_z(h^*(x)-\overline{h}^*_x)\int p(x',y)\left(h^*(x')-\overline{h}^*_x\right)dx'
\end{aligned}
\end{equation*}

where $\overline{h}^*_x=\int h^*(x)p(x)dx$. Combining everything, we have up to the second order,

\begin{equation*}
p_\beta(y|x)=p(y)+\epsilon^2 C_z (h^*(x)-\overline{h}^*_x)\int p(x,y)(h^*(x)-\overline{h}^*_x)dx
\end{equation*}
\end{proof}

\subsection{Proof of Theorem \ref{thm:suff_3}}
\label{app:suff_3}

\begin{proof}
According to Theorem \ref{thm:suff_2}, a sufficient condition for $(X,Y)$ to be $\IB_\beta$-learnable is that $X$ and $Y$ are not independent, and
\begin{equation}
\label{eq:suff_3_app_1}
\begin{aligned}
\beta>\inf_{h(x)}\frac{\frac{\E_{x\sim p(x)}[ h(x)^2]}{\left(\E_{x\sim p(x)} [h(x)]\right)^2}-1}{\E_{y\sim p(y)}\bigg[\left(\frac{\E_{x\sim p(x|y)} [h(x)]}{\E_{x\sim p(x)}[h(x)]}\right)^2\bigg]-1}
\end{aligned}
\end{equation}

We can assume a specific form of $h(x)$, and obtain a (potentially stronger) sufficient condition. Specifically, we let

\begin{equation}
\label{eq:suff_3_app_2}
\begin{aligned}
h(x)=\begin{cases}
1, x\in\Omega_x\\
0, \text{otherwise}
\end{cases}
\end{aligned}
\end{equation}

for certain $\Omega_x\subset \X$. Substituting into Eq. (\ref{eq:suff_3_app_2}), we have that a sufficient condition for $(X,Y)$ to be $\IB_\beta$-learnable is

\begin{equation}
\label{eq:suff_3_app_3}
\begin{aligned}
\beta>\inf_{\Omega_x\subset\X}\frac{\frac{p(\Omega_x)}{p(\Omega_x)^2}-1}{\int dy p(y)\left(\frac{\int_{x\in\Omega_x} dx p(x|y) dx}{p(\Omega_x)}\right)^2-1}>0
\end{aligned}
\end{equation}

where $p(\Omega_x)=\int_{x\in\Omega_x}p(x)dx$.

The denominator of Eq. (\ref{eq:suff_3_app_3}) is

\begin{equation*}
\begin{aligned}
&\int dy p(y)\left(\frac{\int_{x\in\Omega_x} dx p(x|y)dx}{p(\Omega_x)}\right)^2-1\\
&=\int dy p(y)\bigg( \frac{p(\Omega_x|y)}{p(\Omega_x)}\bigg)^2-1\\
&=\int dy  \frac{p(y|\Omega_x)^2}{p(y)}-1\\
&=\E_{y\sim p(y|\Omega_x)}\bigg[\frac{p(y|\Omega_x)}{p(y)}-1\bigg]
\end{aligned}
\end{equation*}

Using the inequality $x-1\ge \log\ x$, we have
\begin{equation*}
\begin{aligned}
&\E_{y\sim p(y|\Omega_x)}\bigg[\frac{p(y|\Omega_x)}{p(y)}-1\bigg]\ge \E_{y\sim p(y|\Omega_x)}\bigg[\log\frac{p(y|\Omega_x)}{p(y)}\bigg]\ge 0
\end{aligned}
\end{equation*}

Both equalities hold iff $p(y|\Omega_x)\equiv p(y)$, at which the denominator of Eq. (\ref{eq:suff_3_app_3}) is equal to 0 and the expression inside the infimum diverge, which will not contribute to the infimum. Except this scenario, the denominator is greater than 0. Substituting into Eq. (\ref{eq:suff_3_app_3}), we have that a sufficient condition for $(X,Y)$ to be $\IB_\beta$-learnable is

\begin{equation}
\label{eq:suff_3_app_4}
\begin{aligned}
\beta>\inf_{\Omega_x\subset\X}\frac{\frac{p(\Omega_x)}{p(\Omega_x)^2}-1}{\E_{y\sim p(y|\Omega_x)}\left[\frac{p(y|\Omega_x)}{p(y)}-1\right]}
\end{aligned}
\end{equation}

Since $\Omega_x$ is a subset of $\X$, by the definition of $h(x)$ in Eq. (\ref{eq:suff_3_app_2}), $h(x)$ is not a constant in the entire $\X$. Hence the numerator of Eq. (\ref{eq:suff_3_app_4}) is positive. Since its denominator is also positive, we can then neglect the ``$>0$", and obtain the condition in Theorem \ref{thm:suff_3}.

Since the $h(x)$ used in this theorem is a subset of the $h(x)$ used in Theorem \ref{thm:suff_2}, the infimum for Eq. (\ref{eq:suff_3}) is greater than or equal to the infimum in Eq. (\ref{eq:suff_2}). Therefore, according to the second statement of Theorem \ref{thm:suff_2}, we have that the $\left(\inf_{\Omega_x\subset\X}\beta_0(\Omega_x)\right)^{-1}$ is also a lower bound of the slope for the Pareto frontier of $I(Y;Z)$ vs. $I(X;Z)$ curve.

Now we prove that the condition Eq. (\ref{eq:suff_3}) is invariant to invertible mappings of $X$. In fact, if $X'=g(X)$ is a uniquely invertible map (if $X$ is continuous, $g$ is additionally required to be continuous), let $\X'=\{g(x)|x\in\Omega_x\}$, and denote $g(\Omega_x)\equiv\{g(x)|x\in\Omega_x\}$ for any $\Omega_x\subset \X$, we have $p(g(\Omega_x))=p(\Omega_x)$, and $p(y|g(\Omega_x))=p(y|\Omega_x)$.  Then for dataset $(X,Y)$, let $\Omega_x'=g(\Omega_x)$, we have
\begin{equation}
\begin{aligned}
&\frac{\frac{1}{p(\Omega_x')} - 1}{\E_{y \sim p(y|\Omega_x')} \bigg[ \frac{p(y|\Omega_x')}{p(y)} - 1 \bigg]}=\frac{\frac{1}{p(\Omega_x)} - 1}{\E_{y \sim p(y|\Omega_x)} \bigg[ \frac{p(y|\Omega_x)}{p(y)} - 1 \bigg]}
\end{aligned}
\end{equation}

Additionally we have $\X'=g(\X)$. Then

\begin{equation}
\begin{aligned}
\label{eq:suf_2_app_3}
\inf_{\Omega_x'\subset \X'}\frac{\frac{1}{p(\Omega_x')} - 1}{\E_{y \sim p(y|\Omega_x')} \bigg[ \frac{p(y|\Omega_x')}{p(y)} - 1 \bigg]}=\inf_{\Omega_x\subset \X}\frac{\frac{1}{p(\Omega_x)} - 1}{\E_{y \sim p(y|\Omega_x)} \bigg[ \frac{p(y|\Omega_x)}{p(y)} - 1 \bigg]}
\end{aligned}
\end{equation}

For dataset $(X',Y)=(g(X),Y)$, applying Theorem \ref{thm:suff_3} we have that a sufficient condition for it to be $\IB_\beta$-learnable is 

\begin{equation}
\begin{aligned}
\beta>\inf_{\Omega_x'\subset \X'}\frac{\frac{1}{p(\Omega_x')} - 1}{\E_{y \sim p(y|\Omega_x')} \bigg[ \frac{p(y|\Omega_x')}{p(y)} - 1 \bigg]}=\inf_{\Omega_x\subset \X}\frac{\frac{1}{p(\Omega_x)} - 1}{\E_{y \sim p(y|\Omega_x)} \bigg[ \frac{p(y|\Omega_x)}{p(y)} - 1 \bigg]}
\end{aligned}
\end{equation}

where the equality is due to Eq. (\ref{eq:suf_2_app_3}). Comparing with the condition for $\IB_\beta$-learnability for $(X,Y)$ (Eq. (\ref{eq:suff_3})), we see that they are the same. Therefore, the condition given by Theorem \ref{thm:suff_3} is invariant to invertible mapping of $X$.

\end{proof}

\subsection{Proof of Corollary \ref{corollary:suff_3_class_conditional} and Corollary \ref{corollary:suff_3_2}}
\label{app:corollaries}

\subsubsection{Proof of Corollary \ref{corollary:suff_3_class_conditional}}

\begin{proof}
We use Theorem \ref{thm:suff_3}. Let $\Omega_x$ contain all elements $x$ whose true class is $y^*$ for some certain $y^*$, and 0 otherwise. Then we obtain a (potentially stronger) sufficient condition. Since the probability $p(y|y^*,x)=p(y|y^*)$ is class-conditional, we have

\begin{equation*}
\begin{aligned}
&\inf_{\Omega_x\subset\X}\frac{\frac{1}{p(\Omega_x)} - 1}{\E_{y \sim p(y|\Omega_x)} \bigg[ \frac{p(y|\Omega_x)}{p(y)} - 1 \bigg]}\\
=&\inf_{y^*}\frac{\frac{1}{p(y^*)} - 1}{\E_{y \sim p(y|y^*)} \bigg[ \frac{p(y|y^*)}{p(y)} - 1 \bigg]}
\end{aligned}
\end{equation*}

By requiring $\beta>\inf_{y^*}\frac{\frac{1}{p(y^*)} - 1}{\E_{y \sim p(y|y^*)} \big[ \frac{p(y|y^*)}{p(y)} - 1 \big]}$, we obtain a sufficient condition for $\IB_\beta$ learnability.
\end{proof}

\subsubsection{Proof of Corollary \ref{corollary:suff_3_2}}

\begin{proof}
We again use Theorem \ref{thm:suff_3}. Since $Y$ is a deterministic function of $X$, let $Y=f(X)$. By the assumption that $Y$ contains at least one value $y$ such that its probability $p(y)>0$, we let $\Omega_x$ contain only $x$ such that $f(x)=y$. Substituting into Eq. (\ref{eq:suff_3}), we have

\begin{equation*}
\begin{aligned}
&\frac{\frac{1}{p(\Omega_x)} - 1}{\E_{y \sim p(y|\Omega_x)} \bigg[ \frac{p(y|\Omega_x)}{p(y)} - 1 \bigg]}\\
=&\frac{\frac{1}{p(y)} - 1}{\E_{y \sim p(y|\Omega_x)} \bigg[ \frac{1}{p(y)} - 1 \bigg]}\\
=&\frac{\frac{1}{p(y)} - 1}{ \frac{1}{p(y)} - 1 }\\
=&1
\end{aligned}
\end{equation*}
\end{proof}
Therefore, the sufficient condition becomes $\beta>1$.

\subsection{\texorpdfstring{$\beta_0$}, hypercontractivity coefficient, contraction coefficient, \texorpdfstring{$\beta_0[h(x)]$}, and maximum correlation}
\label{app:maximum_corr}

In this section, we prove the relations between the IB-Learnability threshold $\beta_0$, the hypercontractivity coefficient $\xi(X;Y)$, the contraction coefficient $\eta_\text{KL}(p(y|x),p(x))$, $\beta_0[h(x)]$ in Eq. (\ref{eq:suff_2}), and maximum correlation $\rho_m(X,Y)$, as follows:

\begin{align}
\frac{1}{\beta_0} = \xi(X;Y)=\eta_\text{KL}(p(y|x),p(x))\ge \sup_{h(x)}\frac{1}{\beta_0[h(x)]} = \rho_m^2(X;Y)
\end{align}

\begin{proof}
The hypercontractivity coefficient $\xi$ is defined as \citep{anantharam2013maximal}:
$$\xi(X;Y)\equiv\sup_{Z-X-Y}\frac{I(Y;Z)}{I(X;Z)}$$

By our definition of IB-learnability, ($X$, $Y$) is IB-Learnable iff there exists $Z$ obeying the Markov chain $Z-X-Y$, such that

$$I(X;Z)-\beta\cdot I(Y;Z)<0=IB_\beta(X,Y;Z)|_{p(z|x)=p(z)}$$

Or equivalently there exists $Z$ obeying the Markov chain $Z-X-Y$ such that 
\begin{equation}
\label{eq:relation_11}
0<\frac{1}{\beta}<\frac{I(Y;Z)}{I(X;Z)}
\end{equation}

By Theorem \ref{thm:beta_monotonic}, the IB-Learnability region for $\beta$ is $(\beta_0, +\infty)$, or equivalently the IB-Learnability region for $1/\beta$ is
\begin{equation}
\label{eq:relation_12}
0<\frac{1}{\beta}<\frac{1}{\beta_0}
\end{equation}

Comparing Eq. (\ref{eq:relation_11}) and Eq. (\ref{eq:relation_12}), we have that

\begin{equation}
\label{eq:relation_13}
\frac{1}{\beta_0} = \sup_{Z-X-Y}\frac{I(Y;Z)}{I(X;Z)}=\xi(X;Y)
\end{equation}

In \citet{anantharam2013maximal}, the authors prove that 

\begin{equation}
\xi(X;Y) =\eta_\text{KL}(p(y|x),p(x))
\end{equation}

where the contraction coefficient $\eta_\text{KL}(p(y|x),p(x))$ is defined as

\begin{equation*}
\eta_\text{KL}(p(y|x),p(x))=\sup_{r(x)\neq p(x)}\frac{\mathbb{D}_\text{KL}(r(y)||p(y))}{\mathbb{D}_\text{KL}(r(x)||p(x))}
\end{equation*}

where $p(y)=\mathbb{E}_{x\sim p(x)}[p(y|x)]$ and $r(y)=\mathbb{E}_{x\sim r(x)}[p(y|x)]$.
Treating $p(y|x)$ as a channel, the contraction coefficient measures how much the two distributions $r(x)$ and $p(x)$ becomes ``nearer" (as measured by the KL-divergence) after passing through the channel.

In \citet{anantharam2013maximal}, the authors also provide a counterexample to an earlier result by \citet{erkip1998efficiency} that incorrectly proved $\xi(X;Y)=\rho_m^2(X;Y)$.
In the specific counterexample \citet{anantharam2013maximal} design, $\xi(X;Y)>\rho_m^2(X;Y)$.

The maximum correlation is defined as
$\rho_m(X;Y)\equiv\max_{f,g} \mathbb{E}[f(X)g(Y)]$ where $f(X)$ and $g(Y)$ are real-valued random variables such that $\mathbb{E}[f(X)]=\mathbb{E}[g(Y)]=0$ and $\mathbb{E}[f^2(X)]=\mathbb{E}[g^2(Y)]=1$ \citep{hirschfeld1935connection,gebelein1941statistische}.

Now we prove $\xi(X;Y)\ge\rho_m^2(X;Y)$, based on Theorem \ref{thm:suff_2}. To see this, we use the alternate characterization of $\rho_m(X;Y)$ by \citet{renyi1959measures}:

\begin{equation}
\label{eq:renyi}
\rho_m^2(X;Y)=\max_{f(X):\mathbb{E}[f(X)]=0,\mathbb{E}[f^2(X)]=1}{\mathbb{E}[\left(\mathbb{E}[f(X)|Y]\right)^2]}
\end{equation}

Denoting $\overline{h}=\mathbb{E}_{p(x)}[h(x)]$, we can transform  $\beta_0[h(x)]$ in Eq. (\ref{eq:suff_2}) as follows:

\begin{equation*}
\begin{aligned}
\beta_0[h(x)]&=\frac{\E_{x \sim p(x)} [h(x)^2] - \left(\E_{x\sim p(x)} [h(x)]\right)^2}{\E_{y \sim p(y)}\left[\left(\E_{x \sim p(x|y)} [h(x)]\right)^2\right] - \left(\E_{x\sim p(x)} [h(x)]\right)^2}\\
&=\frac{\E_{x \sim p(x)} [h(x)^2] - \overline{h}^2}{\E_{y \sim p(y)}\left[\left(\E_{x \sim p(x|y)} [h(x)]\right)^2\right] - \overline{h}^2}\\
&=\frac{\E_{x \sim p(x)} [(h(x)-\overline{h})^2]}{\E_{y \sim p(y)}\left[\left(\E_{x \sim p(x|y)} [h(x)-\overline{h}]\right)^2\right]}\\
&=\frac{1}{\E_{y \sim p(y)}\left[\left(\E_{x \sim p(x|y)} [f(x)]\right)^2\right]}\\
&=\frac{1}{\mathbb{E}[\left(\mathbb{E}[f(X)|Y]\right)^2]}
\end{aligned}
\end{equation*}

where we denote $f(x)=\frac{h(x)-\overline{h}}{\left(\E_{x \sim p(x)} [(h(x)-\overline{h})^2]\right)^{1/2}}$, so that $\mathbb{E}[f(X)]=0$ and $\mathbb{E}[f^2(X)]=1$.

Combined with Eq. (\ref{eq:renyi}), we have

\begin{equation}
\label{eq:relations_14}
\sup_{h(x)}\frac{1}{\beta_0[h(x)]}=\rho_m^2(X;Y) 
\end{equation}

Our Theorem \ref{thm:suff_2} states that

\begin{equation}
\label{eq:relations_15}
\sup_{h(x)}\frac{1}{\beta_0[h(x)]}\leq\frac{1}{\beta_0}
\end{equation}

Combining Eqs. (\ref{eq:relation_12}), (\ref{eq:relations_14}) and Eq. (\ref{eq:relations_15}), we have

\begin{equation}
\label{eq:relations_16}
\rho_m^2(X;Y)\leq\xi(X;Y)
\end{equation}

In summary, the relations among the quantities are:

\begin{equation}
\label{eq:relations_summary}
\frac{1}{\beta_0}=\xi(X;Y)=\eta_\text{KL}(p(y|x),p(x))\ge
\sup_{h(x)}\frac{1}{\beta_0[h(x)]}=\rho_m^2(X;Y)
\end{equation}
\end{proof}

\subsection{Experiment Details}
\label{app:experiment}

We use the Variational Information Bottleneck (VIB) objective from \cite{alemi2016deep}.
For the synthetic experiment, the latent $Z$ has dimension of 2.
The encoder is a neural net with 2 hidden layers, each of which has 128 neurons with ReLU activation. 
The last layer has linear activation and 4 output neurons; the first two parameterize the mean of a Gaussian and the last two parameterize the log variance.
The decoder is a neural net with 1 hidden layer with 128 neurons and ReLU activation. 
Its last layer has linear activation and outputs the logit for the class labels. 
It uses a mixture of Gaussian prior with 500 components (for the experiment with class overlap, 256 components), each of which is a 2D Gaussian with learnable mean and log variance, and the weights for the components are also learnable. 
For the MNIST experiment, the architecture is mostly the same, except the following: (1) for $Z$, we let it have dimension of 256. 
For the prior, we use standard Gaussian with diagonal covariance matrix. 

For all experiments, we use Adam (\cite{kingma2014adam}) optimizer with default parameters. 
We do not add any explicit regularization. 
We use learning rate of $10^{-4}$ and have a learning rate decay of $\frac{1}{1+0.01 \times \text{epoch}}$. 
We train in total $2000$ epochs with mini-batch size of 500. 

For estimation of the observed $\beta_{0}$ in Fig. \ref{fig:gauss_noise_beta}, in the $I(X;Z)$ vs. $\beta_i$ curve ($\beta_i$ denotes the $i^{\text{th}}$ $\beta$), we take the mean and standard deviation of $I(X;Z)$ for the lowest 5 $\beta_i$ values, denoting as $\mu_\beta$, $\sigma_\beta$ ($I(Y;Z)$ has similar behavior, but since we are minimizing $I(X;Z)-\beta \cdot I(Y;Z)$, the onset of nonzero $I(X;Z)$ is less prone to noise).
When $I(X;Z)$ is greater than $\mu_\beta$ + 3$\sigma_\beta$, we regard it as learning a non-trivial representation, and take the average of $\beta_i$ and $\beta_{i-1}$ as the experimentally estimated onset of learning. 
We also inspect manually and confirm that it is consistent with human intuition.

For estimating $\beta_{0}$ using Alg. \ref{alg:estimating_beta},  at step 6 we use the following discrete search algorithm.
We fix $i_\text{left}=1$ and gradually narrow down the range $[a,b]$ of $i_\text{right}$, starting from $[1,N]$.
At each iteration, we set a tentative new range $[a',b']$, where $a'=0.8a+0.2b$, $b'=0.2a+0.8b$, and calculate $\tilde{\beta}_{0,a'}=\textbf{Get}\boldsymbol{\beta}(P_{y|x},p_y,\Omega_{a'})$, $\tilde{\beta}_{0,b'}=\textbf{Get}\boldsymbol{\beta}(P_{y|x},p_y,\Omega_{b'})$ where $\Omega_{a'} =\{1,2,...a'\}$ and $\Omega_{b'} =\{1,2,...b'\}$.
If $\tilde{\beta}_{0,a'}<\tilde{\beta}_{0,a}$, let $a\gets a'$. If $\tilde{\beta}_{0,b'}<\tilde{\beta}_{0,b}$, let $b\gets b'$.
In other words, we narrow down the range of $i_\text{right}$ if we find that the $\Omega$ given by the left or right boundary gives a lower $\tilde{\beta}_0$ value.
The process stops when both $\tilde{\beta}_{0,a'}$ and $\tilde{\beta}_{0,b'}$ stop improving (which we find always happens when $b'=a'+1$), and we return the smaller of the final $\tilde{\beta}_{0,a'}$ and $\tilde{\beta}_{0,b'}$ as $\tilde{\beta}_0$.

For estimation of $p(y|x)$ for (2$'$) Alg. \ref{alg:estimating_beta} and (3$'$) $\hat{\eta}_{\text{KL}}$ for both synthetic and MNIST experiments, we use a 3-layer neuron net where each hidden layer has 128 neurons and ReLU activation. The last layer has linear activation. The objective is cross-entropy loss. We use Adam \citep{kingma2014adam} optimizer with a learning rate of $10^{-4}$, and train for 100 epochs (after which the validation loss does not go down).

For estimating $\beta_0$ via (3$'$) $\hat{\eta}_\text{KL}$ by the algorithm in \citep{kim2017discovering}, we use the code from the GitHub repository provided by the paper\footnote{%
At \href{https://github.com/wgao9/hypercontractivity}{https://github.com/wgao9/hypercontractivity}.
}, using the same $p(y|x)$ employed for (2$'$) Alg. \ref{alg:estimating_beta}. Since our datasets are classification tasks, we use $A_{ij}=p(y_j|x_i)/p(y_j)$ instead of the kernel density for estimating matrix $A$; we take the maximum of 10 runs as estimation of $\mu$.

\begin{table}[th]
\begin{center}
\caption{
Class confusion matrix used in CIFAR10 experiments.
The value in row $i$, column $j$ means for class $i$, the probability of labeling it as class $j$. The mean confusion across the classes is 20\%.
}
\label{tab:cifar_confusion}
\vskip 0.1in
\setlength{\tabcolsep}{4pt}  
\resizebox{1\linewidth}{!}{%
\begin{tabular}{r | c c c c c c c c c c }
\small
& Plane & Auto. & Bird & Cat & Deer & Dog & Frog & Horse & Ship & Truck \\ [0.2ex]
\hline
\hline\noalign{\smallskip}
Plane &
0.82232 & 0.00238 & 0.021   & 0.00069 & 0.00108 & 0       & 0.00017 & 0.00019 & 0.1473  & 0.00489 \\ [0.2ex]
Auto. &
0.00233 & 0.83419 & 0.00009 & 0.00011 & 0       & 0.00001 & 0.00002 & 0       & 0.00946 & 0.15379 \\ [0.2ex]
Bird &
0.03139 & 0.00026 & 0.76082 & 0.0095  & 0.07764 & 0.01389 & 0.1031  & 0.00309 & 0.00031 & 0       \\ [0.2ex]
Cat &
0.00096 & 0.0001  & 0.00273 & 0.69325 & 0.00557 & 0.28067 & 0.01471 & 0.00191 & 0.00002 & 0.0001  \\ [0.2ex]
Deer &
0.00199 & 0       & 0.03866 & 0.00542 & 0.83435 & 0.01273 & 0.02567 & 0.08066 & 0.00052 & 0.00001 \\ [0.2ex]
Dog &
0       & 0.00004 & 0.00391 & 0.2498  & 0.00531 & 0.73191 & 0.00477 & 0.00423 & 0.00001 & 0       \\ [0.2ex]
Frog &
0.00067 & 0.00008 & 0.06303 & 0.05025 & 0.0337  & 0.00842 & 0.8433  & 0       & 0.00054 & 0       \\ [0.2ex]
Horse &
0.00157 & 0.00006 & 0.00649 & 0.00295 & 0.13058 & 0.02287 & 0       & 0.83328 & 0.00023 & 0.00196 \\ [0.2ex]
Ship &
0.1288  & 0.01668 & 0.00029 & 0.00002 & 0.00164 & 0.00006 & 0.00027 & 0.00017 & 0.83385 & 0.01822 \\ [0.2ex]
Truck &
0.01007 & 0.15107 & 0       & 0.00015 & 0.00001 & 0.00001 & 0       & 0.00048 & 0.02549 & 0.81273 \\ [0.2ex]
\hline
\end{tabular}
}
\end{center}
\end{table}

\subsubsection{CIFAR10 Details}
\label{app:cifar_details}

We trained a deterministic 28x10 wide resnet~\citep{resnet,wideresnet}, using the open source implementation from~\citet{autoaugment}.
However, we extended the final 10 dimensional logits of that model through another 3 layer MLP classifier, in order to keep the inference network architecture identical between this model and the VIB models we describe below.
During training, we dynamically added label noise according to the class confusion matrix in Tab. \ref{tab:cifar_confusion}.
The mean label noise averaged across the 10 classes is 20\%.
After that model had converged, we used it to estimate $\beta_0$ with Alg.~\ref{alg:estimating_beta}.
Even with 20\% label noise, $\beta_0$ was estimated to be 1.0483.

We then trained 73 different VIB models using the same 28x10 wide resnet architecture for the encoder, parameterizing the mean of a 10-dimensional unit variance Gaussian.
Samples from the encoder distribution were fed to the same 3 layer MLP classifier architecture used in the deterministic model.
The marginal distributions were mixtures of 500 fully covariate 10-dimensional Gaussians, all parameters of which are trained.
The VIB models had $\beta$ ranging from 1.02 to 2.0 by steps of 0.02, plus an extra set ranging from 1.04 to 1.06 by steps of 0.001 to ensure we captured the empirical $\beta_0$ with high precision.

However, this particular VIB architecture does not start learning until $\beta > 2.5$, so none of these models would train as described.\footnote{%
A given architecture trained using maximum likelihood and with no stochastic layers will tend to have higher effective capacity than the same architecture with a stochastic layer that has a fixed but non-trivial variance, even though those two architectures have exactly the same number of learnable parameters.
}
Instead, we started them all at $\beta = 100$, and annealed $\beta$ down to the corresponding target over 10,000 training gradient steps.
The models continued to train for another 200,000 gradient steps after that.
In all cases, the models converged to essentially their final accuracy within 20,000 additional gradient steps after annealing was completed.
They were stable over the remaining $\sim 180,000$ gradient steps.

\section{Appendix for Chapter \ref{chap4:IB_phase_transition}}

\subsection{Calculus of variations at any order of  \titlemath{$\IB_\beta[p(z|x)]$}}
\label{app:expand_to_order_n}

Here we prove the Lemma \ref{lemma:expand_to_order_n}, which will be crucial in the lemmas and theorems in this paper that follows.

\begin{lemma}
\label{lemma:expand_to_order_n}
For a relative perturbation function $r(z|x)\in\QQ$ for a $p(z|x)$, where $r(z|x)$ satisfies $\E_{z\sim p(z|x)}[r(z|x)]=0$, we have that the IB objective can be expanded as

\begin{align}
\label{eq_app:expand_to_n_order}
&\IBB_\beta[p(z|x)(1+\epsilon \cdot r(z|x))]\nonumber\\
=&\IBB_\beta[p(z|x)]+\epsilon\cdot\left(\E_{x,z\sim p(x,z)}\left[r(z|x)\logg\frac{p(z|x)}{p(z)}\right]-\beta\cdot \E_{y,z\sim p(y,z)}\left[r(z|y)\logg\frac{p(z|y)}{p(z)}\right]\right)\nonumber\\
&+\sum_{n=2}^{\infty}\frac{(-1)^n \epsilon^n}{n(n-1)}\left\{\left(\E[r^n(z|x)]-\E[r^n(z)]\right)-\beta\cdot\left(\E[r^n(z|y)]-\E[r^n(z)]\right)\right\}\nonumber\\
=&\IBB_\beta[p(z|x)]+\epsilon\cdot\left(\E_{x,z\sim p(x,z)}\left[r(z|x)\logg\frac{p(z|x)}{p(z)}\right]-\beta\cdot \E_{y,z\sim p(y,z)}\left[r(z|y)\logg\frac{p(z|y)}{p(z)}\right]\right)\nonumber\\
&+\frac{\epsilon^2}{1\cdot2}\left\{\left(\E[r^2(z|x)]-\E[r^2(z)]\right)-\beta\cdot\left(\E[r^2(z|y)]-\E[r^2(z)]\right)\right\}\nonumber\\
&-\frac{\epsilon^3}{2\cdot3}\left\{\left(\E[r^3(z|x)]-\E[r^3(z)]\right)-\beta\cdot\left(\E[r^3(z|y)]-\E[r^3(z)]\right)\right\}\nonumber\\
&+\frac{\epsilon^4}{3\cdot4}\left\{\left(\E[r^4(z|x)]-\E[r^4(z)]\right)-\beta\cdot\left(\E[r^4(z|y)]-\E[r^4(z)]\right)\right\}\nonumber\\
&-...\nonumber\\
\end{align}

where $r(z|y)=\E_{x\sim p(x|y,z)}[r(z|x)]$ and $r(z)=\E_{x\sim p(x|z)}[r(z|x)]$. The expectations in the equations are all w.r.t. all variables. For example $\E[r^2(z|x)]=\E_{x,z\sim p(x,z)}[r^2(z|x)]$.
\end{lemma}
\begin{proof}
Suppose that we perform a relative perturbation $r(z|x)$ on $p(z|x)$ such that the perturbed conditional probability is $p'(z|x)=p(z|x)\left(1+\epsilon\cdot r(z|x)\right)$, then we have
\begin{equation*}
\begin{aligned}
p'(z)&=\int p(x)p'(z|x)dx=\int dx p(x)p(z|x)\left(1+\epsilon\cdot r(z|x)\right)=p(z)+\e\cdot\int dx p(x) p(z|x)r(z|x)
\end{aligned}
\end{equation*}

Therefore, we can denote the corresponding relative perturbation $r(z)$ on $p(z)$ as 
\begin{equation*}
\begin{aligned}
r(z)\equiv\frac{1}{\epsilon}\frac{p'(z)-p(z)}{p(z)}=\frac{1}{p(z)}\int dx p(x)p(z|x)r(z|x)=\E_{x\sim p(x|z)}[r(z|x)]
\end{aligned}
\end{equation*}

Similarly, we have
\begin{equation*}
\begin{aligned}
p'(z|y)=\frac{p'(y,z)}{p(y)}=\frac{1}{p(y)}\int dx p(x,y)p(z|x)\left(1+\epsilon\cdot r(z|x)\right)=p(z|y)+\epsilon\cdot \frac{1}{p(y)}\int dx p(x,y)p(z|x)r(z|x)
\end{aligned}
\end{equation*}

And we can denote the corresponding relative perturbation $r(z|y)$ on $p(z|y)$ as

\begin{equation*}
\begin{aligned}
r(z|y)\equiv \frac{1}{\e}\frac{p'(z|y)-p(z|y)}{p(z|y)}=\frac{1}{p(z|y)p(y)}\int dx p(x,y)p(z|x)r(z|x)=\E_{x\sim p(x|y,z)}[r(z|x)]
\end{aligned}
\end{equation*}

Since
\begin{equation*}
\begin{aligned}
\IB_\beta[p(z|x)]=I(X;Z)-\beta\cdot I(Y;Z)=\int dxdz p(x,z)\log{\frac{p(z|x)}{p(z)}}-\beta\cdot\int dydz p(y,z)\log{\frac{p(z
|y)}{p(z)}}
\end{aligned}
\end{equation*}

We have
\begin{equation*}
\begin{aligned}
&\IB_\beta[p'(z|x)]=\IB_\beta[p(z|x)(1+\epsilon\cdot r(z|x))]\\
=&\int dxdz p(x)p'(z|x)\log{\frac{p'(z|x)}{p'(z)}}-\beta\cdot\int dydz p(y)p'(z|y)\log{\frac{p'(z|y)}{p'(z)}}\\
=&\int dxdz p(x)p(z|x)(1+\e\cdot r(z|x))\log{\frac{p(z|x)(1+\e\cdot r(z|x))}{p(z)(1+\e\cdot r(z))}}\\
&-\beta\cdot\int dydz p(y)p(z|y)(1+\e\cdot r(z|y))\log{\frac{p(z|y)(1+\e\cdot r(z|y))}{p(z)(1+\e\cdot r(z))}}\\
=&\int dxdz p(x)p(z|x)(1+\e\cdot r(z|x))\left[\log{\frac{p(z|x)}{p(z)}}+\log\left(1+\e\cdot r(z|x)\right)-\log\left(1+\e\cdot r(z)\right)\right]\\
&-\beta\cdot\int dydz p(y)p(z|y)(1+\e\cdot r(z|y))\left[\log{\frac{p(z|y)}{p(z)}}+\log\left(1+\e\cdot r(z|y)\right)-\log\left(1+\e\cdot r(z)\right)\right]\\
=&\int dxdz p(x)p(z|x)(1+\e\cdot r(z|x))\left[\log{\frac{p(z|x)}{p(z)}}+\sum_{n=1}^{\infty}(-1)^{n-1}\frac{\e^n}{n}\left(r(z|x)-r(z)\right)\right]\\
&-\beta\cdot\int dydz p(y)p(z|y)(1+\e\cdot r(z|y))\left[\log{\frac{p(z|y)}{p(z)}}+\sum_{n=1}^{\infty}(-1)^{n-1}\frac{\e^n}{n}\left(r(z|y)-r(z)\right)\right]\\
\end{aligned}
\end{equation*}

The $0^\text{th}$-order term is simply $\IB_\beta[p(z|x)]$. The first order term is $$\delta\IB_\beta[p(z|x)]=\epsilon\cdot\left(\E_{x,z\sim p(x,z)}\left[r(z|x)\log\frac{p(z|x)}{p(z)}\right]-\beta\cdot \E_{y,z\sim p(y,z)}\left[r(z|y)\log\frac{p(z|y)}{p(z)}\right]\right)$$

The $n^\text{th}$-order term for $n\ge2$ is
\begin{equation*}
\begin{aligned}
&\delta^n\IB_\beta[p(z|x)]\\
=&(-1)^n\e^n\int dxdz p(x)p(z|x)\left(-\frac{1}{n}\left[r^n(z|x)-r^n(z)\right]+r(z|x)\frac{1}{n-1}\left[r^{n-1}(z|x)-r^n(z)\right]\right)\\
&-\beta\cdot (-1)^n\e^n\int dydz p(y)p(z|y)\left(-\frac{1}{n}\left[r^n(z|y)-r^n(z)\right]+r(z|y)\frac{1}{n-1}\left[r^{n-1}(z|y)-r^n(z)\right]\right)\\
=&\frac{(-1)^n\e^n}{n(n-1)}\left(\E_{x,z\sim p(x,z)}[r^n(z|x)]-n\E_{x,z\sim p(x,z)}[r(z|x)r^{n-1}(z)]-(n-1)\E_{z\sim p(z)}[r^n(z)]\right)\\
&-\beta\cdot\frac{(-1)^n\e^n}{n(n-1)}\left(\E_{y,z\sim p(y,z)}[r^n(z|y)]-n\E_{y,z\sim p(y,z)}[r(z|y)r^{n-1}(z)]-(n-1)\E_{z\sim p(z)}[r^n(z)]\right)\\
=&\frac{(-1)^n \e^n}{n(n-1)}\left\{\left(\E[r^n(z|x)]-\E[r^n(z)]\right)-\beta\cdot\left(\E[r^n(z|y)]-\E[r^n(z)]\right)\right\}
\end{aligned}
\end{equation*}

In the last equality we have used $$\E_{x,z\sim p(x,z)}[r(z|x)r^{n-1}(z)]=\E_{z\sim p(z)}[r^{n-1}(z)\E_{x\sim p(x|z)}[r(z|x)]]=\E_{z\sim p(z)}[r^{n-1}(z)r(z)]=\E_{z\sim p(z)}[r^n(z)]$$

Combining the terms with all orders, we have
\begin{equation*}
\begin{aligned}
&\IB_\beta[p(z|x)(1+\epsilon \cdot r(z|x))]\\
=&\IB_\beta[p(z|x)]+\epsilon\cdot\left(\E_{x,z\sim p(x,z)}\left[r(z|x)\log\frac{p(z|x)}{p(z)}\right]-\beta\cdot \E_{y,z\sim p(y,z)}\left[r(z|y)\log\frac{p(z|y)}{p(z)}\right]\right)\\
&+\sum_{n=2}^{\infty}\frac{(-1)^n \epsilon^n}{n(n-1)}\left\{\left(\E[r^n(z|x)]-\E[r^n(z)]\right)-\beta\cdot\left(\E[r^n(z|y)]-\E[r^n(z)]\right)\right\}
\end{aligned}
\end{equation*}
\end{proof}

As a side note, the KL-divergence between $p'(z|x)=p(z|x)(1+\e\cdot r(z|x))$ and $p(z|x)$ is

\begin{equation*}
\begin{aligned}
\KL\left(p'(z|x)||p(z|x)\right)&=
\int dz p(z|x)(1+\e\cdot r(z|x))\log\frac{p(z|x)(1+\e\cdot r(z|x))}{p(z|x)}\\
&=\int dz p(z|x)(1+\e\cdot r(z|x))\left(\e\cdot r(z|x)-\frac{\e^2}{2}\cdot r^2(z|x))+O(\e^3)\right)\\
&=\e\cdot\int dz p(z|x) r(z|x)+\frac{\e^2}{2}\int dz p(z|x)r^2(z|x)+O(\e^3)\\
&=\frac{\e^2}{2}\E_{z\sim p(z|x)}[r^2(z|x)]+O(\e^3)
\end{aligned}
\end{equation*}

Therefore, to the second order, we have
\beq{eq_app:KL_r}
\E_{x\sim p(x)}\left[\KL\left(p'(z|x)||p(z|x)\right)\right]=\frac{\e^2}{2}\E[r^2(z|x)]
\eeq

Similarly, we have $\E_{x\sim p(x)}\left[\KL\left(p(z|x)||p'(z|x)\right)\right]=\frac{\e^2}{2}\E[r^2(z|x)]$ up to the second order. Using similar procedure, we have up to the second-order,

\begin{equation*}
\begin{aligned}
&\E_{y\sim p(y)}\left[\KL\left(p'(z|y)||p(z|y)\right)\right]=\E_{y\sim p(y)}\left[\KL\left(p(z|y)||p'(z|y)\right)\right]=\frac{\e^2}{2}\E[r^2(z|y)]\\
&\KL\left(p'(z)||p(z)\right)=\KL\left(p(z)||p'(z)\right)=\frac{\e^2}{2}\E[r^2(z)]\\
\end{aligned}
\end{equation*}

\subsection{Proof of Lemma \ref{lemma:second_order_variation}}
\label{app:lemma_G}
\begin{proof}
From Lemma \ref{lemma:expand_to_order_n}, we have
\beq{eq_app:IB_second_order}
\delta^2\IB_\beta[p(z|x)]=\frac{\epsilon^2}{2}\left\{\left(\E[r^2(z|x)]-\E[r^2(z)]\right)-\beta\cdot\left(\E[r^2(z|y)]-\E[r^2(z)]\right)\right\}
\eeq

The condition of 
\beq{eq:condition_second_order_app}
\forall r(z|x)\in\QQ, \delta^2\IB_\beta[p(z|x)]\ge0
\eeq
is equivalent to 
\beq{eq:condition_second_order_app_2}
\forall r(z|x)\in\QQ, \beta\cdot\left(\E[r^2(z|y)]-\E[r^2(z)]\right)\le\E[r^2(z|x)]-\E[r^2(z)]
\eeq
Using Jensen's inequality and the convexity of the square function, we have

\begin{equation*}
\begin{aligned}
\E[r^2(z|y)]&=\E_{y,z\sim p(y,z)}\left[\left(\E_{x\sim p(x|y,z)}[r(z|x)]\right)^2\right]\\
&=\E_{z\sim p(z)}\left[\E_{y\sim p(y|z)}\left[\left(\E_{x\sim p(x|y,z)}[r(z|x)]\right)^2\right]\right]\\
&\ge \E_{z\sim p(z)}\left[\left(\E_{y\sim p(y|z)}\left[\E_{x\sim p(x|y,z)}[r(z|x)]\right]\right)^2\right]\\
&=\E_{z\sim p(z)}\left[\left(\E_{x\sim p(x|z)}[r(z|x)]\right)^2\right]\\
&=\E[r^2(z)]
\end{aligned}
\end{equation*}

The equality holds iff $r(z|y)=\E_{x\sim p(x|y,z)}[r(z|x)]$ is constant w.r.t. $y$, for any $z$.

Using Jensen's inequality on $\E[r^2(z)]$, we have
$\E[r^2(z)]=\E_{z\sim p(z)}\left[\left(\E_{x\sim p(x|z)}[r(z|x)]\right)^2\right]\le\E_{z\sim p(z)}\left[\E_{x\sim p(x|z)}[r^2(z|x)]\right]=\E[r^2(z|x)]$, where the equality holds iff $r(z|x)$ is constant w.r.t. $x$ for any $z$. 

When $\E[r^2(z|y)]-\E[r^2(z)]>0$, we have that the condition Eq. (\ref{eq:condition_second_order_app_2}) is equivalent to $\forall r(z|x)\in\QQ$, $\beta\le\frac{\E[r^2(z|x)]-\E[r^2(z)]}{\E[r^2(z|y)]-\E[r^2(z)]}$, i.e.

\beq{eq_app:condition_second_order_app_3}
\beta\le G[p(z|x)]\equiv\inf_{r(z|x)\in\QQ}\frac{\E[r^2(z|x)]-\E[r^2(z)]}{\E[r^2(z|y)]-\E[r^2(z)]}
\eeq
where $r(z|y)=\E_{x\sim p(x|y,z)}[r(z|x)]$ and $r(z)=\E_{x\sim p(x|z)}[r(z|x)]$.

If $\E[r^2(z|y)]-\E[r^2(z)]=0$, substituting into Eq. (\ref{eq:condition_second_order_app_2}), we have

\beq{eq_app:condition_second_order_app_4}
\beta\cdot0\le\E[r^2(z|x)]-\E[r^2(z)]
\eeq

which is always true due to that $\E[r^2(z|x)]\ge\E[r^2(z)]$, and will be a looser condition than Eq. (\ref{eq_app:condition_second_order_app_3}) above. Above all, we have Eq. (\ref{eq_app:condition_second_order_app_3}).

\end{proof}

\paragraph{Empirical estimate of \texorpdfstring{$G[p(z|x)]$}:} 

To empirically estimate $G[p(z|x)]$ from a minibatch of $\{(x_i, y_i)\},i=1,2,...N$ and the encoder $p(z|x)$, we can make the following Monte Carlo importance sampling estimation, where we use the samples $\{x_j\}\sim p(x)$ and also get samples of $\{z_i\}\sim p(z)=p(x)p(z|x)$, and have:

\begin{equation*}
\begin{aligned}
\E_{x,z\sim p(x,z)}[r^2(z|x)]=&\int dxdzp(x)p(z) \frac{p(x,z)}{p(x)p(z)}r^2(z|x)\\
\simeq&\frac{1}{N^2}\sum_{i=1}^{N}\sum_{j=1}^{N} \frac{p(x_j,z_i)}{p(x_j)p(z_i)}r^2(z_i|x_j)\\
\E_{z\sim p(z)}[r^2(z)]=&\E_{z\sim p(z)}\left[\left(\E_{x\sim p(x|z)}[r(z|x)]\right)^2\right]\\
\simeq&\frac{1}{N}\sum_{i=1}^{N} \left(\int dx p(x|z_i) r(z_i|x)\right)^2\\
=&\frac{1}{N}\sum_{i=1}^{N} \left(\int dx p(x) \frac{p(z_i|x)}{p(z_i)} r(z_i|x)\right)^2\\
\simeq&\frac{1}{N}\sum_{i=1}^{N}\left(\frac{1}{N}\sum_{j=1}^{N}\frac{p(z_i|x_j)}{p(z_i)}r(z_i|x_j)\right)^2\\
\simeq&\frac{1}{N}\sum_{i=1}^{N}\left(\frac{1}{N}\sum_{j=1}^{N}\frac{p(z_i|x_j)}{\frac{1}{N}\sum_{k=1}^{N}p(z_i|x_k)}r(z_i|x_j)\right)^2\\
=&\frac{1}{N}\sum_{i=1}^{N}\left(\frac{\sum_{j=1}^{N}p(z_i|x_j)r(z_i|x_j)}{\sum_{j=1}^{N}p(z_i|x_j)}\right)^2\\
\end{aligned}
\end{equation*}

\begin{equation*}
\begin{aligned}
\E_{y,z\sim p(y,z)}[r^2(z|y)]=&\E_{y,z\sim p(y,z)}\left[\left(\E_{x\sim p(x|y,z)}[r(z|x)]\right)^2\right]\\
\simeq&\frac{1}{N}\sum_{i=1}^{N}\left(\int dx p(x|y_i,z_i)r(z_i|x)\right)^2\\
=&\frac{1}{N}\sum_{i=1}^{N}\left(\frac{1}{p(y_i,z_i)}\int dx p(y_i)p(x|y_i)p(z_i|x) r(z_i|x)\right)^2\\
=&\frac{1}{N}\sum_{i=1}^{N}\left(\frac{\int dx p(y_i)p(x|y_i)p(z_i|x) r(z_i|x)}{\int dx p(y_i)p(x|y_i)p(z_i|x)}\right)^2\\
\simeq&\frac{1}{N}\sum_{i=1}^{N}\left(\frac{\sum_{x_j\in \Omega_x(y_i)}p(z_i|x_j) r(z_i|x_j)}{\sum_{x_j\in \Omega_x(y_i)}p(z_i|x_j)}\right)^2\\
=&\frac{1}{N}\sum_{i=1}^{N}\left(\frac{\sum_{j=1}^N p(z_i|x_j) r(z_i|x_j)\mathbbm{1}\left[y_i=y_j\right]}{\sum_{j=1}^N p(z_i|x_j)\mathbbm{1}\left[y_i=y_j\right]}\right)^2\\
\end{aligned}
\end{equation*}

Here $\Omega_x(y_i)$ denotes the set of $x$ examples that has label of $y_i$, and $\mathbbm{1}[\cdot]$ is an indicator function that takes value 1 if its argument is true, 0 otherwise. 

The requirement of $\E_{z\sim p(z|x)}[r(z|x)]=0$ yields

\beq{eq_app:G_empirical_cond1}
0=\E_{z\sim p(z|x)}[r(z|x)]=\int dz p(z)\frac{p(z|x)}{p(z)}r(z|x)\simeq\frac{1}{N}\sum_{i=1}^N \frac{p(z_i|x_j)}{p(z_i)}r(z_i|x_j)
\eeq

for any $x_j$.

Combining all terms, we have that the empirical $\hat{G}[p(z|x)]$ is given by

\beq{eq_app:G_empirical1}
\hat{G}[p(z|x)]=\inf_{r(z|x)\in\QQ}\frac{\frac{1}{N}\sum_{i=1}^{N}\sum_{j=1}^{N} \frac{p(x_j,z_i)}{p(x_j)p(z_i)}r^2(z_i|x_j)-\sum_{i=1}^{N}\left(\frac{\sum_{j=1}^{N}p(z_i|x_j)r(z_i|x_j)}{\sum_{j=1}^{N}p(z_i|x_j)}\right)^2}{\sum_{i=1}^{N}\left(\frac{\sum_{j=1}^Np(z_i|x_j) r(z_i|x_j)\mathbbm{1}[y_i=y_j]}{\sum_{j=1}^N p(z_i|x_j)\mathbbm{1}[y_i=y_j]}\right)^2-\sum_{i=1}^{N}\left(\frac{\sum_{j=1}^{N}p(z_i|x_j)r(z_i|x_j)}{\sum_{j=1}^{N}p(z_i|x_j)}\right)^2}
\eeq

where $\{z_i\}\sim p(z)$ and $\{x_i\}\sim p(x)$. It is also possible to use different distributions for importance sampling, which will results in different formulas for empirical estimation of $G[p(z|x)]$.

\subsection{\titlemath{$G_\Theta[p_\thetaa(z|x)]$ for parameterized distribution $p_\thetaa(z|x)$}}
\label{app:G_theta}

\begin{proof}
For the parameterized\footnote{In this paper, $\thetaa=(\theta_1,\theta_2,...\theta_k)^T$ and $\frac{\partial\pthe}{\partial\thetaa}=\left(\frac{\partial \pthe}{\partial \theta_1},\frac{\partial \pthe}{\partial \theta_2},...\frac{\partial \pthe}{\partial \theta_k}\right)^T$ are all column vectors. $\frac{\partial^2\pthe}{\partial\thetaa^2}$ is a $k\times k$ matrix with $(i,j)$ element of $\frac{\partial^2\pthe}{\partial\theta_i\partial\theta_j}$.} 
$p_\thetaa(z|x)$ with $\thetaa\in\Theta$, after $\thetaa'\gets\thetaa+\Delta\thetaa$, where\footnote{Note that since $\Theta$ is a field, it is closed under subtraction, we have $\Dthe\in\Theta$.} $\Delta\thetaa\in\Theta$ is an infinitesimal perturbation on $\thetaa$, we have that the distribution changes from $p_\thetaa(z|x)$ to $p_{\thetaa+\Delta\thetaa}(z|x)$, and thus the relative perturbation on $p_\thetaa(z|x)$ is
\begin{equation*}
\begin{aligned}
&\e\cdot r(z|x)=\frac{p_{\thetaa+\Delta\thetaa}(z|x)-p_{\thetaa}(z|x)}{p_{\thetaa}(z|x)}\\
=&\frac{1}{\pthe}\left(\pthe+\Dthe^T\frac{\partial \pthe}{\partial\thetaa}+\frac{1}{2}\Dthe^T\frac{\partial^2 p_\thetaa(z|x)}{\partial \thetaa^2}\Dthe+O(\norm{\Dthe}^3)-\pthe\right)\\
\simeq&\Dthe^T\frac{\partial}{\partial\thetaa}\log\pthe+\frac{1}{2}\Dthe^T\frac{1}{\pthe}\frac{\partial^2 p_\thetaa(z|x)}{\partial \thetaa^2}\Dthe+O(\norm{\Dthe}^3)
\end{aligned}
\end{equation*}

where $\norm{\Dthe}$ is the norm of $\Dthe$ in the parameter field $\Theta$.

Similarly, we have

\begin{equation*}
\begin{aligned}
\e\cdot r(z|y)&=\Dthe^T\partialthe\log\, p_{\thetaa}(z|y)+\frac{1}{2}\Dthe^T\frac{1}{p_\thetaa(z|y)}\frac{\partial^2 p_\thetaa(z|y)}{\partial \thetaa^2}\Dthe+O(\norm{\Dthe}^3)\\
\e\cdot r(z)&=\Dthe^T\partialthe\log\, p_{\thetaa}(z)+\frac{1}{2}\Dthe^T\frac{1}{p_\thetaa(z)}\frac{\partial^2 p_\thetaa(z)}{\partial \thetaa^2}\Dthe+O(\norm{\Dthe}^3)\\
\end{aligned}
\end{equation*}

Substituting the above expressions into the expansion of $\IB_\beta[p(z|x)]$ in Eq. (\ref{eq_app:expand_to_n_order}), and preserving to the second order  $\norm{\Delta\thetaa}^2$, we have

\begin{equation*}
\begin{aligned}
&\IB_\beta[\pthe(1+\epsilon \cdot r(z|x))]\nonumber\\
=&\IB_\beta[\pthe]+\epsilon\cdot\left(\E_{x,z\sim p_\thetaa(x,z)}\left[r(z|x)\log\frac{\pthe}{p_\thetaa(z)}\right]-\beta\cdot \E_{y,z\sim p_\thetaa(y,z)}\left[r(z|y)\log\frac{p_\thetaa(z|y)}{p_\thetaa(z)}\right]\right)\nonumber\\
&+\frac{\epsilon^2}{1\cdot2}\left\{\left(\E_{x,z\sim p_\thetaa(x,z)}[r^2(z|x)]-\E_{z\sim p_\thetaa(z)}[r^2(z)]\right)-\beta\cdot\left(\E_{y,z\sim p_\thetaa(y,z)}[r^2(z|y)]-\E_{z\sim p_\thetaa(z)}[r^2(z)]\right)\right\}\nonumber\\
=&\IB_\beta[\pthe]+\E_{x,z\sim p_\thetaa(x,z)}\left[\left(\Dthe^T\frac{\partial}{\partial\thetaa}\log\,\pthe+\frac{1}{2}\Dthe^T\frac{1}{\pthe}\frac{\partial^2 p_\thetaa(z|x)}{\partial \thetaa^2}\Dthe\right)\log\frac{\pthe}{p_\thetaa(z)}\right]\\
&-\beta\cdot \E_{y,z\sim p_\thetaa(y,z)}\left[\left(\Dthe^T\partialthe\log\, p_{\thetaa}(z|y)+\frac{1}{2}\Dthe^T\frac{1}{p_\thetaa(z|y)}\frac{\partial^2 p_\thetaa(z|y)}{\partial \thetaa^2}\Dthe\right)\log\frac{p_\thetaa(z|y)}{p_\thetaa(z)}\right]\nonumber\\
&+\frac{1}{2}\left(\E_{x,z\sim p_\thetaa(x,z)}\left[\left(\Dthe^T\frac{\partial}{\partial\thetaa}\log\,\pthe\right)^2\right]-\E_{z\sim p_\thetaa(z)}\left[\left(\Dthe^T\frac{\partial}{\partial\thetaa}\log\, p_\thetaa(z)\right)^2\right]\right)\\
&-\frac{\beta}{2}\left(\E_{y,z\sim p_\thetaa(y,z)}\left[\left(\Dthe^T\frac{\partial}{\partial\thetaa}\log\, p_\thetaa(z|y)\right)^2\right]-\E_{z\sim p_\thetaa(z)}\left[\left(\Dthe^T\frac{\partial}{\partial\thetaa}\log\, p_\thetaa(z)\right)^2\right]\right)\nonumber\\
=&\IB_\beta[\pthe]+\Dthe^T\bigg\{\E_{x,z\sim p_\thetaa(x,z)}\left[\log\frac{\pthe}{p_\thetaa(z)}\frac{\partial}{\partial\thetaa}\log\pthe\right]\\
&-\beta\cdot \E_{x,z\sim p_\thetaa(x,z)}\left[\log\frac{\pthe}{p_\thetaa(z)}\frac{\partial}{\partial\thetaa}\log\pthe\right]\bigg\}\\
&+\frac{1}{2}\Delta\thetaa^T\left\{\left(\mathcal{I}_{Z|X}(\thetaa)-\mathcal{I}_{Z}(\thetaa)\right)-\beta\left(\mathcal{I}_{Z|X}(\thetaa)-\mathcal{I}_{Z}(\thetaa)\right)\right\}\Delta\thetaa\\
\end{aligned}
\end{equation*}

In the last equality we have used $\E_{x,z\sim p_\thetaa(x,z)}[\frac{1}{\pthe}\frac{\partial^2\pthe}{\partial\thetaa^2}]=\int dx p(x)\frac{\partial^2}{\partial\thetaa^2}\int dz\pthe=\int dx p(x)\frac{\partial^2}{\partial\thetaa^2}1=\mathbf{0}$, and similarly $\E_{y,z\sim p_\thetaa(y,z)}[\frac{1}{p_\thetaa(z|y)}\frac{\partial^2 p_\thetaa(z|y)}{\partial\thetaa^2}]=\mathbf{0}$. In other words, the $\norm{\Dthe}^2$ terms in the first-order variation $\delta \IB_\beta[\pthe]$ vanish, and the remaining $\norm{\Dthe}^2$ are all in $\delta^2 \IB_\beta[\pthe]$. Also in the last expression, $\mathcal{I}_{Z}(\thetaa)\equiv\int dz p_\thetaa (z)\left(\frac{\partial \log p_\thetaa(z)}{\partial \thetaa}\right)\left(\frac{\partial \log p_\thetaa(z)}{\partial \thetaa}\right)^T$ is the Fisher information matrix of $\thetaa$ for $Z$, $\mathcal{I}_{Z|X}(\thetaa)\equiv \int dxdz p(x) p_\thetaa(z|x)\left(\frac{\partial \log p_\thetaa(z|x)}{\partial \thetaa}\right)\left(\frac{\partial \log p_\thetaa(z|x)}{\partial \thetaa}\right)^T$, $\mathcal{I}_{Z|Y}(\thetaa)\equiv \int dydz p(y) p_\thetaa(z|y)\left(\frac{\partial \log p_\thetaa(z|y)}{\partial \thetaa}\right)\left(\frac{\partial \log p_\thetaa(z|y)}{\partial \thetaa}\right)^T$  are the conditional Fisher information matrix \citep{fisherproperty} of $\thetaa$ for $Z$ conditioned on $X$ and $Y$, respectively.

Let us look at
\beq{eq_app:IB_second_order_theta}
\delta^2\IB_\beta[\pthe]=\frac{1}{2}\Delta\thetaa^T\left\{\left(\mathcal{I}_{Z|X}(\thetaa)-\mathcal{I}_{Z}(\thetaa)\right)-\beta\left(\mathcal{I}_{Z|X}(\thetaa)-\mathcal{I}_{Z}(\thetaa)\right)\right\}\Delta\thetaa
\eeq

Firstly, note that $\delta^2\IB_\beta[\pthe]$ is a quadratic function of $\Dthe$, and the scale of $\Dthe$ does not change the sign of $\delta^2\IB_\beta[\pthe]$, so the condition of $\forall \Dthe\in\Theta$, $\delta^2\IB_\beta[\pthe]\ge0$ is invariant to the scale of $\Dthe$, and is describing the ``curvature'' in the infinitesimal neighborhood of $\thetaa$. Therefore, $\Dthe$ can explore any value in $\Theta$. Secondly, we see that Eq. (\ref{eq_app:IB_second_order_theta}) is a special case of Eq. (\ref{eq_app:IB_second_order}) with $\e\cdot r(z|x)=\Dthe^T\partialthe\log\,\pthe$. Therefore, The inequalities due to Jensen still hold:
$\e^2\left(\E[r^2(z|x)]-\E[r^2(z)]\right)=\Delta\thetaa^T\left(\mathcal{I}_{Z|X}(\thetaa)-\mathcal{I}_{Z}(\thetaa)\right)\Dthe\ge0$, $\e^2\left(\E[r^2(z|y)]-\E[r^2(z)]\right)=\Delta\thetaa^T\left(\mathcal{I}_{Z|Y}(\thetaa)-\mathcal{I}_{Z}(\thetaa)\right)\Dthe\ge0$. 
If $\Delta\thetaa^T\left(\mathcal{I}_{Z|Y}(\thetaa)-\mathcal{I}_{Z}(\thetaa)\right)\Dthe>0$, then the condition of $\forall \Dthe\in\Theta$, $\delta^2\IB_\beta[\pthe]\ge0$ is equivalent to $\forall \Dthe\in\Theta$,

$$\beta\le\frac{\Delta\thetaa^T\left(\mathcal{I}_{Z|X}(\thetaa)-\mathcal{I}_{Z}(\thetaa)\right)\Delta\thetaa}{\Delta\thetaa^T\left(\mathcal{I}_{Z|Y}(\thetaa)-\mathcal{I}_{Z}(\thetaa)\right)\Delta\thetaa}$$ 

i.e.

\beq{eq_app:G_theta}
\beta\le G_\Theta[\pthe]\equiv\inf_{\Delta\thetaa\in\Theta}\frac{\Delta\thetaa^T\left(\mathcal{I}_{Z|X}(\thetaa)-\mathcal{I}_{Z}(\thetaa)\right)\Delta\thetaa}{\Delta\thetaa^T\left(\mathcal{I}_{Z|Y}(\thetaa)-\mathcal{I}_{Z}(\thetaa)\right)\Delta\thetaa}
\eeq

If $\Delta\thetaa^T\left(\mathcal{I}_{Z|Y}(\thetaa)-\mathcal{I}_{Z}(\thetaa)\right)\Dthe=0$, we have that Eq. (\ref{eq_app:IB_second_order_theta}) always holds, which is a looser condition than Eq. (\ref{eq_app:G_theta}). Above all, we have that the condition of $\forall\Dthe\in\Theta$, $\delta^2\IB_\beta[\pthe]$ is equivalent to $\beta\le G_\Theta[\pthe]$.

Moreover, $\left(G_\Theta[\pthe]\right)^{-1}$ given by Eq. (\ref{eq_app:G_theta}) has the format of a generalized Rayleigh quotient $R(A,B;x)\equiv \frac{\Dthe^TA\Dthe}{\Dthe^TB\Dthe}$ where $A=\mathcal{I}_{Z|Y}(\thetaa)-\mathcal{I}_{Z}(\thetaa)$ and $B=\mathcal{I}_{Z|X}(\thetaa)-\mathcal{I}_{Z}(\thetaa)$ are both Hermitian matrices\footnote{Here all the Fisher information matrices are real symmetric, thus Hermitian.}, which can be reduced to Rayleigh quotient $R(D,C^T\Dthe)=\frac{(C^T\Dthe)^TD(C^T\Dthe)}{(C^T\Dthe)^T(C^T\Dthe)}$, with the transformation $D=C^{-1}A(C^T)^{-1}$ where $CC^T$ is the Cholesky decomposition of $B=\mathcal{I}_{Z|X}(\thetaa)-\mathcal{I}_{Z}(\thetaa)$. Moreover, we have that when  $G_\Theta[\pthe]$ attains its minimum value, the Reyleigh quotient $R(D,C^T\Dthe)$ attains its maximum value of $\lambda_\text{max}$ with $C^T\Dthe=v_\text{max}$, i.e. $\Dthe=(C^T)^{-1}v_\text{max}$, where $\lambda_\text{max}$ is the largest eigenvalue of $D$ and $v_\text{max}$ the corresponding eigenvector.

\end{proof}

\subsection{Proof of Theorem \ref{thm:phase_transition}}
\label{app:phase_transition}

\begin{proof}
Define 

\beq{eq:def_T_beta}
T_\beta(\beta'):=\inf\limits_{r(z|x)\in \QQ}\left[\left(\E_{\beta}[r^2(z|x)]-\E_{\beta}[r^2(z)]\right)-\beta'\cdot\left(\E_{\beta}[r^2(z|y)]-\E_{\beta}[r^2(z)]\right)\right]
\eeq
where $\E_\beta[\cdot]$ denotes taking expectation w.r.t. the optimal solution 
$p_{\beta}^*(x,y,z)=p(x,y)p_{\beta}^*(z|x)$ at $\beta$. Using Lemma \ref{lemma:expand_to_order_n}, we have that the IB phase transition as defined in Definition \ref{definition:phase_transition_0} corresponds to satisfying the following two equations:
\beq{eq:proof_theorem_1}
T_\beta(\beta')\rvert_{\beta'=\beta}\ge0\\
\eeq
\beq{eq:proof_theorem_1_eq20}
\lim\limits_{\beta'\to\beta^+} T_\beta(\beta')=0^-
\eeq

Now we prove that $T_\beta(\beta')$ is continuous at $\beta'=\beta$, i.e. $\forall \varepsilon>0$, $\exists \delta>0$ s.t. $\forall \beta\in(\beta-\delta,\beta+\delta)$, we have $|T_\beta(\beta')-T_\beta(\beta)|<\epsilon$.

From Eq. (\ref{eq:def_T_beta}), we have
$T_\beta(\beta')-T_\beta(\beta)=-(\beta'-\beta)\cdot\left(\E_{\beta}[r^2(z|y)]-\E_{\beta}[r^2(z)]\right)$. Since $r(z|x)$ is bounded, i.e. $\exists M>0 \text{ s.t. }\forall z\in\Z, x\in\X, |r(z|x)|\le M$, we have

\begin{equation*}
\begin{aligned}
\left\lvert\E_\beta\left[r^2(z|y)\right]\right\rvert=\left\lvert\E_\beta\left[\left(\E_{x\sim p(x|y,z)}\left[r(z|x)\right]\right)^2\right]\right\rvert\le &\left\lvert\E_\beta\left[\left(\E_{x\sim p(x|y,z)}\left[M\right]\right)^2\right]\right\rvert=M^2
\end{aligned}
\end{equation*}

Similarly, we have
\begin{equation*}
\begin{aligned}
\left\lvert\E_\beta\left[r^2(z)\right]\right\rvert=\left\lvert\E_\beta\left[\left(\E_{x\sim p(x|z)}\left[r(z|x)\right]\right)^2\right]\right\rvert\le&\left\lvert\E_\beta\left[\left(\E_{x\sim p(x|z)}\left[M\right]\right)^2\right]\right\rvert=M^2
\end{aligned}
\end{equation*}

Hence, $\left\lvert T_\beta(\beta')-T_\beta(\beta)\right\rvert= |\beta'-\beta|\left\lvert E_\beta[r^2(z|y)]-E_\beta[r^2(z)]\right\rvert\le 2|\beta'-\beta|M^2$. 

To prove that $T_\beta(\beta')$ is continuous at $\beta'=\beta$, we have $\forall\varepsilon>0$, $\exists \delta=\frac{\varepsilon}{2M^2}>0$, s.t. $\forall \beta'\in(\beta-\delta,\beta+\delta)$, we have
$$|T_\beta(\beta')-T_\beta(\beta)|\le 2|\beta'-\beta|M^2<2 \delta M^2=2\frac{\varepsilon}{2M^2}M^2=\varepsilon$$
Hence $T_\beta(\beta')$ is continuous at $\beta'=\beta$.

Combining the continuity of $T_\beta(\beta')$ at $\beta'=\beta$, and Eq. (\ref{eq:proof_theorem_1}) and (\ref{eq:proof_theorem_1_eq20}), we have $T_\beta(\beta)=0$, which is equivalent to $G[p_\beta^*(z|x)]=\beta$ after simple manipulation.

\end{proof}

\subsection{Invariance of \titlemath{$\G[r(z|x);p(z|x)]$} to addition of a global representation}
\label{app:invariance_to_sz}

Here we prove the following lemma:

\begin{lemma}
\label{lemma:G_invariance}
$\G[r(z|x);p(z|x)]$ defined in Lemma \ref{lemma:second_order_variation} is invariant to the transformation $r'(z|x)\gets r(z|x)+s(z)$.
\end{lemma}

\begin{proof}
When we $r(z|x)$ is shifted by  a global transformation $r'(z|x)\gets r(z|x)+s(z)$, we have $r'(z)\gets\E_{x\sim p(x|z)}[r(z|x)+s(z)]=\E_{x\sim p(x|z)}[r(z|x)]+s(z)\E_{x\sim p(x|z)}[1]=r(z)+s(z)$, and similarly $r'(z|y)\gets r(z|y)+s(z)$.

The numerator of $\G[r(z|x);p(z|x)]$ is then

\begin{equation*}
\begin{aligned}
&\E_{x,z\sim p(x,z)}\left[\left(r'(z|x)\right)^2\right]-\E_{z\sim p(z)}\left[\left(r'(z)\right)^2\right]\\
=&\E_{x,z\sim p(x,z)}\left[\left(r(z|x)+s(z)\right)^2\right]-\E_{z\sim p(z)}\left[\left(r(z)+s(z)\right)^2\right]\\
=&\left(\E_{x,z\sim p(x,z)}\left[r^2(z|x)\right]+2\E_{x,z\sim p(x,z)}\left[r(z|x)s(z)\right]+\E_{x,z\sim p(x,z)}\left[s^2(z)\right]\right)\\
&-\left(\E_{z\sim p(z)}\left[r^2(z)\right]+2\E_{z\sim p(z)}\left[r(z)s(z)\right]+\E_{z\sim p(z)}\left[s^2(z)\right]\right)\\
=&\left(\E_{x,z\sim p(x,z)}\left[r^2(z|x)\right]+2\E_{z\sim p(z)}\left[s(z)\E_{x\sim p(x|z)}\left[r(z|x)\right]\right]+\E_{z\sim p(z)}\left[s^2(z)\right]\right)\\
&-\left(\E_{z\sim p(z)}\left[r^2(z)\right]+2\E_{z\sim p(z)}\left[r(z)s(z)\right]+\E_{z\sim p(z)}\left[s^2(z)\right]\right)\\
=&\left(\E_{x,z\sim p(x,z)}\left[r^2(z|x)\right]+2\E_{z\sim p(z)}\left[r(z)s(z)\right]+\E_{z\sim p(z)}\left[s^2(z)\right]\right)\\
&-\left(\E_{z\sim p(z)}\left[r^2(z)\right]+2\E_{z\sim p(z)}\left[r(z)s(z)\right]+\E_{z\sim p(z)}\left[s^2(z)\right]\right)\\
=&\E_{x,z\sim p(x,z)}\left[r^2(z|x)\right]-\E_{z\sim p(z)}\left[r^2(z)\right]
\end{aligned}
\end{equation*}

Symmetrically, we have
$$\E_{y,z\sim p(y,z)}\left[\left(r'(z|y)\right)^2\right]-\E_{z\sim p(z)}\left[\left(r'(z)\right)^2\right]=\E_{y,z\sim p(y,z)}\left[r^2(z|y)\right]-\E_{z\sim p(z)}\left[r^2(z)\right]$$

Therefore, $\G[r(z|x);p(z|x)]=\frac{\E_{x,z\sim p(x,z)}\left[r^2(z|x)\right]-\E_{z\sim p(z)}\left[r^2(z)\right]}{\E_{y,z\sim p(y,z)}\left[r^2(z|y)\right]-\E_{z\sim p(z)}\left[r^2(z)\right]}$ is invariant to $r'(z|x)\gets r(z|x)+s(z)$.
\end{proof}

\subsection{Proof of Theorem \ref{thm:rho_r_property}}
\label{app:representational_maximum_correlation}

\begin{proof}
Using the condition of the theorem, we have that $\forall r(z|x)\in\Q^0_{\Z|\X}$, there exists $r_1(z|x)\in\QQ$ and $s(z)\in\{s:\Z\to\R| s\text{ bounded}\}$ s.t. $r(z|x)=r_1(z|x)+s(z)$. Note that the only difference between $\QQ$ and $\QQA$ is that $\QQ$ requires $\E_{p(z|x)}[r_1(z|x)]=0$. Using Lemma \ref{lemma:G_invariance}, we have

$$\inf_{r(z|x)\in\QQA}\G[r(z|x);p(z|x)]=\inf_{r_1(z|x)\in\QQ}\G[r_1(z|x);p(z|x)]=G[p(z|x)]$$

where $r(z|x)$ doesn't have the constraint of $\E_{p(z|x)}[\cdot]=0$.

After dropping the constraint of $\E_{z\sim p(z|x)}[r(z|x)]=0$, again using Lemma \ref{lemma:G_invariance}, we can let $r(z)=\E_{x\sim p(x|z)}[r(z|x)]=0$ (since we can perform the transformation $r'(z|x)\gets r(z|x)-r(z)$, so that the new $r'(z)\equiv0$). Now we get a simpler formula for $G[p(z|x)]$, as follows:

\beq{eq:G_without_rz}
G[p(z|x)]=\inf_{r(z|x)\in\QQB}\frac{\E_{x,z\sim p(x,z)}[r^2(z|x)]}{\E_{y,z\sim p(y,z)}\left[\left(\E_{x\sim p(x|y,z)}[r(z|x)]\right)^2\right]}
\eeq

where $\QQB:=\{r:\X\times\Z\to\R\ \big|\E_{x\sim p(x|z)}[r(z|x)]=0, r\text{ bounded}\}$.

From Eq. (\ref{eq:G_without_rz}), we can further require that $\E_{x,z\sim p(x,z)}[r^2(z|x)]=1$. Define

\begin{equation}
\label{eq:rho_s}
\rho_s^2(X,Y;Z):=\sup_{f(X,Z)\in\QQC}\E[\left(\E[f(X,Z)|Y,Z]\right)^2]=\sup_{f(x,z)\in\QQC}\E_{y,z\sim p(y,z)}\left[\left(\E_{x\sim p(x|y,z)}[f(x,z)]\right)^2\right]
\end{equation}

where\footnote{In the definition of $\rho_r(X,Y;Z)$, we have used an equivalent format $f(x,z)$ instead of $r(z|x)$.} $\QQC:=\{r:\X\times\Z\to\R\ \big|\E_{x\sim p(x|z)}[r(z|x)]=0, \E_{x,z\sim p(x,z)}[r^2(z|x)]=1, r\text{ bounded}\}$. Comparing with Eq. (\ref{eq:G_without_rz}), it immediately follows that

$$G[p(z|x)]=\frac{1}{\rho_s^2(X,Y;Z)}$$

\textbf{(i)} We only have to prove that $\rho_s(X,Y;Z)=\rho_r(X,Y;Z)$, where $\rho_r(X,Y;Z)$ is defined in Definition \ref{def:rho_r}.

We have
\begin{equation*}
\begin{aligned}
&\E[f(X,Z)g(Y,Z)]\\
=&\int dxdydz p(x,y,z)f(x,z)g(y,z)\\
=&\int dydz p(y,z)g(y,z)\int dx p(x|y,z)f(x,z)\\
\equiv&\int dydz p(y,z)g(y,z)F(y,z)\\
\le&\sqrt{\int dydzp(y,z)g^2(y,z)}\cdot \sqrt{\int dydzp(y,z)F^2(y,z)}
\end{aligned}
\end{equation*}

where $F(y,z):= \int dx p(x|y,z)f(x,z)$. We have used Cauchy-Schwarz inequality, where the equality holds when $g(y,z)=\alpha F(y,z)$ for some $\alpha$. Since $\E[g^2(y,z)]=1$, we have $\alpha^2\E[F^2(y,z)]=1$. Taking the supremum of $(\E[f(X,Z)g(Y,Z)])^2$ w.r.t. $f$ and $g$, we have
\begin{equation*}
\begin{aligned}
\rho_r^2(X,Y;Z)
=&\sup_{(f(X,Z),g(Y,Z))\in\S_1}(\E[f(X,Z)g(Y,Z)])^2\\
=&\sup_{(f(x,z),g(y,z))\in\S_1}\int dydzp(y,z)g^2(y,z)\cdot \int dydzp(y,z)F^2(y,z)\\
=&\sup_{f(x,z)\in\QQC}\int dydzp(y,z)F^2(y,z)\\
=&\sup_{f(x,z)\in\QQC}\int dydzp(y,z)\left(\int dx p(x|y,z)f(x,z)\right)^2\\
=&\sup_{f(X,Z)\in\QQC}\E[\left(\E[f(X,Z)|Y,Z]\right)^2]\\
\equiv&\rho_s^2(X,Y;Z)
\end{aligned}
\end{equation*}

Here $\S_1$ is defined in Definition \ref{def:rho_r}. By definition both $\rho_r(X,Y;Z)$ and $\rho_s(X,Y;Z)$ take non-negative values. Therefore, 

\begin{equation}
\label{eq:rho_s_rho_r}
\rho_s(X,Y;Z)=\rho_r(X,Y;Z)
\end{equation}

\textbf{(ii)} 
Using the definition of $\rho_r(X,Y;Z)$, we have

\begin{equation*}
\begin{aligned}
&\rho_r^2(X,Y;Z)\\
\equiv&\sup_{f(x,z)\in\QQC}\int dydzp(y,z)\left(\int dx p(x|y,z)f(x,z)\right)^2\\
=&\sup_{f(x,z)\in\QQC}\int dzp(z)\int dy p(y|z)\left(\int dx p(x|y,z)f(x,z)\right)^2\\
\equiv&\sup_{f(x,z)\in\QQC}\int dzp(z)W[f(x,z)]\\
\end{aligned}
\end{equation*}

where $W[f(x,z)]:=\int dy p(y|z)\left(\int dx p(x|y,z)f(x,z)\right)^2$.

Denote $c(z):= p(z)\E_{x\sim p(x|z)}[f^2(x,z)]$, we have $\int c(z)dz=\E_{x,z\sim p(x,z)}[f^2(x,z)]=1$. Then the supremum $\rho_r^2(X,Y;Z)=\sup\limits_{f(x,z)\in\QQC}\int dzp(z)W[f(x,z)]$ is equivalent to the following two-stage supremum:

\beq{eq_ap:two_stage_supremum}
\rho_r^2(X,Y;Z)=\sup_{c(z):\int c(z)dz=1}\int dz p(z)\sup_{f(x,z)\in\QQD}W[f(x,z)]
\eeq

where $\QQD:=\{\X\times\Z\to\R\ \big|\ \E_{x\sim p(x|z)}[f^2(x,z)]=\frac{c(z)}{p(z)}, \E_{x\sim p(x|z)}[f(x,z)]=0, f\text{ bounded}\}$
We can think of the inner supremum $\sup_{f(x,z)\in\QQD}W[f(x,z)]$ as only w.r.t. $x$, for some given $z$.

Now let's consider another supremum:

\beq{eq:rho_r_x}
\sup_{h(x)\in\Q_{\X}^{(h)}}\int dy p(y|z)\left(\int dx p(x|y,z)h(x)\right)^2
\eeq

where $\Q_{\X}^{(h)}:=\{h:\X\to\R\big|\ \E_{p(x|z)}[h(x)]=0, \E_{p(x|z)}[h^2(x)]=1, h \text{ bounded}\}$. Using similar technique in (ii), it is easy to prove that it equals $\rho_m^2(X,Y|Z)$ as defined in Definition \ref{def:rho_r}.

Comparing Eq. (\ref{eq:rho_r_x}) and the supremum:
$$\sup_{f(x,z)\in\QQD}W[f(x,z)]$$
we see that the only difference is that in the latter $\E_{x\sim p(x|z)}[f^2(x,z)]$ equals $\frac{c(z)}{p(z)}$ instead of 1. Since $W[f(x,z)]$ is a quadratic functional of $f(x,z)$, we have 

$$\sup_{f(x,z)\in\QQD}W[f(x,z)]=\frac{c(z)}{p(z)}\rho_m^2(X,Y|Z)$$

Therefore,

\begin{equation*}
\begin{aligned}
\rho_r(X,Y;Z)&=\sup_{c(z):\int c(z)dz=1}\int dz\, p(z)\sup_{f(x,z)\in\QQD}W[f(x,z)]\\
&=\sup_{c(z):\int c(z)dz=1}\int dz \,p(z)\frac{c(z)}{p(z)}\rho_m^2(X,Y|Z)\\
&=\sup_{c(z):\int c(z)dz=1}\int dz \, c(z)\rho_m^2(X,Y|Z=z)\\
&=\sup_{Z\in\Z}\rho_m^2(X,Y|Z)\\
\end{aligned}
\end{equation*}

where in the last equality we have let $c(z)$ have ``mass'' only on the place where $\rho_m^2(X,Y|Z=z)$ attains supremum w.r.t. $z$.

\textbf{(iii)} When $Z$ is a continuous variable, let $f(x,z)=f_X(x)\sqrt{\frac{\delta(z-z_0)}{p(z)}}$, where $\delta(\cdot)$ is the Dirac-delta function, $z_0$ is a parameter, $f_X(x)\in\Q_{\X|\Z}^{(f)}$, with $\Q_{\X|\Z}^{(f)}:=\{f_X:\X\to\R\,\big|\,f_X\text{ bounded}; \forall Z\in\Z: \E_{X\sim p(X|Z)}[f_X(x)]=0, \E_{X\sim p(X|Z)}[f_X^2(x)]=1\}$. We have

\begin{equation*}
\begin{aligned}
\E_{x\sim p(x|z)}[f(x,z)]&=\int p(x|z)f(x,z)dx\\
&=\sqrt{\frac{\delta(z-z_0)}{p(z)}}\int p(x|z)f_X(x)dx\\
&=\sqrt{\frac{\delta(z-z_0)}{p(z)}}\cdot 0\\ 
&= 0
\end{aligned}
\end{equation*}

And 
\begin{equation*}
\begin{aligned}
\E[f^2(X,Z)]=&\int p(x,z)f^2(x,z)dxdz\\
=&\int p(x,z)f_X^2(x)\frac{\delta(z-z_0)}{p(z)}dxdz\\
=&\int dz \delta(z-z_0)\int dx p(x|z)f_X^2(x)dx\\
=&\int dz \delta(z-z_0)\cdot 1\\
=&1
\end{aligned}
\end{equation*}

Therefore, such constructed $f(x,z)=f_X(x)\sqrt{\frac{\delta(z-z_0)}{p(z)}}\in\QQC$, satisfying the requirement for $\rho_s(X,Y;Z)$ (which equals $\rho_r(X,Y;Z)$ by Eq. \ref{eq:rho_s_rho_r}).

Substituting in the special form of $f(x,z)$ into the expression of $\rho_s(X,Y;Z)$ in Eq. (\ref{eq:rho_s}), we have

\begin{equation*}
\begin{aligned}
&\sup_{f(x,z):f(x,z)=f_X(x)\sqrt{\frac{\delta(z-z_0)}{p(z)}},f_X(x)\in\QQf}\int dzp(z)\int dy p(y|z)\left(\int dx p(x|y,z)f(x,z)\right)^2\\
=&\sup_{f_X(x)\in\QQf,z_0\in\Z}\int dzp(z)\int dy p(y|z)\left(\int dx p(x|y,z)f_X(x)\sqrt{\frac{\delta(z-z_0)}{p(z)}}\right)^2\\
=&\sup_{z_0\in\Z}\int dzp(z)\frac{\delta(z-z_0)}{p(z)}\sup_{f_X(x)\in\QQf}\int dy p(y|z)\left(\int dx p(x|y,z)f_X(x)\right)^2\\
=&\sup_{z_0\in\Z}\int dz\delta(z-z_0)\sup_{f_X(X)\in\QQf}\E[(\E[f_X(X)|Y,Z=z])^2|Z=z]\\
=&\sup_{z_0\in\Z}\int dz\delta(z-z_0)\rho_m^2(X,Y|Z=z)\\
=&\sup_{z_0\in\Z}\rho_m^2(X,Y|Z=z_0)\\
=&\sup_{Z\in\Z}\rho_m^2(X,Y|Z)\\
\end{aligned}
\end{equation*}

We can identify $\sup_{f_X(X)\in\QQf}\E[(\E[f_X(X)|Y,Z=z])^2|Z=z]$ with $\rho_m^2(X,Y|Z=z)$ because $f_X(x)$ satisfies the requirement for conditional maximum correlation that $\E_{p(x|z)}[f_X(x)]=0$ and $\E_{p(x|z)}[f_X^2(x)]=1$, for any $z$, and using the same technique in (i), it is straightforward to prove that $\sup_{f_X(X)\in\QQf}\E[(\E[f_X(X)|Y,Z=z])^2|Z=z]$ equals the conditional maximum correlation as defined in Definition \ref{def:rho_r}.

Since the conditional maximum correlation can be viewed as the maximum correlation between $X$ and $Y$, where $X,Y\sim p(X,Y|Z)$, using the equality of $\left(\beta_0[h(x)]\right)^{-1}=\rho_m^2(X;Y)$ (Eq. 7 in \cite{wu2019learnabilityEntropy}), we can identify the $h(x)$ in $\beta_0[h(x)]$ with the $f_X(X)$ here, and an optimal $f_X^*(X)$ that maximizes $\rho_m^2(X,Y|Z)$ is also an optimal $h^*(x)$ that minimizes $\beta_0[h(x)]$.

\textbf{(iv)}
For discrete $X$, $Y$ and $Z$ and a given $Z=z$, let $Q_{X,Y|Z}:=\left(\frac{p(x,y|z)}{\sqrt{p(x|z)p(y|z)}}\right)_{x,y}=\left(\frac{p(x,y)}{\sqrt{p(x)p(y)}}\sqrt{\frac{p(z|x)}{p(z|y)}}\right)_{x,y}$, we first prove that its second largest singular value is $\rho_m^2(X,Y|Z)=\sup\limits_{(f,g)\in\S_2}\E_{x,y\sim p(x,y|z)}[f(x)g(y)]$ ($\S_2$ is defined in Definition \ref{def:rho_r}). 

Let column vectors $u_1=\sqrt{p(x|z)}$ and $v_1=\sqrt{p(y|z)}$ (note that $z$ is given and fixed). Also let $u_2=f(x)\sqrt{p(x|z)}$ and $v_2=g(y)\sqrt{p(y|z)}$. Denote inner product $\langle u,v\rangle\equiv \sum_i u_iv_i$, and the length of a vector as $||u||=\sqrt{\langle u,u\rangle}$. We have $||u_1||=||v_1||=1$ due to the normalization of probability, $||u_2||=||v_2||=1$ due to $\E_{x\sim p(x|z)}[f^2(x)]=\E_{y\sim p(y|z)}[g^2(y)]=1$, and $\langle u_1, u_2\rangle=\langle v_1, v_2\rangle=0$ due to $\E_{x\sim p(x|z)}[f(x)]=\E_{y\sim p(y|z)}[g(y)]=0$. Furthermore, we have

$$\sup_{(f,g)\in\S_2}\E_{x,y\sim p(x,y|z)}[f(x)g(y)]=\max_{u,v}u^T Q_{X,Y|Z} v$$

which is exactly the second largest singular value $\sigma_2(Z)$ of the matrix $Q_{X,Y|Z}$. Using the result in \textbf{(ii)}, we have that $\rho_r(X,Y;Z)=\max\limits_{Z\in\Z}\sigma_2(Z)$.

\end{proof}

\subsection{Subset separation at phase transitions}
\label{app:subset_separation}

\begin{figure}[b]
\begin{center}
\includegraphics[width=0.3\columnwidth]{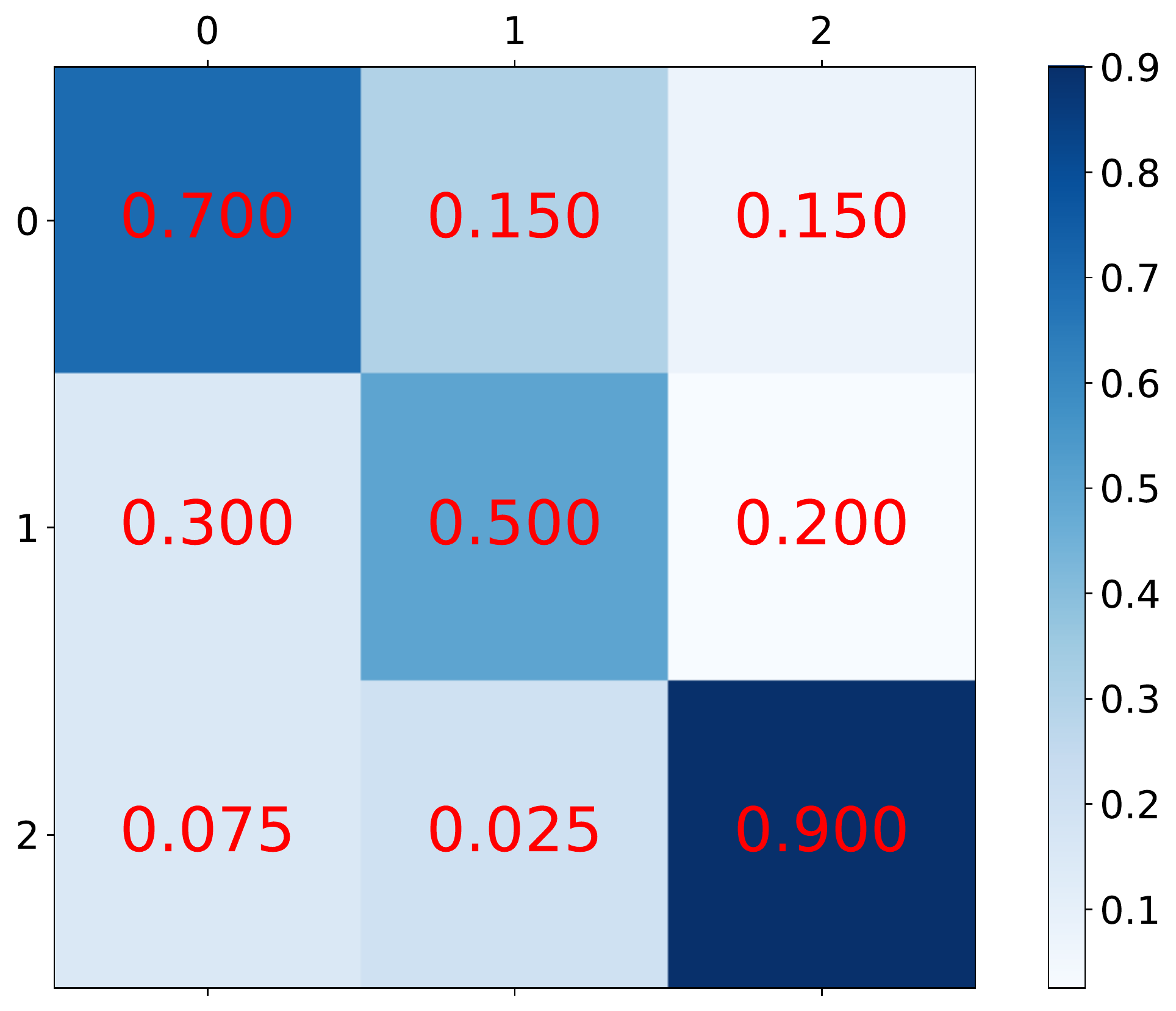}
\end{center}
\caption{$p(y|x)$ for the categorical dataset in Fig. \ref{fig:beta_th_vs_beta} and Fig. \ref{fig:subset_separation}. The value in $i^\text{th}$ row and $j^\text{th}$ column denotes $p(y=j|x=i)$. $p(x)$ is uniform.}
\label{fig:py_x}
\end{figure}

In this section we study the behavior of $p(z|x)$ on the phase transitions. We use the same categorical dataset (where $|X|=|Y|=|Z|=3$ and $p(x)$ is uniform, and $p(y|x)$ is given in Fig. \ref{fig:py_x}). In Fig. \ref{fig:subset_separation} we show the $p(z|x)$ on the simplex before and after each phase transition. We see that the first phase transition corresponds to the separation of $x=2$ (belonging to $y=2$) w.r.t. $x\in\{0,1\}$ (belonging to classes $y\in\{0,1\}$), on the $p(z|x)$ simplex. The second phase transition corresponds to the separation of $x=0$ with $x=1$. Therefore, each phase transition corresponds to the ability to distinguish subset of examples, and learning of new classes.

\begin{figure}[hp]
\begin{center}
\begin{subfigure}[b]{0.51\textwidth}
\centering
\includegraphics[width=\textwidth]{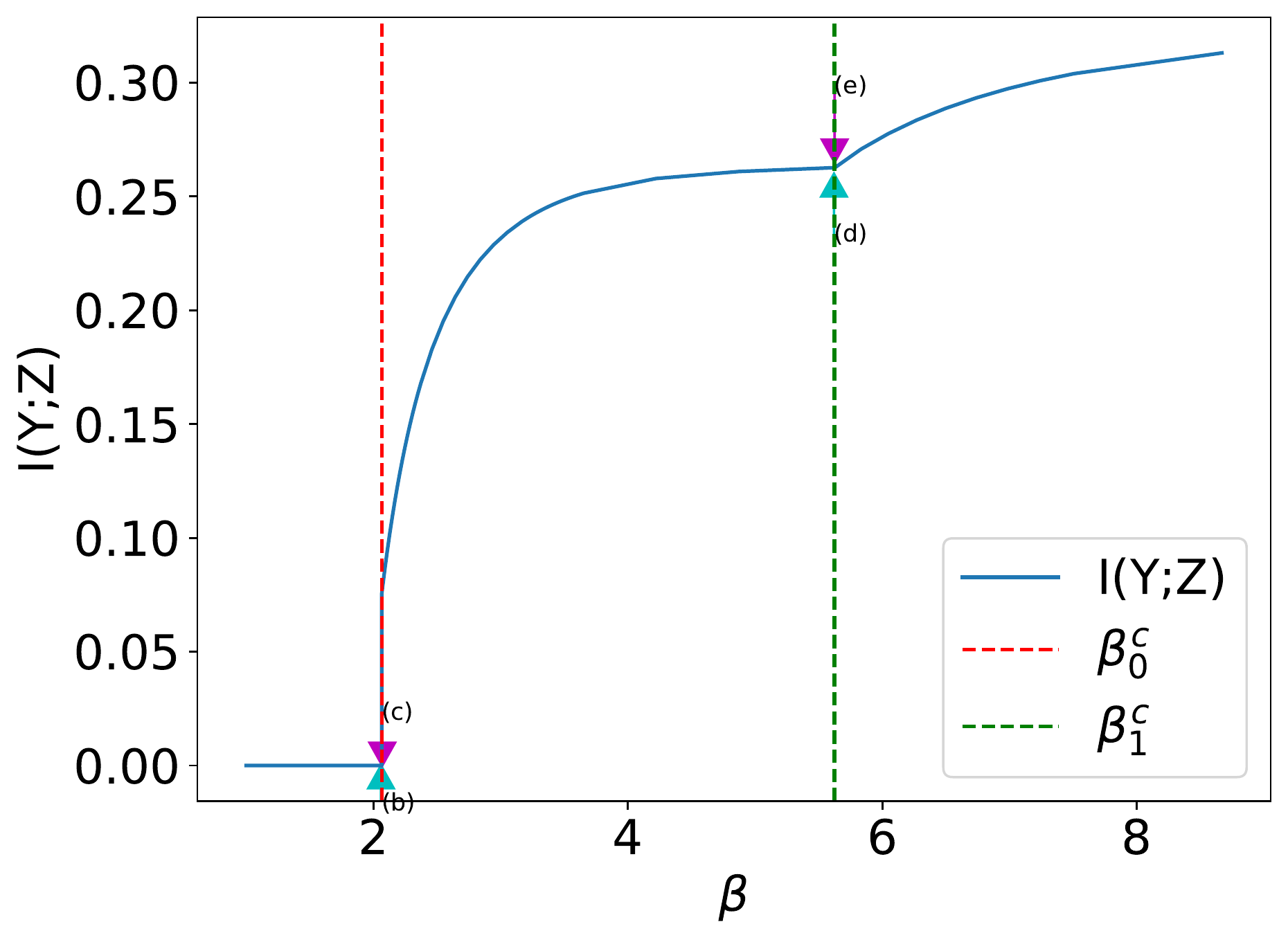}
\caption{}
\end{subfigure}
\begin{subfigure}[b]{0.49\textwidth}
\centering
\includegraphics[width=\textwidth]{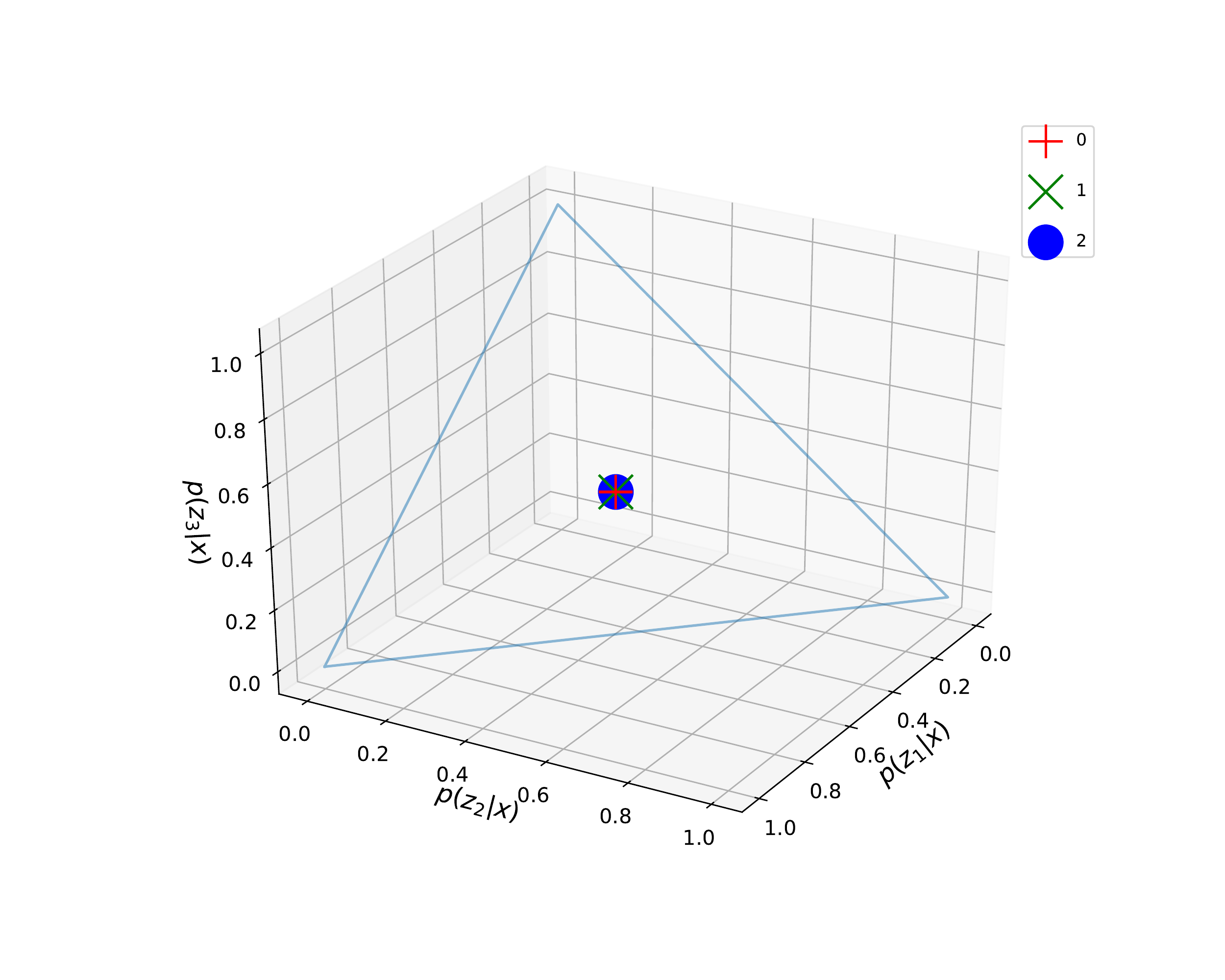}
\caption{}
\end{subfigure}
\begin{subfigure}[b]{0.49\textwidth}
\centering
\includegraphics[width=\textwidth]{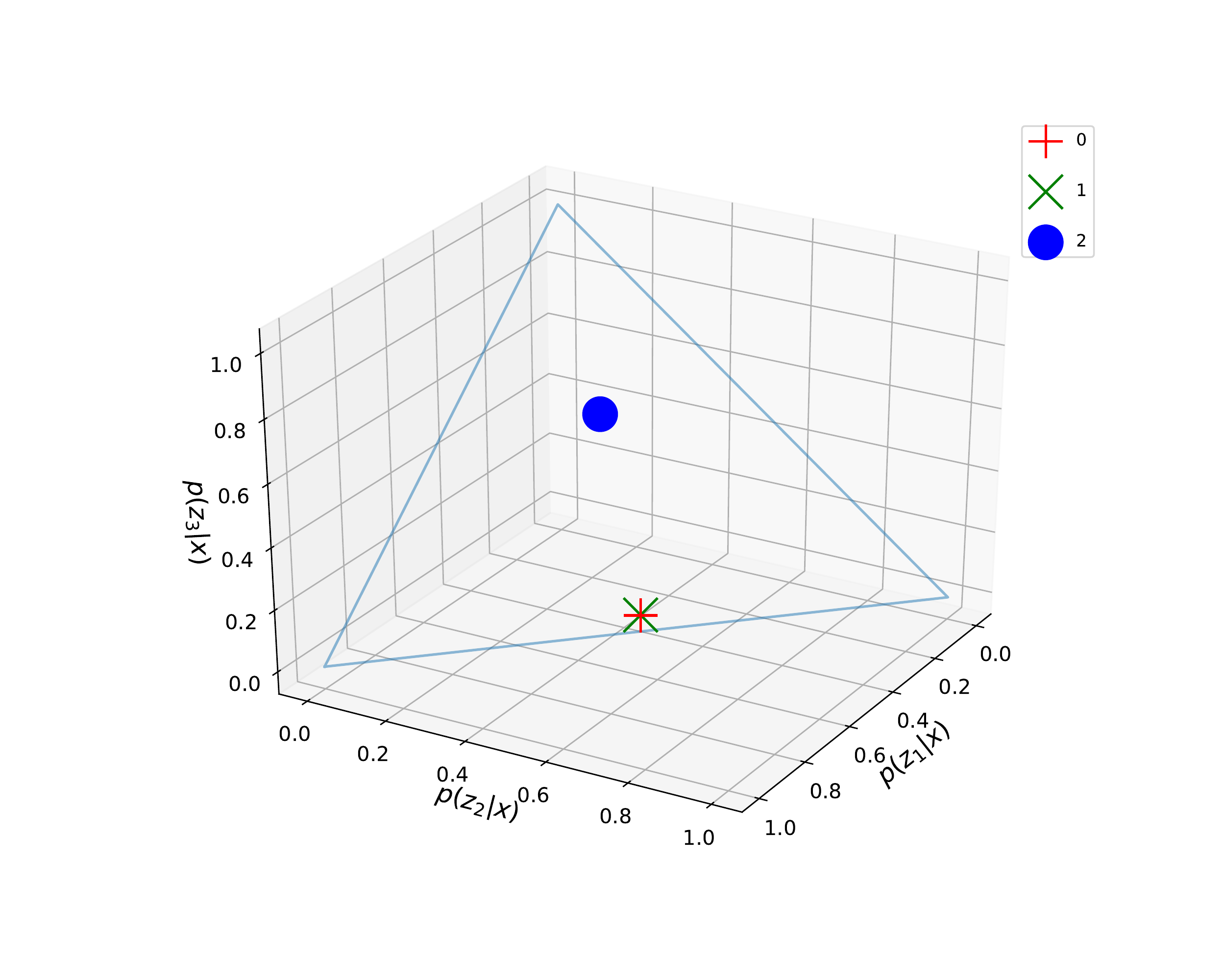}
\caption{}
\end{subfigure}
\begin{subfigure}[b]{0.49\textwidth}
\centering
\includegraphics[width=\textwidth]{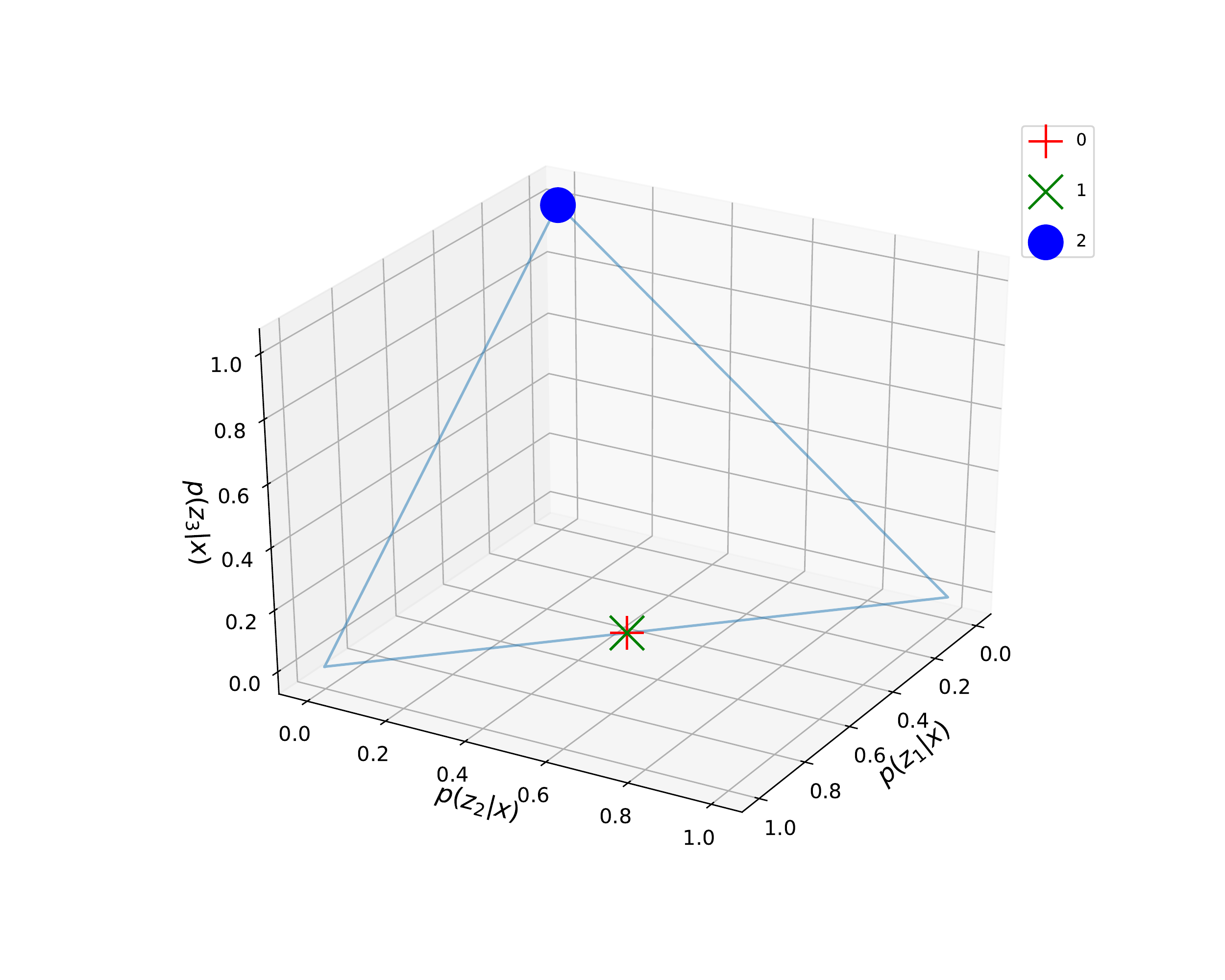}
\caption{}
\end{subfigure}
\begin{subfigure}[b]{0.49\textwidth}
\centering
\includegraphics[width=\textwidth]{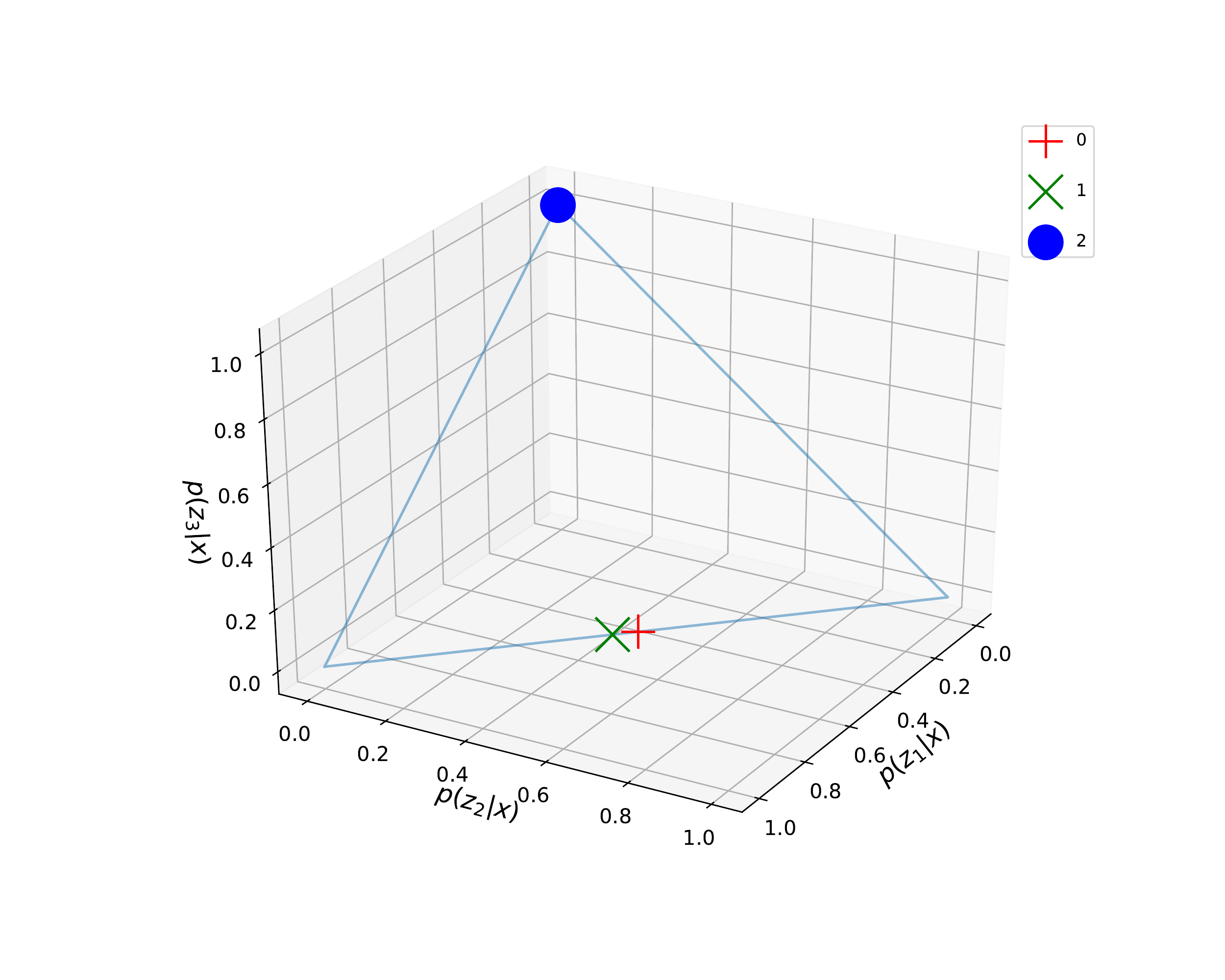}
\caption{}
\end{subfigure}
\end{center}
\caption{(a) $I(Y;Z)$ vs. $\beta$ for the dataset given in Fig. \ref{fig:py_x}. The phase transitions are marked with vertical dashed line, with $\beta_0^c=2.065571$ and $\beta_1^c=5.623333$. (b)-(e) Optimal $p_\beta^*(z|x)$ for four values of $\beta$, i.e. (b) $\beta=2.060$, (c) $\beta=2.070$, (d) $\beta=5.620$ (e) $\beta=5.625$ (their $\beta$ values are also marked in (a)), where each marker  denotes $p(z|x=i)$ for a given $i\in\{0,1,2\}$.
}
\label{fig:subset_separation}
\end{figure}

\subsection{MNIST Experiment Details}
\label{app:MNIST_exp}

\begin{figure}[!ht]
\begin{center}
\includegraphics[width=0.3\columnwidth]{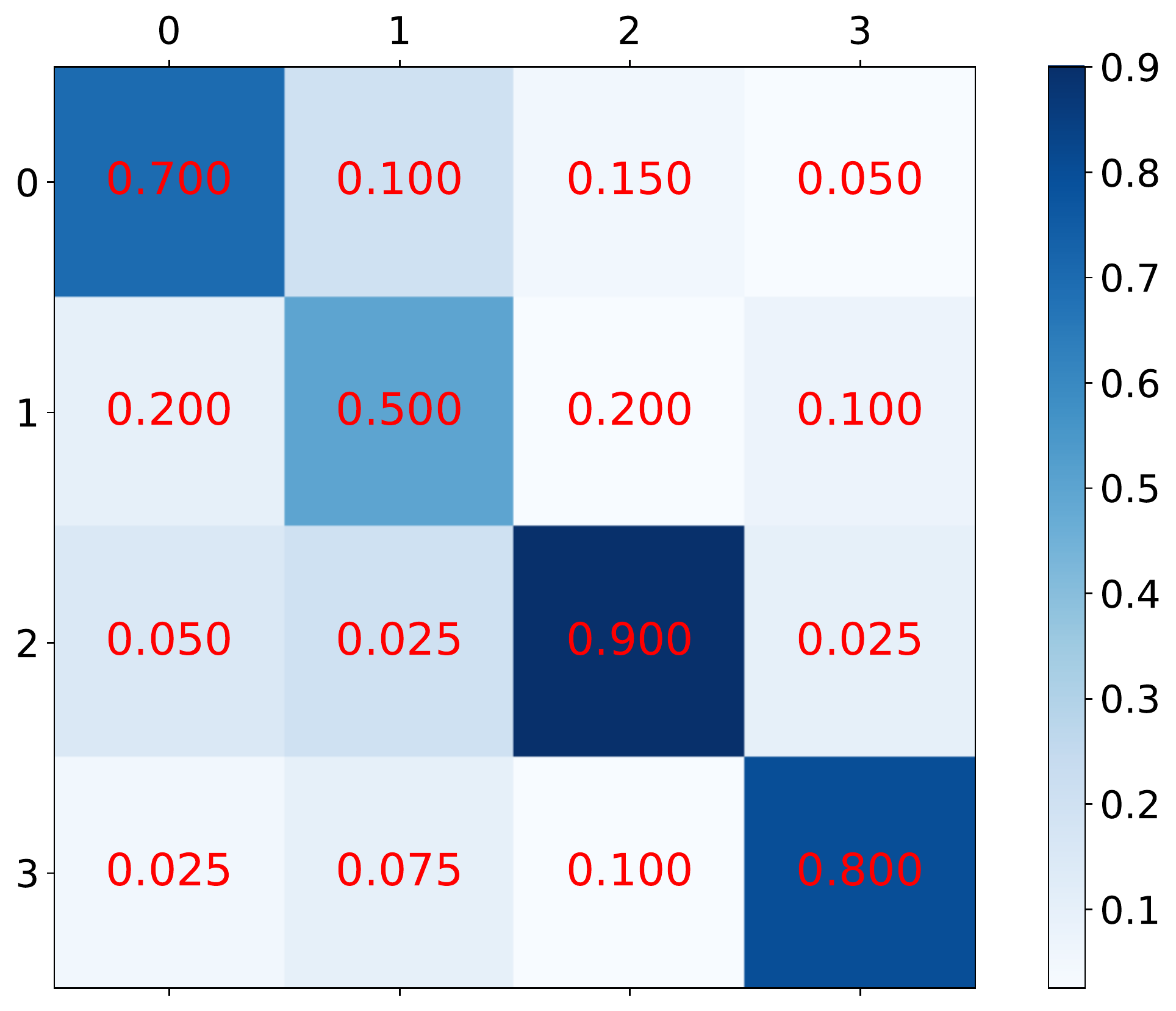}
\end{center}
\caption{Confusion matrix for MNIST experiment. The value in $i^\text{th}$ row and $j^\text{th}$ column denotes $p(\tilde{y}=j|y=i)$ for the label noise.}
\label{fig:py_x_4}
\end{figure}

We use the MNIST training examples with class $0, 1, 2, 3$, with a hidden label-noise matrix as given in Fig. \ref{fig:py_x_4}, based on which at each minibatch we dynamically sample the observed label. We use conditional entropy bottleneck (CEB) \citep{fischer2018the} as the variational IB objective, and run multiple independent instances with different the target $\beta$. We jump start learning by started training at $\beta=100$ for 100 epochs, annealing $\beta$ from 100 down to the target $\beta$ over 600 epochs, and continue to train at the target epoch for another 800 epochs. The encoder is a three-layer neural net, where each hidden layer has 512 neurons and leakyReLU activation, and the last layer has linear activation. The classifier $p(y|z)$ is a 2-layer neural net with a 128-neuron ReLU hidden layer. The backward encoder $p(z|y)$ is also a 2-layer neural net with a 128-neuron ReLU hidden layer. We trained with Adam \citep{kingma2013auto} at learning rate of $10^{-3}$, and  anneal down with factor $1/(1+0.01\cdot \text{epoch})$. For Alg. 1, for the $f_\thetaa$ we use the same architecture as the encoder of CEB, and use $|Z|=50$ in Alg. 1.

\subsection{CIFAR10 Experiment Details}
\label{app:phase_cifar_details}

We use the same CIFAR10 class confusion matrix provided in~\cite{wu2019learnabilityEntropy} to generate noisy labels with about 20\% label noise on average (reproduced in Table~\ref{table:cifar_confusion}).
We trained $28 \times 1$ Wide ResNet~\citep{resnet,wideresnet} models using the open source implementation from~\cite{autoaugment} as encoders for the Variational Information Bottleneck (VIB)~\citep{alemi2016deep}.
The 10 dimensional output of the encoder parameterized a mean-field Gaussian with unit covariance.
Samples from the encoder were passed to the classifier, a 2 layer MLP.
The marginal distributions were mixtures of 500 fully covariate 10-dimensional Gaussians, all parameters of which are trained.

With this standard model, we trained 251 different models at $\beta$ from 1.0 to 6.0 with step size of 0.02.
As in~\citet{wu2019learnability}, we jump-start learning by annealing $\beta$ from 100 down to the target $\beta$.
We do this over the first 4000 steps of training.
The models continued to train for another 56,000 gradient steps after that, a total of 600 epochs.
We trained with Adam~\citep{adam} at a base learning rate of $10^{-3}$, and reduced the learning rate by a factor of 0.5 at 300, 400, and 500 epochs.
The models converged to essentially their final accuracy within 40,000 gradient steps, and then remained stable.
The accuracies reported in Figure~\ref{fig:phase_cifar_exp} are averaged across five passes over the training set. We use $|Z|=50$ in Alg. 1.

\begin{table}[ht]
\begin{center}
\caption{
Class confusion matrix used in CIFAR10 experiments, reproduced from~\citep{wu2019learnabilityEntropy}.
The value in row $i$, column $j$ means for class $i$, the probability of labeling it as class $j$. The mean confusion across the classes is 20\%.
}
\label{table:cifar_confusion}
\vskip 0.1in
\setlength{\tabcolsep}{4pt}  
\begin{tabular}{r | c c c c c c c c c c }
\tiny
& Plane & Auto. & Bird & Cat & Deer & Dog & Frog & Horse & Ship & Truck \\ [0.2ex]
\hline
\hline\noalign{\smallskip}
Plane &
0.82232 & 0.00238 & 0.021   & 0.00069 & 0.00108 & 0       & 0.00017 & 0.00019 & 0.1473  & 0.00489 \\ [0.2ex]
Auto. &
0.00233 & 0.83419 & 0.00009 & 0.00011 & 0       & 0.00001 & 0.00002 & 0       & 0.00946 & 0.15379 \\ [0.2ex]
Bird &
0.03139 & 0.00026 & 0.76082 & 0.0095  & 0.07764 & 0.01389 & 0.1031  & 0.00309 & 0.00031 & 0       \\ [0.2ex]
Cat &
0.00096 & 0.0001  & 0.00273 & 0.69325 & 0.00557 & 0.28067 & 0.01471 & 0.00191 & 0.00002 & 0.0001  \\ [0.2ex]
Deer &
0.00199 & 0       & 0.03866 & 0.00542 & 0.83435 & 0.01273 & 0.02567 & 0.08066 & 0.00052 & 0.00001 \\ [0.2ex]
Dog &
0       & 0.00004 & 0.00391 & 0.2498  & 0.00531 & 0.73191 & 0.00477 & 0.00423 & 0.00001 & 0       \\ [0.2ex]
Frog &
0.00067 & 0.00008 & 0.06303 & 0.05025 & 0.0337  & 0.00842 & 0.8433  & 0       & 0.00054 & 0       \\ [0.2ex]
Horse &
0.00157 & 0.00006 & 0.00649 & 0.00295 & 0.13058 & 0.02287 & 0       & 0.83328 & 0.00023 & 0.00196 \\ [0.2ex]
Ship &
0.1288  & 0.01668 & 0.00029 & 0.00002 & 0.00164 & 0.00006 & 0.00027 & 0.00017 & 0.83385 & 0.01822 \\ [0.2ex]
Truck &
0.01007 & 0.15107 & 0       & 0.00015 & 0.00001 & 0.00001 & 0       & 0.00048 & 0.02549 & 0.81273 \\ [0.2ex]
\hline
\end{tabular}
\end{center}
\end{table}

\section{Appendix for Chapter \ref{chap5:distillation}}

\subsection{Binning can be practically lossless}
\label{LosslessBinningAppendix}

If the conditional probability distribution $p_1(w)\equiv P(Y\hbox{=}1|W\hbox{=}w)$
is a slowly varying function and the range of $W$ is divided into tiny bins, then 
$p_1(w)$ will be almost constant within each bin and so binning $W$ (discarding information about the exact position of $W$ within a bin) should destroy almost no information about $Y$.

This intuition is formalized by the following theorem, which says that a random variable $W$
 can be binned into a finite number of bins at the cost of losing arbitrarily little information about $Y$.
\begin{theorem}
\label{LosslessBinningTheorem}
Binning can be practically lossless: Given a random variable $Y\in\{1,2\}$ and
a uniformly distributed random variable $W \in [0,1]$
such that the conditional probability distribution 
$p_1(w)\equiv P(Y\hbox{=}1|W\hbox{=}w)$ is monotonic,
there exists for any real number $\epsilon>0$ 
a vector $\b\in\mathbb{R}^{N-1}$ of bin boundaries such 
that the information reduction
$$\Delta I\equiv I[W,Y]  - I[B(W,\b),Y] < \epsilon,$$
where $B$ is the binning function defined by
\eq{BinningFuncDef}.
\end{theorem}

\begin{proof}
The binned bivariate probability distribution is 
\beq{FineBinnedProbEq}
P_{ij}\equiv P(Z\hbox{=}j,Y=i) = \int_{b_{j-1}}^{b_j}p_i(w)dw
\eeq
with marginal distribution
\beq{binnedMarginalEq}
P^Z_j\equiv P(Z\hbox{=}j) = b_j-b_{j-1}.
\eeq
Let $\pbar_i(w)$ denote the piecewise constant function 
that in the $j^{\rm th}$ bin $b_{j-1}<w\le b_j$ takes the average value of $p_i(w)$ 
in that bin, \ie, 
\beq{pbarDefEq}
\pbar_i(w)\equiv {1\over b_j-b_{j-1}}\int_{b_{j-1}}^{b_j}p_i(w)dw={P_{ij}\over P^Z_j}.
\eeq
These definitions imply that
\beq{sumIntegralEq}
-\sum_{j=1}^N P_{ij}\log{P_{ij}\over P^Z_j}=\int_0^1 h\left[\pbar_i(w)\right]dw,
\eeq
where $h(x)\equiv -x\log x$.
Since $h(x)$ vanishes at $x=0$ and $x=1$ and takes its intermediate maximum 
value 
at $x=1/e$, the function 
\beq{hstarDefEq}
h_*(x)\equiv
\left\{
\begin{tabular}{ll}
$h(x)$			&if $x<e^{-1}$,\\
$2h(e^{-1})-h(x)$	&if $x\ge e^{-1}$\\
\end{tabular}
\right.
\eeq
is continuous and increases monotonically for
$x\in [0,1]$,
with $h'_*=|h'(x)|$.
This means that if we define the non-negative monotonic function
$$h_+(w)\equiv h_*[p_1(w)] - h_*[p_2(w)],$$
it changes at least as fast as either of its terms, so that 
for any $w_1$, $w_2\in [0,1]$, we have
\beqa{hChangeBoundEq}
\left|h\left[p_i(w_2)\right] - h\left[p_i(w_1)\right]\right|
&\le&\left|h_*\left[p_i(w_2)\right] - h_*\left[p_i(w_1)\right]\right|\nonumber\\
&\le&|h_+(w_2)-h_+(w_1)|.
\eeqa
We will exploit this bound to limit how much $h\left[p_i(w)\right]$ can vary within a bin.
Since $h_+(0)\ge 0$ and $h_+(1)\le 2h_*(1)=4/e\ln 2\approx 2.12 < 3$,
we pick $N-1$ bins boundaries $b_k$ implicitly defined by
\beq{binPlacementEq}
h_+(b_j)=h_+(0) + [h_+(1)-h_+(0)] {j\over N}
\eeq
for some integer $N\gg 1$.
Using \eq{hChangeBoundEq}, this implies that 
\beq{pbarBoundEq}
\left|h\left[\pbar_i(w)\right] - h\left[p_i(w)\right]\right|\le {h_+(1)-h_+(0)\over N}<{3\over N}.
\eeq

The mutual information between two variables is given by 
$I(Y,U)=H(Y)-H(Y|U)$, where the second term (the conditional entropy is given by the following expressions in the cases that we need: 
\beqa{BinnedConditionalEntropyEq}
H(Y|Z)&=&-\sum_{i=1}^N\sum_{j=1}^2 P_{ij}\log{P_{ij}\over P_i},\\
H(Y|W)&=&-\sum_{i=1}^2 \int_0^1 p_i(w)\log p_i(w)dw\label{HYWeq}.
\eeqa
The information loss caused by our binning is therefore
\beqa{DeltaIeq}
\Delta I&=&I(W,Y)  - I(Z,Y) = H(Y|Z)  - H(Y|W)\nonumber\\
           &=&-\sum_{i=1}^2\left( \sum_{j=1}^N P_{ij}\log{P_{ij}\over P_j^Z}+\int_0^1 h\left[p_i(w)\right]dw\right)\nonumber\\
           &=&\sum_{i=1}^2 \int_0^1 \left(h\left[\pbar_i(w)\right] - h\left[p_i(w)\right]\right)dw\nonumber\\
           &\le&\sum_{i=1}^2 \int_0^1 \left|h\left[\pbar_i(w)\right] - h\left[p_i(w)\right]\right|dw\nonumber\\
           &<&\sum_{i=1}^2 \int_0^1 {3\over N} = {6\over N},
\eeqa
where we used \eq{sumIntegralEq}  to obtain the $3^{\rm rd}$ row and
\eq{pbarBoundEq} to obtain the last row.
This means that however small an information loss tolerance $\epsilon$ we want, 
we can guarantee $\Delta I<\epsilon$ by choosing $N>6/\epsilon$ bins placed according to 
\eq{binPlacementEq}, which completes the proof.
\end{proof}

Note that the proof still holds if the function $p_i(w)$ is not monotonic, as long as the number of times $M$ that it changes direction is finite: in that case, we can simply repeat the above-mentioned binning procedure separately in the $M+1$ intervals where $p_i(w)$ is monotonic, using 
$N>6/\epsilon$ bins in each interval, \ie, a total of $N>6M/\epsilon$ bins.

\subsection{More varying conditional probability boosts mutual information}

Mutual information is loosely speaking a measure of how far a probability distribution $P_{ij}$ is from being separable,
\ie, a product of its two marginal distributions.\footnote{Specifically, the mutual information is the 
Kullback–Leibler divergence between the bivariate probability distribution and the product of its marginals.}
If all conditional probabilities for one variable $Y$ given the other variable $Z$ are identical, then the distribution is separable and the mutual information $I(Z,Y)$ vanishes, so one may intuitively expect that 
making conditional probabilities more different from each other will increase $I(Z,Y)$.
The following theorem formalizes this intuition in a way that enables Theorem~\ref{ContiguousBinningTheorem}.
\begin{theorem}
\label{informationTheorem}
Consider two discrete random variables $Z\in\{1,...,n\}$ and $Y\in\{1,2\}$ 
and define $P_i\equiv P(Z=i)$,\\
 $p_i\equiv P(Y=1|Z=i)$, so that the joint probability distribution
$P_{ij}\equiv P(Z=i,Y=j)$ is given by\\
$P_{i1} = P_i p_i$, $P_{i2} = P_i (1- p_i)$.
If two conditional probabilities $p_k$ and $p_l$ differ, then we increase the mutual information 
$I(Y,Z)$ if we bring them further apart by adjusting $P_{kj}$ and $P_{lj}$ in such a way that both marginal distributions remain unchanged.
\end{theorem}

\begin{proof}
The only such change that keep the marginal distributions for both $Z$ and $Y$ unchanged takes the form 
$$
\left(
\begin{tabular}{llllll}
$P_1 p_1$		&$\cdots$		&$P_k p_k - \epsilon$			&$\cdots$		&$P_l p_l+ \epsilon$		&$\cdots$		\\
$P_1(1 - p_1)$		&$\cdots$		&$P_k(1 - p_k) + \epsilon$		&$\cdots$		&$P_l(1 - p_l) - \epsilon$		&$\cdots$		
\end{tabular}
\right)
$$
where the parameter $\epsilon$ that must be kept small enough for all probabilities to remain non-negative.
Without loss of generality, we can assume that $p_k<p_l$, so that we make the conditional probabilities
\beqa{conditionalProbEq}
P(Y=1|Z=k)={P_{k1}\over P_k} = p_k - \epsilon/P_k,\\
P(Y=1|Z=l)={P_{l1}\over P_l} = p_l + \epsilon/P_l
\eeqa
more different when increasing $\epsilon$ from zero.
Computing and differentiating the mutual information with respect to $\epsilon$, most terms cancel and we find that 
\beq{IderivativeEq}
{\partial I(Z,Y) \over\partial\epsilon} \bigg\rvert_{\epsilon=0}= 
\log\left[
{1/p_k-1  \over 1/p_l - 1}
\right] > 0
\eeq
which means that adjusting the probabilities with a sufficiently tiny $\epsilon>0$ will increase the mutual information, completing the proof.
\end{proof}

\section{Appendix for Chapter \ref{chap6:causal}}

\subsection{Hyperparameter \texorpdfstring{$\lambda$}{Lg} selection}
\label{app:select_lambda}

For selecting an appropriate hyperparameter $\lambda$, we run our experiments for the synthetic dataset with $\lambda=0.001,0.002,0.005,0.01,0.02,0.05$. For each experiment involving $N$ time series, we append $\ceil[\big]{N/2}$  independent time series $v_{t-1}^{(s)}$ ($s=1,2,...\ceil[\big]{N/2}$) to $\X_{t-1}$, generated by randomly sampling $\ceil[\big]{N/2}$ time series from $\X_{t-1}$ and performing random permutation across the examples. We also append $\ceil[\big]{N/2}$ time series $w_t^{(i)}, i=1,2,...\ceil[\big]{N/2}$ to $x_t^{(i)}$, such that $w_t^{(i)}=X_{t-1}^{(i)}\cdot Q$, where $Q$ is a fixed random  $K\times1$ matrix, so that we know $X_{t-1}^{(i)}$ causes $w_{t}^{(i)}$, and $v_{t-1}^{(s)}$ does not cause $w_{t}^{(i)}$ for any $i,s$. We apply Alg. \ref{alg:learnable_noise} to the augmented dataset, and produce the estimated predictive strength $W_{ji}$ from $[\X_{t-1},\mathbf{v}_{t-1}]$ to $[\mathbf{x}_t,\mathbf{w}_t]$. For each hyperparameter $\lambda$, we then fit a Gaussian distribution $G_{v\to w}$ to the estimated predictive strengths from $v_{t-1}^{(s)}$ to $w_{t}^{(i)}$ ($s=1,2,...\ceil[\big]{N/2};j=1,2,...\ceil[\big]{N/2}$), and fit another Gaussian distribution $G_{x\to w}$ to the estimated predictive strengths from $X_{t-1}^{(i)}$ to $w_{t}^{(i)}$, $i=1,2,...\ceil[\big]{N/2}$, and select the $\lambda$ such that the upper $4\sigma$ value of $G_{v\to w}$ is smaller than the lower $4\sigma$ value of $G_{x\to w}$. In this way, for the known causal and non-causal relations, they are sufficiently apart. We find that $\lambda=0.001$ and $\lambda=0.002$ satisfy this criterion, while larger $\lambda$ fails to satisfy. We then set $\lambda=0.002$ for all our experiments.

\subsection{Proof and analysis of the Minimum Predictive Information regularized risk}
\label{app:W_proof}

In this section we prove the three properties of $W_{ji}$ in Section \ref{sec:our_method}, and analyze why it is likely to select variables that directly causes the variable of interest.

Firstly we state the assumption that will be used throughout this section:

\begin{assumption}
\label{assump:f_theta}
Assume that $f_\theta\in\mathcal{F}$ is a continuous function and has enough capacity so that it can approximate any $\int dx^{(i)}_t P(x_t^{(i)}|\mathbf{X}_{t-1})x_t^{(i)}$.  Let $j\neq i$ and assume that $P(X^{(j)}_{t-1})$ has support with intrinsic dimension of $KM$.
\end{assumption}

Also we emphasize that in this paper, the expected risks (with symbol $R$) are w.r.t. the distributions, and the empirical risks (with symbol $\hat{R}$) are w.r.t. a dataset drawn from the distribution, with finite number of examples. The theorems in this paper are all proved w.r.t. distributions (assuming infinite number of examples). Sample complexity results will be left for future work.

Before going forward with the main proof, we first prove the following lemma.

\subsubsection{Proving a lemma}

\begin{lemma}
\label{lemma:argmin_risk}
Suppose that Assumption \ref{assump:f_theta} holds. Denote

$$R_{\mathbf{X},x^{(i)}}^{\text{MSE}}[f_{\theta}]=\mathbb{E}_{\mathbf{X}_{t-1},x_t^{(i)}}\left[\left(x_t^{(i)}-f_\theta(\mathbf{X}_{t-1})\right)^2\right]$$ as the standard MSE loss, we have
\begin{equation}
\text{argmin}_{f_\theta}R^{\text{MSE}}_{\mathbf{X},x^{(i)}}[f_{\theta}]=\int dx^{(i)}_t P(x_t^{(i)}|\mathbf{X}_{t-1})x_t^{(i)}
\end{equation}

and
\begin{equation}
\label{eq:min_mse_risk}
\text{min}_{f_\theta}R^{\text{MSE}}_{\mathbf{X},x^{(i)}}[f_{\theta}]=\mathbb{E}_{\mathbf{X}_{t-1},x_t^{(i)}}\left[\left(x_t^{(i)}-\int dx^{(i)}_t P(x_t^{(i)}|\mathbf{X}_{t-1})x_t^{(i)}\right)^2\right]
\end{equation}

In other words, for the MSE loss, its minimum is attained when $f_\theta(\mathbf{X}_{t-1})$ is the expectation of $x_t^{(i)}$ conditioned on  $\mathbf{X}_{t-1}$.

\end{lemma}

\begin{proof}

The proof of the lemma is adapted from \cite{papoulis1985probability}. The risk
\begin{equation*}
\begin{aligned}
R^{\text{MSE}}_{\mathbf{X},x^{(i)}}[f_\theta]&=\mathbb{E}_{\mathbf{X}_{t-1},x_t^{(i)}}\left[\left(x_t^{(i)}-f_\theta(\mathbf{X}_{t-1})\right)^2\right]\\
&=\int d\mathbf{X}_{t-1} dx^{(i)}_t \cdot P(\mathbf{X}_{t-1},x^{(i)}_t) \left(x_t^{(i)}-f_\theta(\mathbf{X}_{t-1})\right)^2 \\
&=\int d\mathbf{X}_{t-1} P(\mathbf{X}_{t-1})\int dx^{(i)}_t P(x_t^{(i)}|\mathbf{X}_{t-1})\left(x_t^{(i)}-f_\theta(\mathbf{X}_{t-1})\right)^2
\end{aligned}
\end{equation*}

Note that here  $(x_t^{(i)}-f_\theta(\mathbf{X}_{t-1}))^2\equiv\big\langle x_t^{(i)}-f_\theta(\mathbf{X}_{t-1}), x_t^{(i)}-f_\theta(\mathbf{X}_{t-1}) \big\rangle$ is an inner product in $\R^M$.

For any $\mathbf{X}_{t-1}$, treating $f_\theta(\mathbf{X}_{t-1})\in \R^M$ as a vector, let's calculate its value such that the integral $F(f_\theta(\mathbf{X}_{t-1})):=\int dx^{(i)}_t P(x_t^{(i)}|\mathbf{X}_{t-1})\left(x_t^{(i)}-f_\theta(\mathbf{X}_{t-1})\right)^2$ attains its minimum.

Let
\begin{equation*}
\begin{aligned}
0&=\frac{\partial}{\partial f_{\theta}(\mathbf{X}_{t-1})}F(f_\theta(\mathbf{X}_{t-1}))\\
&=\frac{\partial}{\partial f_{\theta}(\mathbf{X}_{t-1})}\int dx^{(i)}_t P(x_t^{(i)}|\mathbf{X}_{t-1})\left(x_t^{(i)}-f_\theta(\mathbf{X}_{t-1})\right)^2 \\
&= -2\int dx^{(i)}_t P(x_t^{(i)}|\mathbf{X}_{t-1})\left(x_t^{(i)}-f_\theta(\mathbf{X}_{t-1})\right)
\end{aligned}
\end{equation*}

we have
\begin{equation*}
\begin{aligned}
\int dx^{(i)}_t P(x_t^{(i)}|\mathbf{X}_{t-1})x_t^{(i)}&=\int dx^{(i)}_t P(x_t^{(i)}|\mathbf{X}_{t-1})f_\theta(\mathbf{X}_{t-1})\\
&=f_\theta(\mathbf{X}_{t-1})\int dx^{(i)}_t P(x_t^{(i)}|\mathbf{X}_{t-1})\\
&=f_\theta(\mathbf{X}_{t-1})
\end{aligned}
\end{equation*}

Therefore, for any $\mathbf{X}_{t-1}$, $f_\theta(\mathbf{X}_{t-1})=\int dx^{(i)}_t P(x_t^{(i)}|\mathbf{X}_{t-1})x_t^{(i)}$ is the only stationary point for $F(f_\theta(\mathbf{X}_{t-1}))$.

Taking the second derivative, we have
\begin{equation*}
\begin{aligned}
&\frac{\partial^2}{(\partial f_{\theta}(\mathbf{X}_{t-1}))^2}F(f_\theta(\mathbf{X}_{t-1})) = 2\int dx^{(i)}_t P(x_t^{(i)}|\mathbf{X}_{t-1})\mathbf{I} =2\mathbf{I}
\end{aligned}
\end{equation*}
where $\mathbf{I}$ is an $M\times M$ identity matrix, which is always positive definite.

Therefore, for any $\mathbf{X}_{t-1}$, $f_\theta(\mathbf{X}_{t-1})=\int dx^{(i)}_t P(x_t^{(i)}|\mathbf{X}_{t-1})x_t^{(i)}$ is the only global minimum of $F(f_\theta(\mathbf{X}_{t-1}))$ w.r.t. $f_\theta(\mathbf{X}_{t-1})$.

Since 
\begin{equation*}
R^{\text{MSE}}_{\mathbf{X},x^{(i)}}[f_\theta]=\int d\mathbf{X}_{t-1} P(\mathbf{X}_{t-1})F(f_\theta(\mathbf{X}_{t-1}))
\end{equation*}

The minimum of the risk $R_{\mathbf{X},x^{(i)}}[f_\theta]$ is attained iff $F(f_\theta(\mathbf{X}_{t-1}))$ attains minimum at every $\mathbf{X}_{t-1}$, i.e.,
\begin{equation*}
    f_\theta(\mathbf{X}_{t-1})=\int dx^{(i)}_t P(x_t^{(i)}|\mathbf{X}_{t-1})x_t^{(i)}
\end{equation*}
is true for any $\mathbf{X}_{t-1}$. Given Assumption \ref{assump:f_theta}, we know that $f_\theta\in \mathcal{F}$ has enough capacity such that it can approximate any $\int dx^{(i)}_t P(x_t^{(i)}|\mathbf{X}_{t-1})x_t^{(i)}$. Therefore,
\begin{equation*}
\text{argmin}_{f_\theta}R^{\text{MSE}}_{\mathbf{X},x^{(i)}}[f_{\theta}]=\int dx^{(i)}_t P(x_t^{(i)}|\mathbf{X}_{t-1})x_t^{(i)}
\end{equation*}

and
\begin{equation*}
\text{min}_{f_\theta}R^{\text{MSE}}_{\mathbf{X},x^{(i)}}[f_{\theta}]=\mathbb{E}_{\mathbf{X}_{t-1},x_t^{(i)}}\left[\left(x_t^{(i)}-\int dx^{(i)}_t P(x_t^{(i)}|\mathbf{X}_{t-1})x_t^{(i)}\right)^2\right]
\end{equation*}
\end{proof}

\subsubsection{Proof of the three properties of $W_{ji}$}

The three properties are 
\begin{enumerate}[label={(\arabic*)}]
\item If $x^{(j)}\independent x^{(i)}$, then $W_{ji}=0$.
\item $W_{ji}$ is invariant to affine transformation of each individual $X_{t-1}^{(k)},k=1,2,...N$.
\item $W_{ji}$ is invariant to reparameterization of $\theta$ in $f_\theta$ (the mapping remains the same).
\end{enumerate}

\begin{proof}
\textbf{(1)} If $x^{(j)}\independent x^{(i)}$, then $X^{(j)}_{t-1}\independent x^{(i)}_t$. Since $\tilde{X}^{(j)(\eta_j)}_{t-1}=X^{(j)}_{t-1}+\eta_j\cdot\epsilon_j$ where $\epsilon_j\sim N(\mathbf{0},\mathbf{I})$, we have
$\tilde{X}^{(j)(\eta_j)}_{t-1}\independent x^{(i)}_t$. Recall Eq. (\ref{eq:learnable_risk}):
$$R_{\mathbf{X},x^{(i)}}[f_\theta,\boldsymbol{\eta}]=\mathbb{E}_{\mathbf{X}_{t-1},x_t^{(i)},\boldsymbol{\epsilon}}\left[\left(x_t^{(i)}-f_\theta(\tilde{\mathbf{X}}^{(\boldsymbol{\eta})}_{t-1})\right)^2\right]+\lambda\cdot \sum_{k=1}^{N}I(\tilde{X}^{(k)(\eta_k)}_{t-1};X^{(k)}_{t-1})$$
let $f_{\theta^*_{\boldsymbol{\eta}}}=\text{argmin}_{f_{\theta}}R_{\mathbf{X},x^{(i)}}[f_\theta,\boldsymbol{\eta}]$ given a certain $\boldsymbol{\eta}$, we have

\begin{equation*}
\begin{aligned}
f_{\theta^*_{\boldsymbol{\eta}}}(\tilde{\X}_{t-1}^{(\boldsymbol{\eta})})&=\text{argmin}_{f_{\theta}}R_{\mathbf{X},x^{(i)}}[f_\theta,\boldsymbol{\eta}]\\
&=\text{argmin}_{f_{\theta}}\mathbb{E}_{\tilde{\X}_{t-1}^{(\boldsymbol{\eta})},x_t^{(i)}}\left[\left(x_t^{(i)}-f_\theta(\tilde{\mathbf{X}}^{(\boldsymbol{\eta})}_{t-1})\right)^2\right]\\
&=\int dx^{(i)}_t P(x_t^{(i)}|\tilde{\X}_{t-1}^{(\boldsymbol{\eta})})x_t^{(i)}
\end{aligned}
\end{equation*}

where the second equality is due to that the mutual information term in $R_{\mathbf{X},x^{(i)}}[f_\theta,\boldsymbol{\eta}]$ does not depend on $f_\theta$, and the last equality is due to Lemma \ref{lemma:argmin_risk}. Let $\tilde{\X}_{t-1}^{(\boldsymbol{\eta})(\hat{j})}=\tilde{\X}_{t-1}^{(\boldsymbol{\eta})}\texttt{\textbackslash}\tilde{X}^{(j)(\eta_j)}_{t-1}$, since $\tilde{X}_{t-1}^{(j)(\eta_j)}\independent x^{(i)}_t$, we have

\begin{equation*}
\begin{aligned}
P(x_t^{(i)}|\tilde{\X}_{t-1}^{(\boldsymbol{\eta})})&\equiv P(x_t^{(i)}|\tilde{\X}_{t-1}^{(\boldsymbol{\eta})(\hat{j})},\tilde{X}_{t-1}^{(j)(\eta_j)})\\
&=P(x_t^{(i)}|\tilde{\X}_{t-1}^{(\boldsymbol{\eta})(\hat{j})})
\end{aligned}
\end{equation*}

Therefore,

\begin{equation*}
\begin{aligned}
f_{\theta^*_{\boldsymbol{\eta}}}(\tilde{\mathbf{X}}_{t-1}^{(\boldsymbol{\eta})})=\int dx^{(i)}_t P(x_t^{(i)}|\tilde{\X}_{t-1}^{(\boldsymbol{\eta})(\hat{j})})x_t^{(i)}
\end{aligned}
\end{equation*}

which \emph{does not} depend on $\tilde{X}_{t-1}^{(j)(\eta_j)}$. Finally, we have

\begin{equation*}
\begin{aligned}
&\text{min}_{(f_{\theta},\boldsymbol{\eta})}R_{\mathbf{X},x^{(i)}}[f_\theta,\boldsymbol{\eta}]\\
=&\text{min}_{\boldsymbol{\eta}}\left[R_{\mathbf{X},x^{(i)}}[f_{\theta^*_{\boldsymbol{\eta}}},\boldsymbol{\eta}]\right]\\
=&\text{min}_{\boldsymbol{\eta}}\left[\mathbb{E}_{\mathbf{X}_{t-1},x_t^{(i)},\boldsymbol{\epsilon}}\left[\left(x_t^{(i)}-f_{\theta_{\boldsymbol{\eta}}^*}(\tilde{\mathbf{X}}^{(\boldsymbol{\eta})}_{t-1})\right)^2\right]+\lambda\cdot \sum_{k=1}^{N}I(\tilde{X}^{(k)(\eta_k)}_{t-1};X^{(k)}_{t-1})\right]\\
=&\text{min}_{\boldsymbol{\eta}}\left[\left(\mathbb{E}_{\mathbf{X}_{t-1},x_t^{(i)},\boldsymbol{\epsilon}}\left[\left(x_t^{(i)}-f_{\theta_{\boldsymbol{\eta}}^*}(\tilde{\mathbf{X}}^{(\boldsymbol{\eta})(\hat{j})}_{t-1})\right)^2\right]+\lambda\cdot \sum_{k\neq j}I(\tilde{X}^{(k)(\eta_k)}_{t-1};X^{(k)}_{t-1})\right)+I(\tilde{X}^{(j)(\eta_j)}_{t-1};X^{(j)}_{t-1})\right]\\
\end{aligned}
\end{equation*}

For the last equality, the elements in the parenthesis $(\cdot)$ does not depend on $\tilde{X}_{t-1}^{(j)(\eta_j)}$, and only the $I(\tilde{X}^{(j)(\eta_j)}_{t-1};X^{(j)}_{t-1})$ term depends on $\tilde{X}_{t-1}^{(j)(\eta_j)}$. Therefore, at the minimization of the whole objective $R_{\mathbf{X},x^{(i)}}[f_\theta,\boldsymbol{\eta}]$, we have $I(\tilde{X}^{(j)(\eta_j)}_{t-1};X^{(j)}_{t-1})$ attains its minimum of 0, at which $\eta_j^*\to\infty$. By the definition of $W_{ji}$, we have $W_{ji}=I(\tilde{X}^{(j)(\eta_j^*)}_{t-1};X^{(j)}_{t-1})=0$. Proof completes.

In essence, the proof states that if $x^{(j)}\independent{}x^{(i)}$, then at the minimization of the whole objective, the MSE term does not depend on $X^{(j)}_{t-1}$ or $\tilde{X}^{(j)(\eta_j)}_{t-1}$, and the mutual information term $I(\tilde{X}^{(j)(\eta_j^*)}_{t-1};X^{(j)}_{t-1})$ w.r.t. time series $j$ can be independently minimized and approach 0.

\textbf{(2)} Suppose that we replace $X_{t-1}^{(j)}$ by $X_{t-1}^{'(j)}=a\cdot X_{t-1}^{(j)}+b$ where $a,b\in\mathbb{R}$. Let $\eta_j'=a\cdot \eta_j$. We have $\tilde{X}_{t-1}^{'(j)(\eta'_j)}=X_{t-1}^{'(j)}+\eta_j'\cdot\epsilon_j=a(X_{t-1}^{(j)}+\eta_j\cdot\epsilon_j) + b=a\cdot \tilde{X}_{t-1}^{(j)(\eta_j)}+b$, and therefore $I\left(\tilde{X}_{t-1}^{'(j)(\eta'_j)};X_{t-1}^{'(j)}\right)=I\left(a\cdot \tilde{X}_{t-1}^{(j)(\eta_j)}+b;a\cdot X_{t-1}^{(j)}+b\right)=I\left(\tilde{X}_{t-1}^{(j)(\eta_j)};X_{t-1}^{(j)}\right)$, where the last equality is due to that mutual information is invariant to invertible transformations. Furthermore, due to Assumption \ref{assump:f_theta}, we can find another $f_{\theta'}$ which undoes this affine transformation on $\tilde{X}_{t-1}^{(j)(\eta_j)}$, so the MSE term can be kept the same. Therefore, we have a one-to-one mapping between the original $X_{t-1}^{(j)},\eta_j,f_\theta$ and the new $X_{t-1}^{'(j)},\eta'_j,f_{\theta'}$ such that value of the MSE term and the mutual information term remain unchanged. Thus at the minimization of the objective, $W_{ji}$ remains the same.

\textbf{(3)} This is trivial to prove. We see that in $R_{\mathbf{X},x^{(i)}}[f_\theta,\boldsymbol{\eta}]$, the MSE term remains the same if the mapping $f$ remains the same, regardless of how we parameterize $f$ in terms of parameter $\theta$. The second term does not depend on $f_\theta$. Therefore, at the minimization of $R_{\mathbf{X},x^{(i)}}[f_\theta,\boldsymbol{\eta}]$, the $W_{ji}=I(\tilde{X}^{(j)(\eta_j^*)}_{t-1};X^{(j)}_{t-1})$ is invariant to the reparameterization of the same $f$ in terms of parameter $\theta$. As a direct corollary, $W_{ji}$ is insensitive to the network architecture, as long as the capacity is enough (provided with sufficient number of examples). This is confirmed in Table S\ref{table:synthetic_capacity} in Appendix \ref{app:capacity}. 

Note that L1 and group L1 regularization do not have this property, since they explicitly regularize on the parameter $\theta$.
\end{proof}

\subsubsection{Analysis of the minimum predictive information-regularized risk}

After proving the three properties of $W_{ji}$, now we analyze why the minimum predictive information-regularized risk is likely to select the variables that directly cause $x_t^{(i)}$, under some additional assumptions. We first state the additional assumption needed to perform the analysis, then we restate the definitions of direct causality to make our statements more rigorous. We then prove two lemmas in Appendix \ref{app:lemma_2}, and finally perform the analysis in Appendix \ref{app:analysis_risk}.

\begin{assumption}
\label{assump:additonal_assumption}
Assume that causal sufficiency \cite{peters2017elements} is satisfied, i.e. the observed time series $x^{(i)},i=1,2,...N$ are all the variables that take part in the dynamics (no hidden confounding variables). Also assume that in the response function Eq. (\ref{eq:response_function}), the noise variable $u_i,i=1,2,...N$ are effective variables, so each $h_i$ is not a deterministic mapping. Assume that by saying ``causality", we mean ``causality in mean".
\end{assumption}

To make our statement of causality more rigorous, here we restate the definition of direct (structural) causality \cite{white2011linking} using our notations of the system Eq.~(\ref{eq:response_function}). This definition is a natural extension to Pearl causality \cite{pearl2009causality} in canonical settable systems \cite{white2009settable,white2011linking}, which formalizes time series in its full generality.

\textbf{Direct (structural) causality} \cite{white2011linking} \textit{
We say $X^{(j)}_{t-1}, j\neq i$ does not directly (structurally) cause $x^{(i)}_{t}$, if for all possible values of $\mathbf{X}_{t-1}^{(\hat{j})}$ and $u_l$, $l\in{1,2,...N}$, the function $X^{(j)}_{t-1}\to h_i(\mathbf{X}_{t-1},u_i)$ is constant in $X^{(j)}_{t-1}$. Otherwise, we say $X^{(j)}_{t-1}$ directly (structurally) causes $x^{(i)}_{t}$.
}

The relationship between direct causality and Granger causality in Section \ref{sec:problem_definition} is the following Lemma, which states that for our system, Granger causality is a sufficient condition for direct (structural) causality.

\begin{lemma}
\label{thm:theorem_0}
Assuming causal sufficiency, for system Eq.~\ref{eq:response_function}, for any $i,j\in\{1,2,...N\},i\neq j$, if $X^{(j)}_{t-1}$ Granger-causes $x^{(i)}_{t}$, then $X^{(j)}_{t-1}$ directly structurally causes $x^{(i)}_t$.
\end{lemma}

\begin{proof}
We base the proof on the Theorem 5.6 in \cite{white2011linking}. Firstly, by definition, the system Eq. (\ref{eq:response_function}) belongs to the canonical settable system (Def. 3.3 in \cite{white2011linking}), on which their Theorem 5.6 is based. To prove that in our system Granger causality can deduce direct structural causality, we only have to prove that the assumption A.1 and assumption A.2 in \cite{white2011linking} are satisfied by our system. If we identify our $x_t^{(i)}$ with their $Y_{1,t}$, our $\mathbf{X}_{t-1}$ with their $\mathbf{Y}_{t-1}$, our $x_t^{(j)}$ with their $Y_{2,t}$, our $u_{i,t}$ (our $u_i$ at time $t$) with their $U_{1,t}$, our $u_{j,t}$ with their $U_{2,t}$, their $\mathbf{Z}_t=\emptyset$, $\mathbf{W}_t=\emptyset$, then our system Eq. (\ref{eq:response_function}) satisfies their Assumption A.1. Additionally, by definition, our $u_i\in R^M, i=1,2,...N$ are random variables that are mutually independent, and also independent of any $X^{(i)}_{t-1}, x^{(i)}_t$, $i\in\{1,2,...N\}$. Therefore, our system satisfies their strict exogeneity $(\mathbf{Y}_{t-1},\mathbf{Z}_t)\independent U_{1,t}$ (in our representation  $(\mathbf{X}_{t-1},\emptyset)\independent u_{i,t}$), which is a sufficient condition for Assumption A.2. Therefore, both their Assumption A.1 and Assumption A.2 are satisfied by our system Eq. (\ref{eq:response_function}). Applying their Theorem 5.6, we prove Lemma \ref{thm:theorem_0}.

\end{proof}

Therefore, for our system Eq.~(\ref{eq:response_function}), applying the results by \cite{white2011linking}, we have that Granger causality is a sufficient condition for direct structural causality. The reason that here Granger causality can deduce direct structural causality is in part due to the fact that for system Eq.~(\ref{eq:response_function}), conditional exogeneity \cite{white2011linking} is automatically satisfied. 

Note that the reverse of the statement is not true, i.e. a failed Granger causality test does not necessarily imply that there is no direct structural causality (White \& Lu \cite{white2010granger} give several examples, and also note that these instances are exceptional). 

After stating Assumption \ref{assump:additonal_assumption} and clarifying the definition of causalities, now we prove two lemmas, which are important for the analysis of our objective.

\paragraph{Minimum MSE with different variables}
\label{app:lemma_2}

\begin{lemma}
\label{lemma:d_separation_mmse}
Suppose that Assumption 1 and \ref{assump:additonal_assumption} holds, and $X_{t-1}^{(U)}$,$X_{t-1}^{(V)}$, $X_{t-1}^{(W)}\subset \X_{t-1}$ are mutually exclusive sets of variables satisfying
\begin{equation*}
X_{t-1}^{(W)}\independent{} x_t^{(i)}|X_{t-1}^{(U)},X_{t-1}^{(V)},\ \ \ \ \ \ \ 
X_{t-1}^{(V)}\not\!\perp\!\!\!\perp x_t^{(i)}|X_{t-1}^{(U)},X_{t-1}^{(W)}
\end{equation*}
Then
\begin{equation*}
\text{min}_{f_\theta}\mathbb{E}_{X_{t-1}^{(U)},X_{t-1}^{(V)},x_t^{(i)}}\left[\left(x_t^{(i)}-f_\theta(X_{t-1}^{(U)},X_{t-1}^{(V)})\right)^2\right]<\text{min}_{f_\theta}\mathbb{E}_{X_{t-1}^{(U)},X_{t-1}^{(V)},x_t^{(i)}}\left[\left(x_t^{(i)}-f_\theta(X_{t-1}^{(U)},X_{t-1}^{(W)})\right)^2\right]
\end{equation*}
\end{lemma}

Fig. S\ref{fig:d_seperation_diagram} below shows the relations between the variables, where the dashed arrows denote the potential existence of causal relations between variables. We see that \emph{conditioned} on $(X_{t-1}^{(U)},X_{t-1}^{(V)})$, we have $x_t^{(i)}$ and $X_{t-1}^{(W)}$ are independent, while \emph{conditioned} on $(X_{t-1}^{(U)},X_{t-1}^{(W)})$, we have $x_t^{(i)}$ and $X_{t-1}^{(V)}$ are \emph{not} independent. Lemma \ref{lemma:d_separation_mmse} states that under the above scenario and under Assumptions \ref{assump:f_theta} and \ref{assump:additonal_assumption}, using $X_{t-1}^{(U)}$ and $X_{t-1}^{(V)}$ to predict $x_t^{(i)}$ can achieve a lower MSE than using $X_{t-1}^{(U)}$ and $X_{t-1}^{(W)}$ to predict $x_t^{(i)}$.

\begin{suppfigure}[h!]
    \centering
    \includegraphics[scale=0.45]{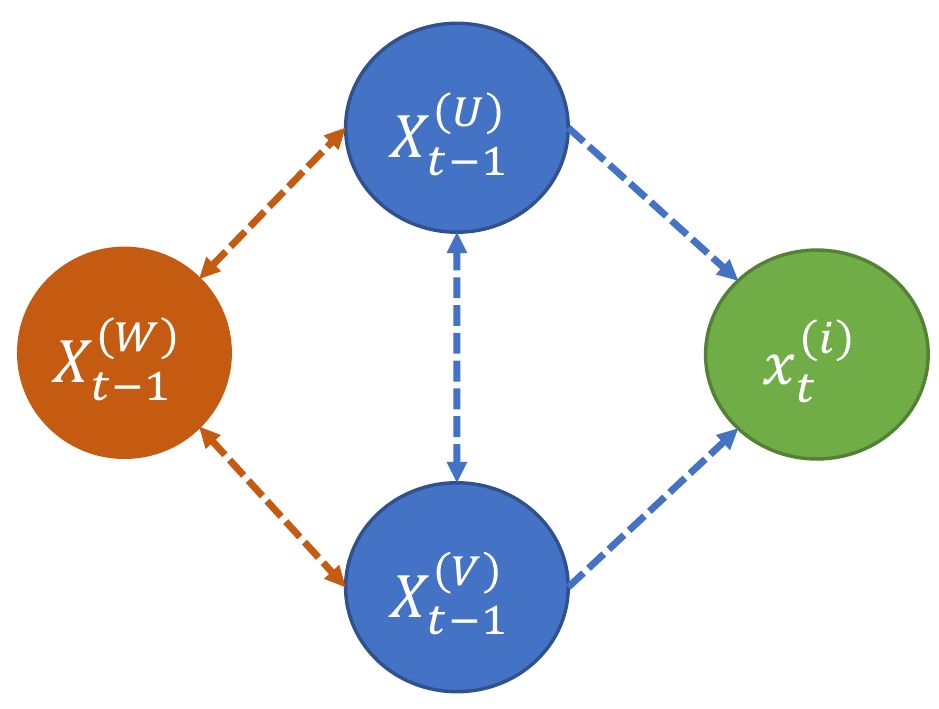}
    \caption{Diagram of variables for Lemma \ref{lemma:d_separation_mmse}. The dashed arrows denote the possible existence of causal relations between variables.}
    \label{fig:d_seperation_diagram}%
\end{suppfigure}

\begin{proof}
Since Assumption \ref{assump:f_theta} holds, according to Lemma \ref{lemma:argmin_risk}, Lemma \ref{lemma:d_separation_mmse} is equivalent to

\begin{equation*}
\begin{aligned}
\mathbb{E}_{X^{(U)}_{t-1},X^{(V)}_{t-1},x_t^{(i)}}&\left[\left(x_t^{(i)}-\int dx^{(i)}_t P(x_t^{(i)}|X^{(U)}_{t-1},X^{(V)}_{t-1})x_t^{(i)}\right)^2\right]\\
&<\mathbb{E}_{X^{(U)}_{t-1},X^{(W)}_{t-1},x_t^{(i)}}\left[\left(x_t^{(i)}-\int dx^{(i)}_t P(x_t^{(i)}|X^{(U)}_{t-1},X^{(W)}_{t-1})x_t^{(i)}\right)^2\right]
\end{aligned}
\end{equation*}
We have
\begin{equation*}
\begin{aligned}
&\mathbb{E}_{X^{(U)}_{t-1},X^{(W)}_{t-1},x_t^{(i)}}\left[\left(x_t^{(i)}-\int dx^{(i)}_t P(x_t^{(i)}|X^{(U)}_{t-1},X^{(W)}_{t-1})x_t^{(i)}\right)^2\right]\\
&=\int dX^{(U)}_{t-1}dX^{(W)}_{t-1}dx_t^{(i)} P(X^{(U)}_{t-1},X^{(W)}_{t-1},x_t^{(i)})\left(x_t^{(i)}-\int dx^{(i)}_t P(x_t^{(i)}|X^{(U)}_{t-1},X^{(W)}_{t-1})x_t^{(i)}\right)^2\\
&=\int dX^{(U)}_{t-1}dX^{(V)}_{t-1}dX^{(W)}_{t-1}dx_t^{(i)} P(X^{(U)}_{t-1},X^{(V)}_{t-1},X^{(W)}_{t-1},x_t^{(i)})\left(x_t^{(i)}-\int dx^{(i)}_t P(x_t^{(i)}|X^{(U)}_{t-1},X^{(W)}_{t-1})x_t^{(i)}\right)^2\\
&=\int dX^{(U)}_{t-1}dX^{(V)}_{t-1}dX^{(W)}_{t-1} P(X^{(U)}_{t-1},X^{(V)}_{t-1})P(X^{(W)}_{t-1}|X^{(U)}_{t-1},X^{(V)}_{t-1})\cdot\\
&\ \ \ \ \ \ \ \ \ \ \ \ \ \ \ \ \ \ \ \ \ \ \ \ \int dx_t^{(i)} P(x_t^{(i)}|X^{(U)}_{t-1},X^{(V)}_{t-1})\left(x_t^{(i)}-\int dx^{(i)}_t P(x_t^{(i)}|X^{(U)}_{t-1},X^{(W)}_{t-1})x_t^{(i)}\right)^2\\
&>\int dX^{(U)}_{t-1}dX^{(V)}_{t-1}dX^{(W)}_{t-1}P(X^{(U)}_{t-1},X^{(V)}_{t-1})P(X^{(W)}_{t-1}|X^{(U)}_{t-1},X^{(V)}_{t-1})\cdot\\
&\ \ \ \ \ \ \ \ \ \ \ \ \ \ \ \ \ \ \ \ \ \ \ \ \int dx_t^{(i)} P(x_t^{(i)}|X^{(U)}_{t-1},X^{(V)}_{t-1})\left(x_t^{(i)}-\int dx^{(i)}_t P(x_t^{(i)}|X^{(U)}_{t-1},X^{(V)}_{t-1})x_t^{(i)}\right)^2\\
&=\int dX^{(U)}_{t-1}dX^{(V)}_{t-1}P(X^{(U)}_{t-1},X^{(V)}_{t-1})\int dx_t^{(i)} P(x_t^{(i)}|X^{(U)}_{t-1},X^{(V)}_{t-1})\left(x_t^{(i)}-\int dx^{(i)}_t P(x_t^{(i)}|X^{(U)}_{t-1},X^{(V)}_{t-1})x_t^{(i)}\right)^2\\
&=\mathbb{E}_{X^{(U)}_{t-1},X^{(V)}_{t-1},x_t^{(i)}}\left[\left(x_t^{(i)}-\int dx^{(i)}_t P(x_t^{(i)}|X^{(U)}_{t-1},X^{(V)}_{t-1})x_t^{(i)}\right)^2\right]
\end{aligned}
\end{equation*}
The third equality (the one before the inequality) is due to that $X_{t-1}^{(W)}\independent{} x_t^{(i)}|X_{t-1}^{(U)},X_{t-1}^{(V)}$, leading to $P(X^{(U)}_{t-1},X^{(V)}_{t-1},X^{(W)}_{t-1},x^{(i)}_{t})=P(X^{(U)}_{t-1},X^{(V)}_{t-1})P(X^{(W)}_{t-1}|X^{(U)}_{t-1},X^{(V)}_{t-1})P(x^{(i)}_{t}|X^{(U)}_{t-1},X^{(V)}_{t-1})$. The inequality step first uses the Assumption \ref{assump:additonal_assumption} that the noise variables $u_i$ are effective arguments of the response functions $h_i$, and that each $h_i$ is ``causality in mean". Therefore, $\int dx^{(i)}_t P(x_t^{(i)}|X^{(U)}_{t-1},X^{(V)}_{t-1})x_t^{(i)}\neq \int dx^{(i)}_t P(x_t^{(i)}|X^{(U)}_{t-1},X^{(W)}_{t-1})x_t^{(i)}$. Using Lemma \ref{lemma:argmin_risk}, we have $f_\theta(X^{(U)}_{t-1},X^{(V)}_{t-1})=\int dx^{(i)}_t P(x_t^{(i)}|X^{(U)}_{t-1},X^{(V)}_{t-1})x_t^{(i)}$ minimizes $$\int dx^{(i)}_t P(x_t^{(i)}|X^{(U)}_{t-1},X^{(V)}_{t-1})\left(x_t^{(i)}-f_\theta(X^{(U)}_{t-1},X^{(V)}_{t-1})\right)^2$$ hence the inequality.
\end{proof}

Using Lemma \ref{lemma:d_separation_mmse} recursively, we see that using variables that directly causes $x^{(i)}$ to predict $x^{(i)}$ can achieve the lowest MSE. Formalizing the above intuition, we have

\begin{lemma}
\label{thm:d_separation_mmse}
Suppose that Assumption 1 and \ref{assump:additonal_assumption} holds, and $X_{t-1}^{(D)}\subseteq\X_{t-1}$ are the set of variables that directly causes $x_t^{(i)}$. Then $\forall X_{t-1}^{(S)}\subseteq\X_{t-1}$ with $X_{t-1}^{(S)}\neq X_{t-1}^{(D)}$, we have

\begin{equation*}
\text{min}_{f_\theta}\mathbb{E}_{X_{t-1}^{(D)}}\left[\left(x_t^{(i)}-f_\theta(X_{t-1}^{(D)})\right)^2\right]<\text{min}_{f_\theta}\mathbb{E}_{X_{t-1}^{(S)}}\left[\left(x_t^{(i)}-f_\theta(X_{t-1}^{(S)})\right)^2\right]
\end{equation*}

Specifically, we have
\begin{equation*}
\text{min}_{f_\theta}\mathbb{E}_{X_{t-1}^{(D)}}\left[\left(x_t^{(i)}-f_\theta(X_{t-1}^{(D)})\right)^2\right]<\text{min}_{f_\theta}\mathbb{E}_{X_{t-1}^{(\hat{D})}}\left[\left(x_t^{(i)}-f_\theta(X_{t-1}^{(\hat{D})})\right)^2\right]
\end{equation*}

where $X_{t-1}^{(\hat{D})}=\X_{t-1}\textbackslash X_{t-1}^{(D)}$.
\end{lemma}

\begin{proof}
For any $X^{(S)}_{t-1}$, let $X^{(U)}_{t-1}=X^{(D)}_{t-1}\cap X^{(S)}_{t-1}$, $X^{(V)}_{t-1}=X^{(D)}_{t-1}\textbackslash X^{(S)}_{t-1}$, $X^{(W)}_{t-1}=X^{(S)}_{t-1}\textbackslash X^{(D)}_{t-1}$. Then $X^{(U)}_{t-1}$, $X^{(V)}_{t-1}$, $X^{(W)}_{t-1}$ are mutually exclusive, and $X^{(D)}_{t-1}=X^{(U)}_{t-1}\cup X^{(V)}_{t-1}$, $X^{(S)}_{t-1}=X^{(U)}_{t-1}\cup X^{(W)}_{t-1}$. Now we prove that $\forall X_{t-1}^{(S)}\subseteq\X_{t-1}$ with $X_{t-1}^{(S)}\neq X_{t-1}^{(D)}$, the corresponding $X^{(U)}_{t-1}$, $X^{(V)}_{t-1}$, $X^{(W)}_{t-1}$, $x_t^{(i)}$ satisfy the condition for Lemma \ref{lemma:d_separation_mmse}.
Since $X^{(D)}_{t-1}$ are the set of variables that directly causes $x_t^{(i)}$, there does not exist a $X^{(S)}_{t-1}$ such that the corresponding $X_{t-1}^{(V)}\independent{} x_t^{(i)}|X_{t-1}^{(U)},X_{t-1}^{(W)}$ (otherwise it violates the direct causality). Thus $X_{t-1}^{(V)}\not\!\perp\!\!\!\perp x_t^{(i)}|X_{t-1}^{(U)},X_{t-1}^{(W)}$. To prove $X_{t-1}^{(W)}\independent{} x_t^{(i)}|X_{t-1}^{(U)},X_{t-1}^{(V)}$, note that $X^{(W)}_{t-1}$ does not directly cause $x_t^{(i)}$, then  $X^{(W)}_{t-1}$ does not Granger-cause $x_t^{(i)}$, i.e. $P(x_t^{(i)}|X_{t-1}^{(U)},X_{t-1}^{(V)})=P(x_t^{(i)}|X_{t-1}^{(U)},X_{t-1}^{(V)},X_{t-1}^{(W)})$, which is equivalent to $X_{t-1}^{(W)}\independent{} x_t^{(i)}|X_{t-1}^{(U)},X_{t-1}^{(V)}$. The special case of $X_{t-1}^{(\hat{D})}$ follows directly that $X_{t-1}^{(\hat{D})}=\X_{t-1}\textbackslash X_{t-1}^{(D)}\neq X_{t-1}^{(D)}$ and letting $X_{t-1}^{(S)}=X_{t-1}^{(\hat{D})}$.
\end{proof}

\paragraph{Qualitative and quantitative behaviors of the mutual information-regularized risk}
\label{app:analysis_risk}

In this section, we analyze the qualitative and quantitative behaviors of the mutual information-regularized risk (Eq. \ref{eq:learnable_risk}), with varying noise levels $\eta_j$. For each variable $X_{t-1}^{(j)}\in\X_{t-1}$, $j=1,2,...N$, define $\rho_{j}=\text{tanh}\left(I(X_{t-1}^{(j)};\tilde{X}_{t-1}^{(j)(\eta_j)})\right)\in[0,1]$ as a ``rescaled" mutual information between $X_{t-1}^{(j)}$ and $\tilde{X}_{t-1}^{(j)(\eta_j)}$. When $\eta_j=\mathbf{0}$ so that $\tilde{X}_{t-1}^{(j)(\eta_j)}=X_{t-1}^{(j)}$, $\rho_j=1$, at which the input $X_{t-1}^{(j)}$ is fully preserved. When all elements of $\eta_j\to\infty$, $\rho_j=0$, at which $X_{t-1}^{(j)}$ is fully corrupted. Denoting $\boldsymbol{\rho}=(\rho_1,\rho_2,...\rho_N)$, we can then rewrite the mutual information-regularized risk (Eq. \ref{eq:learnable_risk}) as

\begin{equation}
\label{eq:learnable_risk_rho}
R_{\mathbf{X},x^{(i)}}[f_\theta,\boldsymbol{\rho}]=\text{MMSE}^{(i)}(\boldsymbol{\rho})+\lambda\cdot \sum_{j=1}^{N}\text{arctanh}(\rho_{j})
\end{equation}

where $\text{MMSE}^{(i)}(\boldsymbol{\rho})=\min_{\boldsymbol{\eta},f_\theta}\mathbb{E}_{\mathbf{X}_{t-1},x_t^{(i)},\boldsymbol{\epsilon}}\left[\left(x_t^{(i)}-f_\theta(\tilde{\mathbf{X}}^{(\boldsymbol{\eta})}_{t-1})\right)^2\right]$ subject to $$\rho_{j}=\text{tanh}\left(I(X_{t-1}^{(j)};\tilde{X}_{t-1}^{(j)(\eta_j)}\right),j=1,2,...N$$
Let $X_{t-1}^{(D)}\subseteq\X_{t-1}$ be the set of variables that directly causes $x_t^{(i)}$, and denote the corresponding set of $\rho_j$ as $\boldsymbol{\rho}^{(D)}$. Denote $X_{t-1}^{(\hat{D})}=\X_{t-1}\textbackslash X_{t-1}^{(D)}$ and the corresponding set of $\rho_j$ as $\boldsymbol{\rho}^{(\hat{D})}$. For any $i=1,2,...N$, it is easy to see that $\text{MMSE}^{(i)}(\boldsymbol{\rho})$ has the following properties:

\begin{enumerate}
\item $\text{MMSE}^{(i)}(\boldsymbol{\rho})$ attains maximum at $\boldsymbol{\rho}=\mathbf{0}$.
\item $\text{MMSE}^{(i)}(\boldsymbol{\rho})$ is monotonically decreasing w.r.t. each $\rho_j$.
\item $\text{MMSE}^{(i)}(\boldsymbol{\rho})\big|_{\boldsymbol{\rho}^{(D)}=\mathbf{1},\boldsymbol{\rho}^{(\hat{D})}=\mathbf{0}}<\text{MMSE}^{(i)}(\boldsymbol{\rho})\big|_{\boldsymbol{\rho}^{(D)}=\mathbf{0},\boldsymbol{\rho}^{(\hat{D})}=\mathbf{1}}$ (using Lemma \ref{thm:d_separation_mmse}).
\item $\text{MMSE}^{(i)}(\boldsymbol{\rho})$ attains minimum at $\boldsymbol{\rho}^{(D)}=\mathbf{1}$. $\text{MMSE}^{(i)}(\boldsymbol{\rho})\big|_{\boldsymbol{\rho}^{(D)}=\mathbf{1}}$ is constant w.r.t. $\boldsymbol{\rho}^{(\hat{D})}$.
\end{enumerate}

To get a better intuition of the landscape of $R_{\mathbf{X},x^{(i)}}[f_\theta,\boldsymbol{\rho}]$, let's investigate a simple example. Let the response function be:
\begin{equation}
\label{eq:response_function_app}
    \begin{cases}
      x^{(1)}_t:=h_1(u_1)=\sqrt{\Sigma_x}\cdot u_1 \\
      x^{(2)}_t:=h_2(x^{(1)}_{t-1},u_2)=x^{(1)}_{t-1}+\sqrt{\Omega_x}\cdot u_2 \\
      x^{(3)}_t:=h_3(x^{(2)}_{t-1},u_3)=x^{(2)}_{t-1}+\sqrt{\Omega_y}\cdot u_3
    \end{cases}
  \end{equation}
where $u_1,u_2,u_3$ are independent unit Gaussian variables, and $\X_{t-1}=(X^{(1)}_{t-1},X^{(2)}_{t-1},X^{(3)}_{t-1})=\left((x_{t-2}^{(1)},x_{t-1}^{(1)}),(x_{t-2}^{(2)},x_{t-1}^{(2)}),(x_{t-2}^{(3)},x_{t-1}^{(3)})\right)$. For $R_{\mathbf{X},x^{(3)}}[f_\theta,\boldsymbol{\rho}]=\text{MMSE}^{(3)}(\boldsymbol{\rho})+\lambda\cdot \sum_{j=1}^{3}\text{arctanh}(\rho_{j})$, since only $x_{t-2}^{(1)}$ and $x_{t-1}^{(2)}$ are d-connected to $x_t^{(3)}$, at the minimization of $R_{\mathbf{X},x^{(3)}}[f_\theta,\boldsymbol{\rho}]$, only $x_{t-2}^{(1)}$ and $x_{t-1}^{(2)}$ may have a finite $\eta_{j,l}^*$ (the other $\eta_{j,l}^*$ are all infinite). Therefore, setting the $\eta_{j,l}$ not corresponding to $x_{t-2}^{(1)}$ and $x_{t-1}^{(2)}$ as infinity, and let $\tilde{x}_{t-2}^{(1)}=x_{t-2}^{(1)}+\eta_x\cdot\epsilon_x$, $\tilde{x}_{t-1}^{(2)}=x_{t-1}^{(2)}+\eta_y\cdot\epsilon_y$, $\epsilon_x$ and $\epsilon_y$ being independent unit Gaussian variables. Let $f_\theta (x_{t-2}^{(1)},x_{t-1}^{(2)})=a\cdot x_{t-2}^{(1)} + b\cdot x_{t-1}^{(2)}$, then we can get an analytic expression for $R_{\mathbf{X},x^{(3)}}[f_\theta,\eta_x,\eta_y]$:

\begin{equation*}
\begin{aligned}
&R_{\mathbf{X},x^{(3)}}[f_\theta,\eta_x,\eta_y]\\
&=a^2\Sigma_x+(b-1)^2(\Sigma_x+\Omega_x)+a^2\eta_x^2+b^2\eta_y^2+2a(b-1)\Sigma_x+\Omega_y\\
&+\frac{\lambda}{2}\text{log}\left(1+\frac{\Sigma_x}{\eta_x^2}\right)+\frac{\lambda}{2}\text{log}\left(1+\frac{\Sigma_x+\Omega_x}{\eta_y^2}\right)
\end{aligned}
\end{equation*}

Minimizing $R_{\mathbf{X},x^{(3)}}[f_\theta,\eta_x,\eta_y]$ w.r.t. $a$ and $b$, we get

\begin{equation*}
\begin{aligned}
a^*&= \frac{\eta_y^2\Sigma_x}{\eta_x^2\eta_y^2+\eta_x^2\Sigma_x+\eta_y^2\Sigma_x+\eta_x^2\Omega_x+\Omega_x\Sigma_x}\\
b^*&= \frac{\eta_x^2(\Sigma_x+\Omega_x)+\Sigma_x\Omega_x}{\eta_x^2\eta_y^2+\eta_x^2\Sigma_x+\eta_y^2\Sigma_x+\eta_x^2\Omega_x+\Omega_x\Sigma_x}
\end{aligned}
\end{equation*}

Substituting into $R_{\mathbf{X},x^{(3)}}[f_\theta,\eta_x,\eta_y]$, we have
\begin{equation*}
\begin{aligned}
&R_{\mathbf{X},x^{(3)}}[\eta_x,\eta_y]\\
&=\min_{f_\theta}R_{\mathbf{X},x^{(3)}}[f_\theta,\eta_x,\eta_y]\\
&=\frac{\eta_y^2(\Sigma_x\Omega_x+\eta_x^2(\Sigma_x+\Omega_x))}{\eta_x^2\eta_y^2+\eta_x^2\Sigma_x+\eta_y^2\Sigma_x+\eta_x^2\Omega_x+\Omega_x\Sigma_x}+\frac{\lambda}{2}\text{log}\left(1+\frac{\Sigma_x}{\eta_x^2}\right)+\frac{\lambda}{2}\text{log}\left(1+\frac{\Sigma_x+\Omega_x}{\eta_y^2}\right)
\end{aligned}
\end{equation*}

Here we have neglected the constant $\Omega_y$. To obtain $R_{\mathbf{X},x^{(3)}}[\boldsymbol{\rho}]$, let $\rho_1=\text{tanh}\left(\frac{1}{2}\text{log}\left(1+\frac{\Sigma_x}{\eta_x^2}\right)\right)$, $\rho_2=\text{tanh}\left(\frac{1}{2}\text{log}\left(1+\frac{\Sigma_x+\Omega_x}{\eta_x^2}\right)\right)$, we have $\eta_x^2=\frac{1-\rho_1}{2\rho_1}\Sigma_x$, $\eta_y^2=\frac{1-\rho_2}{2\rho_2}(\Sigma_x+\Omega_x)$. Substituting, we have

\begin{equation*}
\begin{aligned}
&R_{\mathbf{X},x^{(3)}}[\boldsymbol{\rho}]=\text{MMSE}^{(3)}(\boldsymbol{\rho})+\lambda\cdot \sum_{j=1}^{2}\text{arctanh}(\rho_{j})\\
&=\frac{(\rho_2-1)(\Sigma_x+\Omega_x)((\rho_1-1)\Sigma_x-(\rho_1+1)\Omega_x)}{(1+\rho_1+\rho_2-3\rho_1\rho_2)\Sigma_x+(1+\rho_1)(1+\rho_2)\Omega_x}+\lambda\cdot \text{arctanh}(\rho_1)+\lambda\cdot\text{arctanh}(\rho_2)
\end{aligned}
\end{equation*}

Fig. S\ref{fig:risk} shows the landscape of $\text{MMSE}^{(3)}(\boldsymbol{\rho})$ and $R_{\mathbf{X},x^{(3)}}[\boldsymbol{\rho}]$, for $\Sigma_x=1,\Omega_x=2,\lambda=1$. We see that $\text{MMSE}^{(3)}(\boldsymbol{\rho})$ satisfies the above mentioned four properties. Particularly, $\text{MMSE}^{(3)}(\boldsymbol{\rho})\big|_{\rho_1=1,\rho_2=0}>\text{MMSE}^{(3)}(\boldsymbol{\rho})\big|_{\rho_1=0,\rho_2=1}$. After adding $\lambda\cdot \text{arctanh}(\rho_1)+\lambda\cdot\text{arctanh}(\rho_2)$, the $R_{\mathbf{X},x^{(3)}}[\boldsymbol{\rho}]$ has global minimum along $\rho_1=0$ largely due to this property. Therefore, for this particular example, when $R_{\mathbf{X},x^{(3)}}[\boldsymbol{\rho}]$ is minimized, $\rho_1=0$, i.e. $I(x_{t-2}^{(1)},\tilde{x}_{t-2}^{(1)(\eta_1^*)})=0$.

\begin{suppfigure}
    \centering
    \begin{subfigure}{.45\linewidth}
    \includegraphics[scale=0.45]{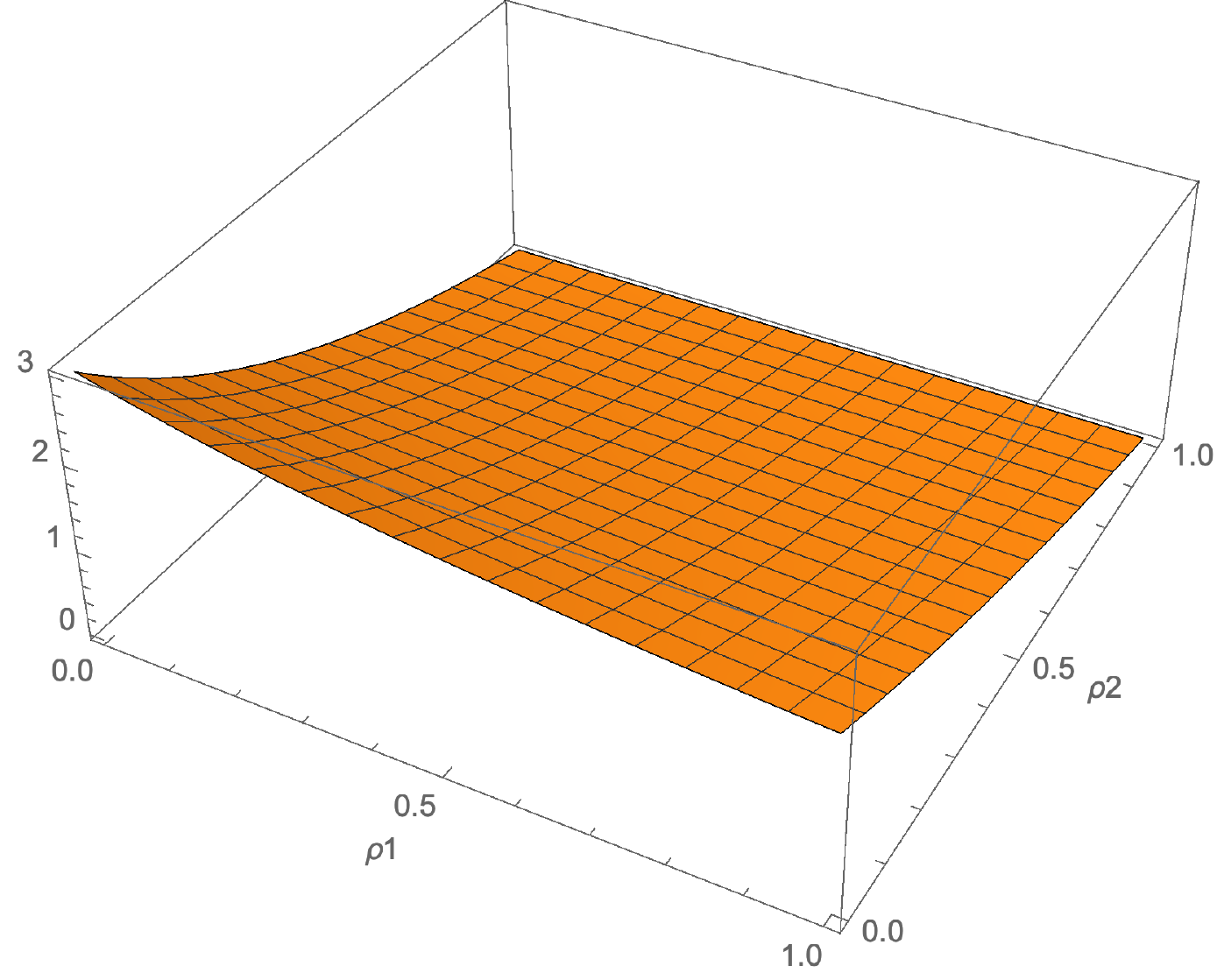}
    \caption{}
    \end{subfigure}
    \begin{subfigure}{.45\linewidth}
    \includegraphics[scale=0.45]{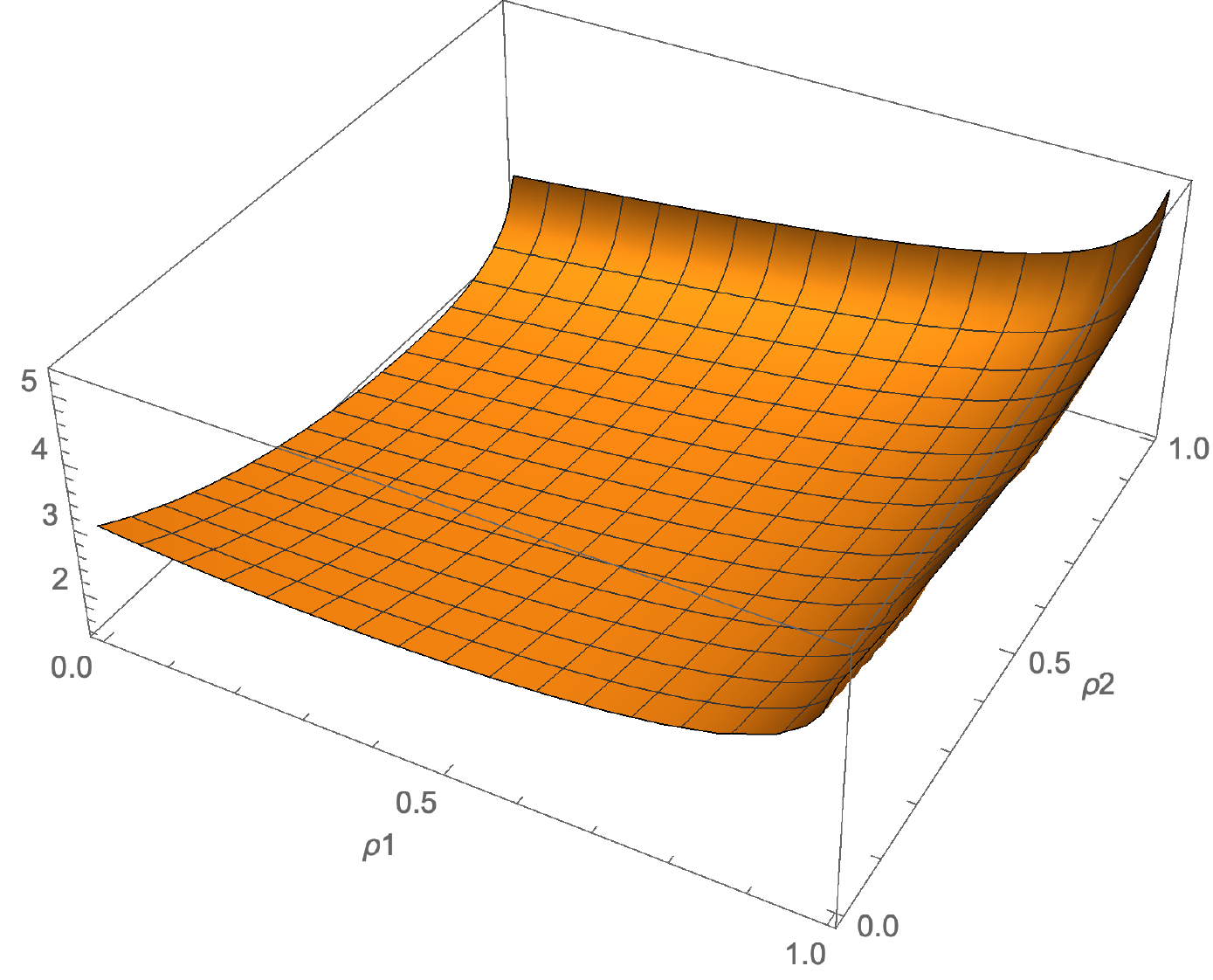}
    \caption{}
    \end{subfigure}
    \caption{(a) $\text{MMSE}^{(3)}(\boldsymbol{\rho})$ and (b) $R_{\mathbf{X},x^{(3)}}[\boldsymbol{\rho}]$ in section \ref{app:analysis_risk}, for $\Sigma_x=1,\Omega_x=2,\lambda=1$.}%
    \label{fig:risk}%
\end{suppfigure}

By varying the value of $\lambda$, we can tune the relative influence of the two terms $\text{MMSE}^{(3)}(\boldsymbol{\rho})$ and $\sum_{j=1}^{2}\text{arctanh}(\rho_{j})$. The landscape corresponding to $\lambda=0.01,0.5,2,10$ are plotted in Fig. S\ref{fig:risk_tune}. We see that when $\lambda\ll1$, the MMSE term dominates, and it is possible that the global minimum of $R_{\mathbf{X},x^{(3)}}[\boldsymbol{\rho}]$ is not at $\rho_1=0$. This is similar to the effect of a L1 regularization, where if the coefficient $\lambda$ for the L1 is vanishingly small, the L1 regularization will barely influence the loss landscape. When $\lambda$ is not vanishingly small, as in Fig. S\ref{fig:risk_tune} (b), we see that the global minimum of $R_{\mathbf{X},x^{(3)}}[\boldsymbol{\rho}]$ lies on $\rho_1=0$. When $\lambda\to+\infty$, the $\sum_{j=1}^{2}\text{arctanh}(\rho_{j})$ term dominates and the global minimum is at $\rho_1=0,\rho_2=0$.

In general, we expect $R_{\mathbf{X},x^{(i)}}[\boldsymbol{\rho}]$ behave qualitatively similar. When $\lambda\to+\infty$, the global minimum for $R_{\mathbf{X},x^{(i)}}[\boldsymbol{\rho}]$ is at $\boldsymbol{\rho}^*=\mathbf{0}$. As we ramp down $\lambda$, the dimension that has largest influence on MMSE will first host the global minimum with nonzero $\rho_j^*$, which is most likely the variable that directly causes $x_i^{(i)}$. When $\lambda$ is further ramping down, we expect that the variables that host the global minimum with nonzero $\rho_j$ will more likely be those that directly causes $x_i^{(i)}$, due to the landscape influenced by the four properties of MMSE. This can justify the mutual information-regularized risk as a good objective for causal discovery/variable selection. The experiments in the paper will empirically test the performance of the mutual information-regularized risk.

\begin{suppfigure}
    \centering
    \begin{subfigure}{.45\linewidth}
    \includegraphics[scale=0.45]{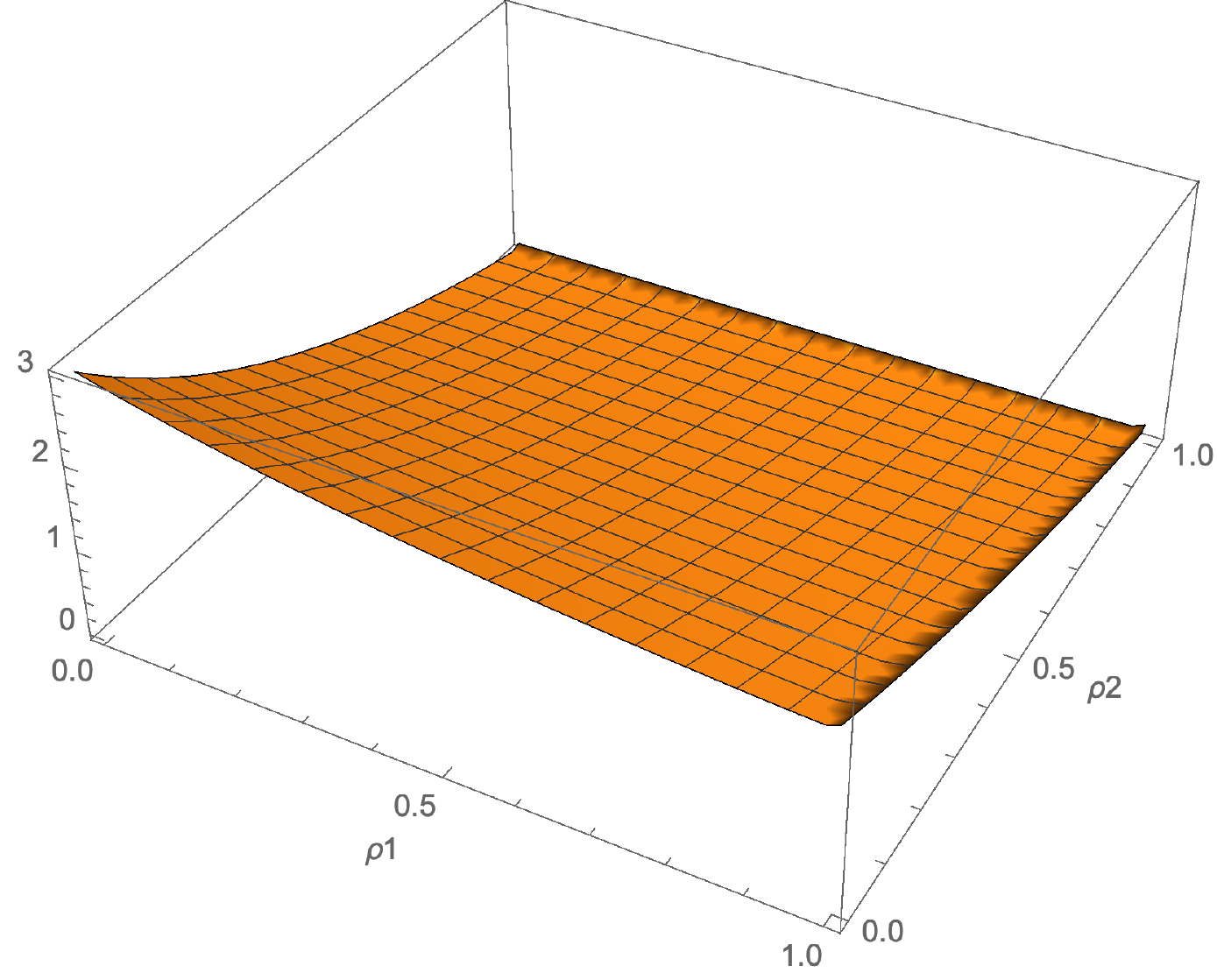}
    \caption{}
    \end{subfigure}
    \begin{subfigure}{.45\linewidth}
    \includegraphics[scale=0.45]{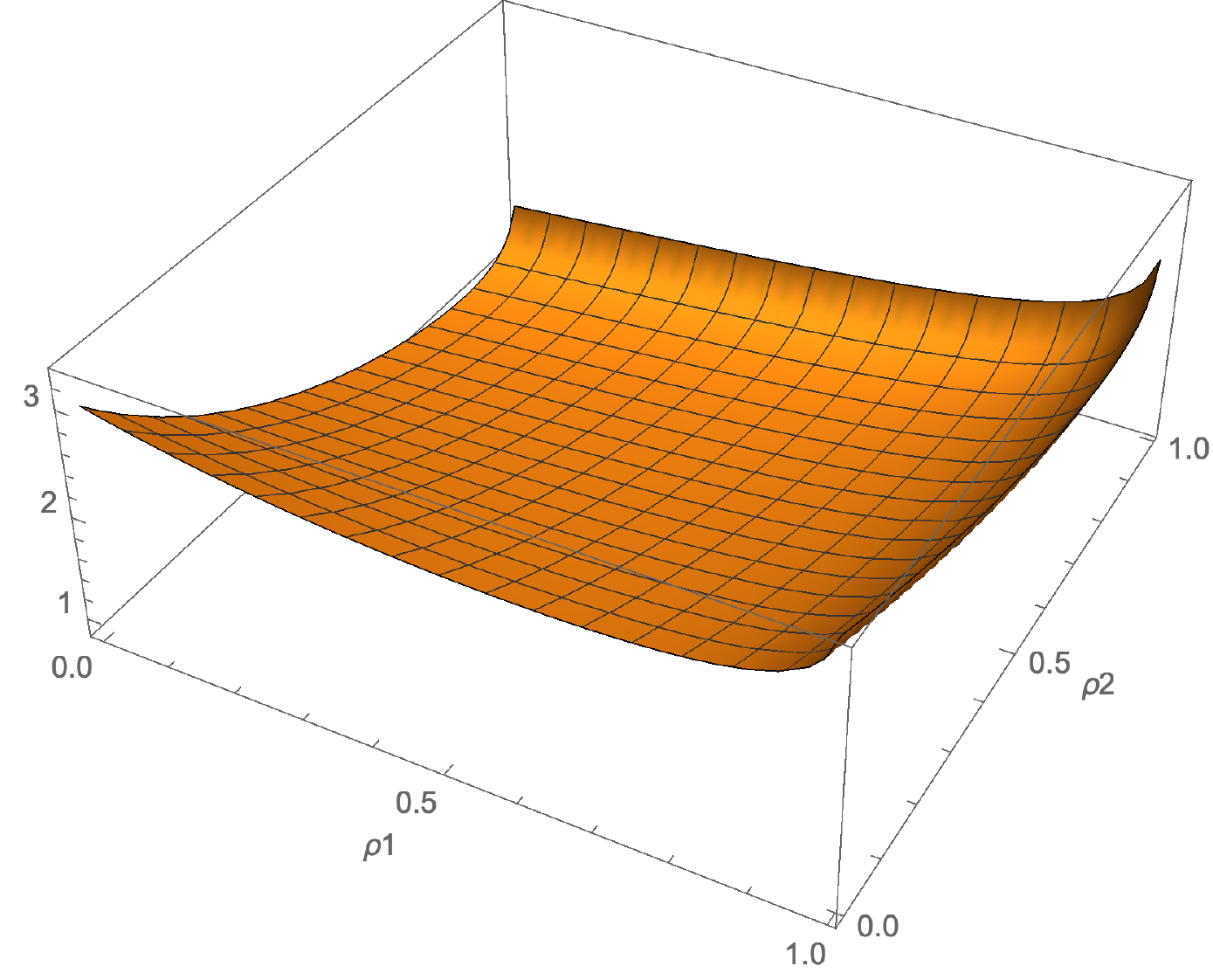}
    \caption{}
    \end{subfigure}
    \begin{subfigure}{.45\linewidth}
    \includegraphics[scale=0.45]{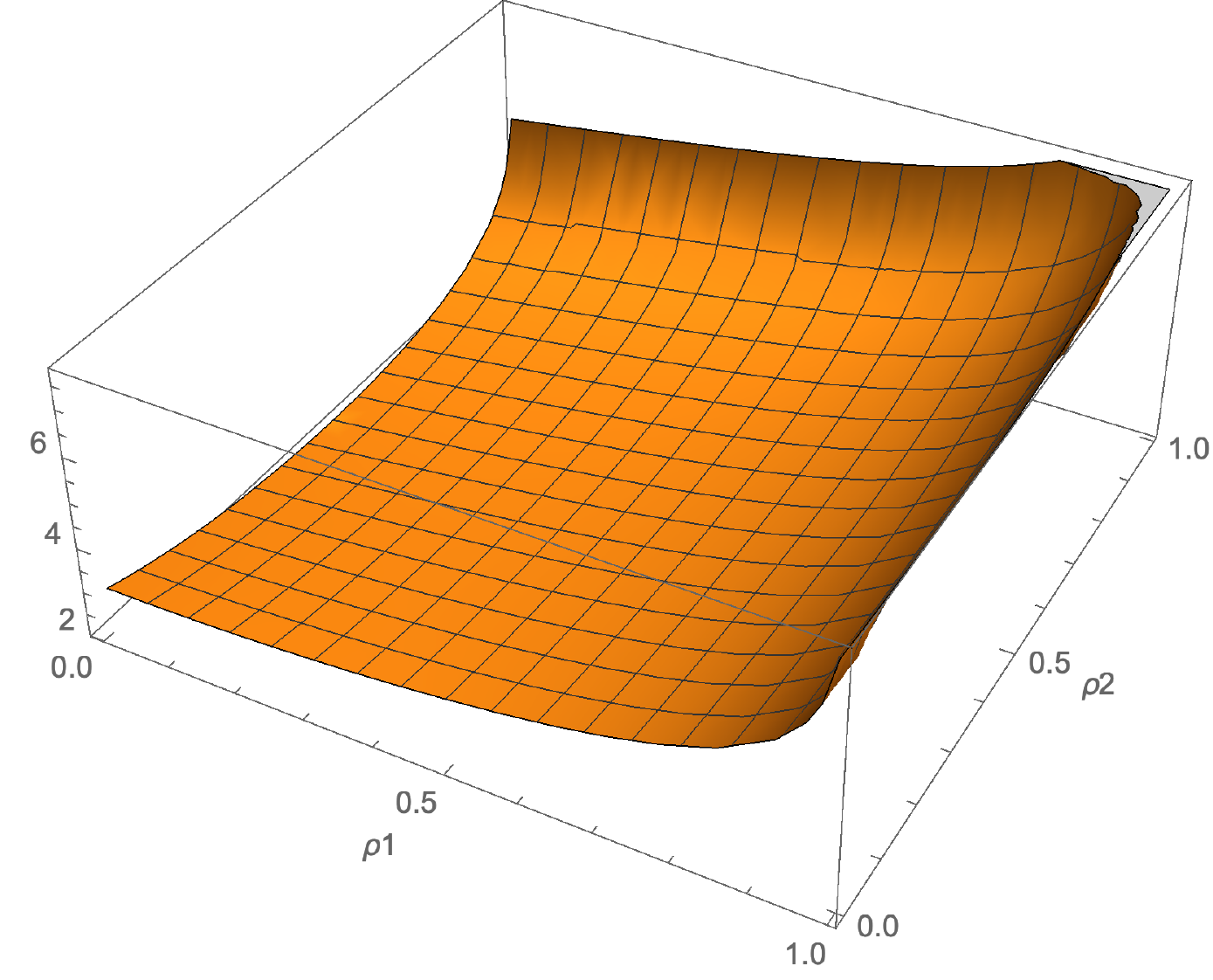}
    \caption{}
    \end{subfigure}
    \begin{subfigure}{.45\linewidth}
    \includegraphics[scale=0.45]{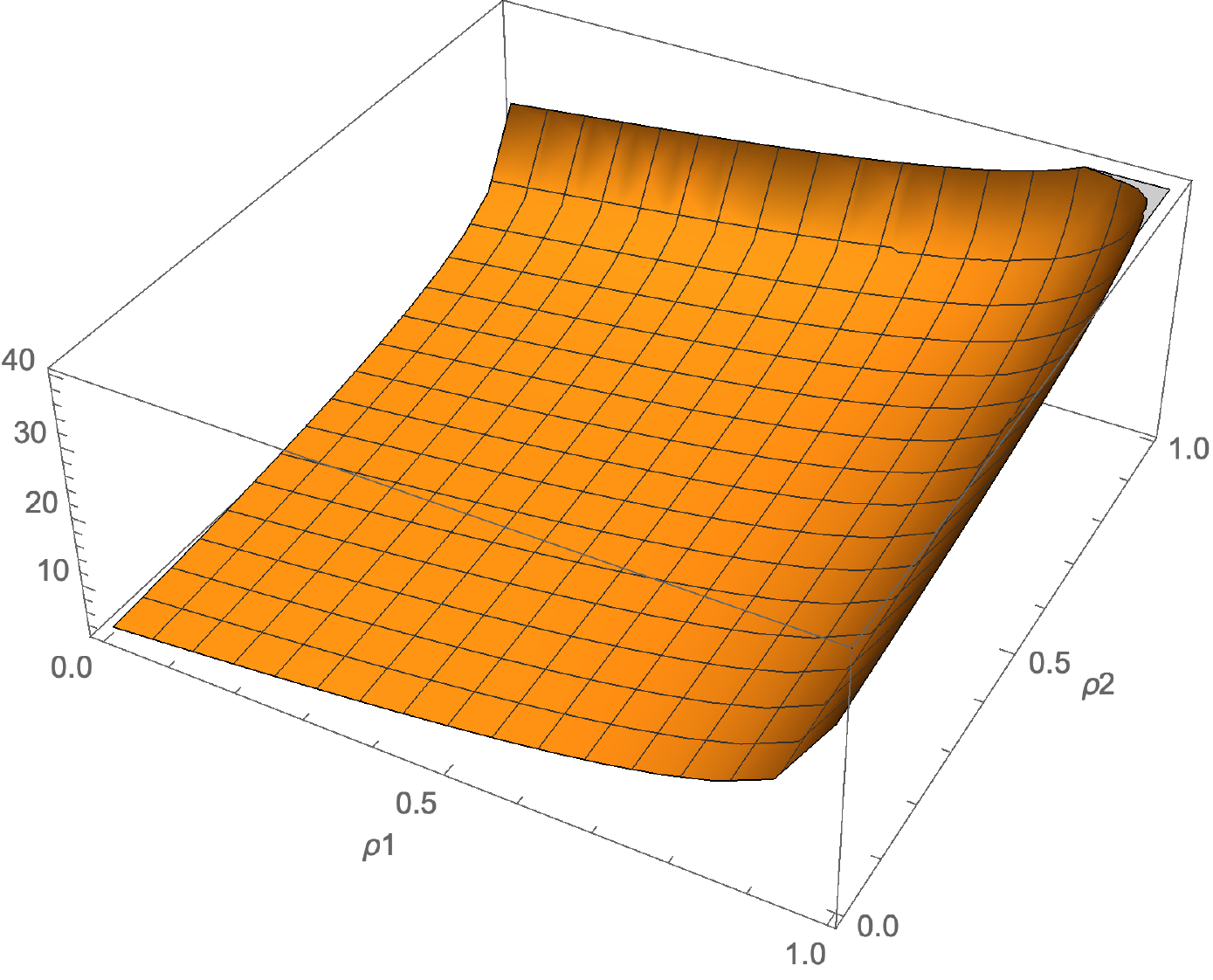}
    \caption{}
    \end{subfigure}
    \caption{(a) $R_{\mathbf{X},x^{(3)}}[\boldsymbol{\rho}]$ for (a) $\lambda=0.01$, (b) $\lambda=0.5$, (c) $\lambda=2$ and (d) $\lambda=10$ in section \ref{app:analysis_risk}, for $\Sigma_x=1,\Omega_x=2$.}%
    \label{fig:risk_tune}%
\end{suppfigure}

\subsection{Upper bound for the mutual information-regularized risk}
\label{app:Gaussian_channel_upper_bound}

In this section, we prove that $I(\tilde{X}^{(j)(\eta_j)}_{t-1};X^{(j)}_{t-1})\leq \frac{1}{2}\sum_{l=1}^{KM}\text{log}\left(1+\frac{\text{Var}(X_{t-1,l}^{(j)})}{\eta_{j,l}^2}\right)$. We formally state the theorem as follows:

\begin{theorem}
Let $\tilde{X}_{t-1}^{(j)(\eta_j)}:=X_{t-1}^{(j)}+\eta_j\cdot \epsilon_j$, $j=1,2,...N$ be the noise-corrupted inputs with \emph{learnable} noise amplitudes $\eta_j\in \R^{KM}$, and $\epsilon_j\sim N(\mathbf{0},\mathbf{I})$. We have

\begin{equation}
I(\tilde{X}^{(j)(\eta_j)}_{t-1};X^{(j)}_{t-1})\leq \frac{1}{2}\sum_{l=1}^{KM}\text{log}\left(1+\frac{\text{Var}(X_{t-1,l}^{(j)})}{\eta_{j,l}^2}\right)
\end{equation}

where $l$ is the $l^{\text{th}}$ element of a vector, $\text{std}(X_{t-1,l}^{(j)})$ is the standard deviation of $X_{t-1,l}^{(j)}$ across $t$. The equality is reached when $X_{t-1}^{(j)}$ obeys a multivariate Gaussian distribution with diagonal covariance matrix $\Sigma$ satisfying $\Sigma_{l,l}=\text{Var}(X^{(j)}_{t-1,l})+\eta_{j,l}^2$.
\end{theorem}

\begin{proof}
We have
\begin{equation*}
\begin{aligned}
&I(\tilde{X}^{(j)(\eta_j)}_{t-1};X^{(j)}_{t-1})=H(\tilde{X}^{(j)(\eta_j)}_{t-1}) - H(\eta_j\cdot\epsilon_j)\\
&=H(\tilde{X}^{(j)(\eta_j)}_{t-1}) - \left(\frac{KM}{2}\text{log}(2\pi e)+\sum_{l=1}^{KM}\frac{1}{2}\text{log}(\eta_{j,l}^2)\right)\\
\end{aligned}
\end{equation*}

Here $H(\cdot)$ is differential entropy. For $\tilde{X}^{(j)(\eta_j)}_{t-1}$, its variance at the $l^{\text{th}}$ dimension is
\begin{equation*}
\begin{aligned}
\text{Var}(\tilde{X}^{(j)(\eta_{j})}_{t-1,l})&=\text{Var}(X^{(j)}_{t-1,l}+\eta_j\cdot\epsilon_j)\\
&=\text{Var}(X^{(j)}_{t-1,l})+\text{Var}(\eta_{j,l}\cdot\epsilon_{j,l})\\
&=\text{Var}(X^{(j)}_{t-1,l})+\eta_{j,l}^2
\end{aligned}
\end{equation*}

The second equality is due to that $X_{t-1}^{(j)}$ is independent of $\epsilon_j$. Using the principle of maximum entropy, the distribution that maximizes $H(\tilde{X}^{(j)(\eta_j)}_{t-1})$ subject to the constraint of $\text{Var}(\tilde{X}^{(j)(\eta_{j})}_{t-1,l})=\text{Var}(X^{(j)}_{t-1,l})+\eta_{j,l}^2, l=1,2,...KM$ is a Gaussian distribution whose diagonal covariance matrix $\Sigma$ satisfies $\Sigma_{l,l}=\text{Var}(X^{(j)}_{t-1,l})+\eta_{j,l}^2$. Its entropy is $H(\tilde{X}^{(j)(\eta_j)}_{t-1})=\frac{KM}{2}\text{log}(2\pi e)+\sum_{l=1}^{KM}\frac{1}{2}\text{log}(\eta_{j,l}^2+\text{Var}(X^{(j)}_{t-1,l}))$. Therefore,

\begin{equation*}
\begin{aligned}
&I(\tilde{X}^{(j)(\eta_j)}_{t-1};X^{(j)}_{t-1})\\
&\leq \left(\frac{KM}{2}\text{log}(2\pi e)+\sum_{l=1}^{KM}\frac{1}{2}\text{log}(\eta_{j,l}^2+\text{Var}(X^{(j)}_{t-1,l}))\right)-\left(\frac{KM}{2}\text{log}(2\pi e)+\sum_{l=1}^{KM}\frac{1}{2}\text{log}(\eta_{j,l}^2)\right)\\
&=\frac{1}{2}\sum_{l=1}^{KM}\text{log}\left(1+\frac{\text{Var}(X_{t-1,l}^{(j)})}{\eta_{j,l}^2}\right)\\
\end{aligned}
\end{equation*}
The equality is reached when $X_{t-1}^{(j)}$ obeys a multivariate Gaussian distribution with diagonal covariance matrix $\Sigma$ satisfying $\Sigma_{l,l}=\text{Var}(X^{(j)}_{t-1,l})+\eta_{j,l}^2$.
\end{proof}

\subsection{Implementation details for the methods}
\label{app:algorithm_implementation}

Here we state the implementation details for our method, as well as other methods being compared. Throughout this paper, unless otherwise specified, we use the standard k-nearest neighbor technique in \cite{kraskov2004estimating} to estimate the KL-divergence and mutual information (with number of neighbors $k=5$) and conditional mutual information (with number of neighbors $k=3$), which is used in our implementations of Mutual information, Transfer Entropy and Causal Influence.

\subsubsection{Our method}

Without stating otherwise, our method (Algorithm \ref{alg:learnable_noise}) as a default uses a three layer neural net, with two hidden layers having 8 neurons and leakyReLU ($\text{max}(0.3x,x)$) activation, and the last layer having linear activation. We set the number of fake time series $S=\text{max}(2,\ceil[\big]{N/2})$, and significance level $\alpha=0.05$. Adam \cite{kingma2014adam} optimizer with learning rate $=10^{-4}$ is used as default throughout this paper. We set $\eta_0=0.01$ and $\lambda=0.002$. We use 30000 epochs.
It also has a 400 epoch warm-up period where the mutual information term is turned off, to allow $f_\theta$ to find a good initial model as a start. We use the the upper bound (Eq. \ref{eq:empirical_upper_bound}) as the risk and also in estimating $W_{ji}$, as discussed in the main text in Section \ref{sec:our_method}. 
In this work, the relative noise amplitude $\chi_{j,l}=\frac{\eta_{j,l}}{\text{std}(X_{t-1,l}^{(j)})}$ is shared across the dimension $l$ for each time series $j$. This simplifies the risk calculation, and is invariant to the rescaling of each time series $X_{t-1}^{(j)}$. We also tested fully parameterizing $\chi_{j,l}$ with a similar performance.

\subsubsection{Transfer Entropy}
We use the definition of transfer entropy as defined in \cite{schreiber2000measuring}. In that work the transfer entropy is defined for two time series. To deal with multiple time series, we let $X_{t-1}^{(\hat{j})}$ also include other time series, similar to the extension of transfer entropy as in \cite{lizier2008local}.

\subsubsection{Causal Influence}
For causal influence \cite{janzing2013quantifying}, we use the same network architecture as in our method, to learn a prediction model. Then the KL divergence is estimated via the technique in \cite{kraskov2004estimating}.

\subsubsection{Linear Granger}

We follow the definition of linear Granger causality (Eq. (7) and (8) in \cite{ding2006granger}) to calculate linear Granger causality. Specifically, we calculate the residual squared error of a linear predictor of $x_{t-1}^{(i)}$ with and without $X_{t-1}^{(j)}$ (both with $\mathbf{X}_{t-1}^{(\hat{j})}$). Then the linear Granger causality equals the log of the ratio of the two residual squared errors.

\subsubsection{Kernel Granger}

We use the implementation\footnote{At \href{https://github.com/danielemarinazzo/KernelGrangerCausality}{https://github.com/danielemarinazzo/KernelGrangerCausality}.} for \cite{marinazzo2008kernel,marinazzo2008kernel2} for estimating kernel Granger causality. We use their default settings, with inhomogeneous polynomial (IP) kernel of degree $p=2$. We follow the normalization requirement of the algorithm to normalize the data for each experiment.

\subsubsection{Elastic Net}
We use elastic net \cite{zou2005regularization} with 5-fold time-series-split cross-validation, along the following regularization path: L1-ratio: 0.5, 0.8, 0.9, 0.95, 0.99, and strength of penalization $\alpha$ being a 200-step geometric series from $10^{-4}$ to $10^{-0.5}$. The score function used for cross-validation is the coefficient of determination ($R^2$). The elastic net is implemented with scikit-learn's ElasticNetCV module\footnote{At \href{https://scikit-learn.org/stable/modules/generated/sklearn.linear_model.ElasticNetCV.html}{https://scikit-learn.org/stable/modules/generated/sklearn.linear\_model.ElasticNetCV.html}.}, with optimization tolerance of $10^{-10}$.

\subsubsection{Gaussian Random}
For Gaussian Random, we draw 10,000 random matrices, each element of which is drawn from a standard Gaussian distribution.

\subsection{Implementation details for synthetic experiments}
\label{app:synthetic_exp}

For all experiments in this section, each metric is obtained by performing the experiments (including generation of the dataset and the training) ten times with seed = 0, 30, 60, 90, 120, 150, 180, 210, 240, 270 and averaging the resulting metrics (for Gaussian random matrices, for each true causal matrix $A$ sample 10,000 random matrices $\tilde{A}$). For the ground-truth causal tensor $A$, each element $A_{ji}$ is a $K\times M$ matrix, with 0.5 probability of being an all-zero matrix, and 0.5 probability of being a nonzero matrix. If $A_{ji}$ is a nonzero matrix, its each element is sampled from a log-normal distribution with $\mu=0$ and $\sigma=1$. For $B$, each $B_j$ is also a $K\times N$ matrix, with each element sampling from $U[-1,1]$. We use $\text{H}_1(x)=\text{softplus}(x)=\text{log}(1+e^x)$, and $\text{H}_2(x)=\text{tanh}(x)$ in equation (\ref{eq:synthetic}). 
As a default, 500 time series each with length of 22 are generated from Eq. (\ref{eq:synthetic}), each of which is wrapped into 19 $(\mathbf{X}_{t-1}, x_{t}^{(i)})$ pairs (since $K=3$), so there are in total $500 \times 19=9500$ examples for each dataset. The train-test-split is 9:1 for all experiments in this paper. See Fig. S\ref{fig:synthetic_example_figure} for example snapshots of time series together with the corresponding $A_{ji}$ matrices.

\subsection{AUC-ROC table for synthetic experiment}
\label{app:synthetic_ROC_AUC}
Table S\ref{table:synthetic_larger_N_AUC_PR} show the AUC-ROC table for the synthetic experiment, where for each $N$, 10 datasets are randomly sampled according to Eq. (\ref{eq:synthetic}) using random seed 0, 30, 60, 90, 120, 150, 180, 210, 240, 270, over which each method is run and their metrics are accumulated. It has similar behavior as the AUC-PR table (Table \ref{table:synthetic_larger_N_AUC_PR}) in the main text.

\begin{suppfigure}[h!]
\centering
\begin{subfigure}{1\columnwidth}
\includegraphics[scale=0.33]{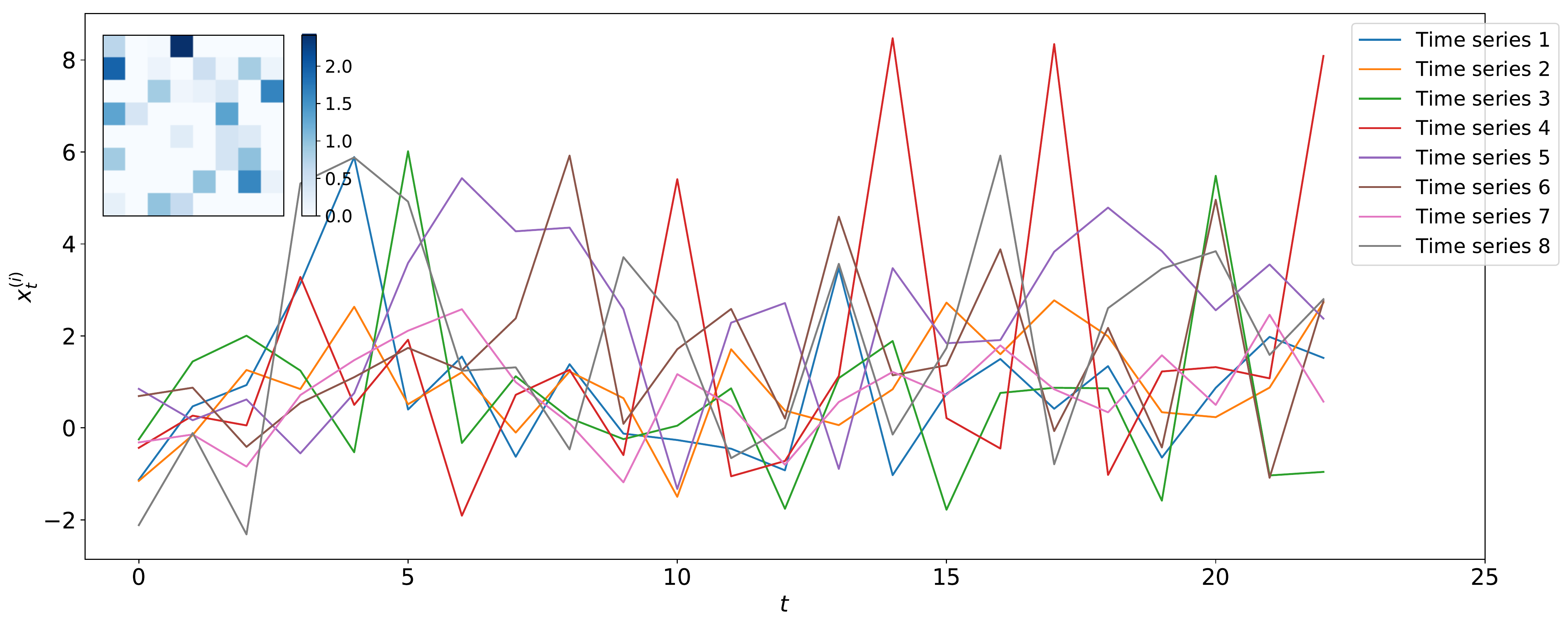}
\caption{}
\end{subfigure}
\begin{subfigure}{1\columnwidth}
\includegraphics[scale=0.33]{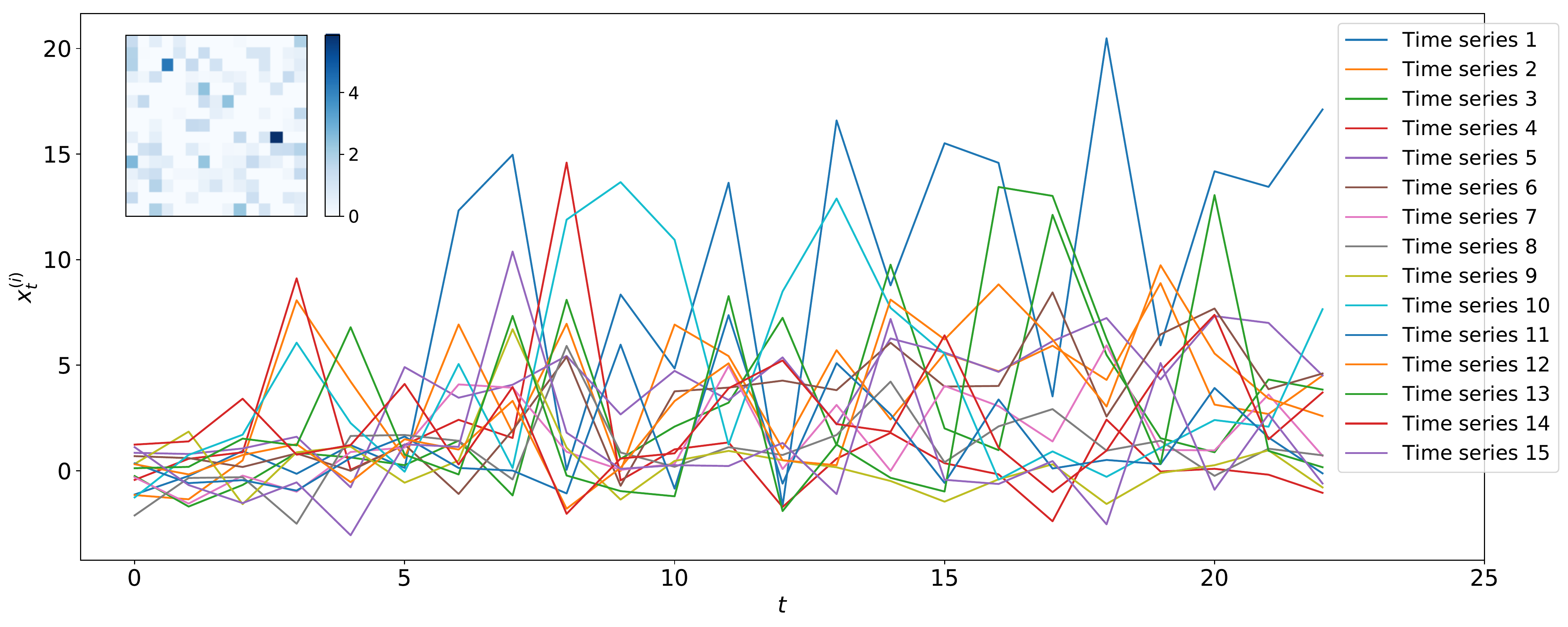}
\caption{}
\end{subfigure}
\begin{subfigure}{1\columnwidth}
\includegraphics[scale=0.33]{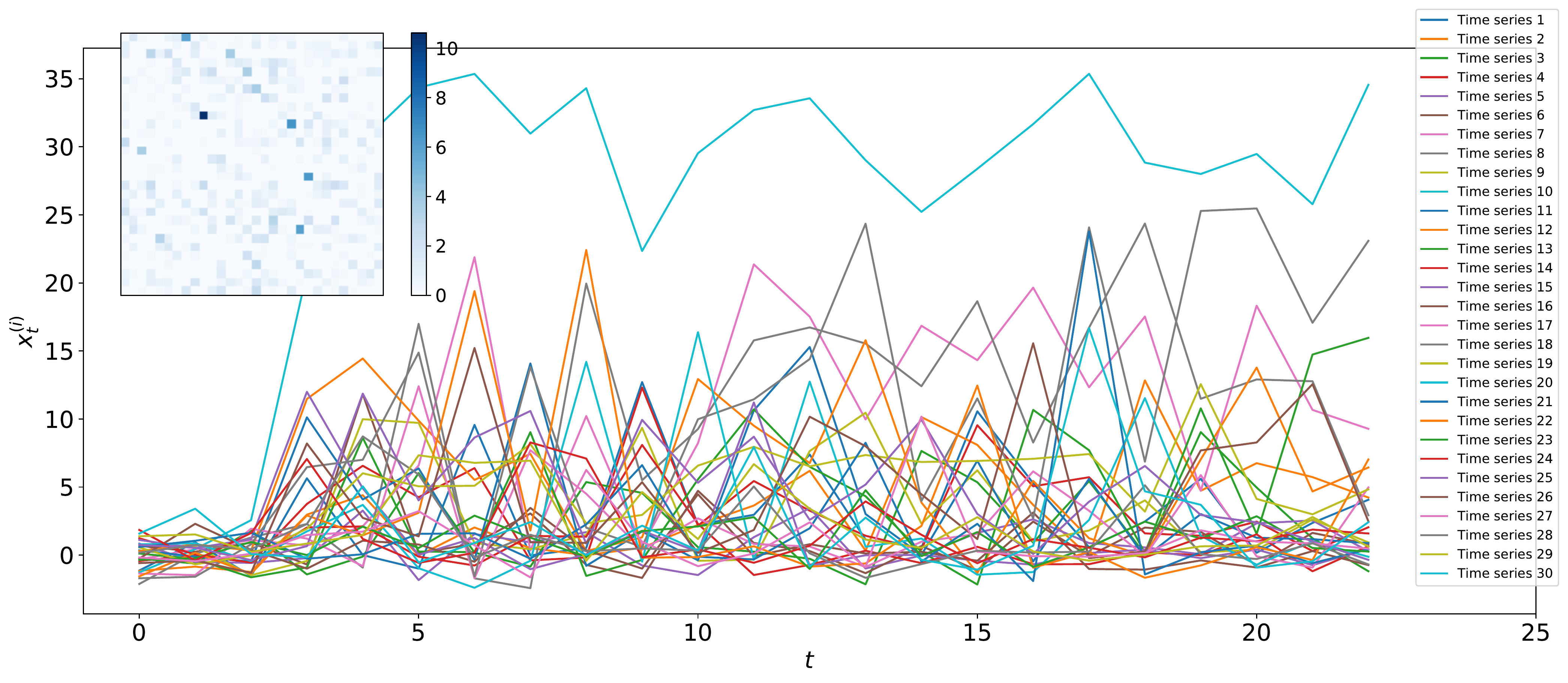}
\caption{}
\end{subfigure}
\caption{Example snapshots of the synthetic time series with (a) $N=8$, (b) $N=15$, and (c) $N=30$. The inset is the hidden underlying $|A_{ji}|$ matrix, whose $(j,i)$ element denotes the causal strength from time series $j$ to $i$. We see that the causal strength varies in orders, making it very difficult to identify each edge correctly.}
\label{fig:synthetic_example_figure}
\end{suppfigure}

\begin{supptable*}[t]
\caption{Mean and standard deviation of AUC-ROC (\%) vs. $N$, over 10 random sampling of datasets. Bold font marks the top method for each $N$.}
\centering
\begin{tabular}{p{3cm}p{1cm}p{1cm}p{1cm}p{1cm}p{1cm}p{1cm}p{1cm}p{1cm}}
\toprule
N &      3  &    4  &    5  &    8  &    10 &    15 &    20 &    30 \\
method             &         &       &       &       &       &       &       &       \\
\midrule
\textbf{MPIR (ours)}        &    95.3{\tiny$\pm$10.0} &  97.6{\tiny$\pm$4.1} &  \textbf{97.3}{\tiny$\pm$3.6} &  \textbf{96.0}{\tiny$\pm$2.4} &  \textbf{94.2}{\tiny$\pm$3.8} &  \textbf{91.0}{\tiny$\pm$3.5} &  \textbf{85.5}{\tiny$\pm$2.4} &  \textbf{76.8}{\tiny$\pm$3.5} \\
Mutual Information &    84.1{\tiny$\pm$18.9} &  90.0{\tiny$\pm$7.6} &  89.0{\tiny$\pm$1.8} &  87.2{\tiny$\pm$3.8} &  81.3{\tiny$\pm$5.3} &  77.5{\tiny$\pm$3.9} &  74.6{\tiny$\pm$3.0} &  72.0{\tiny$\pm$2.0} \\
Transfer Entropy   &    88.3{\tiny$\pm$14.6} &  95.6{\tiny$\pm$5.7} &  89.9{\tiny$\pm$8.7} &  84.4{\tiny$\pm$7.6} &  80.8{\tiny$\pm$5.1} &  69.6{\tiny$\pm$2.5} &  64.7{\tiny$\pm$2.5} &   59.2{\tiny$\pm$1.9} \\
Linear Granger     &    \textbf{98.8}{\tiny$\pm$4.0} &  96.2{\tiny$\pm$5.5} &  91.7{\tiny$\pm$8.9} &  84.1{\tiny$\pm$9.0} &  82.7{\tiny$\pm$7.2} &  73.6{\tiny$\pm$6.9} &  69.9{\tiny$\pm$4.1} &  60.0{\tiny$\pm$2.6} \\
Kernel Granger     &    98.1{\tiny$\pm$5.9} &  \textbf{98.0}{\tiny$\pm$4.4} &  95.4{\tiny$\pm$3.9} &  91.2{\tiny$\pm$2.6} &  89.5{\tiny$\pm$3.3} &  82.4{\tiny$\pm$2.2} &  76.2{\tiny$\pm$2.2} &  68.1{\tiny$\pm$1.3} \\
Elastic Net        &    97.5{\tiny$\pm$7.9} &  97.4{\tiny$\pm$4.5} &  95.3{\tiny$\pm$4.3} &  90.4{\tiny$\pm$5.1} &  87.7{\tiny$\pm$4.1} &  81.8{\tiny$\pm$3.1} &  77.8{\tiny$\pm$3.0} &  72.7{\tiny$\pm$1.4} \\
Causal Influence   &    62.9{\tiny$\pm$28.3} &  58.3{\tiny$\pm$13.8} &  60.4{\tiny$\pm$11.7} &  47.4{\tiny$\pm$7.5} &  50.7{\tiny$\pm$5.6} &  55.3{\tiny$\pm$3.3} &  51.0{\tiny$\pm$3.2} &   50.3{\tiny$\pm$1.6} \\
Gaussian random    &    49.9{\tiny$\pm$0.3} &  50.0{\tiny$\pm$0.1} &  50.0{\tiny$\pm$0.1} &  50.0{\tiny$\pm$0.0} &  50.0{\tiny$\pm$0.1} &  50.0{\tiny$\pm$0.0} &  50.0{\tiny$\pm$0.0} &  50.0{\tiny$\pm$0.0} \\
\bottomrule
\end{tabular}
\label{table:synthetic_larger_N_AUC_ROC}
\end{supptable*}

\subsection{Additional experiment: testing with model capacity variations}
\label{app:capacity}
Since in practice, we do not know the underlying causal structure \textit{a priori}, it presents a greater challenge to select the model capacity for $f_\theta$, as compared with supervised learning method where we can do cross-validation. To see how the capacity of the function approximator $f_\theta$ influences our method, we vary the number of layers and the number of neurons in each layer at $N=10$, using the same 10 datasets as in Section \ref{sec:synthetic}. Table S\ref{table:synthetic_capacity} summarizes the result. We see that our method's performance here is hardly influenced by the model capacity, with only a slight degradation at very low capacity. This shows that our method is quite tolerant and stable with model capacity variations. 

\begin{supptable}[t]
\caption{Average and standard deviation of AUC-PR and AUC-ROC for different network structures for $N=10$ with our method. Here for example, (8, 8, 8) means that the $f_\theta$ has 3 hidden layers, each with 8 neurons.}
    \centering
    \begin{tabular}{lrr}
\toprule
                           &  AUC-PR (\%) &  AUC-ROC (\%) \\
Neurons in hidden layers &         &          \\
\midrule
(8) &  90.0{\small$\pm$4.9} &   91.5{\small$\pm$4.3} \\
  (8, 8) &  93.4{\small$\pm$3.6} &   94.1{\small$\pm$3.7} \\
  (8, 8, 8) &  93.6{\small$\pm$3.6} &   94.4{\small$\pm$3.6} \\
   (8, 8, 8, 8) &  93.8{\small$\pm$4.1} &   94.2{\small$\pm$4.3} \\
 (16, 16) &  94.3{\small$\pm$3.3} &   94.4{\small$\pm$3.5} \\
 (16, 16, 16) &  94.6{\small$\pm$3.0} &   95.1{\small$\pm$2.6} \\
(16, 16, 16, 16) &  92.8{\small$\pm$4.4} &   94.0{\small$\pm$3.2} \\
\bottomrule
\end{tabular}
\label{table:synthetic_capacity}
\end{supptable}

\subsection{Details for the video game dataset}
\label{app:breakout}

Here, we implement a custom Atari Breakout game in the OpenAI Gym \cite{1606.01540} environment, mimicking the original game\footnote{A game playing video can be seen at \href{https://goo.gl/XGzppc}{https://goo.gl/XGzppc}.}, where we can access the state of the ball, paddle and bricks, etc. This representation is also used in the OO-MDP \cite{diuk2008object} paradigm for a more efficient representation of the environment state. We use the DQN algorithm, the same CNN architecture as in \cite{mnih2015human} to train an RL agent. Then we let it play the game for $\sim$45000 steps, obtaining a dataset with time-length of 45000 steps (if the agent dies, we restart the game) and 6 time series: action, paddle's $x$ position, ball's $x$ position, ball's $y$ position, number of bricks and reward. We then feed the time series (each time series normalized to mean of 0 and variance of 1) to our method, the same procedure as performed in the synthetic experiment, to let it produce an inferred matrix $W_{ji}$, which is shown in Fig. 1 in main text. All the datasets used in this paper and code will be open-sourced upon publication of the paper.

\subsection{Implementation details for experiment with heart-rate vs. breath-rate}
\label{app:real_dataset}

For the two real-world datasets, we obtain the data with the same procedure as in \cite{ancona2004radial} (See Fig.S\ref{fig:apnea_figure} for their plots). Then the data (each time series normalized to mean of 0 and variance of 1) are fed into our algorithm to infer the causal strength $W_{ji}$. For each $K=1,2,...20$, the experiments are run for 50 times with seed from 0 to 49, and Fig. \ref{fig:apnea} in the main text is obtained by averaging over the inferred $W$ matrix.

\begin{suppfigure}[t]
\begin{center}
\centerline{\includegraphics[width=1.0\columnwidth]{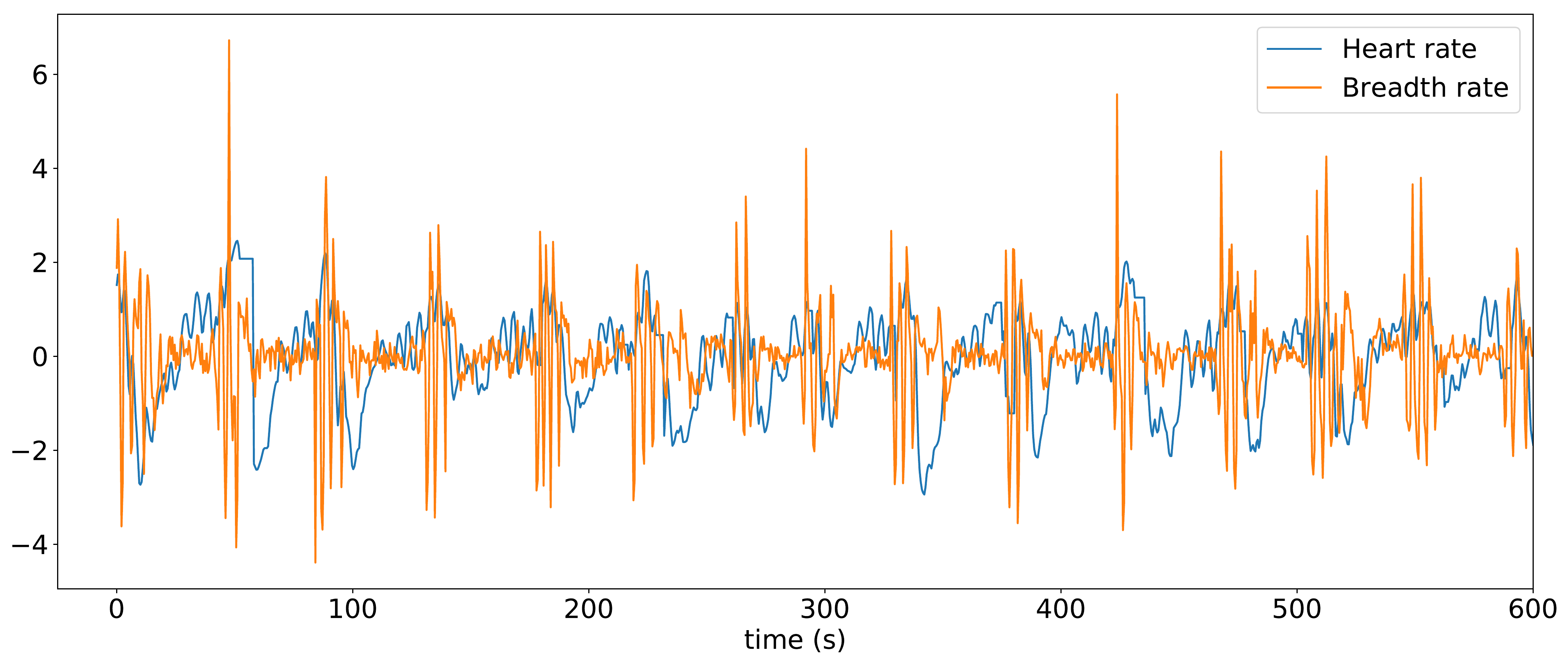}}
\caption{Time series of the heart rate and breath rate of a patient suffering sleep apnea. The data is normalized to have 0 mean and standard deviation of 1. Sample rate is 2Hz.}
\label{fig:apnea_figure}
\end{center}
\end{suppfigure}

\subsection{Additional experiment: rat EEG dataset}
\label{app:rat_EEG_experiment}

As another real-world example, we apply our algorithm to estimate the directional relations of the EEG signals between the right and left cortical intracranial electrodes \cite{ratEEG}, before and after lesion (see Fig. S\ref{fig:ratEEG_before} and S\ref{fig:ratEEG_after_figure} for the signals), also studied in \cite{ancona2004radial,quiroga2002performance,marinazzo2008kernel}. Figure S\ref{fig:ratEEG_W} (left) shows the inferred predictive strength $W_{ji}$ for the EEG signals of a normal rat. We see that there is only a slight asymmetry, with the right channel having a slightly stronger influence on the left channel than the reverse direction. Fig. S\ref{fig:ratEEG_W} (right) shows $W_{ji}$ for the EEG signals with unilateral lesion in the rostral pole of the reticular thalamic nucleus. We see that there is stronger predictive strength from the left to the right channels. Compared with the result of previous works \cite{ancona2004radial,marinazzo2008kernel} as also shown in Fig. S\ref{fig:ratEEG_compare}, we see that all methods correctly infer the directional relations before and after brain lesion. In addition, our method shows only a slight decay of predictive strength with increasing history length, in contrast to the much more rapid decay of causality index in \cite{ancona2004radial}, again demonstrating our method's insensitivity against history length, due to its flexibility in extracting the right amount of information in order to predict the future. This experiment and the breadth rate vs. heart rate experiment in Section \ref{sec:heart_rate} demonstrate our method's capability in inferring the directional relations from noisy, real-world data.

\begin{suppfigure}[t]
\begin{center}
\centerline{\includegraphics[width=1\columnwidth]{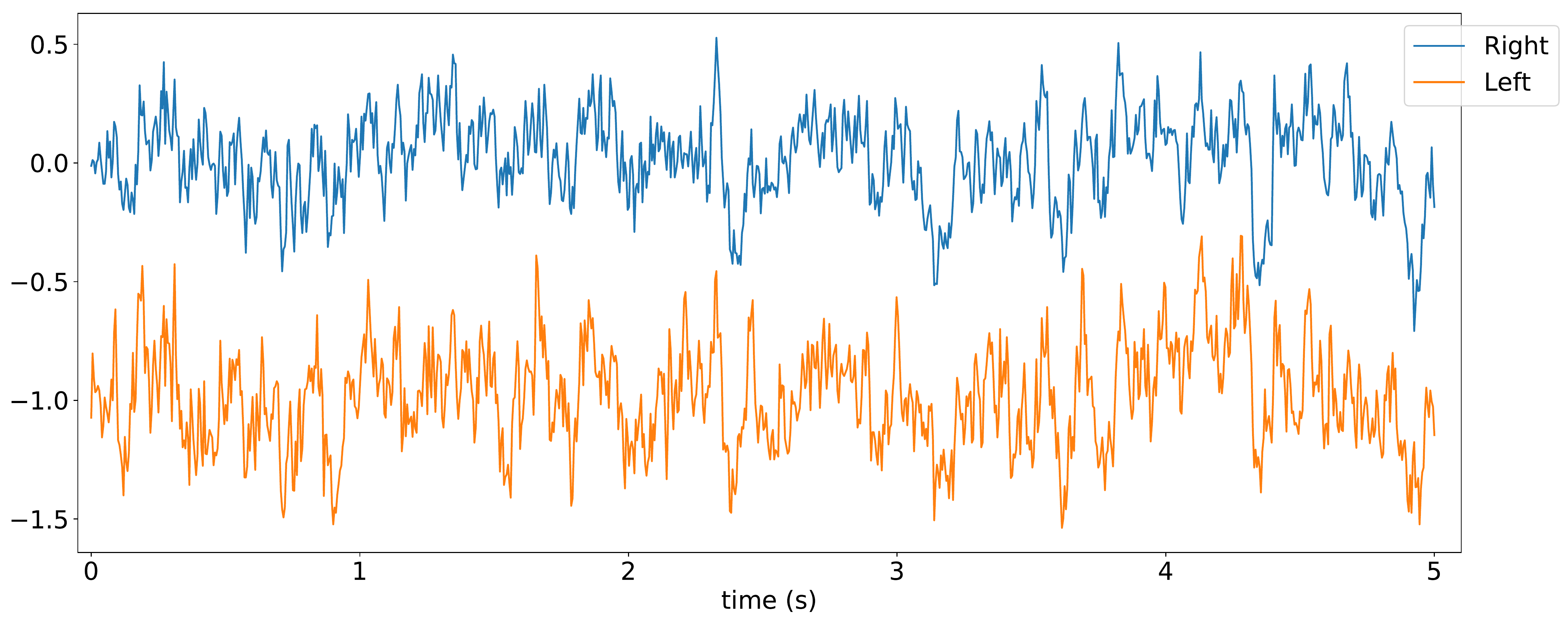}}
\caption{Time series of a normal rat EEG signals from right and left cortical intracranial electrodes. The data is normalized to have 0 mean and standard deviation of 1, and the left signal is plotted with offset for better visualization. Sample rate is 200Hz.}
\label{fig:ratEEG_before}
\end{center}
\vskip -0.3in
\end{suppfigure}

\begin{suppfigure}[t]
\begin{center}
\centerline{\includegraphics[width=1\columnwidth]{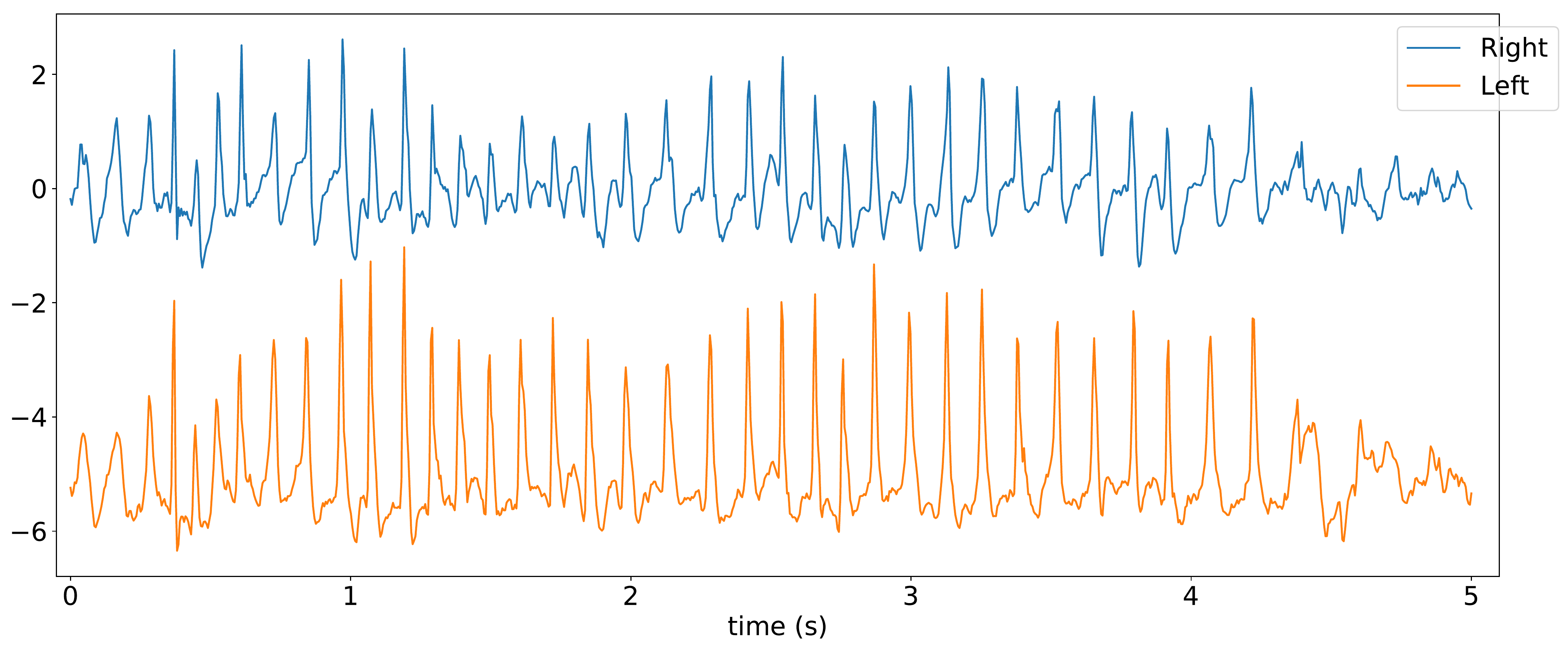}}
\caption{Time series of a rat EEG signals from right and left cortical intracranial electrodes, after lesion. The data is normalized to have 0 mean and standard deviation of 1, and the left signal is plotted with offset for better visualization. Sample rate is 200Hz.}
\label{fig:ratEEG_after_figure}
\end{center}
\vskip -0.3in
\end{suppfigure}

\begin{suppfigure}[t]
\begin{center}
\centerline{\includegraphics[width=0.8\columnwidth]{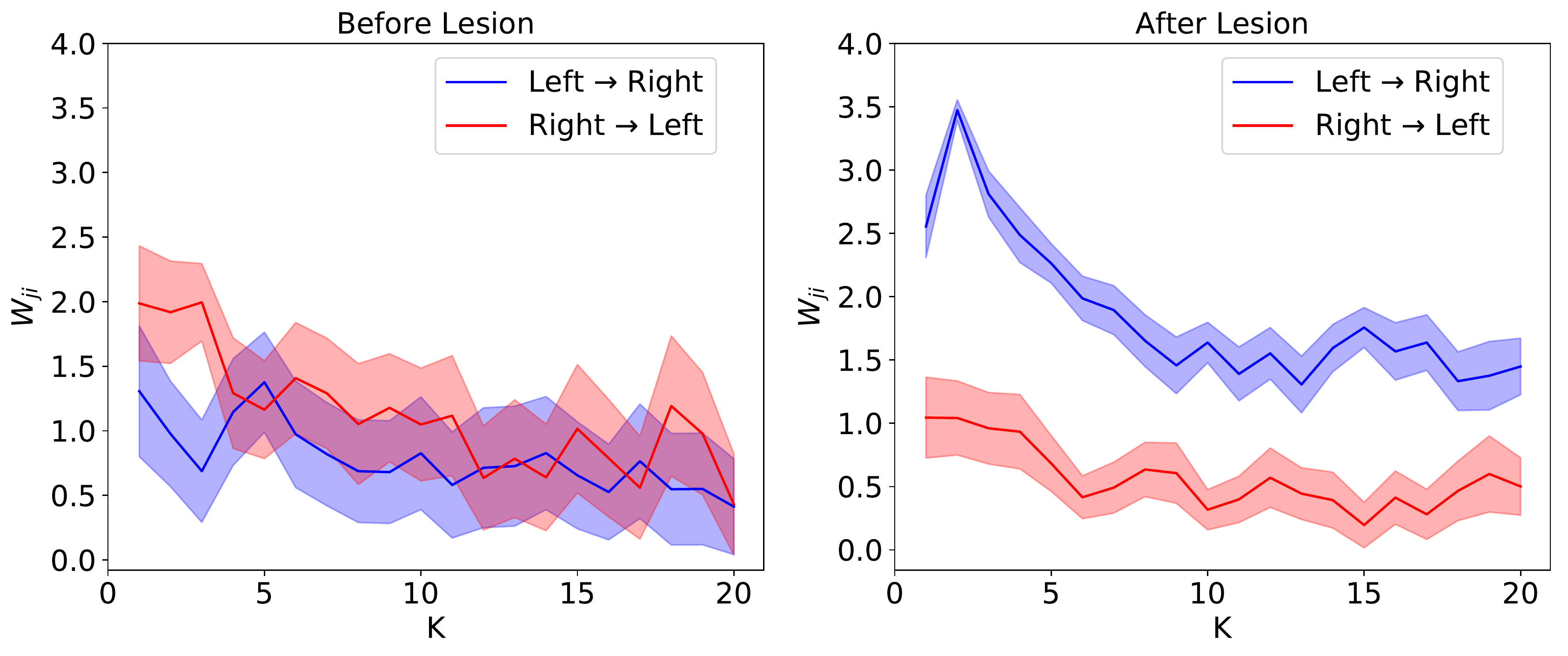}}
\caption{Predictive strength inferred by our method with the EEG datasets, for different maximum time horizon $K$, averaged over 50 initializations of $f_\theta$, for a normal rat (left) and after brain lesion (right).}
\label{fig:ratEEG_W}
\end{center}
\vskip -0.3in
\end{suppfigure}

\begin{suppfigure}
\centering
\begin{subfigure}{.75\linewidth}
\includegraphics[scale=0.55]{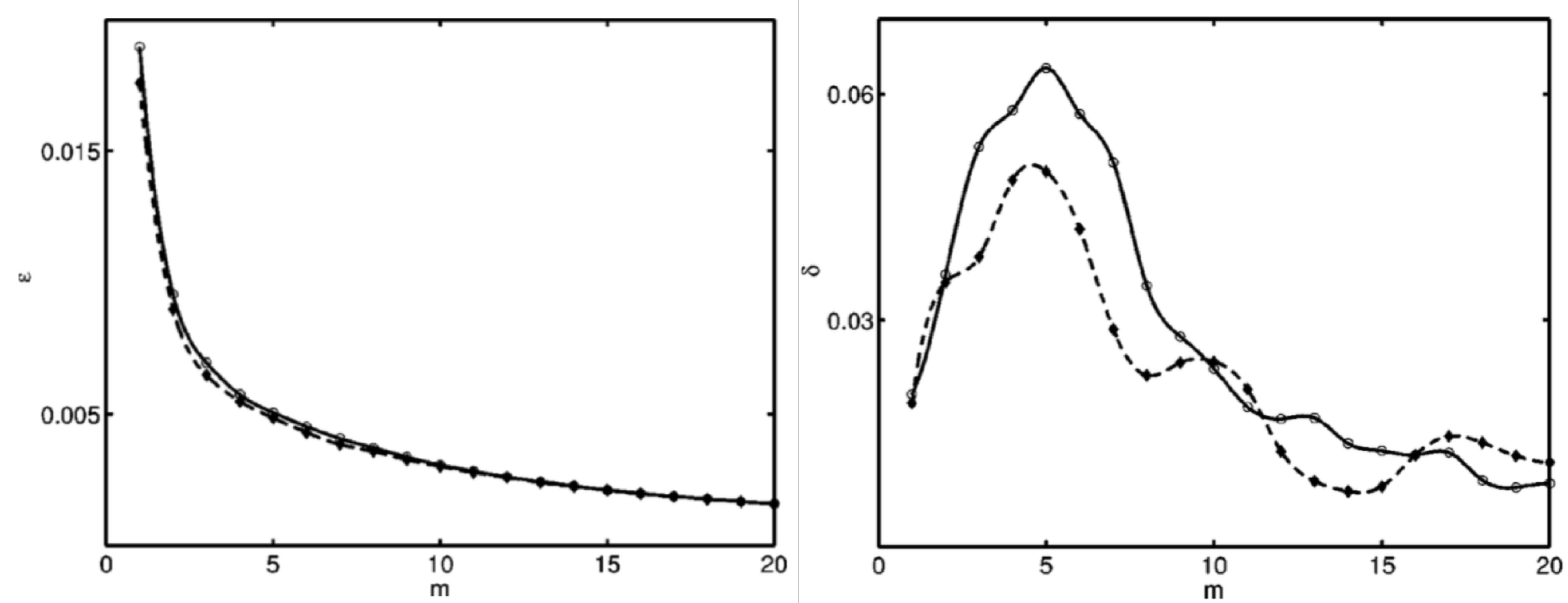}
\caption{}
\end{subfigure}
\begin{subfigure}{.6\linewidth}
\includegraphics[scale=0.55]{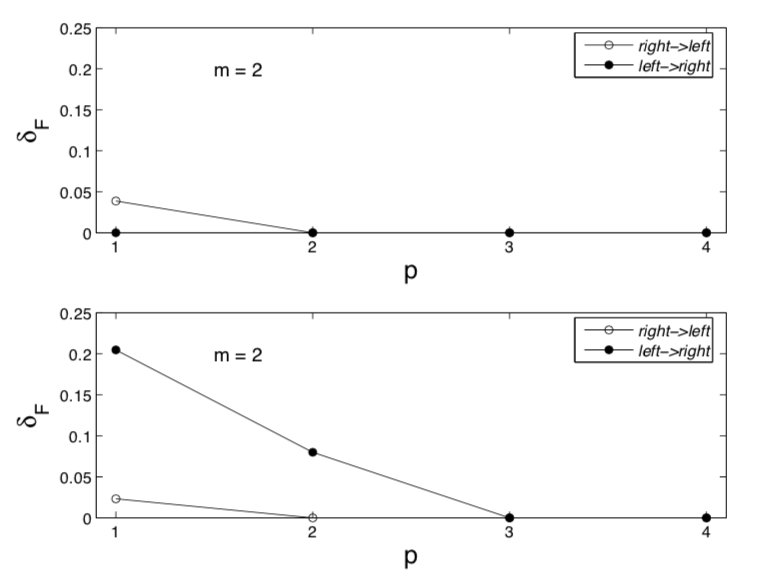}
\caption{}
\end{subfigure}
\caption{Causal indices for the rat EEG dataset with previous methods. (a) By \cite{ancona2004radial}. Left: the variance for the left EEG (open circles) and right EEG (diamonds) vs. time lag $m$ before brain lesion. Right: the causality index after brain lesion. (b) By \cite{marinazzo2008kernel2}. The filtered causality index vs. varying $p$, the order of the inhomogeneous polynomial kernel, before (upper) and after (lower) brain lesion.}
\label{fig:ratEEG_compare}%
\end{suppfigure}


\section{Appendix for Chapter \ref{chap7:mela}}

\textbf{MeLA architectural details}

As described in section \ref{sec:architecture}, MeLA consists of a meta-recognition model $\m_{\muvec}$ and a meta-generative model $\g_{\gammavec}$ that generates the task-specific model $\f_{\phivec}$. The meta-recognition model consists of two blocks. The first block is a MLP with 3 hidden layers, each of which has 60 neurons with leakyReLU activation (unless otherwise specified, the leakyReLU activation in this paper all have a slope of 0.3 when the activation is below 0). The last layer has $\spool=200$ neurons and linear activation. Then a max-pooling is performed along the example dimension, collapsing the $N\times\spool$ matrix to $1\times\spool$ matrix, which feeds into the second block. The second block is an MLP with two hidden layers, each of which has 60 neurons with leakyReLU activation, and the last layer has $s_{\rm code}$ neurons with linear activation. The output is the model code $\z$. 

The meta-generative model $\g_{\gammavec}$ takes as input the model code $\z$, 
and for each layer in the main model $\f_{\phivec}$, it has two separate MLPs that map $\z$ to all the weight and bias parameters of that layer. For all the experiments in this paper, the MLPs in the meta-generative model have 3 hidden layers, each of which has 60 neurons with leakyReLU activation. The last layer of the MLP has linear activation, and has an output size equal to the size of weight or bias in the main network $\f_{\phivec}$. The output of each MLP in the meta-generative model is then reshaped into the size of the corresponding weight or bias matrix, and directly used as the parameters of $\f_{\phivec}$.

The architecture of the main network $\f_{\phivec}$ is dependent on the specific application, which MeLA's architecture is agnostic to. For the simple regression problem in this paper, we implement $\f_{\phivec}$ as an MLP with 2 hidden layers, each of which has 40 neurons with leakyReLU activation. The last layer has linear activation with output size of 1. For the ball bouncing with state representation experiment, $\f_{\phivec}$ is an MLP with input size of 6 and 3 hidden layers, each of which has 40 neurons with leakyReLU activation. The last layer has linear activation with output size of 2. For the video prediction task, the latent dynamics network uses the same architecture. The convolutional autoencoder used in this experiment is as follows. For the encoder, it has 3 convolutional layers with 32 $3\times 3$ kernels with stride 2 and leakyReLU activation. After that, it is flattened into 512 neurons, which feeds into a dense layer with 2 neurons and linear activation. For the decoder, the first layer is a dense layer with 512 neurons and linear activation, then the output is reshaped to a $N\times4\times4\times32$ tensor (32 is the number of channels). The tensor then goes into 3 layers of convolutional-transpose layers with 32 kernels, each with size of 3, stride of 2 and leakyReLU activation. For the leakyReLU activation in the convolutional autoencoder, we use a slope of 0.01 when the activation is below 0.

\section{Appendix for Chapter \ref{chap8:rankpruning}}

\subsection{Proofs}
In this section, we provide proofs for all the lemmas and theorems in the main paper. We always assume that a class-conditional extension of the Classification Noise Process (CNP) \citep{angluin1988learning} maps true labels $y$ to observed labels $s$ such that each label in $P$ is flipped independently with probability $\rho_1$ and each label in $N$ is flipped independently with probability $\rho_0$ ($s \leftarrow CNP(y, \rho_1, \rho_0)$), so that $P(s=s|y=y,x) = P(s=s|y=y)$. Remember that $\rho_1 + \rho_0 <1$ is a necessary condition of minimal information, other we may learn opposite labels.

In Lemma \ref{thm:lemma1_appendix}, Theorem \ref{thm:theorem2_appendix}, Lemma \ref{thm:lemma3_appendix} and Theorem \ref{thm:theorem4_appendix}, we assume that $P$ and $N$ have infinite number of examples so that they are the true,  hidden distributions.

A fundamental equation we use in the proofs is the following lemma:
\medskip

\textbf{Lemma A1}
\textit{When $g$ is ideal, i.e. $g(x)=g^*(x)$ and $P$ and $N$ have non-overlapping support, we have}
\begin{equation} \label{eqA1}
\begin{split}
g(x) = (1 - \rho_1) \cdot \indicator{y=1} + \rho_0 \cdot \indicator{y=0}
\end{split}
\end{equation}

\textbf{Proof:}
Since $g(x)=g^*(x)$ and $P$ and $N$ have non-overlapping support, we have
\begin{equation*}
\begin{split}
g(x) = & g^*(x) = P(s=1|x) \\ 
 = & P(s=1|y=1,x) \cdot P(y=1|x) + P(s=1|y=0,x) \cdot P(y=0|x) \\
 = & P(s=1|y=1) \cdot P(y=1|x) + P(s=1|y=0) \cdot P(y=0|x) \\
 = & (1 - \rho_1) \cdot \mathbbm{1}[[y=1]] + \rho_0 \cdot \mathbbm{1}[[y=0]]
\end{split}
\end{equation*}

\subsubsection{Proof of Lemma 1}\label{sec:A1}

\textbf{Lemma \customlabel{thm:lemma1_appendix}{1}}
\textit{When $g$ is ideal, i.e. $g(x) = g^*(x)$ and $P$ and $N$ have non-overlapping support, we have}
\begin{equation}
\begin{cases}
\tilde{P}_{y=1} = \{x\in P|s=1\}, \tilde{N}_{y=1} = \{x\in P|s=0\}\\
\tilde{P}_{y=0} = \{x\in N|s=1\}, \tilde{N}_{y=0} = \{x\in N|s=0\}
\end{cases}
\end{equation}

\textbf{Proof:} Firstly, we compute the threshold $LB_{y=1}$ and $UB_{y=0}$ used by $\tilde{P}_{y=1}$, $\tilde{N}_{y=1}$, $\tilde{P}_{y=0} $ and $\tilde{N}_{y=0}$. Since $P$ and $N$ have non-overlapping support, we have $P(y=1|x)=\indicator{y=1}$. Also using $g(x) = g^*(x)$, we have

\vskip -0.2in

\begin{align*}
LB_{y=1}=&E_{x\in\tilde{P}}[g(x)]=E_{x\in\tilde{P}}[P(s=1|x)]\\
=&E_{x\in\tilde{P}}[P(s=1|x,y=1)P(y=1|x) +P(s=1|x,y=0)P(y=0|x)]\\
=&E_{x\in\tilde{P}}[P(s=1|y=1)P(y=1|x) +P(s=1|y=0)P(y=0|x)]\\
=&(1-\rho_1)(1-\pi_1)+\rho_0\pi_1\numberthis\label{rh1_prob_eq_A1}
\end{align*}

Similarly, we have
\begin{equation*}
UB_{y=0}=(1-\rho_1)\pi_0+\rho_0(1-\pi_0)
\end{equation*}

Since $\pi_1=P(y=0|s=1)$, we have $\pi_1\in[0,1]$. Furthermore, we have the requirement that $\rho_1+\rho_0<1$, then $\pi_1=1$ will lead to $\rho_1=P(s=0|y=1)=1-P(s=1|y=1)=1-\frac{P(y=1|s=1)P(s=1)}{P(y=1)}=1-0=1$ which violates the requirement of $\rho_1+\rho_0<1$. Therefore, $\pi_1\in[0,1)$. Similarly, we can prove $\pi_0\in[0,1)$. Therefore, we see that both $LB_{y=1}$ and $UB_{y=0}$ are interpolations of $(1-\rho_1)$ and $\rho_0$:
\begin{align*}
&\rho_0<LB_{y=1}\leq1-\rho_1\\
&\rho_0\leq UB_{y=0}<1-\rho_1
\end{align*}
The first equality holds iff $\pi_1=0$ and the second equality holds iff $\pi_0=0$.

Using Lemma A1, we know that under the condition of $g(x)=g^*(x)$ and non-overlapping support, $g(x)=(1 - \rho_1) \cdot \mathbbm{1}[[y=1]] + \rho_0 \cdot \mathbbm{1}[[y=0]]$. In other words, 
\begin{align*}
g(x)&\geq LB_{y=1}\Leftrightarrow x\in P\\
g(x)&\leq UB_{y=0}\Leftrightarrow x\in N
\end{align*}

Since
\begin{equation*}
\begin{cases}
\tilde{P}_{y=1}=\{x\in\tilde{P}|g(x)\geq LB_{y=1}\}\\
\tilde{N}_{y=1}=\{x\in\tilde{N}|g(x)\geq LB_{y=1}\}\\
\tilde{P}_{y=0}=\{x\in\tilde{P}|g(x)\leq UB_{y=0}\}\\
\tilde{N}_{y=0}=\{x\in\tilde{N}|g(x)\leq UB_{y=0}\}
\end{cases}
\end{equation*}

where $\tilde{P}=\{x|s=1\}$ and $\tilde{N}=\{x|s=0\}$, we have
\begin{equation*}
\begin{cases}
\tilde{P}_{y=1} = \{x\in P|s=1\}, \tilde{N}_{y=1} = \{x\in P|s=0\}\\
\tilde{P}_{y=0} = \{x\in N|s=1\}, \tilde{N}_{y=0} = \{x\in N|s=0\}
\end{cases}
\end{equation*}

\subsubsection{Proof of Theorem 2}
We restate Theorem 2 here:

\textbf{Theorem \customlabel{thm:theorem2_appendix}{2}}
\textit{When $g$ is ideal, i.e. $g(x) = g^*(x)$ and $P$ and $N$ have non-overlapping support, we have}

\vskip -0.2in

\begin{align*}
\hat{\rho}_1^{conf}=\rho_1,\hat{\rho}_0^{conf}=\rho_0
\end{align*}

\textbf{Proof:} 
Using the definition of $\hat{\rho}_1^{conf}$ in the main paper:

\begin{equation*}
\hat{\rho}_1^{conf}=\frac{|\tilde{N}_{y=1}|}{|\tilde{N}_{y=1}|+|\tilde{P}_{y=1}|},\ 
\hat{\rho}_0^{conf}=\frac{|\tilde{P}_{y=0}|}{|\tilde{P}_{y=0}|+|\tilde{N}_{y=0}|}
\end{equation*}

Since $g(x) = g^*(x)$ and $P$ and $N$ have non-overlapping support, using Lemma 1, we know 
\begin{equation*}
\begin{cases}
\tilde{P}_{y=1} = \{x\in P|s=1\}, \tilde{N}_{y=1} = \{x\in P|s=0\}\\
\tilde{P}_{y=0} = \{x\in N|s=1\}, \tilde{N}_{y=0} = \{x\in N|s=0\}
\end{cases}
\end{equation*}
Since $\rho_1=P(s=0|y=1)$ and $\rho_0=P(s=1|y=0)$, we immediately have
\begin{align*}
\hat{\rho}_1^{conf}=\frac{|\{x\in P|s=0\}|}{|P|}=\rho_1,\ \hat{\rho}_0^{conf}=\frac{|\{x\in N|s=1\}|}{|N|}=\rho_0
\end{align*}

\subsubsection{Proof of Lemma 3}
We rewrite Lemma 3 below:

\textbf{Lemma \customlabel{thm:lemma3_appendix}{3}}
\textit{When $g$ is unassuming, i.e., $\Delta g(x):=g(x)-g^*(x)$ can be nonzero, and $P$ and $N$ can have overlapping support, we have}

\vskip -0.2in

\begin{equation}\label{rho_imperfect_condition_appendix}
\begin{cases}
LB_{y=1}=LB_{y=1}^*+E_{x\in\tilde{P}}[\Delta g(x)]-\frac{(1-\rho_1-\rho_0)^2}{p_{s1}}\Delta p_o\\
UB_{y=0}=UB_{y=0}^*+E_{x\in\tilde{N}}[\Delta g(x)]+\frac{(1-\rho_1-\rho_0)^2}{1-p_{s1}}\Delta p_o\\
\hat{\rho}_1^{conf}=\rho_1+\frac{1-\rho_1-\rho_0}{|P|-|\Delta P_1| + |\Delta N_1|}|\Delta N_1|\\
\hat{\rho}_0^{conf}=\rho_0+\frac{1-\rho_1-\rho_0}{|N|-|\Delta N_0| + |\Delta P_0|}|\Delta P_0|\\
\end{cases}
\end{equation}

\vskip -0.1in
\textit{where}
\vskip -0.1in

\begin{equation}\label{def_all_delta_P_N}
\begin{cases}
LB_{y=1}^*=(1-\rho_1)(1-\pi_1)+\rho_0\pi_1\\ UB_{y=0}^*=(1-\rho_1)\pi_0+\rho_0(1-\pi_0)\\
\Delta p_o:=\frac{|P \cap N|}{|P \cup N|}\\
\Delta P_1=\{x\in P|g(x)< LB_{y=1}\}\\
\Delta N_1=\{x\in N|g(x)\geq LB_{y=1}\}\\
\Delta P_0=\{x\in P|g(x)\leq UB_{y=0}\}\\
\Delta N_0=\{x\in N|g(x)> UB_{y=0}\}\\
\end{cases}
\end{equation}

\medskip

\noindent \textbf{Proof:} We first calculate $LB_{y=1}$ and $UB_{y=0}$ under unassuming conditions, then calculate $\hat{\rho}_i^{conf}$, $i=0,1$ under unassuming condition.

Note that $\Delta p_o$ can also be expressed as
\begin{align*}
\Delta p_o:=\frac{|P \cap N|}{|P \cup N|}=P(\hat{y}=1,y=0)=P(\hat{y}=0,y=1)
\end{align*}

Here $P(\hat{y}=1,y=0)\equiv P(\hat{y}=1|y=0)P(y=0)$, where $P(\hat{y}=1|y=0)$ means for a perfect classifier $f^*(x)=P(y=1|x)$, the expected probability that it will label a $y=0$ example as positive ($\hat{y}=1$).

\medskip

\textbf{(1) $LB_{y=1}$ and $UB_{y=0}$ under unassuming condition}

Firstly, we calculate $LB_{y=1}$ and $UB_{y=0}$ with perfect probability estimation $g^*(x)$, but the support may overlap. Secondly, we allow the probability estimation to be imperfect, superimposed onto the overlapping support condition, and calculate $LB_{y=1}$ and $UB_{y=0}$.

\textbf{I. Calculating $LB_{y=1}$ and $UB_{y=0}$ when $g(x)=g^*(x)$ and support may overlap}

With overlapping support, we no longer have $P(y=1|x)=\indicator{y=1}$. Instead, we have

\vskip -0.2in

\begin{align*}
LB_{y=1}=&E_{x\in\tilde{P}}[g^*(x)]=E_{x\in\tilde{P}}[P(s=1|x)]\\
=&E_{x\in\tilde{P}}[P(s=1|x,y=1)P(y=1|x) +P(s=1|x,y=0)P(y=0|x)]\\
=&E_{x\in\tilde{P}}[P(s=1|y=1)P(y=1|x) +P(s=1|y=0)P(y=0|x)]\\
=&(1-\rho_1)\cdot E_{x\in\tilde{P}}[P(y=1|x)]+\rho_0\cdot E_{x\in\tilde{P}}[P(y=0|x)]\\
=&(1-\rho_1)\cdot P(\hat{y}=1|s=1)+\rho_0\cdot P(\hat{y}=0|s=1)\\
\end{align*}

\vskip -0.2in

Here $P(\hat{y}=1|s=1)$ can be calculated using $\Delta p_o$:

\vskip -0.15in

\begin{align*}
P(\hat{y}=1|s=1)&=\frac{P(\hat{y}=1,s=1)}{P(s=1)}\\
&=\frac{P(\hat{y}=1,y=1,s=1)+P(\hat{y}=1,y=0,s=1)}{P(s=1)}\\
&=\frac{P(s=1|y=1)P(\hat{y}=1,y=1)+P(s=1|y=0)P(\hat{y}=1,y=0)}{P(s=1)}\\
&=\frac{(1-\rho_1)(p_{y1}-\Delta p_o)+\rho_0\Delta p_o}{p_{s1}}\\
&=(1-\pi_1)-\frac{1-\rho_1-\rho_0}{p_{s1}}\Delta p_o
\end{align*}

Hence,
\begin{align*}
P(\hat{y}=0|s=1)=1-P(\hat{y}=1|s=1)=\pi_1+\frac{1-\rho_1-\rho_0}{p_{s1}}\Delta p_o
\end{align*}

Therefore,
\begin{align*}
LB_{y=1}&=(1-\rho_1)\cdot P(\hat{y}=1|s=1)+\rho_0\cdot P(\hat{y}=0|s=1)\\
&=(1-\rho_1)\cdot \left((1-\pi_1)-\frac{1-\rho_1-\rho_0}{p_{s1}}\Delta p_o\right)+\rho_0\cdot \left(\pi_1+\frac{1-\rho_1-\rho_0}{p_{s1}}\Delta p_o\right)\\
&=LB_{y=1}^*-\frac{(1-\rho_1-\rho_0)^2}{p_{s1}}\Delta p_o\numberthis \label{rho1_prob_overlap}
\end{align*}

where $LB_{y=1}^*$ is the $LB_{y=1}$ value when $g(x)$ is ideal. We see in Eq. (\ref{rho1_prob_overlap}) that the overlapping support introduces a non-positive correction to $LB_{y=1}^*$ compared with the ideal condition.

Similarly, we have
\begin{align}\label{rho0_prob_overlap}
UB_{y=0}=UB_{y=0}^*+\frac{(1-\rho_1-\rho_0)^2}{1-p_{s1}}\Delta p_o
\end{align}

\medskip

\textbf{II. Calculating $LB_{y=1}$ and $UB_{y=0}$ when $g$ is unassuming}

Define $\Delta g(x)=g(x)-g^*(x)$. When the support may overlap, we have

\begin{align*}
LB_{y=1}&=E_{x\in\tilde{P}}[g(x)]\\
&=E_{x\in\tilde{P}}[g^*(x)]+E_{x\in\tilde{P}}[\Delta g(x)]\\
&=LB_{y=1}^*-\frac{(1-\rho_1-\rho_0)^2}{p_{s1}}\Delta p_o+E_{x\in\tilde{P}}[\Delta g(x)]\numberthis\label{rho1_nonideal}
\end{align*}

Similarly, we have
\begin{align*}
UB_{y=0}&=E_{x\in\tilde{N}}[g(x)]\\
&=E_{x\in\tilde{N}}[g^*(x)]+E_{x\in\tilde{N}}[\Delta g(x)]\\
&=UB_{y=0}^*+\frac{(1-\rho_1-\rho_0)^2}{1-p_{s1}}\Delta p_o+E_{x\in\tilde{N}}[\Delta g(x)]\numberthis\label{rho0_nonideal}
\end{align*}

In summary, Eq. (\ref{rho1_nonideal}) (\ref{rho0_nonideal}) give the expressions for $LB_{y=1}$ and $UB_{y=0}$, respectively, when $g$ is unassuming.

\medskip

\textbf{(2) $\hat{\rho}_i^{conf}$ under unassuming condition}

Now let's calculate $\hat{\rho}_i^{conf}$, $i=0,1$. For simplicity, define
\begin{equation}\label{define_all_Delta}
\begin{cases}
PP=\{x\in P|s=1\}\\
PN=\{x\in P|s=0\}\\
NP=\{x\in N|s=1\}\\
NN=\{x\in N|s=0\}\\
\Delta_{PP_1}=\{x\in PP|g(x)< LB_{y=1}\}\\
\Delta_{NP_1}=\{x\in NP|g(x)\geq LB_{y=1}\}\\
\Delta_{PN_1}=\{x\in PN|g(x)< LB_{y=1}\}\\
\Delta_{NN_1}=\{x\in NN|g(x)\geq LB_{y=1}\}\\
\end{cases}
\end{equation}

For $\hat{\rho}_1^{conf}$, we have:
\begin{align*}
\hat{\rho}_1^{conf}=\frac{|\tilde{N}_{y=1}|}{|\tilde{P}_{y=1}| + |\tilde{N}_{y=1}|}
\end{align*}

Here
\begin{align*}
\tilde{P}_{y=1}&=\{x\in\tilde{P}|g(x)\geq LB_{y=1}\}\\
&=\{x\in PP|g(x)\geq LB_{y=1}\}\cup\{x\in NP|g(x)\geq LB_{y=1}\}\\
&=(PP\setminus \Delta_{PP_1})\cup\Delta_{NP_1}
\end{align*}

Similarly, we have
\begin{align*}
\tilde{N}_{y=1}=(PN\setminus\Delta_{PN_1})\cup\Delta_{NN_1}
\end{align*}

Therefore

\begin{align*}
\hat{\rho}_1^{conf}&=\frac{|PN|-|\Delta_{PN_1}|+|\Delta_{NN_1}|}{[(|PP|-|\Delta_{PP_1}|)+(|PN|-|\Delta_{PN_1}|)]+(|\Delta_{NN_1}|+|\Delta_{NP_1}|)}\\
&=\frac{|PN|-|\Delta_{PN_1}|+|\Delta_{NN_1}|}{|P|-|\Delta P_1|+|\Delta N_1|}\numberthis\label{rho_1_conf_general}
\end{align*}

where in the second equality we have used the definition of $\Delta P_1$ and $\Delta N_1$ in Eq. (\ref{def_all_delta_P_N}).

Using the definition of $\rho_1$, we have

\begin{align*}
\frac{|PN|-|\Delta_{PN_1}|}{|P|-|\Delta P_1|}&=\frac{|\{x\in PN|g(x)\geq LB_{y=1}\}|}{|\{x\in P|g(x)\geq  LB_{y=1}\}|}\\
&=\frac{P(x\in PN,g(x)\geq LB_{y=1})}{P(x\in P,g(x)\geq LB_{y=1})}\\
&=\frac{P(x\in PN|x\in P,g(x)\geq LB_{y=1})\cdot P(x\in P,g(x)\geq LB_{y=1})}{P(x\in P,g(x)\geq LB_{y=1})}\\
&=\frac{P(x\in PN|x\in P)\cdot P(x\in P,g(x)\geq LB_{y=1})}{P(x\in P,g(x)\geq LB_{y=1})}\\
&=\rho_1
\end{align*}
Here we have used the property of CNP that $(s \independent x) | y$, leading to $P(x\in PN|x\in P,g(x)\geq LB_{y=1})=P(x\in PN|x\in P)=\rho_1$.

Similarly, we have
\begin{align*}
\frac{|\Delta_{NN_1}|}{|\Delta N_1|}&= 1-\rho_0
\end{align*}

Combining with Eq. (\ref{rho_1_conf_general}), we have

\begin{align*}
\hat{\rho}_1^{conf}=\rho_1+ \frac{1-\rho_1-\rho_0}{|P|-|\Delta P_1| + |\Delta N_1|}|\Delta N_1|\numberthis\label{rho_1_conf_unassuming}
\end{align*}

Similarly, we have

\begin{align}\label{rho_0_conf_unassuming}
\hat{\rho}_0^{conf}=\rho_0+ \frac{1-\rho_1-\rho_0}{|N|-|\Delta N_0| + |\Delta P_0|}|\Delta P_0|
\end{align}

From the two equations above, we see that
\begin{equation}
\hat{\rho}_1^{conf}\geq \rho_1, \hat{\rho}_0^{conf}\geq \rho_0
\end{equation}

In other words, $\hat{\rho}_i^{conf}$ is an \textbf{\textit{upper bound}} of $\rho_i$, $i=0,1$. The equality for $\hat{\rho}_1^{conf}$ holds if $|\Delta N_1|=0$. The equality for $\hat{\rho}_0^{conf}$ holds if $|\Delta P_0|=0$.

\medskip

\subsubsection{Proof of Theorem 4}
Let's restate Theorem 4 below:

\textbf{Theorem \customlabel{thm:theorem4_appendix}{4}}
\textit{Given non-overlapping support condition,}

If $\forall x\in N, \Delta g(x)<LB_{y=1}-\rho_0$, then $\hat{\rho}_1^{conf}=\rho_1$.

If $\forall x\in P, \Delta g(x)>-(1-\rho_1-UB_{y=0}$), then $\hat{\rho}_0^{conf}=\rho_0$.

\medskip

Theorem 4 directly follows from Eq. (\ref{rho_1_conf_unassuming}) and (\ref{rho_0_conf_unassuming}). Assuming non-overlapping support, we have $g^*(x)=P(s=1|x)=(1-\rho_1)\cdot \indicator{y=1}+\rho_0\cdot \indicator{y=0}$. In other words, the contribution of overlapping support to $|\Delta N_1|$ and $|\Delta P_0|$ is 0.  The only source of deviation comes from imperfect $g(x)$.

For the first half of the theorem, since $\forall x\in N, \Delta g(x)<LB_{y=1}-\rho_0$, we have $\forall x\in N, g(x)=\Delta g(x)+g^*(x)<(LB_{y=1}-\rho_0)+\rho_0=LB_{y=1}$, then $|\Delta N_1|=|\{x\in N|g(x)\geq LB_{y=1}\}|=0$, so we have $\hat{\rho}_1^{conf}=\rho_1$. 

Similarly, for the second half of the theorem, since $\forall x\in P,  \Delta g(x)>-(1-\rho_1-UB_{y=0})$, then $|\Delta P_0|=|\{x\in P|g(x)\leq UB_{y=0}\}|=0$, so we have $\hat{\rho}_0^{conf}=\rho_0$.

\subsubsection{Proof of Theorem 5}\label{sec:A5}
Theorem 5 reads as follows:

\textbf{Theorem \customlabel{thm:theorem5_appendix}{5}}
\textit{If $g$ range separates $P$ and $N$ and $\hat{\rho}_i=\rho_i$, $i=0,1$, then for any classifier $f_{\theta}$ and any bounded loss function $l(\hat{y}_i,y_i)$, we have}

\begin{equation}
R_{\tilde{l},\mathcal{D}_{\rho}}(f_{\theta})=R_{l,\mathcal{D}}(f_{\theta})
\end{equation}

\textit{where $\tilde{l}(\hat{y}_i,s_i)$ is Rank Pruning's loss function given by
}

\begin{align*}\label{rankpruning_loss_function_A}
\tilde{l}(\hat{y_i}, s_i)=\frac{1}{1-\hat{\rho}_1}l(\hat{y_i}, s_i)\cdot\indicator{x_i\in \tilde{P}_{conf}}
+&\frac{1}{1-\hat{\rho}_0}l(\hat{y_i}, s_i)\cdot\indicator{x_i\in \tilde{N}_{conf}}\numberthis  
\end{align*}

\textit{and $\tilde{P}_{conf}$ and $\tilde{N}_{conf}$ are given by}

\begin{equation}\label{define_P_N_conf}
\tilde{P}_{conf} := \{x \in \tilde{P} \mid g(x) \geq k_1 \},
\tilde{N}_{conf} := \{x \in \tilde{N} \mid g(x) \leq k_0 \}
\end{equation}

\textit{where $k_1$ is the $(\hat{\pi}_1|\tilde{P}|)^{th}$ smallest $g(x)$ for $x \in \tilde{P}$ and $k_0$ is the $(\hat{\pi}_0|\tilde{N}|)^{th}$ largest $g(x)$ for $x \in \tilde{N}$}

\vskip 0.1in
\textbf{Proof:}

Since $\tilde{P}$ and $\tilde{N}$ are constructed from $P$ and $N$ with noise rates $\pi_1$ and $\pi_0$ using the class-conditional extension of the Classification Noise Process \citep{angluin1988learning}, we have

\begin{equation}
\begin{cases}
P = PP\cup PN\\
N = NP\cup NN\\
\tilde{P} = PP\cup NP\\
\tilde{N} = PN\cup NN\\
\end{cases}
\end{equation}

where
\begin{equation}
\begin{cases}
PP=\{x\in P|s=1\}\\
PN=\{x\in P|s=0\}\\
NP=\{x\in N|s=1\}\\
NN=\{x\in N|s=0\}
\end{cases}
\end{equation}

satisfying
\begin{equation}
\begin{cases}
PP \sim PN \sim P\\
NP \sim NN \sim N\\
\frac{|NP|}{|\tilde{P}|}=\pi_1,\frac{|PP|}{|\tilde{P}|}=1-\pi_1\\
\frac{|PN|}{|\tilde{N}|}=\pi_0,\frac{|NN|}{|\tilde{N}|}=1-\pi_0\\
\frac{|PN|}{|P|}=\rho_1,\frac{|PP|}{|P|}=1-\rho_1\\
\frac{|NP|}{|N|}=\rho_0,\frac{|NN|}{|N|}=1-\rho_0\\
\end{cases}
\end{equation}

Here the $\sim$ means obeying the same distribution.

Since $g$ range separates $P$ and $N$, there exists a real number $z$ such that $\forall x_1\in P$ and $\forall x_0 \in N$, we have $g(x_1)>z>g(x_0)$. Since $P=PP\cup PN$, $N=NP\cup NN$, we have 
\begin{align*}
\forall x\in PP, g(x)>z;\ \forall x\in PN, g(x)>z; \\
\forall x\in NP, g(x)<z;\ \forall x\in NN, g(x)<z\numberthis\label{g_range_separate_PNNP}
\end{align*}

Since $\hat{\rho}_1=\rho_1$ and $\hat{\rho}_0=\rho_0$, we have 

\begin{equation}
\begin{cases}
\hat{\pi}_1=\frac{\hat{\rho}_0}{p_{s1}}\frac{1-p_{s1}-\hat{\rho}_1}{1-\hat{\rho}_1-\hat{\rho}_0}=\frac{\rho_0}{p_{s1}}\frac{1-p_{s1}-\rho_1}{1-\rho_1-\rho_0}=\pi_1\equiv\frac{\rho_0|N|}{|\tilde{P}|}\\
\hat{\pi}_0=\frac{\hat{\rho}_1}{1-p_{s1}}\frac{p_{s1}-\hat{\rho}_0}{1-\hat{\rho}_1-\hat{\rho}_0}=\frac{\rho_1}{1-p_{s1}}\frac{p_{s1}-\rho_0}{1-\rho_1-\rho_0}=\pi_0\equiv\frac{\rho_1|P|}{|\tilde{N}|}
\end{cases}
\end{equation}

Therefore, $\hat{\pi}_1|\tilde{P}|=\pi_1|\tilde{P}|=\rho_0|N|$, $\hat{\pi}_0|\tilde{N}|=\pi_0|\tilde{N}|=\rho_1|P|$. Using $\tilde{P}_{conf}$ and $\tilde{N}_{conf}$'s definition in Eq. (\ref{define_P_N_conf}), and $g(x)$'s property in Eq. (\ref{g_range_separate_PNNP}), we have

\begin{equation}
\tilde{P}_{conf} = PP\sim P, \tilde{N}_{conf} = NN \sim N
\end{equation}

Hence $P_{conf}$ and $N_{conf}$ can be seen as a uniform downsampling of $P$ and $N$, with a downsampling ratio of $(1-\rho_1)$ for $P$ and $(1-\rho_0)$ for $N$. Then according to Eq. (\ref{rankpruning_loss_function_A}), the loss function $\tilde{l}(\hat{y}_i,s_i)$ essentially sees a fraction of $(1-\rho_1)$ examples in $P$ and a fraction of $(1-\rho_0)$ examples in $N$, with a final reweighting to restore the class balance. Then for any classifier $f_{\theta}$ that maps $x\to \hat{y}$ and any bounded loss function $l(\hat{y}_i,y_i)$, we have
\begin{align*}
R_{\tilde{l},\mathcal{D}_{\rho}}(f_{\theta})=&E_{(x,s)\sim \mathcal{D}_{\rho}}[\tilde{l}(f_{\theta}(x),s)]\\
=&\frac{1}{1-\hat{\rho}_1}\cdot E_{(x,s)\sim \mathcal{D}_{\rho}}\left[l(f_{\theta}(x),s)\cdot\indicator{x\in \tilde{P}_{conf}}\right]\\
&+\frac{1}{1-\hat{\rho}_0}\cdot E_{(x,s)\sim \mathcal{D}_{\rho}}\left[l(f_{\theta}(x),s)\cdot\indicator{x\in \tilde{N}_{conf}}\right]\\
=&\frac{1}{1-\rho_1}\cdot E_{(x,s)\sim \mathcal{D}_{\rho}}\left[l(f_{\theta}(x),s)\cdot\indicator{x\in \tilde{P}_{conf}}\right]\\
&+\frac{1}{1-\rho_0}\cdot E_{(x,s)\sim \mathcal{D}_{\rho}}\left[l(f_{\theta}(x),s)\cdot\indicator{x\in \tilde{N}_{conf}}\right]\\
=&\frac{1}{1-\rho_1}\cdot E_{(x,s)\sim \mathcal{D}_{\rho}}\left[l(f_{\theta}(x),s)\cdot\indicator{x\in PP}\right]\\
&+\frac{1}{1-\rho_0}\cdot E_{(x,s)\sim \mathcal{D}_{\rho}}\left[l(f_{\theta}(x),s)\cdot\indicator{x\in NN}\right]\\
=&\frac{1}{1-\rho_1}\cdot (1-\rho_1)\cdot E_{(x,y)\sim \mathcal{D}}\left[l(f_{\theta}(x),y)\cdot\indicator{x\in P}\right]\\
&+\frac{1}{1-\rho_0}\cdot (1-\rho_0)\cdot E_{(x,y)\sim \mathcal{D}}\left[l(f_{\theta}(x),y)\cdot\indicator{x\in N}\right]\\
=&E_{(x,y)\sim \mathcal{D}}\left[l(f_{\theta}(x),y)\cdot\indicator{x\in P}+l(f_{\theta}(x),y)\cdot\indicator{x\in N}\right]\\
=&E_{(x,y)\sim \mathcal{D}}\left[l(f_{\theta}(x),y)\right]\\
=&R_{l,\mathcal{D}}(f_{\theta})
\end{align*}

Therefore, we see that the expected risk for Rank Pruning with corrupted labels, is exactly the same as the expected risk for the true labels, for any bounded loss function $l$ and classifier $f_{\theta}$. The reweighting ensures that after pruning, the two sets still remain unbiased w.r.t. to the true dataset.

Since the ideal condition is more strict than the range separability condition, we immediately have that when $g$ is ideal and $\hat{\rho}_i=\rho_i$, $i=0,1$, $R_{\tilde{l},\mathcal{D}_{\rho}}(f_{\theta})=R_{l,\mathcal{D}}(f_{\theta})$ for any $f_{\theta}$ and bounded loss function $l$.

\subsection{Additional Figures}

Figure B\ref{rankprune_mean_mnist} shows the average image for each digit for the problem “1” or “not 1” in MNIST with logistic regression and high noise ($\rho_1=0.5, \pi_1=0.5$). The number on the bottom and on the right counts the total number of examples (images). From the figure we see that Rank Pruning makes few mistakes, and when it does, the mistakes vary greatly in image from the typical digit.

\subsection{Additional Tables}\label{sec:tables_appendix}

Here we provide additional tables for the comparison of error, Precision-Recall AUC (AUC-PR, \cite{Davis:2006:RPR:1143844.1143874}), and F1 score for the algorithms \emph{RP}, \emph{Nat13}, \emph{Elk08}, \emph{Liu16} with $\rho_1$, $\rho_0$ given to all methods for fair comparison. Additionally, we provide the performance of the ground truth classifier (\emph{true}) trained with uncorrupted labels $(X, y)$, as well as the complete Rank Pruning algorithm (\emph{$\text{RP}_{\rho}$}) trained using the noise rates estimated by Rank Pruning. The top model scores are in bold with $RP_{\rho}$ in red if its performance is better than non-RP models. The $\pi_1=0$ quadrant in each table represents the ``PU learning" case of $\tilde{P}\tilde{N}$ learning. 

\begin{suppfigure}[t]
\begin{center}
\centerline{\includegraphics[width=0.9\columnwidth]{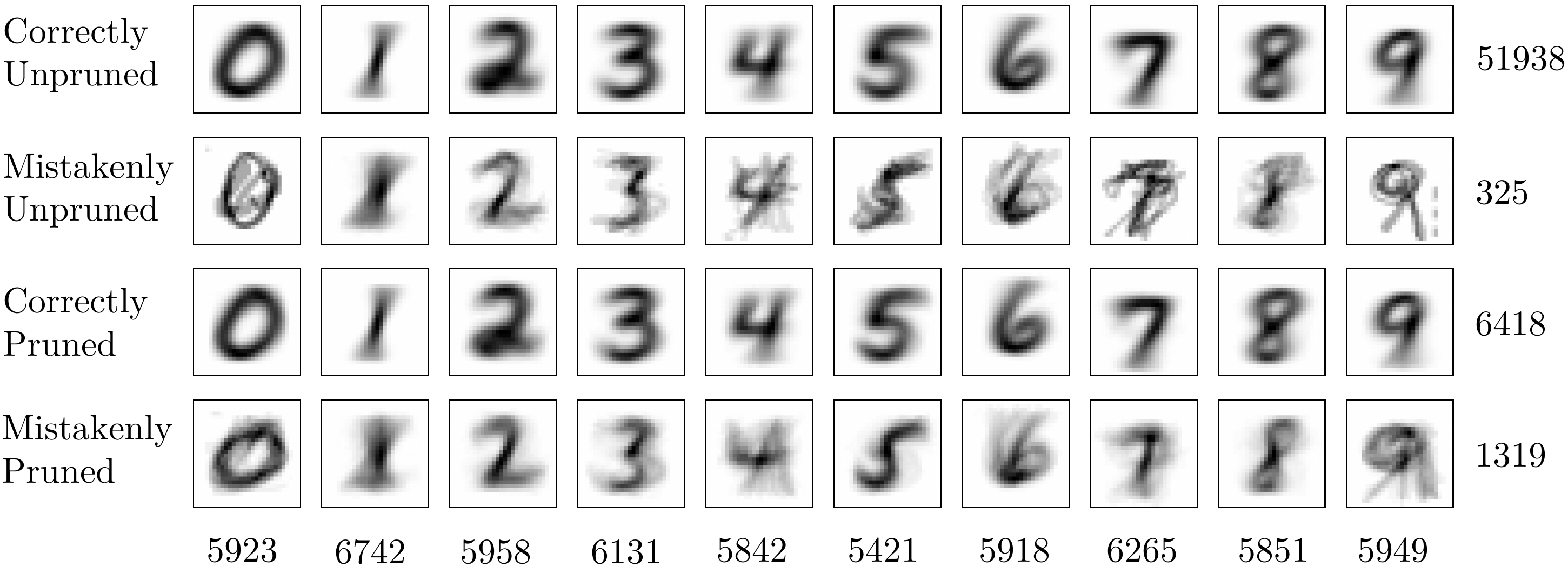}}
\caption{Average image for each digit for the binary classification problem ``1" or ``not 1" in MNIST with logistic regression and significant mislabeling ($\rho_1=0.5, \pi_1=0.5$). The right and bottom numbers count the total number of example images averaged in the corresponding row or column.}
\label{rankprune_mean_mnist}
\end{center}
\end{suppfigure}

Whenever $g(x)=P(\hat{s}=1|x)$ is estimated for any algorithm, we use a 3-fold cross-validation to estimate the probability $g(x)$. For improved performance, a higher fold may be used.

For the logistic regression classifier, we use scikit-learn's LogisticRegression class (\cite{logreg_sklearn}) with default settings (L2 regularization with inverse strength $C=1$).

For the convolutional neural networks (CNN), for MNIST we use the structure in \cite{mnist_cnn_structure} and for CIFAR-10, we use the structure in \cite{cifar_cnn_structure}. A $10\%$ holdout set is used to monitor the weighted validation loss (using the sample weight given by each algorithm) and ends training when there is no decrease for 10 epochs, with a maximum of 50 epochs for MNIST and 150 epochs for CIFAR-10. 

\vskip 0.1in

The following list comprises the MNIST and CIFAR experimental result tables for error, AUC-PR and F1 score metrics:

Table C\ref{table:mnist_logreg_error}: Error for MNIST with logisitic regression as classifier.

Table C\ref{table:mnist_logreg_auc}: AUC-PR for MNIST with logisitic regression as classifier.

Table C\ref{table:mnist_cnn_error}: Error for MNIST with CNN as classifier.

Table C\ref{table:mnist_cnn_auc}: AUC-PR for MNIST with CNN as classifier.

Table C\ref{table:cifar_logreg_f1}: F1 score for CIFAR-10 with logistic regression as classifier.

Table C\ref{table:cifar_logreg_error}: Error for CIFAR-10 with logistic regression as classifier.

Table C\ref{table:cifar_logreg_auc}: AUC-PR for CIFAR-10 with logistic regression as classifier.

Table C\ref{table:cifar_cnn_error}: Error for CIFAR-10 with CNN as classifier.

Table C\ref{table:cifar_cnn_auc}: AUC-PR for CIFAR-10 with CNN as classifier.

\vskip 0.1in
Due to sensitivity to imperfect probability estimation, here \emph{Liu16} always predicts all labels to be positive or negative, resulting in the same metric score for every digit/image in each scenario. Since $p_{y1}\simeq 0.1$, when predicting all labels as positive, \emph{Liu16} has an F1 score of 0.182, error of 0.90, and AUC-PR of 0.55; when predicting all labels as negative, \emph{Liu16} has an F1 score of 0.0, error of 0.1, and AUC-PR of 0.55.

\subsection{Additional Related Work}

In this section we include tangentially related work which was unable to make it into the final manuscript.

\subsubsection{One-class classification}

One-class classification \citep{moya_1993_oneclass} is distinguished from binary classification by a training set containing examples from only one class, making it useful for outlier and novelty detection \citep{ Hempstalk:2008:oneclass}. This can be framed as $\tilde{P}\tilde{N}$ learning when outliers take the form of mislabeled examples. The predominant approach, one-class SVM, fits a hyper-boundary around the training class \citep{oneclasssvm1999}, but often performs poorly due to boundary over-sensitivity \citep{Manevitz:2002:OSD:944790.944808} and fails when the training class contains mislabeled examples.

\subsubsection{\texorpdfstring{$\tilde{P}\tilde{N}$}{PN} learning for Image Recognition and Deep Learning}

Variations of $\tilde{P}\tilde{N}$ learning have been used in the context of machine vision to improve robustness to mislabeling \citep{Xiao2015LearningFM}. In a face recognition task with 90\% of non-faces mislabeled as faces, a bagging model combined with consistency voting was used to remove images with poor voting consistency \citep{angelova2005pruning}. However, no theoretical justification was provided. In the context of deep learning, consistency of predictions for inputs with mislabeling enforces can be enforced by combining a typical cross-entropy loss with an auto-encoder loss \citep{noisy_boostrapping_google}. This method enforces label consistency by constraining the network to uncover the input examples given the output prediction, but is restricted in architecture and generality.

\vskip 0.3in

\begin{supptable*}[th]

\setlength\tabcolsep{2pt} 
\renewcommand{\arraystretch}{0.9}
\caption{Comparison of \textbf{error} for one-vs-rest MNIST (averaged over all digits) using a \textbf{logistic regression} classifier. Except for $RP_\rho$, $\rho_1$, $\rho_0$ are given to all methods. Top model scores are in bold with $RP_\rho$ in red if better (smaller) than non-RP models.} 
\vskip -0.1in
\label{table:mnist_logreg_error}
\begin{center}
\begin{small}
\begin{sc}

\resizebox{\textwidth}{!}{ 
\begin{tabular}{l|rrr|rrrr|rrrr|rrrr}
\toprule

\multicolumn{0}{c}{} & 
\multicolumn{3}{c}{$\pi_1$\textbf{ = 0}}         &   
\multicolumn{4}{c}{$\pi_1$\textbf{ = 0.25}}  & 
\multicolumn{4}{c}{$\pi_1$\textbf{ = 0.5}}  & 
\multicolumn{4}{c}{$\pi_1$\textbf{ = 0.75}}    \\

\textbf{Model,$\rho_1 = $} &      \textbf{0.25} &    \textbf{0.50} &   \textbf{0.75} &   \textbf{0.00} &    \textbf{0.25} &    \textbf{0.50} &   \textbf{0.75} &   \textbf{0.00} &    \textbf{0.25} &    \textbf{0.50} &   \textbf{0.75} &   \textbf{0.00} &    \textbf{0.25} &    \textbf{0.50} &   \textbf{0.75}   \\

\midrule

\textbf{True}   &  0.020 &  0.020 &  0.020 &  0.020 &  0.020 &  0.020 &  0.020 &  0.020 &  0.020 &  0.020 &  0.020 &  0.020 &  0.020 &  0.020 &  0.020 \\
\textbf{RP}$_{\rho}$    &  \textcolor{red}{\textbf{0.023}} &  \textcolor{red}{\textbf{0.025}} &  \textcolor{red}{\textbf{0.031}} &  \textcolor{red}{\textbf{0.024}} &  \textcolor{red}{\textbf{0.025}} &  \textcolor{red}{\textbf{0.027}} &  \textcolor{red}{\textbf{0.038}} &  0.040 &  0.037 &  0.039 &  0.049 &  0.140 &  0.128 &  0.133 &  0.151 \\
\textbf{RP}    &  \textbf{0.022} &  \textbf{0.025} &  \textbf{0.031} &  \textbf{0.021} &  \textbf{0.024} &  \textbf{0.027} &  \textbf{0.035} &  \textbf{0.023} &  \textbf{0.027} &  \textbf{0.031} &  \textbf{0.043} &  \textbf{0.028} &  \textbf{0.036} &  \textbf{0.045} &  0.069 \\
\textbf{Nat13} &  0.025 &  0.030 &  0.038 &  0.025 &  0.029 &  0.034 &  0.042 &  0.030 &  0.033 &  0.038 &  0.047 &  0.035 &  0.039 &  0.046 &  \textbf{0.067} \\
\textbf{Elk08} &  0.025 &  0.030 &  0.038 &  0.026 &  0.028 &  0.032 &  0.042 &  0.030 &  0.031 &  0.035 &  0.051 &  0.092 &  0.093 &  0.123 &  0.189 \\
\textbf{Liu16} &  0.187 &  0.098 &  0.100 &  0.100 &  0.738 &  0.738 &  0.419 &  0.100 &  0.820 &  0.821 &  0.821 &  0.098 &  0.760 &  0.741 &  0.820 \\

\bottomrule
\end{tabular}
}
\end{sc}
\end{small}
\end{center}
\vskip 0.06in
\end{supptable*}

\begin{supptable*}[ht]

\setlength\tabcolsep{2pt} 
\renewcommand{\arraystretch}{0.9}
\caption{Comparison of \textbf{AUC-PR} for one-vs-rest MNIST (averaged over all digits) using a \textbf{logistic regression} classifier. Except for $RP_\rho$, $\rho_1$, $\rho_0$ are given to all methods. Top model scores are in bold with $RP_\rho$ in red if greater than non-RP models.} 
\vskip -0.15in
\label{table:mnist_logreg_auc}
\begin{center}
\begin{small}
\begin{sc}

\resizebox{\textwidth}{!}{ 
\begin{tabular}{l|rrr|rrrr|rrrr|rrrr}
\toprule

\multicolumn{0}{c}{} & 
\multicolumn{3}{c}{$\pi_1$\textbf{ = 0}}         &   
\multicolumn{4}{c}{$\pi_1$\textbf{ = 0.25}}  & 
\multicolumn{4}{c}{$\pi_1$\textbf{ = 0.5}}  & 
\multicolumn{4}{c}{$\pi_1$\textbf{ = 0.75}}    \\

\textbf{Model,$\rho_1 = $} &      \textbf{0.25} &    \textbf{0.50} &   \textbf{0.75} &   \textbf{0.00} &    \textbf{0.25} &    \textbf{0.50} &   \textbf{0.75} &   \textbf{0.00} &    \textbf{0.25} &    \textbf{0.50} &   \textbf{0.75} &   \textbf{0.00} &    \textbf{0.25} &    \textbf{0.50} &   \textbf{0.75}   \\

\midrule

\textbf{True}   &  0.935 &  0.935 &  0.935 &  0.935 &  0.935 &  0.935 &  0.935 &  0.935 &  0.935 &  0.935 &  0.935 &  0.935 &  0.935 &  0.935 &  0.935 \\
\textbf{RP}$_{\rho}$ &  0.921 &  \textcolor{red}{\textbf{0.913}} &  \textcolor{red}{\textbf{0.882}} &  \textcolor{red}{\textbf{0.928}} &  \textcolor{red}{\textbf{0.920}} &  \textcolor{red}{\textbf{0.906}} &  \textcolor{red}{\textbf{0.853}} &  \textcolor{red}{\textbf{0.903}} &  \textcolor{red}{\textbf{0.902}} &  \textcolor{red}{\textbf{0.879}} &  \textcolor{red}{\textbf{0.803}} &  0.851 &  0.835 &  \textcolor{red}{\textbf{0.788}} &  0.640 \\
\textbf{RP}    &  \textbf{0.922} &  \textbf{0.913} &  \textbf{0.882} &  \textbf{0.930} &  \textbf{0.921} &  \textbf{0.906} &  \textbf{0.858} &  \textbf{0.922} &  \textbf{0.903} &  \textbf{0.883} &  \textbf{0.811} &  \textbf{0.893} &  \textbf{0.841} &  \textbf{0.799} &  0.621 \\
\textbf{Nat13} &  \textbf{0.922} &  0.908 &  0.878 &  0.918 &  0.909 &  0.890 &  0.839 &  0.899 &  0.892 &  0.862 &  0.794 &  0.863 &  0.837 &  0.784 &  \textbf{0.645} \\
\textbf{Elk08} &  0.921 &  0.903 &  0.864 &  0.917 &  0.908 &  0.884 &  0.821 &  0.898 &  0.892 &  0.861 &  0.763 &  0.852 &  0.837 &  0.772 &  0.579 \\
\textbf{Liu16} &  0.498 &  0.549 &  0.550 &  0.550 &  0.500 &  0.550 &  0.505 &  0.550 &  0.550 &  0.550 &  0.549 &  0.503 &  0.512 &  0.550 &  0.550 \\

\bottomrule
\end{tabular}
}
\end{sc}
\end{small}
\end{center}
\vskip 0.08in
\end{supptable*}

\begin{supptable*}[ht]

\setlength\tabcolsep{1pt} 
\renewcommand{\arraystretch}{0.85}
\caption{Comparison of \textbf{error} for one-vs-rest MNIST (averaged over all digits) using a \textbf{CNN} classifier. Except for $RP_\rho$, $\rho_1$, $\rho_0$ are given to all methods. Top model scores are in bold with $RP_\rho$ in red if better (smaller) than non-RP models.} 
\vskip -0.1in
\label{table:mnist_cnn_error}
\begin{center}
\begin{small}
\begin{sc}

\resizebox{\textwidth}{!}{
\begin{tabular}{l|c|ccccc|ccccc|ccccc|ccccc}
\toprule

\multicolumn{1}{c}{} &  
\multicolumn{1}{c|}{} & 
\multicolumn{5}{c|}{$\pi_1$\textbf{ = 0}}   &   
\multicolumn{5}{c|}{$\pi_1$\textbf{ = 0.25}}  & 
\multicolumn{10}{c}{$\pi_1$\textbf{ = 0.5}}   \\

\multicolumn{1}{c}{} &  
\multicolumn{1}{c|}{} & 
\multicolumn{5}{c|}{$\rho_1$\textbf{ = 0.5}} & 
\multicolumn{5}{c|}{$\rho_1$\textbf{ = 0.25}}  & 
\multicolumn{5}{c}{$\rho_1$\textbf{ = 0}} & 
\multicolumn{5}{c}{$\rho_1$\textbf{ = 0.5}}    \\

{\textbf{IMAGE}} &   \textbf{True} & \textbf{RP}$_{\rho}$ &    \textbf{RP} & \textbf{Nat13} & \textbf{Elk08} & \textbf{Liu16} & \textbf{RP}$_{\rho}$ &    \textbf{RP} & \textbf{Nat13} & \textbf{Elk08} & \textbf{Liu16} & \textbf{RP}$_{\rho}$ &    \textbf{RP} & \textbf{Nat13} & \textbf{Elk08} & \textbf{Liu16} & \textbf{RP}$_{\rho}$ &    \textbf{RP} & \textbf{Nat13} & \textbf{Elk08} & \textbf{Liu16} \\
\midrule

\textbf{0}     &  0.0013 &  \textcolor{red}{\textbf{0.0018}} &  \textbf{0.0023} &  0.0045 &  0.0047 &  0.9020 &  \textcolor{red}{\textbf{0.0017}} &  \textbf{0.0016} &  0.0034 &  0.0036 &  0.9020 &  \textcolor{red}{\textbf{0.0017}} &  \textbf{0.0016} &  0.0031 &  0.0026 &  0.0029 &  \textcolor{red}{\textbf{0.0021}} &  \textbf{0.0022} &  0.0116 &  0.0069 &  0.9020 \\
\textbf{1}     &  0.0015 &  \textcolor{red}{\textbf{0.0022}} &  \textbf{0.0020} &  0.0025 &  0.0034 &  0.8865 &  \textcolor{red}{\textbf{0.0019}} &  \textbf{0.0019} &  0.0035 &  0.0030 &  0.8865 &  0.0023 &  0.0020 &  0.0018 &  \textbf{0.0016} &  0.0023 &  \textcolor{red}{\textbf{0.0025}} &  \textbf{0.0025} &  0.0036 &  0.0027 &  0.8865 \\
\textbf{2}     &  0.0027 &  \textcolor{red}{\textbf{0.0054}} &  \textbf{0.0049} &  0.0057 &  0.0062 &  0.8968 &  \textcolor{red}{\textbf{0.0032}} &  \textbf{0.0035} &  0.0045 &  0.0051 &  0.8968 &  0.0030 &  0.0029 &  0.0031 &  0.0029 &  \textbf{0.0024} &  \textcolor{red}{\textbf{0.0059}} &  \textbf{0.0050} &  0.0066 &  0.0083 &  0.8968 \\
\textbf{3}     &  0.0020 &  \textcolor{red}{\textbf{0.0032}} &  \textbf{0.0032} &  0.0055 &  0.0038 &  0.8990 &  \textcolor{red}{\textbf{0.0029}} &  \textbf{0.0029} &  0.0043 &  0.0043 &  0.8990 &  \textcolor{red}{\textbf{0.0021}} &  0.0027 &  \textbf{0.0023} &  \textbf{0.0023} &  0.0032 &  \textcolor{red}{\textbf{0.0038}} &  \textbf{0.0042} &  0.0084 &  0.0057 &  0.8990 \\
\textbf{4}     &  0.0012 &  \textcolor{red}{\textbf{0.0037}} &  0.0040 &  \textbf{0.0038} &  0.0044 &  0.9018 &  \textcolor{red}{\textbf{0.0029}} &  \textbf{0.0025} &  0.0055 &  0.0069 &  0.9018 &  0.0026 &  0.0020 &  \textbf{0.0019} &  0.0021 &  0.0030 &  \textcolor{red}{\textbf{0.0044}} &  \textbf{0.0035} &  0.0086 &  0.0077 &  0.9018 \\
\textbf{5}     &  0.0019 &  \textcolor{red}{\textbf{0.0032}} &  \textbf{0.0035} &  0.0039 &  0.0038 &  0.9108 &  \textcolor{red}{\textbf{0.0027}} &  \textbf{0.0031} &  0.0062 &  0.0060 &  0.9108 &  \textcolor{red}{\textbf{0.0021}} &  0.0024 &  0.0024 &  0.0028 &  \textbf{0.0023} &  \textcolor{red}{\textbf{0.0061}} &  \textbf{0.0056} &  0.0066 &  0.0074 &  0.9108 \\
\textbf{6}     &  0.0021 &  \textcolor{red}{\textbf{0.0027}} &  \textbf{0.0028} &  0.0053 &  0.0035 &  0.9042 &  \textcolor{red}{\textbf{0.0028}} &  \textbf{0.0025} &  0.0042 &  0.0036 &  0.9042 &  0.0029 &  0.0029 &  \textbf{0.0022} &  0.0024 &  0.0028 &  \textcolor{red}{\textbf{0.0032}} &  \textbf{0.0035} &  0.0098 &  0.0075 &  0.9042 \\
\textbf{7}     &  0.0026 &  \textcolor{red}{\textbf{0.0039}} &  \textbf{0.0041} &  0.0066 &  0.0103 &  0.8972 &  \textcolor{red}{\textbf{0.0050}} &  \textbf{0.0052} &  0.0058 &  0.0058 &  0.8972 &  0.0049 &  0.0040 &  \textbf{0.0030} &  0.0037 &  0.0035 &  \textcolor{red}{\textbf{0.0054}} &  \textbf{0.0064} &  0.0113 &  0.0085 &  0.8972 \\
\textbf{8}     &  0.0022 &  \textcolor{red}{\textbf{0.0047}} &  \textbf{0.0043} &  0.0106 &  0.0063 &  0.9026 &  \textcolor{red}{\textbf{0.0034}} &  \textbf{0.0036} &  0.0062 &  0.0091 &  0.9026 &  0.0036 &  \textbf{0.0030} &  0.0035 &  0.0041 &  0.0032 &  \textcolor{red}{\textbf{0.0044}} &  \textbf{0.0048} &  0.0234 &  0.0077 &  0.9026 \\
\textbf{9}     &  0.0036 &  0.0067 &  \textbf{0.0052} &  0.0056 &  0.0124 &  0.8991 &  \textcolor{red}{\textbf{0.0048}} &  \textbf{0.0051} &  0.0065 &  0.0064 &  0.8991 &  0.0048 &  0.0050 &  0.0051 &  \textbf{0.0043} &  0.0059 &  \textcolor{red}{\textbf{0.0081}} &  0.0114 &  0.0131 &  \textbf{0.0112} &  0.8991 \\
\midrule
\textbf{AVG} &  0.0021 &  \textcolor{red}{\textbf{0.0038}} &  \textbf{0.0036} &  0.0054 &  0.0059 &  0.9000 &  \textcolor{red}{\textbf{0.0031}} &  \textbf{0.0032} &  0.0050 &  0.0054 &  0.9000 &  0.0030 &  \textbf{0.0028} &  \textbf{0.0028} &  0.0029 &  0.0032 &  \textcolor{red}{\textbf{0.0046}} &  \textbf{0.0049} &  0.0103 &  0.0074 &  0.9000 \\

\bottomrule
\end{tabular}
}
\end{sc}
\end{small}
\end{center}
\end{supptable*}

\begin{supptable*}[ht]

\setlength\tabcolsep{1pt} 
\renewcommand{\arraystretch}{0.85}
\caption{Comparison of \textbf{AUC-PR} for one-vs-rest MNIST (averaged over all digits) using a \textbf{CNN} classifier. Except for $RP_\rho$, $\rho_1$, $\rho_0$ are given to all methods. Top model scores are in bold with $RP_\rho$ in red if greater than non-RP models.} 
\vskip -0.1in
\label{table:mnist_cnn_auc}
\begin{center}
\begin{small}
\begin{sc}

\resizebox{\textwidth}{!}{
\begin{tabular}{l|c|ccccc|ccccc|ccccc|ccccc}
\toprule

\multicolumn{1}{c}{} &  
\multicolumn{1}{c|}{} & 
\multicolumn{5}{c|}{$\pi_1$\textbf{ = 0}}   &   
\multicolumn{5}{c|}{$\pi_1$\textbf{ = 0.25}}  & 
\multicolumn{10}{c}{$\pi_1$\textbf{ = 0.5}}   \\

\multicolumn{1}{c}{} &  
\multicolumn{1}{c|}{} & 
\multicolumn{5}{c|}{$\rho_1$\textbf{ = 0.5}} & 
\multicolumn{5}{c|}{$\rho_1$\textbf{ = 0.25}}  & 
\multicolumn{5}{c}{$\rho_1$\textbf{ = 0}} & 
\multicolumn{5}{c}{$\rho_1$\textbf{ = 0.5}}    \\

{\textbf{IMAGE}} &   \textbf{True} & \textbf{RP}$_{\rho}$ &    \textbf{RP} & \textbf{Nat13} & \textbf{Elk08} & \textbf{Liu16} & \textbf{RP}$_{\rho}$ &    \textbf{RP} & \textbf{Nat13} & \textbf{Elk08} & \textbf{Liu16} & \textbf{RP}$_{\rho}$ &    \textbf{RP} & \textbf{Nat13} & \textbf{Elk08} & \textbf{Liu16} & \textbf{RP}$_{\rho}$ &    \textbf{RP} & \textbf{Nat13} & \textbf{Elk08} & \textbf{Liu16} \\
\midrule
\textbf{0}     & 0.9998 & \textcolor{red}{\textbf{0.9992}} & \textbf{0.9990} & 0.9986 & 0.9982 & 0.5490 & \textcolor{red}{\textbf{0.9996}} & \textbf{0.9996} & 0.9986 & 0.9979 & 0.5490 & \textcolor{red}{\textbf{0.9989}} & \textbf{0.9995} & 0.9976 & 0.9979 & 0.9956 & \textcolor{red}{\textbf{0.9984}} & \textbf{0.9982} & 0.9963 & 0.9928 & 0.5490 \\
\textbf{1}     & 0.9999 & \textcolor{red}{\textbf{0.9995}} & \textbf{0.9995} & 0.9976 & 0.9974 & 0.5568 & \textcolor{red}{\textbf{0.9996}} & 0.9993 & \textbf{0.9995} & \textbf{0.9995} & 0.5568 & \textcolor{red}{\textbf{0.9995}} & \textbf{0.9998} & 0.9982 & 0.9972 & 0.9965 & \textcolor{red}{\textbf{0.9995}} & \textbf{0.9994} & 0.9978 & 0.9985 & 0.5568 \\
\textbf{2}     & 0.9994 & \textcolor{red}{\textbf{0.9971}} & \textbf{0.9969} & 0.9917 & 0.9942 & 0.5516 & \textcolor{red}{\textbf{0.9980}} & \textbf{0.9977} & 0.9971 & 0.9945 & 0.5516 & \textcolor{red}{\textbf{0.9988}} & \textbf{0.9992} & 0.9958 & 0.9934 & 0.9940 & \textcolor{red}{\textbf{0.9938}} & \textbf{0.9947} & 0.9847 & 0.9873 & 0.5516 \\
\textbf{3}     & 0.9996 & \textcolor{red}{\textbf{0.9986}} & \textbf{0.9987} & 0.9983 & 0.9984 & 0.5505 & \textcolor{red}{\textbf{0.9991}} & \textbf{0.9989} & 0.9982 & 0.9980 & 0.5505 & \textcolor{red}{\textbf{0.9993}} & \textbf{0.9994} & 0.9991 & 0.9971 & 0.9974 & \textcolor{red}{\textbf{0.9969}} & \textbf{0.9959} & 0.9951 & \textbf{0.9959} & 0.5505 \\
\textbf{4}     & 0.9997 & 0.9982 & \textbf{0.9989} & 0.9939 & 0.9988 & 0.0891 & \textcolor{red}{\textbf{0.9992}} & \textbf{0.9991} & 0.9976 & 0.9965 & 0.5491 & \textcolor{red}{\textbf{0.9994}} & \textbf{0.9996} & 0.9985 & 0.9978 & 0.9986 & \textcolor{red}{\textbf{0.9983}} & \textbf{0.9977} & 0.9961 & 0.9919 & 0.5491 \\
\textbf{5}     & 0.9993 & \textcolor{red}{\textbf{0.9982}} & \textbf{0.9976} & 0.9969 & 0.9956 & 0.5446 & \textcolor{red}{\textbf{0.9986}} & \textbf{0.9987} & 0.9983 & 0.9979 & 0.5446 & \textcolor{red}{\textbf{0.9984}} & \textbf{0.9982} & 0.9971 & 0.9963 & 0.9929 & \textcolor{red}{\textbf{0.9958}} & \textbf{0.9965} & 0.9946 & 0.9934 & 0.5446 \\
\textbf{6}     & 0.9987 & \textcolor{red}{\textbf{0.9976}} & \textbf{0.9970} & 0.9928 & 0.9931 & 0.5479 & \textcolor{red}{\textbf{0.9974}} & \textbf{0.9980} & 0.9956 & 0.9959 & 0.5479 & \textcolor{red}{\textbf{0.9968}} & \textbf{0.9983} & 0.9933 & 0.9950 & 0.9905 & \textcolor{red}{\textbf{0.9964}} & 0.9957 & 0.9942 & \textbf{0.9961} & 0.5479 \\
\textbf{7}     & 0.9989 & \textcolor{red}{\textbf{0.9973}} & \textbf{0.9972} & 0.9965 & 0.9944 & 0.0721 & 0.9968 & 0.9973 & 0.9966 & \textbf{0.9979} & 0.5514 & 0.9969 & \textbf{0.9983} & 0.9961 & 0.9958 & 0.9974 & \textcolor{red}{\textbf{0.9933}} & \textbf{0.9937} & 0.9896 & 0.9886 & 0.5514 \\
\textbf{8}     & 0.9996 & \textcolor{red}{\textbf{0.9974}} & \textbf{0.9964} & \textbf{0.9964} & 0.9946 & 0.5487 & \textcolor{red}{\textbf{0.9981}} & \textbf{0.9981} & 0.9973 & 0.9971 & 0.5487 & \textbf{0.9983} & 0.9988 & 0.9984 & 0.9976 & \textbf{0.9989} & \textcolor{red}{\textbf{0.9976}} & \textbf{0.9975} & 0.9873 & 0.9893 & 0.5487 \\
\textbf{9}     & 0.9979 & \textcolor{red}{\textbf{0.9931}} & \textbf{0.9951} & 0.9901 & 0.9922 & 0.5504 & \textcolor{red}{\textbf{0.9935}} & \textbf{0.9951} & 0.9933 & 0.9920 & 0.5504 & \textcolor{red}{\textbf{0.9961}} & \textbf{0.9951} & 0.9924 & 0.9922 & 0.9912 & \textcolor{red}{\textbf{0.9877}} & \textbf{0.9876} & 0.9819 & 0.9828 & 0.5504 \\
\midrule
\textbf{AVG} & 0.9993 & \textcolor{red}{\textbf{0.9976}} & \textbf{0.9976} & 0.9953 & 0.9957 & 0.4561 & \textcolor{red}{\textbf{0.9980}} & \textbf{0.9982} & 0.9972 & 0.9967 & 0.5500 & \textcolor{red}{\textbf{0.9983}} & \textbf{0.9986} & 0.9966 & 0.9960 & 0.9953 & \textcolor{red}{\textbf{0.9958}} & \textbf{0.9957} & 0.9918 & 0.9917 & 0.5500 \\

\bottomrule
\end{tabular}
}
\end{sc}
\end{small}
\end{center}
\vskip -0.08in
\end{supptable*}

\begin{supptable*}[ht]

\setlength\tabcolsep{1pt} 
\renewcommand{\arraystretch}{0.85}
\caption{Comparison of \textbf{F1 score} for one-vs-rest CIFAR-10 (averaged over all images) using a \textbf{logistic regression} classifier. Except for $RP_\rho$, $\rho_1$, $\rho_0$ are given to all methods. Top model scores are in bold with $RP_\rho$ in red if greater than non-RP models.} 
\vskip -0.1in
\label{table:cifar_logreg_f1}
\begin{center}
\begin{small}
\begin{sc}

\resizebox{\textwidth}{!}{
\begin{tabular}{l|c|ccccc|ccccc|ccccc|ccccc}
\toprule

\multicolumn{1}{c}{} &  
\multicolumn{1}{c|}{} & 
\multicolumn{5}{c|}{$\pi_1$\textbf{ = 0}}   &   
\multicolumn{5}{c|}{$\pi_1$\textbf{ = 0.25}}  & 
\multicolumn{10}{c}{$\pi_1$\textbf{ = 0.5}}   \\

\multicolumn{1}{c}{} &  
\multicolumn{1}{c|}{} & 
\multicolumn{5}{c|}{$\rho_1$\textbf{ = 0.5}} & 
\multicolumn{5}{c|}{$\rho_1$\textbf{ = 0.25}}  & 
\multicolumn{5}{c}{$\rho_1$\textbf{ = 0}} & 
\multicolumn{5}{c}{$\rho_1$\textbf{ = 0.5}}    \\

{\textbf{IMAGE}} &   \textbf{True} & \textbf{RP}$_{\rho}$ &    \textbf{RP} & \textbf{Nat13} & \textbf{Elk08} & \textbf{Liu16} & \textbf{RP}$_{\rho}$ &    \textbf{RP} & \textbf{Nat13} & \textbf{Elk08} & \textbf{Liu16} & \textbf{RP}$_{\rho}$ &    \textbf{RP} & \textbf{Nat13} & \textbf{Elk08} & \textbf{Liu16} & \textbf{RP}$_{\rho}$ &    \textbf{RP} & \textbf{Nat13} & \textbf{Elk08} & \textbf{Liu16} \\
\midrule
\textbf{plane}   &  0.272 &  \textcolor{red}{\textbf{0.311}} &  \textbf{0.252} &  0.217 &  0.220 &  0.182 &  \textcolor{red}{\textbf{0.329}} &  \textbf{0.275} &  0.222 &  0.224 &  0.182 &  \textcolor{red}{\textbf{0.330}} &  \textbf{0.265} &  0.231 &  0.259 &    0.0 &  \textcolor{red}{\textbf{0.266}} &  \textbf{0.188} &  0.183 &  0.187 &  0.182 \\
\textbf{auto} &  0.374 &  \textcolor{red}{\textbf{0.389}} &  \textbf{0.355} &  0.318 &  0.320 &  0.182 &  \textcolor{red}{\textbf{0.388}} &  \textbf{0.368} &  0.321 &  0.328 &  0.182 &  \textcolor{red}{\textbf{0.372}} &  \textbf{0.355} &  0.308 &  0.341 &    0.0 &  \textcolor{red}{\textbf{0.307}} &  0.287 &  0.287 &  \textbf{0.297} &  0.182 \\
\textbf{bird}       &  0.136 &  \textcolor{red}{\textbf{0.241}} &  0.167 &  0.143 &  0.136 &  \textbf{0.182} &  \textcolor{red}{\textbf{0.248}} &  \textbf{0.185} &  0.137 &  0.137 &  0.182 &  \textcolor{red}{\textbf{0.258}} &  \textbf{0.147} &  0.100 &  0.126 &    0.0 &  \textcolor{red}{\textbf{0.206}} &  0.153 &  0.132 &  0.150 &  \textbf{0.182} \\
\textbf{cat}        &  0.122 &  \textcolor{red}{\textbf{0.246}} &  0.170 &  0.141 &  0.150 &  \textbf{0.182} &  \textcolor{red}{\textbf{0.232}} &  0.163 &  0.112 &  0.127 &  \textbf{0.182} &  \textcolor{red}{\textbf{0.241}} &  \textbf{0.125} &  0.068 &  0.103 &    0.0 &  \textcolor{red}{\textbf{0.209}} &  0.148 &  0.119 &  0.157 &  \textbf{0.182} \\
\textbf{deer}       &  0.166 &  \textcolor{red}{\textbf{0.250}} &  \textbf{0.184} &  0.153 &  0.164 &  0.182 &  \textcolor{red}{\textbf{0.259}} &  0.175 &  0.146 &  0.163 &  \textbf{0.182} &  \textcolor{red}{\textbf{0.259}} &  \textbf{0.177} &  0.126 &  0.164 &    0.0 &  \textcolor{red}{\textbf{0.222}} &  0.162 &  0.132 &  0.164 &  \textbf{0.182} \\
\textbf{dog}        &  0.139 &  \textcolor{red}{\textbf{0.245}} &  0.174 &  0.146 &  0.148 &  \textbf{0.182} &  \textcolor{red}{\textbf{0.262}} &  0.171 &  0.115 &  0.126 &  \textbf{0.182} &  \textcolor{red}{\textbf{0.254}} &  \textbf{0.152} &  0.075 &  0.120 &    0.0 &  \textcolor{red}{\textbf{0.203}} &  0.151 &  0.128 &  0.137 &  \textbf{0.182} \\
\textbf{frog}       &  0.317 &  \textcolor{red}{\textbf{0.322}} &  \textbf{0.315} &  0.289 &  0.281 &  0.182 &  \textcolor{red}{\textbf{0.350}} &  \textbf{0.319} &  0.283 &  0.299 &  0.182 &  \textcolor{red}{\textbf{0.346}} &  \textbf{0.305} &  0.239 &  0.279 &    0.0 &  \textcolor{red}{\textbf{0.308}} &  0.252 &  0.244 &  \textbf{0.269} &  0.182 \\
\textbf{horse}      &  0.300 &  \textcolor{red}{\textbf{0.300}} &  \textbf{0.299} &  0.283 &  0.263 &  0.182 &  \textcolor{red}{\textbf{0.334}} &  \textbf{0.313} &  0.272 &  0.281 &  0.182 &  \textcolor{red}{\textbf{0.322}} &  \textbf{0.310} &  0.260 &  0.292 &    0.0 &  \textcolor{red}{\textbf{0.275}} &  \textbf{0.258} &  0.240 &  0.245 &  0.182 \\
\textbf{ship}       &  0.322 &  \textcolor{red}{\textbf{0.343}} &  \textbf{0.322} &  0.297 &  0.272 &  0.182 &  \textcolor{red}{\textbf{0.385}} &  \textbf{0.319} &  0.287 &  0.289 &  0.182 &  \textcolor{red}{\textbf{0.350}} &  \textbf{0.303} &  0.250 &  0.293 &    0.0 &  \textcolor{red}{\textbf{0.304}} &  \textbf{0.248} &  0.230 &  0.237 &  0.182 \\
\textbf{truck}      &  0.330 &  \textcolor{red}{\textbf{0.359}} &  \textbf{0.323} &  0.273 &  0.261 &  0.182 &  \textcolor{red}{\textbf{0.369}} &  \textbf{0.327} &  0.293 &  0.290 &  0.182 &  \textcolor{red}{\textbf{0.343}} &  \textbf{0.302} &  0.278 &  0.299 &    0.0 &  \textcolor{red}{\textbf{0.313}} &  0.246 &  0.252 &  \textbf{0.262} &  0.182 \\
\midrule
\textbf{AVG}      &  0.248 &  \textcolor{red}{\textbf{0.301}} &  \textbf{0.256} &  0.226 &  0.221 &  0.182 &  \textcolor{red}{\textbf{0.316}} &  \textbf{0.262} &  0.219 &  0.226 &  0.182 &  \textcolor{red}{\textbf{0.308}} &  \textbf{0.244} &  0.194 &  0.228 &   0.000 &  \textcolor{red}{\textbf{0.261}} &  0.209 &  0.195 &  \textbf{0.210} &  0.182 \\
\bottomrule
\end{tabular}
}
\end{sc}
\end{small}
\end{center}
\vskip -0.1in
\end{supptable*}

\begin{supptable*}[ht]

\setlength\tabcolsep{1pt} 
\renewcommand{\arraystretch}{0.85}
\caption{Comparison of \textbf{error} for one-vs-rest CIFAR-10 (averaged over all images) using a \textbf{logistic regression} classifier. Except for $RP_\rho$, $\rho_1$, $\rho_0$ are given to all methods. Top model scores are in bold with $RP_\rho$ in red if better (smaller) than non-RP models. Here the logistic regression classifier severely underfits CIFAR, resulting in Rank Pruning pruning out some correctly labeled examples that ``confuse" the classifier, hence in this scenario, RP and RP$_{\rho}$ generally have slightly smaller precision, much higher recall, and hence larger F1 scores than other models and even the ground truth classifier (Table C\ref{table:cifar_logreg_f1}). Due to the class inbalance ($p_{y1}=0.1$) and their larger recall, RP and RP$_{\rho}$ here have larger error than the other models.} 
\vskip -0.1in
\label{table:cifar_logreg_error}
\begin{center}
\begin{small}
\begin{sc}

\resizebox{\textwidth}{!}{
\begin{tabular}{l|c|ccccc|ccccc|ccccc|ccccc}
\toprule

\multicolumn{1}{c}{} &  
\multicolumn{1}{c|}{} & 
\multicolumn{5}{c|}{$\pi_1$\textbf{ = 0}}   &   
\multicolumn{5}{c|}{$\pi_1$\textbf{ = 0.25}}  & 
\multicolumn{10}{c}{$\pi_1$\textbf{ = 0.5}}   \\

\multicolumn{1}{c}{} &  
\multicolumn{1}{c|}{} & 
\multicolumn{5}{c|}{$\rho_1$\textbf{ = 0.5}} & 
\multicolumn{5}{c|}{$\rho_1$\textbf{ = 0.25}}  & 
\multicolumn{5}{c}{$\rho_1$\textbf{ = 0}} & 
\multicolumn{5}{c}{$\rho_1$\textbf{ = 0.5}}    \\

{\textbf{IMAGE}} &   \textbf{True} & \textbf{RP}$_{\rho}$ &    \textbf{RP} & \textbf{Nat13} & \textbf{Elk08} & \textbf{Liu16} & \textbf{RP}$_{\rho}$ &    \textbf{RP} & \textbf{Nat13} & \textbf{Elk08} & \textbf{Liu16} & \textbf{RP}$_{\rho}$ &    \textbf{RP} & \textbf{Nat13} & \textbf{Elk08} & \textbf{Liu16} & \textbf{RP}$_{\rho}$ &    \textbf{RP} & \textbf{Nat13} & \textbf{Elk08} & \textbf{Liu16} \\
\midrule
\textbf{plane}   & 0.107 &  0.287 & 0.133 &  0.123 &  \textbf{0.122} &  0.900 &  0.177 & 0.128 &  \textbf{0.119} &  0.123 &  0.900 &  0.248 & 0.124 &  0.110 &  0.118 &  \textbf{0.100} &  0.202 & 0.147 &  \textbf{0.142} &  0.160 &  0.900 \\
\textbf{auto} & 0.099 &  0.184 & 0.120 &  \textbf{0.110} &  \textbf{0.110} &  0.900 &  0.132 & 0.114 &  \textbf{0.105} &  0.109 &  0.900 &  0.189 & 0.110 &  0.105 &  0.110 &  \textbf{0.100} &  0.159 & 0.129 &  \textbf{0.125} &  0.139 &  0.900 \\
\textbf{bird}       & 0.117 &  0.354 & 0.148 &  0.133 &  \textbf{0.131} &  0.900 &  0.217 & 0.135 &  \textbf{0.120} &  0.125 &  0.900 &  0.277 & 0.135 &  0.115 &  0.123 &  \textbf{0.100} &  0.226 & 0.147 &  \textbf{0.139} &  0.158 &  0.900 \\
\textbf{cat}        & 0.114 &  0.351 & 0.138 &  \textbf{0.129} &  \textbf{0.129} &  0.900 &  0.208 & 0.139 &  \textbf{0.122} &  0.125 &  0.900 &  0.303 & 0.132 &  0.114 &  0.122 &  \textbf{0.100} &  0.225 & 0.151 &  \textbf{0.141} &  0.158 &  0.900 \\
\textbf{deer}       & 0.112 &  0.336 & 0.143 &  \textbf{0.128} &  0.130 &  0.900 &  0.194 & 0.135 &  \textbf{0.120} &  0.122 &  0.900 &  0.271 & 0.133 &  0.118 &  0.126 &  \textbf{0.100} &  0.209 & 0.150 &  \textbf{0.147} &  0.161 &  0.900 \\
\textbf{dog}        & 0.119 &  0.370 & 0.150 &  \textbf{0.136} &  0.138 &  0.900 &  0.205 & 0.142 &  \textbf{0.129} &  0.132 &  0.900 &  0.288 & 0.135 &  0.120 &  0.128 &  \textbf{0.100} &  0.229 & 0.154 &  \textbf{0.147} &  0.168 &  0.900 \\
\textbf{frog}       & 0.107 &  0.228 & 0.128 &  \textbf{0.117} &  \textbf{0.117} &  0.900 &  0.155 & 0.124 &  \textbf{0.113} &  0.115 &  0.900 &  0.228 & 0.118 &  0.110 &  0.116 &  \textbf{0.100} &  0.167 & 0.137 &  \textbf{0.130} &  0.142 &  0.900 \\
\textbf{horse}      & 0.104 &  0.251 & 0.127 &  \textbf{0.114} &  0.116 &  0.900 &  0.153 & 0.123 &  \textbf{0.110} &  0.112 &  0.900 &  0.224 & 0.116 &  0.108 &  0.113 &  \textbf{0.100} &  0.178 & 0.134 &  \textbf{0.129} &  0.144 &  0.900 \\
\textbf{ship}       & 0.112 &  0.239 & 0.134 &  \textbf{0.121} &  0.126 &  0.900 &  0.160 & 0.131 &  \textbf{0.119} &  0.123 &  0.900 &  0.236 & 0.122 &  0.113 &  0.120 &  \textbf{0.100} &  0.193 & 0.145 &  \textbf{0.139} &  0.159 &  0.900 \\
\textbf{truck}      & 0.106 &  0.210 & 0.130 &  \textbf{0.121} &  0.122 &  0.900 &  0.145 & 0.125 &  \textbf{0.113} &  0.117 &  0.900 &  0.213 & 0.121 &  0.108 &  0.117 &  \textbf{0.100} &  0.165 & 0.142 &  \textbf{0.134} &  0.150 &  0.900 \\
\midrule
\textbf{AVG}      & 0.110 &  0.281 & 0.135 &  \textbf{0.123} &  0.124 &  0.900 &  0.175 & 0.130 &  \textbf{0.117} &  0.120 &  0.900 &  0.248 & 0.125 &  0.112 &  0.119 &  \textbf{0.100} &  0.195 & 0.144 &  \textbf{0.137} &  0.154 &  0.900 \\

\bottomrule
\end{tabular}
}
\end{sc}
\end{small}
\end{center}
\end{supptable*}

\begin{supptable*}[ht]

\setlength\tabcolsep{1pt} 
\renewcommand{\arraystretch}{0.85}
\caption{Comparison of \textbf{AUC-PR} for one-vs-rest CIFAR-10 (averaged over all images) using a \textbf{logistic regression} classifier. Except for $RP_\rho$, $\rho_1$, $\rho_0$ are given to all methods. Top model scores are in bold with $RP_\rho$ in red if greater than non-RP models. Since $p_{y1}=0.1$, here \emph{Liu16} always predicts all labels as positive or negative, resulting in a constant AUC-PR of 0.550.} 
\vskip -0.1in
\label{table:cifar_logreg_auc}
\begin{center}
\begin{small}
\begin{sc}

\resizebox{\textwidth}{!}{
\begin{tabular}{l|c|ccccc|ccccc|ccccc|ccccc}
\toprule

\multicolumn{1}{c}{} &  
\multicolumn{1}{c|}{} & 
\multicolumn{5}{c|}{$\pi_1$\textbf{ = 0}}   &   
\multicolumn{5}{c|}{$\pi_1$\textbf{ = 0.25}}  & 
\multicolumn{10}{c}{$\pi_1$\textbf{ = 0.5}}   \\

\multicolumn{1}{c}{} &  
\multicolumn{1}{c|}{} & 
\multicolumn{5}{c|}{$\rho_1$\textbf{ = 0.5}} & 
\multicolumn{5}{c|}{$\rho_1$\textbf{ = 0.25}}  & 
\multicolumn{5}{c}{$\rho_1$\textbf{ = 0}} & 
\multicolumn{5}{c}{$\rho_1$\textbf{ = 0.5}}    \\

{\textbf{IMAGE}} &   \textbf{True} & \textbf{RP}$_{\rho}$ &    \textbf{RP} & \textbf{Nat13} & \textbf{Elk08} & \textbf{Liu16} & \textbf{RP}$_{\rho}$ &    \textbf{RP} & \textbf{Nat13} & \textbf{Elk08} & \textbf{Liu16} & \textbf{RP}$_{\rho}$ &    \textbf{RP} & \textbf{Nat13} & \textbf{Elk08} & \textbf{Liu16} & \textbf{RP}$_{\rho}$ &    \textbf{RP} & \textbf{Nat13} & \textbf{Elk08} & \textbf{Liu16} \\
\midrule
\textbf{plane}   & 0.288 &  0.225 & 0.224 &  0.225 &  0.207 &  \textbf{0.550} &  0.261 & 0.235 &  0.225 &  0.217 &  \textbf{0.550} &  0.285 & 0.251 &  0.245 &  0.248 &  \textbf{0.550} &  0.196 & 0.171 &  0.171 &  0.159 &  \textbf{0.550} \\
\textbf{auto} & 0.384 &  0.350 & 0.317 &  0.312 &  0.316 &  \textbf{0.550} &  0.342 & 0.335 &  0.331 &  0.331 &  \textbf{0.550} &  0.328 & 0.348 &  0.334 &  0.333 &  \textbf{0.550} &  0.256 & 0.257 &  0.259 &  0.261 &  \textbf{0.550} \\
\textbf{bird}       & 0.198 &  0.160 & 0.169 &  0.166 &  0.161 &  \textbf{0.550} &  0.188 & 0.185 &  0.179 &  0.177 &  \textbf{0.550} &  0.186 & 0.173 &  0.174 &  0.175 &  \textbf{0.550} &  0.150 & 0.154 &  0.150 &  0.147 &  \textbf{0.550} \\
\textbf{cat}        & 0.188 &  0.164 & 0.175 &  0.174 &  0.175 &  \textbf{0.550} &  0.163 & 0.169 &  0.168 &  0.170 &  \textbf{0.550} &  0.148 & 0.156 &  0.154 &  0.152 &  \textbf{0.550} &  0.145 & 0.143 &  0.140 &  0.145 &  \textbf{0.550} \\
\textbf{deer}       & 0.215 &  0.161 & 0.177 &  0.180 &  0.183 &  \textbf{0.550} &  0.194 & 0.180 &  0.180 &  0.182 &  \textbf{0.550} &  0.174 & 0.175 &  0.176 &  0.175 &  \textbf{0.550} &  0.151 & 0.152 &  0.146 &  0.151 &  \textbf{0.550} \\
\textbf{dog}        & 0.188 &  0.162 & 0.161 &  0.165 &  0.155 &  \textbf{0.550} &  0.175 & 0.160 &  0.161 &  0.158 &  \textbf{0.550} &  0.173 & 0.169 &  0.162 &  0.164 &  \textbf{0.550} &  0.145 & 0.142 &  0.139 &  0.133 &  \textbf{0.550} \\
\textbf{frog}       & 0.318 &  0.246 & 0.264 &  0.262 &  0.258 &  \textbf{0.550} &  0.292 & 0.277 &  0.272 &  0.273 &  \textbf{0.550} &  0.276 & 0.274 &  0.277 &  0.277 &  \textbf{0.550} &  0.239 & 0.212 &  0.206 &  0.212 &  \textbf{0.550} \\
\textbf{horse}      & 0.319 &  0.242 & 0.267 &  0.269 &  0.260 &  \textbf{0.550} &  0.283 & 0.264 &  0.264 &  0.263 &  \textbf{0.550} &  0.288 & 0.282 &  0.279 &  0.278 &  \textbf{0.550} &  0.223 & 0.218 &  0.208 &  0.207 &  \textbf{0.550} \\
\textbf{ship}       & 0.317 &  0.257 & 0.267 &  0.271 &  0.248 &  \textbf{0.550} &  0.296 & 0.266 &  0.267 &  0.259 &  \textbf{0.550} &  0.279 & 0.268 &  0.259 &  0.262 &  \textbf{0.550} &  0.220 & 0.212 &  0.207 &  0.191 &  \textbf{0.550} \\
\textbf{truck}      & 0.329 &  0.288 & 0.261 &  0.271 &  0.263 &  \textbf{0.550} &  0.298 & 0.275 &  0.286 &  0.284 &  \textbf{0.550} &  0.289 & 0.272 &  0.276 &  0.277 &  \textbf{0.550} &  0.241 & 0.213 &  0.208 &  0.204 &  \textbf{0.550} \\
\midrule
\textbf{AVG}      & 0.274 &  0.226 & 0.228 &  0.229 &  0.223 &  \textbf{0.550} &  0.249 & 0.235 &  0.233 &  0.231 &  \textbf{0.550} &  0.243 & 0.237 &  0.234 &  0.234 &  \textbf{0.550} &  0.197 & 0.187 &  0.183 &  0.181 &  \textbf{0.550} \\

\bottomrule
\end{tabular}
}
\end{sc}
\end{small}
\end{center}
\vskip -0.08in
\end{supptable*}

\begin{supptable*}[ht!]

\setlength\tabcolsep{1pt} 
\renewcommand{\arraystretch}{0.85}
\caption{Comparison of \textbf{error} for one-vs-rest CIFAR-10 (averaged over all images) using a \textbf{CNN} classifier. Except for $RP_\rho$, $\rho_1$, $\rho_0$ are given to all methods. Top model scores are in bold with $RP_\rho$ in red if better (smaller) than non-RP models.} 
\vskip -0.1in
\label{table:cifar_cnn_error}
\begin{center}
\begin{small}
\begin{sc}

\resizebox{\textwidth}{!}{
\begin{tabular}{l|c|ccccc|ccccc|ccccc|ccccc}
\toprule

\multicolumn{1}{c}{} &  
\multicolumn{1}{c|}{} & 
\multicolumn{5}{c|}{$\pi_1$\textbf{ = 0}}   &   
\multicolumn{5}{c|}{$\pi_1$\textbf{ = 0.25}}  & 
\multicolumn{10}{c}{$\pi_1$\textbf{ = 0.5}}   \\

\multicolumn{1}{c}{} &  
\multicolumn{1}{c|}{} & 
\multicolumn{5}{c|}{$\rho_1$\textbf{ = 0.5}} & 
\multicolumn{5}{c|}{$\rho_1$\textbf{ = 0.25}}  & 
\multicolumn{5}{c}{$\rho_1$\textbf{ = 0}} & 
\multicolumn{5}{c}{$\rho_1$\textbf{ = 0.5}}    \\

{\textbf{IMAGE}} &   \textbf{True} & \textbf{RP}$_{\rho}$ &    \textbf{RP} & \textbf{Nat13} & \textbf{Elk08} & \textbf{Liu16} & \textbf{RP}$_{\rho}$ &    \textbf{RP} & \textbf{Nat13} & \textbf{Elk08} & \textbf{Liu16} & \textbf{RP}$_{\rho}$ &    \textbf{RP} & \textbf{Nat13} & \textbf{Elk08} & \textbf{Liu16} & \textbf{RP}$_{\rho}$ &    \textbf{RP} & \textbf{Nat13} & \textbf{Elk08} & \textbf{Liu16} \\
\midrule
\textbf{plane}   & 0.044 &  \textcolor{red}{\textbf{0.054}} & \textbf{0.057} &  0.059 &  0.063 &  0.900 &  \textcolor{red}{\textbf{0.050}} & \textbf{0.051} &  0.054 &  0.057 &  0.900 &  \textcolor{red}{\textbf{0.048}} & \textbf{0.045} &  0.049 &  0.048 &  0.100 &  \textcolor{red}{\textbf{0.063}} & \textbf{0.061} &  0.074 &  0.065 &  0.900 \\
\textbf{auto} & 0.021 &  \textcolor{red}{\textbf{0.040}} & \textbf{0.037} &  0.041 &  0.043 &  0.100 &  \textcolor{red}{\textbf{0.032}} & \textbf{0.034} &  0.040 &  0.039 &  0.900 &  0.028 & \textbf{0.026} &  \textbf{0.026} &  \textbf{0.026} &  0.100 &  \textcolor{red}{\textbf{0.047}} & \textbf{0.049} &  0.062 &  0.070 &  0.900 \\
\textbf{bird}       & 0.055 &  0.083 & \textbf{0.078} &  0.080 &  0.082 &  0.900 &  \textcolor{red}{\textbf{0.074}} & \textbf{0.074} &  0.077 &  0.078 &  0.900 &  0.072 & \textbf{0.066} &  0.072 &  0.070 &  0.100 &  0.124 & \textbf{0.084} &  0.089 &  0.093 &  0.900 \\
\textbf{cat}        & 0.077 &  0.108 & \textbf{0.091} &  0.092 &  0.095 &  0.100 &  0.111 & 0.090 &  \textbf{0.086} &  0.089 &  0.900 &  0.113 & \textbf{0.084} &  0.086 &  0.088 &  0.100 &  0.117 & 0.098 &  \textbf{0.094} &  0.100 &  0.900 \\
\textbf{deer}       & 0.049 &  0.081 & \textbf{0.078} &  \textbf{0.078} &  0.079 &  0.900 &  0.080 & \textbf{0.069} &  0.075 &  0.070 &  0.900 &  0.076 & 0.062 &  \textbf{0.061} &  0.062 &  0.100 &  0.106 & \textbf{0.086} &  0.091 &  0.093 &  0.900 \\
\textbf{dog}        & 0.062 &  \textcolor{red}{\textbf{0.075}} & \textbf{0.071} &  0.079 &  0.080 &  0.100 &  0.071 & 0.069 &  0.070 &  \textbf{0.067} &  0.900 &  0.069 & 0.061 &  \textbf{0.057} &  0.076 &  0.100 &  0.103 & \textbf{0.081} &  0.084 &  0.086 &  0.900 \\
\textbf{frog}       & 0.038 &  0.050 & \textbf{0.048} &  \textbf{0.048} &  0.054 &  0.100 &  \textcolor{red}{\textbf{0.047}} & \textbf{0.052} &  0.056 &  0.062 &  0.900 &  0.045 & \textbf{0.040} &  0.042 &  0.043 &  0.100 &  \textcolor{red}{\textbf{0.058}} & \textbf{0.062} &  0.066 &  0.071 &  0.900 \\
\textbf{horse}      & 0.035 &  \textcolor{red}{\textbf{0.050}} & \textbf{0.052} &  0.057 &  0.054 &  0.900 &  \textcolor{red}{\textbf{0.048}} & \textbf{0.051} &  0.052 &  0.057 &  0.900 &  0.045 & \textbf{0.040} &  0.042 &  0.046 &  0.100 &  \textcolor{red}{\textbf{0.065}} & \textbf{0.063} &  0.066 &  0.075 &  0.900 \\
\textbf{ship}       & 0.028 &  \textcolor{red}{\textbf{0.042}} & \textbf{0.042} &  0.046 &  \textbf{0.042} &  0.900 &  \textcolor{red}{\textbf{0.037}} & \textbf{0.036} &  0.042 &  0.047 &  0.900 &  0.035 & 0.033 &  \textbf{0.031} &  0.033 &  0.100 &  \textcolor{red}{\textbf{0.051}} & \textbf{0.049} &  0.064 &  0.058 &  0.900 \\
\textbf{truck}      & 0.027 &  \textcolor{red}{\textbf{0.044}} & \textbf{0.046} &  0.054 &  0.056 &  0.900 &  \textcolor{red}{\textbf{0.034}} & \textbf{0.032} &  0.038 &  0.043 &  0.900 &  \textcolor{red}{\textbf{0.034}} & \textbf{0.031} &  0.034 &  0.034 &  0.100 &  \textcolor{red}{\textbf{0.060}} & 0.066 &  0.067 &  \textbf{0.065} &  0.900 \\
\midrule
\textbf{AVG}      & 0.043 &  \textcolor{red}{\textbf{0.063}} & \textbf{0.060} &  0.064 &  0.065 &  0.580 &  \textcolor{red}{\textbf{0.059}} & \textbf{0.056} &  0.059 &  0.061 &  0.900 &  0.056 & \textbf{0.049} &  0.050 &  0.053 &  0.100 &  0.080 & \textbf{0.070} &  0.076 &  0.077 &  0.900 \\

\bottomrule
\end{tabular}
}
\end{sc}
\end{small}
\end{center}
\vskip -0.08in
\end{supptable*}

\begin{supptable*}[ht!]

\setlength\tabcolsep{1pt} 
\renewcommand{\arraystretch}{0.85}
\caption{Comparison of \textbf{AUC-PR} for one-vs-rest CIFAR-10 (averaged over all images) using a \textbf{CNN} classifier. Except for $RP_\rho$, $\rho_1$, $\rho_0$ are given to all methods. Top model scores are in bold with $RP_\rho$ in red if greater than non-RP models.} 
\vskip -0.1in
\label{table:cifar_cnn_auc}
\begin{center}
\begin{small}
\begin{sc}

\resizebox{\textwidth}{!}{
\begin{tabular}{l|c|ccccc|ccccc|ccccc|ccccc}
\toprule

\multicolumn{1}{c}{} &  
\multicolumn{1}{c|}{} & 
\multicolumn{5}{c|}{$\pi_1$\textbf{ = 0}}   &   
\multicolumn{5}{c|}{$\pi_1$\textbf{ = 0.25}}  & 
\multicolumn{10}{c}{$\pi_1$\textbf{ = 0.5}}   \\

\multicolumn{1}{c}{} &  
\multicolumn{1}{c|}{} & 
\multicolumn{5}{c|}{$\rho_1$\textbf{ = 0.5}} & 
\multicolumn{5}{c|}{$\rho_1$\textbf{ = 0.25}}  & 
\multicolumn{5}{c}{$\rho_1$\textbf{ = 0}} & 
\multicolumn{5}{c}{$\rho_1$\textbf{ = 0.5}}    \\

{\textbf{IMAGE}} &   \textbf{True} & \textbf{RP}$_{\rho}$ &    \textbf{RP} & \textbf{Nat13} & \textbf{Elk08} & \textbf{Liu16} & \textbf{RP}$_{\rho}$ &    \textbf{RP} & \textbf{Nat13} & \textbf{Elk08} & \textbf{Liu16} & \textbf{RP}$_{\rho}$ &    \textbf{RP} & \textbf{Nat13} & \textbf{Elk08} & \textbf{Liu16} & \textbf{RP}$_{\rho}$ &    \textbf{RP} & \textbf{Nat13} & \textbf{Elk08} & \textbf{Liu16} \\
\midrule
\textbf{plane}   & 0.856 &  0.779 & 0.780 &  \textbf{0.784} &  0.756 &  0.550 &  \textcolor{red}{\textbf{0.808}} & \textbf{0.797} &  0.770 &  0.742 &  0.550 &  \textcolor{red}{\textbf{0.813}} & \textbf{0.824} &  0.792 &  0.794 &  0.550 &  \textcolor{red}{\textbf{0.710}} & \textbf{0.722} &  0.662 &  0.682 &  0.550 \\
\textbf{auto} & 0.954 &  0.874 & \textbf{0.889} &  0.878 &  0.833 &  0.550 &  \textcolor{red}{\textbf{0.905}} & \textbf{0.900} &  0.871 &  0.866 &  0.550 &  \textcolor{red}{\textbf{0.931}} & \textbf{0.927} &  0.924 &  0.910 &  0.550 &  \textcolor{red}{\textbf{0.824}} & \textbf{0.814} &  0.756 &  0.702 &  0.550 \\
\textbf{bird}       & 0.761 &  0.559 & 0.566 &  \textbf{0.569} &  0.568 &  0.550 &  \textcolor{red}{\textbf{0.619}} & \textbf{0.618} &  0.584 &  0.597 &  0.550 &  \textcolor{red}{\textbf{0.623}} & \textbf{0.679} &  0.613 &  0.619 &  0.115 &  0.465 & 0.492 &  0.436 &  0.434 &  \textbf{0.550} \\
\textbf{cat}        & 0.601 &  0.387 & 0.447 &  \textbf{0.463} &  0.433 &  0.550 &  0.423 & 0.454 &  0.487 &  0.480 &  \textbf{0.550} &  0.483 & \textbf{0.512} &  0.493 &  0.473 &  0.050 &  0.373 & 0.375 &  0.382 &  0.371 &  \textbf{0.550} \\
\textbf{deer}       & 0.820 &  \textcolor{red}{\textbf{0.620}} & 0.600 &  \textbf{0.615} &  0.573 &  0.550 &  0.646 & \textbf{0.660} &  0.610 &  0.657 &  0.550 &  0.658 & \textbf{0.707} &  0.700 &  0.703 &  0.550 &  0.434 & 0.487 &  0.414 &  0.435 &  \textbf{0.550} \\
\textbf{dog}        & 0.758 &  \textcolor{red}{\textbf{0.629}} & \textbf{0.662} &  0.617 &  0.573 &  0.550 &  \textcolor{red}{\textbf{0.673}} & \textbf{0.667} &  0.658 &  0.660 &  0.550 &  0.705 & 0.722 &  \textbf{0.741} &  0.705 &  0.550 &  0.541 & 0.545 &  0.496 &  0.519 &  \textbf{0.550} \\
\textbf{frog}       & 0.891 &  0.812 & \textbf{0.815} &  0.812 &  0.776 &  0.550 &  \textcolor{red}{\textbf{0.821}} & \textbf{0.827} &  0.808 &  0.749 &  0.550 &  \textcolor{red}{\textbf{0.841}} & \textbf{0.851} &  0.828 &  0.831 &  0.550 &  \textcolor{red}{\textbf{0.753}} & \textbf{0.710} &  0.691 &  0.620 &  0.550 \\
\textbf{horse}      & 0.897 &  \textcolor{red}{\textbf{0.810}} & \textbf{0.817} &  0.799 &  0.779 &  0.550 &  \textcolor{red}{\textbf{0.824}} & \textbf{0.809} &  0.801 &  0.772 &  0.550 &  \textcolor{red}{\textbf{0.826}} & \textbf{0.844} &  0.818 &  0.819 &  0.550 &  \textcolor{red}{\textbf{0.736}} & \textbf{0.699} &  \textbf{0.699} &  0.600 &  0.550 \\
\textbf{ship}       & 0.922 &  \textcolor{red}{\textbf{0.870}} & 0.862 &  \textbf{0.864} &  0.853 &  0.550 &  \textcolor{red}{\textbf{0.889}} & \textbf{0.885} &  0.843 &  0.848 &  0.550 &  0.889 & \textbf{0.897} &  0.891 &  0.887 &  0.550 &  \textcolor{red}{\textbf{0.800}} & \textbf{0.808} &  0.767 &  0.741 &  0.550 \\
\textbf{truck}      & 0.929 &  \textcolor{red}{\textbf{0.845}} & \textbf{0.848} &  0.824 &  0.787 &  0.550 &  \textcolor{red}{\textbf{0.887}} & \textbf{0.894} &  0.873 &  0.853 &  0.550 &  \textcolor{red}{\textbf{0.904}} & \textbf{0.902} &  0.898 &  0.883 &  0.550 &  \textcolor{red}{\textbf{0.740}} & \textbf{0.709} &  0.695 &  0.690 &  0.550 \\
\midrule
\textbf{AVG}      & 0.839 &  0.719 & \textbf{0.729} &  0.722 &  0.693 &  0.550 &  \textcolor{red}{\textbf{0.750}} & \textbf{0.751} &  0.730 &  0.722 &  0.550 &  0.767 & \textbf{0.787} &  0.770 &  0.762 &  0.457 &  \textcolor{red}{\textbf{0.637}} & \textbf{0.636} &  0.600 &  0.579 &  0.550 \\

\bottomrule
\end{tabular}
}
\end{sc}
\end{small}
\end{center}
\end{supptable*}

\begin{singlespace}
\bibliography{main}

\newcommand{\etalchar}[1]{$^{#1}$}
\begin{thebibliography}{vdOKE{\etalchar{+}}16}

\bibitem[AAMP05]{angelova2005pruning}
Anelia Angelova, Yaser Abu-Mostafam, and Pietro Perona.
\newblock Pruning training sets for learning of object categories.
\newblock In {\em CVPR}, volume~1, pages 494--501. IEEE, 2005.

\bibitem[ADG{\etalchar{+}}16]{andrychowicz2016learning}
Marcin Andrychowicz, Misha Denil, Sergio Gomez, Matthew~W Hoffman, David Pfau,
  Tom Schaul, Brendan Shillingford, and Nando De~Freitas.
\newblock Learning to learn by gradient descent by gradient descent.
\newblock In {\em Advances in Neural Information Processing Systems}, pages
  3981--3989, 2016.

\bibitem[AFDM16]{alemi2016deep}
Alexander~A Alemi, Ian Fischer, Joshua~V Dillon, and Kevin Murphy.
\newblock Deep variational information bottleneck.
\newblock {\em arXiv preprint arXiv:1612.00410}, 2016.

\bibitem[AG19]{amjad2019learning}
Rana~Ali Amjad and Bernhard~Claus Geiger.
\newblock Learning representations for neural network-based classification
  using the information bottleneck principle.
\newblock {\em IEEE Transactions on Pattern Analysis and Machine Intelligence},
  2019.

\bibitem[AGKN13]{anantharam2013maximal}
Venkat Anantharam, Amin Gohari, Sudeep Kamath, and Chandra Nair.
\newblock On maximal correlation, hypercontractivity, and the data processing
  inequality studied by erkip and cover.
\newblock {\em arXiv preprint arXiv:1304.6133}, 2013.

\bibitem[AKA91]{Aha1991}
David~W. Aha, Dennis Kibler, and Marc~K. Albert.
\newblock Instance-based learning algorithms.
\newblock {\em Mach. Learn.}, 6(1):37--66, 1991.

\bibitem[AL88]{angluin1988learning}
Dana Angluin and Philip Laird.
\newblock Learning from noisy examples.
\newblock {\em Machine Learning}, 2(4):343--370, 1988.

\bibitem[Alb91]{albus1991outline}
James~S Albus.
\newblock Outline for a theory of intelligence.
\newblock {\em IEEE Transactions on Systems, Man, and Cybernetics},
  21(3):473--509, 1991.

\bibitem[AMS04]{ancona2004radial}
Nicola Ancona, Daniele Marinazzo, and Sebastiano Stramaglia.
\newblock Radial basis function approach to nonlinear granger causality of time
  series.
\newblock {\em Physical Review E}, 70(5):056221, 2004.

\bibitem[AMS18]{achille2018dynamics}
Alessandro Achille, Glen Mbeng, and Stefano Soatto.
\newblock {The Dynamics of Differential Learning I: Information-Dynamics and
  Task Reachability}.
\newblock {\em arXiv preprint arXiv:1810.02440}, 2018.

\bibitem[Ana92]{anastasi1992counselors}
Anne Anastasi.
\newblock What counselors should know about the use and interpretation of
  psychological tests.
\newblock {\em Journal of Counseling \& Development}, 70(5):610--615, 1992.

\bibitem[AOS{\etalchar{+}}16]{amodei2016concrete}
Dario Amodei, Chris Olah, Jacob Steinhardt, Paul Christiano, John Schulman, and
  Dan Man{\'e}.
\newblock Concrete problems in ai safety.
\newblock {\em arXiv preprint arXiv:1606.06565}, 2016.

\bibitem[Ari72]{arimoto1972algorithm}
Suguru Arimoto.
\newblock An algorithm for computing the capacity of arbitrary discrete
  memoryless channels.
\newblock {\em IEEE Transactions on Information Theory}, 18(1):14--20, 1972.

\bibitem[AS18a]{achille2018emergence}
Alessandro Achille and Stefano Soatto.
\newblock Emergence of invariance and disentanglement in deep representations.
\newblock {\em The Journal of Machine Learning Research}, 19(1):1947--1980,
  2018.

\bibitem[AS18b]{achille2018information}
Alessandro Achille and Stefano Soatto.
\newblock Information dropout: Learning optimal representations through noisy
  computation.
\newblock {\em IEEE Transactions on Pattern Analysis and Machine Intelligence},
  2018.

\bibitem[BBC{\etalchar{+}}17]{boden2017principles}
Margaret Boden, Joanna Bryson, Darwin Caldwell, Kerstin Dautenhahn, Lilian
  Edwards, Sarah Kember, Paul Newman, Vivienne Parry, Geoff Pegman, Tom Rodden,
  et~al.
\newblock Principles of robotics: regulating robots in the real world.
\newblock {\em Connection Science}, 29(2):124--129, 2017.

\bibitem[BBCG92]{bengio1992optimization}
Samy Bengio, Yoshua Bengio, Jocelyn Cloutier, and Jan Gecsei.
\newblock On the optimization of a synaptic learning rule.
\newblock In {\em Preprints Conf. Optimality in Artificial and Biological
  Neural Networks}, pages 6--8. Univ. of Texas, 1992.

\bibitem[BCP{\etalchar{+}}16]{1606.01540}
Greg Brockman, Vicki Cheung, Ludwig Pettersson, Jonas Schneider, John Schulman,
  Jie Tang, and Wojciech Zaremba.
\newblock Openai gym, 2016.

\bibitem[BES01]{bradley2001reasoning}
Elizabeth Bradley, Matthew Easley, and Reinhard Stolle.
\newblock Reasoning about nonlinear system identification.
\newblock {\em Artificial Intelligence}, 133(1):139 -- 188, 2001.

\bibitem[BFP{\etalchar{+}}73]{Blum:1973:TBS:1739940.1740109}
Manuel Blum, Robert~W. Floyd, Vaughan Pratt, Ronald~L. Rivest, and Robert~E.
  Tarjan.
\newblock Time bounds for selection.
\newblock {\em J. Comput. Syst. Sci.}, 7(4):448--461, August 1973.

\bibitem[BGNR16]{baker2016designing}
Bowen Baker, Otkrist Gupta, Nikhil Naik, and Ramesh Raskar.
\newblock Designing neural network architectures using reinforcement learning.
\newblock {\em arXiv preprint arXiv:1611.02167}, 2016.

\bibitem[BHB{\etalchar{+}}18]{battaglia2018relational}
Peter~W Battaglia, Jessica~B Hamrick, Victor Bapst, Alvaro Sanchez-Gonzalez,
  Vinicius Zambaldi, Mateusz Malinowski, Andrea Tacchetti, David Raposo, Adam
  Santoro, Ryan Faulkner, et~al.
\newblock Relational inductive biases, deep learning, and graph networks.
\newblock {\em arXiv preprint arXiv:1806.01261}, 2018.

\bibitem[Bla72]{blahut1972computation}
Richard Blahut.
\newblock Computation of channel capacity and rate-distortion functions.
\newblock {\em IEEE transactions on Information Theory}, 18(4):460--473, 1972.

\bibitem[BLS10]{Blanchard:2010:SND:1756006.1953028}
Gilles Blanchard, Gyemin Lee, and Clayton Scott.
\newblock Semi-supervised novelty detection.
\newblock {\em J. Mach. Learn. Res.}, 11:2973--3009, December 2010.

\bibitem[BM98]{Blum:1998:CLU:279943.279962}
Avrim Blum and Tom Mitchell.
\newblock Combining labeled and unlabeled data with co-training.
\newblock In {\em 11th Conf. on COLT}, pages 92--100, New York, NY, USA, 1998.
  ACM.

\bibitem[BNVB13]{bellemare2013arcade}
Marc~G Bellemare, Yavar Naddaf, Joel Veness, and Michael Bowling.
\newblock The arcade learning environment: An evaluation platform for general
  agents.
\newblock {\em Journal of Artificial Intelligence Research}, 47:253--279, 2013.

\bibitem[BO18]{blierdescription}
L{\'e}onard Blier and Yann Ollivier.
\newblock The description length of deep learning models.
\newblock 2018.

\bibitem[BPL{\etalchar{+}}16]{battaglia2016interaction}
Peter Battaglia, Razvan Pascanu, Matthew Lai, Danilo Jimenez~Rezende, and koray
  kavukcuoglu.
\newblock Interaction networks for learning about objects, relations and
  physics.
\newblock In D.~D. Lee, M.~Sugiyama, U.~V. Luxburg, I.~Guyon, and R.~Garnett,
  editors, {\em Advances in Neural Information Processing Systems 29}, pages
  4502--4510. Curran Associates, Inc., 2016.

\bibitem[BRB{\etalchar{+}}18]{belghazi2018mine}
Ishmael Belghazi, Sai Rajeswar, Aristide Baratin, R~Devon Hjelm, and Aaron
  Courville.
\newblock Mine: mutual information neural estimation.
\newblock {\em arXiv preprint arXiv:1801.04062}, 2018.

\bibitem[Bre96]{Breiman:1996:BP:231986.231989}
Leo Breiman.
\newblock Bagging predictors.
\newblock {\em Machine Learning}, 24(2):123--140, August 1996.

\bibitem[BRLW17]{bera2017generalized}
Manabendra~N Bera, Arnau Riera, Maciej Lewenstein, and Andreas Winter.
\newblock Generalized laws of thermodynamics in the presence of correlations.
\newblock {\em Nature communications}, 8(1):2180, 2017.

\bibitem[BSW14]{baldi2014searching}
Pierre Baldi, Peter Sadowski, and Daniel Whiteson.
\newblock Searching for exotic particles in high-energy physics with deep
  learning.
\newblock {\em Nature communications}, 5:4308, 2014.

\bibitem[BSXT18]{bramley2018learning}
Neil Bramley, Eric Schulz, Fei Xu, and Joshua Tenenbaum.
\newblock Learning as program induction.
\newblock 2018.

\bibitem[{\v{C}}er85]{vcerny1985thermodynamical}
Vladim{\'\i}r {\v{C}}ern{\`y}.
\newblock Thermodynamical approach to the traveling salesman problem: An
  efficient simulation algorithm.
\newblock {\em Journal of optimization theory and applications}, 45(1):41--51,
  1985.

\bibitem[CGTW05]{chechik2005information}
Gal Chechik, Amir Globerson, Naftali Tishby, and Yair Weiss.
\newblock Information bottleneck for gaussian variables.
\newblock {\em Journal of machine learning research}, 6(Jan):165--188, 2005.

\bibitem[Chi02]{chickering2002optimal}
David~Maxwell Chickering.
\newblock Optimal structure identification with greedy search.
\newblock {\em Journal of machine learning research}, 3(Nov):507--554, 2002.

\bibitem[Cho16a]{cifar_cnn_structure}
Francois Chollet.
\newblock {\em Keras CIFAR CNN}, 2016.
\newblock \href{http://bit.ly/2mVKR3d}{bit.ly/2mVKR3d}.

\bibitem[Cho16b]{mnist_cnn_structure}
Francois Chollet.
\newblock {\em Keras MNIST CNN}, 2016.
\newblock \href{http://bit.ly/2nKiqJv}{bit.ly/2nKiqJv}.

\bibitem[Cho19]{chollet2019measure}
Fran{ç}ois Chollet.
\newblock On the measure of intelligence.
\newblock {\em arXiv preprint arXiv:11911.01547}, 2019.

\bibitem[CJM{\etalchar{+}}19]{cubero2019statistical}
Ryan~John Cubero, Junghyo Jo, Matteo Marsili, Yasser Roudi, and Juyong Song.
\newblock Statistical criticality arises in most informative representations.
\newblock {\em Journal of Statistical Mechanics: Theory and Experiment},
  2019(6):063402, 2019.

\bibitem[CJS{\etalchar{+}}18]{capper2018dna}
David Capper, David~TW Jones, Martin Sill, Volker Hovestadt, Daniel Schrimpf,
  Dominik Sturm, Christian Koelsche, Felix Sahm, Lukas Chavez, David~E Reuss,
  et~al.
\newblock Dna methylation-based classification of central nervous system
  tumours.
\newblock {\em Nature}, 555(7697):469, 2018.

\bibitem[CLB19]{clark2019unsupervised}
David~G Clark, Jesse~A Livezey, and Kristofer~E Bouchard.
\newblock Unsupervised discovery of temporal structure in noisy data with
  dynamical components analysis.
\newblock {\em arXiv preprint arXiv:1905.09944}, 2019.

\bibitem[CLRS09]{cormen2009introduction}
Thomas~H Cormen, Charles~E Leiserson, Ronald~L Rivest, and Clifford Stein.
\newblock {\em Introduction to algorithms}.
\newblock MIT press, 2009.

\bibitem[CM17]{carrasquilla2017machine}
Juan Carrasquilla and Roger~G Melko.
\newblock Machine learning phases of matter.
\newblock {\em Nature Physics}, 13(5):431, 2017.

\bibitem[CMT16]{chalk2016relevant}
Matthew Chalk, Olivier Marre, and Gasper Tkacik.
\newblock Relevant sparse codes with variational information bottleneck.
\newblock In {\em Advances in Neural Information Processing Systems}, pages
  1957--1965, 2016.

\bibitem[Col15]{colman2015dictionary}
Andrew~M Colman.
\newblock {\em A dictionary of psychology}.
\newblock Oxford Quick Reference, 2015.

\bibitem[CRBD18]{chen2018neural}
Tian~Qi Chen, Yulia Rubanova, Jesse Bettencourt, and David~K Duvenaud.
\newblock Neural ordinary differential equations.
\newblock In {\em Advances in neural information processing systems}, pages
  6571--6583, 2018.

\bibitem[CSSM15]{Claesen201573}
Marc Claesen, Frank~De Smet, Johan~A.K. Suykens, and Bart~De Moor.
\newblock A robust ensemble approach to learn from positive and unlabeled data
  using \{SVM\} base models.
\newblock {\em Neurocomputing}, 160:73 -- 84, 2015.

\bibitem[CTMA19]{carrasquilla2019reconstructing}
Juan Carrasquilla, Giacomo Torlai, Roger~G Melko, and Leandro Aolita.
\newblock Reconstructing quantum states with generative models.
\newblock {\em Nature Machine Intelligence}, 1(3):155, 2019.

\bibitem[CUTT16]{chang2016compositional}
Michael~B Chang, Tomer Ullman, Antonio Torralba, and Joshua~B Tenenbaum.
\newblock A compositional object-based approach to learning physical dynamics.
\newblock {\em arXiv preprint arXiv:1612.00341}, 2016.

\bibitem[CV99]{Chapelle:1999:MSS:3009657.3009690}
Olivier Chapelle and Vladimir Vapnik.
\newblock Model selection for support vector machines.
\newblock In {\em Proc. of 12th NIPS}, pages 230--236, Cambridge, MA, USA,
  1999.

\bibitem[CZM{\etalchar{+}}18]{autoaugment}
Ekin~D Cubuk, Barret Zoph, Dandelion Mane, Vijay Vasudevan, and Quoc~V Le.
\newblock Autoaugment: Learning augmentation policies from data.
\newblock {\em arXiv preprint arXiv:1805.09501}, 2018.

\bibitem[DCB06]{ding2006granger}
Mingzhou Ding, Yonghong Chen, and Steven~L Bressler.
\newblock Granger causality: basic theory and application to neuroscience.
\newblock {\em Handbook of time series analysis: recent theoretical
  developments and applications}, pages 437--460, 2006.

\bibitem[DCL08]{diuk2008object}
Carlos Diuk, Andre Cohen, and Michael~L Littman.
\newblock An object-oriented representation for efficient reinforcement
  learning.
\newblock In {\em Proceedings of the 25th international conference on Machine
  learning}, pages 240--247. ACM, 2008.

\bibitem[DG06a]{davis2006relationship}
Jesse Davis and Mark Goadrich.
\newblock The relationship between precision-recall and roc curves.
\newblock In {\em Proceedings of the 23rd international conference on Machine
  learning}, pages 233--240. ACM, 2006.

\bibitem[DG06b]{Davis:2006:RPR:1143844.1143874}
Jesse Davis and Mark Goadrich.
\newblock The relationship between precision-recall and roc curves.
\newblock In {\em Proc. of 23rd ICML}, pages 233--240, NYC, NY, USA, 2006. ACM.

\bibitem[dHJL19]{de2019causal}
Pim de~Haan, Dinesh Jayaraman, and Sergey Levine.
\newblock Causal confusion in imitation learning.
\newblock {\em arXiv preprint arXiv:1905.11979}, 2019.

\bibitem[DKPR87]{DUANE1987216}
Simon Duane, A.D. Kennedy, Brian~J. Pendleton, and Duncan Roweth.
\newblock Hybrid monte carlo.
\newblock {\em Physics Letters B}, 195(2):216 -- 222, 1987.

\bibitem[DLT07]{dvzeroski2007computational}
Sa{\v{s}}o D{\v{z}}eroski, Pat Langley, and Ljup{\v{c}}o Todorovski.
\newblock {\em Computational Discovery of Scientific Knowledge}, pages 1--14.
\newblock Springer Berlin Heidelberg, Berlin, Heidelberg, 2007.

\bibitem[DMAT13]{dechter2013bootstrap}
Eyal Dechter, Jonathan Malmaud, Ryan~P Adams, and Joshua~B Tenenbaum.
\newblock Bootstrap learning via modular concept discovery.
\newblock In {\em IJCAI}, pages 1302--1309, 2013.

\bibitem[DN15]{daniels2015automated}
Bryan~C Daniels and Ilya Nemenman.
\newblock Automated adaptive inference of phenomenological dynamical models.
\newblock {\em Nature communications}, 6:8133, 2015.

\bibitem[DSC{\etalchar{+}}16]{duan2016rl}
Yan Duan, John Schulman, Xi~Chen, Peter~L Bartlett, Ilya Sutskever, and Pieter
  Abbeel.
\newblock Rl2: Fast reinforcement learning via slow reinforcement learning.
\newblock {\em arXiv preprint arXiv:1611.02779}, 2016.

\bibitem[DT95]{dzeroski1995discovering}
Saso Dzeroski and Ljupco Todorovski.
\newblock Discovering dynamics: From inductive logic programming to machine
  discovery.
\newblock {\em Journal of Intelligent Information Systems}, 4(1):89--108, Jan
  1995.

\bibitem[DUB{\etalchar{+}}17]{devlin2017robustfill}
Jacob Devlin, Jonathan Uesato, Surya Bhupatiraju, Rishabh Singh, Abdel-rahman
  Mohamed, and Pushmeet Kohli.
\newblock Robustfill: Neural program learning under noisy i/o.
\newblock {\em arXiv preprint arXiv:1703.07469}, 2017.

\bibitem[EC98]{erkip1998efficiency}
Elza Erkip and Thomas~M Cover.
\newblock The efficiency of investment information.
\newblock {\em IEEE Transactions on Information Theory}, 44(3):1026--1040,
  1998.

\bibitem[EHJ{\etalchar{+}}04]{efron2004least}
Bradley Efron, Trevor Hastie, Iain Johnstone, Robert Tibshirani, et~al.
\newblock Least angle regression.
\newblock {\em The Annals of statistics}, 32(2):407--499, 2004.

\bibitem[EN08]{Elkan:2008:LCO:1401890.1401920}
Charles Elkan and Keith Noto.
\newblock Learning classifiers from only positive and unlabeled data.
\newblock In {\em Proc. of 14th KDD}, pages 213--220, NYC, NY, USA, 2008. ACM.

\bibitem[ES16]{edwards2016towards}
Harrison Edwards and Amos Storkey.
\newblock Towards a neural statistician.
\newblock {\em arXiv preprint arXiv:1606.02185}, 2016.

\bibitem[ESLT15]{ellis2015unsupervised}
Kevin Ellis, Armando Solar-Lezama, and Josh Tenenbaum.
\newblock Unsupervised learning by program synthesis.
\newblock In C.~Cortes, N.~D. Lawrence, D.~D. Lee, M.~Sugiyama, and R.~Garnett,
  editors, {\em Advances in Neural Information Processing Systems 28}, pages
  973--981. Curran Associates, Inc., 2015.

\bibitem[EY36]{eckart1936SVD}
Carl Eckart and Gale Young.
\newblock The approximation of one matrix by another of lower rank.
\newblock {\em Psychometrika}, 1(3):211--218, 1936.

\bibitem[FAL17]{pmlr-v70-finn17a}
Chelsea Finn, Pieter Abbeel, and Sergey Levine.
\newblock Model-agnostic meta-learning for fast adaptation of deep networks.
\newblock In Doina Precup and Yee~Whye Teh, editors, {\em Proceedings of the
  34th International Conference on Machine Learning}, volume~70 of {\em
  Proceedings of Machine Learning Research}, pages 1126--1135, International
  Convention Centre, Sydney, Australia, 06--11 Aug 2017. PMLR.

\bibitem[Fis18]{fischer2018the}
Ian Fischer.
\newblock The conditional entropy bottleneck, 2018.

\bibitem[FKPW17]{NIPS2017_69510}
Marco Fraccaro, Simon Kamronn, Ulrich Paquet, and Ole Winther.
\newblock A disentangled recognition and nonlinear dynamics model for
  unsupervised learning.
\newblock In I.~Guyon, U.~V. Luxburg, S.~Bengio, H.~Wallach, R.~Fergus,
  S.~Vishwanathan, and R.~Garnett, editors, {\em Advances in Neural Information
  Processing Systems 30}, pages 3601--3610. Curran Associates, Inc., 2017.

\bibitem[FS97]{freund1997decision}
Yoav Freund and Robert~E Schapire.
\newblock A decision-theoretic generalization of on-line learning and an
  application to boosting.
\newblock {\em Journal of computer and system sciences}, 55(1):119--139, 1997.

\bibitem[F{\"u}r99]{furnkranz1999separate}
Johannes F{\"u}rnkranz.
\newblock Separate-and-conquer rule learning.
\newblock {\em Artificial Intelligence Review}, 13(1):3--54, 1999.

\bibitem[Geb41]{gebelein1941statistische}
Hans Gebelein.
\newblock Das statistische problem der korrelation als variations-und
  eigenwertproblem und sein zusammenhang mit der ausgleichsrechnung.
\newblock {\em ZAMM-Journal of Applied Mathematics and Mechanics/Zeitschrift
  f{\"u}r Angewandte Mathematik und Mechanik}, 21(6):364--379, 1941.

\bibitem[GHY00]{granger2000bivariate}
Clive~WJ Granger, Bwo-Nung Huangb, and Chin-Wei Yang.
\newblock A bivariate causality between stock prices and exchange rates:
  evidence from recent asianflu.
\newblock {\em The Quarterly Review of Economics and Finance}, 40(3):337--354,
  2000.

\bibitem[GIS{\etalchar{+}}19]{goyal2019infobot}
Anirudh Goyal, Riashat Islam, Daniel Strouse, Zafarali Ahmed, Matthew
  Botvinick, Hugo Larochelle, Sergey Levine, and Yoshua Bengio.
\newblock Infobot: Transfer and exploration via the information bottleneck.
\newblock {\em arXiv preprint arXiv:1901.10902}, 2019.

\bibitem[GMP05]{grunwald2005advances}
Peter~D Gr{\"u}nwald, In~Jae Myung, and Mark~A Pitt.
\newblock {\em Advances in minimum description length: Theory and
  applications}.
\newblock MIT press, 2005.

\bibitem[GPAM{\etalchar{+}}14]{goodfellow2014generative}
Ian Goodfellow, Jean Pouget-Abadie, Mehdi Mirza, Bing Xu, David Warde-Farley,
  Sherjil Ozair, Aaron Courville, and Yoshua Bengio.
\newblock Generative adversarial nets.
\newblock In {\em Advances in neural information processing systems}, pages
  2672--2680, 2014.

\bibitem[Gra69]{granger1969investigating}
Clive~WJ Granger.
\newblock Investigating causal relations by econometric models and
  cross-spectral methods.
\newblock {\em Econometrica: Journal of the Econometric Society}, pages
  424--438, 1969.

\bibitem[Gra80]{granger1980testing}
Clive~WJ Granger.
\newblock Testing for causality: a personal viewpoint.
\newblock {\em Journal of Economic Dynamics and control}, 2:329--352, 1980.

\bibitem[GS{\etalchar{+}}00]{gelfand2000calculus}
Izrail~Moiseevitch Gelfand, Richard~A Silverman, et~al.
\newblock {\em Calculus of variations}.
\newblock Courier Corporation, 2000.

\bibitem[GSD{\etalchar{+}}18]{grace2018will}
Katja Grace, John Salvatier, Allan Dafoe, Baobao Zhang, and Owain Evans.
\newblock When will ai exceed human performance? evidence from ai experts.
\newblock {\em Journal of Artificial Intelligence Research}, 62:729--754, 2018.

\bibitem[GSR{\etalchar{+}}17a]{ghosh2017divide}
Dibya Ghosh, Avi Singh, Aravind Rajeswaran, Vikash Kumar, and Sergey Levine.
\newblock Divide-and-conquer reinforcement learning.
\newblock {\em arXiv preprint arXiv:1711.09874}, 2017.

\bibitem[GSR{\etalchar{+}}17b]{gilmer2017neural}
Justin Gilmer, Samuel~S Schoenholz, Patrick~F Riley, Oriol Vinyals, and
  George~E Dahl.
\newblock Neural message passing for quantum chemistry.
\newblock {\em arXiv preprint arXiv:1704.01212}, 2017.

\bibitem[GV04]{grunwald2004shannon}
Peter Grunwald and Paul Vit{\'a}nyi.
\newblock Shannon information and kolmogorov complexity.
\newblock {\em arXiv preprint cs/0410002}, 2004.

\bibitem[GVW{\etalchar{+}}16]{guttenberg2016permutation}
Nicholas Guttenberg, Nathaniel Virgo, Olaf Witkowski, Hidetoshi Aoki, and Ryota
  Kanai.
\newblock Permutation-equivariant neural networks applied to dynamics
  prediction.
\newblock {\em arXiv preprint arXiv:1612.04530}, 2016.

\bibitem[GWD14]{graves2014neural}
Alex Graves, Greg Wayne, and Ivo Danihelka.
\newblock Neural turing machines.
\newblock {\em arXiv preprint arXiv:1410.5401}, 2014.

\bibitem[GZ87]{gregory1987oxford}
Richard~L Gregory and Oliver~Louis Zangwill.
\newblock {\em The Oxford companion to the mind.}
\newblock Oxford university press, 1987.

\bibitem[HB12]{hauser2012characterization}
Alain Hauser and Peter B{\"u}hlmann.
\newblock Characterization and greedy learning of interventional markov
  equivalence classes of directed acyclic graphs.
\newblock {\em Journal of Machine Learning Research}, 13(Aug):2409--2464, 2012.

\bibitem[HD13]{harris2013pc}
Naftali Harris and Mathias Drton.
\newblock Pc algorithm for nonparanormal graphical models.
\newblock {\em The Journal of Machine Learning Research}, 14(1):3365--3383,
  2013.

\bibitem[HFW08]{Hempstalk:2008:oneclass}
Kathryn Hempstalk, Eibe Frank, and Ian~H. Witten.
\newblock One-class classification by combining density and class probability
  estimation.
\newblock In {\em Proc. of ECML-PKDD}, pages 505--519, Berlin, Heidelberg,
  2008. Springer-Verlag.

\bibitem[Hir35]{hirschfeld1935connection}
Hermann~O Hirschfeld.
\newblock A connection between correlation and contingency.
\newblock In {\em Mathematical Proceedings of the Cambridge Philosophical
  Society}, volume~31, pages 520--524. Cambridge University Press, 1935.

\bibitem[HJ94]{hiemstra1994testing}
Craig Hiemstra and Jonathan~D Jones.
\newblock Testing for linear and nonlinear granger causality in the stock
  price-volume relation.
\newblock {\em The Journal of Finance}, 49(5):1639--1664, 1994.

\bibitem[HMD15]{han2015deep}
Song Han, Huizi Mao, and William~J Dally.
\newblock Deep compression: Compressing deep neural networks with pruning,
  trained quantization and huffman coding.
\newblock {\em arXiv preprint arXiv:1510.00149}, 2015.

\bibitem[HMP{\etalchar{+}}17]{betavae}
Irina Higgins, Loic Matthey, Arka Pal, Christopher Burgess, Xavier Glorot,
  Matthew Botvinick, Shakir Mohamed, and Alexander Lerchner.
\newblock beta-vae: Learning basic visual concepts with a constrained
  variational framework.
\newblock In {\em International Conference on Learning Representations}, 2017.

\bibitem[Hor91]{hornik1991approximation}
Kurt Hornik.
\newblock Approximation capabilities of multilayer feedforward networks.
\newblock {\em Neural networks}, 4(2):251--257, 1991.

\bibitem[Hos17]{hoshen2017vain}
Yedid Hoshen.
\newblock Vain: Attentional multi-agent predictive modeling.
\newblock In {\em Advances in Neural Information Processing Systems}, pages
  2701--2711, 2017.

\bibitem[Hot36]{hotellingCCA1936}
Harold Hotelling.
\newblock Relation between two sets of variates.
\newblock {\em Biometrica}, 28(3-4):321--377, 1936.

\bibitem[Hot92]{hotelling1992relations}
Harold Hotelling.
\newblock Relations between two sets of variates.
\newblock In {\em Breakthroughs in statistics}, pages 162--190. Springer, 1992.

\bibitem[HPTD15]{han2015learning}
Song Han, Jeff Pool, John Tran, and William Dally.
\newblock Learning both weights and connections for efficient neural network.
\newblock In {\em Advances in neural information processing systems}, pages
  1135--1143, 2015.

\bibitem[HS93]{hassibi1993second}
Babak Hassibi and David~G. Stork.
\newblock Second order derivatives for network pruning: Optimal brain surgeon.
\newblock In S.~J. Hanson, J.~D. Cowan, and C.~L. Giles, editors, {\em Advances
  in Neural Information Processing Systems 5}, pages 164--171. Morgan-Kaufmann,
  1993.

\bibitem[Hut00]{hutter2000theory}
Marcus Hutter.
\newblock A theory of universal artificial intelligence based on algorithmic
  complexity.
\newblock {\em arXiv preprint cs/0004001}, 2000.

\bibitem[HvC93]{Hinton:1993:KNN:168304.168306}
Geoffrey~E. Hinton and Drew van Camp.
\newblock Keeping the neural networks simple by minimizing the description
  length of the weights.
\newblock In {\em Proceedings of the Sixth Annual Conference on Computational
  Learning Theory}, COLT '93, pages 5--13, New York, NY, USA, 1993. ACM.

\bibitem[HZRS16a]{he2016deep}
Kaiming He, Xiangyu Zhang, Shaoqing Ren, and Jian Sun.
\newblock Deep residual learning for image recognition.
\newblock In {\em Proceedings of the IEEE conference on computer vision and
  pattern recognition}, pages 770--778, 2016.

\bibitem[HZRS16b]{resnet}
Kaiming He, Xiangyu Zhang, Shaoqing Ren, and Jian Sun.
\newblock {Deep Residual Learning for Image Recognition}.
\newblock In {\em The IEEE Conference on Computer Vision and Pattern
  Recognition (CVPR)}, June 2016.

\bibitem[IMW{\etalchar{+}}18]{iten2018discovering}
Raban Iten, Tony Metger, Henrik Wilming, L{\'\i}dia Del~Rio, and Renato Renner.
\newblock Discovering physical concepts with neural networks.
\newblock {\em arXiv preprint arXiv:1807.10300}, 2018.

\bibitem[JBGW{\etalchar{+}}13]{janzing2013quantifying}
Dominik Janzing, David Balduzzi, Moritz Grosse-Wentrup, Bernhard Sch{\"o}lkopf,
  et~al.
\newblock Quantifying causal influences.
\newblock {\em The Annals of Statistics}, 41(5):2324--2358, 2013.

\bibitem[JKL{\etalchar{+}}17]{januszewski2017high}
Micha{$\l$} Januszewski, J{\"o}rgen Kornfeld, Peter~H Li, Art Pope, Tim
  Blakely, Larry Lindsey, Jeremy~B Maitin-Shepard, Mike Tyka, Winfried Denk,
  and Viren Jain.
\newblock High-precision automated reconstruction of neurons with flood-filling
  networks.
\newblock {\em bioRxiv}, page 200675, 2017.

\bibitem[KB14]{kingma2014adam}
Diederik~P Kingma and Jimmy Ba.
\newblock Adam: A method for stochastic optimization.
\newblock {\em arXiv preprint arXiv:1412.6980}, 2014.

\bibitem[KB15]{adam}
Diederik Kingma and Jimmy Ba.
\newblock Adam: A method for stochastic optimization.
\newblock In {\em International Conference on Learning Representations}, 2015.

\bibitem[KDV16]{krakovna2016increasing}
Viktoriya Krakovna and Finale Doshi-Velez.
\newblock Increasing the interpretability of recurrent neural networks using
  hidden markov models.
\newblock {\em arXiv preprint arXiv:1606.05320}, 2016.

\bibitem[KdW05]{kim2005adaptive}
Il~Yong Kim and Oliver~L de~Weck.
\newblock Adaptive weighted-sum method for bi-objective optimization: Pareto
  front generation.
\newblock {\em Structural and multidisciplinary optimization}, 29(2):149--158,
  2005.

\bibitem[KGK{\etalchar{+}}17]{kim2017discovering}
Hyeji Kim, Weihao Gao, Sreeram Kannan, Sewoong Oh, and Pramod Viswanath.
\newblock Discovering potential correlations via hypercontractivity.
\newblock In {\em Advances in Neural Information Processing Systems}, pages
  4577--4587, 2017.

\bibitem[KGV83]{kirkpatrick1983optimization}
Scott Kirkpatrick, C~Daniel Gelatt, and Mario~P Vecchi.
\newblock Optimization by simulated annealing.
\newblock {\em science}, 220(4598):671--680, 1983.

\bibitem[KH09]{cifar}
Alex Krizhevsky and Geoffrey Hinton.
\newblock Learning multiple layers of features from tiny images.
\newblock Technical report, Citeseer, 2009.

\bibitem[KNH]{cifar10}
Alex Krizhevsky, Vinod Nair, and Geoffrey Hinton.
\newblock Cifar-10 (canadian institute for advanced research).

\bibitem[KNH14]{krizhevsky2014cifar}
Alex Krizhevsky, Vinod Nair, and Geoffrey Hinton.
\newblock The cifar-10 dataset.
\newblock {\em online: http://www. cs. toronto. edu/kriz/cifar. html}, 55,
  2014.

\bibitem[Kol63]{kolmogorov1963tables}
Andrei~N Kolmogorov.
\newblock On tables of random numbers.
\newblock {\em Sankhy{\=a}: The Indian Journal of Statistics, Series A}, pages
  369--376, 1963.

\bibitem[KPR{\etalchar{+}}17]{kirkpatrick2017overcoming}
James Kirkpatrick, Razvan Pascanu, Neil Rabinowitz, Joel Veness, Guillaume
  Desjardins, Andrei~A. Rusu, Kieran Milan, John Quan, Tiago Ramalho, Agnieszka
  Grabska-Barwinska, Demis Hassabis, Claudia Clopath, Dharshan Kumaran, and
  Raia Hadsell.
\newblock Overcoming catastrophic forgetting in neural networks.
\newblock {\em Proceedings of the National Academy of Sciences},
  114(13):3521--3526, 2017.

\bibitem[KS94]{Kirkpatrick1297}
Scott Kirkpatrick and Bart Selman.
\newblock Critical behavior in the satisfiability of random boolean
  expressions.
\newblock {\em Science}, 264(5163):1297--1301, 1994.

\bibitem[KSG04]{kraskov2004estimating}
Alexander Kraskov, Harald St{\"o}gbauer, and Peter Grassberger.
\newblock Estimating mutual information.
\newblock {\em Physical review E}, 69(6):066138, 2004.

\bibitem[KTVK19]{kolchinsky2018caveats}
Artemy Kolchinsky, Brendan~D Tracey, and Steven Van~Kuyk.
\newblock Caveats for information bottleneck in deterministic scenarios.
\newblock {\em {ICLR}}, 2019.

\bibitem[Kur00]{kurzweil2000age}
Ray Kurzweil.
\newblock {\em The age of spiritual machines: When computers exceed human
  intelligence}.
\newblock Penguin, 2000.

\bibitem[KW13]{kingma2013auto}
Diederik~P Kingma and Max Welling.
\newblock Auto-encoding variational bayes.
\newblock {\em arXiv preprint arXiv:1312.6114}, 2013.

\bibitem[KY14]{kurkoski2014quantization}
Brian~M Kurkoski and Hideki Yagi.
\newblock Quantization of binary-input discrete memoryless channels.
\newblock {\em IEEE Transactions on Information Theory}, 60(8):4544--4552,
  2014.

\bibitem[KZS15]{koch2015siamese}
Gregory Koch, Richard Zemel, and Ruslan Salakhutdinov.
\newblock Siamese neural networks for one-shot image recognition.
\newblock In {\em ICML Deep Learning Workshop}, volume~2, 2015.

\bibitem[LA15]{langley2015heuristic}
Pat Langley and Adam Arvay.
\newblock Heuristic induction of rate-based process models.
\newblock In {\em Proceedings of the Twenty-Ninth AAAI Conference on Artificial
  Intelligence}, 2015.

\bibitem[LALR09]{lozano2009grouped}
Aur{\'e}lie~C Lozano, Naoki Abe, Yan Liu, and Saharon Rosset.
\newblock Grouped graphical granger modeling for gene expression regulatory
  networks discovery.
\newblock {\em Bioinformatics}, 25(12):i110--i118, 2009.

\bibitem[Lan81]{langley1981data}
Pat Langley.
\newblock Data-driven discovery of physical laws.
\newblock {\em Cognitive Science}, 5(1):31 -- 54, 1981.

\bibitem[LBBH98]{mnist}
Yann LeCun, L{\'e}on Bottou, Yoshua Bengio, and Patrick Haffner.
\newblock Gradient-based learning applied to document recognition.
\newblock {\em Proceedings of the IEEE}, 86(11):2278--2324, 1998.

\bibitem[LBH15]{lecun2015deep}
Yann LeCun, Yoshua Bengio, and Geoffrey Hinton.
\newblock Deep learning.
\newblock {\em Nature}, 521(7553):436, 2015.

\bibitem[LC10]{lecun-mnisthandwrittendigit-2010}
Yann LeCun and Corinna Cortes.
\newblock {MNIST} handwritten digit database.
\newblock 2010.

\bibitem[LCB10]{lecun2010mnist}
Yann LeCun, Corinna Cortes, and CJ~Burges.
\newblock Mnist handwritten digit database.
\newblock {\em AT\&T Labs [Online]. Available: http://yann. lecun.
  com/exdb/mnist}, 2:18, 2010.

\bibitem[LCH{\etalchar{+}}06]{lecun2006tutorial}
Yann LeCun, Sumit Chopra, Raia Hadsell, M~Ranzato, and F~Huang.
\newblock A tutorial on energy-based learning.
\newblock {\em Predicting structured data}, 1(0), 2006.

\bibitem[LD94]{lavrac1994inductive}
Nada Lavrac and Saso Dzeroski.
\newblock Inductive logic programming.
\newblock In {\em WLP}, pages 146--160. Springer, 1994.

\bibitem[LDL{\etalchar{+}}03]{Liu:2003:BTC:951949.952139}
Bing Liu, Yang Dai, Xiaoli Li, Wee~Sun Lee, and Philip~S. Yu.
\newblock Building text classifiers using positive and unlabeled examples.
\newblock In {\em Proc. of 3rd ICDM}, pages 179--, Washington, DC, USA, 2003.
  IEEE Computer Society.

\bibitem[LF92]{lenat1992thresholds}
D~Lenat and E~Feigenbaum.
\newblock On the thresholds of knowledge.
\newblock {\em Foundations of Artificial Intelligence, MIT Press, Cambridge,
  MA}, pages 185--250, 1992.

\bibitem[LGBS03]{langley2003robust}
Pat Langley, Dileep George, Stephen Bay, and Kazumi Saito.
\newblock Robust induction of process models from time-series data.
\newblock In {\em Proceedings of the Twentieth International Conference on
  International Conference on Machine Learning}, pages 432--439. AAAI Press,
  2003.

\bibitem[LH07a]{legg2007universal}
Shane Legg and Marcus Hutter.
\newblock Universal intelligence: A definition of machine intelligence.
\newblock {\em Minds and machines}, 17(4):391--444, 2007.

\bibitem[LH{\etalchar{+}}07b]{legg2007collection}
Shane Legg, Marcus Hutter, et~al.
\newblock A collection of definitions of intelligence.
\newblock {\em Frontiers in Artificial Intelligence and applications}, 157:17,
  2007.

\bibitem[LH18]{li2017learning}
Z.~{Li} and D.~{Hoiem}.
\newblock Learning without forgetting.
\newblock {\em IEEE Transactions on Pattern Analysis and Machine Intelligence},
  40(12):2935--2947, Dec 2018.

\bibitem[LJK10]{liang2010learning}
Percy Liang, Michael~I Jordan, and Dan Klein.
\newblock Learning programs: A hierarchical bayesian approach.
\newblock In {\em Proceedings of the 27th International Conference on Machine
  Learning (ICML-10)}, pages 639--646, 2010.

\bibitem[LK18]{lam2018machine}
Christopher Lam and David Kipping.
\newblock A machine learns to predict the stability of circumbinary planets.
\newblock {\em Monthly Notices of the Royal Astronomical Society},
  476(4):5692--5697, 2018.

\bibitem[LL03]{lee2003PUlearning_weightedlogreg}
{Wee Sun} Lee and Bing Liu.
\newblock Learning with positive and unlabeled examples using weighted logistic
  regression.
\newblock In {\em Proc. of 20th ICML}, volume~1, pages 448--455, 12 2003.

\bibitem[LML{\etalchar{+}}10]{lane2010survey}
Nicholas~D Lane, Emiliano Miluzzo, Hong Lu, Daniel Peebles, Tanzeem Choudhury,
  and Andrew~T Campbell.
\newblock A survey of mobile phone sensing.
\newblock {\em IEEE Communications}, 48(9), 2010.

\bibitem[LPR17]{lopez2017gradient}
David Lopez-Paz and Marc\textquotesingle~Aurelio Ranzato.
\newblock Gradient episodic memory for continual learning.
\newblock In I.~Guyon, U.~V. Luxburg, S.~Bengio, H.~Wallach, R.~Fergus,
  S.~Vishwanathan, and R.~Garnett, editors, {\em Advances in Neural Information
  Processing Systems 30}, pages 6467--6476. Curran Associates, Inc., 2017.

\bibitem[LPZ08]{lizier2008local}
Joseph~T Lizier, Mikhail Prokopenko, and Albert~Y Zomaya.
\newblock Local information transfer as a spatiotemporal filter for complex
  systems.
\newblock {\em Physical Review E}, 77(2):026110, 2008.

\bibitem[LT16a]{lin2016criticality}
Henry~W Lin and Max Tegmark.
\newblock Criticality in formal languages and statistical physics.
\newblock {\em arXiv preprint arXiv:1606.06737}, 2016.

\bibitem[LT16b]{liu2016classification}
Tongliang Liu and Dacheng Tao.
\newblock Classification with noisy labels by importance reweighting.
\newblock {\em IEEE Transactions on pattern analysis and machine intelligence},
  38(3):447--461, 2016.

\bibitem[LT16c]{Liu:2016:CNL:2914183.2914328}
Tongliang Liu and Dacheng Tao.
\newblock Classification with noisy labels by importance reweighting.
\newblock {\em IEEE Trans. Pattern Anal. Mach. Intell.}, 38(3):447--461, March
  2016.

\bibitem[LUTG17]{lake2017building}
Brenden~M. Lake, Tomer~D. Ullman, Joshua~B. Tenenbaum, and Samuel~J. Gershman.
\newblock Building machines that learn and think like people.
\newblock {\em Behavioral and Brain Sciences}, 40:e253, 2017.

\bibitem[LZ89]{langley1989data}
Pat Langley and Jan~M. Zytkow.
\newblock Data-driven approaches to empirical discovery.
\newblock {\em Artificial Intelligence}, 40(1):283 -- 312, 1989.

\bibitem[LZCL17]{li2017meta}
Zhenguo Li, Fengwei Zhou, Fei Chen, and Hang Li.
\newblock Meta-sgd: Learning to learn quickly for few shot learning.
\newblock {\em arXiv preprint arXiv:1707.09835}, 2017.

\bibitem[MB16]{muller2016future}
Vincent~C M{\"u}ller and Nick Bostrom.
\newblock Future progress in artificial intelligence: A survey of expert
  opinion.
\newblock In {\em Fundamental issues of artificial intelligence}, pages
  555--572. Springer, 2016.

\bibitem[MCM86]{Michalski:1986:MLA:21934}
S~Ryszard Michalski, G~Jaime Carbonell, and M~Tom Mitchell.
\newblock {\em ML an AI Approach}.
\newblock Morgan Kaufmann Publishers Inc., San Francisco, CA, USA, 1986.

\bibitem[MJV{\etalchar{+}}12]{Menon2012PredictingAP}
Aditya~Krishna Menon, Xiaoqian Jiang, Shankar Vembu, Charles Elkan, and Lucila
  Ohno-Machado.
\newblock Predicting accurate probabilities with a ranking loss.
\newblock {\em CoRR}, abs/1206.4661, 2012.

\bibitem[MKH93]{moya_1993_oneclass}
M.~M. {Moya}, M.~W. {Koch}, and L.~D. {Hostetler}.
\newblock {One-class classifier networks for target recognition applications}.
\newblock {\em NASA STI/Recon Technical Report N}, 93, 1993.

\bibitem[MKS{\etalchar{+}}15]{mnih2015human}
Volodymyr Mnih, Koray Kavukcuoglu, David Silver, Andrei~A Rusu, Joel Veness,
  Marc~G Bellemare, Alex Graves, Martin Riedmiller, Andreas~K Fidjeland, Georg
  Ostrovski, et~al.
\newblock Human-level control through deep reinforcement learning.
\newblock {\em Nature}, 518(7540):529, 2015.

\bibitem[MLC16]{mengistu2016evolvability}
Henok Mengistu, Joel Lehman, and Jeff Clune.
\newblock Evolvability search:directly selecting for evolvability in order to
  study and produce it.
\newblock In {\em Proceedings of the Genetic and Evolutionary Computation
  Conference 2016}, pages 141--148. ACM, 2016.

\bibitem[MPS08a]{marinazzo2008kernel2}
Daniele Marinazzo, Mario Pellicoro, and Sebastiano Stramaglia.
\newblock Kernel-granger causality and the analysis of dynamical networks.
\newblock {\em Physical review E}, 77(5):056215, 2008.

\bibitem[MPS08b]{marinazzo2008kernel}
Daniele Marinazzo, Mario Pellicoro, and Sebastiano Stramaglia.
\newblock Kernel method for nonlinear granger causality.
\newblock {\em Physical review letters}, 100(14):144103, 2008.

\bibitem[Mug91]{Muggleton1991}
Stephen Muggleton.
\newblock Inductive logic programming.
\newblock {\em New Generation Computing}, 8(4):295--318, Feb 1991.

\bibitem[MV14]{Mordelet:2014:BSL:2565612.2565683}
F.~Mordelet and J.~P. Vert.
\newblock A bagging svm to learn from positive and unlabeled examples.
\newblock {\em Pattern Recogn. Lett.}, 37:201--209, February 2014.

\bibitem[MVDGB08]{meier2008group}
Lukas Meier, Sara Van De~Geer, and Peter B{\"u}hlmann.
\newblock The group lasso for logistic regression.
\newblock {\em Journal of the Royal Statistical Society: Series B (Statistical
  Methodology)}, 70(1):53--71, 2008.

\bibitem[MY02]{Manevitz:2002:OSD:944790.944808}
Larry~M. Manevitz and Malik Yousef.
\newblock One-class svms for document classification.
\newblock {\em JMLR}, 2:139--154, March 2002.

\bibitem[Nak99]{nakashima1999ai}
Hideyuki Nakashima.
\newblock Ai as complex information processing.
\newblock {\em Minds and machines}, 9(1):57--80, 1999.

\bibitem[NCB08]{neves2008synaptic}
Guilherme Neves, Sam~F Cooke, and Tim~VP Bliss.
\newblock Synaptic plasticity, memory and the hippocampus: a neural network
  approach to causality.
\newblock {\em Nature Reviews Neuroscience}, 9(1):65, 2008.

\bibitem[NDRT13a]{natarajan2013learning}
Nagarajan Natarajan, Inderjit~S Dhillon, Pradeep~K Ravikumar, and Ambuj Tewari.
\newblock Learning with noisy labels.
\newblock In {\em Advances in neural information processing systems}, pages
  1196--1204, 2013.

\bibitem[NDRT13b]{NIPS2013_5073}
Nagarajan Natarajan, Inderjit~S Dhillon, Pradeep~K Ravikumar, and Ambuj Tewari.
\newblock Learning with noisy labels.
\newblock In {\em Adv. in NIPS 26}, pages 1196--1204. Curran Associates, Inc.,
  2013.

\bibitem[NG00]{Nigam00understandingthe}
Kamal Nigam and Rayid Ghani.
\newblock Understanding the behavior of co-training.
\newblock In {\em KDD Workshop}, 2000.

\bibitem[NHM{\etalchar{+}}18]{nandy2018high}
Preetam Nandy, Alain Hauser, Marloes~H Maathuis, et~al.
\newblock High-dimensional consistency in score-based and hybrid structure
  learning.
\newblock {\em The Annals of Statistics}, 46(6A):3151--3183, 2018.

\bibitem[NLBT17]{nguyen2017variational}
Cuong~V Nguyen, Yingzhen Li, Thang~D Bui, and Richard~E Turner.
\newblock Variational continual learning.
\newblock {\em arXiv preprint arXiv:1710.10628}, 2017.

\bibitem[NM92]{287172}
D.~K. Naik and R.~J. Mammone.
\newblock Meta-neural networks that learn by learning.
\newblock In {\em [Proceedings 1992] IJCNN International Joint Conference on
  Neural Networks}, volume~1, pages 437--442 vol.1, Jun 1992.

\bibitem[NOPF10]{Nettleton2010}
David~F. Nettleton, Albert Orriols-Puig, and Albert Fornells.
\newblock A study of the effect of different types of noise on the precision of
  supervised learning techniques.
\newblock {\em Artificial Intelligence Review}, 33(4):275--306, 2010.

\bibitem[NWC17]{northcutt2017learning}
Curtis~G Northcutt, Tailin Wu, and Isaac~L Chuang.
\newblock Learning with confident examples: Rank pruning for robust
  classification with noisy labels.
\newblock {\em arXiv preprint arXiv:1705.01936}, 2017.

\bibitem[OLV18]{oord2018representation}
Aaron van~den Oord, Yazhe Li, and Oriol Vinyals.
\newblock Representation learning with contrastive predictive coding.
\newblock {\em arXiv preprint arXiv:1807.03748}, 2018.

\bibitem[P{\etalchar{+}}09]{pearl2009causal}
Judea Pearl et~al.
\newblock Causal inference in statistics: An overview.
\newblock {\em Statistics surveys}, 3:96--146, 2009.

\bibitem[Pap85]{papoulis1985probability}
A~Papoulis.
\newblock Probability, random variables and stochastic processes.
\newblock 1985.

\bibitem[PCI10]{mech_turk_quality}
Gabriele Paolacci, Jesse Chandler, and Panagiotis~G. Ipeirotis.
\newblock Running experiments on amazon mechanical turk.
\newblock {\em Judgment and Decision Making}, 5(5):411--419, 2010.

\bibitem[Pea01]{pearsonPCA1901}
Karl Pearson.
\newblock Liii. on lines and planes of closest fit to systems of points in
  space.
\newblock {\em The London, Edinburgh, and Dublin Philosophical Magazine and
  Journal of Science}, 2(11):559--572, 1901.

\bibitem[Pea02]{pearl2002causality}
Judea Pearl.
\newblock Causality: models, reasoning, and inference.
\newblock {\em IIE Transactions}, 34(6):583--589, 2002.

\bibitem[Pea09]{pearl2009causality}
Judea Pearl.
\newblock {\em Causality}.
\newblock Cambridge university press, 2009.

\bibitem[Phy]{Physionet}
PhysioNet.
\newblock Physionet data bank.

\bibitem[Pia05]{piaget2005psychology}
Jean Piaget.
\newblock {\em The psychology of intelligence}.
\newblock Routledge, 2005.

\bibitem[PJS17]{peters2017elements}
Jonas Peters, Dominik Janzing, and Bernhard Sch{\"o}lkopf.
\newblock {\em Elements of causal inference: foundations and learning
  algorithms}.
\newblock MIT press, 2017.

\bibitem[PKT{\etalchar{+}}18]{peng2018variational}
Xue~Bin Peng, Angjoo Kanazawa, Sam Toyer, Pieter Abbeel, and Sergey Levine.
\newblock Variational discriminator bottleneck: Improving imitation learning,
  inverse rl, and gans by constraining information flow.
\newblock {\em arXiv preprint arXiv:1810.00821}, 2018.

\bibitem[PMS{\etalchar{+}}16]{parisotto2016neuro}
Emilio Parisotto, Abdel-rahman Mohamed, Rishabh Singh, Lihong Li, Dengyong
  Zhou, and Pushmeet Kohli.
\newblock Neuro-symbolic program synthesis.
\newblock {\em arXiv preprint arXiv:1611.01855}, 2016.

\bibitem[PSJ{\etalchar{+}}18]{peurifoy2018nanophotonic}
John Peurifoy, Yichen Shen, Li~Jing, Yi~Yang, Fidel Cano-Renteria, Brendan~G
  DeLacy, John~D Joannopoulos, Max Tegmark, and Marin Solja{\v{c}}i{\'c}.
\newblock Nanophotonic particle simulation and inverse design using artificial
  neural networks.
\newblock {\em Science advances}, 4(6):eaar4206, 2018.

\bibitem[PSST{\etalchar{+}}99]{oneclasssvm1999}
John Platt, Bernhard Schölkopf, John Shawe-Taylor, Alex~J. Smola, and
  Robert~C. Williamson.
\newblock Estimating support of a high dimensional distribution.
\newblock Technical report, MSR, 1999.

\bibitem[PW17]{polyanskiy2017strong}
Yury Polyanskiy and Yihong Wu.
\newblock Strong data-processing inequalities for channels and bayesian
  networks.
\newblock In {\em Convexity and Concentration}, pages 211--249. Springer, 2017.

\bibitem[QKKG02]{quiroga2002performance}
R~Quian Quiroga, A~Kraskov, T~Kreuz, and Peter Grassberger.
\newblock Performance of different synchronization measures in real data: a
  case study on electroencephalographic signals.
\newblock {\em Physical Review E}, 65(4):041903, 2002.

\bibitem[Qui]{ratEEG}
Rodrigo~Quian Quiroga.
\newblock The dataset can be downloaded from.

\bibitem[RDF15]{reed2015neural}
Scott Reed and Nando De~Freitas.
\newblock Neural programmer-interpreters.
\newblock {\em arXiv preprint arXiv:1511.06279}, 2015.

\bibitem[RDT15]{russell2015research}
Stuart Russell, Daniel Dewey, and Max Tegmark.
\newblock Research priorities for robust and beneficial artificial
  intelligence.
\newblock {\em Ai Magazine}, 36(4):105--114, 2015.

\bibitem[R{\'e}n59]{renyi1959measures}
Alfr{\'e}d R{\'e}nyi.
\newblock On measures of dependence.
\newblock {\em Acta mathematica hungarica}, 10(3-4):441--451, 1959.

\bibitem[RI96]{raviv_heart}
Yuval Raviv and Nathan Intrator.
\newblock Bootstrapping with noise: An effective regularization technique.
\newblock {\em Connection Science}, 8(3-4):355--372, 1996.

\bibitem[Ris78]{rissanen1978modeling}
J.~Rissanen.
\newblock Modeling by shortest data description.
\newblock {\em Automatica}, 14(5):465 -- 471, 1978.

\bibitem[Ris83]{rissanen1983universal}
Jorma Rissanen.
\newblock A universal prior for integers and estimation by minimum description
  length.
\newblock {\em The Annals of Statistics}, 11(2):416--431, 1983.

\bibitem[RJJ{\etalchar{+}}18]{alphafold}
R.Evans, J.Jumper, J.Kirkpatrick, L.Sifre, T.F.G.Green, C.Qin, A.Zidek,
  A.Nelson, A.Bridgland, H.Penedones, S.Petersen, K.Simonyan, D.T.Jones,
  K.Kavukcuoglu, D.Hassabis, and A.W.Senior.
\newblock De novo structure prediction with deep-learning based scoring, 2018.

\bibitem[RL16]{ravi2016optimization}
Sachin Ravi and Hugo Larochelle.
\newblock Optimization as a model for few-shot learning.
\newblock 2016.

\bibitem[RLA{\etalchar{+}}15]{noisy_boostrapping_google}
Scott~E. Reed, Honglak Lee, Dragomir Anguelov, Christian Szegedy, Dumitru
  Erhan, and Andrew Rabinovich.
\newblock Training deep neural networks on noisy labels with bootstrapping.
\newblock In {\em ICLR}, 2015.

\bibitem[RMS{\etalchar{+}}17]{real2017large}
Esteban Real, Sherry Moore, Andrew Selle, Saurabh Saxena, Yutaka~Leon Suematsu,
  Jie Tan, Quoc Le, and Alex Kurakin.
\newblock Large-scale evolution of image classifiers.
\newblock {\em arXiv preprint arXiv:1703.01041}, 2017.

\bibitem[RR12]{rey2012meta}
M{\'e}lanie Rey and Volker Roth.
\newblock Meta-gaussian information bottleneck.
\newblock In {\em Advances in Neural Information Processing Systems}, pages
  1916--1924, 2012.

\bibitem[RT18]{rolnick2018the}
David Rolnick and Max Tegmark.
\newblock The power of deeper networks for expressing natural functions.
\newblock In {\em International Conference on Learning Representations}, 2018.

\bibitem[Rus19]{russell2019humancompatible}
Stuart Russell.
\newblock {\em Human Compatible: Artificial Intelligence and the Problem of
  Control}.
\newblock Viking, 2019.

\bibitem[RV18]{rezende2018taming}
Danilo~Jimenez Rezende and Fabio Viola.
\newblock Taming {VAE}s.
\newblock {\em arXiv preprint arXiv:1810.00597}, 2018.

\bibitem[SBB15]{seth2015granger}
Anil~K Seth, Adam~B Barrett, and Lionel Barnett.
\newblock Granger causality analysis in neuroscience and neuroimaging.
\newblock {\em Journal of Neuroscience}, 35(8):3293--3297, 2015.

\bibitem[SBB{\etalchar{+}}16]{pmlr-v48-santoro16}
Adam Santoro, Sergey Bartunov, Matthew Botvinick, Daan Wierstra, and Timothy
  Lillicrap.
\newblock Meta-learning with memory-augmented neural networks.
\newblock In Maria~Florina Balcan and Kilian~Q. Weinberger, editors, {\em
  Proceedings of The 33rd International Conference on Machine Learning},
  volume~48 of {\em Proceedings of Machine Learning Research}, pages
  1842--1850, New York, New York, USA, 20--22 Jun 2016. PMLR.

\bibitem[SBH13]{ScottBH13}
Clayton Scott, Gilles Blanchard, and Gregory Handy.
\newblock Classification with asymmetric label noise: Consistency and maximal
  denoising.
\newblock In {\em COLT}, pages 489--511, 2013.

\bibitem[Sch87]{schmidhuber1987evolutionary}
J{\"u}rgen Schmidhuber.
\newblock {\em Evolutionary principles in self-referential learning, or on
  learning how to learn: the meta-meta-... hook}.
\newblock PhD thesis, Technische Universit{\"a}t M{\"u}nchen, 1987.

\bibitem[Sch91]{schank1991s}
Roger~C Schank.
\newblock Where's the ai?
\newblock {\em AI magazine}, 12(4):38--38, 1991.

\bibitem[Sch92]{schrodinger1992life}
Erwin Schr{\"o}dinger.
\newblock {\em What is life?: With mind and matter and autobiographical
  sketches}.
\newblock Cambridge University Press, 1992.

\bibitem[Sch00]{schreiber2000measuring}
Thomas Schreiber.
\newblock Measuring information transfer.
\newblock {\em Physical review letters}, 85(2):461, 2000.

\bibitem[SCHU17]{scardapane2017group}
Simone Scardapane, Danilo Comminiello, Amir Hussain, and Aurelio Uncini.
\newblock Group sparse regularization for deep neural networks.
\newblock {\em Neurocomputing}, 241:81--89, 2017.

\bibitem[Sco15]{scott2015rate}
Clayton Scott.
\newblock A rate of convergence for mixture proportion estimation, with
  application to learning from noisy labels.
\newblock {\em JMLR}, 38:838--846, 2015.

\bibitem[SGL{\etalchar{+}}19]{sharma2019dynamics}
Archit Sharma, Shixiang Gu, Sergey Levine, Vikash Kumar, and Karol Hausman.
\newblock Dynamics-aware unsupervised discovery of skills.
\newblock {\em arXiv preprint arXiv:1907.01657}, 2019.

\bibitem[SGS{\etalchar{+}}00]{spirtes2000causation}
Peter Spirtes, Clark~N Glymour, Richard Scheines, David Heckerman, Christopher
  Meek, Gregory Cooper, and Thomas Richardson.
\newblock {\em Causation, prediction, and search}.
\newblock MIT press, 2000.

\bibitem[Sha48a]{shannon1948mathematical}
Claude~Elwood Shannon.
\newblock A mathematical theory of communication.
\newblock {\em Bell system technical journal}, 27(3):379--423, 1948.

\bibitem[Sha48b]{shannon}
Claude~Elwood Shannon.
\newblock A {M}athematical {T}heory of {C}ommunication.
\newblock {\em The Bell System Technical Journal}, 27:379--423, 1948.

\bibitem[SHS01]{suzuki2001simple}
Kenji Suzuki, Isao Horiba, and Noboru Sugie.
\newblock A simple neural network pruning algorithm with application to filter
  synthesis.
\newblock {\em Neural Processing Letters}, 13(1):43--53, 2001.

\bibitem[Sim03]{simonton2003interview}
DK~Simonton.
\newblock An interview with dr. simonton.
\newblock {\em Human intelligence: Historical influences, current
  controversies, teaching resources. http://www. indiana. edu/~ intell}, 2003.

\bibitem[SKCK17]{pxpp}
Tim Salimans, Andrej Karpathy, Xi~Chen, and Diederik~P. Kingma.
\newblock {PixelCNN++: A PixelCNN Implementation with Discretized Logistic
  Mixture Likelihood and Other Modifications}.
\newblock In {\em ICLR}, 2017.

\bibitem[SL09]{schmidt2009distilling}
Michael Schmidt and Hod Lipson.
\newblock Distilling free-form natural laws from experimental data.
\newblock {\em science}, 324(5923):81--85, 2009.

\bibitem[sl16]{logreg_sklearn}
scikit learn.
\newblock {\em LogisticRegression Class at scikit-learn}, 2016.

\bibitem[Sol64]{solomonoff1964formal}
Ray~J Solomonoff.
\newblock A formal theory of inductive inference. part i.
\newblock {\em Information and control}, 7(1):1--22, 1964.

\bibitem[SQL12]{sindhwani2012scalable}
Vikas Sindhwani, Minh~Ha Quang, and Aur{\'e}lie~C Lozano.
\newblock Scalable matrix-valued kernel learning for high-dimensional nonlinear
  multivariate regression and granger causality.
\newblock {\em arXiv preprint arXiv:1210.4792}, 2012.

\bibitem[SS17a]{strouse2017deterministic}
DJ~Strouse and David~J Schwab.
\newblock The deterministic information bottleneck.
\newblock {\em Neural computation}, 29(6):1611--1630, 2017.

\bibitem[SS17b]{strouse2017information}
DJ~Strouse and David~J Schwab.
\newblock The information bottleneck and geometric clustering.
\newblock {\em arXiv preprint arXiv:1712.09657}, 2017.

\bibitem[SS19]{strouse2019information}
DJ~Strouse and David~J Schwab.
\newblock The information bottleneck and geometric clustering.
\newblock {\em Neural computation}, 31(3):596--612, 2019.

\bibitem[SSK12]{Sugiyama:2012:DRE:2181148}
Masashi Sugiyama, Taiji Suzuki, and Takafumi Kanamori.
\newblock {\em Density Ratio Estimation in ML}.
\newblock Cambridge University Press, New York, NY, USA, 1st edition, 2012.

\bibitem[SST10]{shamir2010learning}
Ohad Shamir, Sivan Sabato, and Naftali Tishby.
\newblock Learning and generalization with the information bottleneck.
\newblock {\em Theoretical Computer Science}, 411(29-30):2696--2711, 2010.

\bibitem[SsWF15]{sukhbaatar2015end}
Sainbayar Sukhbaatar, arthur szlam, Jason Weston, and Rob Fergus.
\newblock End-to-end memory networks.
\newblock In C.~Cortes, N.~D. Lawrence, D.~D. Lee, M.~Sugiyama, and R.~Garnett,
  editors, {\em Advances in Neural Information Processing Systems 28}, pages
  2440--2448. Curran Associates, Inc., 2015.

\bibitem[SSZ17]{snell2017prototypical}
Jake Snell, Kevin Swersky, and Richard Zemel.
\newblock Prototypical networks for few-shot learning.
\newblock In {\em Advances in Neural Information Processing Systems}, pages
  4077--4087, 2017.

\bibitem[Sti17]{still2017thermodynamic}
Susanne Still.
\newblock Thermodynamic cost and benefit of data representations.
\newblock {\em arXiv preprint arXiv:1705.00612}, 2017.

\bibitem[SW89]{stock1989interpreting}
James~H Stock and Mark~W Watson.
\newblock Interpreting the evidence on money-income causality.
\newblock {\em Journal of Econometrics}, 40(1):161--181, 1989.

\bibitem[TCF{\etalchar{+}}18]{tank2018neural}
Alex Tank, Ian Covert, Nicholas Foti, Ali Shojaie, and Emily Fox.
\newblock Neural granger causality for nonlinear time series.
\newblock {\em arXiv preprint arXiv:1802.05842}, 2018.

\bibitem[Teg17]{tegmark2017life}
Max Tegmark.
\newblock {\em Life 3.0: Being human in the age of artificial intelligence}.
\newblock Knopf, 2017.

\bibitem[Teg19]{tegmark2019optimal}
Max Tegmark.
\newblock Optimal latent representations: Distilling mutual information into
  principal pairs.
\newblock {\em arXiv preprint arXiv:1902.03364}, 2019.

\bibitem[TGM{\etalchar{+}}18]{thompson2018causal}
Jayne Thompson, Andrew~JP Garner, John~R Mahoney, James~P Crutchfield, Vlatko
  Vedral, and Mile Gu.
\newblock Causal asymmetry in a quantum world.
\newblock {\em Physical Review X}, 8(3):031013, 2018.

\bibitem[Tib96]{tibshirani1996regression}
Robert Tibshirani.
\newblock Regression shrinkage and selection via the lasso.
\newblock {\em Journal of the Royal Statistical Society. Series B
  (Methodological)}, pages 267--288, 1996.

\bibitem[Tis18]{tishbyinfo}
Naftali Tishby.
\newblock {Lecture: the information theory of deep neural networks: the
  statistical physics aspects}.
\newblock
  \url{https://www.perimeterinstitute.ca/videos/information-theory-deep-neural-networks-statistical-physics-aspects/},
  2018.

\bibitem[TMBS19]{tan2019renormalization}
Andrew Tan, Leenoy Meshulam, William Bialek, and David Schwab.
\newblock The renormalization group and information bottleneck: a unified
  framework.
\newblock {\em Bulletin of the American Physical Society}, 2019.

\bibitem[TP12]{thrun2012learning}
Sebastian Thrun and Lorien Pratt.
\newblock {\em Learning to learn}.
\newblock Springer Science \& Business Media, 2012.

\bibitem[TPB00]{tishby2000information}
Naftali Tishby, Fernando~C Pereira, and William Bialek.
\newblock The information bottleneck method.
\newblock {\em arXiv preprint physics/0004057}, 2000.

\bibitem[TW20]{tegmark2019pareto}
Max Tegmark and Tailin Wu.
\newblock Pareto-optimal data compression for binary classification tasks.
\newblock {\em Entropy}, 22(1):7, 2020.

\bibitem[UT19]{udrescu2019ai}
Silviu-Marian Udrescu and Max Tegmark.
\newblock Ai feynman: a physics-inspired method for symbolic regression.
\newblock {\em arXiv preprint arXiv:1905.11481}, 2019.

\bibitem[VBL{\etalchar{+}}16]{NIPS2016_6385}
Oriol Vinyals, Charles Blundell, Tim Lillicrap, koray kavukcuoglu, and Daan
  Wierstra.
\newblock Matching networks for one shot learning.
\newblock In D.~D. Lee, M.~Sugiyama, U.~V. Luxburg, I.~Guyon, and R.~Garnett,
  editors, {\em Advances in Neural Information Processing Systems 29}, pages
  3630--3638. Curran Associates, Inc., 2016.

\bibitem[vdOKE{\etalchar{+}}16]{pixelcnn}
Aaron van~den Oord, Nal Kalchbrenner, Lasse Espeholt, Koray Kavukcuoglu, Oriol
  Vinyals, and Alex Graves.
\newblock {Conditional Image Generation with PixelCNN Decoders}.
\newblock In D.~D. Lee, M.~Sugiyama, U.~V. Luxburg, I.~Guyon, and R.~Garnett,
  editors, {\em Advances in Neural Information Processing Systems 29}, pages
  4790--4798. Curran Associates, Inc., 2016.

\bibitem[VLBM08]{vincent2008extracting}
Pascal Vincent, Hugo Larochelle, Yoshua Bengio, and Pierre-Antoine Manzagol.
\newblock Extracting and composing robust features with denoising autoencoders.
\newblock In {\em Proceedings of the 25th international conference on Machine
  learning}, pages 1096--1103. ACM, 2008.

\bibitem[VNLH17]{van2017learning}
Evert~PL Van~Nieuwenburg, Ye-Hua Liu, and Sebastian~D Huber.
\newblock Learning phase transitions by confusion.
\newblock {\em Nature Physics}, 13(5):435, 2017.

\bibitem[vSCGS18]{van2018relational}
Sjoerd van Steenkiste, Michael Chang, Klaus Greff, and J{\"u}rgen Schmidhuber.
\newblock Relational neural expectation maximization: Unsupervised discovery of
  objects and their interactions.
\newblock {\em arXiv preprint arXiv:1802.10353}, 2018.

\bibitem[VV04]{vereshchagin2004kolmogorov}
Nikolai~K Vereshchagin and Paul~MB Vit{\'a}nyi.
\newblock Kolmogorov's structure functions and model selection.
\newblock {\em IEEE Transactions on Information Theory}, 50(12):3265--3290,
  2004.

\bibitem[Wan95]{wang1995working}
Pei Wang.
\newblock On the working definition of intelligence.
\newblock {\em Center for Research on Concepts and Cognition CRCC, Indiana
  University}, 1995.

\bibitem[Wan16]{wang2016discovering}
Lei Wang.
\newblock Discovering phase transitions with unsupervised learning.
\newblock {\em Physical Review B}, 94(19):195105, 2016.

\bibitem[WBSK20]{wu2020discovering}
Tailin Wu, Thomas Breuel, Michael Skuhersky, and Jan Kautz.
\newblock Discovering nonlinear relations with minimum predictive information
  regularization.
\newblock {\em arXiv preprint arXiv:2001.01885}, 2020.

\bibitem[WC09]{white2009settable}
Halbert White and Karim Chalak.
\newblock Settable systems: an extension of pearl's causal model with
  optimization, equilibrium, and learning.
\newblock {\em Journal of Machine Learning Research}, 10(Aug):1759--1799, 2009.

\bibitem[WCL11]{white2011linking}
Halbert White, Karim Chalak, and Xun Lu.
\newblock Linking granger causality and the pearl causal model with settable
  systems.
\newblock In {\em NIPS Mini-Symposium on Causality in Time Series}, pages
  1--29, 2011.

\bibitem[WF20]{Wu2020Phase}
Tailin Wu and Ian Fischer.
\newblock Phase transitions for the information bottleneck in representation
  learning.
\newblock In {\em International Conference on Learning Representations}, 2020.

\bibitem[WFCT19a]{wu2019learnability}
Tailin Wu, Ian Fischer, Isaac Chuang, and Max Tegmark.
\newblock Learnability for the information bottleneck.
\newblock {\em arXiv preprint arXiv:1907.07331}, 2019.

\bibitem[WFCT19b]{wu2019learnabilityEntropy}
Tailin Wu, Ian Fischer, Isaac Chuang, and Max Tegmark.
\newblock Learnability for the information bottleneck.
\newblock {\em Entropy}, 21(3):924, 2019.

\bibitem[WKNT{\etalchar{+}}16]{wang2016learning}
Jane~X Wang, Zeb Kurth-Nelson, Dhruva Tirumala, Hubert Soyer, Joel~Z Leibo,
  Remi Munos, Charles Blundell, Dharshan Kumaran, and Matt Botvinick.
\newblock Learning to reinforcement learn.
\newblock {\em arXiv preprint arXiv:1611.05763}, 2016.

\bibitem[WL10]{white2010granger}
Halbert White and Xun Lu.
\newblock Granger causality and dynamic structural systems.
\newblock {\em Journal of Financial Econometrics}, 8(2):193--243, 2010.

\bibitem[WLK{\etalchar{+}}17]{NIPS2017_6620}
Jiajun Wu, Erika Lu, Pushmeet Kohli, Bill Freeman, and Josh Tenenbaum.
\newblock Learning to see physics via visual de-animation.
\newblock In I.~Guyon, U.~V. Luxburg, S.~Bengio, H.~Wallach, R.~Fergus,
  S.~Vishwanathan, and R.~Garnett, editors, {\em Advances in Neural Information
  Processing Systems 30}, pages 153--164. Curran Associates, Inc., 2017.

\bibitem[WPCT18]{wu2018meta}
Tailin Wu, John Peurifoy, Isaac~L Chuang, and Max Tegmark.
\newblock Meta-learning autoencoders for few-shot prediction.
\newblock {\em arXiv preprint arXiv:1807.09912}, 2018.
\newblock URL \url{https://arxiv.org/abs/1807.09912}.

\bibitem[WT19]{wu2018toward}
Tailin Wu and Max Tegmark.
\newblock Toward an artificial intelligence physicist for unsupervised
  learning.
\newblock {\em Phys. Rev. E}, 100:033311, Sep 2019.

\bibitem[WZW{\etalchar{+}}17a]{watters2017visual}
Nicholas Watters, Daniel Zoran, Theophane Weber, Peter Battaglia, Razvan
  Pascanu, and Andrea Tacchetti.
\newblock Visual interaction networks: Learning a physics simulator from video.
\newblock In I.~Guyon, U.~V. Luxburg, S.~Bengio, H.~Wallach, R.~Fergus,
  S.~Vishwanathan, and R.~Garnett, editors, {\em Advances in Neural Information
  Processing Systems 30}, pages 4539--4547. Curran Associates, Inc., 2017.

\bibitem[WZW{\etalchar{+}}17b]{NIPS2017_7040}
Nicholas Watters, Daniel Zoran, Theophane Weber, Peter Battaglia, Razvan
  Pascanu, and Andrea Tacchetti.
\newblock Visual interaction networks: Learning a physics simulator from video.
\newblock In I.~Guyon, U.~V. Luxburg, S.~Bengio, H.~Wallach, R.~Fergus,
  S.~Vishwanathan, and R.~Garnett, editors, {\em Advances in Neural Information
  Processing Systems 30}, pages 4539--4547. Curran Associates, Inc., 2017.

\bibitem[XRV17]{xiao2017fashion}
Han Xiao, Kashif Rasul, and Roland Vollgraf.
\newblock Fashion-mnist: a novel image dataset for benchmarking machine
  learning algorithms.
\newblock {\em arXiv preprint arXiv:1708.07747}, 2017.

\bibitem[XXY{\etalchar{+}}15a]{xiao2015learning}
Tong Xiao, Tian Xia, Yi~Yang, Chang Huang, and Xiaogang Wang.
\newblock Learning from massive noisy labeled data for image classification.
\newblock In {\em Proceedings of the IEEE Conference on Computer Vision and
  Pattern Recognition}, pages 2691--2699, 2015.

\bibitem[XXY{\etalchar{+}}15b]{Xiao2015LearningFM}
Tong Xiao, Tian Xia, Yi~Yang, Chang Huang, and Xiaogang Wang.
\newblock Learning from massive noisy labeled data for image classification.
\newblock In {\em CVPR}, 2015.

\bibitem[Yai19]{yaida2018fluctuationdissipation}
Sho Yaida.
\newblock Fluctuation-dissipation relations for stochastic gradient descent.
\newblock In {\em International Conference on Learning Representations}, 2019.

\bibitem[YMJ{\etalchar{+}}12]{ICML2012Yang_127}
Tianbao Yang, Mehrdad Mahdavi, Rong Jin, Lijun Zhang, and Yang Zhou.
\newblock Multiple kernel learning from noisy labels by stochastic programming.
\newblock In {\em Proc. of 29th ICML}, pages 233--240, New York, NY, USA, 2012.
  ACM.

\bibitem[YSB{\etalchar{+}}18]{yildirim2018neurocomputational}
Ilker Yildirim, Kevin~A Smith, Mario Belledonne, Jiajun Wu, and Joshua~B
  Tenenbaum.
\newblock Neurocomputational modeling of human physical scene understanding.
\newblock In {\em 2nd Conference on Cognitive Computational Neuroscience},
  2018.

\bibitem[Zeg15]{fisherproperty}
Pablo Zegers.
\newblock Fisher information properties.
\newblock {\em Entropy}, 17(7):4918--4939, 2015.

\bibitem[ZH05]{zou2005regularization}
Hui Zou and Trevor Hastie.
\newblock Regularization and variable selection via the elastic net.
\newblock {\em Journal of the Royal Statistical Society: Series B (Statistical
  Methodology)}, 67(2):301--320, 2005.

\bibitem[ZK16]{wideresnet}
S.~{Zagoruyko} and N.~{Komodakis}.
\newblock {Wide Residual Networks}.
\newblock {\em arXiv: 1605.07146}, 2016.

\bibitem[ZL16]{zoph2016neural}
Barret Zoph and Quoc~V Le.
\newblock Neural architecture search with reinforcement learning.
\newblock {\em arXiv preprint arXiv:1611.01578}, 2016.

\bibitem[ZLWT18]{zheng2018unsupervised}
David Zheng, Vinson Luo, Jiajun Wu, and Joshua~B Tenenbaum.
\newblock Unsupervised learning of latent physical properties using
  perception-prediction networks.
\newblock {\em arXiv preprint arXiv:1807.09244}, 2018.

\end{thebibliography}
\bibliographystyle{alpha}

\end{singlespace}

\end{document}